%% file: main.tex
\documentclass[11pt]{article}

\usepackage[round, compress]{natbib}
\usepackage[utf8]{inputenc} %
\usepackage[T1]{fontenc}    %
\usepackage{hyperref}       %
\usepackage{url}            %
\usepackage{booktabs}       %
\usepackage{amsfonts}       %
\usepackage{nicefrac}       %
\usepackage{microtype}      %
\usepackage{xcolor}         %

\usepackage{amsmath, nccmath, amssymb}
\usepackage{amsthm}
\usepackage[ruled]{algorithm2e}
\usepackage{cleveref}
\usepackage{graphicx}
\usepackage[most]{tcolorbox}
\usepackage{mathrsfs}
\usepackage{graphbox}
\usepackage{caption}
\usepackage{subcaption}
\usepackage{wrapfig}
\usepackage{float}
\usepackage{enumitem}
\usepackage{tabu}
\usepackage{doi}
\usepackage[parfill]{parskip}
\usepackage{circledsteps}
\newcommand{\CircledTall}[1]{\Circled{\vphantom{f}#1}}

\usepackage[letterpaper,margin=0.75in]{geometry}

\newcommand{\evalshort}{%
    \mathclose{}\left.\vphantom{\big|}\right|%
}

\input{macros}
\usepackage{mathtools}
\usepackage{xspace}
\usepackage{xfrac}
\usepackage{tcolorbox}
\usepackage{cancel}
\usepackage{enotez}
\usepackage{float}
\usepackage[bottom,stable]{footmisc}
\usepackage{tablefootnote}
\newtagform{nowidth}{\llap\bgroup(}{)\egroup}
\usetagform{nowidth}

\newcommand{\gd}{gradient descent\xspace}
\newcommand{\ngd}{normalized gradient descent\xspace}
\newcommand{\gflow}{gradient flow\xspace}
\newcommand{\rmsnorm}{Scalar RMSProp\xspace}
\newcommand{\rmsprop}{RMSProp\xspace}

\newtheoremstyle{spaced} %
  {5pt}    %
  {0pt}    %
  {\normalfont}  %
  {}       %
  {\bfseries}  %
  {.}      %
  {.5em}   %
  {}       %
\theoremstyle{spaced}

\DeclareMathOperator{\sdcp}{SDCP}
\newcommand{\wbar}{\overline{w}}

\newcommand{\Seff}{S^{\text{eff}}}

\renewcommand{\paragraph}[1]{\noindent\textbf{#1}\quad}

\newtheorem{corollary}{Corollary}
\newtheorem{lemma}{Lemma}
\newtheorem{proposition}{Proposition}
\newtheorem{remark}{Remark}
\newtheorem{fact}{Fact}

\newtheorem{definition}{Definition}
\newtheorem{example}{Example}

\newcommand{\succeqOver}[1]{\succeq\!_{#1}\;}

\newcommand{\preceqOver}[1]{\preceq\!_{#1}\;}

\usepackage{caption}

\usepackage{etoolbox}
\usepackage{multido}
\makeatletter
\def\@@bfil{\leaders \vrule \@height \ht\z@ \@depth \z@ \hfill}%
\def\@bLfil{\@@bfil}%
\def\@bRfil{\@@bfil}%
\def\resetbraceratio{\gdef\@bLfil{\@@bfil}\gdef\@bRfil{\@@bfil}}%
\def\setbraceratio#1#2{%
  \let\@bLfil\relax%
  \multido{\iA=1+1}{#1}{\gappto\@bLfil{\@@bfil}}%
  \let\@bRfil\relax%
  \multido{\iA=1+1}{#2}{\gappto\@bRfil{\@@bfil}}%
}
\def\upbracefill{$\m@th\setbox\z@\hbox{$\braceld$}\bracelu\@bLfil\bracerd\braceld\@bRfil\braceru$}
\def\downbracefill{$\m@th\setbox\z@\hbox{$\braceld$}\braceld\@bLfil\braceru\bracelu\@bRfil\bracerd$}
\makeatother

\newcommand{\anote}[1]{\textcolor{blue}{AD: #1}}
\newcommand{\jnote}[1]{\textcolor{red}{JC: #1}}
\renewcommand{\anote}[1]{}
\renewcommand{\jnote}[1]{}

\newcounter{specialfigprefix}

\newenvironment{specialfigures}{%
    \setcounter{specialfigprefix}{\value{figure}}
    \stepcounter{specialfigprefix}

    \setcounter{figure}{0}
    
}{%
    \setcounter{figure}{\value{specialfigprefix}}

}

\makeatletter
  \newcounter{saved@figure}
  \newcounter{letterfigure}

  \newenvironment{letterfigures}{%
    \setcounter{saved@figure}{\value{figure}}%
    \setcounter{letterfigure}{0}%

    \let\orig@figure\figure
    \let\endorig@figure\endfigure

    \renewenvironment{figure}[1][]{%
      \refstepcounter{letterfigure}%
      \orig@figure[##1]%
    }{%
      \endorig@figure
    }%
  }{%
    \let\figure\orig@figure
    \let\endfigure\endorig@figure

    \setcounter{figure}{\value{saved@figure}}%
    \addtocounter{figure}{1}%

  }
\makeatother

\crefformat{footnote}{#2\footnotemark[#1]#3}
\ifdefined\usebigfont

\usepackage{times}
\usepackage[fontsize=13pt]{scrextend}
\makeatletter
\@ifpackageloaded{geometry}{\AtBeginDocument{\newgeometry{letterpaper,left=1.56in,right=1.56in,top=1.71in,bottom=1.77in}}}{\usepackage[letterpaper,left=1.56in,right=1.56in,top=1.71in,bottom=1.77in]{geometry}}
\AtBeginDocument{\newgeometry{letterpaper,left=1.56in,right=1.56in,top=1.71in,bottom=1.77in}}
\usepackage{hyperref} %

\else
\fi

\begin{document}

\begin{center}
    {\LARGE Understanding Optimization in Deep Learning with Central Flows} \\[2em]
    \begin{minipage}{0.45\textwidth}
    \centering
    Jeremy Cohen* \\
    Carnegie Mellon and Flatiron Institute \\
    \href{https://jmcohen.github.io}{\texttt{jmcohen.github.io}}
  \end{minipage}
  \begin{minipage}{0.45\textwidth}
    \centering
    Alex Damian* \\
    Princeton University \\
    \href{https://alex-damian.github.io}{\texttt{alex-damian.github.io}}
  \end{minipage}
  \\[1.5em]
  \begin{minipage}{0.3\textwidth}
    \centering
    Ameet Talwalkar \\
    Carnegie Mellon University
  \end{minipage}
  \begin{minipage}{0.3\textwidth}
    \centering
    J. Zico Kolter \\
    Carnegie Mellon University
  \end{minipage}
  \begin{minipage}{0.3\textwidth}
    \centering
    Jason D. Lee \\
    Princeton University
  \end{minipage}
  \\[2em]
\end{center}

\begin{NoHyper}
\def\thefootnote{*}\footnotetext{Equal contribution; author ordering determined by coin flip over a Zoom call \citep{kingma2014adam}.  Please direct correspondence to both the first authors (see websites for latest contact information).  Alex Damian is now at Harvard and Jason D. Lee is now at U.C. Berkeley.  This is the full version of a paper that was published at ICLR 2025.
}
\end{NoHyper}

\input{0_abstract}

\input{1_introduction}

\newpage

{\small \tableofcontents}

\input{2_related_work}

\input{3_gd}

\input{4_scalar_rmsprop}

\input{5_rmsprop}

\input{6_experiments}

\input{7_discussion}

\input{8_conclusion}

\newpage
\input{9_acknowledgements}

\newpage
\bibliography{ref}

\clearpage

\appendix

\newpage

\input{appendix-0-flow_derivations}

\input{appendix-1-experimental-details}

\input{appendix-2-miscellaneous}

\input{appendix-3-additional-figures}

\input{appendix-4-bulk-experiments}

\end{document}

%% file: macros.tex
\usepackage{amsmath, amssymb, amsthm, amsfonts}
\usepackage{times}
\usepackage{xcolor}
\usepackage{graphicx}
\usepackage{enumitem}
\usepackage{titlesec}
\usepackage{mathrsfs}
\usepackage{physics}
\usepackage{booktabs}
\usepackage{tensor}
\usepackage{cleveref}

\DeclareMathOperator{\E}{\mathbb{E}}

\DeclareMathOperator*{\argmax}{arg\,max}
\DeclareMathOperator*{\argmin}{arg\,min}
\newcommand{\R}{\mathbb{R}}

\DeclareMathOperator{\diag}{diag}

\DeclareMathOperator{\spn}{span}

\DeclareMathOperator{\sym}{Sym}

%% file: 0_abstract.tex
\begin{abstract}
Traditional theories of optimization cannot describe the dynamics of optimization in deep learning, even in the simple setting of deterministic training.  
The challenge is that optimizers typically operate in a complex, oscillatory regime called the \emph{edge of stability}.  In this paper, we develop theory that can describe the dynamics of optimization in this regime.
Our key insight is that while the \emph{exact} trajectory of an oscillatory optimizer may be challenging to analyze, the time-averaged (i.e. smoothed) trajectory is often much more tractable.  To analyze an optimizer, we derive a differential equation called a \emph{central flow} that characterizes this time-averaged trajectory. We empirically show that these central flows can predict long-term optimization trajectories for generic neural networks with a high degree of numerical accuracy.  By interpreting these central flows, we are able to understand  how gradient descent makes progress even as the loss sometimes goes up; how adaptive optimizers ``adapt'' to the local loss landscape; and how adaptive optimizers implicitly navigate towards regions where they can take larger steps.  
Our results suggest that central flows can be a valuable theoretical tool for reasoning about optimization in deep learning.
\end{abstract}

%% file: 1_introduction.tex
\section{Introduction}

While there is a rich body of work on the theory of optimization, few works attempt to analyze optimization in ``real'' deep learning settings. Instead, even works motivated by deep learning often rely on unrealistic assumptions such as convexity, or restrict their analyses to simplified models. Practitioners cannot use such theories to reason directly about their optimization problems.
Our goal in this paper is to develop optimization theory that applies \emph{directly} to deep learning problems. This is a difficult task: prior research has shown that, even in the seemingly simple setting of deterministic (i.e. full-batch) training, optimization typically operates in a complex, oscillatory regime called the \emph{edge of stability} (EOS) \citep{xing2018walk, wu2018dynamical, jastrzębski2018on, jastrzebski2020the, cohen2021gradient, cohen2022adaptive}.  
The dynamics of optimization in this regime cannot be captured by traditional optimization theory.

\vspace{-5px}
\begin{figure}[b]
\centering
\captionsetup[subfigure]{justification=centering}
\includegraphics[width=\textwidth]{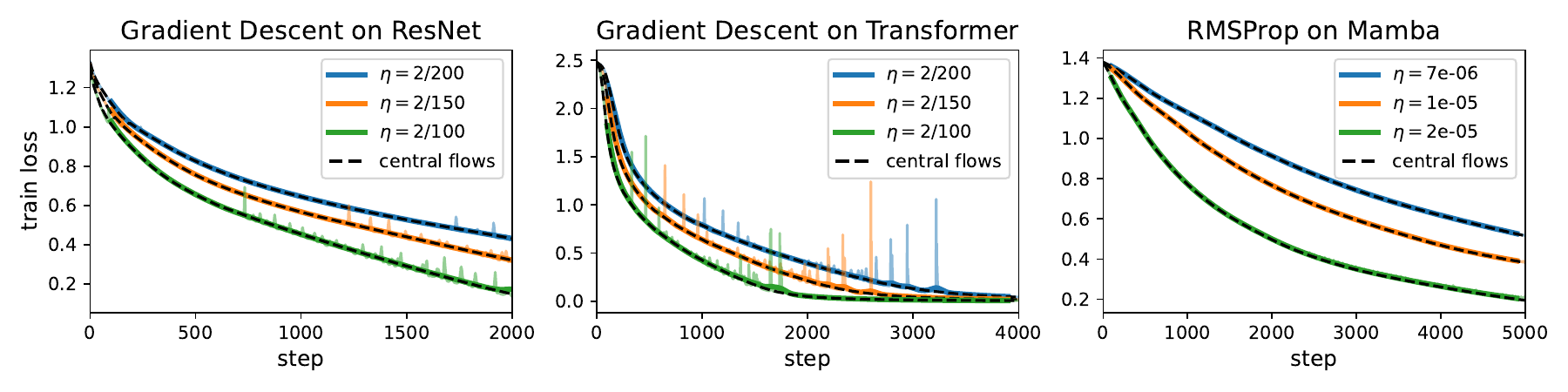}
\caption{\textbf{Our theory accurately predicts long-term optimization trajectories of practical neural networks}.\quad  We hold our theory to the high standard of rendering accurate \emph{numerical} predictions about the optimization of practical (i.e. non-toy) neural networks.  For example, this figure shows that our central flows can accurately predict the time-averaged (smoothed) loss curves of gradient descent and RMSProp on various practical architectures.\vspace{-5px}}
\label{fig:figure-one-losscurves}
\end{figure}

In this paper, we devise a methodology for analyzing these oscillatory deep learning dynamics.  Our key insight is that while the \emph{fine-grained} trajectory of an oscillatory optimizer may be challenging to analyze, the \emph{time-averaged} (i.e. locally smoothed) trajectory is often much more tractable.  To analyze an optimization algorithm, we derive a differential equation called a \emph{central flow} which explicitly captures this time-averaged trajectory (\Cref{fig:central-flow-cartoon}). Being a smooth curve, the central flow is a simpler object than the original oscillatory trajectory.  Hence, by interpreting the central flow, we can reason more easily about the original optimizer. 

\begin{figure}[t]
    \centering
    \includegraphics[width=\linewidth]{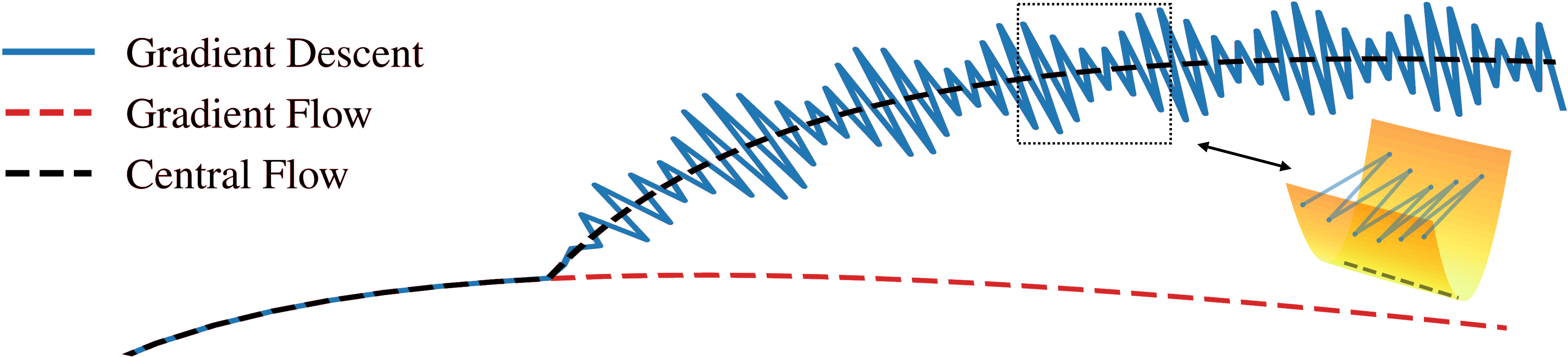}
    \caption{\textbf{The central flow models the time-averaged (i.e. smoothed) trajectory of the oscillatory optimizer}. In this illustrative cartoon of the weight-space dynamics, gradient descent (blue) takes an oscillatory path through weight space.  The central flow (black) is a smooth curve that characterizes this trajectory, whereas gradient flow (red) takes a different path.
    As illustrated in the inset, an oscillatory optimizer can be visualized as moving through a ``valley'' while bouncing between the ``valley walls'' \citep{xing2018walk,cohen2021gradient,Wen2024UnderstandingWL}.}
    \label{fig:central-flow-cartoon}
\end{figure}

We start in \Cref{sec:gd} by analyzing gradient descent, the simplest optimizer.
We first explain why traditional analyses cannot capture the typical dynamics of gradient descent in deep learning, and we then present a new analysis that does capture these dynamics.
The product of this analysis is a central flow.
We use this central flow to understand various aspects of gradient descent's behavior, such as how the train loss can behave non-monotonically over the short-term while nevertheless decreasing over the long term.
We then examine a simple adaptive optimizer in \Cref{sec:rmsprop_norm}, before turning to RMSProp (i.e. Adam without momentum) in  \Cref{sec:rmsprop}. 
We show that much of the behavior of these optimizers is actually \emph{implicit} in their oscillatory dynamics, and we render such behaviors \emph{explicit} via our central flow analysis.
In particular, our central flows reveal how these adaptive optimizers: (1) implicitly adapt their step size(s) to the local curvature, and (2) implicitly steer towards lower-curvature regions where they can take larger steps.

We focus in this paper on the simple, idealized setting of deterministic (i.e. full-batch) training.  However, we emphasize that similar optimization dynamics have been observed in the practical stochastic setting \citep{jastrzębski2018on, jastrzebski2020the, andreyev2024edge}.  We view our analysis of deterministic optimization as a necessary stepping stone to a subsequent analysis of stochastic optimization. 

While we derive each central flow using informal mathematical reasoning, we show that these flows can accurately predict long-term optimization trajectories in a variety of deep learning settings --- a high standard of empirical proof.  
Thus, we believe that central flows hold promise as a framework for analyzing, reasoning about, and perhaps even inventing, deep learning optimizers.

\begin{tcolorbox}[left=2pt,right=2pt,top=3pt,bottom=3pt,colback=white!0,halign=center]
We strongly encourage readers to look at the companion \href{http://centralflows.github.io}{\textbf{blog version}} of this paper for an interactive exposition with animations, as well as the accompanying \href{http://github.com/centralflows/centralflows}{\textbf{code}} we used to simulate the central flows.\footnotemark
\end{tcolorbox}
\footnotetext{The blog version is \href{http://centralflows.github.io}{\texttt{centralflows.github.io}} and the code is at \href{http://github.com/centralflows/centralflows}{\texttt{github.com/centralflows/centralflows}}.}

%% file: 2_related_work.tex
\newpage
\section{Related Work}
\label{sec:related_work}

\paragraph{Edge of stability} 
The dynamics of optimization in deep learning remain poorly understood, even in the seemingly simple setting of deterministic (i.e. full-batch) training. 
Indeed, recent research showed that gradient descent on neural networks typically operates in a regime termed the ``edge of stability'' (EOS) in which (1) the largest Hessian eigenvalue equillibrates around the \emph{critical threshold} $2/\eta$, and (2) the algorithm oscillates along high-curvature directions without diverging \citep{xing2018walk, wu2018dynamical, jastrzębski2018on, jastrzebski2020the, cohen2021gradient}. These dynamics could not be explained by existing optimization theory, which led \citet{cohen2021gradient} to  observe that there was no explanation for how or why gradient descent can function properly in deep learning.

Subsequently, several studies sought to theoretically explain EOS dynamics. Some works rigorously analyzed EOS dynamics on specific objective functions 
\citep{agarwala2022secondorder, ahn2024learning, chen2023edge, even2023s, kreisler2023gradient, song2023trajectory, li2022analyzing, wu2024implicit, zhu2023understanding}, while other works \citep{arora2022understanding, lyu2022understanding, damian2023selfstabilization} gave generic analyses based on a local \emph{third-order} Taylor expansion of the loss, which is one order higher than is normally used in the theoretical analysis of gradient descent.
Similar arguments were first used by \citet{blanc2020implicit} to study implicit regularization in SGD.
Our work is most directly inspired by \citet{damian2023selfstabilization}, as their analysis applies to generic objective functions, and holds throughout training, not just near convergence.
However, whereas they analyze the \emph{fine-grained} oscillatory dynamics, we argue that analyzing the \emph{time-averaged} dynamics is simpler, and is sufficient for most purposes.

\paragraph{Continuous-time models for optimization}
The standard continuous-time model for \gd is the gradient flow. \citet{barrett2021implicit, smith2021origin} argued that \gd is, instead, better approximated by a \emph{modified} gradient flow that is augmented with a penalty on the squared gradient norm.  We find that on deep learning objectives, this modified flow improves slightly over gradient flow in the stable regime, but fails in the edge of stability regime, where most of the discrepancy between \gd and gradient flow originates (see  \Cref{sec:igr}).  \citet{rosca2023continuous} proposed a flow that can model oscillations by using complex numbers. However, this flow still cannot track the long-term trajectory of \gd in EOS regime.

Many works propose to model the dynamics of stochastic optimizers using stochastic differential equations (SDEs) 
(\citealp{li2017stochastic}; \citealp{li2021validity}, \citealp{malladi2022adaptivesde}; \citealp{compagnoni2023sde, compagnoni2024adaptive}). In the full batch limit, where SGD reduces to \gd, these SDEs reduce to \gflow, which is a poor approximation to \gd at the edge of stability. Thus, these SDEs cannot be accurate in all hyperparameter regimes. Further, even when these SDEs do well-approximate the real optimizer trajectory, the SDE trajectories are themselves oscillatory, and accordingly can possess behaviors that are \emph{implicit} in the oscillatory dynamics.  Our central flows, by contrast, average out the oscillations and render all such behaviors \emph{explicit}.  Developing an analogue of the central flow for stochastic optimization is an interesting open question (see \Cref{sec:discussion}).  While some works do aim to explicitly characterize the time-averaged trajectory of SGD \citep{blanc2020implicit, damian2021label, li2022happens}, existing analyses only apply in limiting regimes (e.g. $\eta \to 0$), and only when the loss is already near zero.

\paragraph{Understanding adaptive optimizers}
\citet{ma2022qualitative} observed that RMSProp and Adam oscillate, and \citet{cohen2022adaptive} showed that such dynamics can be viewed as an adaptive version of the edge of stability.%
\citet{khaled2023dowg} and \citet{mishkin2024directional} observed that on  quadratic functions, certain adaptive optimizers implicitly adapt their effective step size to the maximum stable step size; we show this holds more generally, beyond quadratics.
The phenomenon we call ``acceleration via regularization.'' explains experiments in \citet{roulet2024stepping} and \citet{wang2024improvinggeneralizationconvergenceenhancing}.
Many works have also conducted rigorous convergence analyses of adaptive optimizers, generally focused on deriving rates of convergence to a global minimizer or stationary point \citep{duchi2011adaptive, reddi2019convergence, chen2018convergence, chen2018universal, zaheer2018adaptive, zou2019sufficient, defossez2020simple, li2024frac, chen2022towards, wang2024closing, yang2024two, guo2021novel, shi2021rmsprop, zhang2023adam, crawshaw2022robustness, li2024convergence, wang2024convergence, hong2024convergence, zhang2024convergence, wang2022provable, hubler2024parameter}.

\paragraph{Dynamical systems}
Our work likely has rich connections to various topics from the theory of dynamical systems such as the method of averaging and slow-fast systems \citep[e.g.,][]{guckenheimer1983nonlinear}.  We hope that these connections can be explored by future work.

%% file: 3_gd.tex
\section{Gradient Descent}
\label{sec:gd}
The simplest first-order optimizer is deterministic \gd with a fixed learning rate $\eta$:
\begin{align}
    w_{t+1} = w_t - \eta \nabla L(w_t). \label{eq:gd}
\end{align}
Perhaps surprisingly, \citet{cohen2021gradient} showed that traditional optimization analyses cannot capture the typical dynamics of \gd in deep learning.  We now present a new analysis that \emph{does} capture these dynamics.

\begin{itemize}
    \item In \Cref{sec:gd:dynamics}, we describe the typical dynamics of gradient descent in deep learning, and we explain why these oscillatory \emph{edge of stability} dynamics cannot be captured by traditional optimization theory.
    \item In \Cref{sec:gd:deriving}, we show that while the \emph{exact} oscillatory trajectory may be hard to analyze, the \emph{time-averaged} trajectory is more tractable.  We derive a central flow that characterizes this time-averaged trajectory.
    \item In \Cref{sec:gd:interpreting} we use this central flow to understand the behavior of gradient descent. For example, we show that while gradient descent's loss curve is non-monotonic, it can be viewed as the superposition of the loss along the central flow, plus a contribution from the oscillations. The central flow loss is a smoothly varying quantity that monotonically decreases, and therefore constitutes a hidden progress metric for gradient descent.
\end{itemize}

Our analysis of gradient descent will set the stage for subsequent analyses of more complex optimizers.

\subsection{The Dynamics of Gradient Descent}\label{sec:gd:dynamics}

\begin{figure}[b!]
\centering
\vspace{-15px}
\includegraphics[width=0.85\textwidth]{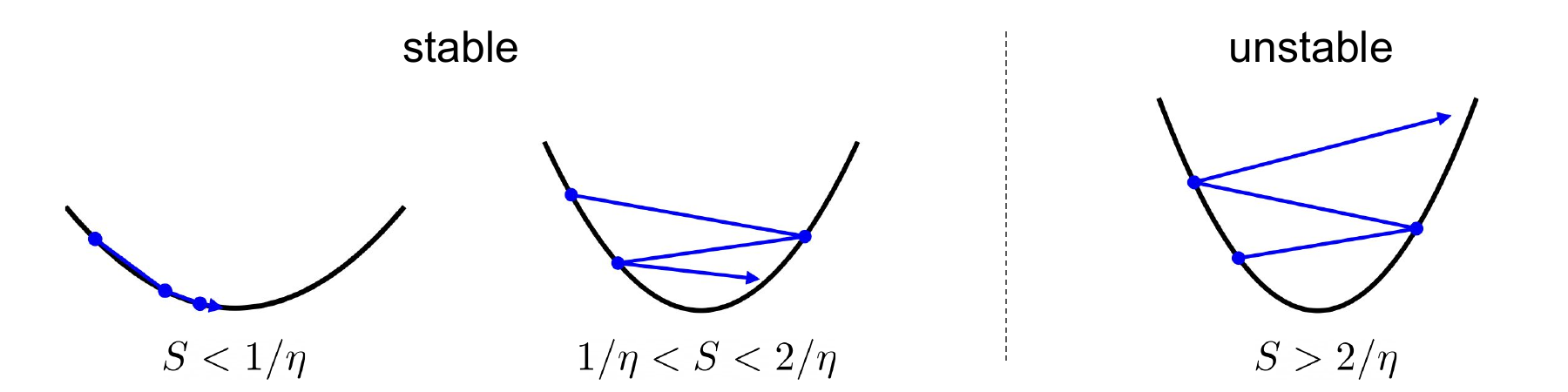}
\caption{\textbf{Gradient descent on a quadratic function. } Consider gradient descent with learning rate $\eta$ on a quadratic function $\frac{1}{2} S x^2$, with sharpness $S$.  If $S > 2/\eta$, gradient descent oscillates with exponentially growing magnitude.}
\label{fig:gd:quadratic}
\end{figure}

To understand the oscillatory dynamics of \gd in deep learning, it is instructive to first consider the case of quadratic objective functions.  On quadratic functions, \gd oscillates if the \emph{curvature} (i.e. Hessian) is too large relative to the learning rate.
For example, consider a one-dimensional quadratic objective $L(x) = \tfrac{1}{2} S x^2$, which has global curvature $S$.
Under \gd with learning rate $\eta$, the iterates $\{x_t\}$ evolve via $x_{t+1} = (1 - \eta S) x_t$.  If $S$ exceeds the \emph{critical threshold} $2/\eta$, then $(1 - \eta S) < -1$, so the iterate $x_t$ flips signs and grows in magnitude at each step, i.e. gradient descent oscillates with exponentially growing magnitude, as shown in \Cref{fig:gd:quadratic}.  
More generally, on a quadratic objective in multiple dimensions, the curvature is quantified by the Hessian matrix, and gradient descent oscillates with exponentially growing magnitude along Hessian eigenvectors with eigenvalues exceeding $2/\eta$.\footnote{For an explicit derivation of this well-known fact, see \citet[Proposition 1]{cohen2021gradient}.  Note that an exception is if the initial iterate has exactly zero alignment with these Hessian eigenvectors. 
However, this event has probability zero under any typical random initialization.}

Of course, deep learning objectives $L(w)$ are not globally quadratic.  Still, at any point $w$ in weight space, the objective can be locally approximated by a quadratic Taylor expansion around $w$.  The dynamics of \gd on this quadratic are controlled by the largest eigenvalue of the Hessian $H(w)$, which we call the \emph{sharpness} $S(w)$:
 \begin{align}
    S(w) := \lambda_1(H(w)).
\end{align}
If the sharpness $S(w)$ exceeds $2/\eta$, then \gd on the quadratic Taylor approximation would oscillate with exponentially growing magnitude along the top Hessian eigenvector(s).  This argument suggests that \gd cannot function properly in regions of weight space where the sharpness $S(w)$ exceeds $2/\eta$.

In light of this discussion, why does gradient descent converge in deep learning? The natural explanation is that the sharpness remains below $2/\eta$ throughout training.  In other words, if we define the ``stable region'' $\{w: S(w) \le 2/\eta\}$ as the subset of weight space where the sharpness is bounded by $2/\eta$, then one might suppose that gradient descent remains inside the stable region throughout training, as depicted in the cartoon \Cref{fig:gd:stable-region}(a). This is the picture suggested by traditional analyses of gradient descent.\footnote{We are referring to analyses which assume $L$-smoothness, i.e. Lipschitzness of the gradient / boundedness of the Hessian spectral norm.  This assumption is usually stated as a global condition, but analyses generally only require it to hold locally, in the vicinity of the trajectory.}

\begin{figure}[t!]
\centering
\vspace{-15px}
\begin{subfigure}[t]{0.42\textwidth}
    \centering
    \includegraphics[width=0.95\textwidth]{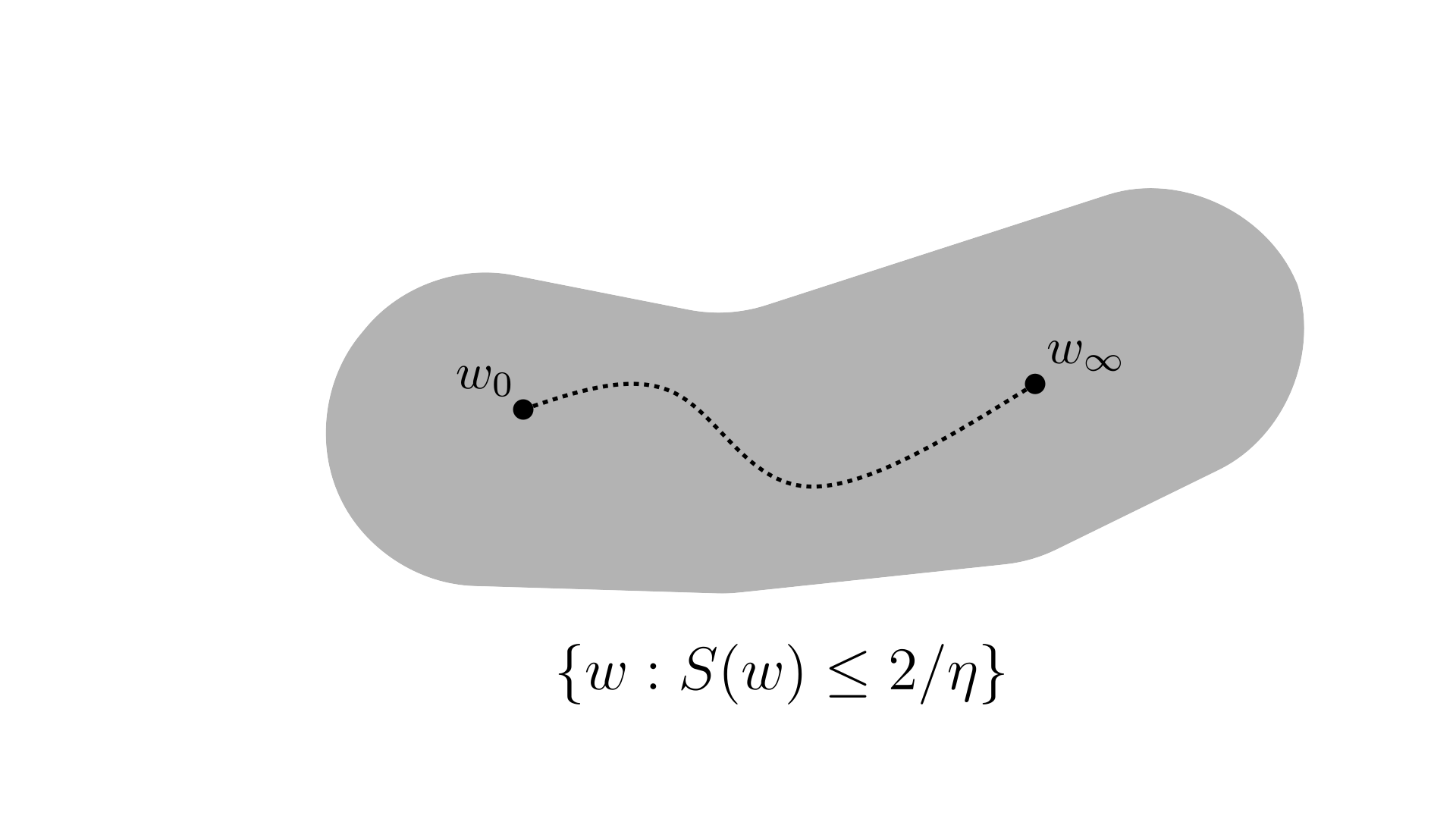}
    \caption{\textbf{Expectation:} gradient descent stays throughout training inside the \emph{stable region} (gray), the subset of weight space where the sharpness is bounded by $2/\eta$.}
    \label{fig:gd:stable-region:expectation}
\end{subfigure}
\hspace{0.03\textwidth}
\begin{subfigure}[t]{0.42\textwidth}
    
    \centering
    \includegraphics[width=0.95\textwidth]{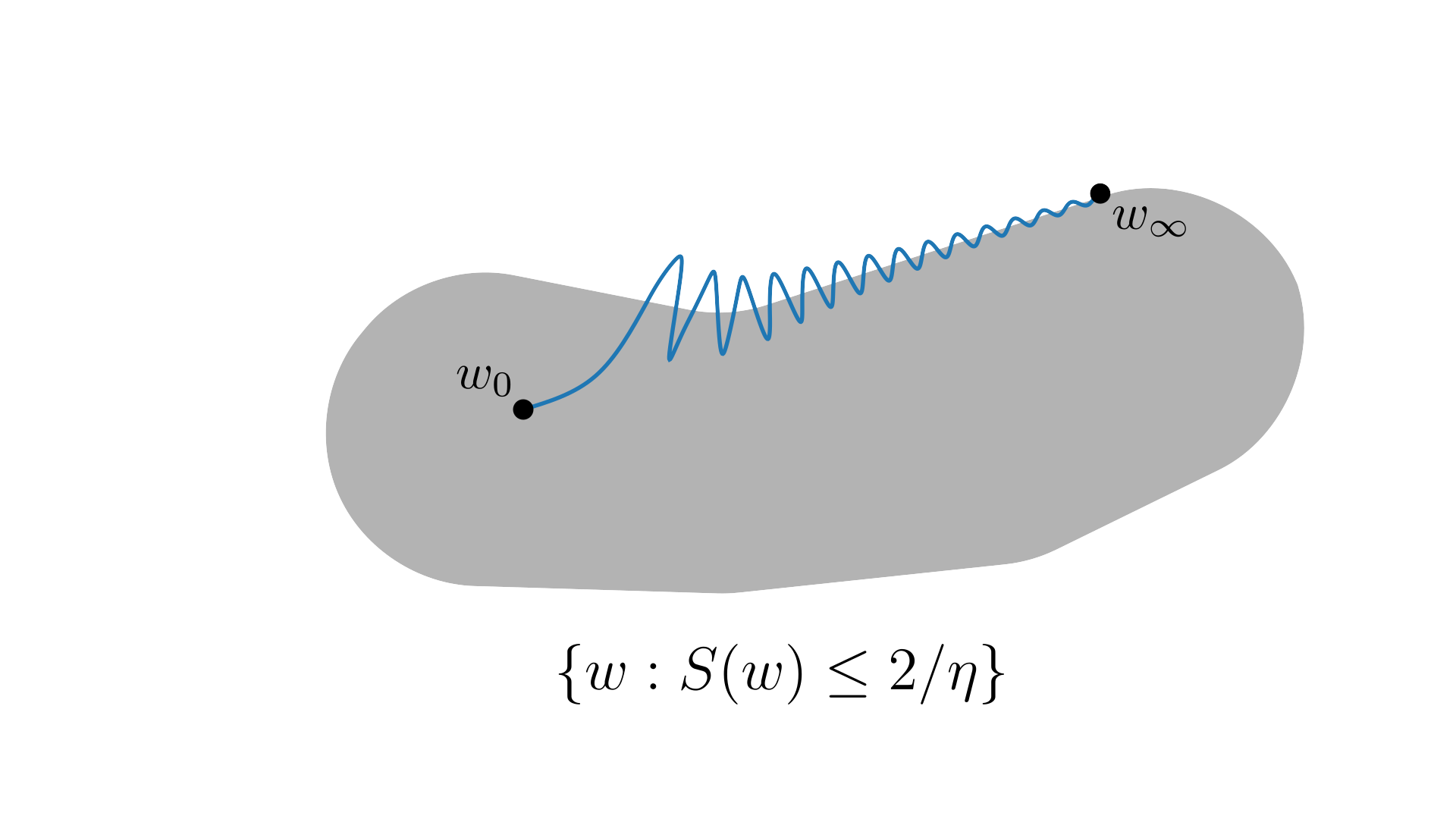}
    \caption{\textbf{Reality:} gradient descent frequently exits the stable region, but dynamically steers itself back inside.}
    \label{fig:gd:stable-region:reality}
\end{subfigure}
\caption{\textbf{Why does gradient descent converge in deep learning?}  The reality (\textbf{right}) is dramatically different from the picture suggested by traditional theory (\textbf{left}).}
\label{fig:gd:stable-region}
\end{figure}

Yet, \citet{cohen2021gradient} observed a very different reality when training neural networks using gradient descent. \Cref{fig:gradient-descent-typical} depicts a typical gradient descent trajectory, with important events annotated \textcolor{black}{\CircledTall{a}} - \textcolor{black}{\CircledTall{g}}.  Initially, the sharpness rises \textcolor{black}{\CircledTall{a}}.\footnote{Throughout this paper, we report the Hessian eigenvalues measured not at the iterates themselves, but rather at the second-order midpoints between the iterates (see \Cref{appendix:experimental-details:implementation}).  This results in plots that are slightly crisper, while retaining all essential features.} Indeed, it is a robust empirical phenomenon, dubbed \emph{progressive sharpening}, that the sharpness tends to rise when training neural networks.\footnote{Progressive sharpening remains theoretically unexplained.  Our goal in this paper is not to understand the origin of progressive sharpening in deep learning, but rather to understand the dynamics of gradient descent on objective functions which may or may not possess this property.}\footnote{In the literature, progressive sharpening has also been called a ``narrowing valley'' \citep{liu2025focus, liu2025neuralthermodynamiclawslarge} and ``lower loss as sharper'' loss landscape structure \citep{li2023loss, bai2025adaptive}.} Soon enough, the sharpness rises past the critical threshold $2/\eta$.  Once this happens, gradient descent begins to oscillate with growing magnitude along the top Hessian eigenvector,\footnote{ To compute ``displacement along top Hessian eigenvector'' for \Cref{fig:gradient-descent-typical}, we let $t_0$ be the first step of the figure (i.e. step 2990), we let $u$ be the top Hessian eigenvector computed at step $t_0$, and we report $u^\top(w_t - w_{t_0})$.} just as one would predict from a quadratic Taylor approximation \textcolor{black}{\CircledTall{b}}. These oscillations grow large enough that the train loss starts to go up rather than down  \textcolor{black}{\CircledTall{c}}.
Yet, gradient descent does not diverge. 
Instead, something odd happens: as if ``by magic,'' the sharpness rapidly \emph{drops}  \textcolor{black}{\CircledTall{d}}. Indeed, it drops all the way below the critical threshold $2/\eta$, after which point the oscillations start to shrink in magnitude  \textcolor{black}{\CircledTall{e}}, as one would expect from a new quadratic Taylor approximation. The unexplained rapid drop in the sharpness has conveniently prevented gradient descent from diverging.\footnote{This process is similar to the ``catapult'' phenomenon observed in \citet{lewkowycz2020large} at initialization.  Indeed, the EOS dynamics with one unstable eigenvalue resemble a sequential series of catapults.  However, the dynamics with >1 unstable eigenvalues are more complex.} Similar dynamics recur throughout the rest of training: gradient descent oscillates without diverging along the highest-curvature direction(s),\footnote{As gradient descent oscillates along the high-curvature directions, it can be visualized as moving through a ``valley'' while bouncing between the ``walls'' of the valley \citep{xing2018walk,cohen2021gradient,Wen2024UnderstandingWL}.} as the sharpness stays dynamically regulated around the critical threshold $2/\eta$ \textcolor{black}{\CircledTall{f}}.  Meanwhile, the train loss decreases over the long run, but behaves non-monotonically over the short run \textcolor{black}{\CircledTall{g}}.

Intuitively, whereas the traditional theory implies that gradient descent remains inside the stable region throughout training, as in \Cref{fig:gd:stable-region}(a), in reality gradient descent is frequently exiting the stable region, but is somehow steering itself back inside, as in \Cref{fig:gd:stable-region}(b). \citet{cohen2021gradient} dubbed these dynamics \emph{edge of stability} (EOS), and noted that they could not be explained by traditional optimization theory.

\begin{figure}[t]
\centering
\includegraphics[width=\textwidth]{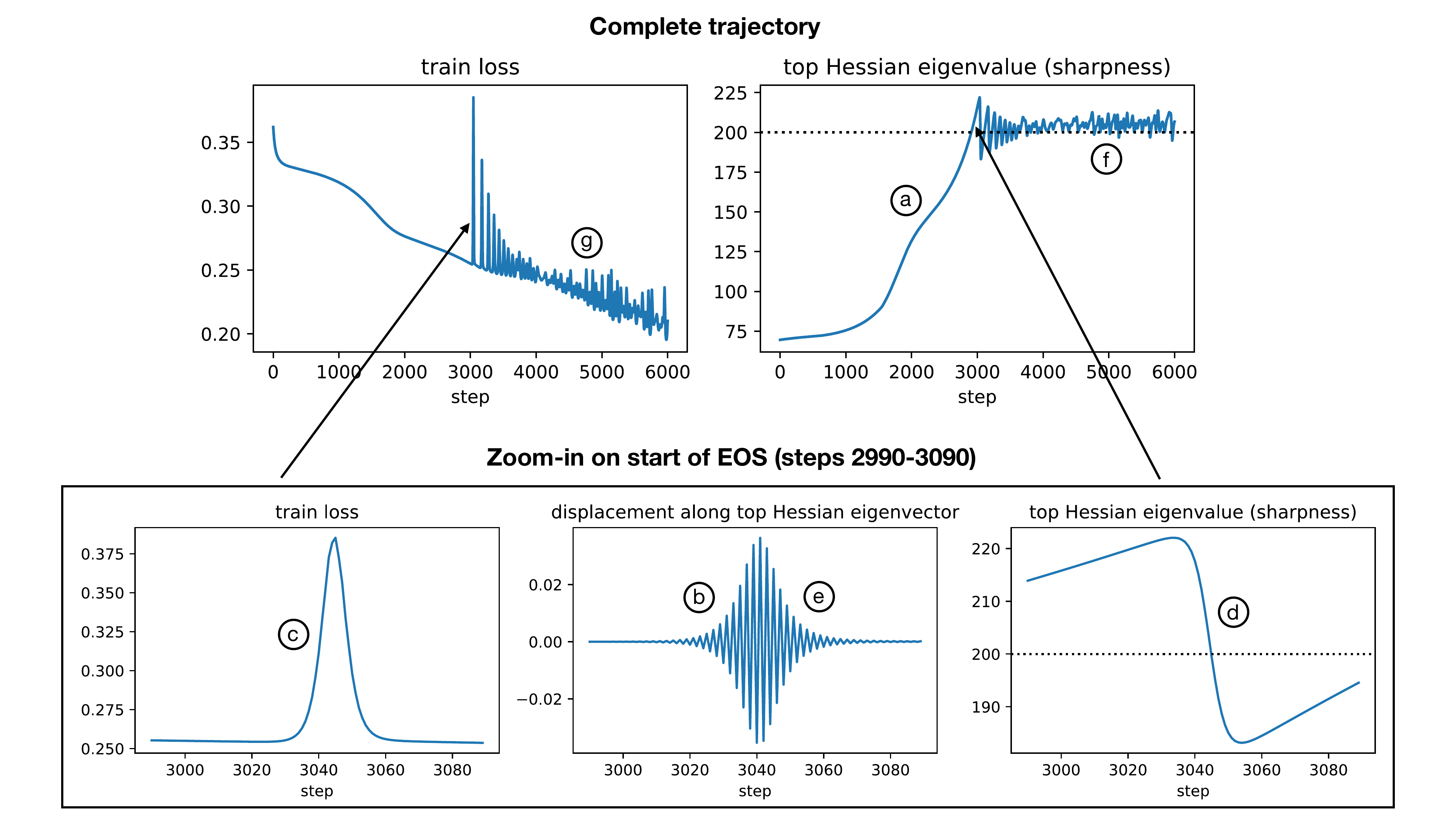}
\caption{\textbf{A typical gradient descent trajectory in deep learning}.  We train a neural network using gradient descent with step size $\eta = 0.01$.  The top row shows the long-term trajectory, while the bottom row zooms in on a particular time segment.  \textbf{(a)} The sharpness rises, reaching the critical threshold $2/\eta$ around step 2900. \textbf{(b)} Once the sharpness crosses the critical threshold $2/\eta$, gradient descent oscillates with growing magnitude along the top Hessian eigenvector. \textbf{(c)}  These oscillations cause the train loss to go up rather than down.  \textbf{(d)} However, gradient descent does not diverge; instead, as if ``by magic'', the sharpness decreases until falling below $2/\eta$.  \textbf{(e)}  Once the sharpness is below $2/\eta$, the oscillations shrink.   Throughout the rest of training: \textbf{(f)} the sharpness stays regulated around the critical threshold $2/\eta$ and \textbf{(g)} the train loss behaves non-monotonically over short timescales, while decreasing over long timescales.  \textit{Details}: the network is a Vision Transformer trained on a subset of CIFAR-10 using MSE loss. }
\vspace{10px}
\label{fig:gradient-descent-typical}
\hrule
\end{figure}

\newpage
\citet{damian2023selfstabilization} showed that the key for understanding these surprising dynamics is to Taylor-expand the objective to \emph{third} order, which is one order higher than traditionally used in analyses of gradient descent.  A third-order Taylor expansion reveals the crucial ingredient missing from traditional optimization theory:
\begin{tcolorbox}[left=2pt,right=2pt,top=3pt,bottom=3pt,colback=white!0,halign=center]
Oscillations along the top Hessian eigenvector automatically trigger reduction of the top Hessian eigenvalue.
\end{tcolorbox}

\begin{minipage}{0.6\textwidth}
Let us informally sketch this argument. Suppose that gradient descent is oscillating around a reference point $\overline{w}$, along the top Hessian eigenvector $u$, with current magnitude $x$, so that (illustration on right):
\begin{align}
    w = \overline{w} + xu.
\end{align}
\end{minipage}
\begin{minipage}{0.4\textwidth}
    \quad \quad \quad
    \includegraphics[width=0.5\linewidth]{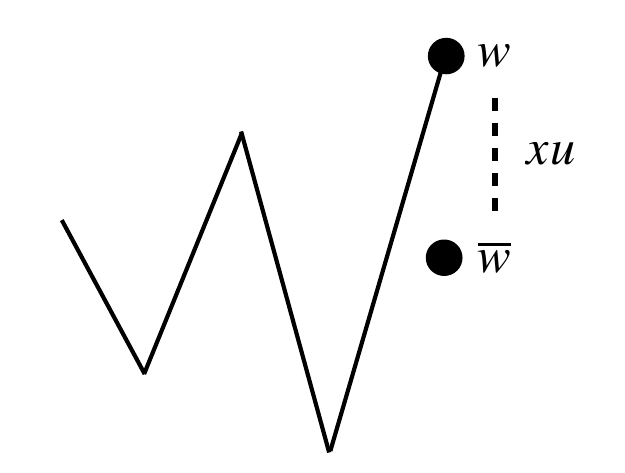}
\end{minipage}

Due to the oscillation, the optimizer follows the gradient at $w$ rather than the gradient at $\overline{w}$.  How do the two relate?  
A Taylor expansion of $\nabla L$ around $\overline{w}$ yields: 
\begin{align}
\nabla L(\overline{w} + xu)
= \stackrel{\raisebox{0.35em}{\scriptsize \textcolor{red}{first term}}}{
    \nabla L(\overline{w})
}
+ \stackrel{\raisebox{0.35em}{\scriptsize \textcolor{red}{second term}}}{
\underbrace{xH(\overline{w}) u}_{\color{red} {= \, x S(\overline{w}) u}}
}
+ \mathcal{O}(x^2).
\end{align}
Since $u$ is an eigenvector of $H(\overline{w})$ with eigenvalue $S(\overline{w})$, we recognize the second term as $x S(\overline{w}) u$. 
This term causes a negative gradient step computed at $\overline{w} + xu$ to move in the $-u$ direction.  In other words, this term is causing \gd to oscillate back and forth along the top Hessian eigenvector $u$, as predicted by the traditional theory.
The ``magic'' comes from the \emph{next} term, which arises from third-order terms in the Taylor expansion of the loss:
\begin{align}
\textcolor{gray}{
\nabla L(\overline{w} + xu)
= \stackrel{\raisebox{0.35em}{\scriptsize \textcolor{gray}{first term}}}{
    \nabla L(\overline{w})
}
+ \stackrel{\raisebox{0.35em}{\scriptsize \textcolor{gray}{second term}}}{xS(\overline{w}) u}}
+ \stackrel{\raisebox{0.15em}{\scriptsize \textcolor{red}{third term}}}{
\underbrace{\tfrac{1}{2} x^2 \, \nabla_{\overline{w}}[ u^T H(\overline{w}) u]}_{ \,\color{red} {= \tfrac{1}{2} x^2 \nabla S(\overline{w})}}
}
+ \mathcal{O}(x^3).
\label{eq:cubic-taylor}
\end{align}
Since $u^T H(\overline{w}) u = S(\overline{w})$, we recognize this term as $\frac{1}{2}x^2 \nabla S(\overline{w})$, where $\nabla S$ is none other than the \emph{gradient of the sharpness}.\footnote{Technically, equating $\nabla_{\overline{w}} [u^T H(\overline{w}) u ] = \nabla_{\overline{w}} S(\overline{w})$ requires invoking Danskin's theorem.  This is made precise in \Cref{sec:misc-math}, \Cref{lem:gd_taylor_proof}.}
Thus, a negative gradient step computed at $\overline{w} + xu$ implicitly takes a negative gradient step \emph{on the sharpness} with step size $\tfrac{1}{2} \eta x^2$.  This is the key ingredient missing from the traditional theory.  When \gd exits the stable region, it oscillates along the top Hessian  eigenvector, just as the traditional theory predicts; but what the traditional theory fails to anticipate is that these oscillations in turn perform gradient descent \emph{on the sharpness}, thereby steering the trajectory back into the stable region automatically. 

Note that traditional optimization theory fails to capture the basic \emph{causal structure} of the optimization process: gradient descent converges not because the sharpness is ``already'' small, but rather due to an automatic negative feedback mechanism that \emph{keeps} the sharpness small.

\citet{damian2023selfstabilization} analyzed the EOS dynamics in the special case where only one Hessian eigenvalue has crossed the critical threshold $2/\eta$, as in steps \textasciitilde2900-3600 in \Cref{fig:gd:multiple}.  In this setting, the dynamics consist of consecutive cycles in which: (1) the sharpness rises above $2/\eta$; (2) this triggers growing oscillations along the top Hessian eigenvector; (3) such oscillations reduce sharpness via \cref{eq:cubic-taylor}, pushing it below $2/\eta$; (4) the oscillations consequently shrink in magnitude.\footnote{Notice that the drop in the sharpness is rapid, yielding a sawtooth-like plot for the evolution of the sharpness. 
 This is because the strength of the sharpness reduction effect is proportional to $x^2$.  When $x$ is small, the sharpness-reduction effect is negligible, but when $x$ grows larger, the effect quickly becomes strong.  See \citet{damian2023selfstabilization} for a simplified ODE model of the joint dynamics between $x$ and sharpness.}
However, a more common situation is when \emph{multiple} Hessian eigenvalues have reached $2/\eta$, as in steps \textasciitilde3600-4000 in \Cref{fig:gd:multiple}.  Here, \gd oscillates simultaneously along all the corresponding eigenvectors,\footnote{With $k>1$ unstable eigenvalues, the corresponding eigenvectors are not individually identifiable; instead, one should think of \gd as oscillating within the $k$-dimensional eigenspace spanned by the $k$ eigenvectors at the edge of stability.} and these oscillations cause all such eigenvalues to remain dynamically regulated around $2/\eta$.

\begin{figure}[b!]
\centering
\includegraphics[width=\textwidth]{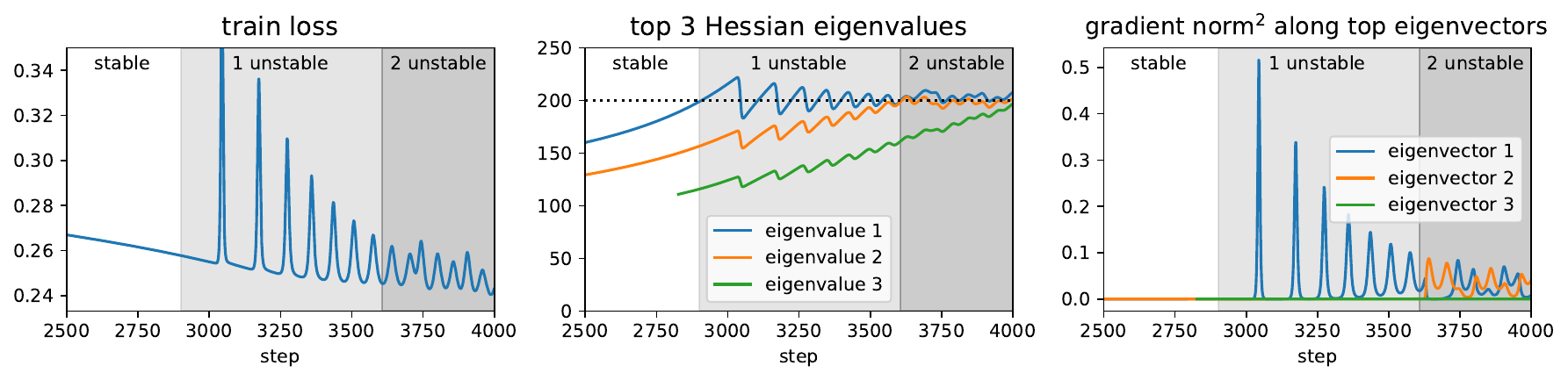}
\caption{\textbf{Multiple Hessian eigenvalues can be at the edge of stability.}  From steps \textasciitilde2900-3600, one Hessian eigenvalue is at the edge of stability, and gradient descent oscillates along the top Hessian eigenvector. From steps \textasciitilde3600-4000, two Hessian eigenvalues are at the edge of stability, and gradient descent oscillates simultaneously along both the corresponding eigenvectors. The number of oscillating directions can be easily read off from the right plot, which shows the squared norm of the gradient when projected onto each of the top 3 Hessian eigenvectors.}
\label{fig:gd:multiple}
\end{figure}

Unfortunately, analyzing EOS dynamics in fine-grained detail is challenging \citep{damian2023selfstabilization}.  The difficulty arises from the need to account for the mutual interactions between the oscillations and the curvature.  Even in the special case of one unstable eigenvalue, these dynamics are nonlinear and highly sensitive to initial conditions. The more typical case of multiple unstable eigenvalues is even harder to analyze: the dynamics with $k$ unstable eigenvalues do not decouple into $k$ independent systems, and instead involve $O(k^2)$ mutually interacting quantities, yielding complex and often chaotic behavior.

Our key insight in this paper is that a fine-grained analysis of the EOS dynamics may not be necessary.  Rather, we argue that the more important question is: what \emph{macroscopic} (i.e. long-term) trajectory does gradient descent take through weight space?  In the next section, we will use a heuristic time-averaging argument to characterize this macroscopic trajectory. Our analysis will not only recover the main finding of \citet{damian2023selfstabilization} for a single unstable eigenvalue, but will also readily generalize to the challenging setting of multiple unstable eigenvalues.  

\subsection{Deriving the Gradient Descent Central Flow}
\label{sec:gd:deriving}

The standard continuous-time approximation to gradient descent is the \gflow:\footnote{We fold $\eta$ into the definition of gradient flow so that there is a correspondence between step $t$ of \gd and time $t$ of \gflow.  This will especially be useful when analyzing adaptive optimizers where the effective step size is a dynamic quantity.}
\begin{align}
    \frac{dw}{dt} = - \eta \nabla L(w). \label{eq:gflow}
\end{align}
\citet{cohen2021gradient} observed that trajectory of \gd is well-approximated\footnote{ \citet{barrett2021implicit} argued that the accuracy of the gradient flow approximation can be improved by adding a penalty on the squared gradient norm.  However, their modified flow does not hold in the EOS regime, and in the stable regime, we found that the accuracy improvement it brings is relatively small (\Cref{sec:igr}).  Therefore, for simplicity, we leave out any such term from our flows.}\footnote{It remains theoretically unexplained why \gflow is such a good fit to \gd. Existing bounds for the distance between \gd and \gflow increase exponentially with time, with an exponent determined by the most negative Hessian eigenvalue \citep{elkabetz2021continuous}.  Empirically, such bounds are overly conservative.}
by that of \gflow so long as training is \emph{stable}, i.e. so long as the sharpness $S(w)$ remains below $2/\eta$. However, once the sharpness reaches $2/\eta$ and the dynamics enter the EOS regime, gradient descent departs from the gradient flow trajectory and takes a different path, as illustrated in \Cref{fig:gd:three-flows}.\footnote{Even in the simplest setting of one unstable eigenvalue, capturing the EOS dynamics necessarily requires three variables: one for the oscillations along the top Hessian eigenvector, one for the top Hessian eigenvalue (sharpness), and one for the remaining directions.  Since visualizing three-dimensional dynamics is difficult, we will frequently resort to two-dimensional cartoons (e.g. \Cref{fig:gd:three-flows}).  Such a ``projection'' will necessarily drop information. Accordingly, \Cref{fig:gd:three-flows} captures sharpness and remaining directions, but leaves out the back-and-forth oscillations along the top Hessian eigenvector.  \Cref{fig:central-flow-cartoon}, by contrast, captures these back-and-forth oscillations but leaves out the sharpness.}

\begin{figure}[t!]
\centering
\includegraphics[width=\textwidth]{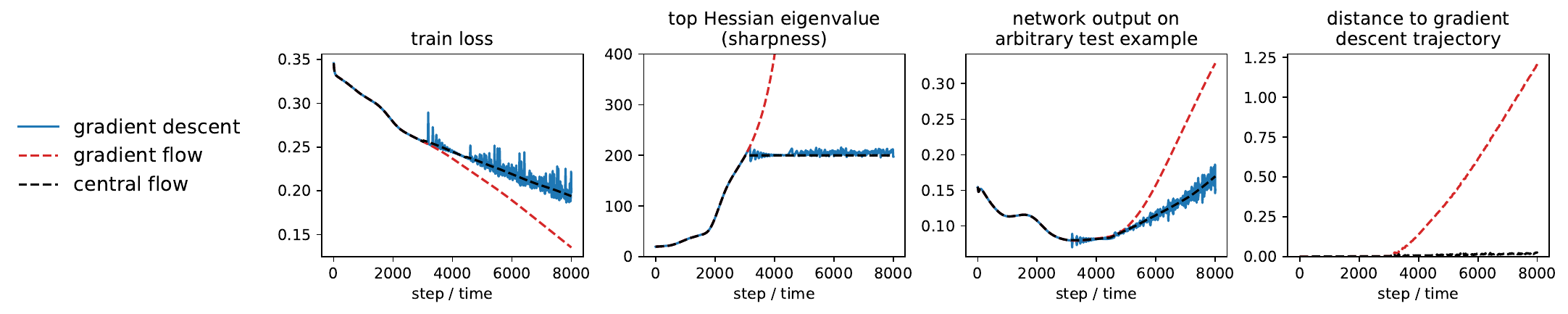}
\hrule
\vspace{0.75em}
    \begin{minipage}{0.55\textwidth}
    \vspace{5px}
    \caption{ \textbf{What macroscopic path does gradient descent take?} Gradient descent (blue) is well-approximated by gradient flow (red) so long as the sharpness is below $2/\eta$.  However, once \gd reaches the edge of stability, it takes a different path.  Our \emph{central flow} (black) approximates \gd even at the edge of stability.  The plots on top present data from an experiment (same as \Cref{fig:gradient-descent-typical}); the drawing on the right is a cartoon of the underlying weight-space dynamics.}
    \label{fig:gd:three-flows}
    \end{minipage}
    \quad
    \vrule
    \quad
    \begin{minipage}{0.40\textwidth}
    \includegraphics[width=0.9\textwidth]{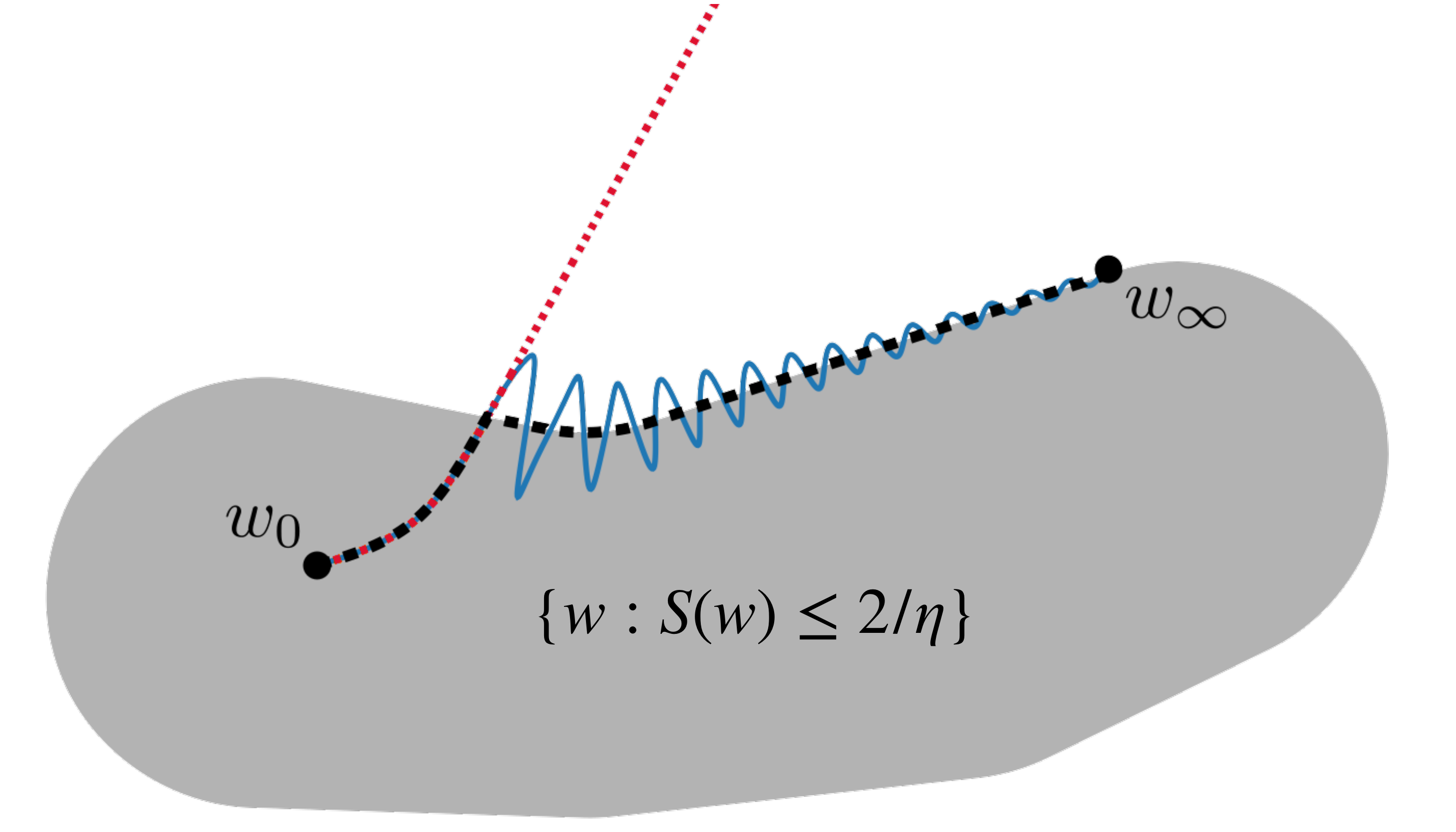}
    \end{minipage}
\end{figure}

We now derive a more general differential equation, which we call a central flow, that approximates the trajectory of gradient descent even at the edge of stability.
The central flow directly models the \emph{time-averaged} (i.e. smoothed) trajectory of the oscillatory optimizer. In other words, the central flow averages out the oscillations while retaining their lasting effect on the macroscopic trajectory. 
We will derive the central flow using a heuristic time-averaging argument, and we will empirically demonstrate that it can accurately predict long-term optimization trajectories on a variety of neural networks with a high degree of numerical accuracy, as illustrated in \Cref{fig:gd:three-flows}.

\begin{tcolorbox}[enhanced jigsaw,colback=white!10,left=2pt,right=2pt,top=2pt,bottom=2pt,halign=flush center]
We will abuse notation and use $\E$ to denote ``local time-averages'' of deterministic quantities --- see \Cref{appendix:time-average} for discussion. The gradient descent central flow is intended to model the time-averaged trajectory $\E[w_t]$. To simplify notation, we will also use $\overline{w}_t := \E[w_t]$ to denote the time-averaged trajectory.
\end{tcolorbox}

\subsubsection{Warm-up: the Special Case of One Unstable Eigenvalue}\label{sec:gd:single}
We will introduce our time-averaging methodology by analyzing the special case when only the largest Hessian eigenvalue has crossed the critical threshold $2/\eta$, and \gd oscillates along a single direction --- the corresponding eigenvector.  Our fully general analysis, given later in \Cref{sec:gd:multi}, will allow for an arbitrary number of eigenvalues to be at the edge of stability, and for eigenvalues to enter and leave the edge of stability.

Thus, in this section, we start our analysis at the instant when the sharpness $S(w)$ first reaches $2/\eta$. From this point onward, we will model the \gd trajectory by:
\begin{align}
    w_t = \overline{w}_t + x_t u_t,
\end{align}
where $w_t$ is the gradient descent iterate, $\overline{w}_t$ is the time-averaged iterate, $u_t$ is the top Hessian eigenvector at $\overline{w}_t$, and $x_t$ denotes the displacement between $w_t$ and $\overline{w}_t$ along the $u_t$ direction.\footnote{\label{footnote:oscillate_dS_single} This is a simplification.  In reality, we know that \gd is displaced from $\overline{w}$ in at least \emph{two} directions: the $u$ direction and the $\nabla S(\overline{w})$ direction, with the latter responsible for the fluctuations in the sharpness. However, when modeling the time-averaged gradient, we will only account for the displacement in the $u$ direction.  This is analogous to \citet[Assumption 5]{damian2023selfstabilization}. The success of our experiments validates this simplification.}
Note that by definition, $\E[x_t] = 0$, i.e. the time-averaged displacement is zero. To track the evolution of the time-averaged iterate $\overline{w}_t$, we time-average both sides of the gradient descent update \cref{eq:gd}:
\begin{align}
    \overline{w}_{t+1} = \overline{w}_t - \eta \underbrace{\E[\nabla L(w_t)]}_{\mathclap{\text{time-averaged gradient}}}.\label{eq:gd:expected_update}
\end{align}
That is, the time-averaged iterates follow the (negative) time-averaged gradient. 
To approximate the time-averaged gradient, we first Taylor-expand the gradient $\nabla L$ around the time-averaged iterate $\overline{w}_t$:
\begin{align}
    \nabla L(w_t) = \underbrace{\nabla L(\overline{w}_t)}_{\text{gradient at }\overline{w}_t} + \underbrace{x_t \, S(\overline{w}_t) u_t}_{\text{oscillation}} \, + \underbrace{\tfrac{1}{2} \, x_t^2 \, \nabla S(\overline{w}_t)}_{\text{sharpness reduction}} + \; \mathcal{O}(x^3).
    \label{eq:gd:cubic_taylor}
\end{align}
We then take the time average of both sides, averaging over the $x$ oscillations. This reflects an implicit assumption that the $x$ oscillations are happening fast relative to the remaining training dynamics:\footnote{When time-averaging \cref{eq:gd:cubic_taylor}, we assume that the eigenvector $u$ changes slowly relative to the displacement $x$ so that $\E[x_t u_t] \approx \E[x_t] u_t$.}
\begin{align}
    \E[\nabla L(w_t)] \enskip\approx\enskip \nabla L(\overline{w}_t) +   \underbrace{\cancel{ \E[x_t] S(\overline{w}_t) u_t}}_{0\text{ because }\E[x_t] = 0} \; + \; \underbrace{\tfrac{1}{2} \E[x_t^2] \nabla S(\overline{w}_t)}_{\mathclap{\text{implicit sharpness penalty}}} \enskip . \label{eq:gd:expected_gradient_single}
\end{align}
This calculation shows that the time-averaged gradient $\E[\nabla L(w_t)]$ is equal to the gradient at the time-averaged iterate $\nabla L(\overline{w}_t)$, plus an implicit sharpness penalty whose strength is proportional to $\E[x_t^2]$, the variance of the oscillations at step $t$. Substituting \cref{eq:gd:expected_gradient_single} into \cref{eq:gd:expected_update} and switching to continuous time, we therefore model the time-averaged iterates $\overline{w}_t$ by the sharpness-penalized gradient flow $w(t)$ defined by:
\begin{align}
    \frac{dw}{dt} = - \eta \qty\Big[\;\nabla L(w) + \underbrace{\tfrac{1}{2} \sigma^2(t) \nabla S(w) }_{\mathclap{\text{implicit sharpness penalty}}}\;].
    \label{eq:gd:single_ansatz}
\end{align}
Here, $\sigma^2(t)$ is a still-unknown quantity intended to model $\E[x_t^2]$, the instantaneous variance of the oscillations at time $t$. This quantity also controls the strength of the implicit sharpness penalty.
To determine $\sigma^2(t)$, we argue that only one value is consistent with the observed behavior of \gd.
Empirically, once the sharpness reaches the critical threshold $2/\eta$, it does not continue to rise indefinitely; rather, it remains dynamically regulated around $2/\eta$. Thus, we will enforce that the central flow never increases the sharpness $S(w(t))$ past $2/\eta$. The time derivative of the sharpness under a flow of the form \cref{eq:gd:single_ansatz} can be easily computed using the chain rule:
\begin{align}
    \frac{dS(w)}{dt} = \ev{\nabla S(w), \frac{dw}{dt}} = \underbrace{\vphantom{\tfrac{1}{2}}\eta \ev{\nabla S(w), -\nabla L(w)}}_{\mathclap{\substack{\text{change in sharpness}\\\text{under gradient flow}}}} \;-\; \underbrace{\tfrac{1}{2}\eta \sigma^2(t) \norm{\nabla S(w)}^2}_{\mathclap{\substack{\text{sharpness reduction}\\\text{from oscillations}}}}.\quad\quad
    \label{eq:gd:dSdt}
\end{align}
When the first term, the change in sharpness under the \gflow, is \emph{negative}, \gd will leave the edge of stability and will once again follow gradient flow --- this is made precise in \Cref{sec:gd:multi}. Therefore, we focus on the case where this first term is positive, i.e. where progressive sharpening holds. As the sharpness is currently at $2/\eta$ and must remain at $2/\eta$, we must have that $\tfrac{dS(w)}{dt} = 0$. Since $\frac{dS(w)}{dt}$ is affine in $\sigma^2(t)$, we can easily solve for the unique value of $\sigma^2(t)$ that ensures $\tfrac{dS(w)}{dt} = 0$:
\begin{align}
    \sigma^2(t) = \frac{2 \ev{\nabla S(w), -\nabla L(w)}}{\norm{\nabla S(w)}^2} .
    \label{eq:gd_x}
\end{align}
Intuitively, this is the unique $\sigma^2(t)$ for which the downward force of oscillation-induced sharpness reduction ``cancels out'' the upwards force of progressive sharpening so the sharpness remains locked at $2/\eta$. The central flow for a single unstable eigenvalue is given by substituting this $\sigma^2(t)$ into \cref{eq:gd:single_ansatz}:
\begin{align}
    \frac{dw}{dt} = - \eta \qty\Big[\;\nabla L(w) + \tfrac{1}{2} \sigma^2(t) \, \nabla S(w)] \quad \text{where} \quad \sigma^2(t) =  \frac{2 \ev{\nabla S(w), -\nabla L(w)}}{\norm{\nabla S(w)}^2}.
    \label{eq:gd:single}
\end{align}

\begin{figure}[t!]
\centering
\includegraphics[width=\textwidth]{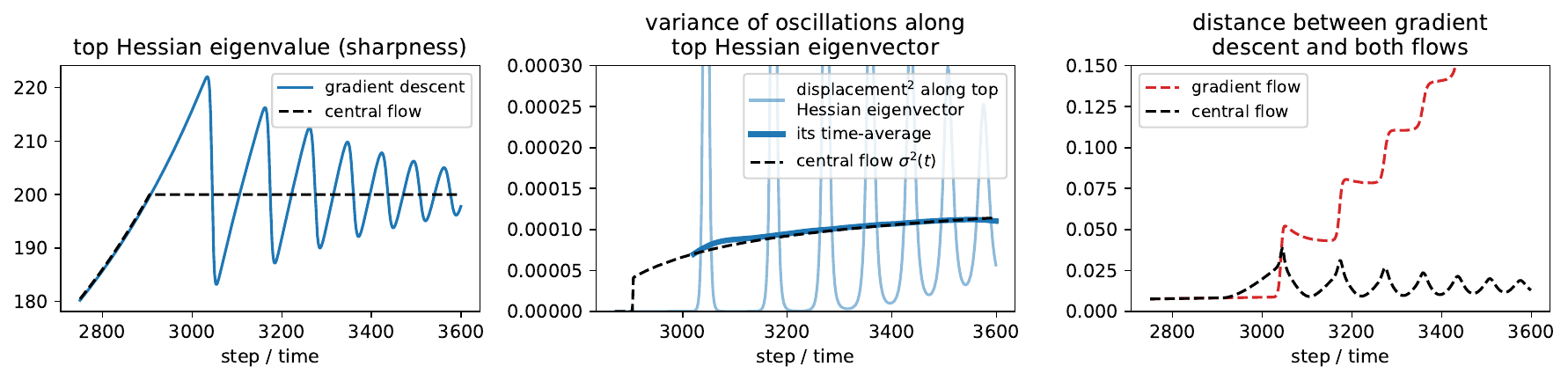}
\caption{ \textbf{Illustrating the central flow with one unstable eigenvalue}.  
\textbf{Left}: the sharpness (top Hessian eigenvalue) cycles around $2/\eta$ under gradient descent, and stays capped at $2/\eta$ under the central flow.  \textbf{Center}: the central flow's $\sigma^2(t)$ accurately predicts the variance of the oscillations.  In light blue, we plot $(u(t)^\top (w_t - w(t)))^2$, the squared displacement between gradient descent $w_t$ and the central flow $w(t)$ along $u(t)$, the top Hessian eigenvector at $w(t)$.  In dark blue, using Gaussian smoothing, we plot its time average, i.e. the empirical variance of the oscillations.  Observe that this is well-predicted by the central flow's $\sigma^2(t)$, \cref{eq:gd:single}, plotted in black. \textbf{Right}: the Euclidean distance $\| w_t - w(t) \|$ between gradient descent and the central flow (black) stays small over time, indicating that the central flow accurately predicts the long-term trajectory of gradient descent.  In contrast, the distance (red) between gradient descent and the gradient flow  \cref{eq:gflow} grows large over time. This figure depicts the same ViT trajectory as \Cref{fig:gradient-descent-typical}.}
\label{fig:demonstrate-cf-one}
\end{figure}

\Cref{fig:demonstrate-cf-one} demonstrates this flow in action.
We run gradient flow until the sharpness hits $2/\eta$, and then switch to \cref{eq:gd:single} at that time.  (The complete central flow, defined in the next section, will handle such switches automatically.)
Observe that the distance in weight space between gradient descent and the central flow remains small over time,\footnote{Observant readers might notice that the distance between \gd and the central flow starts to grow once the sharpness hits $2/\eta$ and is actually initially larger than the distance between \gd and the gradient flow.  This is because the central flow has already started to apply sharpness regularization, but the discrete \gd trajectory has not yet done so.} verifying that the central flow accurately predicts the long-term trajectory of \gd.
Moreover, observe that $\sigma^2(t)$ from \cref{eq:gd_x} accurately predicts the empirical variance of the oscillations along the top Hessian eigenvector, further demonstrating that our time-averaging argument is accurately capturing gradient descent's behavior.

Intuitively, whereas gradient descent reduces sharpness in impulse-like spurts which are triggered whenever the oscillations grow large, the central flow applies a sharpness-reduction force continuously, with the same average strength over time.
That these two processes stay close over long timescales implies that the oscillations are only affecting the long-term \gd trajectory via their \emph{variance} rather than via their fine-grained details (e.g the precise shape of the light blue line in \Cref{fig:demonstrate-cf-one}, center). This is good news: while the fine-grained oscillations may be challenging to analyze, we have shown that their variance is easy to analyze, as there is only one possible value that is consistent with the observed edge of stability equilibrium.  In this way, we have successfully used a heuristic argument to solve for the time-averaged trajectory of \gd.

\paragraph{Interpretation as Projection}
While we have derived the central flow as a sharpness-penalized gradient flow, it can be equivalently interpreted as a \emph{projected} gradient flow.  In particular, simplifying \cref{eq:gd:single} gives:
\begin{wrapfigure}[11]{r}{0.3\textwidth}
\includegraphics[width=0.28\textwidth]{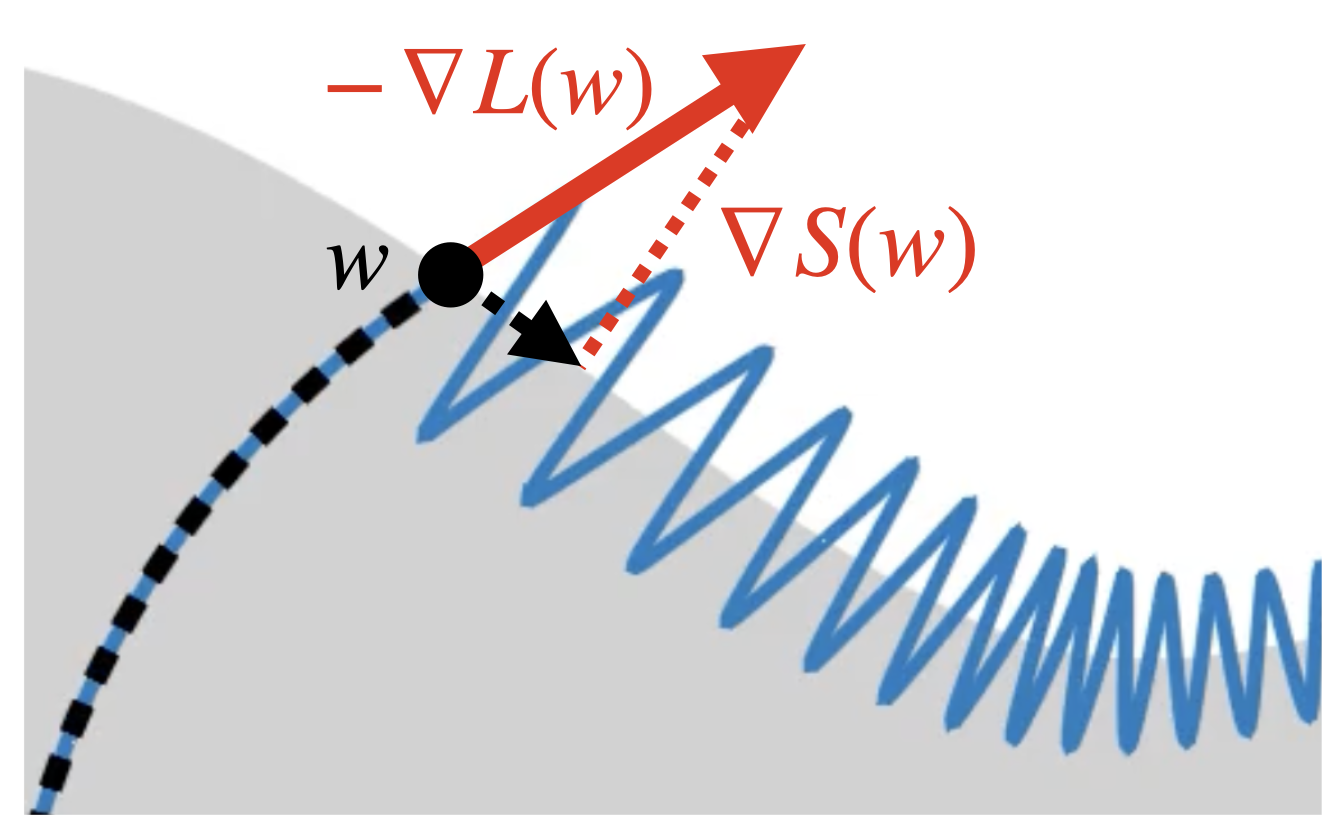}
\caption{The central flow projects out the $\nabla S(w)$ direction from the loss gradient $ \nabla L(w)$.}
\end{wrapfigure}
\begin{align}
    \frac{dw}{dt}
    &= -\eta \qty[I - \frac{\nabla S(w) \nabla S(w)^\top}{\norm{\nabla S(w)}^2}] \nabla L(w) = - \eta \Pi_{\nabla S(w)}^\perp \nabla L(w), \label{eq:sharpness-constrained-flow}
\end{align}
where $\Pi_{v}^\perp := I - \frac{vv^T}{\norm{v}^2}$ denotes the projection matrix onto the orthogonal complement of $v$.  This flow projects out the $\nabla S$ direction from the gradient $\nabla L$ to keep the sharpness fixed at $2/\eta$, as shown in the cartoon on the right.

Previously, \citet{damian2023selfstabilization} proved that under certain conditions, gradient descent at the edge of stability implicitly follows the trajectory of projected gradient descent constrained to the stable region.  We have thus nearly\footnote{The flow \cref{eq:sharpness-constrained-flow} actually differs slightly from the constrained trajectory in \citet{damian2023selfstabilization}, beyond being continuous rather than discrete. In \citet{damian2023selfstabilization}, the stable region was defined as the set where $S(w) \le 2/\eta$ \emph{and} the the loss is directionally minimized along the top Hessian eigenvector.  The latter condition prevents their theory from applying to certain models \citep[e.g.][Appendix A]{kreisler2023gradient}. Our central flow does not use the latter condition and hence does not suffer from such restrictions.} rederived their result in a simpler, albeit non-rigorous, manner.

The projection interpretation will be useful below for reasoning about gradient descent's behavior.

\subsubsection{The Fully General Case}\label{sec:gd:multi}

We will now derive the complete central flow, which applies in the fully general setting where any number of eigenvalues can be at the edge of stability, including zero.  When no eigenvalues are at the edge of stability, the central flow will automatically reduce to the gradient flow.
As above, we decompose the \gd trajectory as:
\begin{align}
    w_t = \overline{w}_t + \delta_t,
\end{align}
where $w_t$ is the gradient descent iterate, $\overline{w}_t := \E[w_t]$ is the time-averaged iterate, and $\delta_t$ denotes the displacement between $w_t$ and $\overline{w}_t$, i.e. the oscillation. Because \gd oscillates along the Hessian eigenvectors that are at the edge of stability, we model $\delta_t$ as lying within the span of these eigenvectors.\footnote{Similar to footnote \ref{footnote:oscillate_dS_single}, this neglects the motion in the top Hessian eigenvalues. The success of our experiments justifies this simplification.}
For example, in the case where only one direction is at the edge of stability, taking $\delta_t = x_t u_t$ recovers the analysis in \Cref{sec:gd:single}.
Note that by definition of $\overline{w}_t$, we have that $\E[\delta_t] = 0$.
As before, the time-averaged iterates follow the time-averaged gradient $\E[\nabla L(w_t)]$. To compute the time-averaged gradient, we first Taylor-expand the gradient around $\overline{w}_t$:
\begin{align}
    \nabla L(w_t) \enskip=\enskip \underbrace{\nabla L(\overline{w}_t)}_{\text{gradient at $\overline{w}$}} \;+\; \underbrace{H(\overline{w}_t) \delta_t}_\text{oscillation} \;+\; \tfrac{1}{2} \underbrace{\nabla_{\overline{w}_t} \; \delta_t^T H(\overline{w}_t) \delta_t}_{\mathclap{\text{implicit curvature penalty}}}  \;+\; \mathcal{O}(\|\delta_t\|^3).
\end{align}
The third term in this Taylor expansion reveals that the negative gradient at the iterate $w_t$ implicitly acts to decrease the directional curvature in the direction $\delta_t$.
Time-averaging both sides and rearranging the third term yields:
\begin{align}
    \E[\nabla L(w_t)] \enskip\approx\enskip \nabla L(\overline{w}_t) \;+\; \underbrace{\cancel{H(\overline{w}_t) \E[\delta_t]}}_{0\text{ because }\E[\delta_t]=0} \;+\; \tfrac{1}{2} \underbrace{\nabla_{\overline{w}_t} \ev{H(\overline{w}_t), \E[\delta_t \delta_t^T]}}_{\text{implicit curvature penalty}}, \label{eq:gd:expected_gradient_multi}
\end{align}
where we use $\ev{\cdot, \cdot}$ to denote the Frobenius inner product between two matrices, equivalent to flattening the matrices into vectors and taking the dot product.
Thus, we see that the time-averaged gradient is the gradient at the time-averaged iterate, plus an implicit curvature penalty whose strength and direction are determined by the \emph{covariance} of the oscillations $\E[\delta_t \delta_t^T]$.
Substituting \cref{eq:gd:expected_gradient_multi} into the time-averaged \gd update (eq. \ref{eq:gd:expected_update}) and switching to continuous time, we model the time-averaged iterates $\overline{w}_t$ by a differential equation of the form:
\begin{align}
    \frac{dw}{dt} = - \eta \qty\Big[\;\nabla L(w) + \tfrac{1}{2} \underbrace{\nabla_w \ev{H(w), \Sigma(t)}}_{\mathclap{\text{implicit curvature penalty}}}\;].
    \label{eq:gd:multi_ansatz}
\end{align}
Here, $\Sigma(t)$ is a still-unknown quantity intended to model $\E[\delta_t \delta_t^T]$, the instantaneous covariance of the oscillations at time $t$. 
This matrix also controls an implicit curvature penalty which penalizes the $\Sigma$-weighted Hessian $\langle \Sigma(t), H(w) \rangle$.\footnote{This is a weighted sum of all entries in the Hessian matrix, where each entry is weighted by the corresponding entry of $\Sigma(t)$.}
\begin{figure}[t!]
\centering
\includegraphics[width=\textwidth]{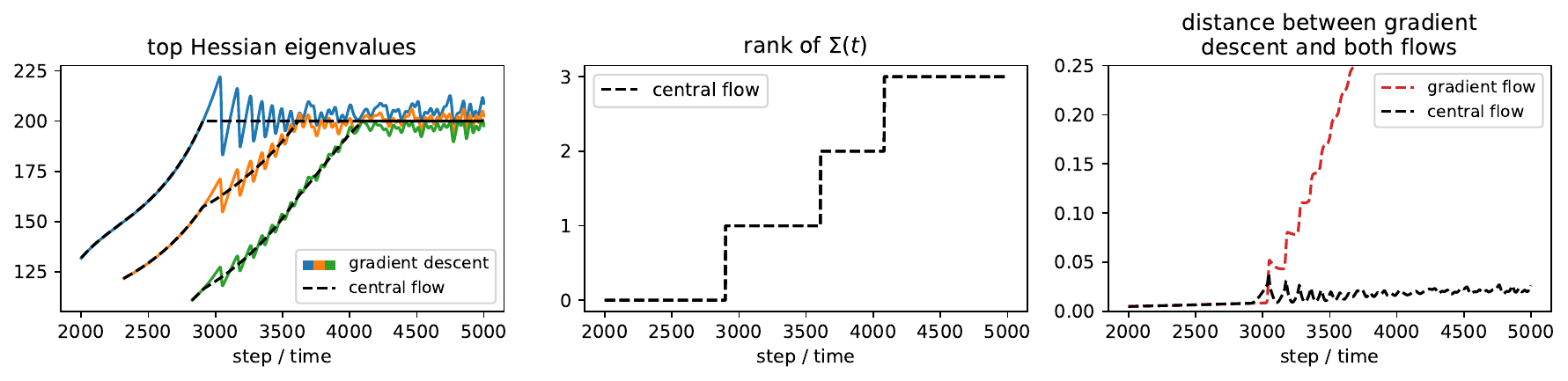}
\caption{ \textbf{Illustrating the central flow (general case)}.  
\textbf{Left}: whenever a Hessian eigenvalue rises to $2/\eta$, the central flow prevents it from increasing further.  \textbf{Center}: since $\Sigma(t)$ models the covariance of the oscillations, its rank is always equal to the number of Hessian eigenvalues at the edge of stability.  We show in \Cref{fig:cf-predict-oscillations} that $\Sigma(t)$ accurately predicts the covariance of oscillations.  \textbf{Right}: the Euclidean distance between gradient descent $w_t$ and the central flow $w(t)$ stays small over time, indicating that the central flow accurately predicts the long-term trajectory of gradient descent.  In contrast, the distance between \gd and the gradient flow \cref{eq:gflow} grows large over time.  This figure depicts the same ViT trajectory as \Cref{fig:gradient-descent-typical}.}
\label{fig:gd:demonstrate-cf-multiple}
\end{figure}
Similar as before, to determine $\Sigma(t)$, we impose three conditions:
\begin{enumerate}
    \item Since Hessian eigenvalues which reach the critical threshold $2/\eta$ do not continue to rise further, we impose the condition that $\Sigma(t)$ should not allow any Hessian eigenvalues to rise beyond $2/\eta$. 
    \item Since gradient descent oscillates within the span of the unstable eigenvectors, we impose the condition that $
    \Sigma(t)$, which models the covariance of these oscillations, should be supported\footnote{We mean that $\spn[\Sigma] \subseteq \mathcal{U}$, where $\mathcal{U}$ is the span of the Hessian eigenvectors with eigenvalue $2/\eta$.  Equivalently, we mean that $\Sigma$ can be written as $\Sigma = U X U^T$ where the $k$ columns of $U$ form a basis for the $k$-dimensional subspace $\mathcal{U}$, and $X$ is a $k \times k$ symmetric matrix.} within the span of the Hessian eigenvectors whose eigenvalue is equal to $2/\eta$.\footnote{For gradient descent, the unstable eigenvectors have eigenvalues which fluctuate around $2/\eta$.  However, for the central flow, the unstable eigenvectors will have eigenvalues which are exactly equal to $2/\eta$.}

    \item Since $\Sigma(t)$ models a covariance matrix, we impose the condition that $\Sigma(t)$ should be positive semidefinite.
\end{enumerate}

These three conditions turn out to imply a unique value of $\Sigma(t)$. In particular, we detail in \Cref{appendix:derivations:gd} that $\Sigma(t)$ must be the unique solution to a type of convex program known as a \emph{semidefinite complementarity problem} (SDCP), which are described in \Cref{sec:derivations:sdcp}.\footnote{Interestingly, complementarity problems arise frequently in the the study of contact mechanics. EOS can be interpreted as the \gd trajectory making ``contact'' with the boundary of the stable region, and then sliding along the boundary.} The central flow is defined as \cref{eq:gd:multi_ansatz} with this $\Sigma(t)$:
\begin{align}
    \frac{dw}{dt} = - \eta \qty\Big[\;\nabla L(w) + \tfrac{1}{2} \nabla_w \ev{H(w), \Sigma(t)}\;] \quad \text{where} \quad \Sigma(t) \text{ solves the SDCP in \cref{eq:appendix:gd:sigmat}.}
    \label{eq:gd:full-central-flow}
    \end{align}
A formal definition for the central flow is given in \Cref{appendix:derivations:gd}, \Cref{def:gd:flow:ode}. We note that $\Sigma(t)$ can be efficiently represented numerically as it is a low rank matrix, with rank at most the number of unstable eigenvalues. 

We now elaborate on the behavior of the central flow:
\begin{enumerate}
    \item \textbf{Stable regime:} If all Hessian eigenvalues are below $2/\eta$, then the SDCP returns $\Sigma(t) = 0$, and the central flow reduces to the gradient flow. 
    \item \textbf{One unstable eigenvalue}:  If one Hessian eigenvalue is at $2/\eta$, and if the gradient flow  would \emph{increase} this eigenvalue above $2/\eta$, then our analysis reduces to that of \Cref{sec:gd:single}.  In particular, the SDCP returns a rank-one matrix of the form $\Sigma(t) = \sigma^2 \, u \, u^\top$ where $u$ is the top Hessian eigenvector at $w$, and $\sigma^2$ is defined in \cref{eq:gd:single}.  On the other hand, if the gradient flow would \emph{decrease} this eigenvalue below $2/\eta$, then $\Sigma(t) = 0$, and the central flow will follow the gradient flow out of the edge of stability.
    \item \textbf{Multiple unstable eigenvalues:} In general, the SDCP returns a $\Sigma(t)$ which constrains all Hessian eigenvalues currently at $2/\eta$ from rising above that value.  Often, this $\Sigma(t)$ causes all Hessian eigenvalues currently at $2/\eta$ to remain fixed at $2/\eta$.\footnote{There is a unique $\Sigma$ that causes all Hessian eigenvalues currently at $2/\eta$ to remain fixed at $2/\eta$, and it can be found by solving a linear inverse, generalizing \cref{eq:gd_x}.  The solution to the SDCP coincides with this $\Sigma$ at almost all times.  However, this $\Sigma$ cannot be used to define the central flow, as it would never allow an eigenvalue to leave the edge of stability, and it is not necessarily PSD.}. However, it also allows for eigenvalues to leave EOS when appropriate.
\end{enumerate}

\begin{figure}[t!]
\centering
\includegraphics[width=\textwidth]{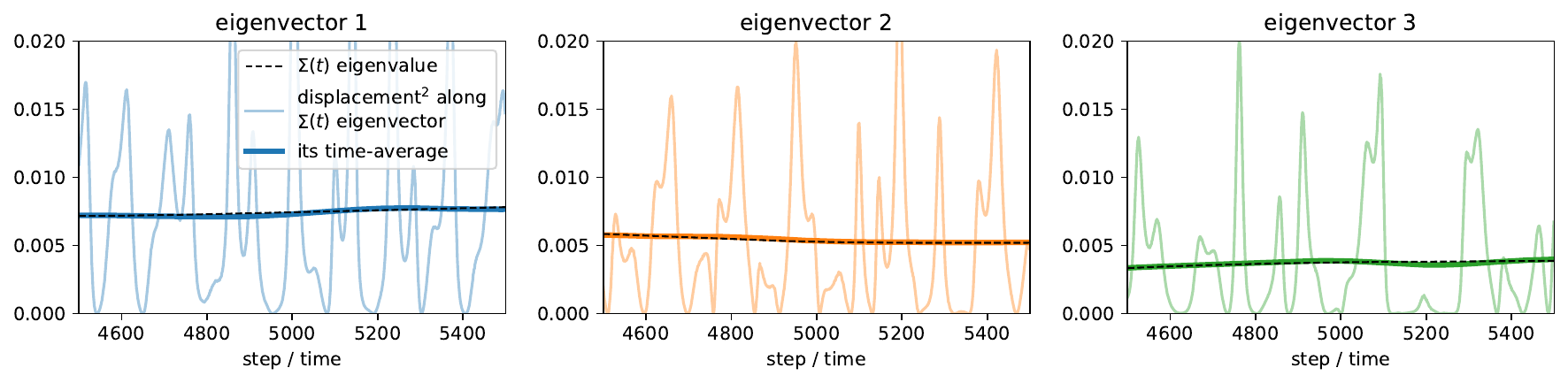}
\caption{\textbf{Central flow's $\Sigma(t)$ accurately predicts covariance of oscillations.}  We show that the central flow's $\Sigma(t)$ accurately predicts the covariance of gradient descent's oscillations.  For this stretch of training, there are 3 Hessian eigenvalues at the edge of stability, so $\Sigma(t)$ has 3 nonzero eigenvalues (subpanels).  In black, we plot each eigenvalue of $\Sigma(t)$; in colors, we plot the squared magnitude of gradient descent's displacement from the central flow along the corresponding eigenvectors (light = raw values, dark = time average using Gaussian smoothing).  Observe that each eigenvalue of $\Sigma(t)$ accurately predicts the instantaneous variance of oscillations along the corresponding eigenvector, even though these oscillations might initially appear erratic.  This figure depicts the same ViT trajectory as \Cref{fig:gradient-descent-typical}.}
\label{fig:cf-predict-oscillations}
\end{figure}
\Cref{fig:gd:demonstrate-cf-multiple} demonstrates the central flow in action.  Initially, all Hessian eigenvalues are below $2/\eta$, so $\Sigma(t) = 0$ and the central flow reduces to the gradient flow.
Once the top Hessian eigenvalue reaches $2/\eta$ around step 2900, $\Sigma(t)$ becomes a rank-one matrix, and the central flow keeps the top Hessian eigenvalue locked at $2/\eta$, as it mimics the effects of oscillating along the the top eigenvector direction.  Once the second Hessian eigenvalue also reaches $2/\eta$ around step 3600, $\Sigma(t)$ becomes a rank-two matrix, and the central flow keeps the top two eigenvalues both locked at $2/\eta$, as it mimics the effects of oscillating simultaneously along the top two eigenvector directions.  Throughout, the Euclidean distance between gradient descent's $w_t$ and the central flow's $w(t)$ stays small over time (right plot), indicating that the central flow accurately tracks the long-term trajectory of gradient descent.  In contrast, the distance between \gd and the \emph{gradient flow} \cref{eq:gflow} grows large over time.

In \Cref{fig:cf-predict-oscillations}, we show that the central flow's $\Sigma(t)$ accurately predicts the covariance with which \gd is oscillating around the central flow.
In particular, we show that each eigenvalue of $\Sigma(t)$ accurately predicts the instantaneous variance of oscillations along the corresponding eigenvector of $\Sigma(t)$.
We find it striking that our theory is able to accurately predict the covariance of these oscillations.  While the oscillations are erratic and might appear unpredictable, our findings reveal that a certain statistic --- their covariance --- is predictable after all.  Moreover, predicting this covariance seems to be sufficient to predict the long-term trajectory of \gd.

\paragraph{Interpretation as projection} The projection interpretation in \Cref{sec:gd:single} generalizes to the case of an arbitrary number of unstable eigenvalues. 
In particular, the central flow \cref{eq:gd:full-central-flow} can be written as a flow which orthogonally projects the negative gradient onto the so-called \emph{tangent cone} of the stable region $\mathbb{S} = \{w: S(w) \le 2/\eta\}$, which is the set of directions in which one can move while still staying, to first order, within the stable region:\footnote{In the interior of the stable region (i.e. $S(w) < 2/\eta$), the tangent cone is the entire space, so this projection is a no-op and the central flow reduces to the gradient flow.  When exactly one eigenvalue is at the edge of stability, the tangent cone is the half-space $\{v: \nabla S(w)^T v \le 0 \}$.}
 \begin{align}
     \frac{dw}{dt} = \eta \underbrace{\mathrm{proj}_{T_w \mathbb{S}} [- \nabla L(w)] }_{\substack{\text{project negative gradient onto} \\ \text{tangent cone $T_w \mathbb{S}$ of set $\mathbb{S}$} }} \quad \text{where} \quad  \mathbb{S} = \underbrace{\{w: S(w) \le 2/\eta\}}_{\text{stable region } \mathbb{S}}.
     \label{eq:projected-flow}
 \end{align}
A formal definition is given in \Cref{def:gd:flow:proj}.
This projection interpretation will be used in \Cref{sec:gd:interpreting} to show that the loss along the central flow decreases monotonically.

\paragraph{Where does deep learning come in?} 
Our principal claim is that, if initialized stably, gradient descent will approximately follow the central flow over the long term.
In the case where the sharpness does not rise to $2/\eta$ (e.g. on a quadratic objective, where the sharpness is constant), then the central flow reduces to the gradient flow, and so our claim reduces to the somewhat ``uninteresting'' claim that gradient descent will approximately follow the gradient flow.
The central flow only becomes nontrivial, and our claim only becomes ``interesting,'' in the event that the sharpness rises to $2/\eta$.  This empirically tends to happen on deep learning objectives.
However, we suspect that the central flow might also hold on other kinds of objectives where the sharpness rises to $2/\eta$ during gradient descent.

\subsubsection{Understanding the train loss curve and more}
\begin{figure}[b!]
\centering
\includegraphics[width=0.66\textwidth]{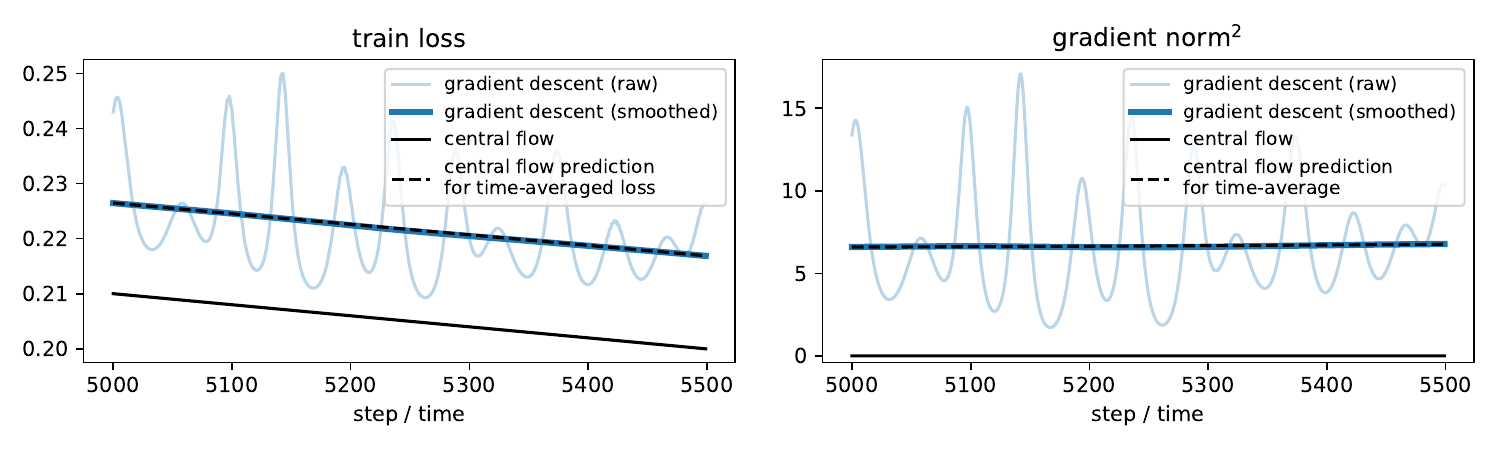}
\includegraphics[width=0.3\textwidth]{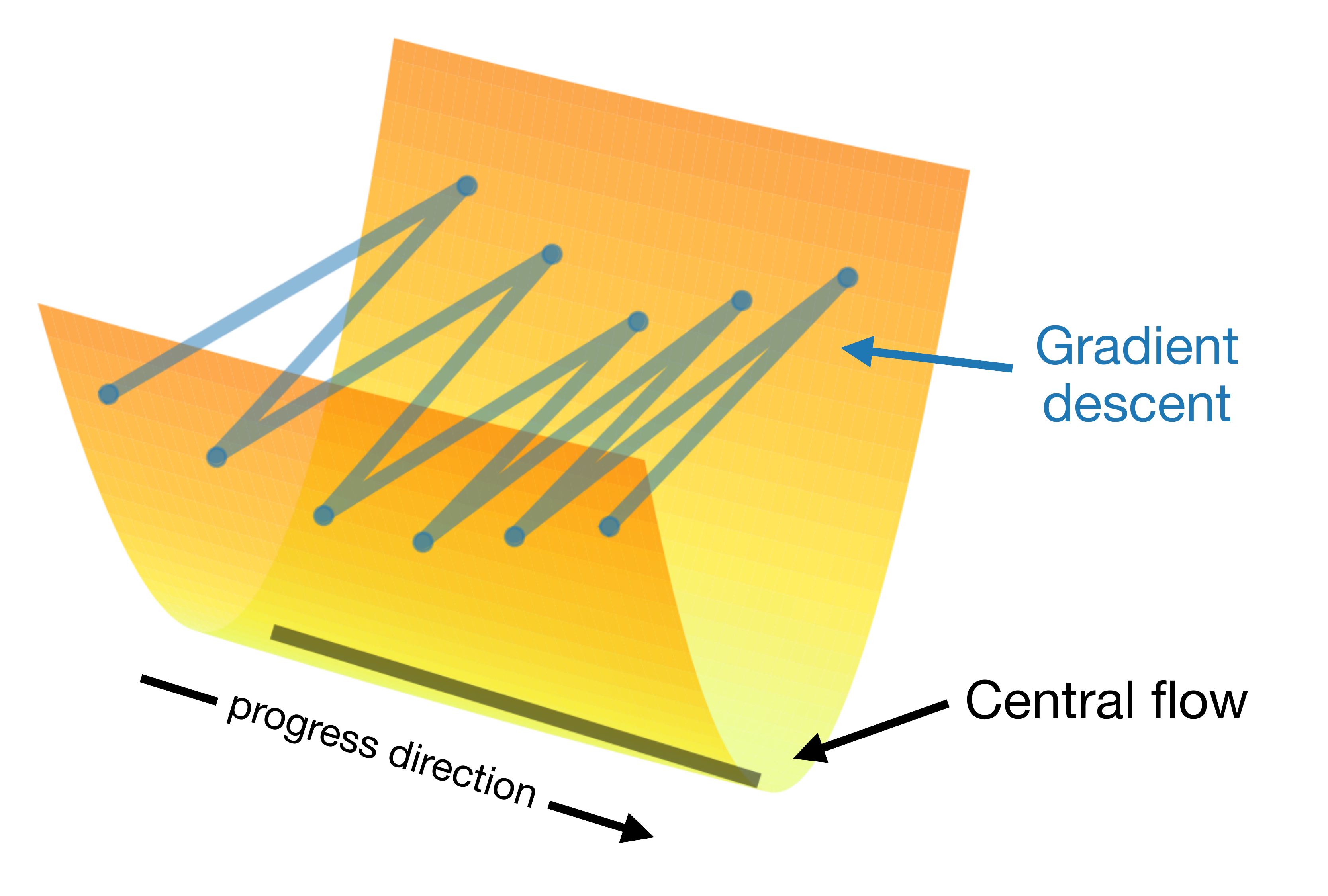}
\caption{The train loss and gradient norm$^2$ are larger along the raw \gd trajectory (light blue) than along the central flow (solid black). The intuitive reason is that an oscillatory optimizer can be visualized as oscillating between the walls of a ``valley'' (see cartoon on right).  Because the central flow models the covariance $\Sigma(t)$ of the oscillations, it can render predictions for the \emph{time-averaged} values of the loss and gradient norm$^2$ along the \gd trajectory, via \cref{eq:central-flow-predicted-loss} and \cref{eq:gd-predict-gradient-norm-sq} respectively (dashed black). These predictions are close to the empirical time-averaged values, computed via Gaussian smoothing (dark blue).  This figure depicts the same ViT as \Cref{fig:gradient-descent-typical}.
}
\label{fig:gd:cf-predictions}
\end{figure}
\Cref{fig:gd:cf-predictions} shows that the loss along the actual gradient descent trajectory is consistently \emph{higher} than the loss along the central flow.
The intuitive explanation is that when gradient descent oscillates along the top Hessian eigenvector(s), it can  be visualized as ``bouncing between valley walls'' \citep{xing2018walk, cohen2021gradient, Wen2024UnderstandingWL}, whereas the time-averaged iterates run nearly along the ``valley floor'' (called a ``river'' in \citet{Wen2024UnderstandingWL}), as illustrated on the right of \Cref{fig:gd:cf-predictions}.  The loss is higher on the valley walls, where the actual iterates are located, than on the valley floor, where the time-averaged iterates are located.\footnote{This does not mean that the loss landscape is ``locally convex'' in any deep sense (indeed, the Hessian generally has negative eigenvalues). It merely reflects that the optimizer is oscillating along directions of positive curvature.}

Fortunately, the central flow framework will still let us reason about the loss along the actual gradient descent trajectory. 
Recall that the central flow predicts not just the time-averaged iterates $w(t)$, but also the covariance of the oscillations $\Sigma(t)$.
In particular, the central flow models the \gd trajectory $\{w_t\}$ as:\footnote{We emphasize that the central flow does not model the ``distribution'' (so to speak) of the oscillations $\delta_t$.  Rather, it only models the second moment $\E[\delta_t \delta_t^T]$, under the theory that the macroscopic trajectory of \gd is completely characterized by this second moment.}
\begin{align*}
    w_t = w(t) + \delta_t, \quad \text{where} \quad \E[\delta_t] = 0 \quad \text{and} \quad \E[\delta_t \delta_t^T] = \Sigma(t).
\end{align*}
Thus, for any quantity $f(w)$ derived from the weights (e.g. loss or gradient norm), we can predict its time-averaged value $\E f(w_t)$ along the \gd trajectory $w_t$ by taking a quadratic Taylor expansion along the central flow $w(t)$, and time-averaging over $\delta_t$:
\begin{align}
\underbrace{\E[f(w_t)]}_{\substack{\text{time-averaged} \\ \text{value along trajectory}}} \approx \underbrace{f(w(t))}_{\text{value along central flow}} + \underbrace{\tfrac{1}{2} \ev{\nabla^2 f(w(t)), \Sigma(t)}}_{\text{contribution from oscillations}}.
\label{eq:central-flow-predictions}
\end{align}
For example, if $f$ is the loss $L$, then because $\Sigma(t)$ is supported on the Hessian eigenvectors with eigenvalue $2/\eta$:
\begin{align}
    \underbrace{\E[L(w_t)]}_{ \substack{\text{time-averaged} \\ \text{loss along trajectory}} } \approx \underbrace{L(w(t))}_{\text{loss along central flow}}  + \underbrace{\tfrac{1}{\eta} \, \tr \Sigma(t)}_{\text{contribution from oscillations}} := \bar L(t). \label{eq:central-flow-predicted-loss}
\end{align}
See \Cref{appendix:derivations:gd:sdcp} for an explicit derivation. \Cref{fig:gd:cf-predictions} shows that this prediction $\bar L(t)$ for the time-averaged loss closely matches the actual time-averaged loss (computed with Gaussian smoothing).
Both the central flow loss $L(w(t))$ and the predicted time-averaged loss $\bar L(t)$ model important quantities that are meaningful to DL practice:
\begin{itemize}
    \item The central flow's prediction for the time-averaged loss $\bar L(t)$ models the smoothed training loss curve, often monitored in practice by Tensorboard \citep{tensorflow2015} or Weights and Biases \citep{wandb}.
    \item The central flow loss $L(w(t))$ models the loss at the time-averaged iterate. This is similar to the loss at an exponential moving average of the weights, or the loss after annealing the learning rate \citep{sandler2023}.
\end{itemize}

The central flow perspective allows us to quantify both of these loss values and reason about them separately.

A similar point holds for other quantities,\footnote{Interestingly, for some quantities (such as the network outputs), we find that the value along the central flow is already an excellent approximation to the time-averaged value along the discrete optimizer trajectory, and \cref{eq:central-flow-predictions} is not necessary.} such as the gradient norm.  \Cref{fig:gd:cf-predictions} shows that the squared gradient norm along the central flow, $\|\nabla L(w(t))\|^2$ is much smaller than at the actual iterates, $\| \nabla L(w_t) \|^2$.  Intuitively, most of the gradient at the iterates is spent ``oscillating across the valley'' and cancels out over the long run.  The central flow's gradient is smaller because it leaves out these oscillations.  Yet, because the central flow models the covariance $\Sigma(t)$ of the oscillations, it can still predict the \emph{time-average} of the squared gradient norm at the iterates, using \cref{eq:gd-predict-gradient-norm-sq} from \Cref{appendix:derivations:gd:sdcp}. \Cref{fig:gd:cf-predictions} demonstrates the accuracy of this prediction.

\subsubsection{Empirical verification}
We empirically find that the central flow can accurately predict the long-term trajectory of gradient descent in a variety of deep learning settings.  For example, \Cref{fig:gd:archs} shows the central flow on several different deep learning settings (details in \Cref{sec:experiments}).  Observe that the central flow accurately predicts the weight-space trajectory, the covariance of the oscillations, and the time-averaged training loss curve.
\Cref{fig:experiments:gd:loss-curves-mse,fig:experiments:gd:loss-curves-ce} show that the central flow can accurately predict the time-averaged training loss curve at different learning rates across a variety of deep learning settings.  Our full set of \gd experiments can be found in \Cref{sec:bulk-experiments-gd}.

\begin{figure}[p!]
\centering
\includegraphics[width=0.95\textwidth]{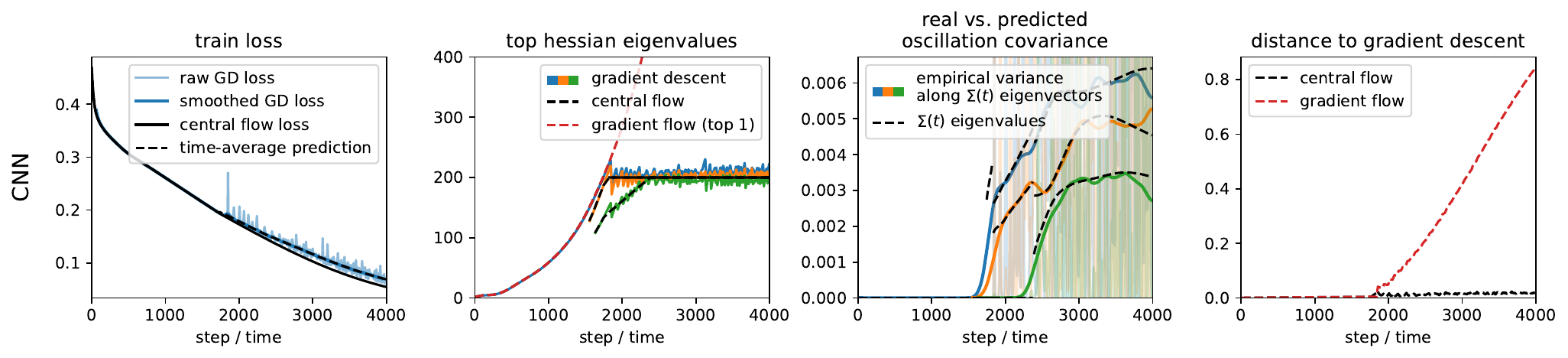}
\includegraphics[width=0.95\textwidth]{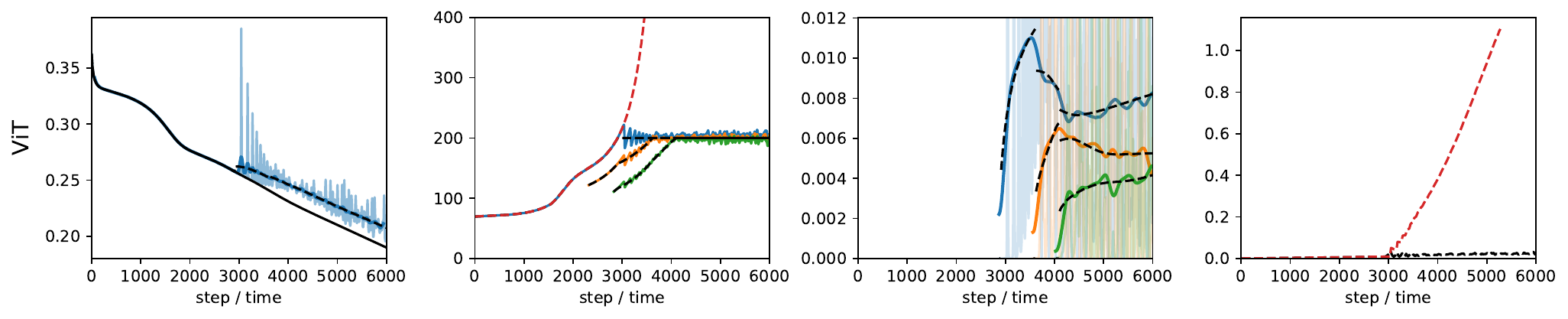}
\includegraphics[width=0.95\textwidth]{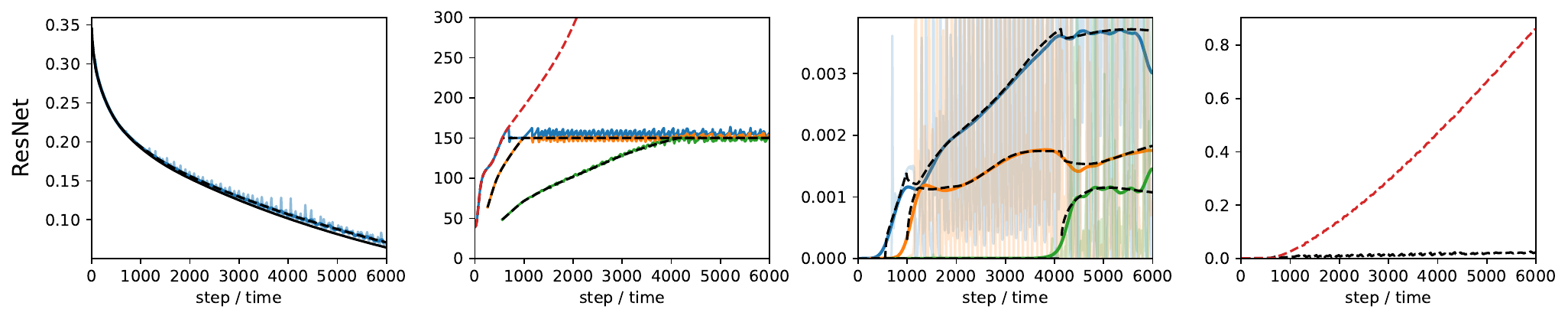}
\includegraphics[width=0.95\textwidth]{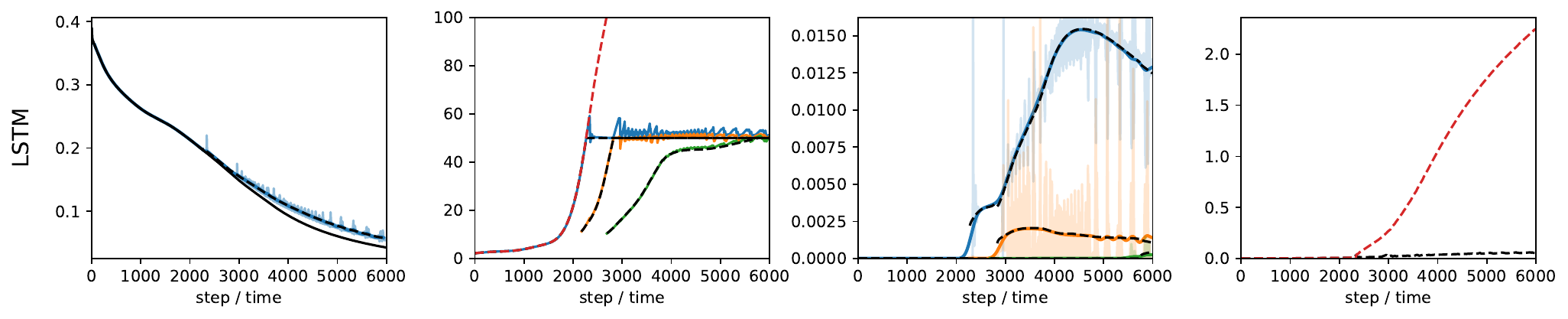}
\includegraphics[width=0.95\textwidth]{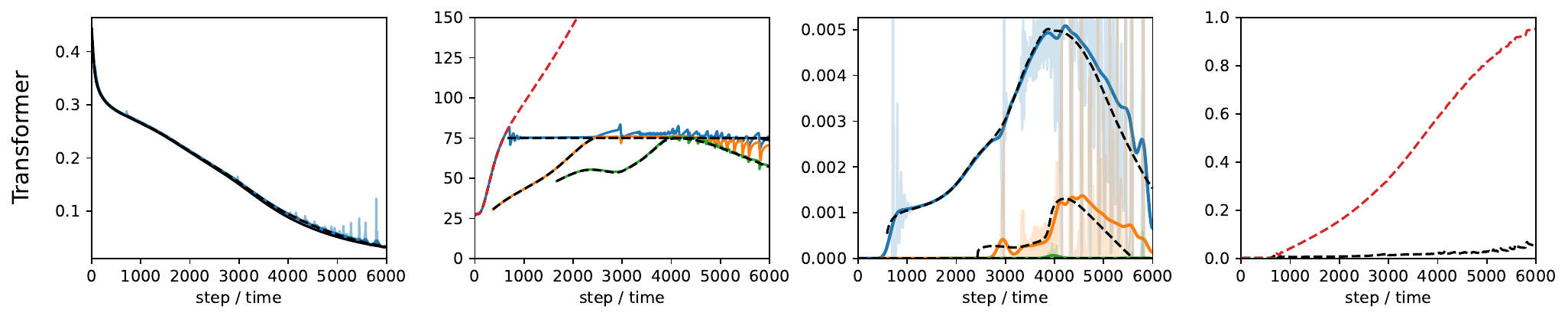}
\includegraphics[width=0.95\textwidth]{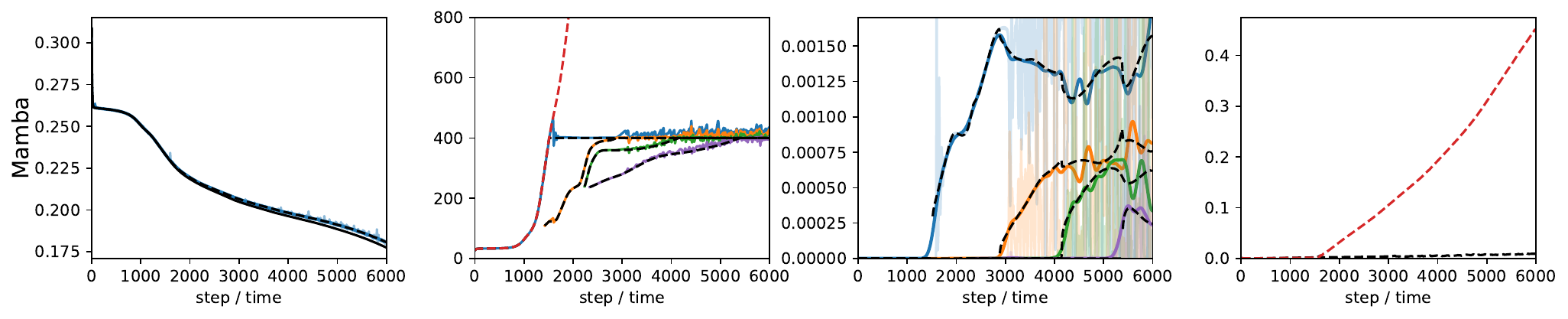}
\caption{\textbf{Verifying the gradient descent central flow across various architectures.} Across various architectures, the central flow accurately predicts the weight-space trajectory, the covariance of the oscillations, and the time-averaged loss curve.  These experiments all use MSE loss. See \Cref{sec:experiments} for more experimental details and \Cref{sec:bulk-experiments} for our full set of raw experiments.}
\label{fig:gd:archs}
\end{figure}

Nevertheless, our derivation relied on informal mathematical reasoning, and certain factors do empirically affect the quality of the central flow approximation.  First, the central flow tends to become less accurate as the learning rate $\eta$ is made increasingly large.  Second, on some deep learning problems, higher-order terms cause the central flow to slightly mispredict $\Sigma(t)$, causing error to accumulate over the long run.  Third, large spikes also can throw off the central flow.  The latter two issues empirically seem to be more common when the loss criterion is cross-entropy rather than MSE.  We discuss these points at greater length in \Cref{sec:experiments}. We hope that future work can rigorously understand the conditions under which the central flow does or does not approximate the \gd trajectory.

\subsection{Understanding Gradient Descent via its Central Flow}
\label{sec:gd:interpreting}
We have shown that the central flow is a smooth curve that characterizes the macroscopic trajectory of \gd.
We now explain why this makes it a useful theoretical tool for reasoning about optimization.

\begin{wrapfigure}[10]{R}{0.39\textwidth}
    \centering
    \vspace{-15px}
    \includegraphics[width=0.39\textwidth]{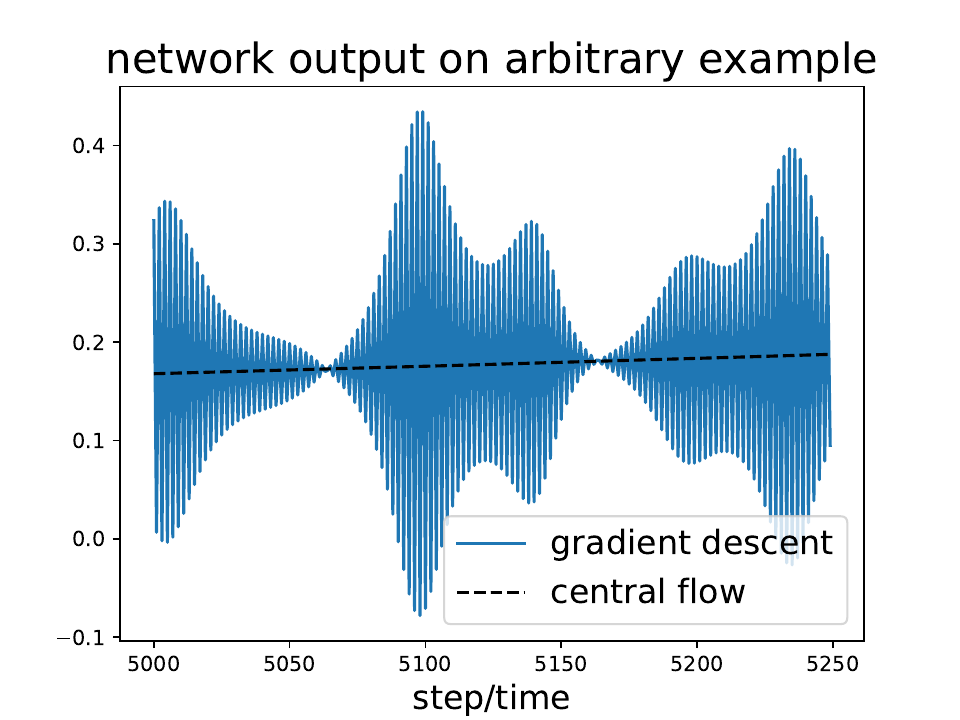}
    \label{fig:gd-interpreting-output}
\end{wrapfigure}
\paragraph{Averaging out oscillations reveals the underlying order} At the edge of stability, \gd's oscillations lead to wild fluctuations in many training-related quantities, such as the training loss and the network's predictions \citep{rosenfeld2024outliers}.
For example, the figure on the right shows the evolution under \gd of the network's prediction on an example.
One can see that along the actual gradient descent trajectory (blue), training proceeds erratically.
In contrast, the central flow (black) is a more coherent training process which makes steady, continuous progress over time.
By averaging out the oscillations, the central flow reveals the underlying order hidden beneath the chaotic oscillatory dynamics.\footnote{Arguably, the central flow can even be viewed as the ``true'' training process, with the actual \gd trajectory being merely a noisy realization of this idealized trajectory which is computationally cheap to obtain.}

\begin{wrapfigure}[11]{L}{0.34\textwidth}
    \centering
    \includegraphics[width=0.34\textwidth]{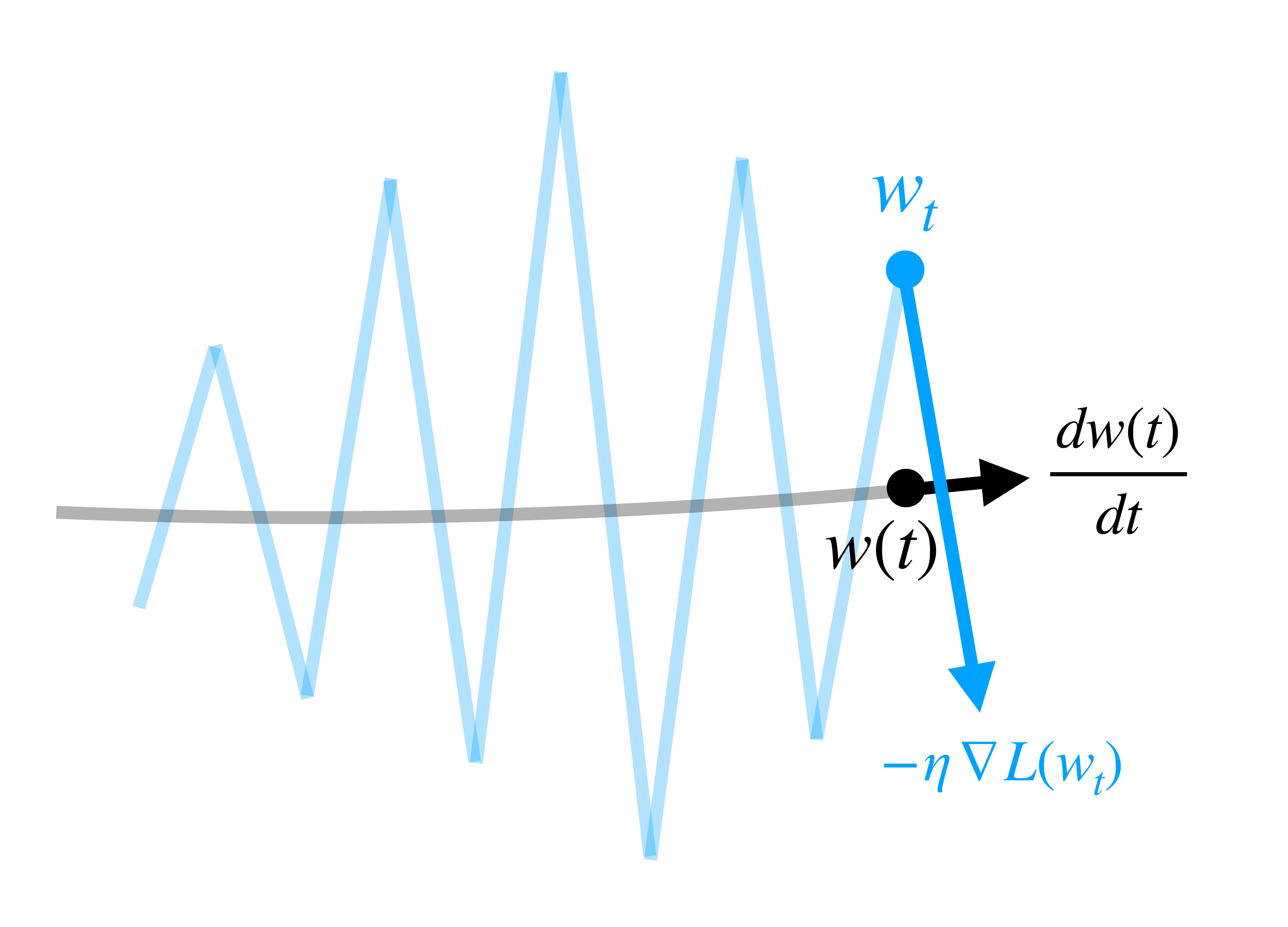}
\end{wrapfigure}
\paragraph{A smooth curve can be analyzed using calculus} Because the central flow is a smooth curve, we can leverage calculus to reason quantitatively about the dynamics of training.  Crucially, along the central flow, the time derivative $\tfrac{d w(t)}{dt}$ meaningfully reflects the optimizer's direction of motion over the near term (see cartoon on left).  In comparison, along the \gd trajectory, the analogous update $- \eta \nabla L(w_t)$ is dominated by oscillations and hence does \emph{not} meaningfully reflect the direction of motion over the near term --- only over the current step.

For any quantity $f(w)$ derived from the weights $w$, we can use the chain rule to compute its rate of change under the central flow: $\tfrac{df}{dt} = \langle \nabla f(w), \tfrac{dw}{dt} \rangle$. We will now use this to reason about the rate of loss decrease.

\begin{wrapfigure}[10]{R}{0.38\textwidth}
    \centering
    \vspace{-15px}
    \includegraphics[width=0.37\textwidth]{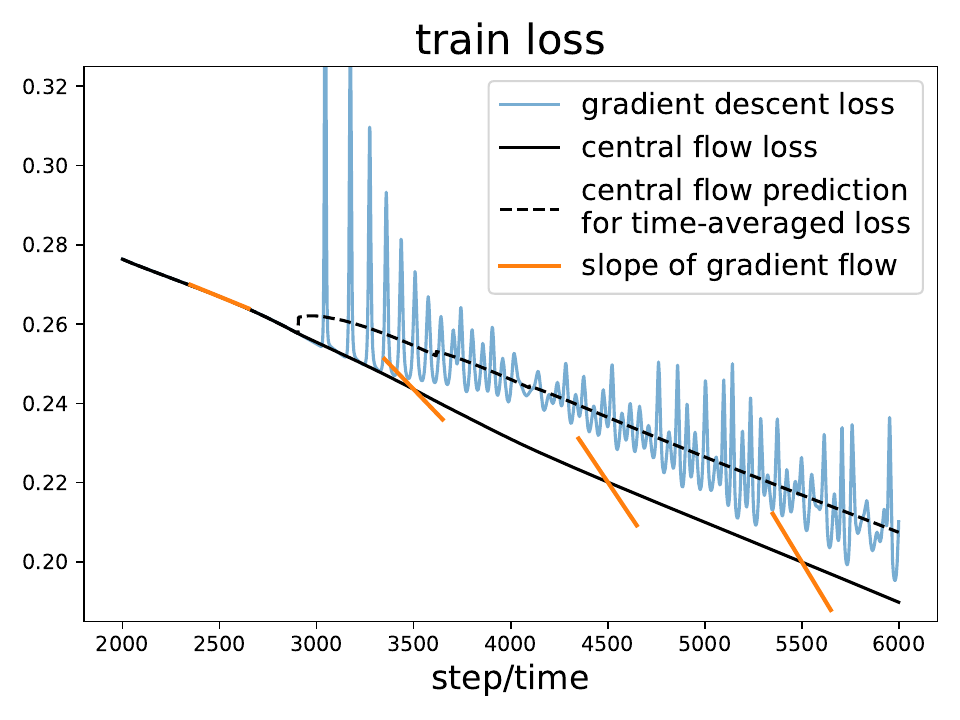}
    \label{fig:gd-interpreting-trainloss}
\end{wrapfigure}
\paragraph{Reasoning about training loss curve}
Consider the most basic question one can ask about an optimization algorithm: how fast is the loss going down?
For the ``raw'' gradient descent trajectory, the loss \emph{doesn't} always go down --- instead, the loss behaves non-monotonically over short timescales, while only decreasing over long timescales.
Thus, reasoning about the rate of loss decrease is challenging.
In contrast, under the central flow, the loss evolves smoothly, and its rate of decrease can be quantified using the chain rule: 
$\frac{dL(w)}{dt} = \ev{\nabla L(w), \frac{dw}{dt}}$.  Combining this with the projection interpretation (\Cref{def:gd:flow:proj}), one can easily prove that the loss along the central flow $L(w(t))$ is monotonically decreasing.  In other words, the central flow loss is a valid \textbf{potential function} for the optimization process:
\begin{proposition}
    Under the central flow $w(t)$, we have $\frac{d}{dt} L(w(t)) \le 0$.
    \label{prop:gd:loss-decrease}
\end{proposition}
See \Cref{prop:gd:loss-decrease-restated} for the proof.  The intuition is that, even after the negative gradient is projected onto the tangent cone of the stable region, it will still be negatively aligned with the gradient.

While averaging out the oscillations yields a central flow with a smoothly decreasing loss curve, the oscillations still have an effect on this loss curve through their implicit curvature reduction effect, which can be shown to slow down training. In particular, whereas the unregularized gradient flow \cref{eq:gflow} decreases the loss at the speed $\frac{dL}{dt} = - \eta \|\nabla L(w)\|^2$, it is straightforward to show that the central flow optimizes at a slower speed:
\vspace{5px}
\begin{proposition}
    Under the central flow $w(t)$, we have $\tfrac{d}{dt} L(w(t)) \ge - \eta \|\nabla L(w(t))\|^2$.
    \label{prop:gd:slowdown}
\end{proposition}
See \Cref{prop:gd:slowdown-restated} for the proof. The intuition is that because the central flow projects out the components of the loss gradient that would cause the sharpness to rise above $2/\eta$, it has less gradient available with which to decrease the loss.  This effect is illustrated in the figure on the previous page, which shows that at various points during training, the slope of the central flow loss curve is less steep than the rate of loss decrease under the gradient flow.  

\begin{wrapfigure}[9]{R}{0.35\textwidth}
    \centering
    \includegraphics[width=0.33\textwidth]{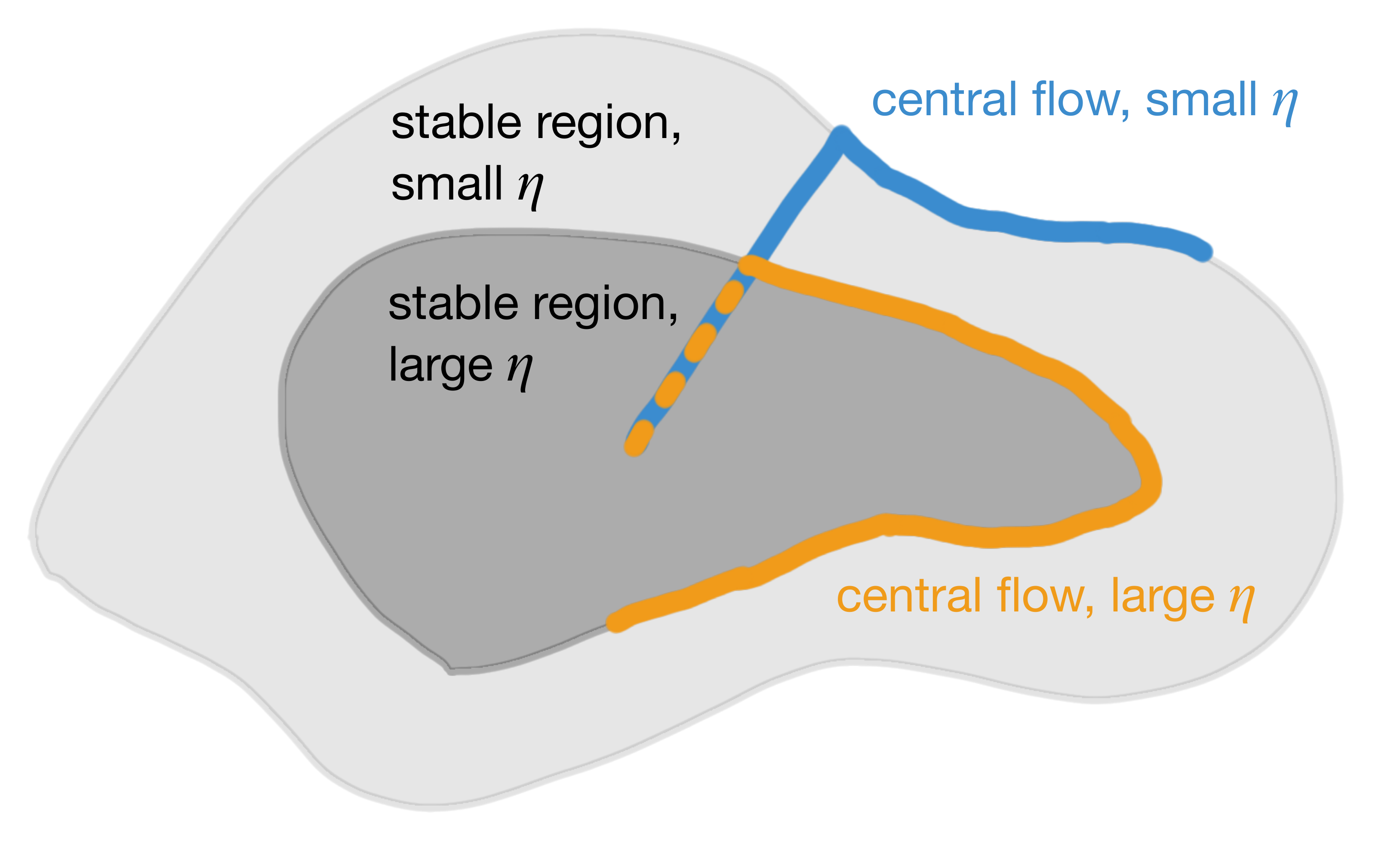}
    \label{fig:gd-hyperparameters}
\end{wrapfigure}
\paragraph{Understanding the effect of hyperparameters}
A notorious peculiarity of deep learning is that optimizer hyperparameters affect not just the speed of training, but also the particular path that the optimizer takes through weight space \citep[e.g.][]{keskar2017largebatch, jastrzębski2018on}. As a result, these hyperparameters can affect many properties of the final learned model, including its robustness and generalization.\footnote{Large learning rates are necessary for obtaining good generalization in some deep learning settings \citep[e.g.][]{li2019towards}.  However, obtaining the best generalization performance usually also requires stochastic optimization with a sufficiently small batch size.  Since our paper exclusively studies the deterministic setting, we decided to not focus on generalization in this paper.}  Such effects are \emph{implicit} in the \gd update \cref{eq:gd}.  In contrast, the central flow renders \emph{explicit} all effects of the learning rate hyperparameter $\eta$ on the optimization process, allowing us to disentangle these effects from one another.  Recall from \cref{eq:projected-flow} that the central flow is a projected gradient flow with learning rate $\eta$ that is constrained to the stable region $\mathbb{S} = \{w : S(w) \le 2/\eta\}$.  From this characterization, we see that the learning rate hyperparameter $\eta$ has two distinct effects on the central flow: (1) it acts as a time rescaling, which controls the speed of optimization without affecting the overall trajectory; and (2) it determines the stable region, which affects the overall trajectory.   Thus, increasing the learning rate constrains gradient descent to a smaller subset of weight space, but also allows it to traverse this set at a faster speed.\footnote{The learning rate $\eta$ that is optimal from an optimization perspective (i.e. that will decrease the loss the fastest) will depend on the trade-off between these two effects.   Empirically, we observe that for deterministic gradient descent, larger learning rates usually optimize faster (provided that training does not diverge), implying that the former effect is stronger.}

Having introduced the central flows framework with an analysis of \gd, we will now use this methodology to understand the behavior of two adaptive optimizers.

%% file: 4_scalar_rmsprop.tex
\newpage
\section{Scalar RMSProp}
\label{sec:rmsprop_norm}

As a stepping stone to the analysis of \rmsprop in \Cref{sec:rmsprop}, we first study ``\rmsnorm,'' a simplification of \rmsprop which uses one global step size, rather than separate step sizes for each coordinate:\footnote{Note that we have re-indexed $\nu$ compared to the standard definition of RMSProp (i.e. $\nu_{t+1} \to \nu_t$). This does not affect the trajectory and just ensures the effective learning rate at step $t$ is determined by $\nu_t$, rather than $\nu_{t+1}$, which simplifies the notation.}\footnote{This algorithm was also studied by \citet{lyu2022understanding}. However, their analysis only applies along a manifold of global minima, as $\eta \to 0$.}
\begin{align}
    \nu_t &= \beta_2 \nu_{t-1} + (1 - \beta_2) \| \nabla L(w_t) \|^2 \qc
    w_{t+1} = w_t - \frac{\eta}{\sqrt{\nu_{t}}} \nabla L(w_t). \qquad\qquad\qquad
\label{eq:rmsprop_norm}
\end{align}
The algorithm maintains an exponential moving average (EMA), $\nu$, of the squared gradient norm, and takes gradient steps of size $\eta / \sqrt{\nu}$, which we call the \emph{effective step size}.\footnote{The terms ``learning rate'' and ``step size'' are usually interchangeable.  In this paper, to avoid ambiguity, we will use the phrase ``learning rate'' to denote the hyperparameter, and ``step size'' or ``effective step size'' to denote the actual step sizes that are taken.}
The EMA hyperparameter $\beta_2$ is a knob that interpolates the algorithm between \gd when $\beta_2 = 1$ and normalized gradient descent (NGD) when $\beta_2 = 0$.\footnote{When $\beta_2 = 1$, \rmsnorm reduces to \gd with learning rate $\eta / \sqrt{\nu_0}$.
Conversely, when $\beta_2 = 0$, it reduces to \ngd with learning rate $\eta$: $w_{t+1} = w_t - \eta \cdot \frac{\nabla L(w_t)}{\| \nabla L(w_t) \|}$.}

While optimizers such as \rmsnorm are often said to utilize an ``adaptive step size,'' it has remained unclear what precise property of the local landscape the step size is being adapted \emph{to} \citep{orabona2020neural}.  In this section, we will use the central flows framework to answer this basic question.
After describing the dynamics of \rmsnorm in \Cref{sec:rmsnorm:mechanics} and deriving a central flow in \Cref{sec:rmsnorm:deriving}, we will interpret this flow to understand the optimizer's behavior in \Cref{sec:rmsnorm:interpreting}.  In particular:

\begin{itemize}
    \item In \Cref{sec:rmsnorm:interpreting:adapt}, we make precise how \rmsnorm adapts its step size to the local loss landscape.  Specifically, we show that the optimizer's dynamics implicitly set the effective step size to the value $2/S(w)$, where $S(w)$ is the current sharpness; this value is the \emph{largest stable step size} at the current weights $w$. 
    \item In \Cref{sec:rmsnorm:interpreting:regularize}, we show that step size adaptation is not the full story: \rmsnorm also implicitly regularizes curvature throughout training, and in fact,  at EOS, the hyperparameters $\eta, \beta_2$ only affect the time-averaged trajectory by modulating the strength of this curvature regularization.
    \item Bringing it all together, in \Cref{sec:rmsnorm:interpreting:acceleration} we describe how the interplay between step size adaptation and curvature regularization gives rise to a mechanism we call \emph{acceleration via regularization}, whereby the optimizer implicitly steers itself towards low-curvature regions where it can take larger steps. We show that this mechanism is key to the efficacy of \rmsnorm and to the function of its hyperparameters.
\end{itemize}

These points will generalize to \rmsprop in \Cref{sec:rmsprop}, but are simpler to understand for \rmsnorm.

\begin{figure}[b!]
\centering
\includegraphics[width=\textwidth]{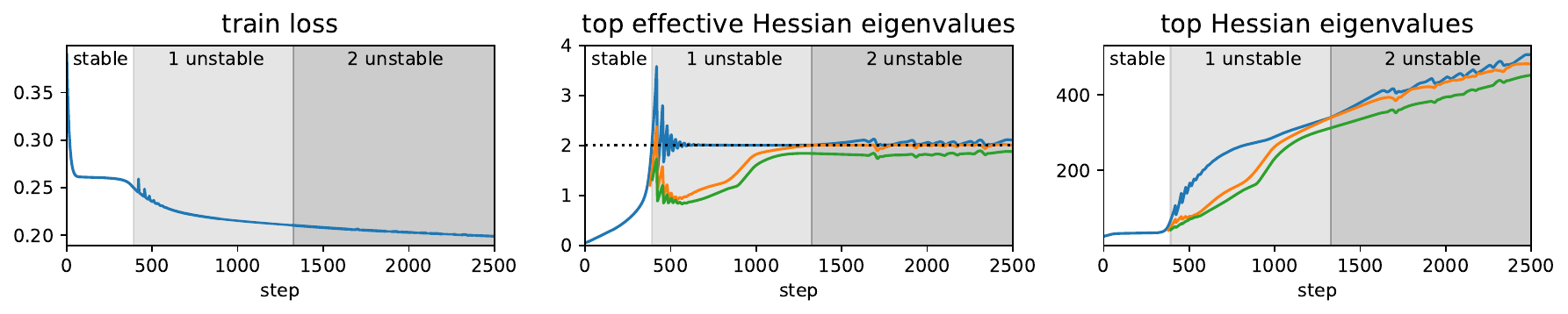}
\caption{\textbf{A typical Scalar RMSProp trajectory}.  We train a Mamba network on a sequence task using \rmsnorm with $\eta = 2/400$ and $\beta_2  = 0.99$.  While the top eigenvalues of the ``raw'' Hessian $H(w)$ evolve freely (right), the top eigenvalue of the \emph{effective} Hessian $\eta H(w)/\sqrt{\nu}$ equilibrates at the critical threshold $2$ (center). }
\label{fig:rmsprop-norm-intro}
\end{figure}

\subsection{The Dynamics of Scalar RMSProp}\label{sec:rmsnorm:mechanics}
The dynamics of \rmsnorm revolve around the \emph{effective sharpness}, defined as $\Seff := \eta S(w)/\sqrt{\nu}$.\footnote{While we could have defined effective sharpness as $S(w)/ \sqrt{\nu}$ so that it would equilibrate at $2/\eta$, this version makes the analysis easier.}
First, the effective sharpness controls the oscillations: when $\Seff > 2$, \rmsnorm oscillates with growing magnitude along high curvature direction(s).
Second, such oscillations in turn trigger a reduction of effective sharpness.
This occurs via a combination of two distinct mechanisms.
One mechanism, shared with \gd, is that oscillations implicitly reduce sharpness due to \cref{eq:gd:cubic_taylor}, thereby decreasing the effective sharpness via its \emph{numerator}.
The other mechanism, new to \rmsnorm, is that oscillations increase the gradient norm and hence $\nu$, thereby decreasing effective sharpness via its \emph{denominator}.
These dynamics give rise to a negative feedback loop that keeps the effective sharpness automatically regulated around the value $2$, as depicted in \Cref{fig:rmsprop-norm-intro}.
The fine-grained dynamics are complex and challenging to analyze, even in the case of a single unstable eigenvalue.
Fortunately, we will see in the next section that analyzing the \emph{time-averaged} dynamics is much simpler.

\subsection{Deriving the \protect\rmsnorm Central Flow}\label{sec:rmsnorm:deriving}

Recall that while \gd trains stably, it is well-approximated by \gflow.
One can derive an analogous ``stable flow'' for \rmsnorm 
\citep[cf.]{ma2022qualitative}:\footnote{\label{footnote:beta2} The $1-\beta_2 \to \frac{1-\beta_2}{\beta_2}$ correction is necessary for small values of $\beta_2$. For example, when $\beta_2 = 0$ (i.e. \ngd), $\nu_{t} = \norm{\nabla L(w_t)}^2$ so in the continuous time ODE, $\nu(t)$ needs to adapt ``instantly'' to $\norm{\nabla L(w(t))}^2$. See \Cref{sec:continuous-time-ema} for additional justification for this correction term.}
\begin{align}
    \frac{dw}{dt} = -\frac{\eta}{\sqrt{\nu}} \nabla L(w) ,\quad\quad \frac{d\nu}{dt} = \frac{1-\beta_2}{\beta_2}\qty\big[\norm{\nabla L(w)}^2 - \nu] \label{eq:rmsprop_norm_stable_flow}.
\end{align}
However, at the edge of stability, the trajectory of \rmsnorm deviates from \cref{eq:rmsprop_norm_stable_flow}. We will now derive a more general \emph{central flow} that characterizes the time-averaged trajectory even at EOS. In the main text, we will focus on the special case where one eigenvalue is, and remains at, the edge of stability. See \Cref{appendix:derivations:rmsprop_norm} for our full derivation which accounts for multiple eigenvalues at EOS and for eigenvalues entering and leaving EOS.

In \Cref{sec:gd:single}, we derived an approximation for the time-averaged gradient, $\E[\nabla L(w_t)]$. Using the first two terms of \cref{eq:gd:cubic_taylor}, we can also derive a time-averaged approximation for the squared gradient norm $\E[\norm{\nabla L(w_t)}^2]$:
\begin{align*}
    \E[\norm{\nabla L(w_t)}^2] \approx \norm{\nabla L(\overline{w}_t)}^2 + \cancel{2\ev{\nabla L(\overline{w}_t), u} S(\overline{w}_t) \E[x_t]} +  S(\overline{w}_t)^2 \E[x_t^2]
\end{align*}
where we again used $\E[x_t] = 0$ to ignore the middle term.
This calculation makes clear that larger oscillations (i.e. higher $\E[x_t^2]$) increase the squared gradient norm on average over time.
Based on these time averages, we make the ansatz that the joint dynamics of $(w_t,\nu_{t})$ follow a central flow $(w(t),\nu(t))$ of the form:
\begin{align}
    \frac{dw}{dt} = - \frac{\eta}{\sqrt{\nu}} \qty\Big[\underbrace{\nabla L(w) + \tfrac{1}{2} \sigma^2(t) \nabla S(w)}_{\E[\nabla L(w_t)]}] ,\quad\quad
    \frac{d\nu}{dt} = \frac{1-\beta_2}{\beta_2}\qty\Big[\underbrace{\norm{\nabla L(w)}^2 + S(w)^2 \sigma^2(t)}_{\E[\norm{\nabla L(w_t)}^2]} - \nu],
\label{eq:rmsprop_norm_flow}
\end{align}
where $\sigma^2(t)$ is a still-unknown quantity intended to model $\E[x_t^2]$, the instantaneous variance of the oscillations.
As in our analysis of \gd, there is a unique value of $\sigma^2(t)$ that maintains $\Seff(w,\nu) = 2$. To compute it, we expand $\frac{d\Seff}{dt}$ using the chain rule:
$
    \frac{d\Seff}{dt} = \langle \frac{\partial \Seff}{\partial w}, \frac{dw}{dt}\rangle + \frac{\partial \Seff}{\partial \nu} \cdot \frac{d\nu}{dt}
$.
Plugging in $\frac{dw}{dt}, \frac{d\nu}{dt}$ from \cref{eq:rmsprop_norm_flow} shows that $\frac{d\Seff}{dt}$ is linear in $\sigma^2$. Thus, there is a unique value of $\sigma^2$ that will ensure  $\frac{d\Seff}{dt} = 0$, which is given by:
\begin{align}
    \sigma^2(w;\eta,\beta_2) = \frac{\beta_2 \overbrace{\ev{-\nabla L(w), \nabla S(w)}}^{\mathclap{\text{progressive sharpening}}} + (1{-}\beta_2) \overbrace{\qty[S(w)^2/4 - \|\nabla L(w)\|^2/\eta^2]}^{\mathclap{\text{effect of mean reversion on $\nu$}}}}{\beta_2\underbrace{\tfrac{1}{2} \norm{\nabla S(w)}^2}_{\mathclap{\text{sharpness reduction}}} + (1{-}\beta_2) \underbrace{S(w)^2/\eta^2}_{\mathclap{\text{effect of oscillation on $\nu$}}}}.\label{eq:rmsprop_norm_x}
\end{align}
The central flow for \rmsnorm in the special case of one unstable eigenvalue is given by \cref{eq:rmsprop_norm_flow} with this value of $\sigma^2$.\footnote{The \rmsnorm central flow can be interpreted as a projected flow in the augmented space $(w,\nu)$ under a certain non-Euclidean norm. However, because this flow is not a gradient flow, it does not immediately suggest a decreasing potential function for \rmsnorm.}
The fully general central flow is given in \Cref{appendix:derivations:rmsprop_norm}, \Cref{def:rmsnorm:flow:ode}.
\Cref{fig:rmsprop-norm} illustrates how this central flow can accurately predict the long-term trajectory of \rmsnorm as well as the covariance with which \rmsnorm oscillates around that trajectory.  \Cref{fig:experiments:scalar-rmsprop:loss-curves-mse,fig:experiments:scalar-rmsprop:loss-curves-ce} show, in a variety of deep learning settings, that this central flow can accurately predict the loss curve of \rmsnorm at different learning rates.  \Cref{fig:experiments:scalar-rmsprop:beta2} shows that the central flow holds across different values of $\beta_2$.  See \Cref{sec:bulk-experiments-scalar-rmsprop} for the full set of raw experiments.

\begin{figure}[t]
\centering
\includegraphics[width=\textwidth]{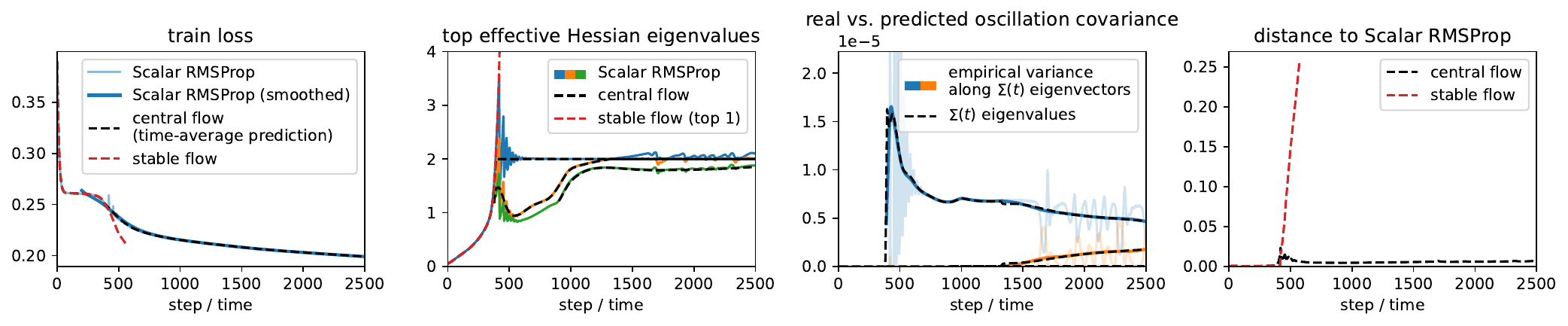}
\caption{\textbf{Central flow for Scalar RMSProp.} The central flow (black) accurately models the time-averaged trajectory of \protect\rmsnorm even at the edge of stability, whereas the naive stable flow (red) follows a different path. As with gradient descent, our analysis can accurately predict the covariance $\Sigma(t)$ with which \protect\rmsnorm oscillates around the central flow (third panel).  The setting is the same as \Cref{fig:rmsprop-norm-intro}.}
\label{fig:rmsprop-norm}
\end{figure}

The analysis in this section highlights the potential of our time-averaging methodology. With just a single invocation of the chain rule, we have characterized the long-term trajectory of a complex dynamical system involving mutual interactions between the oscillations, the sharpness, and the adaptive step size.

\subsection{Understanding Scalar RMSProp via its Central Flow}\label{sec:rmsnorm:interpreting}

We now interpret the \rmsnorm central flow to shed light on the behavior of the algorithm and the function of its hyperparameters $\eta$ and $\beta_2$.
Because the dynamics usually transition from stable to EOS quite early in training, we focus on interpreting the central flow in the EOS regime.\footnote{In the stable regime ($\Seff < 2$), the central flow is given by the stable flow \cref{eq:rmsprop_norm_stable_flow}. For this flow, $\frac{dw}{dt}$ is directly proportional to that of gradient flow, implying these flows traverse the same trajectory, just at a different speed (i.e. with a nonlinear time-rescaling). In this regime, the effective step size generally increases monotonically, so \rmsnorm follows \gflow with a learning rate warmup.}

\subsubsection{Implicit step size selection} \label{sec:rmsnorm:interpreting:adapt}
\begin{wrapfigure}[12]{r}{0.38\textwidth}
  \vspace{-18px}
  \includegraphics[width=\linewidth]{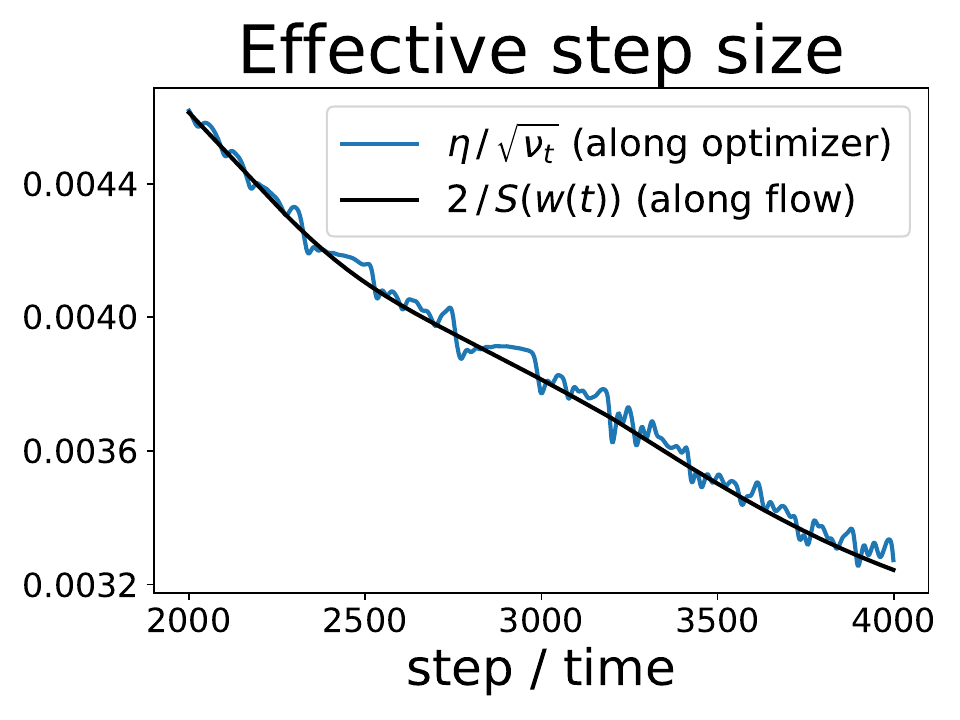}
\label{fig:rmsnorm-ess}
\end{wrapfigure}
The central flow renders \emph{explicit} the step size strategy that is \emph{implicit} in the oscillatory dynamics of \rmsnorm.
Recall that while the central flow is at EOS, the effective sharpness $\Seff := \eta S(w)/\sqrt{\nu}$ is fixed at $2$.  Indeed, this is the equilibrium condition that is automatically maintained by the dynamics of optimization.  This EOS condition can be rearranged into a statement about the effective step size:
\begin{align}
    \eta/\sqrt{\nu} = 2/S(w).
\end{align}
That is, at EOS, the effective step size along the central flow is always equal to the value $2/S(w)$.
Notably, the value $2/S(w)$ is the \emph{largest stable step size} for gradient descent at location $w$.  Thus, while \rmsnorm is at EOS, \textbf{the oscillatory dynamics continually adapt the effective step size to the current largest stable step size}, even as this value evolves throughout training. This is the precise sense in which \rmsnorm ``adapts'' its step size to the local loss landscape.

In principle, it would be possible for an optimizer to \emph{manually} compute the sharpness $S(w)$ at each iteration (e.g. by using the power method), and to manually set the step size to $2/S(w)$.  However, computing the sharpness would incur some computational overhead, whereas we have shown that \rmsnorm finds the maximum stable step size of $2/S(w)$ \emph{efficiently}, using no more computation than is already used by gradient descent (namely, one gradient computation per iteration).  This rich behavior is implicit in the algorithm's oscillatory dynamics.

Furthermore, note that even comprehending this behavior requires an appeal to some notion of time-averaging.  The effective step size is usually not \emph{exactly} at $2/S(w)$, but rather is fluctuating around $2/S(w)$.  The important point is that it is $2/S(w)$ on average over time.  The central flow perspective gives a way to reason about this behavior.

\subsubsection{Implicit curvature reduction} \label{sec:rmsnorm:interpreting:regularize}

Understanding the implicit step size strategy employed by \rmsnorm is not sufficient to fully characterize the behavior of the algorithm. To do so, we need to return to the central flow, which additionally accounts for the curvature regularization induced by oscillations.
In general, the \rmsnorm central flow is a joint flow over $(w, \nu)$. However, at EOS, because $\eta/\sqrt{\nu} = 2/S(w) \iff \nu = \tfrac{\eta^2 S(w)^2}{4}$, we can eliminate $\nu$ from the expression for $\frac{dw}{dt}$, and write the central flow in terms of $w$ alone:
\begin{align}
    \frac{dw}{dt} = - \underbrace{\frac{2}{S(w)}}_{\mathclap{\text{effective step size}}} \bigl[ \nabla L(w) + \underbrace{\frac{1}{2} \sigma^2(w;\eta,\beta_2) \nabla S(w)}_{\mathclap{\text{implicit sharpness penalty}}} \bigr]
    \label{eq:rmsnorm_central_flow_one_unstable}
\end{align}
where $\sigma^2(w;\eta,\beta_2)$ is given by \cref{eq:rmsprop_norm_x}. In other words, the time-averaged trajectory of \rmsnorm at EOS is essentially equivalent to that of the following simpler-to-understand algorithm: 

\begin{tcolorbox}[left=2pt,right=2pt,top=3pt,bottom=3pt,colback=white!0,halign=center]
At each iteration, compute the sharpness $S(w)$, and take a gradient step of size $2/S(w)$ on a sharpness-regularized objective, where the strength of the sharpness regularizer is given by \cref{eq:rmsprop_norm_x}.
\end{tcolorbox}

Interestingly, the hyperparameters $\eta,\beta_2$ are not used to determine the effective step size $2/S(w)$. Instead, their only role is to modulate $\sigma^2$, which controls the strength of the implicit sharpness penalty. The effect of the learning rate hyperparameter $\eta$ is to \emph{monotonically increase} $\sigma^2$ --- indeed, the numerator of \cref{eq:rmsprop_norm_x} is increasing in $\eta$ while the denominator is decreasing in $\eta$, which implies the overall expression for $\sigma^2$ is increasing in $\eta$. The simplest case is that of NGD, i.e. when $\beta_2 = 0$, for which \cref{eq:rmsprop_norm_x} reduces to $\sigma^2 \approx \frac{\eta^2}{4}$ (see \Cref{appendix:derivations:rmsprop_norm}). Meanwhile, the effect of the hyperparameter $\beta_2$ is to monotonically interpolate $\sigma^2$ between that of \ngd when $\beta_2 = 0$ and that of \gd when $\beta_2 = 1$.\footnote{We note that which of these is larger is situation dependent, so $\sigma^2$ can be either monotonically increasing or monotonically decreasing in $\beta_2$.  That said, because when $\beta_2 = 0$, $\sigma^2(w;\eta,0) \approx \eta^2/4$ and when $\beta_2 = 1$, $\sigma^2(w;\eta,1)$ is independent of $\eta$, a general rule is that for small learning rates, $\sigma^2$ is monotonically increasing in $\beta_2$, while for large learning rates, $\sigma^2$ is monotonically decreasing in $\beta_2$.}  The interpretations of $\eta,\beta_2$ generalize to the setting of multiple unstable eigenvalues, as detailed in \Cref{appendix:derivations:rmsprop_norm}, \Cref{prop:rmsnorm_multiple_monotonic}.

\begin{figure}[t!]
\centering
\includegraphics[width=\textwidth]{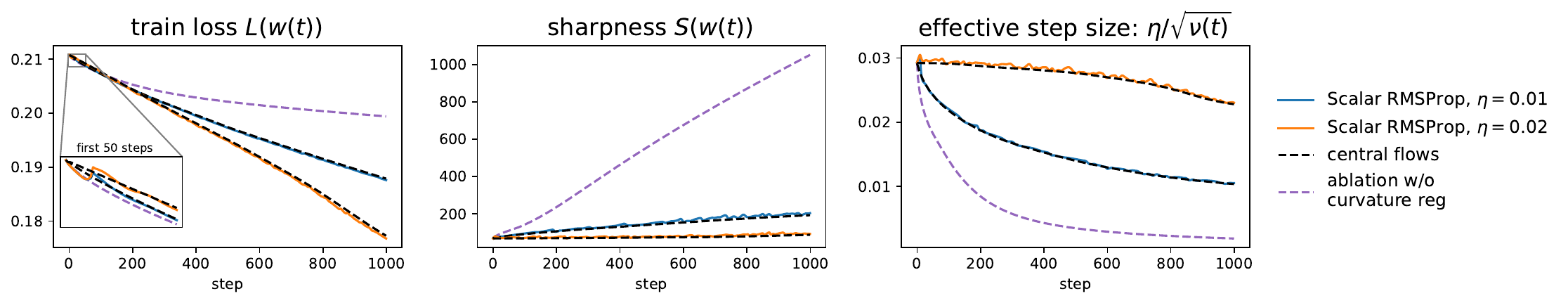}
\caption{\textbf{Implicit curvature regularization accelerates optimization for \protect \rmsnorm}.  Starting from the same initial point, we run \protect \rmsnorm at two different learning rates (blue and orange), alongside the corresponding central flows (black).  We also run an ablated flow $\frac{dw}{dt} = - \frac{2}{S(w)} \nabla L(w)$ which has curvature regularization removed (purple). All three flows use the same step size strategy, and differ only in the strength of implicit curvature regularization.  Initially (see inset), the flows with higher curvature regularization optimize slower; however, over the longer run, they take larger steps and optimize faster.  This figure is in the same setting as \Cref{fig:rmsprop-norm}.\protect\footnotemark}
\label{fig:rmsprop-norm-learning-rate}
\end{figure}

\subsubsection{Acceleration via regularization}\label{sec:rmsnorm:interpreting:acceleration}
To fully grasp the \emph{modus operandi} of \rmsnorm, it is necessary to consider the link between step size adaptation and curvature regularization.
By regularizing sharpness $S(w)$, \rmsnorm is able to steer itself towards regions where the maximal locally stable step size of $2/S(w)$ is larger.
In such regions, \rmsnorm can and does take larger steps.
Thus, \textbf{by regularizing sharpness, \rmsnorm enables itself to take larger steps later in training}. We call this mechanism \emph{acceleration via regularization}.
Our experiments suggest that this mechanism is a critical component of the algorithm's effectiveness. 
In \Cref{fig:rmsprop-norm-learning-rate}, we compare the \rmsnorm central flow to an ablated version which adapts the step size to $2/S(w)$ but does not regularize sharpness.
Over the long term, this ablated flow optimizes slower than the \rmsnorm central flow, because it traverses sharper regions of weight space in which it is forced to take smaller steps. (See \Cref{sec:supplementary-figures}, \Cref{fig:rmsnorm-acc-by-reg} for more settings.)

The mechanism of ``acceleration via regularization'' is also key for understanding the function of the learning rate hyperparameter $\eta$.
We have seen that at EOS, the only direct effect of $\eta$ on the central flow is to modulate the strength of sharpness regularization, with higher $\eta$ inducing stronger sharpness regularization.
Thus, counterintuitively, the \emph{instantaneous} effect of a higher $\eta$ is often to \emph{slow down} optimization.
However, as we illustrate in \Cref{fig:rmsprop-norm-learning-rate,fig:rmsnorm-acc-by-reg}, over longer timescales, higher $\eta$ steers the trajectory into lower-sharpness regions, in which \rmsnorm's effective step size will be larger, thereby tending to \emph{speed up} optimization. Thus, as one would expect of a learning rate hyperparameter, larger $\eta$ can accelerate optimization; however they do so through this \emph{indirect} mechanism.
\footnotetext{In this figure, for \rmsnorm, we show the train loss at the second-order midpoints between iterates (see \Cref{appendix:experimental-details:implementation}).}

%% file: 5_rmsprop.tex
\newpage
\section{RMSProp}
\label{sec:rmsprop}
We now study \rmsprop \citep{tieleman2012lecture}, which is equivalent to Adam \citep{kingma2014adam} without momentum.  \rmsprop maintains an EMA $\nu$ of the elementwise squared gradients $\nabla L(w)^{\odot 2}$, and uses \emph{per-coordinate} effective step sizes of $\eta / \sqrt{\nu}$:\footnote{Our analysis can accommodate both bias correction and an $\epsilon$-dampening (dividing by $\sqrt{\nu + \epsilon}$ rather than $\sqrt{\nu}$) which are used by Adam (see \Cref{sec:arbitrary-preconditioned}). However, to simplify exposition, the main text focuses on this simpler version of \rmsprop.}
\begin{align}
    \nu_{t} = \beta_2 \nu_{t-1} + (1 - \beta_2) \nabla L(w_t)^{\odot 2} \qc
    w_{t+1} = w_t - \frac{\eta}{\sqrt{\nu_t}} \odot \nabla L(w_t),\qquad\qquad
\label{eq:rmsprop}
\end{align}
where $\odot$ represents the entrywise product. 
\rmsprop can also be viewed as preconditioned gradient descent $w_{t+1} = w_t - P_t^{-1} \nabla L(w_t)$ with the dynamic preconditioner $P_t := \mathrm{diag}(\sqrt{\nu_t}/\eta)$.\footnote{Folding $\eta$ into the preconditioner is unconventional, but will make the analysis clearer.} 
While Adam employs the same dynamic preconditioner and has achieved widespread success in deep learning, it has remained unclear why this specific preconditioning strategy is so effective \citep{kunstner2019limitations, orabona2020neural, martens2020}.
A common folklore belief is that Adam/RMSProp adapts to the local ``curvature'' (i.e. Hessian).  However, it is a priori unclear how this can be so, since the algorithm uses the (squared) \emph{gradient}, not the Hessian, to update its preconditioner.

In this section, we use the central flows framework to understand the behavior of \rmsprop.
We will show that \rmsprop \emph{does} adapt to the local Hessian after all, but the reason is inextricably tied to its oscillatory dynamics, which have not been previously studied.

We start by describing the dynamics of \rmsprop in \Cref{sec:rmsprop:mechanics}.
We then derive a central flow in \Cref{sec:rmsprop:deriving}.  Finally, in \Cref{sec:rmsprop:interpreting}, we interpret this flow to understand the optimizer's behavior.  In particular:
\begin{itemize}
    \item In \Cref{sec:rmsprop:interpreting:stationary-preconditioner}, we show that \rmsprop's preconditioner is \emph{implicitly} determined by the algorithm's oscillatory dynamics, and we make this preconditioner \emph{explicit} for the first time. Specifically, we show that \rmsprop computes its preconditioner by solving a convex program (eq. \ref{eq:rmsprop_nu_convex_program}) involving the Hessian. This clarifies that \rmsprop is \textbf{implicitly a second-order optimizer}, despite only accessing the loss through first-order gradients.
    
    \item In \Cref{sec:rmsprop:interpreting:stationary-flow}, we show that, like \rmsnorm, the success of \rmsprop relies not only on this preconditioning strategy, but also on an \emph{acceleration via regularization} mechanism whereby implicitly regularizing curvature allows the optimizer to take larger steps later in training.
\end{itemize}

\begin{figure}[b!]
\centering
\includegraphics[width=\textwidth]{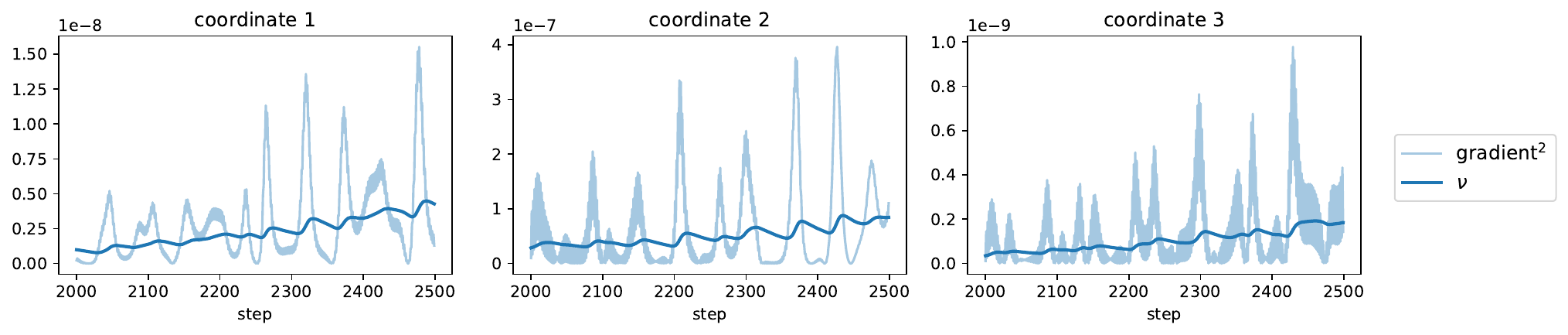}
\caption{\textbf{RMSProp $\nu$ is determined by oscillations.} 
While training a network using \rmsprop, we plot the squared gradient $\nabla L(w_t)^{\odot 2}$ (light blue) and its EMA $\nu_t$ (dark blue) at three coordinates (subpanels).  Due to the EOS oscillations, the squared gradient fluctuates, causing the EMA $\nu_t$ to also fluctuate.  Since this EMA is used to determine the effective step sizes $\eta / \sqrt{\nu_t}$, analyzing these dynamics is necessary for understanding \rmsprop's adaptivity. This network is a ResNet trained on a subset of CIFAR-10 using $\eta = 2$e-5, $\beta_2=0.99$ and MSE loss.}
\label{fig:rmsprop:squared-gradient-fluctuations}
\end{figure}
These basic insights into the functioning of \rmsprop are a prerequisite for a similar understanding of Adam, and may guide us in using these optimizers more effectively and in improving upon them.

\subsection{The Dynamics of RMSProp}\label{sec:rmsprop:mechanics}

To give some intuition into \rmsprop's behavior, \Cref{fig:rmsprop:squared-gradient-fluctuations} plots the dynamics of the squared gradient $\nabla L(w_t)^{\odot 2}$ and its EMA $\nu_t$ at several coordinates over a stretch of training.  Observe that the entries of the squared gradient fluctuate rapidly, causing their EMA to also fluctuate.
Since this EMA $\nu$ directly determines the effective step sizes $\eta / \sqrt{\nu}$, understanding the origin of this behavior is necessary to understand how \rmsprop sets its effective step sizes.

These fluctuations in the gradient arise because \rmsprop is operating in an oscillatory \emph{edge of stability} regime. 
To understand why \rmsprop oscillates, first consider running preconditioned gradient descent $w_{t+1} = w_t - P^{-1} \nabla L(w_t)$ on a quadratic function with Hessian $H$.  The resulting dynamics are controlled by the \emph{effective Hessian} $P^{-1} H$.
Namely, if any eigenvalues of this matrix exceed the critical threshold 2, then preconditioned GD will oscillate with exponentially growing magnitude along the corresponding (right) eigenvectors.\footnote{On a quadratic $\tfrac{1}{2} w^T H w$, this algorithm evolves via: $w_{t+1} = (I - P^{-1} H) w_t \implies w_t = (I - P^{-1} H) w_0$.  If $P^{-1} H$ has any eigenvalues greater than 2, $(I - P^{-1} H)$ has eigenvalues less than $-1$, and the iterates diverge along the corresponding right eigenvectors.} For \rmsprop in deep learning, both the Hessian $H(w_t)$ and the preconditioner $P_t = \diag(\sqrt{\nu_t} / \eta)$ can vary.  However, a local quadratic Taylor approximation suggests that \rmsprop will oscillate if the largest eigenvalue of the \emph{current} effective Hessian $P_t^{-1} H(w_t)$ exceeds the critical threshold 2.\footnote{With this argument, we are also implicitly assuming that the preconditioner evolves sufficiently slowly that its movement can be neglected.} We refer to this quantity as the effective sharpness $\Seff(w_t, \nu_t)$:
\begin{align}
    \Seff(w_t, \nu_t) := \lambda_1( P_t^{-1} H(w_t) ).
    \label{eq:rmsprop:effective-sharpness}
\end{align}
Paralleling the dynamics of gradient descent, \citet{cohen2022adaptive} observed that \rmsprop typically operates in an oscillatory EOS regime that revolves around the effective sharpness \cref{eq:rmsprop:effective-sharpness}.  On the one hand, oscillations ensue whenever the effective sharpness rises above the critical threshold $2$.\footnote{For \protect \rmsprop (and \protect \rmsnorm),  the effective sharpness \cref{eq:rmsprop:effective-sharpness} tends to rise \emph{both} because the curvature tends to rise (progressive sharpening) \emph{and} because the gradient (and hence $\nu$) tends to shrink.  Due to the second effect, \protect \rmsprop often enters EOS sooner, and for a large range of learning rates, than \protect \gd.  Further, \protect \rmsprop enters EOS even on quadratics, whereas \protect \gd does not.}  On the other hand, such oscillations reduce the effective sharpness, both by inducing implicit regularization of curvature (i.e. shrinking $H(w_t)$), and by growing the gradient and hence the preconditioner $P_t$.  The net result is that the effective sharpness equilibrates around the value 2 (see \Cref{fig:rmsprop}), as the optimizer oscillates along the top eigenvectors of the effective Hessian.

\subsection{Deriving the RMSProp Central Flow}\label{sec:rmsprop:deriving}

\begin{figure}[t]
\centering
\includegraphics[width=\textwidth]{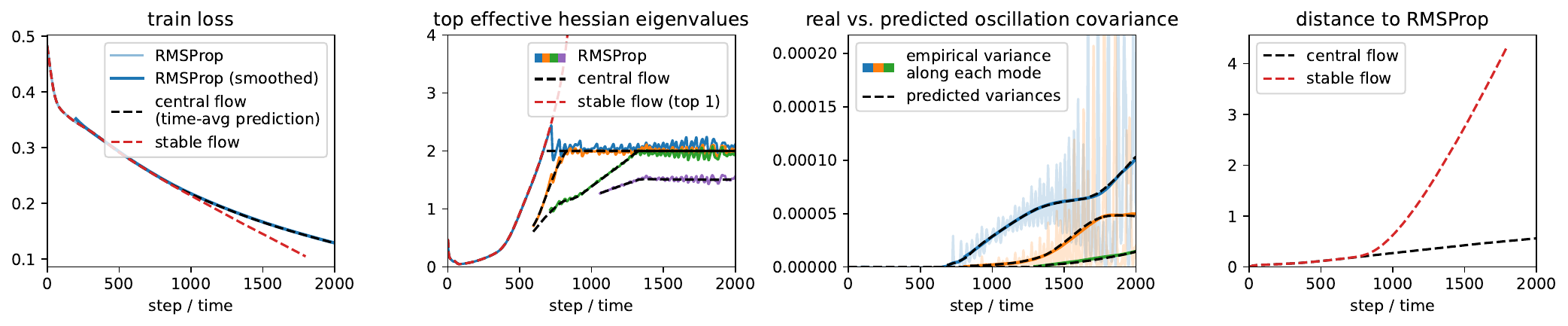}
\caption{\textbf{Central flow for RMSProp.} The \rmsprop central flow (black) accurately models the macroscopic trajectory of \protect\rmsprop even at EOS, whereas the naive stable flow (red) follows a different path. As for gradient descent and \protect\rmsnorm, we are able to predict the covariance $\Sigma(t)$ with which \protect\rmsprop oscillates around the central flow (third panel).  This figure is in the same setting as \Cref{fig:rmsprop:squared-gradient-fluctuations}.}
\label{fig:rmsprop}
\end{figure}

Similar as before, we now derive a central flow $(w(t),\nu(t))$ that jointly models the time-averaged dynamics of $w_t$, $\nu_t$.  We defer the full details to \Cref{appendix:derivations:rmsprop} and sketch the argument here.

If \rmsprop is oscillating around its time-averaged trajectory $\{\wbar_t\}$, so that $w_t = \overline{w}_t + \delta_t$, then the time average of the elementwise squared gradient is approximately:
\begin{align}
    \underbrace{\E[\nabla L(w_t)^{\odot 2}]}_{\substack{\text{time-average of} \\ \text{ squared gradient}}}
    &\approx \underbrace{\nabla L(\wbar_t)^{\odot 2}}_{\substack{\text{squared gradient at} \\ \text{time-averaged iterate}}} + \underbrace{\diag[H(\wbar_t) \E[\delta_t \delta_t^T] H(\wbar_t) ]}_{\substack{\text{contribution from oscillations}}}.
\end{align}
The first term is the squared gradient at the time-averaged iterate; the second term is the contribution to the squared gradient that originates from oscillating with covariance $\E[\delta_t \delta_t^T]$.

If we further assume then these oscillations are contained within the right eigenspace of the effective Hessian $\diag(\eta/\sqrt{\nu_t}) H(w_t)$ that corresponds to the eigenvalue 2, then the rightmost term simplifies as follows:
\begin{align}
    \underbrace{\E[\nabla L(w_t)^{\odot 2}]}_{\substack{\text{time-average of} \\ \text{ squared gradient}}}
    &\approx \underbrace{\nabla L(\wbar_t)^{\odot 2}}_{\substack{\text{squared gradient at} \\ \text{time-averaged iterate}}} + \underbrace{\frac{4\nu}{\eta^2} \odot \diag[\E[\delta_t \delta_t^T]]}_{\substack{\text{contribution from oscillations}}}
    \label{eq:rmsprop:time-average-squared-gradient}
\end{align}
Based on this calculation, and on the time-averaged gradient computed in \cref{eq:gd:expected_gradient_multi}, we make the ansatz that the time-averaged dynamics of $w_t, \nu_t$ follow a central flow $(w(t), \nu(t))$ with the functional form:
\begin{align}
    \begin{split}
    \frac{dw}{dt} = -\frac{\eta}{\sqrt{\nu}} \odot \qty[\underbrace{\nabla L(w) + \tfrac{1}{2} \nabla \langle \Sigma(t), H(w) \rangle}_{\E[\nabla L(w_t)]}], \quad
    \frac{d\nu}{dt} = \frac{1-\beta_2}{\beta_2} \qty[\underbrace{\nabla L(w)^{\odot 2} + \frac{4\nu}{\eta^2} \odot \diag[\Sigma(t)]}_{\E[\nabla L(w_t)^{\odot 2}]} - \nu].
    \end{split}
    \label{eq:rmsprop:ansatz}
\end{align}
To determine $\Sigma(t)$, we impose three conditions on this flow, analogous to those from \Cref{sec:gd:multi}.  As before, it can be shown that there is a unique matrix $\Sigma(t)$ satisfying these three conditions, and this matrix can be characterized as the solution to a semidefinite complementarity problem.  The \rmsprop central flow is defined as \cref{eq:rmsprop:ansatz} with this value of $\Sigma(t)$.
See \Cref{appendix:derivations:rmsprop}, \Cref{def:rmsprop:flow:ode} for a formal statement.

\Cref{fig:rmsprop} illustrates how this central flow can accurately predict the macroscopic trajectory $w(t)$ of \rmsprop, as well as the covariance $\Sigma(t)$ with which \rmsprop is oscillating around that trajectory.\footnote{To match the preconditioned geometry of the optimizer, we assess whether each eigenvalue of $P(\nu(t))^{1/2} \Sigma(t) \, P(\nu(t))^{1/2}$ accurately predicts the $P$-whitened variance of oscillations along the corresponding eigenvector; see \cref{eq:rmsprop-predict-oscillation-variance} in \Cref{appendix:derivations:rmsprop}.}  \Cref{fig:rmsprop:central-flow-predictions} shows how the central flow can accurately predict the time-average of the elementwise squared gradient via \cref{eq:rmsprop:time-average-squared-gradient}, as well as the macroscopic trajectory $\nu(t)$ of the EMA.  \Cref{fig:experiments:rmsprop:loss-curves-mse} and \Cref{fig:experiments:rmsprop:loss-curves-ce} in \Cref{sec:supplementary-figures} show in a variety of deep learning settings that the central flow can accurately predict the \rmsprop loss curve across different learning rates.  \Cref{fig:experiments:rmsprop:beta2} and \Cref{fig:experiments:rmsprop:eps} show that the central flow can accurately predict the \rmsprop trajectory at different values of $\beta_2$ and $\epsilon$, respectively. The full set of raw \rmsprop experiments can be found in \Cref{sec:bulk-experiments-rmsprop}.  

\begin{figure}[t!]
\centering
\includegraphics[width=0.75\textwidth]{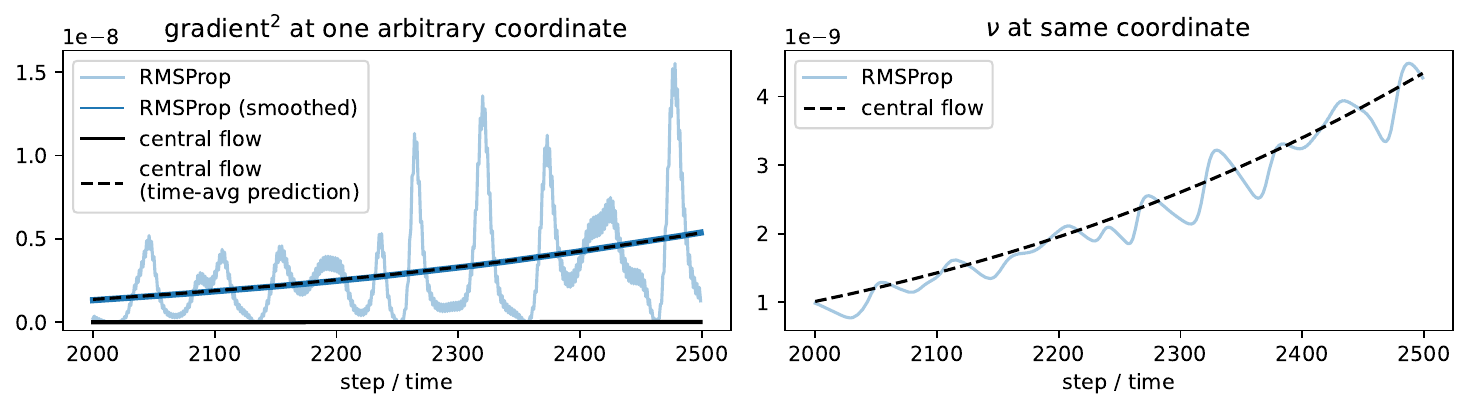}
\caption{\textbf{Central flow can successfully predict both the time-averaged gradient$^2$ (left) and the EMA $\nu$ (right).} 
On the left, we show that although the squared gradient is fluctuating erratically (recall \Cref{fig:rmsprop:squared-gradient-fluctuations}), its time-average can be predicted by the central flow using \cref{eq:rmsprop:time-average-squared-gradient}.  On the right, we show that the central flow's $\nu(t)$ accurately tracks the macroscopic trajectory of the real EMA $\nu_t$. This figure is in the same setting as \Cref{fig:rmsprop:squared-gradient-fluctuations}.}
\label{fig:rmsprop:central-flow-predictions}
\end{figure}

As with gradient descent, we find that the central flow approximation tends to become less accurate as the learning rate $\eta$ grows; see \Cref{sec:experiments} for our general discussion about the accuracy of the central flow.  In addition, we observe that the central flow for \rmsprop tends to be a bit less accurate overall than that for \gd, at least as measured by weight-space distance between the flow and the discrete optimizer.   Finally, we expect the central flow for \rmsprop to break down when $\beta_2$ becomes too close to zero (i.e. the sign GD limit), as then \rmsprop would no longer resemble preconditioned \gd with a slowly-changing preconditioner.

\subsection{Understanding RMSProp via its Central Flow}\label{sec:rmsprop:interpreting}
We now interpret the \rmsprop central flow to understand the behavior of \rmsprop, including how the algorithm sets its effective step sizes $\eta / \sqrt{\nu}$.
Because the dynamics usually transition from stable to EOS early in training, we focus on the EOS regime.

\subsubsection{The stationary preconditioner} \label{sec:rmsprop:interpreting:stationary-preconditioner}

\paragraph{Stationarity of $\nu$} 
Unfortunately, even at the edge of stability, $\nu(t)$ cannot be expressed as a closed-form function of $w(t)$ (as it could for \rmsnorm in \Cref{sec:rmsprop_norm}), and instead remains an independent variable that must be tracked. This reflects the fact that for any $w$, there are potentially many values for $\nu$ that could stabilize optimization, and the actual value used by \rmsprop depends on the history. Nevertheless, we will now see that under the \rmsprop central flow, $\nu$ often implicitly converges to a value that depends on the current $w$ alone.

Intuitively, the \rmsprop central flow \cref{eq:rmsprop:ansatz} involves two simultaneous processes of optimization (the $w$ dynamics) and preconditioner adaptation (the $\nu$ dynamics).
Suppose that the $\nu$ dynamics of preconditioner adaptation occur \emph{fast} relative to the $w$ dynamics of optimization, so that $\nu$ reaches a stationary point w.r.t the current weights $w$.
In \Cref{lem:rmsprop_nu_convex_program_unique} we show that for any $w$, there is in fact a \emph{unique} $\nu$ that satisfies the stationarity condition $\frac{d\nu}{dt} = 0$.  We call this unique $\nu$ the \emph{stationary $\nu$} for the weights $w$, denoted as $\overline{\nu}(w)$.
Empirically, we observe that $\nu(t)$ usually starts to attain its stationary value $\overline{\nu}(w(t))$ at some point during training (after the dynamics have entered EOS), and continues to match $\overline{\nu}(w(t))$ thereafter, even as this value evolves.  Indeed, \Cref{fig:rmsprop:stationary-nu-convergence} illustrates how $\nu(t)$ converges to $\overline{\nu}(w(t))$ both in cosine similarity (left) and coordinate-wise (right).  See \Cref{fig:cosnu-mse,fig:cosnu-ce,fig:nu-mse,fig:nu-ce} for more settings.\footnote{Since the speed of the $\nu$ dynamics in \cref{eq:rmsprop:ansatz} is controlled by the $\beta_2$ hyperparameter, one might suspect that $\nu(t)$ will converge quicker to $\overline{\nu}(w(t))$ when $\beta_2$ is smaller, and we confirm this in \Cref{fig:experiments:rmsprop:beta2-stationarity}.  Nevertheless, we emphasize that the quasistationarity of $\nu$ w.r.t $w$ empirically holds even when $\beta_2$ is relatively large (e.g. 0.99).  In keeping with our attitude throughout this paper, we do not claim to have an explanation.} 

\begin{figure}[t!]
\centering
\includegraphics[width=0.9\textwidth]{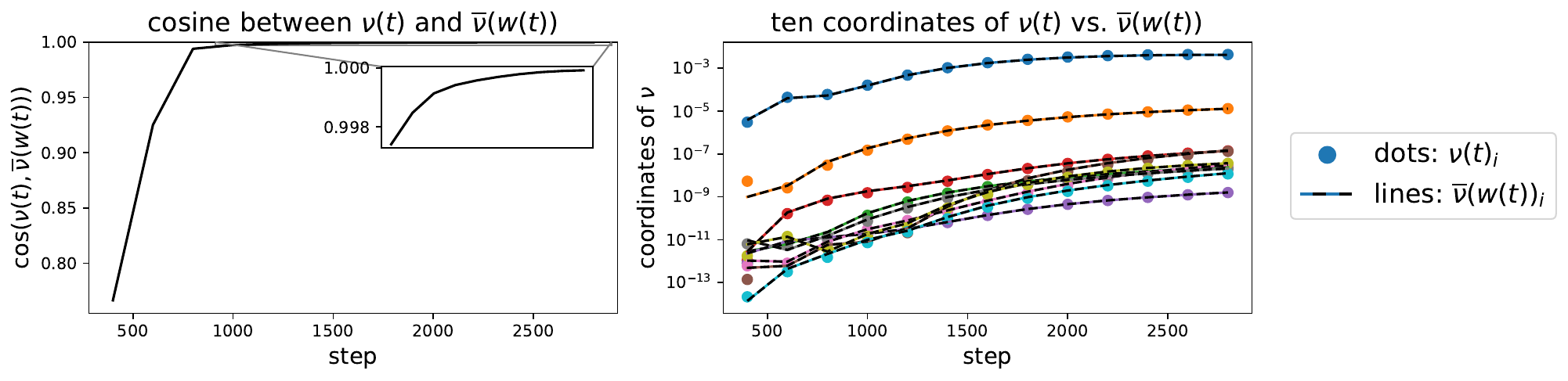}
\caption{\textbf{The EMA $\nu$ converges to its stationary value.} While running the \rmsprop central flow, we compare the actual EMA  $\nu(t)$ to its stationary value $\overline{\nu}(w(t))$ w.r.t the current weights $w(t)$.  On the left, we plot the cosine similarity between $\nu(t)$ and $\overline{\nu}(w(t))$; on the right, we compare ten individual coordinates (colors), spaced uniformly throughout the network.  The plots begin when training enters EOS, just before step 500.  Observe that after a bit of time, the cosine similarity between $\nu(t)$ and $\overline{\nu}(w(t))$ reaches high values (near 1), and the individual coordinates coincide as well. 
 This figure depicts the same setting as \Cref{fig:rmsprop:squared-gradient-fluctuations}; see \Cref{fig:cosnu-mse,fig:cosnu-ce,fig:nu-mse,fig:nu-ce} for more settings.}
\label{fig:rmsprop:stationary-nu-convergence}
\end{figure}
The stationarity of $\nu$ will allow us to reason about \rmsprop's preconditioning strategy with relative ease, i.e. without needing to account for the history of $\nu$.
At any weights $w$, we can view the corresponding \emph{stationary preconditioner} $\overline{P}(w) := \diag(\sqrt{\overline{\nu}(w)}/\eta)$ as ``the RMSProp preconditioner'' that is implicitly used by \rmsprop at weights $w$.
We will now interpret this preconditioner to gain insight into \rmsprop's preconditioning strategy.

\paragraph{Interpreting the stationary preconditioner} In \Cref{prop:stationary_nu}, we show that this stationary preconditioner $\overline{P}(w) := \diag(\sqrt{\overline{\nu}(w)}/\eta)$ is, remarkably, the optimal solution to a convex optimization problem over preconditioners:
\begin{align}
   \overline{P}(w) \quad:=\quad \argmin_{\mathclap{\text{$P$ diagonal, $P{\succeq}0$}}} \qquad \tr(P) + \tfrac{1}{\eta^2} \underbrace{\norm{\nabla L(w)}^2_{P^{-1}}}_{\mathclap{\text{optimization speed}}} \qq{such that} \underbrace{H(w) \preceq 2P}_{\text{local stability}}. \quad  \label{eq:rmsprop_nu_convex_program}
\end{align}
That is, \textbf{\rmsprop implicitly solves the convex program  \cref{eq:rmsprop_nu_convex_program} to compute its preconditioner}.\footnote{Interestingly, this SDP is the dual to the max-cut SDP: $\max_{\Sigma \succeq 0} \ev{\Sigma,H}$ such that $\Sigma_{ii} = 1$ for all $i$. Thus, this preconditioning strategy could be described as solving the max-cut SDP with the Hessian as the weight matrix, and using the resulting dual variable as its preconditioner.}  This is the precise sense in which RMSProp ``adapts'' its preconditioner to the local loss landscape.

We can now understand \rmsprop's preconditioning strategy by interpreting the optimization problem \cref{eq:rmsprop_nu_convex_program}.
The constraint $H(w) \preceq 2P$ is equivalent to $\Seff \le 2$ and hence stipulates that the preconditioner $P$ should keep \rmsprop locally stable.
The first term of the objective, $\tr(P)$,  is the sum of the inverse effective step sizes. If this were the only term in the objective, \rmsprop's preconditioning strategy could be simply summarized as maximizing the \emph{harmonic mean} of the effective step sizes while maintaining local stability --- a sensible preconditioning strategy.  Indeed, consider a variant of \cref{eq:rmsprop_nu_convex_program} with only the first term:
\begin{align}
   \hat{P}(w) := \argmin_{\mathclap{\text{$P$ diagonal, $P{\succeq}0$}}} \qquad \tr(P) \qq{such that} H(w) \preceq 2P. \quad  \label{eq:rmsprop_nu_first_term_only}
\end{align}
\begin{figure}[t!]
\centering
\includegraphics[width=1\textwidth]{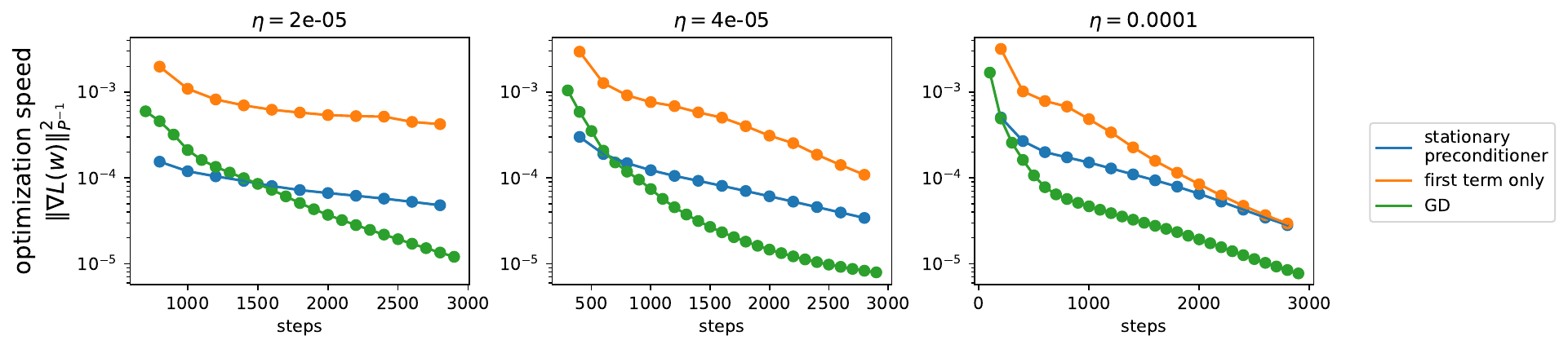}
\caption{\textbf{Optimization speeds for various preconditioners.} Along the \rmsprop central flow for various learning rates (columns), we assess the efficacy of three different preconditioners $P$: the \rmsprop stationary preconditioner \cref{eq:rmsprop_nu_convex_program}, in blue; a variant \cref{eq:rmsprop_nu_first_term_only} with only the first term, in orange; and the preconditioner corresponding to vanilla \gd with the largest locally stable learning rate, i.e. $P^{-1}=(2/S(w)) \, I$, in green.  We assess each preconditioner $P$ by reporting $\| \nabla L(w) \|^2_{P^{-1}} := \nabla L(w)^T P^{-1} \nabla L(w)$, the instantaneous rate of loss decrease under preconditioned gradient flow with preconditioner $P$.  Observe that the ``first term only'' preconditioner (orange) is much better than the vanilla GD (green) preconditioner, and is also better than the actual stationary preconditioner (blue).  The actual stationary preconditioner (blue) is usually better than vanilla GD (green), but not always, especially when $\eta$ is smaller.  See \Cref{fig:stationary-suboptimal-mse,fig:stationary-suboptimal-ce} for more experimental settings.}
\label{fig:rmsprop:optimization-speeds}
\end{figure}
\Cref{fig:rmsprop:optimization-speeds} demonstrates that this preconditioner is a substantial improvement over vanilla gradient descent.  (We describe in \Cref{appendix:derivations:rmsprop:stationary} how we numerically solve \cref{eq:rmsprop_nu_convex_program} and \cref{eq:rmsprop_nu_first_term_only}.)
Interestingly, if the diagonal constraint in \cref{eq:rmsprop_nu_first_term_only} were removed, and if $H(w)$ were PSD, then the optimization problem \cref{eq:rmsprop_nu_first_term_only} would have the closed-form solution $\hat{P}(w) = \frac{1}{2} H(w)$.  That is, the preconditioner $P$ would be a scaling of the Hessian, and preconditioned gradient descent would move in the same direction as Newton's method.\footnote{Even if $H(w)$ were not PSD, a similar point would hold: the optimization problem \cref{eq:rmsprop_nu_first_term_only} would have the closed-form solution $\hat{P} = \frac{1}{2} \Pi_{\mathbb{S}_+} H(w)$, where $\Pi_{\mathbb{S}_+}$ denotes projection onto the cone of positive semidefinite matrices.}

However, matters are complicated by the presence of the second term in the \cref{eq:rmsprop_nu_convex_program} objective.
The quantity $\norm{\nabla L(w)}^2_{P^{-1}}$ is the instantaneous rate of loss decrease under preconditioned gradient flow with preconditioner $P$. \emph{Minimizing} this term necessarily acts to \emph{slow down} optimization.\footnote{That is, for any $w$, the optimization speed $\norm{\nabla L(w)}^2_{P^{-1}}$ must necessarily be smaller (worse) for \cref{eq:rmsprop_nu_convex_program} than for \cref{eq:rmsprop_nu_first_term_only}.} Indeed, \Cref{fig:rmsprop:optimization-speeds} shows that the stationary preconditioner \cref{eq:rmsprop_nu_convex_program} underperforms the variant \cref{eq:rmsprop_nu_first_term_only} with only the first term.

Since the second term in \cref{eq:rmsprop_nu_convex_program} is proportional to $\tfrac{1}{\eta^2}$, its influence diminishes as the learning rate hyperparameter $\eta$ grows.  Indeed, it can be seen in \Cref{fig:rmsprop:optimization-speeds} that the performance of the stationary preconditioner tends closer to that of \cref{eq:rmsprop_nu_first_term_only} as the learning rate hyperparameter $\eta$ is made larger.  In the limit of large $\eta$, the second term vanishes entirely, and the stationary preconditioner reduces completely to \cref{eq:rmsprop_nu_first_term_only}. Interestingly, in this limit, the stationary preconditioner ceases to depend on $\eta$: for example, doubling $\eta$ will cause $\overline{\nu}$ to quadruple in scale (due to larger oscillations), while keeping the effective step sizes $\eta / \sqrt{\overline{\nu}}$ unchanged.  This parallels the situation for \rmsnorm in \Cref{sec:rmsprop_norm}, where the effective step size at EOS was $2/S(w)$, independent of $\eta$.

\paragraph{The stationary flow} 
\footnotetext{In this figure, for \rmsprop, we show the train loss at the second-order midpoints between iterates (see \Cref{appendix:experimental-details:implementation}).}
Substituting $\overline{P}$ into the central flow, we can obtain a \emph{stationary flow} over $w$ alone:
\begin{align}
     \frac{dw}{dt} = - \underbrace{\overline{P}(w)^{-1}}_{\mathclap{\substack{\text{stationary}\\\text{ preconditioner}}}} \qty\Big[\; \nabla L(w) +  \underbrace{ \tfrac{1}{2} \nabla_w \ev{\Sigma, H(w)} }_{\text{implicit curvature penalty}}\;],
     \label{eq:rmsprop-stationary-flow}
\end{align}
where $\Sigma = \Sigma(w;\eta;\beta_2)$ is defined as the solution to a certain semidefinite complementarity problem (\Cref{appendix:derivations:rmsprop:stationary}, \Cref{def:rmsprop:flow:stationary}). This model assumes that the $\nu$ dynamics (preconditioner adaptation) happen \emph{infinitely fast} relative to the $w$ dynamics (optimization), so that we can treat the preconditioner $P$ as always being fixed at its current stationary value $\overline{P}(w)$ (eq. \ref{eq:rmsprop_nu_convex_program}).
The appeal of this characterization is that it eliminates $\nu$ from the picture entirely, and expresses the time-averaged dynamics of \rmsprop as a closed system in $w$ alone.
Namely, it suggests that the time-averaged trajectory of \rmsprop is equivalent to that of the following simpler-to-understand algorithm:

\begin{tcolorbox}[left=2pt,right=2pt,top=3pt,bottom=3pt,colback=white!0,halign=center]
At each iteration, compute the preconditioner $\overline{P}(w)$ using \cref{eq:rmsprop_nu_convex_program} and then take a preconditioned gradient step using this preconditioner on a curvature-penalized objective.
\end{tcolorbox}

Empirically, we find that the stationary flow  \cref{eq:rmsprop-stationary-flow} is often a reasonable model for the \rmsprop trajectory.  For example, \Cref{fig:optspeed-mse,fig:optspeed-ce} show that the stationary flow can accurately predict the instantaneous rate of loss decrease at various points along the central flow trajectory, even though it only has access to $w(t)$ and not $\nu(t)$.  Meanwhile, \Cref{fig:stationary-flow-cnn-mse,fig:stationary-flow-transformer-mse} show that the stationary flow can tolerably predict the trajectory of \rmsprop over moderate timescales, although we note that its accuracy is not as high as the full central flow.\footnote{As one might expect, the stationary flow tends to only be an accurate model for \rmsprop once $\nu$ has reached stationarity.}

\begin{figure}[t]
\centering
\includegraphics[width=\textwidth]{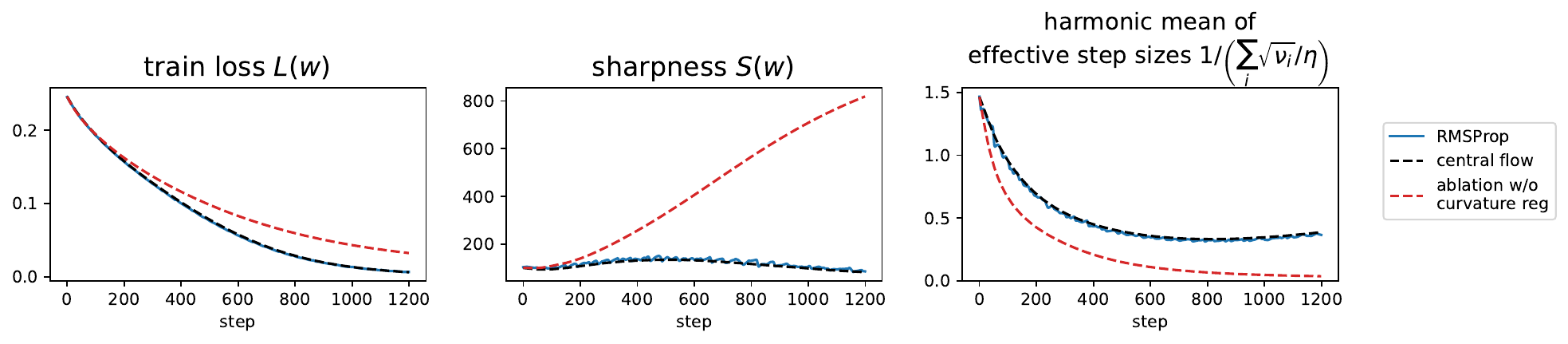}
\caption{\textbf{Implicit curvature regularization accelerates optimization for RMSProp}.  Starting from the same initial point, we compare RMSProp (blue) and its central flow (black) to an ablated flow where the implicit curvature regularization is disabled (red).  Relative to this ablated flow, the central flow takes a lower-curvature trajectory (middle), in which it takes larger steps (right) and optimizes faster (left).  The setting is the same as \Cref{fig:rmsprop:squared-gradient-fluctuations}.\protect\footnotemark} 
\label{fig:rmsprop:acc-via-reg}
\end{figure}
\subsubsection{Acceleration via regularization} 
\label{sec:rmsprop:interpreting:stationary-flow}
As with \rmsnorm, we find that \rmsprop's implicit curvature regularization enables it to optimize faster.  In \Cref{fig:rmsprop:acc-via-reg}, we show that when the curvature regularization is disabled, the \rmsprop central flow navigates into increasingly sharp regions, where it takes smaller steps, and optimizes slower (see \Cref{fig:rmsprop-acc-by-reg} for more settings).\footnote{To run this ablated flow, we manually set $\nabla H(w) = 0$ both in the expression for $\beta$ and in the expression for $\tfrac{dw}{dt}$ (see \Cref{appendix:derivations:rmsprop}).}

Establishing this claim theoretically is more difficult for \rmsprop than \rmsnorm.\footnote{Partly, the difficulty of analysis is due to the independent $\nu$ dynamics.  However, this analysis is also not easy under the stationary flow, because the second term in \cref{eq:rmsprop_nu_convex_program} causes the effective step sizes to depend not just on the current Hessian but also on the current gradient.}  However, in the limit of large $\eta$ and small $\beta_2$, it can be argued (\Cref{appendix:derivations:rmsprop:stationary}) that the stationary flow \cref{eq:rmsprop-stationary-flow} reduces to:
\begin{align}
    \frac{dw}{dt} &= - \hat{P}(w)^{-1} \qty[ \nabla L(w) + \frac{\eta^2}{4} \nabla \tr \hat{P}(w)],
\end{align}
where $\hat{P}(w)$ was defined in \cref{eq:rmsprop_nu_first_term_only}.
This model says that \rmsprop implicitly picks the diagonal preconditioner  with minimal trace (equivalently, the preconditioner $P$ where the effective learning rates $P^{-1}$ have maximal harmonic mean), and also implicitly moves in a direction in which the trace of this preconditioner will become even smaller.
The strength of the latter effect is controlled by $\eta$, and in fact, this is the only means by which the learning rate hyperparameter $\eta$ affects the trajectory, since the preconditioner $\hat{P}(w)$ is independent of $\eta$.
Thus, as with \rmsnorm, larger learning rates translate to larger steps, but only via this indirect mechanism.

%% file: 6_experiments.tex
\newpage
\section{Experiments}
\label{sec:experiments}

The goal of our experiments is to establish that each central flow accurately approximates the trajectory of its corresponding optimizer in a variety of deep learning settings, and to understand the circumstances under which this approximation breaks down.
Because it is computationally costly to discretize central flows, we experiment on small-scale networks and datasets.
Note that there is no evidence that scale itself fundamentally affects the dynamics of optimization in deep learning; for example, EOS dynamics have been observed at both smaller (e.g. CIFAR-10) and larger (e.g. ImageNet or WMT) scales, without noticeable differences.  Therefore, we expect that the central flow approximation would similarly hold true at larger scales, if such experiments were computationally feasible.

We emphasize that the central flow is a theoretical tool for understanding optimizer behavior, not a practical optimization method.
In practice, maintaining an exponential moving average of the iterates  \citep[e.g.,][]{morales2024exponential} is likely a computational feasible way to estimate the optimizer's time-averaged trajectory.

\paragraph{Architectures} We experiment on a diverse set of six architectures: a convolutional neural network (CNN), a ResNet \citep{he2015deep}, a Vision Transformer (ViT) \citep{dosovitskiy2021image}, an LSTM \citep{hochreiter1997long}, a (sequence) Transformer \citep{vaswani2017attention}, and a Mamba sequence model \citep{gu2024mamba}. 
 Architectural details can be found in \Cref{appendix:experimental-details:architectures}.

\paragraph{Datasets}  We test the vision architectures (CNN, ResNet, ViT) on a subset of CIFAR-10 \citep{krizhevsky2009learning}, and the sequence architectures (LSTM, Transformer, Mamba) on a synthetic sorting task \citep{karpathy2020mingpt}.
Further details on these datasets can be found in \Cref{appendix:experimental-details:datasets}.
For each architecture and each dataset, we test both cross-entropy loss and MSE loss.  As discussed below, the central flow tends to be somewhat more accurate with MSE loss.

\paragraph{Implementation}
Discretizing the central flows is somewhat nontrivial, as the flows are non-smooth at points where there is a change in the number of unstable eigenvalues (e.g. going from 0 to 1).
We describe our solution in \Cref{appendix:derivations:gd:discretizing} for \gd and in \Cref{appendix:derivations:arbitrary-preconditioned:discretizing} for a generic (potentially) adaptive optimizer.
The time complexity of each discretization step scales quadratically with the number of eigenvalues that are at the edge of stability.  Most of the computational cost arises from the need to continually re-estimate the top eigenvectors and eigenvalues of the (effective) Hessian, and to compute the necessary third derivatives (gradients of these eigenvalues).   Full implementation details can be found in \Cref{appendix:experimental-details:implementation}.

Our code can be found at:
\url{https://github.com/centralflows/centralflows}.

\subsection{Experimental Results}
To assess the accuracy of the central flow approximation, we run both the discrete optimizer and the central flow simultaneously, starting from the same initialization.
As a baseline, we also run the corresponding stable flow (e.g. for gradient descent, the gradient flow), which we expect to poorly approximate the discrete optimizer when the latter is at the edge of stability.

Our full experimental results, which can be found in \Cref{sec:bulk-experiments}, make clear that the central flow can accurately approximate the long-term optimization trajectory in a variety of deep learning settings.  We find that the weight-space distance between the discrete optimizer and the central flow generally stays small over time, and is much smaller than the distance between the discrete optimizer and the stable flow baseline.  Meanwhile, the network's predictions under the central flow generally match those of the discrete optimizer, whereas the stable flow takes a different path through function space.  The central flow can also accurately predict the time-averaged train loss curve and squared gradient norm curve, as well as the covariance of the discrete optimizer's oscillations around the central flow.

That said, the central flow approximation can break down in certain circumstances, which we now describe.
An interesting direction for future work would be to rigorously characterize the conditions under which the central flow does or does not approximate the real optimizer trajectory.

\paragraph{Sufficiently large learning rates} For all three optimizers studied in this paper, we reliably observe that as the learning rate hyperparameter is made increasingly large, the real optimization trajectory tends to deviate more from the central flow, as illustrated in \Cref{fig:flow-accuracy-lr}.   (As an extreme example, even if the optimizer is initialized stably, very large learning rates sometimes cause the real optimizer to explosively diverge in the middle of training, whereas this never happens under the central flow.)  We do not know whether some corrected version of the central flow would be more successful at capturing the real optimization trajectory in these scenarios, or whether the real trajectory is simply too chaotic to be captured by any flow.

\begin{figure}[t]
\centering
\includegraphics[width=\textwidth]{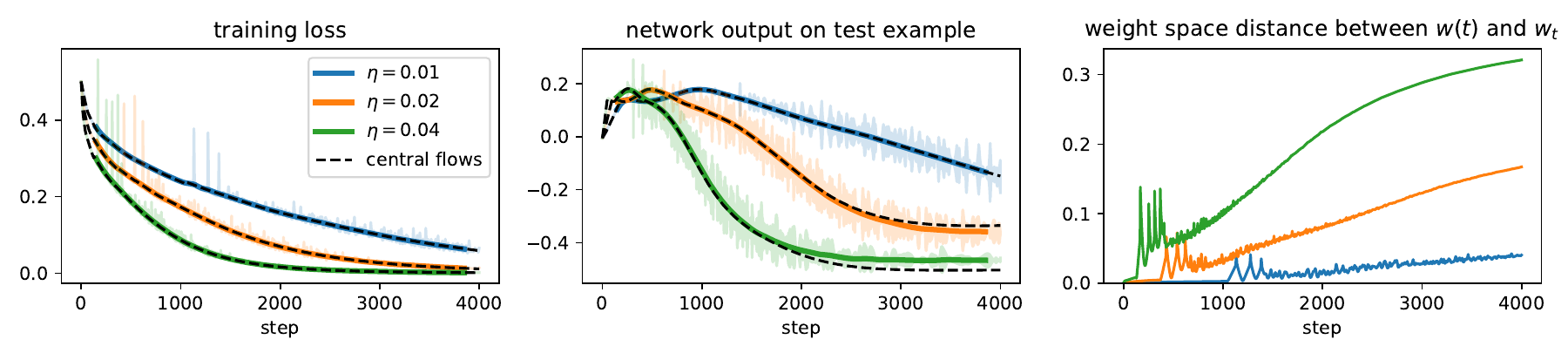}
\caption{\textbf{Central flow approximation is less accurate at larger learning rates}.  We run both gradient descent and its central flow at three learning rates (colors).  The larger the learning rate, the faster the growth in the accumulated approximation error (right).  Indeed, at larger learning rates, the network's output on an arbitrary test example can be visually seen to be slightly different between the central flow and gradient descent (middle).  Nevertheless, the central flow approximation is still accurate enough here to accurately capture the train loss curves (left). \textit{Details}: a CNN is trained on CIFAR-10 using MSE loss.}
\label{fig:flow-accuracy-lr}
\end{figure}

\paragraph{Higher-order terms}
Sometimes, the local curvature is not well-modeled by the cubic Taylor approximation, as is assumed by our theory.
This leads the central flow to mispredict $\Sigma(t)$, causing error to accumulate over the long run.
We elaborate on this failure mode in \Cref{sec:higher-order}.

\paragraph{Smoothness of architecture} The smoothness of the architecture seems to affect the accuracy of the central flow approximation; non-smooth components such as ReLU or max pooling often cause the quality of the approximation to break down.  For example, in \Cref{fig:beta-gelu-distances}, we show that as a network's activation function is interpolated from GeLU (smooth) to ReLU (non-smooth), the accuracy of the central flow approximation degrades, both in weight space and function space.  Note that it is not clear how best to precisely quantify ``smoothness'' in this context.

\paragraph{Large spikes} When the EOS dynamics lead to extremely large spikes (e.g. in the gradient norm), we have found such spikes can cause the real trajectory to deviate from the central flow, as illustrated in \Cref{fig:spikes-distance}. This may be related to the large learning rate issue described above.

\paragraph{Interactions with loss criterion}
We have empirically found that the higher-order terms issue and the large spike issue are more common with cross-entropy loss than with mean squared error loss (although we do not have a satisfactory explanation for these observations).  Consequently, the central flow approximation is often more accurate under the MSE loss than the cross-entropy loss.

Overall, despite these limitations, we argue that the central flow describes the behavior of the corresponding optimizer ``to a first approximation.''
Even in cases where the central flow is not a perfect quantitative match to the discrete trajectory, it may still capture the important qualitative trends.

%% file: 7_discussion.tex
\newpage
\section{Discussion}
\label{sec:discussion}

\subsection{Modeling decisions}

\paragraph{Deterministic setting}
Our analysis is restricted to the simple setting of deterministic (i.e. full-batch) training, whereas practical deep learning generally involves minibatch training. We study the full-batch setting because, as the simplest special case, understanding full-batch training is a necessary prerequisite for understanding minibatch training. However, we also believe that understanding full-batch training might suffice for some practical purposes, such as designing optimizers. For example, \citet{kunstner2024heavy} showed that the advantage of adaptive methods over SGD grows larger with larger batch sizes, suggesting that the relevant algorithmic principles can be best understood in the deterministic setting.

An interesting direction for future research is to try to extend our central flows methodology to the stochastic setting.
Like deterministic optimizers, stochastic optimizers are known to implicitly regularize the curvature along their trajectories, and in fact this effect is \emph{stronger} in the stochastic setting \citep{keskar2017largebatch,jastrzebski2020the,jastrzebski2021catastrophic,lee2023a,andreyev2024edge}.
However, extending the central flows methodology to the stochastic setting may be nontrivial; due to the randomness, it is not clear whether there exists a deterministic differential equation around which SGD oscillates.
An interesting question is whether there exists a differential equation that can predict derived metrics such as network predictions or training loss curves, even if it cannot model the weight-space trajectory of SGD. Finally, while our analysis in this paper sheds light on adaptive optimizers in the deterministic setting, these optimizers could exhibit substantially different behavior in the stochastic setting.

\paragraph{Black-box model of the loss}
Our analysis treats the loss function as a black box, and never uses that the optimization problem at hand involves training a neural network.
The advantage of this approach is its generality: we expect our analysis to apply to generic deep learning architectures and learning problems, including those that do not yet exist. The disadvantage, however, is that the predictions made by our theory are at the abstraction level of the \emph{loss landscape}, and would need to be further translated in order to make concrete claims about the network architecture or learning problem. For example, our theory tells us that the learning rate hyperparameter modulates the strength of an implicit sharpness penalty, but does not tell us how this sharpness penalty affects learning. Nor does our theory shed light on how different layers of the neural network are mechanistically implicated in progressive sharpening or sharpness reduction.

On the one hand, the loss landscape level of abstraction is in some sense ``natural'' --- the overall path that optimizers follow really does intrinsically depend on the (effective) sharpness. But on the other hand, understanding many important aspects of optimization in deep learning will likely require cracking open the black box a bit more.

\subsection{Takeaways from our analysis}

\paragraph{The unreasonable effectiveness of time-averaging}
Prior works on EOS show that it is challenging to analyze the oscillatory EOS dynamics in fine-grained detail. Our work shows that, perhaps surprisingly, simple heuristics allow us to analyze the \emph{time-averaged} trajectory with excellent numerical accuracy. Interestingly, the success of this time-averaging approach seems to imply that the oscillations only affect the macroscopic trajectory in an ergodic sense, i.e. via their \emph{covariance} rather than via their fine-grained details. An promising direction for future work is to identify realistic conditions under which our heuristic time-averaging arguments can be made rigorous.  

\paragraph{Necessity of third-order Taylor expansions}
While optimization theory generally relies on second-order Taylor expansions of the loss, \citet{damian2023selfstabilization} showed that a \emph{third-order} Taylor expansion is necessary for understanding the convergence of gradient descent; such a Taylor expansion reveals that oscillations implicitly trigger curvature reduction, a form of negative feedback which stabilizes optimization.
In this work, we have shown that a third-order Taylor expansion is similarly necessary for understanding the \emph{acceleration via mechanism} which underlies the success of adaptive optimizers.
Thus, our work further underscores the necessity of a third-order Taylor expansion when analyzing optimization in deep learning.

\paragraph{Oscillatory first-order methods are implicitly second-order methods}
Over the last decade, optimizers that explicitly use Hessian information have failed to outperform first-order adaptive optimizers which employ only gradient information.  Our work demystifies this observation. We have shown that that when first-order optimizers oscillate, they implicitly leverage second order information. Thus, even though \rmsprop is a first-order optimizer, it implicitly employs a second-order preconditioning strategy, detailed in \Cref{sec:rmsprop:interpreting:stationary-preconditioner}.
Further, this preconditioning strategy is efficient, requiring no more gradient queries than \gd does. 
An exciting direction for future work is to \emph{intentionally} design first-order adaptive methods with such implicit preconditioners in mind.

\paragraph{Adapting to curvature is not enough}
Traditional optimization theory views the curvature of the loss as a pre-existing feature of the optimization problem, and views the job of an optimizer as \emph{adapting} to this pre-existing curvature.  We have shown that the adaptive optimizers that we study do not merely passively adapt to the curvature; they also actively \emph{shape} the curvature along their trajectory, by steering away from high-curvature regions where they would need to take small steps.
Further, we have shown that this effect  is crucial for their optimization efficacy.  Thus, our work suggests that \emph{acceleration via regularization} is a vital design principle for adaptive optimizers.

%% file: 8_conclusion.tex
\section{Limitations}

Before concluding, we review the limitations of our approach.
First, our analysis is currently nonrigorous and is ultimately supported by experiments rather than mathematical proof.  Second, our approach is currently restricted to the setting of deterministic (i.e. full-batch) training and does not yet apply to stochastic training or to momentum.  Third, it is quite computationally expensive to discretize the central flows, and doing so is only feasible for small networks on small datasets.  Fourth, the central flows tend to degrade in accuracy as the learning rate hyperparameter is made increasingly large, and are generally more accurate for MSE loss than cross-entropy loss.  Fifth, the central flows may not work well on networks with ReLU or other non-smooth architectural components.  Despite these limitations, we believe that central flows are the best available tool for reasoning about the dynamics of optimization in deep learning.

\section{Conclusion}

In this paper, we have developed a methodology for analyzing deep learning optimizers.
To analyze an optimization algorithm, we derive a \emph{central flow} which models the optimizer's time-averaged trajectory, rendering explicit what was previously implicit in the optimizer's oscillatory dynamics.
We have empirically shown that these central flows can accurately predict long-term optimization trajectories of neural networks, and by interpreting these flows we have obtained new insights about optimizers' behavior.

These advances are made possible by the fact that we adopt different goals from most works in optimization. Rather than try to characterize global convergence rates, we set ourselves the more modest goal of characterizing the \emph{local} optimization dynamics throughout training.
The local dynamics are important, they are more interesting than may have been assumed (even vanilla gradient descent gives rise to rich, complex dynamics), and they are empirically consistent across different deep learning settings, which suggests that a general theory is feasible.
We believe that similar analyses can be fruitfully conducted for other optimizers, and we hope to inspire work in that direction.

%% file: 9_acknowledgements.tex
\section*{Acknowledgements}
AD acknowledges support from an NSF Graduate Research Fellowship and a Jane Street Graduate Research Fellowship. JDL acknowledges support of Open Philanthropy, NSF IIS 2107304,  NSF CCF 2212262, NSF CAREER Award 2144994, and NSF CCF 2019844.
AT acknowledges support from National Science Foundation grants IIS1705121, IIS1838017, IIS2046613, IIS2112471, and funding from Meta, Morgan Stanley, Amazon, Google, and Scribe. JC would like to thank Jim and Marilyn Simons for their support of basic research via the Flatiron Institute. Any opinions, findings and conclusions or recommendations expressed in this material are those of the author(s) and do not necessarily reflect the views of any of these funding agencies.

The authors are grateful for feedback from Nikhil Ghosh, Yiding Jiang, Sam Sokota, and Zhili Feng.

%% file: appendix-0-flow_derivations.tex
\newcommand{\U}{\mathcal{U}}

\section{Central Flow Derivations}
\label{sec:flow_derivations}

In this appendix, we derive central flows for the three optimizers studied in this paper: \gd (\Cref{appendix:derivations:gd}), \rmsnorm (\Cref{appendix:derivations:rmsprop_norm}), and \rmsprop (\Cref{appendix:derivations:rmsprop}).  We reiterate that these derivations rely on informal mathematical reasoning; our claim that the central flow accurately matches the optimizer trajectory is ultimately supported \emph{empirically} by the experiments described in \Cref{sec:experiments}.  An interesting direction for future work is to prove that, under realistic conditions, the discrete optimizer follows the central flow.

\subsection{Preliminaries}

\subsubsection{Notation}
\label{sec:flow_derivations:notation}

We use $L(w)$ to denote the training objective, a function of weights $w \in \R^d$. We will assume that $L$ is three-times differentiable. We will frequently use $H(w)$ as a shorthand for $\nabla^2 L(w)$, the Hessian matrix at $w$.

 We use $\ev{A, B}$ to denote the Frobenius inner product $\tr(A^\top B)$ between two matrices $A$ and $B$, equivalent to flattening the matrices into vectors and taking the dot product. We use $\ker$ and $\spn$ to denote the kernel and span of a matrix, and $\dim$ to denote the dimension of a vector space.

 We use $\sym(\R^d)$ to denote the set of $d \times d$ symmetric matrices. For a $k$-dimensional subspace $\U \subseteq \R^d$, we use $\sym(\U)$ to denote the set of $d \times d$ symmetric matrices whose span is contained within $\U$.  Equivalently, this is the set of matrices that can be written as $U X U^\top$, where $U \in \R^{d \times k}$ is a basis for $\U$ and $X \in \sym(\R^k)$.

For a subspace $\U \subseteq \R^d$ and a matrix $A \in \R^{d \times d}$, we use $A \evalshort_\U$ to denote the \emph{restriction} of $A$ to $\U$, that is,
\begin{align}
    A \evalshort_\U := \Pi_\U \, A \, \Pi_\U,
\end{align}
where $\Pi_\U$ is the matrix that projects onto $\U$, e.g. $\Pi_\U = U U^\top$ for any orthogonal basis $U$ of $\U$.

For a subspace $\U \subseteq \R^d$ and a matrix $A \in \sym(\R^d)$, we write $A \succeqOver{\U} 0$ to denote that $A$ is positive semidefinite (PSD) over the subspace $\U$, i.e.   $u^\top A u \ge 0 \; \forall u \in \U$, or equivalently, $A \evalshort_{\U} \succeq 0$.
We analogously define $A \preceqOver{\U} 0$ to mean that $A$ is negative semidefinite (NSD) over $\U$.

We will need to work with higher-order tensors, though we will try our best to keep the tensor notation at a minimum.  For an order-$k$ tensor $T$ and an order-$s$ tensor $X$, with $s < k$, we define the contraction $T[X]$ as the order-$(k-s)$ tensor obtained by multiplying $T$ and $X$ componentwise along the last $s$ coordinates of $T$ and all coordinates of $X$, and then summing over those coordinates.
$$
(T[X])_{i_1,\ldots,i_{k-s}} \;:=\; \sum_{j_1,\ldots,j_s}
T_{i_1,\ldots,i_{k-s},j_1,\ldots,j_s} \, X_{j_1,\ldots,j_s}.
$$
We will also write $T[u_1,\ldots,u_s] $ for the result of contracting $T$ sequentially with the vectors $u_1,\ldots,u_s$ one at a time, equivalent to $T[u_1 \otimes \ldots \otimes u_s]$.  When $T$ is fully symmetric, permuting $u_1,\ldots,u_s$ does not change the result.

It can be helpful to reshape a higher-order tensor into a matrix; this viewpoint lets us work with higher-order tensors while using the familiar language of linear algebra.
For example, consider $\nabla^3 L(w)$, the order-3 tensor of third derivatives.
By flattening together the first two dimensions, we could view this as a matrix of shape $\R^{d^2 \times d}$.  In tensor product notation, we are identifying $\nabla^3 L$ as an element of $\R^{d^2} \otimes \R^d$. (We note that understanding tensor product notation is not necessary for understanding this paper.)
One could similarly identify  $\nabla^3 L$ as an element of $\R^{d \times d} \otimes \R^d$ or $\sym(\R^d) \otimes \R^d$.
These views naturally correspond to linear operators; for example, an element of the tensor product space $\sym(\R^d) \otimes \R^d$ can be treated as a linear operator $\R^d \to \sym(\R^d)$.  We will switch freely among the full-tensor, tensor-product, and operator views as convenient.

In particular, we will use the notation $\nabla H(w) \in \sym(\R^d) \otimes \R^{d}$  to denote the reshaping of $\nabla^3 L(w)$ that collects together the first two indices (i.e. the ``gradient of the Hessian''):
\begin{align*}
    \nabla H(w)_{ij,p} := \nabla^3 L(w)_{ijp} = \frac{\partial H_{ij}(w)}{ \partial w_p }.
\end{align*}
Intuitively, for any direction $v \in \R^d$, the operator $\nabla H(w): \R^d \to \sym(\R^d)$ returns the directional derivative of the Hessian $H(w)$ when moving in the direction $v$:
\begin{align*}
    \qty[\nabla H(w)[v]]_{ij} = \sum_{p=1}^d \frac{\partial H_{ij}(w)}{\partial w_p}\,v_p.
\end{align*}
Meanwhile, for any matrix $\Sigma \in \sym(\R^d)$, the transpose operator $\nabla H(w)^\top: \sym(\R^d) \to \R^d$ returns the gradient of the $\Sigma$-weighted Hessian $\ev{\Sigma, H(w)}$:
\begin{align*}
    \qty[\nabla H(w)^\top[\Sigma]]_p = \sum_{i=1}^d\sum_{j=1}^d \frac{\partial H_{ij}(w)}{\partial w_p} \Sigma_{ij} \implies \nabla H(w)^\top[\Sigma ] =  \nabla_w \ev{\Sigma, H(w)}.
\end{align*}

For a vector space $V$, we abbreviate $V \otimes V$ as $V^{\otimes 2}$, e.g. we abbreviate $\sym(\R^d) \otimes \sym(\R^d)$ as $\sym(\R^d)^{\otimes 2}$.

\subsubsection{Third-order Taylor expansions}
\label{sec:third-order-taylor-expansions}

If $L: \R^d \to \R$ is $k$-times differentiable, then Taylor's theorem for the $k$-th order Taylor expansion of $L$ is:
\begin{align*}
    L(w + \delta) = \sum_{j=0}^k \frac{1}{j!} \nabla^j L(w)[\delta^{\otimes j}] + o(\norm{\delta}^k).
\end{align*}
In particular, the third-order Taylor expansion of $L$ around any $w\in \R^d$ is given by:
\begin{align}
    L(w + \delta) = L(w) + \nabla L(w)[\delta] + \tfrac{1}{2} \nabla^2 L(w)[\delta, \delta] + \tfrac{1}{6} \nabla^3 L(w)[\delta, \delta, \delta] + o(\|\delta\|^3).
\end{align}
Likewise, the second-order Taylor expansion of the \emph{gradient} $\nabla L$ around $w$ is:
\begin{align}
    \nabla L(w + \delta) = \nabla L(w) + \nabla^2L(w)[\delta] + \tfrac{1}{2} \nabla^3 L(w)[\delta, \delta] + o(\| \delta\|^2).
\end{align}
Since $\nabla^3 L(w)[\delta, \delta] = \nabla^3 L(w) [\delta \delta^\top]$, and recalling our $H$, $\nabla H$ notation from the preceding section, we can equivalently write this Taylor expansion in the form:
\begin{align}
    \nabla L(w + \delta) = \nabla L(w) + H(w)[\delta] + \tfrac{1}{2} \nabla H(w)^\top[\delta \delta^\top] + o(\| \delta\|^2).
    \label{eq:gradient-taylor-expansion}
\end{align}
This is the form that we will directly use in our central flow derivations.

\subsubsection{Complementarity}\label{sec:derivations:complementarity}

A \emph{complementarity relation} is a constraint on two non-negative variables which enforces that at least one of the variables is zero, i.e. both cannot be strictly positive.  An example is the following set of three conditions over the two scalar-valued variables $x, y \in \R$:
\begin{align}
    x \ge 0,\quad y\ge 0, \quad xy = 0.
\end{align}
By convention, such a condition is often abbreviated using the shorthand notation:
\begin{align}
    0 \le x \perp y \ge 0.
\end{align}
Complementarity relations can be extended to vectors and matrices.  We will be particularly interested in the matrix case.  For two symmetric matrices $X, Y \in \sym(\R^d)$, consider the following complementarity relation:
\begin{align}
    X \succeq 0,\quad Y\succeq 0, \quad \langle X, Y \rangle = 0, \label{eq:matrix-complementarity}
\end{align}
which we will often abbreviate using the shorthand:
\begin{align}
    0 \preceq X \perp Y \succeq 0.
\end{align}
This condition is equivalent to $X,Y$ being PSD with orthogonal spans (\Cref{sec:arbitrary-preconditioned}, \Cref{fact:orthogonal-spans}).
It is also equivalent to $X,Y$ being PSD with $\spn X \subseteq \ker Y$ (\Cref{sec:arbitrary-preconditioned}, \Cref{fact:orthogonal-spans-corollary}).

\subsubsection{Semidefinite complementarity problems}
\label{sec:derivations:sdcp}
The \emph{semidefinite complementarity problem} (SDCP) will be a recurring primitive throughout this work.
An SDCP asks to find a matrix $\Sigma \in \sym(\R^d)$ that is complementary to an affine function of itself.
Namely, let $\alpha \in \sym(\R^d)$ be a symmetric matrix and let $\beta \in \sym(\R^d)^{\otimes 2}$ be a tensor, viewed as a linear operator over symmetric matrices $\sym(\R^{d}) \to \sym(\R^d)$.
The semidefinite complementarity problem is to find a matrix $\Sigma \in \sym(\R^d)$ such that:
\begin{align}
    0 \preceq \Sigma \perp \alpha + \beta[\Sigma] \succeq 0. \label{eq:sdcp-full}
\end{align}
This is a generalization of the well-studied linear complementarity problem from vectors in the non-negative orthant to matrices in the positive semidefinite cone.

\begin{remark}
    It is easily verified that if $\beta^{-1}[-\alpha] \succeq 0$, then the linear inverse $\Sigma = \beta^{-1}[-\alpha]$ is a solution to the SDCP \cref{eq:sdcp-full}.
    Interestingly, along the central flows, this will be the solution to the SDCP at almost all times.
    \label{remark:linear-inverse}
\end{remark}

\begin{remark}
    In the scalar-valued case, where $\alpha, \beta \in \R$, one can verify by case-checking that the SDCP
    \begin{align}
        0 \le \sigma^2 \perp \alpha + \beta \sigma^2 \ge 0
    \end{align}
    has the closed-form solution $\sigma^2 = \max(- \tfrac{\alpha}{\beta}, 0)$ provided that $\beta > 0$.
    \label{remark:sdcp-1d}
\end{remark}

It will be useful to restrict the domain of an SDCP to an arbitrary subspace $\U \subseteq \R^d$.
Recall that we use $\sym(\U)$ to denote the set of symmetric matrices whose span is contained within $\U$.
We thus have the following definition:
\begin{definition}[Semidefinite Complementarity Problem]
For a subspace $\U \subseteq \R^d$, matrix $\alpha \in \sym(\U)$, and tensor $\beta \in \sym(\U)^{\otimes 2}$, we define the solution set of the SDCP as:
\begin{align}
    \sdcp_\U(\alpha,\beta) := \{\Sigma \in \sym(\U) ~:~ 0 \preceq \Sigma \perp \alpha + \beta[\Sigma] \succeqOver{\U} 0\}. \label{eq:sdcp}
\end{align}
\end{definition}
A priori, it is unclear whether the SDCP has zero, one, or many solutions.
The following lemma shows that the SDCP always has one unique solution if $\beta$ is symmetric and positive definite as an operator over $\sym(\U)$, i.e if:
\begin{align*}
    \beta[\Sigma, \Sigma'] = \beta[\Sigma', \Sigma] \quad \forall \, \Sigma, \Sigma' \in \sym(\U) \qand \beta[\Sigma, \Sigma] > 0 \quad \forall \, \Sigma \in \sym(\U) \backslash \{0\}.
\end{align*}
\begin{lemma}
    If $\beta$ is symmetric and positive definite over $\sym(\U)$, then the cardinality of the solution set satisfies $\abs{\sdcp_\U(\alpha,\beta)}=1$.
    \label{lemma:sdcp:unique}
\end{lemma}
\begin{proof}
    Consider the following quadratic program with a semidefinite constraint:
    \begin{align}
        \min_{\Sigma \in \sym(\U)} \ev{\alpha,\Sigma} + \tfrac{1}{2} \beta[\Sigma,\Sigma] \quad \text{subject to} \quad \Sigma \succeq 0. \label{eq:sdcp-quadratic-program}
    \end{align}
    As $\beta \succ 0$,
    the objective is strictly convex so there is a unique minimizer $\Sigma^\star$. The KKT conditions for $\Sigma^\star$ are exactly $0 \preceq \Sigma^\star \perp \alpha + \beta[\Sigma^\star] \succeq 0$ and $\Sigma \in \sym(\U)$, so $\Sigma^\star \in \sdcp_\U(\alpha,\beta)$. Similarly if $\Sigma \in \sdcp_\U(\alpha,\beta)$ then $\Sigma$ satisfies the KKT conditions for the strictly convex semidefinite quadratic program, so $\Sigma = \Sigma^\star$.
\end{proof}
Note that this lemma is a straightforward adaptation of a standard argument for linear complementarity problems.
In the case where the lemma applies and the solution to the SDCP is unique, we will overload notation and use $\sdcp_\U(\alpha,\beta)$ to denote this unique solution.

\paragraph{Efficient computation}
    Fortunately, solving $\sdcp_\U(\alpha, \beta)$ does not actually require materializing $\alpha \in \sym(\R^d)$ and $\beta \in \sym(\R^d)^{\otimes 2}$ in full.
    Let $k := \dim \U$ denote the dimension of $\U$, which is typically $\ll d$, and let $U \in \R^{d \times k}$ denote a basis for $\U$.
    Then any $\Sigma \in \sym(\U)$ can be expressed as $\Sigma = U X U^\top$ for some $X \in \sym(\R^k)$.  The SDCP condition \cref{eq:sdcp} then reduces to a $k$-dimensional SDCP over $\R^k$:
    \begin{align}
        0 \preceq X \perp \alpha_U + \beta_U[X] \succeq 0 \iff X \in \sdcp_{\R^k}(\alpha_U, \beta_U),
    \end{align}
    where the matrix $\alpha_U \in \sym(\R^k)$ is defined as:
    \begin{align}
        \alpha_U:= U^\top \alpha U \quad\iff\quad (\alpha_U)_{ij} = u_i^\top \alpha u_j \label{eq:sdcp-basis-alpha}
    \end{align}
    and the tensor $\beta_U \in  \sym(\R^k)^{\otimes 2}$ is defined via its action as:
    \begin{align}
        \beta_U[X] := U^\top \beta[U X U^\top] U \quad\iff\quad (\beta_U)_{ijpq} = \beta[u_i,u_j,u_p,u_q]. \label{eq:sdcp-basis-beta}
    \end{align}
    Thus, to solve the original $d$-dimensional problem $\Sigma \in \sdcp_\U(\alpha, \beta)$, one can instead solve the $k$-dimensional problem $X \in \sdcp_{\R^k}(\alpha_U, \beta_U)$, and then represent $\Sigma = U X U^\top$.

    To solve $\sdcp_{\R^k}(\alpha_U, \beta_U)$, we use a standard convex solver to solve the convex program \cref{eq:sdcp-quadratic-program}:
    \begin{align}
        \min_{X \in \sym(\R^k)} \ev{\alpha_U, X} + \tfrac{1}{2} \beta_U[X, X] \quad \text{subject to} \quad X \succeq 0. \label{eq:sdcp-basis-quadratic-program}
    \end{align}

\subsubsection{On local time averaging}\label{appendix:time-average}

We intentionally do not specialize to a specific notion of ``local time-average''. The only properties of the local time-averaging operator $\E$ that we use are:
\begin{enumerate}
    \item linearity, i.e. $\E[f+g] = \E[f] + \E[g]$ and $\E[cf] = c \E[f]$ for any constant $c$
    \item the local time average of a constant $c$ is itself: $\E[c] = c$
    \item in the EOS regime when the sharpness fluctuates around $2/\eta$, the time-average is coarse enough to smooth out these fluctuations so that $S(\E[w_t]) = 2/\eta$
\end{enumerate}
One reason why we do not further define the time-averaging operator is that even the appropriate timescale for the averaging operation (e.g. window size or kernel bandwidth) depends nontrivially on the local dynamics.  Recall that in the relatively simple setting of one unstable eigenvalue that was analyzed in \citet{damian2023selfstabilization}, the EOS dynamics consists of consecutive cycles where the sharpness rises above, then falls below, the critical threshold $2/\eta$.  In this setting, it is natural to choose an averaging timescale so as to average over a cycle.  Yet, the analysis of \citet{damian2023selfstabilization} shows that the length of the cycle depends on the initial position of the iterate along the top Hessian eigenvector at the instant where the sharpness crosses $2/\eta$ (the closer the iterate is to the directionwise optimum, the longer the cycle).  Hence, even the choice of timescale is very nontrivial.

\subsubsection{Smoothness of the central flows}
\label{appendix:smoothness}
Central flows are ordinary differential equations of the form: $\frac{dw(t)}{dt} = f(w)$, where $f$ is not continuous everywhere.
Thus, the ODE should be interpreted in the sense of Carathéodory: $w(t)$ is only differentiable at \emph{almost} all $t$, and can have points of non-differentiability where the left and right derivatives differ.  For example, at the instant when the gradient descent central flow first reaches EOS, the right and left derivatives of $w(t)$ will differ, as the left derivative is $-\eta \nabla L(w)$ while the right derivative is $-\eta \Pi_{\nabla S(w)}^\perp \nabla L(w)$.  However, the central flow is \emph{right}-differentiable for all $t$.  Therefore, when we write $\tfrac{d}{dt}$ (e.g. in \Cref{prop:gd:loss-decrease}), it can either be interpreted as holding for almost all $t$, or it can be alternatively interpreted as a statement about the right derivative.

\subsection{Gradient Descent}\label{appendix:derivations:gd}

We now derive the central flow for gradient descent.  In \Cref{sec:gd:single} we considered the special case where one eigenvalue is at the edge of stability, and is continuing to remain there.  The complete central flow, derived here, applies in the more general setting where multiple eigenvalues are potentially at the edge of stability.  It also allows eigenvalues to enter and leave the edge of stability when appropriate.

This section is structured as follows:
\begin{enumerate}
    \item First, in \Cref{appendix:derivations:gd:dvi}, we  formulate the central flow as a \emph{differential complementarity problem} (DCP): a dynamical system defined implicitly by combining differential equations with complementarity constraints.
    \item Next, in \Cref{appendix:derivations:gd:sdcp}, we show that this DCP can be re-formulated into an ordinary differential equation with an explicit right-hand side that involves the solution to a semidefinite complementarity problem.
    \item In \Cref{appendix:derivations:gd:projection}, we show that the central flow can be equivalently formulated as a \emph{projected gradient flow} that projects the negative gradient onto the tangent cone of the stable region.  Leveraging this projection interpretation, we prove that the central flow decreases the loss monotonically (\Cref{prop:gd:loss-decrease}), but at a slower rate than the unregularized gradient flow (\Cref{prop:gd:slowdown}).
    \item Finally, in \Cref{appendix:derivations:gd:discretizing}, we describe how to discretize the central flow in practice.
\end{enumerate}

\subsubsection{The Differential Complementarity Problem (DCP) formulation}
\label{appendix:derivations:gd:dvi}
The central flow $w(t)$ will model the time-averaged trajectory of gradient descent $\{w_t\}$:
\begin{align}
    w(t) := \E[w_t].
\end{align}
Let $\delta_t := w_t - w(t)$ denote the displacement between gradient descent and the central flow (i.e. ``the oscillation'') at step/time $t$.  From the definition of $w(t)$ as the time-average, it follows that $\E[\delta_t] = 0$.
Let $\Sigma(t) := \E[\delta_t \delta_t^\top]$ denote the covariance of these oscillations.
That is, we are modeling the gradient descent trajectory as:
\begin{align}
    w_t = w(t) + \delta_t, \quad \text{where} \quad \E[\delta_t] = 0 \quad  \text{and}  \quad \E[\delta_t \delta_t^\top] = \Sigma(t).
\end{align}
Recall that gradient descent oscillates along the eigenvectors that are at the edge of stability.
When the Hessian has multiple eigenvalues at $2/\eta$, the corresponding eigenvectors are not individually identifiable, since any linear combination of eigenvectors is also an eigenvector. Instead, what is identifiable is the corresponding eigenspace, i.e. the linear subspace comprising all such eigenvectors. We refer to this eigenspace as the \emph{critical subspace}:
\begin{definition}[Critical subspace for gradient descent]
    Given weights $w \in \mathbb{R}^d$, the \emph{critical subspace} $\U(w) \subseteq{\R^d}$ is defined as the Hessian's eigenspace corresponding to the eigenvalue $2/\eta$:
    \begin{align}
        \U(w) := \ker\qty[H(w) - \tfrac{2}{\eta} I] = \left \{u \in \R^d: H(w) \, u = \tfrac{2}{\eta} u \right \}.
    \end{align}
    \label{def:critical-subspace-gd}
\end{definition}
Thus, we assume that the oscillations $\{\delta_t\}$ are fully contained within the critical subspace:
\begin{align}
    \delta_t \in \U(w(t)) \iff \spn[\Sigma(t)] \subseteq \U(w(t)).
\end{align}
We will now derive the central flow.
By Taylor expansion of $\nabla L$ around $w(t)$ (see \Cref{sec:third-order-taylor-expansions}, \cref{eq:gradient-taylor-expansion}), the gradient at step $t$ is:
\begin{align}
    \nabla L(w_t) \approx \nabla L(w(t)) + H (w(t)) \delta_t + \tfrac{1}{2} \nabla H(w(t))^\top[\delta_t \delta_t^\top].
\end{align}
Time-averaging both sides and using that $\E[\delta_t]=0$ and $\E[\delta_t \delta_t^\top] = \Sigma(t)$ gives:
\begin{align}
    \E[\nabla L(w_t)] \approx \nabla L(w(t)) + \tfrac{1}{2} \nabla H(w(t))^\top[\Sigma(t)].
\end{align}
Therefore, we make the ansatz that the central flow $w(t)$ takes the form:
\begin{align}
    \frac{dw}{dt} = -\eta \qty[\nabla L(w) + \tfrac{1}{2} \nabla H(w)^\top[\Sigma(t)]],
    \label{eq:appendix:gd:ansatz}
\end{align}
for some unknown $\Sigma(t)$ which we will now determine.

To solve for $\Sigma(t)$, we impose three conditions for all times $t$: 
\begin{enumerate}
    \item \textbf{PSD:} As a covariance matrix, $\Sigma(t)$ is positive semidefinite (PSD): $\Sigma(t) \succeq 0$.
    \item \textbf{Stability:} The sharpness remains bounded by $2/\eta$: $H(w(t)) \preceq (2/\eta) I$.
    \item \textbf{Complementarity:} The oscillations are contained within the critical subspace: $\spn[\Sigma(t)] \subseteq \U(w(t))$.
\end{enumerate}
It will now be helpful to define the ``residual'' matrix $A(w)$ as:
\begin{align}
    A(w) := \tfrac{2}{\eta} I - H(w).
\end{align}
Notably, $\ker A(w)$ is precisely the critical subspace $\U(w)$, i.e. the eigenspace of $H(w)$ with eigenvalue $2/\eta$. With this notation, Conditions 1-3 can be expressed as:
\begin{align}
    \underbrace{\Sigma(t) \succeq 0}_{\text{PSD}}, \quad \underbrace{A(w(t)) \succeq 0}_{\text{stability}}, \quad \underbrace{\spn[\Sigma(t)] \subseteq  \ker A(w(t))}_{\text{complementarity}}.
\end{align}
As discussed in \Cref{sec:derivations:complementarity}, these three conditions are equivalent to the complementarity relation: 
\begin{align}
    \Sigma(t) \succeq 0, \quad A(w(t)) \succeq 0,  \quad  \Sigma(t) \perp A(w(t)),
\end{align}
which we write more compactly as:
\begin{align}
    0 \preceq \Sigma(t) \perp A(w(t)) \succeq 0. \label{eq:gd:complementarity-relation}
\end{align}

We say $(w(t),\Sigma(t))$ follow the central flow if they satisfy \cref{eq:appendix:gd:ansatz} along with this complementarity relation:\footnote{We only require $w(t),\Sigma(t)$ to satisfy \cref{eq:appendix:gd:ansatz} for ``almost all $t$'' as $w(t)$ may not be differentiable when an eigenvalue enters or leaves EOS. This is in line with the standard definition of a differential complementarity problem.}

\newtheoremstyle{named}{1em}{1em}{\itshape}{}{\bfseries}{}{1em}{\thmnote{#3}}
\theoremstyle{named}
\newtheorem*{flow}{}

\begin{definition}[Gradient Descent Central Flow, DCP Formulation] \label{def:gd:flow:dcp}
    We say that $\{(w(t),\Sigma(t))\}_{t \ge 0}$ follow the gradient descent central flow if, for almost all $t$, they satisfy \cref{eq:appendix:gd:ansatz} along with the conditions: $0 \preceq \Sigma(t) \perp A(w(t)) \succeq 0$, where $A(w) := \frac{2}{\eta} I - H(w)$.
\end{definition}
\Cref{def:gd:flow:dcp} is an example of a \emph{differential complementarity problem} (DCP) \citep{stewart2011dynamics}, which are described in more detail in \Cref{sec:derivations:dcp}.  In a DCP, $\Sigma(t)$ is not defined explicitly, but is rather defined \emph{implicitly} via a complementarity relation that the trajectory $(w(t), \Sigma(t))$ is required to satisfy.

A priori, it is not clear that a feasible $\Sigma(t)$ exists or is unique.  To give an explicit expression for $\Sigma(t)$, and thereby turn \Cref{def:gd:flow:dcp} into an ODE with an explicit right-hand side, we will next show that for almost all times $t$, $\Sigma(t)$ must be the unique solution to a certain semidefinite complementarity problem.

\subsubsection{The Ordinary Differential Equation (ODE) formulation}
\label{appendix:derivations:gd:sdcp}

Before deriving the ODE formulation of the central flow, let us first explain why the DCP formulation, \Cref{def:gd:flow:dcp}, fails to immediately specify $\Sigma(t)$.
At any instant $t$, there can be multiple $\Sigma$'s which satisfy $0 \preceq \Sigma \perp A(w(t)) \succeq 0$; for example, the trivial choice $\Sigma = 0$ always works.
Yet, most of these $\Sigma$'s would cause the stability constraint $A(w(t+\epsilon)) \succeq 0$ to be violated if the dynamics \cref{eq:appendix:gd:ansatz} are run for an infinitesimal amount of time $\epsilon$.  Intuitively, we need to meld together the static constraint $0 \preceq \Sigma(t) \perp A(w(t)) \succeq 0$ with the dynamics \cref{eq:appendix:gd:ansatz}.

To do so, we appeal to \Cref{corollary:right_continuous_sdcp} in \Cref{sec:derivations:dcp}.
This result essentially ``differentiates'' the complementarity relation $0 \preceq \Sigma(t) \perp A(w(t)) \succeq 0$, to yield a new complementarity relation between $\Sigma(t)$ and the time derivative $\frac{d}{dt} A(w(t))$.
In particular, \Cref{corollary:right_continuous_sdcp} implies that under the flow defined by \Cref{def:gd:flow:dcp}, we must have:
\begin{align}
    0 \preceq \Sigma(t) \perp \tfrac{d}{dt} A(w(t)) \succeqOver{\U(w)}  0. \label{eq:gd:differential-complementarity-relation}
\end{align}
We have thus turned a ``position-level'' constraint on the residual $A(w(t))$ into a ``velocity-level'' constraint on its time derivative $\tfrac{d}{dt} A(w(t))$. We now expand $\tfrac{d}{dt} A(w(t))$ to reveal its dependence on $\Sigma(t)$:
\begin{align*}
    \tfrac{d}{dt} A(w(t)) &= \nabla A(w) \qty[\tfrac{dw}{dt}] \tag{chain rule} \\
    &= - \nabla H(w) \qty[ \tfrac{dw}{dt} ] \tag{definition of $A$} \\
    &= - \nabla H(w) \qty[  -\eta \qty[ \nabla L(w) + \tfrac{1}{2} \nabla H(w)^\top[\Sigma(t)] ] ] \tag{form of $\tfrac{dw}{dt}$} \\
    &=  \underbrace{\eta \nabla H(w) \qty[\nabla L(w)]}_{=:\alpha(w)} + \underbrace{\tfrac{1}{2} \eta \nabla H(w) \nabla H(w)^\top }_{=:\beta(w)} \qty[\Sigma(t)] \tag{linearity}
\end{align*}
This reveals that $\tfrac{d}{dt} A(w(t))$ is \emph{affine} in $\Sigma(t)$.
Namely, if we define the matrix $\alpha(w)$ and tensor $\beta(w)$ as:
\begin{align}
    \alpha(w) := \eta \nabla H(w) \qty[\nabla L(w)] \in \sym(\R^d), \quad \beta(w) := \tfrac{\eta}{2} \nabla H(w) \nabla H(w)^\top \in \sym(\R^d)^{\otimes 2}, \label{eq:appendix:gd:alpha-beta}
\end{align}
then  $\tfrac{d}{dt} A(w(t))$ is given by the affine expression:
\begin{align}
     \tfrac{d}{dt} A(w(t)) = \alpha(w) + \beta(w)[\Sigma(t)].
\end{align}

Substituting this into \cref{eq:gd:differential-complementarity-relation} implies that $\Sigma(t)$ must satisfy the complementarity relation:
\begin{align}
    0 \preceq \Sigma(t) \perp \alpha(w) + \beta(w)[\Sigma(t)]  \succeqOver{\U(w)} 0. \label{eq:gd:sdcp}
\end{align}
Since $\Sigma(t) \in \sym(\U(w))$, this is precisely an SDCP (\Cref{sec:derivations:sdcp}) defined over the critical subspace $\U(w)$:
\begin{align}
\Sigma(t) \in \sdcp_{\U(w)} \qty(\alpha(w), \beta(w)). \label{eq:appendix:gd:sigmat}
\end{align}

Thus, we are now ready to state the ODE formulation of the central flow.

\begin{definition}[Gradient Descent Central Flow, ODE Formulation]\label{def:gd:flow:ode}
    We say $\{w(t)\}_{t \ge 0}$ follows the \gd central flow if for almost all $t \ge 0$, $w(t)$ satisfies \cref{eq:appendix:gd:ansatz} for some $\Sigma(t) \in \sdcp_{\U(w(t))}\qty(\alpha(w(t)), \beta(w(t)))$.
\end{definition}

This DCP-to-ODE conversion can be straightforwardly generalized to a broader class of DCPs, and this is done in \Cref{sec:derivations:dcp}, \Cref{lem:dcp_to_sdcp}.  Our subsequent central flow derivations will directly invoke \Cref{lem:dcp_to_sdcp}.

\paragraph{Existence and uniqueness of SDCP solution $\Sigma(t)$}
Recall from \Cref{sec:derivations:sdcp}, \Cref{lemma:sdcp:unique} that the SDCP \cref{eq:gd:sdcp} will have a unique solution $\Sigma(t)$ when $\beta(w)$ is symmetric and positive definite as a linear operator over $\sym(\U)$.
Due to its outer product structure, $\beta(w)$ is always symmetric and positive \emph{semi}-definite.
If $\beta$ is also full rank as an operator acting on $\sym(\U(w))$, then it is positive \emph{definite}, and $\Sigma(t)$ is unique.
Empirically, in our experiments, we always do observe that this full-rank condition is satisfied (it is equivalent to full-rankness of $\beta_U(w)$ defined below in \cref{eq:alpha_beta_U_efficient}), and thus $\Sigma(t)$ is unique.
Note that existence and uniqueness of $\Sigma(t)$ does not necessarily imply existence and uniqueness of the central flow ODE.

\paragraph{Existence and uniqueness of central flow}
Provided that $\beta(w)$ is full-rank as an operator on $\sym(\U(w))$ for all $w$, prior results imply existence and uniqueness for the gradient descent central flow (see \Cref{sec:derivations:dcp}).

\paragraph{One unstable eigenvalue}
As a sanity check, we now verify that when there is one eigenvalue at the edge of stability (i.e. when the critical subspace has dimension 1), \Cref{def:gd:flow:ode} recovers the central flow defined in \Cref{sec:gd:single}.

In general, a subspace $\U$ of dimension 1 has the form $\U = \spn u$
 for some $u \in \R^d$, so $\sym(\U) = \{ \sigma^2 \, u u^\top : \sigma^2 \in \R \}$, and $\sdcp_\U(\alpha, \beta)$ reduces to a 1-dimensional SDCP:
 \begin{align*}
     \sdcp_\U(\alpha, \beta) = \sigma^2 u u^\top \qc \sigma^2 = \sdcp_\R(\alpha_u, \beta_u) \qc \underbrace{\alpha_u := u^\top \alpha \, u}_{\in \R} \qc \beta_u :=  \underbrace{u^\top \beta[u u^\top ] u}_{\in \R}.
 \end{align*}
 Thus, $\sigma^2$ has the closed-form solution described in \Cref{remark:sdcp-1d}:
\begin{align*}
\sigma^2 = \max \qty(- \frac{\alpha_u}{ \beta_u } , 0).
\end{align*}
Therefore, when there is one eigenvalue at the edge of stability, $\Sigma(t)$ from \cref{eq:appendix:gd:sigmat} becomes:
\begin{align*}
    \Sigma(t) = \sigma^2 u u^\top \qc \sigma^2 = \max \qty( - \frac{ \alpha_u(w)}{\beta_u(w)} ) \qc \alpha_u(w) = u^\top \alpha(w) u \qc \beta_u(w) = u^\top \beta(w)[ u u^\top] u,
\end{align*}
where $u \in \R^d$ is the top eigenvector of $H(w)$ at $w$, and $\alpha(w), \beta(w)$ were defined in \cref{eq:appendix:gd:alpha-beta}.
These simplify to:
\begin{align}
    \alpha_u(w) = \eta \nabla L(w) ^\top \nabla S(w) \qc \beta_u(w) = \tfrac{\eta}{2} \| \nabla S(w) \|^2,
\end{align}
where we recall that $S(w)$ denotes the top eigenvalue of $H(w)$ at $w$.
Therefore:
\begin{align}
    \sigma^2 = \max \qty( \frac{- 2\nabla L(w)^\top \nabla S(w) }{  \|\nabla S(w)\|^2 }, 0 ).
\end{align}
When $\ev{-\nabla L(w), \nabla S(w)} > 0$, i.e. when progressive sharpening holds, this recovers \cref{eq:gd_x}.
Else, $\sigma^2 = 0$, and the central flow will leave the edge of stability.

Finally, since $\Sigma(t) = \sigma^2 u u^\top$, \cref{eq:appendix:gd:ansatz} reduces to:
\begin{align*}
    \frac{dw}{dt} = - \eta \qty[ \nabla L(w) + \tfrac{1}{2} \sigma^2 \nabla S(w) ],
\end{align*}
which recovers \cref{eq:gd:single}.

\paragraph{Predicting time-averages}
The central flow can predict the time-average of various quantities, such as the loss or squared gradient norm, along the \gd trajectory.  For any quantity $f(w)$, we write $\bar{f}(t)$ for the central flow's prediction for $\E[f(w_t)]$ at step $t$.

For example, the central flow's prediction $\bar L(t)$ for the time-averaged loss $\E[L(w_t)]$ at step $t$ is given by:
\begin{align}
    \E[L(w_t)] &=\E[L(w(t) + \delta_t)] \nonumber \\
     &\approx \E\qty[L(w(t)) + \nabla L(w(t))^\top\delta_t  + \tfrac{1}{2} \delta_t^\top H(w(t)) \delta_t ] \hspace{10px} \nonumber \tag{Taylor expansion} \\
    &= L(w(t)) + \tfrac{1}{2} \ev{ H(w(t)), \Sigma(t)} \nonumber \tag{$\E[\delta_t] = 0$, $\E[\delta_t \delta_t^\top] = \Sigma(t)$}  \\
    &= L(w(t)) + \tfrac{1}{2} \tr \qty[\tfrac{2}{\eta} \Sigma ]    \tag{$H \Sigma = \tfrac{2}{\eta} \Sigma$} \nonumber \\
    &= L(w(t)) + \tfrac{1}{\eta} \tr \Sigma(t) \nonumber \\
    &:= \bar L(t).  \label{eq:gd-predict-loss}
\end{align}
Similarly, the prediction for the time-averaged squared gradient norm $\E[\|\nabla L(w_t)\|^2]$ at step $t$ is:
\begin{align}
    \E[\|\nabla L(w_t)\|^2] &\approx \E\qty[ \| \nabla L(w(t)) + H(w(t)) \delta_t \|^2 ] \nonumber \\
    &= \|\nabla L(w(t)) \|^2 + \ev{ H^2(w(t)), \Sigma(t)} \nonumber \\   
    &= \| \nabla L(w(t)) \|^2 + \tfrac{4}{\eta^2} \tr \Sigma(t). \nonumber \\
    &=:  \overline{\| \nabla L (t)\|^2}
    \label{eq:gd-predict-gradient-norm-sq}
\end{align}
In general, for any function $f(w)$, the central flow predicts that the time-average of $f(w_t)$ at step $t$ is:
\begin{align}
     \E[f(w_t)] &\approx \E \qty[ f(w(t)) + \nabla f(w(t))^\top \delta_t + \tfrac{1}{2} \delta_t^\top \nabla^2 f(w(t)) \delta_t ] \nonumber \\
    &=f(w(t)) + \tfrac{1}{2} \ev{\nabla^2 f(w(t)), \Sigma(t)} \nonumber \\
    &:= \bar{f}(t) \label{eq:gd-predict-general}
\end{align}
(Note: our prediction \cref{eq:gd-predict-gradient-norm-sq} for the squared gradient norm does not fit this template, as we choose to do a first-order of expansion of $\nabla L$ and then take the norm, rather than do a second-order expansion of $f(w) = \| \nabla L(w) \|^2$.)

The central flow can also predict the covariance with which \gd oscillates around the central flow.
Let $\Sigma(t) = V(t) \, \Lambda(t) \,  V(t)^\top$ be the (reduced) eigenvalue decomposition of the rank-$k$ matrix $\Sigma(t)$, where $V(t) \in \R^{d \times k}$ and $\Lambda(t) \in \diag(\R^k)$.  Define $x_t := V(t)^\top (w_t - w(t)) \in \R^k$ as the displacement of \gd from the central flow along these eigenvectors.
Then the central flow predicts that the covariance of these displacements is:
\begin{align*}
    \E[x_t x_t^\top] = V(t)^\top \E[\delta_t \delta_t^\top] V(t) = V(t)^\top \Sigma(t) V(t) = \Lambda(t).
\end{align*}
In particular, if we consider the $i$-th diagonal entry, the central flow predicts that the variance of oscillations along the $i$-th eigenvector of $\Sigma(t)$ should be equal to the $i$-th eigenvalue of $\Sigma(t)$:
\begin{align}
    \E \qty[\qty(v_i(t)^\top(w_t - w(t)))^2] = \lambda_i(t). \label{eq:gd-predict-oscillation-variance}
\end{align}

\paragraph{Basis-dependent version}
Naively computing the central flow's $\tfrac{dw}{dt}$ would be impractical, as storing $\Sigma(t)$ and $\alpha(w)$ would require $O(d^2)$ space, and storing $\beta(w)$ would require $O(d^4)$ space.
Fortunately, because all necessary quantities are supported on the low-rank critical subspace, the central flow's $\tfrac{dw}{dt}$ can be computed efficiently using only $O(k^2d + k^4)$ space, where $k = \dim \U(w)$ is the dimension of the critical subspace, which is typically $\ll d$.

In particular, fix $t$, and let $U \in \R^{d \times k}$ be a basis for the critical subspace $\U(w(t))$.
Then recall from \Cref{sec:derivations:sdcp} that $\Sigma(t)$ can be represented as $\Sigma(t) = U X U^\top$ for some low-dimensional matrix $X \in \sym(\R^k)$ that solves\linebreak$X \in \sdcp_{\R^k}(\alpha_U(w), \beta_U(w))$, where $\alpha_U \in \sym(\R^k)$ and $\beta_U \in \sym(\R^k)^{\otimes 2}$ were defined in \cref{eq:sdcp-basis-alpha,eq:sdcp-basis-beta}.

Now we define $H_U(w) := U^\top H(w) U \in \sym(\R^k)$ and its gradient $\nabla H_U(w) \in \sym(\R^k) \otimes \R^d$:
\begin{align}
    H_U(w)_{ij} := u_i^\top H(w) u_j \qand \nabla H_U(w)_{ij} := \nabla_w [u_i^\top H(w) u_j]. \label{eq:appendix:gd:nabla-H-U}
\end{align}
The tensor $\nabla H_U(w)$ only requires $O(k^2d)$ space to store, and can be computed in $O(k^2 d)$ time by looping over all pairs $(u_i, u_j)$ of columns of $U$ and computing the third derivative $\nabla_w \qty[u_i^\top H(w) u_j] \in \R^d$.
Crucially, computing $\tfrac{dw}{dt}$ only requires access to the smaller $\nabla H_U(w)$ rather than the full $\nabla H(w)$.
To see this, note that:
\begin{align}
    \nabla H_U(w)[v] = U^\top \nabla H(w)[v] U \qand \nabla H_U(w)^\top[X] = \nabla H(w)^\top[U X U^\top].
\end{align}
Thus, the central flow \cref{eq:appendix:gd:ansatz} takes the form:
\begin{align}
    \frac{dw}{dt} &= -\eta \qty[ \nabla L(w) + \tfrac{1}{2} \nabla H^\top_U(w)[X] ], \label{eq:central-flow-efficient}
\end{align}
and $\alpha_U(w) \in \sym(\R^k), \beta_U(w) \in \sym(\R^k)^{\otimes 2}$ take the form:
\begin{align}
    \alpha_U(w) = \eta \nabla H_U(w)\qty[\nabla L(w)]  \qc
    \beta_U(w) = \tfrac{\eta}{2} \nabla H_U(w) \, \nabla H_U(w)^\top \label{eq:alpha_beta_U_efficient}.
\end{align}
Thus, to compute $\tfrac{dw}{dt}$, we can compute $\nabla H_U(w)$, then use this to compute $\alpha_U(w)$ and $\beta_U(w)$ via \cref{eq:alpha_beta_U_efficient}, then solve the $k$-dimensional problem $X \in \sdcp_{\R^k}(\alpha_U(w), \beta_U(w))$, and then compute $\tfrac{dw}{dt}$ via \cref{eq:central-flow-efficient}.

In practice, due to the non-smoothness of the central flow, we do not discretize the central flow by computing $\tfrac{dw}{dt}$ and taking an Euler step; instead, we directly discretize the DCP formulation, as described in \Cref{sec:derivations:dcp:discretize}.

The time-averaged predictions can also be computed efficiently given a basis.  If we pick $U$ to be orthonormal $(U^\top U = I)$, then the central flow's prediction \cref{eq:gd-predict-loss} for the time-averaged training loss at step $t$ is:
\begin{align}
    \bar L(t) &:= L(w(t)) +  \tfrac{1}{\eta} \tr \qty[U X U^\top ] \nonumber \\
     &= L(w(t)) +  \tfrac{1}{\eta} \tr \qty[ X U^\top U ] \nonumber \\
    &= L(w(t)) +  \tfrac{1}{\eta} \tr X. \label{eq:appendix:gd:time-average-loss-basis}
\end{align}
Similarly, the prediction \cref{eq:gd-predict-gradient-norm-sq} for the time-averaged squared gradient norm at step $t$ is:
\begin{align}
    \overline{ \| \nabla L(t) \|^2} &:=  \|\nabla L(w(t))\|^2 +  \tfrac{4}{\eta^2} \tr \qty[ U X U^\top ]. \nonumber \\
    &=  \|\nabla L(w(t))\|^2 +  \tfrac{4}{\eta^2}  \tr X. \label{eq:appendix:gd:time-average-sq-gradient-norm-basis}
\end{align}
In general, for any function $f(w)$, the prediction \cref{eq:gd-predict-general} can be computed as:
\begin{align*}
    \bar f(t) &= f(w(t)) + \tfrac{1}{2} \ev{U^\top \nabla^2 f(w(t)) U, X}.
\end{align*}
As for predicting the oscillation covariance, we can evaluate both sides of \cref{eq:gd-predict-oscillation-variance} without needing to materialize $\Sigma(t)$ in full.
If $X = U_X  \, \Lambda(t) \, U_X^\top$ denotes the eigenvalue decomposition of $X$, and if we define $V(t) = U U_X$, then $\Sigma(t) = V(t) \, \Lambda(t) \, V(t)^\top$ is the eigenvalue decomposition of $\Sigma(t)$. 
Note that $U_X$ and $X$ will depend on the basis $U$, while $V(t)$ and $\Lambda(t)$ are independent of $U$.

\paragraph{Smoothness of the central flow}
At a finite set of times, a new eigenvalue enters or leaves the edge of stability. 
We refer to these instants as \emph{breakpoints}.
In between the breakpoints, $\Sigma(t)$ is continuous and $w(t)$ is differentiable.
Moreover, the SDCP is solved by the linear inverse $\Sigma = -U \beta_U^{-1}[\alpha_U] U^\top$ where $\alpha_U,\beta_U$ are defined in \cref{eq:alpha_beta_U_efficient} (see \Cref{remark:linear-inverse}).  Further, $\tfrac{d}{dt} A(w(t)) \evalshort_{\U(w)} = 0$, i.e. all Hessian eigenvalues that are at EOS remain fixed at $2/\eta$.
However, at the breakpoints, $\Sigma(t)$ is discontinuous and $w(t)$ is not differentiable (although they are still right-continuous and right-differentiable, respectively).

\subsubsection{The Projection Formulation}
\label{appendix:derivations:gd:projection}

In this section we will show that the \gd central flow (\Cref{def:gd:flow:dcp} and \Cref{def:gd:flow:ode}) can be equivalently interpreted as projected gradient flow constrained to the \emph{stable set} $\mathbb{S}_\eta$, i.e. the subset of weight space where gradient descent is locally stable:
\begin{align}
    \mathbb{S}_\eta := \{w ~:~ S(w) \le 2/\eta\}.
\end{align}
For general constrained optimization problems, a projected gradient flow projects the negative gradient onto the \emph{tangent cone} of the constraint set before taking an infinitesimal step.  The tangent cone consists of the set of allowable directions that would not cause any constraints to be violated.

In our case, the tangent cone $T_{\mathbb{S}_\eta}(w)$ of the stable set $\mathbb{S}_\eta$ at the point $w \in \mathbb{S}_\eta$ is the set of directions that, to first order, would not increase the sharpness if we moved in that direction from $w$. This tangent cone is given by:
\begin{align}
    T_{\mathbb{S}_\eta}(w) = \{z \in \mathbb{R}^d: \nabla H(w)[z] \preceqOver{\U(w)} 0 \}.
\end{align}
Note that this is a convex cone, since it is closed under linear combinations with non-negative weights.

We use $\mathrm{proj}_M(\cdot)$ to denote the usual Euclidean projection onto a set $M \subseteq \R^d$:
\begin{align}
    \mathrm{proj}_M(v) = \argmin_{z \in M}  \| v - z \|_2^2.
\end{align}

Projecting a vector onto the tangent cone of the stable set involves solving a certain SDCP:

\begin{lemma}\label{lem:gd:SDCP_to_proj}
    The projection of a vector $v \in \R^d$ onto the tangent cone of $ \mathbb{S}_\eta$ at $w \in \mathbb{S}_\eta$ is given by:
    \begin{align}
        \mathrm{proj}_{T_{\mathbb{S}_\eta}(w)}[v] = v - \tfrac{1}{2} \nabla H(w)^\top[\Sigma]
    \quad \text{where} \quad
        \Sigma \in \sdcp_{\U(w)}\qty(-\nabla H(w)[v], \tfrac{1}{2} \nabla H(w) \nabla H(w)^\top),
    \end{align}
    where $\U(w) := \ker \qty[H(w) - \tfrac{2}{\eta} I]$ is the critical subspace (\Cref{def:critical-subspace-gd}).
\end{lemma}
\begin{proof}
    Recall that the tangent cone of $\mathbb{S}_\eta$ is the set: $\{z\in \R^d ~:~ \nabla H(w)[z] \preceqOver{\U(w)} 0\}$.  Therefore, the projection of $v$ onto this set is given by:
    \begin{align}
        \mathrm{proj}_{T_{\mathbb{S}_\eta}(w)}[v] = v + \delta^*,
    \end{align}
    where the perturbation $\delta^*$ is the optimal solution to the optimization problem:
    \begin{align}
        \min_\delta \|\delta\|^2 \qq{such that} \nabla H(w)[v + \delta] \preceqOver{\U(w)} 0 \label{eq:projection_optimization_problem}.
    \end{align}
    This is a quadratic program with a semidefinite constraint.
    Introducing a dual variable $\Sigma \in \sym(\U(w))$, the KKT conditions for this optimization problem are:
    \begin{align}
        \underbrace{\delta = -\tfrac{1}{2} \nabla H(w)^\top[\Sigma]}_{\text{stationarity}} \qc \underbrace{\ev{\Sigma, \nabla H(w)[v + \delta]} = 0}_{\text{complementary  slackness}} \qc \underbrace{\nabla H(w) [v + \delta] \preceqOver{\U(w)} 0}_{\text{primal feasibility}} \qc \underbrace{\Sigma \succeq 0}_{\text{dual feasibility}}.
    \end{align}
    Substituting the first condition into the middle two yields the following three conditions:
    \begin{align}
        0 \preceq \Sigma \perp -\nabla H(w)[v] + \tfrac{1}{2} \nabla H(w)\nabla H(w)^\top[\Sigma] \succeqOver{\U(w)} 0.
        \label{eq:projection:conditions}
    \end{align}
    We recognize these as precisely the characterization of an SDCP:
    \begin{align}
    \Sigma \in \sdcp_{\U(w)} \qty(-\nabla H(w)[v], \tfrac{1}{2} \nabla H(w) \nabla H(w)^\top).
    \label{eq:gd:projection-sdcp}
    \end{align}
    Therefore, if $\Sigma$ satisfies \cref{eq:gd:projection-sdcp}, then $\delta = - \tfrac{1}{2} \nabla H(w)^\top[\Sigma]$ is an optimal solution to the optimization problem \cref{eq:projection_optimization_problem}, and $v + \delta$ is the desired projection.
\end{proof}

We are now ready to state the projection formulation of the \gd central flow:

\begin{definition}[Gradient Descent Central Flow, Projection Formulation]\label{def:gd:flow:proj} We say that $\{w(t)\}_{t \ge 0}$ follows the \gd central flow if for almost all $t$,
    \begin{align}
        \frac{dw}{dt} = \mathrm{proj}_{T_{\mathbb{S}_\eta}(w)}[-\eta \nabla L(w)] \qq{where} \mathbb{S}_\eta := \{w ~:~ S(w) \le 2/\eta\}.
        \label{eq:gd_general_projection}
    \end{align}
\end{definition}
The equivalence between the projection formulation (\Cref{def:gd:flow:proj}) and the ODE formulation (\Cref{def:gd:flow:ode}) follows from applying \Cref{lem:gd:SDCP_to_proj} to the vector $v = - \eta \nabla L(w)$ and noting $\sdcp(\alpha,\beta) = \sdcp(c \alpha, c \beta)$ for $c > 0$.

\paragraph{Understanding the projection formulation}
We now give intuition for the projection formulation:
\begin{itemize}
    \item When $S(w) < 2/\eta$, $w$ is in the interior of $\mathbb{S}_\eta$ so $\U(w) = \emptyset$, the tangent cone is the entire space, and the projection is the identity map. Therefore \cref{eq:gd_general_projection} reduces to \gflow.

    \item When there is a single eigenvalue at $2/\eta$, $w$ is on the boundary of $\mathbb{S}_\eta$ and the tangent cone is given by the halfspace: $T_{\mathbb{S}_\eta}(w) = \{v:\ev{\nabla S(w), v} \le 0\}$. If the negative gradient lies outside this halfspace (i.e. if \gflow threatens to increase the sharpness above $2/\eta$), then the projection onto the halfspace is given by the projection onto the hyperplane: $-\eta \Pi_{\nabla S(w)}^\perp \nabla L(w)$ (see \Cref{fig:projection-illustration}). But, if the negative gradient already lies in the halfspace, the projection is the identity map, so the central flow follows \gflow and leaves EOS.

    \item In general, computing the projection onto $T_{\mathbb{S}_\eta(w)}$ requires solving a semidefinite quadratic program for which $\Sigma$ is the Lagrangian dual variable. The KKT conditions of this quadratic program are equivalent (up to a constant) to the SDCP that defines $\Sigma$ above.
\end{itemize}

\begin{figure}[t]
    \centering
    \includegraphics[width=0.45\linewidth]{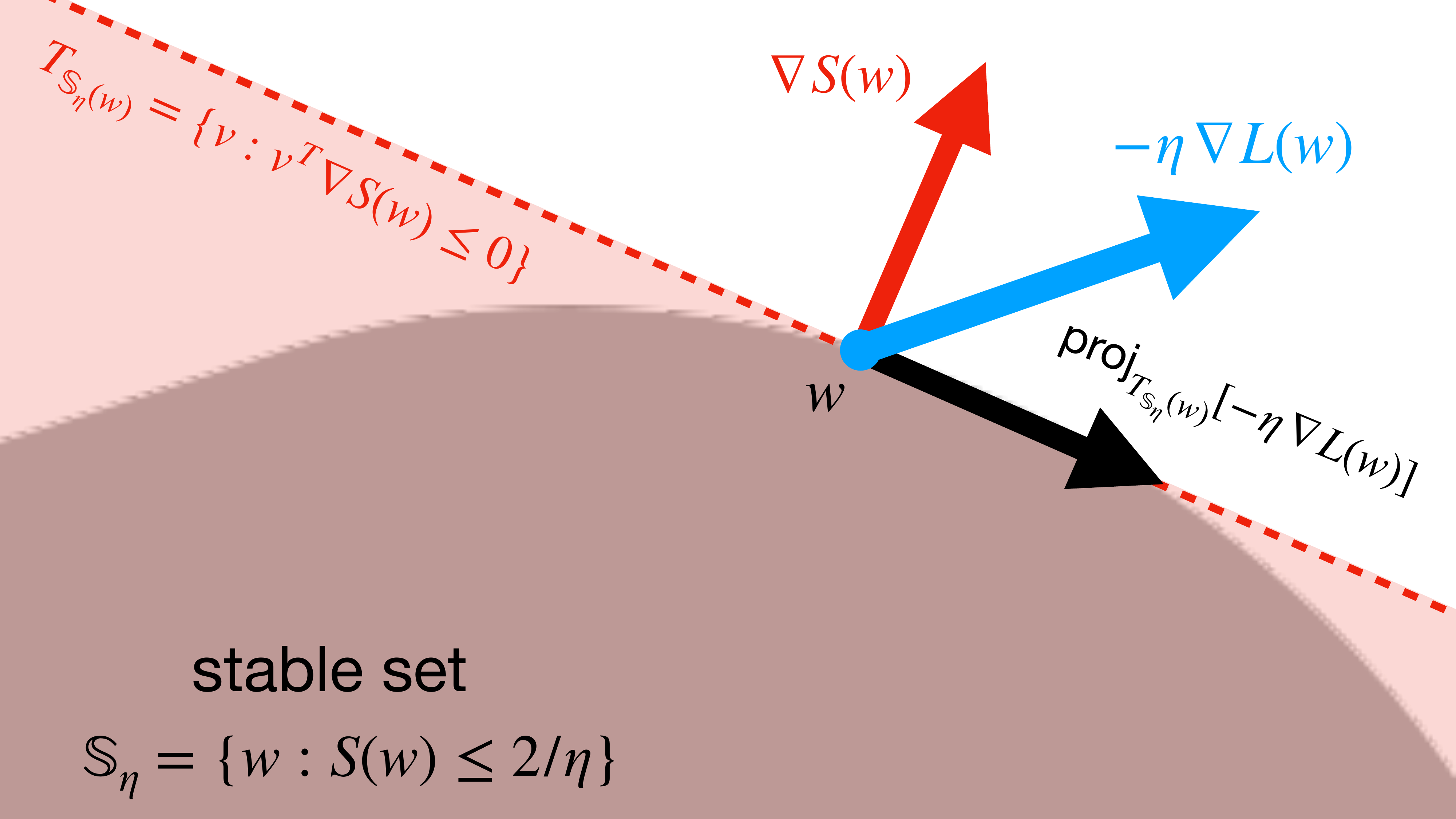}
    \hspace{0.1cm}
    \includegraphics[width=0.45\linewidth]{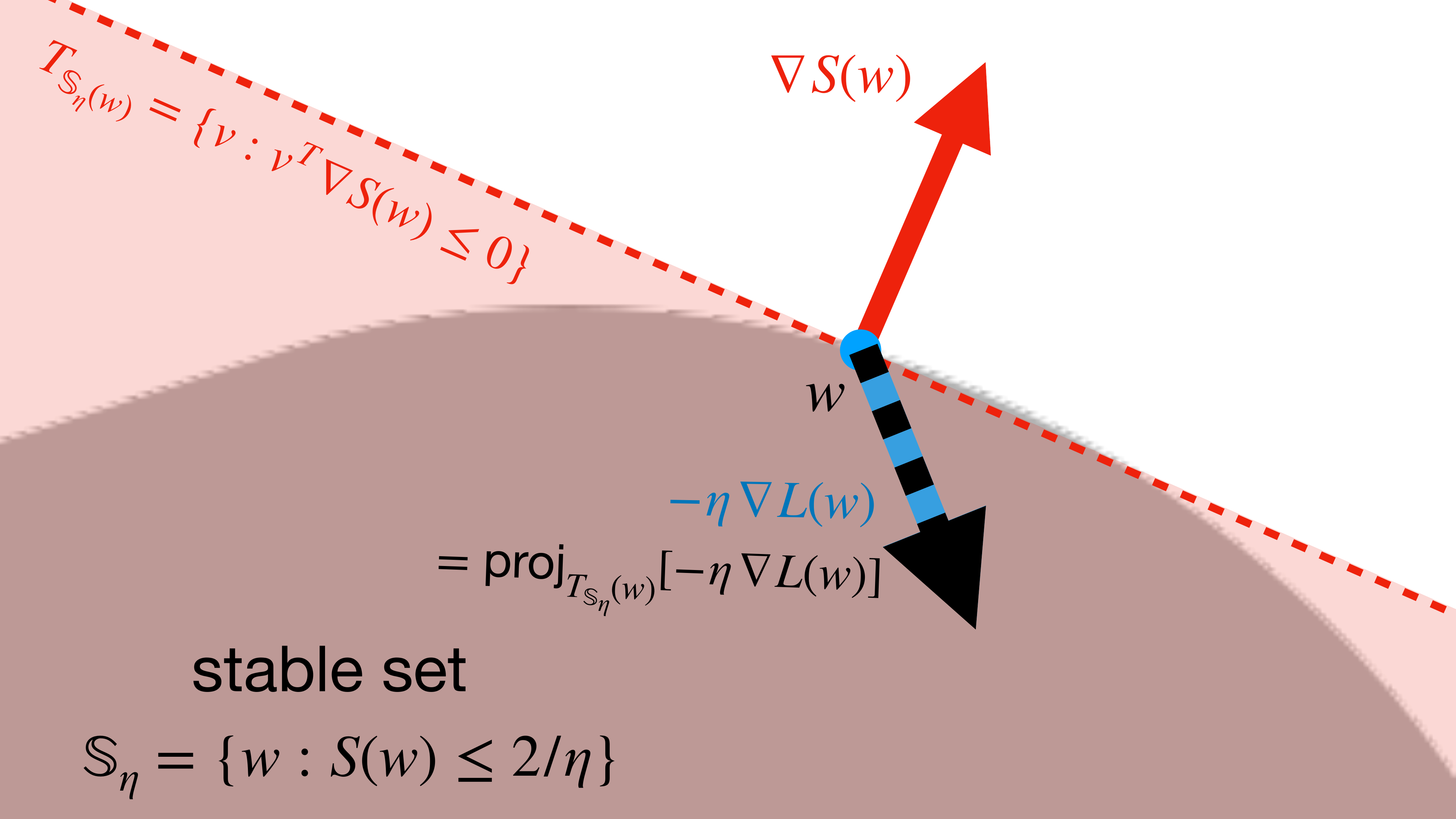}
    \caption{This cartoon illustrates projecting onto the tangent cone of the stable set  $T_{\mathbb{S}_\eta(w)}$ in the case where one eigenvalue is at the edge of stability.  The iterate $w$ is on the border of the stable set (grey blob).  The tangent cone is the half-space $\{v: v^\top \nabla S(w) \le 0\}$ (shaded red).  \textbf{Left}: on the one hand, if the negative gradient (blue arrow) points \emph{out of} the stable set, then the projection (black arrow) removes the component aligned with $S(w)$ (red arrow).  \textbf{Right}: on the other hand, if the negative gradient (blue arrow) already points \emph{into} the stable set, then the projection (black arrow) does nothing.}
    \label{fig:projection-illustration}
\end{figure}

\paragraph{Properties of projection} This projection formulation is helpful because Euclidean projection onto a convex cone shares some useful properties with Euclidean projection onto a linear subspace.  We will use these properties below to reason about the rate of loss decrease.

First, projection onto a convex cone $C$ is positive homogeneous: for any scalar $c > 0$ we have:
\begin{align}
    \mathrm{proj}_{C}[cv] = c \; \mathrm{proj}_{C}[v].
    \label{eq:projection-positive-homogeneous}
\end{align}
In our case, this can also be seen directly by combining the characterization of the projection in \Cref{lem:gd:SDCP_to_proj}, with the identity $X \in \sdcp(\alpha, \beta) \iff cX \in \sdcp(c \alpha, \beta)$.

Second, by the Moreau decomposition \citep{moreau1962decomposition}, any vector $v$ can be orthogonally decomposed into the projection onto a convex cone $C$ and the projection onto its dual cone $C^*$:
\begin{align}
    v = \mathrm{proj}_{C}[v] + \mathrm{proj}_{C^*}[v] \quad \text{where} \quad \langle \mathrm{proj}_{C}[v], \mathrm{proj}_{C^*}[v] \rangle = 0.
\end{align}
In particular, this implies that:
\begin{align}
\langle \mathrm{proj}_{C}[v], v - \mathrm{proj}_{C}[v] \rangle = 0.
\label{eq:projection-moreau-orthogonal}
\end{align}
and:
\begin{align}
\ev{v, \mathrm{proj}_{C}[v]} = \| \mathrm{proj}_{C}[v]\|^2.
\label{eq:projection-moreau-orthogonal2}
\end{align}
In our case, $C$ is the tangent cone to the stable set at $w$, its dual cone $C^*$ is the so-called \emph{normal cone} to the stable set at $w$: $\{\nabla H(w)^\top[\Sigma]: \Sigma \in \sym(\U(w)), \Sigma \succeq 0 \}$, and \cref{eq:projection-moreau-orthogonal} can be proved by rearranging the complementarity relation in \cref{eq:projection:conditions}.

\paragraph{Rate of loss decrease}
We now use the projection formulation (\Cref{def:gd:flow:proj}) to reason about the rate of loss decrease under the central flow.  We first show the following helper lemma:
\begin{lemma}\label{lem:gd_helper}
    Under the gradient descent central flow (\Cref{def:gd:flow:proj}), for almost all $t$ we have
    \begin{align}
        \frac{dL(w)}{dt} = - \eta \norm{\mathrm{proj}_{T_{\mathbb{S}_\eta}(w)}[- \nabla L(w)]}^2.
    \end{align}
\end{lemma}
\begin{proof}
    By the chain rule, we have
    \begin{align*}
        \frac{dL(w)}{dt}
        &= \ev{\nabla L(w), \frac{dw}{dt}} \\
        &= \ev{\nabla L(w), \mathrm{proj}_{T_{\mathbb{S}_\eta}(w)}[-\eta \nabla L(w)]} \\
        &= - \eta \ev{- \nabla L(w), \mathrm{proj}_{T_{\mathbb{S}_\eta}(w)}[- \nabla L(w)]} \\
        &= - \eta \norm{\mathrm{proj}_{T_{\mathbb{S}_\eta}(w)}[- \nabla L(w)]}^2,
    \end{align*}
    where the first line is the chain rule, the second line is the projection formulation of the central flow, the third line is due to positive homogeneity of the projection operation \cref{eq:projection-positive-homogeneous}, and the last line is due to the orthogonality of the projection \cref{eq:projection-moreau-orthogonal2}.
\end{proof}
A simple corollary is that the training loss monotonically decreases under the \gd central flow:

\begingroup
\renewcommand{\theproposition}{\ref{prop:gd:loss-decrease}}
\begin{proposition}[Restated]
    Under the GD central flow (\Cref{def:gd:flow:proj}), for almost all $t$, the loss curve $L(w(t))$ is monotonically decreasing:
    \begin{align}
        \frac{dL(w(t))}{dt} \le 0.
    \end{align}
\end{proposition}
\begin{proof}
The claim follows by combining \Cref{lem:gd_helper} with the fact that a norm is always non-negative.
\label{prop:gd:loss-decrease-restated}
\end{proof}
\endgroup

Another simple corollary is that the central flow decreases the loss at a less steep rate than would gradient flow.  In other words, the oscillations induce a slowdown in the rate of loss decrease:
\begingroup
\renewcommand{\theproposition}{\ref{prop:gd:slowdown}}
\begin{proposition}[Restated]
    Under the GD central flow (\Cref{def:gd:flow:proj}), for almost all $t$, the slope of the loss curve is less steep than that of gradient flow:
    \begin{align}
        -\eta \|\nabla L(w(t))\|^2 \le \frac{dL(w(t))}{dt}.
    \end{align}
\end{proposition}
\label{prop:gd:slowdown-restated}
\endgroup
\begin{proof}
Projecting a vector $v$ onto a convex cone $C$ always makes the norm smaller:
\begin{align*}
    \|v\|^2 &= \| \mathrm{proj}_{C}[v] + [v - \mathrm{proj}_C[v]] \|^2 \\
    &= \| \mathrm{proj}_{C}[v]\|^2 + \| v - \mathrm{proj}_{C}[v] \|^2  \\
    &\ge \| \mathrm{proj}_{C}[v]\|^2,
\end{align*}
where the second line is due to the orthogonality of $\mathrm{proj}_{C}[v]$ and $v - \mathrm{proj}_{C}[v]$ i.e. \cref{eq:projection-moreau-orthogonal}, and the third line is due to the non-negativity of a square.

Thus, recalling \Cref{lem:gd_helper}, we have:
\begin{align*}
    \frac{dL(w)}{dt} &= - \eta \norm{\mathrm{proj}_{T_{\mathbb{S}_\eta}(w)}[- \nabla L(w)]}^2 \\
    &\ge - \eta \|- \nabla L(w) \|^2  \\
    &= - \eta \|\nabla L(w)\|^2.
\end{align*}
\end{proof}

\paragraph{Why not start with the projection formulation?}
Since the projection formulation of the central flow is arguably the simplest one, one might ask why we first went through the DCP/ODE formulations before arriving at the projection formulation.  After all, since the sharpness equilibrates at $2/\eta$ at EOS, one might think that a projected gradient flow constrained to the set $\{w: S(w) \le 2/\eta\}$ is already a natural approximation.  The trouble with this thinking is that there are actually an infinite number of flows which keep the sharpness locked at $2/\eta$, moving within the tangent cone of the stable set. Among these, the significance of the central flow (\Cref{def:gd:flow:proj}) is that it follows the particular vector within the tangent cone that is \emph{closest} (in Euclidean distance) to the negative gradient; that is, it makes the \emph{smallest perturbation} to the negative gradient that will force it inside the tangent cone.  A priori, there is no reason why this should be the case, and thus jumping straight to the projection formulation would be arbitrary. In fact, we will see in our analyses of \rmsnorm and \rmsprop that the central flows \emph{do not} pick the closest tangent vector in Euclidean distance and their central flows cannot be interpreted as a projected gradient flow.

\subsubsection{Discretizing the \protect \gd central flow}
\label{appendix:derivations:gd:discretizing}

Discretizing the central flow is nontrivial, because the flow is nonsmooth at points where the dimension of the critical subspace (i.e. the number of unstable eigenvalues) undergoes a change.
To discretize the flow, we directly discretize the DCP formulation (\Cref{def:gd:flow:dcp}) rather than going through the ODE formulation (\Cref{def:gd:flow:ode}).
We describe our general procedure for discretizing DCPs in \Cref{sec:derivations:dcp:discretize}. 
Let us now describe how this general procedure specializes to the gradient descent case.

We use $w^{(t)}$, $\Sigma^{(t)}$ to denote our estimate for the central flow's $w(t)$, $\Sigma(t)$.
Let $\epsilon > 0$ be the discretization step size, e.g. $\epsilon = 0.25$.
For some tolerance $\tau > 0$, e.g. $\tau = \tfrac{0.05}{\eta}$, we will regard Hessian eigenvalues greater than $\tfrac{2}{\eta} - \tau$ as those which might become unstable in the next discretization time step.

At each discretization step, we first compute all Hessian eigenvalues that are greater than $\tfrac{2}{\eta} - \tau$, as well as the corresponding eigenvectors.
Let $k$ be the number of such eigenvalues, let $D \in \R^{k \times k}$ be a diagonal matrix containing such eigenvalues on the diagonal, and let $U \in \R^{d \times k}$ be the corresponding orthonormal eigenvectors.
Then, we compute the tensor $\nabla H_U$ as in \cref{eq:appendix:gd:nabla-H-U}, though note that $U$ now refers to a basis of eigenvectors whose eigenvalues are \emph{almost} 2 rather than exactly equal to 2.
Then, we compute $\alpha_U$ and $\beta_U$ as in \cref{eq:alpha_beta_U_efficient}.
Then, we solve the following $k$-dimensional SDCP:
\begin{align}
    X^{(t)} = \sdcp_{\R^k}(\tfrac{2}{\eta} I - D + \epsilon \; \alpha_U, \epsilon \, \beta_U),
\end{align}
so that $\Sigma^{(t)} = U X^{(t)} U^\top$.  Then, we update the weights via:
\begin{align}
    w^{(t+\epsilon)} &= w^{(t)} - \epsilon \; \eta \qty[\nabla L(w^{(t)}) + \tfrac{1}{2} \nabla H_U^\top[X^{(t)} ] ].
\end{align}
To predict the time-average of the train loss, the squared gradient, and the covariance of the oscillations, we use \cref{eq:appendix:gd:time-average-loss-basis}, \cref{eq:appendix:gd:time-average-sq-gradient-norm-basis}, and \cref{eq:gd-predict-oscillation-variance}, respectively.

\subsection{Scalar RMSProp}\label{appendix:derivations:rmsprop_norm}

We model the \rmsnorm iterates $\{w_t\}$ as oscillating around a central flow $w(t)$ with covariance $\Sigma(t)$. That is, if $\delta_t := w_t - w(t) $ denotes the displacement (``the oscillation''), then $\E[\delta_t] = 0$ and $\E[\delta_t \delta_t^\top] = \Sigma(t)$.
Let $\nu(t) := \E[\nu_t]$ model the time-averaged $\nu_t$, and we will frequently neglect the distinction between the two, implicitly assuming that $\nu_t$ concentrates tightly around $\nu(t)$.

A similar argument as for \gd implies that the time-averaged gradient is approximately:
\begin{align}
    \E[\nabla L(w_t)] \approx \nabla L(w(t)) + \tfrac{1}{2} \nabla H(w(t))^\top [\Sigma(t)]. \label{eq:scalar-rmsprop:time-averaged-gradient}
\end{align}
Meanwhile, we approximate the time-average of the squared gradient norm as:\footnote{We note here that this is not a ``faithful'' second-order Taylor expansion of the gradient $\nabla L(w_t)$ around $w(t)$. We are implicitly assuming that $\norm{\E[\nabla L(w_t)]}^2 \approx \norm{\nabla L(\E[w_t])}^2$ which neglects the term $\nabla^3 L(w(t))[\nabla L(w(t)), \Sigma(t)]$. This omission simplifies the expressions and our numerical experiments still successfully predict $\Sigma(t)$ across a variety of datasets and architectures, which justifies this omission.}
\begin{align}
    \E[\| \nabla L(w_t) \|^2] &\approx \E[ \|\nabla L(w(t)) + H (w(t)) \, \delta_t\|^2 ] \nonumber\\
    &= \|\nabla L(w(t))\|^2 + \E[\|H(w(t)) \, \delta_t\|^2 ] \label{eq:scalar-rmsprop:grad-norm-sq}    
\end{align}
where the first line is a first-order Taylor expansion of $\nabla L$ around $w(t)$ and the second line is because $\E[\delta_t] = 0$.

We define the critical subspace $\U(w, \nu)$ as the eigenspace of the effective Hessian $\tfrac{\eta}{\sqrt{\nu}} H(w)$ corresponding to the eigenvalue 2:
\begin{align}
    \U(w, \nu) := \ker\qty[\tfrac{\eta}{\sqrt{\nu}} H(w) - 2I].
\end{align}
We model \rmsnorm as oscillating within the critical subspace, i.e. we assume that $\delta_t \in \U(w(t), \nu(t))$.

This implies that $H(w(t)) \, \delta_t = \tfrac{2 \sqrt{\nu(t)}}{\eta} \delta_t$, which lets us simplify \cref{eq:scalar-rmsprop:grad-norm-sq} as:
\begin{align}
    \E[\| \nabla L(w_t) \|^2]& \approx \|\nabla L(w(t))\|^2 + \E[\|H(w(t)) \, \delta_t\|^2 ] \nonumber \\
    &=  \|\nabla L(w(t)) \|^2 + \tfrac{4 \nu(t)}{\eta^2} \E[\|\delta_t\|^2] \nonumber \\
    &= \|\nabla L(w(t)) \|^2 + \tfrac{4 \nu(t)}{\eta^2} \tr\Sigma(t), \label{eq:scalar-rmsprop:grad-norm-sq-better}
\end{align}
where the final line is because $\E[\|\delta_t\|^2] = \E[\tr(\delta_t \delta_t^\top)] = \tr \Sigma(t)$.

Based on the time averages \cref{eq:scalar-rmsprop:time-averaged-gradient} and \cref{eq:scalar-rmsprop:grad-norm-sq-better}, we make the ansatz that the joint dynamics of $(w_t, \nu_t)$ can be modeled by a central flow $(w(t), \nu(t))$ of the form:
\begin{align}
\begin{split}
    \frac{dw}{dt} &= -\frac{\eta}{\sqrt{\nu}}\qty[\nabla L(w) + \tfrac{1}{2} \nabla H(w)^\top[\Sigma(t)]] \\
    \frac{d\nu}{dt} &= \tfrac{1-\beta_2}{\beta_2}\qty[\norm{\nabla L(w)}^2 + \tfrac{4 \nu}{\eta^2} \mathrm{tr}(\Sigma(t)) - \nu]. \label{eq:appendix:rmsnorm:ansatz}
\end{split}
\end{align}
As with \gd, to determine $\Sigma(t)$ we will impose three conditions for all times $t$:
\begin{enumerate}
    \item \textbf{PSD:} As a covariance matrix, $\Sigma(t)$ is positive semidefinite (PSD), i.e. $\Sigma(t) \succeq 0$.
    \item \textbf{Stability:} The effective sharpness remains bounded by $2$, i.e. $\tfrac{\eta}{\sqrt{\nu}} H(w(t)) \preceq 2 I$.
    \item \textbf{Complementarity:} The oscillations are contained within the critical subspace, i.e $\spn \Sigma(t) \subseteq \U(w(t), \nu(t))$.
\end{enumerate}

To write these concisely, it will be convenient to define the matrix-valued ``residual'' function $A(w, \nu)$ by:
\begin{align}
    A(w,\nu) := \tfrac{2 \sqrt{\nu}}{\eta} I - H(w),    
\end{align}
so that stability is equivalent to $A(w,\nu) \succeq 0$ and the critical subspace is precisely $\U(w, \nu) = \ker A(w, \nu)$.  With this notation, we can concisely express the above three conditions as a semidefinite complementarity relation:
\begin{align*}
    0 \preceq \Sigma(t) \perp A(w(t),\nu(t)) \succeq 0.
\end{align*}
We say that $(w(t),\Sigma(t))$ follow the \rmsnorm central flow if they follow \cref{eq:appendix:rmsnorm:ansatz} along with this semidefinite complementarity relation:
\begin{definition}[Scalar RMSProp Central Flow, DCP Formulation]\label{def:rmsnorm:flow:dcp}
    We say that $\{(w(t),\nu(t),\Sigma(t))\}_{t \ge 0}$ satisfy the \rmsnorm central flow if they satisfy \cref{eq:appendix:rmsnorm:ansatz} and $0 \preceq \Sigma(t) \perp A(w(t),\nu(t)) \succeq 0$ for almost all $t$.
\end{definition}

Since this DCP can be expressed in the general form described in \Cref{sec:derivations:dcp}, we can use \Cref{lem:dcp_to_sdcp} to convert it into an equivalent ODE.
In particular, if $\widetilde w := [w,\nu]^\top$ denotes the augmented state, we can write \cref{eq:appendix:rmsnorm:ansatz} as:
\begin{align*}
    \frac{d\widetilde{w}}{dt} = f\qty(\widetilde{w}) + B\qty(\widetilde{w})[\Sigma] \qq{where} 
    f\qty(\widetilde{w}) := \begin{bmatrix}
        -\frac{\eta}{\sqrt{\nu}} \nabla L(w) \\ \frac{1-\beta_2}{\beta_2} [\norm{\nabla L(w)}^2 - \nu]
    \end{bmatrix} \qand
    B\qty(\widetilde{w})[\Sigma] := \begin{bmatrix}
        -\frac{\eta}{2\sqrt{\nu}} \nabla H(w)^\top[\Sigma] \\ \frac{1-\beta_2}{\beta_2} \cdot \frac{4\nu}{\eta^2} \tr \Sigma
    \end{bmatrix},
\end{align*}
and we can write the complementarity relation as $0 \preceq \Sigma(t) \perp A(\widetilde{w}(t)) \succeq 0$.

Therefore, \Cref{lem:dcp_to_sdcp} implies that
$$\Sigma(t) \in \sdcp_{\U(w, \nu)}\qty(\alpha(w, \nu),\beta(w, \nu)),$$
where the matrix $\alpha(w, \nu) \in \sym(\R^d)$ and tensor $\beta(w, \nu) \in \sym(\R^d)^{\otimes 2}$ are defined by:
\begin{align}
    \alpha\qty(w, \nu)
    &:= \nabla A\qty(\widetilde{w})[f(\widetilde{w})] \nonumber \\
    &=
    \begin{bmatrix}
        -\nabla H(w) & \frac{1}{\eta \sqrt{\nu}} I
    \end{bmatrix}\begin{bmatrix}
        -\frac{\eta}{\sqrt{\nu}} \nabla L(w) \\ \frac{1-\beta_2}{\beta_2} \qty[\norm{\nabla L(w)}^2 - \nu]
    \end{bmatrix} \nonumber \\
    &= \frac{\eta}{\sqrt{\nu}} \nabla H(w)[\nabla L(w)] + \frac{1}{\eta \sqrt{\nu}} \frac{1-\beta_2}{\beta_2} \qty[\norm{\nabla L(w)}^2 - \nu] I \label{eq:scalar-rmsprop-alpha}
\end{align}
and
\begin{align}
    \beta\qty(w, \nu)[\Sigma]
    &:= \nabla A\qty(\widetilde{w})B(\widetilde{w})[\Sigma] \nonumber \\
    &=
    \begin{bmatrix}
        -\nabla H(w) & \frac{1}{\eta \sqrt{\nu}} I
    \end{bmatrix}\begin{bmatrix}
        -\frac{\eta}{2\sqrt{\nu}} \nabla H(w)^\top[\Sigma] \\ \frac{1-\beta_2}{\beta_2} \cdot \frac{4\nu}{\eta^2} \tr \Sigma
    \end{bmatrix} \nonumber \\
    &= \frac{\eta}{2\sqrt{\nu}}\nabla H(w) \nabla H(w)^\top[\Sigma] + \frac{4\sqrt{\nu}}{\eta^3} \cdot \frac{1-\beta_2}{\beta_2} \tr[\Sigma] I. \label{eq:scalar-rmsprop-beta}
\end{align}

Note that $\beta(\widetilde{w})$ is indeed symmetric PSD, as for \gd:
\begin{align*}
    &\beta(\tilde w)[\Sigma,\Sigma'] = \frac{\eta}{2 \sqrt{\nu}} \ev{\nabla H(w)^\top[\Sigma], \nabla H(w)^\top[\Sigma']} + \frac{4\sqrt{\nu}}{\eta^3} \cdot \frac{1-\beta_2}{\beta_2} \tr[\Sigma]\tr[\Sigma'] \\
    &\implies \beta(\tilde w)[\Sigma,\Sigma] = \frac{\eta}{2 \sqrt{\nu}} \norm{\nabla H(w)^\top[\Sigma]}^2 + \frac{4\sqrt{\nu}}{\eta^3} \cdot \frac{1-\beta_2}{\beta_2} \tr[\Sigma]^2 \ge 0.
\end{align*}
This gives us the ODE formulation of the \rmsnorm central flow:
\begin{definition}[\rmsnorm Central Flow, ODE Formulation]\label{def:rmsnorm:flow:ode}
    We say that $\{w(t),\nu(t)\}_{t \ge 0}$ follow the \rmsnorm central flow if for almost all $t$, they satisfy \cref{eq:appendix:rmsnorm:ansatz} for some $$\Sigma(t) \in \sdcp_{\U(w(t), \nu(t))}\qty(\alpha(w(t), \nu(t)),\beta(w(t), \nu(t))).$$
\end{definition}

\paragraph{One unstable eigenvalue}
As a sanity check, we now verify that this formulation recovers \cref{eq:rmsprop_norm_x} when there is one eigenvalue at the edge of stability.
Just as with \gd (see above), in this setting $\Sigma(t)$ reduces to:
\begin{align*}
    \Sigma(t) = \sigma^2 u u^\top \qc \sigma^2 = \max \qty( - \frac{ \alpha_u(w, \nu)}{\beta_u(w, \nu)}, 0 ) \qc \alpha_u(w, \nu) = u^\top \alpha(w, \nu) u \qc \beta_u(w, \nu) = u^\top \beta(w, \nu)[ u u^\top] u,
\end{align*}
where $u \in \R^d$ is the unit-norm top eigenvector of $H(w)$ at $w$. Simplifying, we have:
\begin{align*}
    \alpha_u(w,\nu) &= \frac{\eta}{\sqrt{\nu}} \ev{\nabla L(w), \nabla S(w)} + \frac{1}{\eta \sqrt{\nu}} \cdot \frac{1-\beta_2}{\beta_2}\qty[\norm{\nabla L(w)}^2 - \nu]  \\
    \beta_u(w,\nu) &= \frac{\eta}{2\sqrt{\nu}}\norm{\nabla S(w)}^2 + \frac{4 \sqrt{\nu}}{\eta^3} \cdot \frac{1-\beta_2}{\beta_2}.
\end{align*}
Then, when $\alpha_u(w, \nu)$ is negative, we have:
\begin{align}
    \sigma^2 &= -\frac{\alpha_u(w,\nu)}{\beta_u(w,\nu)} \nonumber \\
    &= - \frac{\frac{\eta}{\sqrt{\nu}} \ev{\nabla L(w), \nabla S(w)} + \frac{1}{\eta \sqrt{\nu}} \cdot \frac{1-\beta_2}{\beta_2} \qty[\norm{\nabla L(w)}^2 - \nu]}{\frac{\eta}{2\sqrt{\nu}}\norm{\nabla S(w)}^2 + \frac{4 \sqrt{\nu}}{\eta^3} \cdot \frac{1-\beta_2}{\beta_2}} \tag{definition of $\alpha_u,\beta_u$} \\
    &= \frac{\beta_2 \ev{-\nabla L(w), \nabla S(w)} + \frac{1}{\eta^2} \cdot (1-\beta_2) \qty[\nu - \norm{\nabla L(w)}^2]}{\beta_2 \cdot \tfrac{1}{2}\norm{\nabla S(w)}^2 + (1-\beta_2) \frac{4\nu}{\eta^4}} \tag{simplify} \\
    &= \frac{\beta_2 \ev{-\nabla L(w), \nabla S(w)} + (1-\beta_2) \qty[\frac{S(w)^2}{4} - \frac{\norm{\nabla L(w)}^2}{\eta^2}]}{\beta_2 \cdot \tfrac{1}{2}\norm{\nabla S(w)}^2 + (1-\beta_2) S(w)^2/\eta^2}.\qquad\qquad~ \tag{$\nu = \frac{\eta^2 S(w)^2}{4}$ at EOS}
\end{align}
which indeed recovers \cref{eq:rmsprop_norm_x}.
On the other hand, when $\alpha_u(w, \nu)$ is positive, $\sigma^2 = 0$ and the central flow reduces to the stable flow \cref{eq:rmsprop_norm_stable_flow}.

\paragraph{Normalized gradient descent}
As $\beta_2 \to 0$, \rmsnorm becomes normalized gradient descent, and the formula \cref{eq:rmsprop_norm_x} for $\sigma^2(w; \eta, \beta_2)$ reduces to:
\begin{align*}
    \sigma^2(w; \eta)  = \frac{\eta^2}{4} - \frac{\norm{\nabla L(w)}^2}{S(w)^2}.
\end{align*}
In practice, we observe that the first term generally dominates the second term. Thus, for NGD with one unstable eigenvalue, we have approximately:
\begin{align}
    \sigma^2(w; \eta) \approx \frac{\eta^2}{4}. \label{eq:ngd-approximation}
\end{align}
This clearly illustrates how $\sigma^2$ grows monotonically with $\eta$.
\label{remark:ngd}

\paragraph{Does the loss decrease?}
Whereas the \gd central flow decreases the loss monotonically (\Cref{prop:gd:loss-decrease}), this is not true for the \rmsnorm central flow.
Indeed, we have seen that with one unstable eigenvalue, the \rmsnorm central flow takes the form:
\begin{align*}
    \frac{dw}{dt} = -\frac{2}{S(w)} \qty[\nabla L(w) + \frac{1}{2} \sigma^2(w; \eta, \beta_2) \nabla S(w)].
\end{align*}
Thus, by the chain rule, the rate of change in the loss is given by:
\begin{align*}
    \frac{dL}{dt} &= \ev{\nabla L(w), \frac{dw}{dt}} \\
    &= -\frac{2}{S(w)} \qty[ \|\nabla L(w)\|^2 + \frac{1}{2} \sigma^2(w; \eta, \beta_2) \ev{\nabla L(w), \nabla S(w)} ].
\end{align*}
If progressive sharpening holds, i.e. if $\ev{\nabla L(w), \nabla S(w)} < 0$, then the second term is acting to \emph{increase} $L$, and for sufficiently large values of $\sigma^2$, this increase will outweigh the decrease from the first term, causing $L$ to go up.
Indeed, if progressive sharpening holds, then $\tfrac{dL}{dt} > 0$ so long as:
\begin{align*}
    \sigma^2(w; \eta, \beta_2)  > \frac{2 \|\nabla L(w)\|^2}{- \ev{\nabla L(w), \nabla S(w)}}.
\end{align*}
This can indeed occur.
For instance, if we consider the case of normalized gradient descent ($\beta_2 \to 0$) and if we make the approximation described in \cref{eq:ngd-approximation}, then $\tfrac{dL}{dt} > 0$ so long as the learning rate $\eta$ satisfies:
\begin{align}
    \eta > \sqrt{\frac{8 \|\nabla L(w)\|^2}{-\ev{\nabla L(w), \nabla S(w)}}}.
\end{align}
Thus, for sufficiently large learning rates $\eta$, the train loss will go up rather than down under the central flow.

\jnote{say it's an open question whether there is a potential function that does decrease?}

\paragraph{The effect of the hyperparameters $\eta, \beta_2$}
We can use the \rmsnorm central flow to reason about the effect of the algorithm's hyperparameters.  In \Cref{sec:rmsnorm:interpreting}, we discussed the case of a single eigenvalue at EOS; we showed that at any point $w \in \R^d$ in weight space, the strength $\sigma^2$ of the implicit sharpness regularization is monotonically increasing in $\eta$.  We also showed that $\beta_2$ monotonically interpolates $\sigma^2$ between a certain value for NGD ($\beta_2= 0$) and a certain value for GD ($\beta_1 = 1$). \Cref{prop:rmsnorm_multiple_monotonic} extends both of these results to the fully general setting where more than one eigenvalue can be at EOS, and eigenvalues can enter or leave EOS.

\begin{proposition}\label{prop:rmsnorm_multiple_monotonic}
    Fix an initial point $w = w(0)$. Let $\U \subseteq \R^d$ be the top eigenspace of $H(w)$, and assume that $\nabla H(w)$ has full rank on $\U$.  For a choice of parameters $\eta,\beta_2$, let $\nu = \nu(0) = \eta^2 S(w)^2/4$ so that the effective sharpness is $2$. Run the \rmsnorm central flow starting at $(w, \nu)$ with parameters $(\eta,\beta_2)$ to get $\{(w(t), \nu(t), \Sigma(t))\}_{t \ge 0}$. Then:
    \begin{itemize}
        \item $\tr \Sigma(0)$ is non-decreasing in $\eta$, i.e. larger learning rate leads to larger oscillations. Furthermore, $\frac{dS(w)}{dt}\evalshort_{t=0}$ is non-increasing in $\eta$, i.e. larger learning rates lead to more curvature reduction.
        \item $\tr \Sigma(0)$ is either non-decreasing or non-increasing in $\beta_2$, i.e. $\beta_2$ has a monotonic effect on oscillation size. Furthermore, if $\tr \Sigma(0)$ is non-decreasing in $\beta_2$, $\frac{dS(w)}{dt}\evalshort_{t=0}$ is non-increasing in $\beta_2$ and vice-versa, i.e. larger oscillations always lead to smaller curvature.
    \end{itemize}
\end{proposition}

Before proving \Cref{prop:rmsnorm_multiple_monotonic}, we begin with the following simple lemma for SDCPs:
\begin{lemma}
\label{lem:sdcp-monotonicity}
    If $0 \preceq X \perp Y \succeq 0$ and $0 \preceq X' \perp Y' \succeq 0$, then $\ev{X'-X, Y'-Y} \le 0$.
\end{lemma}
\begin{proof}
    \begin{align*}
        \ev{X'-X,Y'-Y} = \underbrace{\ev{X',Y'}}_{=0} - \underbrace{\ev{X',Y}}_{\ge 0} - \underbrace{\ev{X,Y'}}_{\ge 0} + \underbrace{\ev{X,Y}}_{=0} \le 0.
    \end{align*}
\end{proof}
The heavy lifting will be done by the following result:
\begin{proposition}\label{prop:heavy-lifting}
    For some scalar parameter $s \in \R$, consider the family of SDCP's:
    \begin{align*}
        \Sigma(s) =  \sdcp(\alpha(s), \beta(s)),
    \end{align*}
    where the parameterized matrix $\alpha(s) \in \sym(\R^k)$ and tensor $\beta(s) \in \sym(\R^k)^{\otimes 2}$ have the form:
    \begin{align*}
        \alpha(s) := \alpha_0 + s c_\alpha I \qand \beta(s) := \beta_0 + s c_\beta I \otimes I,
    \end{align*}
    for matrix $\alpha_0 \in \sym(\R^k)$, tensor $\beta_0 \in \sym(\R^k)^{\otimes 2}$, and scalars $c_\alpha, c_\beta \in \R$.  Assume $\beta_0 \succ 0$ and $c_\beta > 0$. Then for any $s \ge 0$,
    \begin{itemize}
        \item If $c_\alpha > 0$, then $\tr \Sigma(s)$ is non-increasing in the parameter $s$.
        \item In general, $\tr \Sigma(s)$ is monotonic (either non-increasing or non-decreasing) in the parameter $s$.
    \end{itemize}
    Furthermore, if $\lambda(s) := \lambda_{\min}(\alpha_0 + \beta_0[\Sigma(s)])$, then $\lambda(s)$ and $\tr \Sigma(s)$ are co-monotone: 
    \begin{align*}
        (\tr \Sigma(s')-\tr \Sigma(s))(\lambda(s') - \lambda(s)) \ge 0,
    \end{align*}
    meaning they either increase together or decrease together in $s$.
\end{proposition}
\begin{proof}
    Define the helper function $F(s) := s (c_\alpha + c_\beta \tr \Sigma(s))$, so that:
    \begin{align}
        \alpha(s) + \beta(s)[\Sigma(s)] &= \alpha_0 + s c_\alpha I + \beta_0[\Sigma(s)] + s c_\beta \tr \Sigma(s) I \nonumber \\
        &= \alpha_0 + \beta_0[\Sigma(s)] + F(s) I. \label{eq:F-residual}
    \end{align}
    For any $s, s' \ge 0$, \Cref{lem:sdcp-monotonicity} implies that:
    \begin{align*}
        0 &\ge \ev{\Sigma(s')-\Sigma(s), \alpha(s') + \beta(s')[\Sigma(s')] - \alpha(s) - \beta(s)[\Sigma(s)]} \\
        &= \ev{\Sigma(s')-\Sigma(s), \beta_0[\Sigma(s') - \Sigma(s)] + (F(s') - F(s)) I} \\
        &= \beta_0[\Sigma(s') - \Sigma(s),\Sigma(s')-\Sigma(s)] + (\tr \Sigma(s') - \tr \Sigma(s))(F(s') - F(s)) \\
        &\ge (\tr \Sigma(s') - \tr \Sigma(s))(F(s') - F(s)),
    \end{align*}
    where the last line is because $\beta_0$ is positive definite.
    Thus we see that $\tr \Sigma(s)$ and $F(s)$ are counter-monotone:
    \begin{align}
        (\tr \Sigma(s') - \tr \Sigma(s))(F(s') - F(s)) \le 0. \label{eq:tr-Sigma-F-countermonotone}
    \end{align}
    Next, we can write $F(s') - F(s)$ as:
    \begin{align*}
        F(s') - F(s) &=  c_\alpha (s'-s) + c_\beta (s' \tr \Sigma(s') - s \tr \Sigma(s)) \\
        &= (s'-s)\qty(c_\alpha + c_\beta \tr \Sigma(s)) + c_\beta s' [\tr \Sigma(s') - \tr \Sigma(s)].
    \end{align*}
    Plugging this into \cref{eq:tr-Sigma-F-countermonotone} gives:
    \begin{align}
        0 &\ge (\tr \Sigma(s') - \tr \Sigma(s))(s'-s)\qty(c_\alpha + c_\beta \tr \Sigma(s)) + c_\beta s' [\tr \Sigma(s') - \tr \Sigma(s)]^2 \nonumber \\
        &\ge (\tr \Sigma(s') - \tr \Sigma(s))(s'-s)\qty(c_\alpha + c_\beta \tr \Sigma(s)), \label{eq:tr-Sigma-inequality}
    \end{align}
    since $s', c_\beta \ge0$ by assumption. First, if $c_\alpha,c_\beta > 0$ then since $\tr \Sigma(s) \ge 0$, \cref{eq:tr-Sigma-inequality} implies that
    \begin{align*}
        (\tr \Sigma(s') - \tr \Sigma(s))(s' - s)\le 0,
    \end{align*}
    i.e. that $\tr \Sigma(s)$ is non-increasing in $s$, which was the first stated claim. More generally, for any pair $s,s' \ge 0$, applying \cref{eq:tr-Sigma-inequality} to both $s,s'$ and the reverse $s',s$ yields:
    \begin{align*}
        (\tr \Sigma(s') - \tr \Sigma(s))(s'-s)\qty(c_\alpha + c_\beta \tr \Sigma(s)) \le 0 \qand (\tr \Sigma(s') - \tr \Sigma(s))(s'-s)\qty(c_\alpha + c_\beta \tr \Sigma(s')) \le 0,
    \end{align*}
    which implies that the sign of $c_\alpha + c_\beta \tr \Sigma(\cdot)$ must be the same at both $s$ and $s'$.
    Since this hold for arbitrary $s, s'$ we conclude that the sign of $c_\alpha + c_\beta \tr \Sigma(\cdot)$ must be the same everywhere.  Thus, \cref{eq:tr-Sigma-inequality} implies that $\tr \Sigma(s)$ must be monotone in $s$, which proves the second stated claim.
    
    Finally, we will prove $\lambda(s) := \lambda_{\min}(\alpha_0 + \beta_0[\Sigma(s)])$ is co-monotone with $\tr \Sigma(s)$. To do so, we will prove that:
    \begin{align}
        \lambda(s) = \max\left(
            \lambda_{\min}(\alpha_0), -F(s)\right). \label{eq:lambda-s-expression}
    \end{align}
    To show this, we split into the two cases based on whether $\alpha(s)$ is PSD (or equivalently whether $\lambda_{min}(\alpha_0) + s c_\alpha \ge 0$).
    
    First, if $\alpha(s) \succeq 0$ then $\Sigma(s) = 0$ is the unique solution to the SDCP, so $\lambda(s) = \lambda_{\min}(\alpha_0)$. In addition, $\alpha(s) \succeq 0$ is equivalent to $\lambda_{min}(\alpha_0) + s c_\alpha \ge 0$ and since $F(s) = s c_\alpha$ whenever $\Sigma(s) = 0$, this implies that $\lambda_{min}(\alpha_0) \ge -F(s)$ which proves \cref{eq:lambda-s-expression} when $\alpha(s) \succeq 0$.

    Next, suppose that $\alpha (s) \not \succeq 0$, or equivalently that $\lambda_{min}(\alpha_0) + s c_\alpha < 0$. This implies that $\Sigma(s) \neq 0$ in order for it to satisfy the SDCP. Thus, there exists some $v \in \spn \Sigma(s)$. Since $\spn \Sigma(s) \subseteq \ker[\alpha(s) + \beta(s)[\Sigma(s)]]$, we have that:
    \begin{align*}
        v \in \ker \qty[\alpha(s) + \beta(s)[\Sigma(s)]] \iff \qty(\alpha_0 + \beta_0[\Sigma(s)])v = -F(s) v,
    \end{align*}
    which means that $-F(s)$ is an eigenvalue of $\alpha_0 + \beta_0[\Sigma(s)]$.
    Meanwhile, we also have that:
    \begin{align*}
        \alpha(s) + \beta(s)[\Sigma(s)] \succeq 0 \iff \alpha_0 + \beta_0[\Sigma(s)] \succeq - F(s) I,
    \end{align*}
    which means that $-F(s)$ lower-bounds all eigenvalues of $\alpha_0 + \beta_0[\Sigma(s)]$.  From these, we can conclude that $-F(s)$ is the smallest eigenvalue of $\alpha_0 + \beta_0[\Sigma(s)]$, i.e. $\lambda(s) = -F(s)$. Finally, note that complementarity implies that
    \begin{align*}
        0 = \ev{\Sigma(s), \alpha_0 + \beta_0[\Sigma(s)] + F(s)I} = \underbrace{\ev{\Sigma(s),\alpha_0}}_{\mathclap{\ge \tr \Sigma(s) \lambda_{min}(\alpha_0)}} + \underbrace{\beta_0[\Sigma(s),\Sigma(s)]}_{\ge 0} + F(s) \tr \Sigma(s) \ge \tr \Sigma(s)[\lambda_{min}(\alpha_0) + F(s)].
    \end{align*}
    Since $\Sigma(s) \ne 0$ and $\Sigma(s) \succeq 0$, we must have $\tr \Sigma(s) > 0$ so we can divide by $\tr \Sigma(s)$ to get $-F(s) \ge \lambda_{min}(\alpha_0)$ which completes the proof of \cref{eq:lambda-s-expression}.

    To conclude, we observe that since $g(x) := \max(\lambda_{\min}(\alpha_0), -x)$ is non-increasing in $x$, and $F$ is counter-monotone with $\tr \Sigma$, it immediately follows that $\lambda(s)=g(F(s))$ is co-monotone with $\tr \Sigma$.
    This completes the proof.
\end{proof}

We are now ready to prove \Cref{prop:rmsnorm_multiple_monotonic}:
\begin{proof}[Proof of \Cref{prop:rmsnorm_multiple_monotonic}]
    Let $U \in \R^{d \times k}$ be a basis for $\U(w)$, let $H_U(w) = U^\top H(w) U \in \sym(\R^k)$, and let $\nabla H_U(w) \in \sym(\R^k) \otimes \R^d$ be the tensor defined as $\nabla H_U(w)_{ij,p}=\frac{\partial [u_i^\top H(w) u_j]}{\partial w_p}$.

    To simplify notation, define $\gamma := \frac{1-\beta_2}{\beta_2}$, and note that this is decreasing in $\beta_2$.

    The time derivative of the sharpness at $t=0$ is given by:
    \begin{align*}
        \frac{dS(w)}{dt}\eval_{t=0} = \lambda_{\max} \qty(\frac{d H_U(w)}{dt}  \eval_{t=0}).
    \end{align*}
    Since $\nu = \frac{\eta^2 S(w)^2}{4}$,  we have that $\tfrac{dH_U}{dt}$ is given by:
    \begin{align*}
        \frac{d H_U(w)}{dt}  \eval_{t=0} &= -(\alpha_H + \beta_H[X]),
    \end{align*}
    where:
    \begin{align*}
        \alpha_H := \frac{2}{S(w)} \nabla H_U(w) \qty[\nabla L(w)] \qand \beta_H := \frac{1}{S(w)}  \nabla H_U(w) \nabla H_U(w)^\top,
    \end{align*}
    and $X$ is the covariance of the oscillations within the basis $U$, i.e. $\Sigma(0) = U X U^\top$.
    This is, in turn, defined as the solution to an SDCP:
    \begin{align*}
        X = \sdcp(\alpha, \beta),
    \end{align*}
    where:
    \begin{align*}
        \alpha = \alpha_H + \gamma \qty(\frac{1}{\eta^2})  \qty(\frac{2}{S(w)}) \|\nabla L(w)\|^2 \, I - \gamma \frac{S(w)}{2} I,
    \end{align*}
    and:
    \begin{align*}
        \beta = \beta_H + \gamma \qty(\frac{1}{\eta^2}) \qty(2 S(w)) \, I \otimes I.
    \end{align*}

    For the $\eta$ result, observe that $\alpha, \beta$ can be written in the form $\alpha = \alpha_0 + \qty(\tfrac{1}{\eta^2}) c_\alpha I$ and $\beta = \beta_0 + \qty(\tfrac{1}{\eta^2}) c_\beta I \otimes I$, where:
    \begin{align*}
        \alpha_0 = \alpha_H - \gamma \tfrac{S(w)}{2} I \qc c_\alpha = \gamma \qty(\frac{2}{S(w)}) \|\nabla L(w)\|^2 \qc \beta_0 = \beta_H \qc c_\beta = \gamma (2 S(w)).
    \end{align*}
    Further, note that $c_\alpha > 0$.
    Therefore, \Cref{prop:heavy-lifting} implies that $\tr X$ is monotonically non-increasing in $\tfrac{1}{\eta^2}$.  Since $\tr \Sigma(0) = \tr X$, and $\frac{1}{\eta^2}$ is decreasing in $\eta$, this implies that $\tr \Sigma(0)$ is monotonically non-decreasing in $\eta$, as claimed.
    \Cref{prop:heavy-lifting} further implies that $\lambda_{\min}(\alpha_0 + \beta_0[X]) = \lambda_{\min}(\alpha_H + \beta_H[X] - \gamma \tfrac{S(w)}{2} I)$ is co-monotone with $\tr X$, which means that it must be monotonically non-increasing in $\tfrac{1}{\eta^2}$, or monotonically non-decreasing in $\eta$.
    This implies that $\lambda_{\min}(\alpha_H + \beta_H[X])$ must also be monotonically non-decreasing in $\eta$, which implies that the negative $\lambda_{\max}(- (\alpha_H + \beta_H[X])) = \tfrac{dS(w)}{dt} \eval_{t=0}$ is monotonically non-increasing in $\eta$, as claimed.
    
    For the $\beta_2$ result, observe that $\alpha, \beta$ can be written in the form $\alpha = \alpha_0 + \gamma c_\alpha I$ and $\beta = \beta_0 + \gamma c_\beta I \otimes I$, where:
    \begin{align*}
        \alpha_0 = \alpha_H \qc c_\alpha = \qty(\frac{1}{\eta^2}) \qty(\frac{2}{S(w)}) \|\nabla L(w)\|^2 -  \frac{S(w)}{2} \qc \beta_0 = \beta_H \qc c_\beta = \qty(\frac{1}{\eta^2}) (2 S(w)).
    \end{align*}
    Therefore, \Cref{prop:heavy-lifting} implies that $\tr X$ is monotone in $\gamma$.
    Since $\tr \Sigma(0) = \tr X$, and $\gamma$ is decreasing in $\beta_2$, this implies that $\tr \Sigma(0)$ is also monotone in $\beta_2$, as was claimed.  \Cref{prop:heavy-lifting} further implies that $\lambda_{\min}(\alpha_0 + \beta_0[X]) = \lambda_{\min}(\alpha_H + \beta_H[X])$ is co-monotone with $\tr X = \tr \Sigma(0)$.
    Thus its negative $\lambda_{\max}(- (\alpha_H + \beta_H[X])) = \tfrac{dS(w)}{dt} \eval_{t=0}$ must be counter-monotone with $\tr \Sigma(0)$, as was claimed.

\end{proof}

\paragraph{Predicting time-averages}
As with \gd, the \rmsnorm central flow can predict the time-average of various quantities, such as the loss or squared gradient norm, along the \rmsnorm trajectory.
Recall that for a function $f(w)$, we use $\bar f(t)$ to denote the central flow's prediction for the time-average $\E[f(w_t)]$ at step $t$.

The prediction $\bar L(t)$ for the time-averaged loss at step $t$ is:
\begin{align}
    \E[ L(w_t) ] &\approx L(w(t)) + \tfrac{\sqrt{\nu}(t)}{\eta} \tr \Sigma(t) =: \bar L(t). \label{eq:rmsnorm-predict-loss}
\end{align}
The prediction for the time-averaged squared gradient norm at step $t$ is:
\begin{align}
    \E[\|\nabla L(w_t)\|^2] &\approx \|\nabla L(w(t))\|^2 + \tfrac{4 \nu(t)}{\eta^2} \tr \Sigma(t) =: \overline{\| \nabla L \|^2}(t). \label{eq:rmsnorm-predict-gradient-norm-sq}
\end{align}
These can be derived along similar lines as \cref{eq:gd-predict-loss,eq:gd-predict-gradient-norm-sq} for \gd, but using the \rmsnorm complementarity condition $H(w) \, \Sigma(t) = \frac{2 \sqrt{\nu}(t)}{\eta} \Sigma(t)$.

As for predicting the covariance of the oscillations, let $\Sigma(t) = V(t) \Lambda(t) V(t)^\top$ be the (reduced) eigenvalue decomposition of the rank-$k$ matrix $\Sigma(t)$, where $V(t) \in \R^{d \times k}$ and $\Lambda(t) \in \diag(\R^k)$.
Then the central flow predicts that the variance of oscillations along the $i$-th eigenvector of $\Sigma(t)$ should be equal to the $i$-th eigenvalue of $\Sigma(t)$:
\begin{align}
    \E \qty[\qty(v_i(t)^\top(w_t - w(t)))^2] = \lambda_i(t). \label{eq:rmsnorm-predict-oscillation-variance}
\end{align}

\paragraph{Practical implementation}
When implementing the \rmsnorm central flow, we treat \rmsnorm as an instance of a more general class of adaptive preconditioned methods that is described in \Cref{sec:arbitrary-preconditioned}.
Please refer to that section for details on how we discretize the central flow in practice.

\subsection{RMSProp}\label{appendix:derivations:rmsprop}
We will begin by describing the stability condition for preconditioned gradient descent on a quadratic. Consider optimizing the quadratic $L(w) = \tfrac{1}{2} w^\top H w$ using preconditioned \gd with preconditioner $P \succ 0$:
\begin{align}
    w \leftarrow w - P^{-1} \nabla L(w) =  w - P^{-1} H w = (I - P^{-1} H) w.
\end{align}
The matrix $P^{-1} H$ is non-symmetric, but is similar to the symmetric matrix $P^{-1/2} H P^{-1/2}$, and thus has the same eigenvalues, which are necessarily real.  If all eigenvalues of $P^{-1} H$ are contained in $(0,2)$ then all eigenvalues of $(I - P^{-1} H)$ are contained in $(-1, 1)$, and hence  $w \to 0$ exponentially fast. Otherwise, the dynamics will diverge along the \emph{right eigenvectors} of $P^{-1} H$ with eigenvalues outside this range.\footnote{Each right eigenvector $u$ of $P^{-1} H$ corresponds to an eigenvector $v$ of $P^{-1/2} H P^{-1/2}$ via the relation $v = P^{1/2} u$.}

We now derive the \rmsprop central flow.
We model the \rmsprop iterates $\{w_t\}$ as oscillating around a central flow $w(t)$ with mean zero and covariance $\Sigma(t)$. That is, if $\delta_t := w_t - w(t) $ denotes the displacement between \rmsprop and the central flow (``the oscillation''), then $\E[\delta_t] = 0$ and $\E[\delta_t \delta_t^\top] = \Sigma(t)$.
Let $\nu(t) := \E[v_t]$ model the time-averaged $\nu_t$, and we will frequently neglect the distinction between the two, implicitly assuming that $\nu_t$ concentrates tightly around $\nu(t)$.  \jnote{what step exactly is that?}

A similar argument as for \gd implies that the time-averaged gradient is approximately:
\begin{align}
    \E[\nabla L(w_t)] \approx \nabla L(w(t)) + \tfrac{1}{2} \nabla H(w(t))^\top [\Sigma(t)]. \label{eq:rmsprop:time-averaged-gradient}
\end{align}
Meanwhile, we approximate the time-average of the elementwise squared gradient as:
\begin{align}
    \E\qty[\nabla L(w_t)^{\odot 2}]
    &= \E\qty[\nabla L(w(t) + \delta_t)^{\odot 2}] \tag{$w_t = w(t) + \delta_t$} \\
    &\approx \E\qty[\qty(\nabla L(w(t)) + H(w(t))\delta_t)^{\odot 2}] \tag{Taylor expansion} \\
    &= \E\qty[\nabla L(w(t))^{\odot 2} + 2 \nabla L(w(t)) \odot H(w(t))\delta_t + (H(w(t))\delta_t)^{\odot 2}] \hspace{5em}  \tag{expand the square} \\
    &= \nabla L(w(t))^{\odot 2} + \E\qty[(H(w(t))\delta_t)^{\odot 2}] \tag{$\E[\delta_t] = 0$} \\
    \implies \E[\nabla L(w_t)^{\odot 2}] &\approx \nabla L(w(t))^{\odot 2} + \E\qty[(H(w(t))\delta_t)^{\odot 2}].
     \label{eq:rmsprop:time-averaged-sq-gradient}
\end{align}

Recall that \rmsprop can be viewed as preconditioned gradient descent with the dynamic preconditioner:
\begin{align}
     P(\nu) := \diag(\sqrt{\nu}/\eta).
\end{align}
We therefore define stability for \rmsprop by the condition $\lambda_{\max} (P(\nu)^{-1} H(w))\le 2$, and we define the critical subspace $\U(w,\nu)$ as the eigenspace of the effective Hessian $P(\nu)^{-1} H(w)$ corresponding to the eigenvalue 2:
\begin{align*}
    \U(w,\nu) := \ker\qty[P(\nu)^{-1} H(w) - 2 I].
\end{align*}
We model \rmsprop as oscillating within the critical subspace, i.e. we assume that $\delta_t \in \U(w(t), \nu(t))$.
This allows us to simplify \cref{eq:rmsprop:time-averaged-sq-gradient} as:
\begin{align*}
    \E\qty[\nabla L(w_t)^{\odot 2}] &\approx \nabla L(w(t))^{\odot 2} + \E\qty[(H(w(t))\delta_t)^{\odot 2}] \\
    &= \nabla L(w(t))^{\odot 2} + 4\E\qty[(P(\nu(t))\delta_t)^{\odot 2}] \tag{$ H(w) \delta_t = 2 P(\nu) \delta_t$} \\
    &= \nabla L(w(t))^{\odot 2} + \frac{4}{\eta^2} \nu(t) \odot \E\qty[\delta_t^{\odot 2}] \tag{$P(\nu) = \diag[\nu^{1/2}/\eta]$} \\
    &= \nabla L(w(t))^{\odot 2} + \frac{4}{\eta^2} \nu(t) \odot \E\qty[\diag[\delta_t \delta_t^\top]] \tag{$v^{\odot 2} = \diag[vv^\top]$} \\
    &= \nabla L(w(t))^{\odot 2} + \frac{4}{\eta^2} \nu(t) \odot \diag[\Sigma(t)]. \tag{$\E[\delta_t\delta_t^\top] = \Sigma(t)$}
\end{align*}
Based on these time averages, we make the central flow ansatz:
\begin{align}
    \begin{split}
    \frac{dw}{dt} &= -\frac{\eta}{\sqrt{\nu}} \odot \qty[\nabla L(w) + \tfrac{1}{2} \nabla H(w)^\top[\Sigma(t)]] \\
    \frac{d\nu}{dt} &= \frac{1-\beta_2}{\beta_2} \qty[\nabla L(w)^{\odot 2} + \frac{4\nu}{\eta^2} \odot \diag[\Sigma(t)] - \nu].
    \end{split}
    \label{eq:appendix:rmsprop:ansatz}
\end{align}

As with the previous optimizers, to determine $\Sigma(t)$ we impose three conditions for all times $t$:
\begin{enumerate}
    \item \textbf{PSD:} As a covariance matrix, $\Sigma(t)$ is positive semidefinite (PSD), i.e. $\Sigma(t) \succeq 0$.
    \item \textbf{Stability:} The effective sharpness remains bounded by $2$, i.e. $\lambda_{\max}(P(\nu(t))^{-1} H(w(t)) ) \le 2$.
    \item \textbf{Complementarity:} The oscillations are contained within the critical subspace, i.e. $\spn \Sigma(t) \subseteq \U(w(t), \nu(t))$.
\end{enumerate}
To express these conditions more concisely, we define the matrix-valued function $A(w, \nu)$ as:
\begin{align*}
    A(w,\nu) := 2 P(\nu) - H(w),
\end{align*}
so that stability is $A(w,\nu) \succeq 0$, and the critical subspace is precisely $\U(w, \nu) = \ker A(w, \nu)$.
Using this notation, the above conditions can be summarized more compactly as the semidefinite complementarity relation:
\begin{align*}
    0 \preceq \Sigma(t) \perp A(w(t),\nu(t)) \succeq 0.
\end{align*}

\begin{definition}[\rmsprop Central Flow, Differential Complementarity Problem]\label{def:rmsprop:flow:dcp}
    We say that $\{w(t),\nu(t)\}_{t \ge 0}$ follow the \rmsprop central flow if for almost all $t \ge 0$, they satisfy \cref{eq:appendix:rmsprop:ansatz} along with the complementarity relation: $0 \preceq \Sigma(t) \perp A(w(t),\nu(t)) \succeq 0$.
\end{definition}

Since this DCP can be expressed in the general form described in \Cref{sec:derivations:dcp}, we can use \Cref{lem:dcp_to_sdcp} to convert it into an equivalent ODE.
In particular, if $\widetilde w := [w,\nu]^\top$ denotes the augmented state, we can write \cref{eq:appendix:rmsprop:ansatz} as:
\begin{align*}
    \frac{d\widetilde{w}}{dt} = f\qty(\widetilde{w}) + B\qty(\widetilde{w})[\Sigma] \qq{where} 
    f\qty(\widetilde{w}) := \begin{bmatrix}
        -P(\nu)^{-1} \nabla L(w) \\ \frac{1-\beta_2}{\beta_2} [\nabla L(w)^{\odot 2} - \nu]
    \end{bmatrix} \qand
    B\qty(\widetilde{w})[\Sigma] := \begin{bmatrix}
        -\tfrac{1}{2} P(\nu)^{-1} \nabla H(w)^\top[\Sigma] \\ \frac{1-\beta_2}{\beta_2} \cdot 4 P(\nu)^2 \diag \Sigma
    \end{bmatrix},
\end{align*}
and we can write the complementarity relation as $0 \preceq \Sigma(t) \perp A(\widetilde{w}(t)) \succeq 0$, where:
\begin{align*}
    \nabla A(\widetilde w) = [-\nabla H(w), 2\nabla_\nu P(\nu)] \quad \text{where} \quad \nabla_\nu P(\nu)[z] = \diag\qty[\frac{1}{2 \eta \sqrt{\nu}}\odot z] = \frac{1}{2\eta^2} \diag[P(\nu)^{-1} z] \quad\forall z.
\end{align*}
Therefore, \Cref{lem:dcp_to_sdcp} implies that 
$$\Sigma(t) \in \sdcp_{\U(w, \nu)}\qty(\alpha(w, \nu),\beta(w, \nu))$$
where the matrix $\alpha(w, \nu) \in \sym(\R^d)$ and tensor $\beta(w, \nu) \in \sym(\R^d)^{\otimes 2}$ are defined by:
\begin{align*}
    \alpha\qty(w, \nu)
    &:= \nabla A\qty(\widetilde{w})[f(\widetilde{w})] \\
    &=
    \begin{bmatrix}
        -\nabla H(w) & 2\nabla_\nu P(\nu)
    \end{bmatrix}\begin{bmatrix}
        -P(\nu)^{-1} \nabla L(w) \\ \frac{1-\beta_2}{\beta_2} [\nabla L(w)^{\odot 2} - \nu]
    \end{bmatrix} \\
    &= \nabla H(w)[P(\nu)^{-1} \nabla L(w)] + \frac{1-\beta_2}{\beta_2} \cdot \frac{1}{\eta^2} \cdot \diag\qty[P(\nu)^{-1} \qty(\nabla L(w)^{\odot 2} - \nu)]
\end{align*}
and
\begin{align*}
    \beta(w, \nu)[\Sigma]
    &:= \nabla A\qty(\widetilde{w})B(\widetilde{w})[\Sigma] \\
    &=
    \begin{bmatrix}
        -\nabla H(w) & 2\nabla_\nu P(\nu)
    \end{bmatrix}
    \begin{bmatrix}
        -\tfrac{1}{2} P(\nu)^{-1} \nabla H(w)^\top[\Sigma] \\ \frac{1-\beta_2}{\beta_2} \cdot 4 P(\nu)^2 \diag \Sigma
    \end{bmatrix} \\
    &= \tfrac{1}{2} \nabla H(w) P(\nu)^{-1} \nabla H(w)^\top[\Sigma] + \frac{1-\beta_2}{\beta_2} \cdot \frac{4}{\eta^2} \cdot \diag[P(\nu)\diag[\Sigma]].
\end{align*}

Note that $\beta(\widetilde{w})$ is indeed symmetric PSD, as for \gd and \rmsnorm:
\begin{align*}
    &\beta(\tilde w)[\Sigma,\Sigma'] = \tfrac{1}{2} \ev{\nabla H(w)^\top[\Sigma], \nabla H(w)^\top[\Sigma']}_{P(\nu)^{-1}} + \frac{1-\beta_2}{\beta_2} \cdot \frac{4}{\eta^2} \ev{\diag[\Sigma], \diag[\Sigma']}_{P(\nu)} \\
    &\implies \beta(\tilde w)[\Sigma,\Sigma] = \tfrac{1}{2} \norm{\nabla H(w)^\top[\Sigma]}^2_{P(\nu)^{-1}} + \frac{1-\beta_2}{\beta_2} \cdot \frac{4}{\eta^2} \norm{\diag[\Sigma]}_{P(\nu)}^2 \ge 0.
\end{align*}
This gives us the ODE formulation of the \rmsprop central flow:
\begin{definition}[\rmsprop Central Flow, ODE Formulation]\label{def:rmsprop:flow:ode}
    We say that $\{w(t),\nu(t)\}_{t \ge 0}$ follow the \rmsprop central flow if, for almost all $t$, they satisfy \cref{eq:appendix:rmsprop:ansatz} with $$\Sigma(t) \in \sdcp_{\U(w(t), \nu(t))}\qty(\alpha(w(t), \nu(t)),\beta(w(t), \nu(t))).$$
\end{definition}

\paragraph{Predicting time-averages}
As with the previous optimizers, the \rmsprop central flow can predict the time-average of various quantities, such as the loss or squared gradient norm, along the \rmsprop trajectory.
Recall that for a function $f(w)$, we use $\bar f(t)$ to denote the central flow's prediction for the time-average $\E[f(w_t)]$ at step $t$.
See \Cref{sec:arbitrary-preconditioned} for the derivation of the following statements.  We write $P(t) := P(\nu(t))$ for brevity.

The prediction $\bar L(t)$ for the time-averaged loss at step $t$ is:
\begin{align}
    \E[ L(w_t) ] &\approx L(w(t)) + \tr \qty[P(t) \, \Sigma(t)] =: \bar L(t). \label{eq:rmsprop-predict-loss}
\end{align}
The prediction for the time-averaged squared gradient norm at step $t$ is:
\begin{align}
    \E[\|\nabla L(w_t)\|^2] &\approx \| \nabla L(w(t)) \|^2 + 4 \tr \qty[ P(t) \, \Sigma(t) \, P(t)^\top  ] =: \overline{\| \nabla L \|^2}(t). \label{eq:rmsprop-predict-gradient-norm-sq}
\end{align}
As for predicting the covariance of the oscillations, let $\Sigma(t) = V(t) \, \Lambda(t) \, V(t)^\top$ be the (reduced) eigenvalue decomposition of the rank-$k$ matrix $P^{1/2}(t) \, \Sigma(t) \, P^{1/2}(t)$, where $V(t) \in \R^{d \times k}$ and $\Lambda(t) \in \diag(\R^k)$.
Then the central flow predicts that the $P$-whitened variance of oscillations along the $i$-th eigenvector of $P(t)^{1/2} \, \Sigma(t) \, P(t)^{1/2}$ should be equal to the $i$-th eigenvalue of that matrix:
\begin{align}
    \E \qty[\qty(v_i(t)^\top P(t)^{1/2}(w_t - w(t)))^2] = \lambda_i(t). \label{eq:rmsprop-predict-oscillation-variance}
\end{align}
See \Cref{sec:arbitrary-preconditioned} for a discussion of why we predict the $P$-whitened covariance of oscillations.

\paragraph{Practical implementation}
When implementing the \rmsprop central flow, we treat \rmsprop as an instance of a more general class of adaptive preconditioned methods that is described in \Cref{sec:arbitrary-preconditioned}.
Please refer to that section for details on how we discretize the central flow in practice.

\subsubsection{Stationarity analysis}
\label{appendix:derivations:rmsprop:stationary}

In this appendix, we provide supporting derivations for our analysis of \rmsprop's stationary preconditioner.

\paragraph{Stationary preconditioner}
The \rmsprop central flow \cref{eq:appendix:rmsprop:ansatz} is a joint flow over $(w, \nu)$.
However, suppose that the $\nu$ dynamics occur ``fast'' relative to the $w$ dynamics, so that $\nu$ is always ``fully caught up'' to the current $w$. 
For any fixed $w$, solving \cref{eq:appendix:rmsprop:ansatz} for the stationarity condition $\frac{d\nu}{dt} = 0$ gives the condition:
\begin{align}
    \nu = \nabla L(w)^{\odot 2} + \frac{4}{\eta^2} \, \nu \odot \diag[\Sigma].\label{eq:appendix:rmsprop:stationary:dnu_dt}
\end{align}
In addition, by the PSD, stability, and complementarity conditions, we have:
\begin{align}
    0 \preceq \Sigma \perp 2 P(\nu) - H(w) \succeq 0 \qq{where} P(\nu) := \diag \qty(\tfrac{\sqrt{\nu}}{\eta}). \label{eq:appendix:rmsprop:stationary:complementarity}
\end{align}

We will now show that for any $w$, there is a \emph{unique} pair $\nu, \Sigma$ that satisfies \cref{eq:appendix:rmsprop:stationary:dnu_dt,eq:appendix:rmsprop:stationary:complementarity}.
We will denote this pair as $\overline{\nu}(w)$ and $\overline{\Sigma}(w)$.
We will further show that the corresponding preconditioner $\diag(\sqrt{\overline{\nu}(w)}/\eta)$ is the unique optimum to the following convex optimization problem over diagonal preconditioners:
\begin{align}
    \argmin_{\mathclap{\text{$P$ diagonal, $P{\succeq}0$}}} \qquad \tr(P) + \tfrac{1}{\eta^2} \|\nabla L(w)\|^2_{P^{-1}} \qq{such that} H(w) \preceq 2P,  \label{eq:rmsprop_nu_convex_program2}
\end{align}
where $\|v\|^2_{P^{-1}} := v^\top P^{-1} v$.
We will denote this preconditioner as $\overline{P}(w)$.

We show this in two parts.
First, in \Cref{prop:stationary_nu} we prove that if $\nu, \Sigma$ satisfy \cref{eq:appendix:rmsprop:stationary:dnu_dt,eq:appendix:rmsprop:stationary:complementarity}, then $P(\nu)$ is an optimum for \cref{eq:rmsprop_nu_convex_program2}. Then, in \Cref{lem:rmsprop_nu_convex_program_unique} we show that the solution to \cref{eq:rmsprop_nu_convex_program2} is unique, implying that this must be the \emph{unique} optimum. At the end of this section, we describe how to numerically compute $\overline{P}(w)$ and $\bar{\nu}(w)$ when $H(w)$ is so large that it can only be feasibly accessed via matrix-vector products.

\begin{proposition}\label{prop:stationary_nu} Define $P(\nu) := \diag(\sqrt{\nu}/\eta)$. For any $g \in \R^d, H \in \sym(\R^d)$, if $\nu \in \R^d, \Sigma \in \sym(\R^d)$ satisfy:
\begin{align}
    \nu = g^{\odot 2} + \tfrac{4}{\eta^2} \nu \odot \diag(\Sigma) \label{eq:rmsprop-stationary-1}\\
    0 \preceq \Sigma \perp 2P(\nu) - H \succeq 0 \label{eq:rmsprop-stationary-2},
\end{align}
then $P(\nu)$ is an optimum for the convex program:
\begin{align}
        \argmin_{\mathclap{\text{$P$ diagonal, $P{\succeq}0$}}} \qquad \tr(P) + \tfrac{1}{\eta^2} g^\top P^{-1} g \qq{such that} H \preceq 2P. \label{eq:rmsprop_nu_convex_program3}
\end{align}
\end{proposition}

\begin{proof}
    Parameterizing $P = \diag(p)$ for a vector $p \in \R^d$, the convex program \cref{eq:rmsprop_nu_convex_program3} can be written as:
\begin{align}
    \min_{p \in \R^d, \; p \ge 0} \sum_{i=1}^d p_i + \frac{1}{\eta^2} \frac{ g_i^2}{p_i} \qq{such that} H \preceq 2 \diag(p). 
\end{align}
Introducing a dual variable $Z \succeq 0$ for the semidefinite constraint, the Lagrangian is:
\begin{align}
    L(p, Z) &= \sum_{i=1}^d \qty[p_i + \frac{1}{\eta^2} \frac{ g_i^2}{p_i}] - \ev{Z, 2 \diag(p) - H}.
\end{align}
Therefore, the KKT conditions are:
\begin{align}
    \underbrace{1 - \frac{1}{\eta^2} \frac{g_i^2}{p_i^2} - 2 Z_{ii} = 0 \quad\forall i}_{\text{stationarity}} \qc
    \underbrace{H \preceq 2\diag(p)}
    _{\text{primal feasibility}} \qc
    \underbrace{Z \succeq 0}_{\text{dual feasibility}} \qc
    \underbrace{2 \diag(p) - H \perp Z \label{eq:fixed_point_kkt_point_relations}}_{\text{complementary slackness}},
\end{align}
as well as $p \ge 0$.
We claim that if $(\nu, \Sigma)$ satisfy \cref{eq:rmsprop-stationary-1,eq:rmsprop-stationary-2}, then $(p, Z) = (\frac{\sqrt{\nu}}{\eta}, \frac{2}{\eta^2} \Sigma)$ solve these KKT conditions. 
First, elementwise dividing both sides of \cref{eq:rmsprop-stationary-1} by $\nu$ gives:
\begin{align}
    1 - \frac{g_i^{2}}{\nu_i} - \qty(\frac{4}{\eta^2}) \Sigma_{ii} = 0 \quad\forall i
\end{align}
which is equivalent to the stationary condition in \cref{eq:fixed_point_kkt_point_relations} after substituting $p = \sqrt{\nu}/\eta$ and $Z = \frac{2}{\eta^2} \Sigma$.
Next, the remaining parts of \cref{eq:fixed_point_kkt_point_relations} are implied by \cref{eq:rmsprop-stationary-2}.
Finally, we must have $\sqrt{\nu} \ge 0$ or $P(\nu)$ would be imaginary.
\end{proof}

We now prove that the solution to the optimization problem \cref{eq:rmsprop_nu_convex_program3} is unique.
A custom proof is needed because, while both the objective and constraints of \cref{eq:rmsprop_nu_convex_program3} are convex, the objective is not \emph{strictly} convex.  

\begin{proposition}\label{lem:rmsprop_nu_convex_program_unique}
For any  $g \in \R^d, H \in \sym(\R^d)$, the solution to \cref{eq:rmsprop_nu_convex_program3} is unique.
\end{proposition}
\begin{proof}
    Assume there are two minimizers $P,P'$ and let $p := \diag(P), \delta := \diag(P'-P)$. Then by convexity, $\diag[p + \epsilon \delta]$ also minimizes \cref{eq:rmsprop_nu_convex_program2} for any $\epsilon \le 1$. Therefore, differentiating the objective function in this direction gives:
    \begin{align}
        \sum_i \delta_i\qty[1 - \frac{1}{\eta^2} \frac{g_i^2}{p_i^2}] = 0.
    \end{align}
    Taking another derivative implies that:
    \begin{align}
        \sum_i \frac{g_i^2}{p_i^3} \delta_i^2 = 0.
    \end{align}
    This implies that $\delta_i = 0$ in any direction where $g_i \ne 0$. Let \(I := \{ i : g_i = 0 \}\). From the above, \(\delta_i=0\) for \(i\in I^c\), so
\(p_{I^c} = p'_{I^c}\). It remains to show equality on \(I\).
Define \(G\) by
\[
  G[v_{I}]_i :=
    \begin{cases}
      p_i & i \in I^c \\
      v_i & i \in I
    \end{cases}.
\]
Define \(\mathcal{A}^\top[v_{I}] := \diag\!\big(G[v_{I}]\big) \oplus \diag[v_{I}]\).
Then both \(p_{I}, {p'}_{I}\) minimize the reduced SDP
\[
  \min_{p_{I}} \sum_{i\in I} p_i
  \quad \text{s.t.}\quad \tfrac{1}{2}H \oplus 0_{|I|\times |I|} \preceq \mathcal{A}^\top(p_{I}).
\]
Now we apply \citet[Proposition 1]{silva2018strictcomplementaritymaxcutsdp} with $(\mathcal{A},\mathbf{1}_{|I|})$. First, note that $\mathcal{A}[I_{d + |I|}] = 2 \mathbf{1}_{|I|}$ which satisfies the first condition. Next, for any $y \ne 0$, we can take $z = |y|$ to satisfy the second condition, as in the proof of \citep[Corollary 2]{silva2018strictcomplementaritymaxcutsdp}. Therefore $p_I = p'_I$, and as we have already shown equality on $I^c$, we must have $p = p'$.

\end{proof}

\paragraph{Stationary flow} Suppose that the $\nu$ dynamics (preconditioner adaptation) happen \emph{infinitely fast} relative to the $w$ dynamics (optimization), so that we can treat $\nu$ as always being fixed at its current stationary value $\overline{\nu}(w)$.
This motivates the stationary flow:

\begin{definition}[\rmsprop Stationary Flow]\label{def:rmsprop:flow:stationary}
    We say that $\{w(t)\}_{t \ge 0}$ follow the \rmsprop stationary flow if, for almost all $t$, they satisfy
    \begin{align}
    \begin{split}
    \frac{dw}{dt} &= -\frac{\eta}{\sqrt{\overline{\nu}(w)}} \odot \qty[\nabla L(w) + \tfrac{1}{2} \nabla H(w)^\top[\Sigma(t)]]
    \end{split}
    \label{eq:appendix:rmsprop:stationary}
    \end{align}
    with $$\Sigma(t) \in \sdcp_{\U(w(t), \overline{\nu}(w(t)))}\qty(\alpha(w(t), \overline{\nu}(w(t))),\beta(w(t), \overline{\nu}(w(t)))).$$
\end{definition}

Note that $\Sigma(t)$ is defined as the solution to an SDCP, and is \emph{not}, in general, equal to $\bar{\Sigma}(w(t))$. However, we often expect these to be close, especially for small $\beta_2$, and we can consider the approximation where $\Sigma(t)$ is replaced with $\bar{\Sigma}(w(t))$\footnote{Note that $\Sigma(t)$ would be exactly equal to $\bar \Sigma(w(t))$ if $\alpha(w, \nu)$ and $\beta(w, \nu)$ had only the second terms (i.e. the ones that scale with $\tfrac{1-\beta_2}{\beta_2}$) and not the first terms.}:
\begin{align}
    \frac{dw}{dt} = -\frac{\eta}{\sqrt{\bar{\nu}(w)}} \odot \qty[ \nabla L(w) + \tfrac{1}{2} \nabla H(w)^\top[\bar{\Sigma}(w)] ] .\label{eq:appendix:rmsprop:stationary-beta2-0}
\end{align}

\paragraph{Limit of large $\eta$}
In the limit of large $\eta$, the second term in the objective \cref{eq:rmsprop_nu_convex_program2} vanishes, and the stationary preconditioner $\overline{P}(w)$ tends towards the \emph{minimum-trace, diagonal stable preconditioner}, which we denote $\hat{P}(w)$:
\begin{align}
    \hat{P}(w) \; := \; \argmin_{\mathclap{\text{$P$ diagonal, $P{\succeq}0$}}} \qquad \tr(P) \qq{such that} H(w) \preceq 2P.  \label{eq:appendix:rmsprop:min-trace-preconditioner}
\end{align}
Note that the dual to this semidefinite program is the following semidefinite program:
\begin{align}
    \hat{Z}(w) := \argmax_{Z \succeq 0} \ev{Z,  H(w)} \quad \text{subject to} \quad Z_{ii} \le \tfrac{1}{2}
    \label{eq:appendix:rmsprop:min-trace-preconditioner-dual}
\end{align}

and a primal/dual optimal pair $\hat{P}(w), \hat{Z}(w)$ must satisfy the KKT conditions:
\begin{align}
    0 \preceq \hat{Z}(w) \perp 2 \hat{P}(w) - H(w) \succeq 0 \qc \hat{Z}_{ii}(w) = \tfrac{1}{2} \; \text{whenever} \; \hat P_{ii}(w)  > 0.
\end{align}

In the limit of large $\eta$, the stationary EMA $\bar{\nu}(w)$ and the stationary oscillation covariance $\bar{\Sigma}(w)$ become:
\begin{align}
    \bar{\nu}(w) \to \eta^2 \diag[\hat{P}(w)]^{\circ 2} \qand \bar{\Sigma}(w) \to \frac{\eta^2}{2} \hat{Z}(w).
\end{align}
As a result, the approximation \cref{eq:appendix:rmsprop:stationary-beta2-0} can be shown to be equivalent to:
\begin{align}
    \frac{dw}{dt} = - \hat{P}(w)^{-1} \qty[ \nabla L(w) + \tfrac{\eta^2}{4} \nabla \tr \hat{P}(w) ].
    \label{eq:appendix:rmsprop:stationary-beta2-0-large-lr}
\end{align}

This is because $\tr \hat{P}(w) = \ev{\hat{Z}(w), H(w)}$ by duality, so by Danskin's theorem:
\begin{align*}
    \nabla \tr \hat{P}(w) &= \nabla H(w)^\top [ \hat{Z}(w) ] \\
    &= \tfrac{2}{\eta^2} \nabla H(w)^\top [\bar{\Sigma}(w)].
\end{align*}
Hence, $\tfrac{\eta^2}{4} \nabla \tr \hat{P}(w) = \tfrac{1}{2} \nabla H(w)^\top[\bar{\Sigma}(w)]$, and \cref{eq:appendix:rmsprop:stationary-beta2-0-large-lr} follows.

\paragraph{Connection to MaxCut}
Interestingly, the SDP \cref{eq:appendix:rmsprop:min-trace-preconditioner-dual} is precisely the SDP relaxation of MaxCut \citep{goemans1995improved} where the Laplacian matrix of the graph is given by $\tfrac{1}{2} H(w)$.
Meanwhile, the SDP \cref{eq:appendix:rmsprop:min-trace-preconditioner} that defines $\hat{P}$ is the dual to the MaxCut SDP relaxation.

\paragraph{Numerically solving for the stationary preconditioner}
When the problem dimension $d$ is small, the optimization problem \cref{eq:rmsprop_nu_convex_program2} can be solved exactly using a standard convex solver, e.g. cvxpy.  But when $d$ is large (e.g. the number of weights in a reasonably sized neural network), solving \cref{eq:rmsprop_nu_convex_program2} exactly is not practical, as it is not even practical to materialize the matrix $H \in \R^{d \times d}$.  Therefore, we instead solve \cref{eq:rmsprop_nu_convex_program2} using a fixed point iteration which only requires access to $H$ using matrix-vector products.

We parameterize $\Sigma$ in the factorized form $\Sigma = DD^\top$ where $D \in \R^{d \times r}$ and $r$ is intended to be at least as large as the rank of $\Sigma$.  This is similar to the Burer-Monteiro factorization \citep{burer2005local}.  We start from a random initial guess for $D$ and then iteratively update $D$ and $\nu$ by:
\begin{align}
    \nu &\leftarrow g^{\odot 2} + (HD)^{\odot 2}\mathbf{1}, \\
    D &\leftarrow \frac{\eta}{2}\diag[\nu^{-1/2}] HD.
\end{align}
where the second update uses the $\nu$ that was just computed in the first update.

If this update scheme reaches a fixed point $(D, \nu)$, then we have:
\begin{align}
    \nu &= g^{\odot 2} + (HD)^{\odot 2}\mathbf{1}, \label{eq:appendix:rmsprop:fixed-point-1} \\
     HD &= \tfrac{2}{\eta}\diag[\nu^{1/2}] D \label{eq:appendix:rmsprop:fixed-point-2}.
\end{align}
If $(D, \nu)$ satisfy these two conditions, as well as the stability condition $H \preceq 2 \diag(\sqrt{\nu}/\eta)$, then it can be shown that $\Sigma = DD^\top$ and $\nu$ satisfy \cref{eq:rmsprop-stationary-1,eq:rmsprop-stationary-2}. Indeed, $\Sigma \succeq 0$ holds by construction, \cref{eq:rmsprop-stationary-1} follows by substituting \cref{eq:appendix:rmsprop:fixed-point-2} into \cref{eq:appendix:rmsprop:fixed-point-1}, and  \cref{eq:appendix:rmsprop:fixed-point-2} implies $[2 \diag(\sqrt{\nu}/\eta) - H]D = 0$ which implies $ 2 \diag(\sqrt{\nu}/\eta) - H \perp \Sigma$ provided that the stability condition holds.

Thus, if the update scheme reaches a fixed point \cref{eq:appendix:rmsprop:fixed-point-1,eq:appendix:rmsprop:fixed-point-2}, and if the stability condition $H \preceq 2 \diag(\sqrt{\nu} / \eta)$ is also satisfied there, then we know that $P = \diag(\sqrt{\nu}/\eta)$ solves the optimization problem \cref{eq:rmsprop_nu_convex_program2}.

Empirically, we observe that this update scheme does reach a fixed point in practice.  We moreover observe that if $r$ is sufficiently large (in particular, if it is as large as the rank of the true $\Sigma$), then the stability condition is satisfied at this fixed point, implying that the corresponding preconditioner indeed solves \cref{eq:rmsprop_nu_convex_program2}.  On the other hand, we observe that if $r$ is too small (less than the rank of the true $\Sigma$), then while the update scheme converges to a fixed point, the stability condition is not satisfied there.

\begin{algorithm}
\SetAlgoLined
\KwIn{Gradient $g \in \mathbb{R}^d$, Hessian-vector oracle $v \mapsto Hv$, learning rate $\eta$, rank parameter $r$, number of steps $\text{nsteps}$, tolerance parameter $\text{tol}_\nu$}
\KwOut{$\nu$, $P$, and $D$}

Initialize $D \in \mathbb{R}^{d \times r}$ with standard normal entries\;
\For{$i = 1$ \KwTo $\text{nsteps}$}{
    \If{i > 1}{
        $\nu_{\text{prev}} \leftarrow \nu$\;
    }
    $\nu \leftarrow g^{\odot 2} + \texttt{sum\_rows}((HD)^{\odot 2})$  \;
    $D \leftarrow \frac{\eta}{2}\text{diag}[\nu^{-1/2}] HD$ \;
}
\If{$\frac{\|\nu - \nu_{\text{prev}}\|_2}{\|\nu\|_2} \geq \text{tol}_\nu$}{
    \Return "Error: more steps needed"\;
}

$p \leftarrow \sqrt{\nu}/\eta$\;

\If{$\lambda_{\max}(\diag(p^{-{1/2}}) H \diag(p^{-1/2})) > 2$}{
    \Return "Error: higher $r$ needed"\;
}

\Return $\nu$, $\diag (p)$, $D$\;
\caption{Solving for the Stationary Preconditioner}
\end{algorithm}

\subsection{General Class of Adaptive Preconditioned Methods}
\label{sec:arbitrary-preconditioned}
In this section, we derive a central flow for a general class of adaptive preconditioned methods that subsumes \gd, \rmsnorm, and \rmsprop as special cases. In these special cases, this flow will reduce to the central flows that we have already derived, and whose accuracy we have verified empirically.  However, we do \emph{not} claim that the central flow derived in this section will be empirically accurate for \emph{any} method within this class.
Rather, we include this section because it allows us to treat all three considered optimizers in a unified manner, and to easily generalize our central flows to minor variants of the same algorithms (e.g. \gd with a learning rate schedule, \rmsprop with bias correction).   Our implementation in code is based on this formulation.

We consider methods which update some ``optimizer state'' $\nu \in \R^{d_\nu}$ based on the current gradient, and then take a preconditioned gradient step using some preconditioner $P(\nu)$ that is derived from this state:
\begin{align}
    \nu_t = \nu_{t-1} + G(\nu_{t-1}, \nabla L(w_t)) \qc w_{t+1} = w_t - P(\nu_t)^{-1} \nabla L(w_t). \label{eq:appendix:derivations:adaptive:update}
\end{align}
Here, $G: \R^{d_\nu} \times \R^d \to \R^{d_\nu}$ determines how the optimizer state $\nu \in \R^{d_\nu}$ is updated based on the gradient, and $P: \R^{d_\nu} \to \sym(\R^d)$ determines how the optimizer state affects the preconditioner.

This formulation is very general and includes a wide variety of optimizers including:
\begin{itemize}
    \item \emph{Vanilla GD:} ignore $\nu$ and set $P = \eta^{-1} I$.
    \item \emph{GD with a learning rate schedule $\eta(t)$:} Set $G(\nu,g) = 1$ so that $\nu(t) = t$, and $P(t) = \eta(t)^{-1} I$.
    \item \emph{Vanilla RMSProp:}\footnote{This formalism doesn't directly handle our small $\beta_2$ correction of $\beta_2 \to \frac{1-\beta_2}{\beta_2}$. We view this correction as ``orthogonal'' in the sense that it comes when deriving a stable flow analogue of the discrete time update $\nu_t = \nu_{t-1} + G(\nu_{t-1},\nabla L(w_t))$ when this takes the form of an EMA.} Set $G(\nu,g) = (1-\beta_2)[g^{\odot 2} - \nu]$ and $P(\nu) = \diag[\sqrt{\nu}/\eta]$
    \item \emph{RMSProp with $\epsilon$, bias correction, and learning rate schedule $\eta(t)$:} Set $\nu = [v,t]$, $G([v,t],g) = [(1-\beta_2)[g^{\odot 2} - v],1]$ and define $$P([v,t]) = \frac{1}{\eta(t)} \diag\qty[\sqrt{\frac{v}{1-\beta_2^t}} + \epsilon].$$
\end{itemize}
Note that this trick of embedding $t$ into the state variable $\nu$ allows us to automatically derive central flows for any smooth hyperparameter schedule (e.g. $\eta(t),\beta_2(t),\epsilon(t)$) as a simple corollary.

\begin{remark}
Not all algorithms of the form \cref{eq:appendix:derivations:adaptive:update} are necessarily sensible optimizers.  Moreover, we will see below that the central flow is only well-defined if $G$ and $P$ satisfy a certain condition (\Cref{remark:adaptive-beta-symmetry}).
Thus, it may make sense for future work to further restrict the formulation \cref{eq:appendix:derivations:adaptive:update}.
\end{remark}

The stability of the algorithm requires $\lambda_{\max} (P(\nu)^{-1} H(w)) \le 2$ or equivalently $H(w) \preceq 2 P(\nu)$. We define
\begin{align}
    A(w,\nu) := 2 P(\nu) - H(w),
\end{align}
so that stability is equivalent to $A(w,\nu) \succeq 0$. We define the critical subspace by $\U(w,\nu) := \ker A(w,\nu)$.

To derive the central flow, we model $w_t = w(t) + \delta_t$ with $\E[\delta_t] = 0$, $\E[\delta_t \delta_t^\top] = \Sigma(t)$, and $\delta_t \in \ker A(w(t),\nu(t))$. 
For ease of notation, we will sometimes use $g(w)$ as a shorthand for the gradient $\nabla L(w)$.

We will first compute the time-average of $G$, i.e. $\E[G(\nu, g(w_t))] = \E[G(\nu, g(w(t) + \delta_t))]$.
A first-order Taylor expansion of $g$ around any point $w$ yields:
\begin{align*}
    g(w+\delta) \approx g(w) + H(w) \, \delta.
\end{align*}

Meanwhile, a second-order Taylor expansion of $G$ in its second argument yields:
\begin{align}
    G(v, g + \Delta g) \approx G(\nu, g) + \nabla_g G(\nu, g)^\top \Delta g + \tfrac{1}{2} \nabla_g^2 G(\nu, g)[\Delta g \, \Delta g^\top],
\end{align}
where $\nabla_g^2 G(\nu, g) \in \R^{d_\nu} \otimes \sym(\R^{d})$ is a tensor that represents the Hessian of each entry of $G$.
Putting these together, with $\Delta g = H(w) \delta$, we have:
\begin{align*}
    G(\nu, g(w+\delta)) \approx G(\nu, g(w)) + \nabla_g G(\nu, g)^\top H(w) \delta +  \tfrac{1}{2} \nabla_g^2 G(\nu, g)[H(w) \, \delta \delta^\top \, H(w)],
\end{align*}

Taking the time-average over $\delta$ with $\E[\delta] = 0$ and $\E[\delta \delta^\top] = \Sigma$ yields:
\begin{align}
    \E[G(v, g(w + \delta)] &\approx G(\nu, g(w)) + \tfrac{1}{2} \nabla^2_g G(\nu, g(w)) \qty[ H(w) \, \Sigma \, H(w) ].
\end{align}
Since $H(w) \Sigma = 2 P(\nu) \Sigma$, we have $H(w) \Sigma H(w) = 4 P(\nu) \Sigma P(\nu)$, and therefore the second term becomes:
\begin{align}
    \E[G(v, g(w + \delta)] &\approx G(\nu, g(w)) + 2 \nabla^2_g G(\nu, g(w)) \qty[ P(\nu) \, \Sigma \, P(\nu) ].
\end{align}

These motivate the central flow ansatz:
\begin{align}
    \begin{split}
    \frac{dw}{dt} &= -P(\nu)^{-1} \qty[\nabla L(w) + \tfrac{1}{2} \nabla H(w)^\top[\Sigma]] \\
    \frac{d\nu}{dt} &= G(\nu,g(w)) + 2 \nabla^2_g G(\nu, g(w))[P(\nu) \, \Sigma \, P(\nu)].
    \end{split}\label{eq:appendix:generalP:ansatz}
\end{align}
We say that $\{w(t),\nu(t),\Sigma(t)\}_{t \ge 0}$ satisfy the DCP formulation of the central flow if for almost all $t \ge 0$ they satisfy \cref{eq:appendix:generalP:ansatz} along with the complementarity relation:
\begin{align*}
    0 \preceq \Sigma(t) \perp A(w(t),\nu(t)) \succeq 0.
\end{align*}

As for \gd, \rmsnorm, and \rmsprop, we can invoke \Cref{lem:dcp_to_sdcp} to obtain an equivalent ODE formulation that makes $\Sigma(t)$ explicit. We begin by writing \cref{eq:appendix:generalP:ansatz} as:
\begin{align*}
    \frac{d\widetilde{w}}{dt} = f\qty(\widetilde{w}) + B\qty(\widetilde{w})[\Sigma] \qq{where} 
    f\qty(\widetilde{w}) := \begin{bmatrix}
        -P(\nu)^{-1} \nabla L(w) \\ G(\nu,g(w))
    \end{bmatrix} \qand
    B\qty(\widetilde{w})[\Sigma] := \begin{bmatrix}
        -\tfrac{1}{2} P(\nu)^{-1} \nabla H(w)^\top[\Sigma] \\ 2 \nabla^2_g G(\nu,g(w)) \, [P(\nu) \, \Sigma \, P(\nu)]
    \end{bmatrix}.
\end{align*}
where the complementarity relation is $0 \preceq \Sigma \perp A(\widetilde{w}) \succeq 0$ with $\nabla A$ given by:
\begin{align*}
    \nabla A(\widetilde w) = [-\nabla H(w), 2\nabla P(\nu)].
\end{align*}
Therefore, \Cref{lem:dcp_to_sdcp} implies that:
\begin{align*}
    \Sigma(t) \in \sdcp_{\U(w, \nu)}\qty(\alpha(w, \nu),\beta(w, \nu)),
\end{align*}
where the matrix $\alpha(w, \nu) \in \sym(\R^d)$ and tensor $\beta(w, \nu) \in \sym(\R^d)^{\otimes 2}$ are defined by:
\begin{align*}
    \alpha\qty(\widetilde{w})
    &:= \nabla A\qty(\widetilde{w})[f(\widetilde{w})] \\
    &=
    \begin{bmatrix}
        -\nabla H(w) & 2\nabla P(\nu)
    \end{bmatrix}\begin{bmatrix}
        -P(\nu)^{-1} \nabla L(w) \\ G(\nu,g(w))
    \end{bmatrix} \\
    &= \nabla H(w)[P(\nu)^{-1} \nabla L(w)] + 2 \nabla P(\nu)[G(\nu,g(w))]
\end{align*}
and
\begin{align*}
    \beta\qty(\widetilde{w})[\Sigma]
    &:= \nabla A\qty(\widetilde{w})B(\widetilde{w})[\Sigma] \\
    &=
    \begin{bmatrix}
        -\nabla H(w) & 2\nabla P(\nu)
    \end{bmatrix}
    \begin{bmatrix}
        -\tfrac{1}{2} P(\nu)^{-1} \nabla H(w)^\top[\Sigma] \\ 2 \nabla^2_g G(\nu,g(w))[P(\nu) \Sigma P(\nu)]
    \end{bmatrix} \\
    &= \tfrac{1}{2} \nabla H(w) P(\nu)^{-1} \nabla H(w)^\top[\Sigma] + 4 \nabla P(\nu) \; \nabla^2_g G(\nu,g(w)) \, [ P(\nu) \, \Sigma \, P(\nu)].
\end{align*}

\begin{remark}
    \label{remark:adaptive-beta-symmetry}
    For \gd, \rmsnorm, and \rmsprop, the $\beta$ we derived was always symmetric PSD. However, at the current level of generality this is not always true:
    \begin{align*}
        \beta(w,\nu)[\Sigma,\Sigma'] = \tfrac{1}{2} \ev{\nabla H(w)^\top[\Sigma], \nabla H(w)^\top[\Sigma']}_{P(\nu)^{-1}} + \ev{\nabla P(\nu)^\top[\Sigma],\nabla^2_g G(\nu,g(w))[P(\nu) \Sigma' P(\nu)]}.
    \end{align*}
    While the first term is symmetric in $\Sigma,\Sigma'$, the second is not necessarily symmetric, so results about existence and uniqueness for solutions to the SDCP may no longer hold.
\end{remark}

\paragraph{Predicting time-averages}
The central flow's predictions for time-averages can be computed as follows. 
Recall that for a function $f(w)$, we use $\bar f(t)$ to denote the central flow's prediction for the time-average $\E[f(w_t)]$ at step $t$.
In what follows, we write $P(t) := P(\nu(t))$ for brevity.

The prediction for the time-averaged loss at time $t$ is:
\begin{align}
     \E[L(w_t)] &\approx \E\qty[L(w(t)) + \nabla L(w(t))^\top\delta_t  + \tfrac{1}{2} \delta_t^\top H(w(t)) \delta_t ] \hspace{10px} \tag{Taylor expansion} \nonumber \\
    &= L(w(t)) + \tfrac{1}{2} \ev{ H(w(t)), \Sigma(t)}  \tag{$\E[\delta_t] = 0$, $\E[\delta_t \delta_t^\top] = \Sigma(t)$} \nonumber \\
    &= L(w(t)) + \tfrac{1}{2} \tr \qty[2 P(t) \Sigma ]    \tag{$H \Sigma = 2 P \Sigma$} \nonumber \\
    &= L(w(t)) + \tr \qty[P(t) \, \Sigma(t)] \nonumber \\
    &=: \bar L(t). \label{eq:abstract-predict-loss}
\end{align}
Similarly, the prediction for the time-averaged squared gradient norm at time $t$ is:
\begin{align}
    \E[\|\nabla L(w_t)\|^2] &\approx \E\qty[ \| \nabla L(w(t)) + H(w(t)) \delta_t \|^2 ] \nonumber \\
    &= \|\nabla L(w(t)) \|^2 + \ev{ H^2(w(t)), \Sigma(t)} \nonumber \\   
    &= \| \nabla L(w(t)) \|^2 + 4 \tr \qty[ P(t) \, \Sigma(t) \, P(t)^\top  ] \nonumber \\
    &=: \overline{\| \nabla L \|^2}(t). \label{eq:abstract-predict-gradient-norm-sq}
\end{align}
The central flow can also predict the covariance of the oscillations.
The natural basis to examine the oscillations is the one in which the dynamics of each coordinate are decoupled under preconditioned gradient descent on the local quadratic Taylor approximation.   Thus, we define the $P$-whitened displacement $\hat{\delta}_t := P(t)^{1/2} \delta_t = P(t)^{1/2} (w_t - w(t))$, and the $P$-whitened covariance matrix of the oscillations $\E[\hat \delta_t \hat \delta_t^\top] = P(t)^{1/2} \, \Sigma(t) \, P(t)^{1/2}$.
Let $V(t) \, \Lambda(t) \, V^\top(t)$ be the eigenvalue decomposition of $P(t)^{1/2} \, \Sigma \, P(t)^{1/2}$, and define $x_t := V(t)^\top \hat \delta_t = V(t)^\top P(t)^{1/2} (w_t - w(t))$ as the $P$-whitened displacement between the discrete optimizer and the central flow along the eigenvectors $V(t)$.
Then the central flow predicts that:
\begin{align*}
    \E[x_t x_t^\top] =  V(t)^\top P(t)^{1/2} \E[ \delta_t \delta_t^\top] P(t)^{1/2} V(t) = V(t)^\top \, P(t)^{1/2} \, \Sigma(t) \, P(t)^{1/2} \, V(t) = \Lambda(t).
\end{align*}
In particular, the $P$-whitened variance of oscillations along the $i$-th eigenvector of $P(t)^{1/2} \, \Sigma(t) \, P(t)^{1/2}$ is predicted to be the $i$-th eigenvalue of that matrix:
\begin{align}
    \E \qty[\qty(v_i(t)^\top P(t)^{1/2}(w_t - w(t)))^2] = \lambda_i(t). \label{eq:abstract-predict-oscillation-variance}
\end{align}

\paragraph{Basis-dependent version}
We can use the general recipe given in \Cref{sec:derivations:dcp} to obtain a basis-dependent version of the central flow ODE that can be computed in time linear in $d$, when the preconditioner is diagonal.

Fix a time $t$, and we will often abbreviate $w(t)$, $\nu(t)$, and $P(\nu(t))$ as $w$, $\nu$ and $P$.
Let $U\in \R^{d \times k}$ be a basis for the critical subspace $\U(w, \nu)$.
Define $H_U(w) := U^\top H(w) U \in \sym(\R^k)$ and $P_U(\nu) := U^\top P(\nu) U \in \sym(\R^k) $, as well as their gradients $\nabla H_U(w) \in \sym(\R^k) \otimes \R^d$ and $\nabla P_U(\nu) \in \sym(\R^k) \otimes \R^{d_\nu}$.  Explicitly:
\begin{align}
    \nabla H_U(w)_{ij} = \nabla_w[u_i^\top H(w) u_j] \qand \nabla P_U(\nu)_{ij} = \nabla_\nu[ u_i^\top P(\nu) u_j ]. \label{eq:arbitrary:nabla-H-U}
\end{align}
Similarly, let $\nabla^2 G_{PU}(\nu, g(w)) \in \R^{d_\nu} \otimes \sym(\R^k)$ be the tensor defined as:
\begin{align}
    \nabla^2 {G_{PU}}(\nu, g(w))_{q,ij} = (Pu_i)^\top \, \nabla^2_g G(\nu, g(w))_{q} \, (P u_j). \label{eq:arbitrary:nabla_GPU}
\end{align}
Then the central flow takes the form:
\begin{align}
    \frac{dw}{dt} &= -P(\nu)^{-1} \qty[ \nabla L(w) + \tfrac{1}{2} \nabla H_U(w)^\top [X] ] \\
    \frac{d\nu}{dt} &= G(\nu, g(w)) + 2 \nabla^2 G_{PU}(\nu, g(w))[X] \\
    X &\in \sdcp_{\R^k} (\alpha_U(w, \nu), \beta_U(w, \nu)) \\
    \alpha_U(w, \nu) &= \nabla H_U(w) \qty[ P(\nu)^{-1}  \nabla L(w) ] + 2 \nabla P_U(\nu) \qty[ G(\nu, g(w)) ] \label{eq:arbitrary:alpha-basis} \\
    \beta_U(w, \nu) &= \tfrac{1}{2} \nabla H_U(w) P(\nu)^{-1} \nabla H_U(w)^\top + 4 \nabla P_U(\nu) \, \nabla^2 G_{PU}(\nu, g(w)). \label{eq:arbitrary:beta-basis}
\end{align}

Suppose that we pick the basis $U$ to be orthonormal w.r.t the preconditioner $P$, i.e. $U^\top P U = I$.
Then the central flow's prediction \cref{eq:abstract-predict-loss} for the time-averaged train loss can be efficiently computed as:
\begin{align}
    \bar{L}(t) &:= L(w(t)) + \tr \qty[P  U X U^\top ] \nonumber \\
     &= L(w(t)) + \tr \qty[ X U^\top P U ] \nonumber \\
    &= L(w(t)) + \tr \qty[ X ]. \label{appendix:derivations:adaptive:time-average-basis-loss}
\end{align}
Similarly, the prediction \cref{eq:abstract-predict-gradient-norm-sq} for the time-averaged squared gradient norm can be computed as:
\begin{align}
    \overline{\| \nabla L \|^2}(t) &=  \|\nabla L(w(t))\|^2 + 4 \tr \qty[ (P U) X (P U)^\top  ]. \label{appendix:derivations:adaptive:time-average-basis-sq-gradient-norm}
\end{align}
As for the oscillation covariance, we can compute both sides of \cref{eq:abstract-predict-oscillation-variance} without materializing $\Sigma(t)$ in full.  If $X = U_X \, \Lambda(t)\, U_X^\top$ is the eigenvalue decomposition of $X$, and $V(t) := P^{1/2} U U_X$, then $V(t) \, \Lambda(t) \, V(t)^\top$ is the eigenvalue decomposition of $P^{1/2} \, \Sigma(t) \, P^{1/2}$.
Note that $X$ and $U_X$ are dependent on the basis $U$, while $V(t)$ and $\Lambda(t)$ are independent of $U$.

\paragraph{Warning: variation in notation}
Although \gd can be cast as an instance of an adaptive preconditioned method (with $P = \eta^{-1} I$), the $U$ defined here is different from the $U$ in \Cref{appendix:derivations:gd:discretizing}, as the $U$ there was orthonormal, i.e. $U^\top U = I$, whereas the $U$ here is orthonormal w.r.t the preconditioner $P$, i.e. $U^\top P U = I$ or $U^\top U = \eta I$.
As a result, the $X$ defined here differs from the $X$ defined in \Cref{appendix:derivations:gd:discretizing} by a factor of $\eta$.

\subsubsection{Discretizing the central flow for a generic adaptive preconditioned method}
\label{appendix:derivations:arbitrary-preconditioned:discretizing}

We describe our general procedure for discretizing DCPs in \Cref{sec:derivations:dcp:discretize}. 
Let us now describe how this general procedure specializes to the case of our generic adaptive preconditioned method.

We use $w^{(t)}$, $\nu^{(t)}$, and $\Sigma^{(t)}$ to denote our estimate for the central flow's $w(t)$, $\nu(t)$, and $\Sigma(t)$.
Let $\epsilon > 0$ be the discretization step size, e.g. $\epsilon = 0.25$.
For some tolerance $\tau > 0$, e.g. $\tau = 0.05$, we will regard eigenvalues of the effective Hessian $P(\nu)^{-1} H(w)$ greater than $2 - \tau$ as those which might become unstable in the next discretization time step.\footnote{We remark that there is a slight discrepancy between the procedure described in \Cref{sec:derivations:dcp:discretize} and this one, as we look at eigenvalues of $P^{-1} H$ close to $2$, rather than eigenvalues of $2P - H$ close to $0$. Our procedure here is equivalent to looking at the \emph{generalized} eigenvalues of $2P-H$ with respect to $P$, i.e. $(2P-H) v = \lambda P v$ as these eigenvalues are exactly equal to the eigenvalues of $P^{-1} H$ shifted by $2$. These procedures produce the same trajectory as $\epsilon \to 0$, and even at finite $\epsilon$ we believe that the trajectories will empirically remain very close.}

At each discretization step, we do the following.
Abbreviate $w = w^{(t)}$ and $P = P(\nu^{(t)})$.
First, we compute all eigenvalues of the effective Hessian $P^{-1} H(w)$ that are greater than $2 - \tau$, as well as the corresponding eigenvectors.
Let $k$ be the number of such eigenvalues, let $D \in \diag(\R^k)$ be a diagonal matrix containing such eigenvalues on the diagonal, and let $U \in \R^{d \times k}$ be the corresponding eigenvectors, normalized so that they are orthonormal w.r.t $P$, that is, $U^\top P U = I$.
For example, one could set $U = P^{-1/2} \tilde{U}$, where the columns of $\tilde{U} \in \R^{d \times k}$ are orthonormal eigenvectors of $P^{-1/2} H(w) P^{-1/2}$.

Then, we compute the tensors $\nabla H_U$, $\nabla P_U$, $\nabla^2 G_{PU}$ as in \cref{eq:arbitrary:nabla-H-U,eq:arbitrary:nabla_GPU}, though note that $U$ now refers to a basis of eigenvectors whose eigenvalues are \emph{almost} 2 rather than exactly equal to 2.
Then, we compute $\alpha_U$ and $\beta_U$ as in \cref{eq:arbitrary:alpha-basis,eq:arbitrary:beta-basis}.
Then, we solve the following $k$-dimensional SDCP:
\begin{align}
    X^{(t)} = \sdcp_{\R^k}(2 I - D + \epsilon \, \alpha_U, \epsilon\, \beta_U),
\end{align}
so that $\Sigma^{(t)} = U X^{(t)} U^\top$.  Then, we update the weights and optimizer state via:
\begin{align}
    w^{(t+\epsilon)} &= w^{(t)} - \epsilon \; P(\nu^{(t)})^{-1} \qty[\nabla L(w^{(t)}) + \tfrac{1}{2} \nabla H_ U^\top[X^{(t)}] ] \label{eq:arbitrary:update-w}\\
    \nu^{(t+\epsilon)} &= \nu^{(t)} + \epsilon \qty[G(\nu^{(t)}, g(w^{(t)})) + 2 \nabla^2 G_{PU}(\nu^{(t)}, g(w^{(t)}))[X^{(t)}] ] \label{eq:arbitrary:update-nu}.
\end{align}
To predict the time-average of the train loss, the squared gradient, and the covariance of the oscillations, we use \cref{appendix:derivations:adaptive:time-average-basis-loss}, \cref{appendix:derivations:adaptive:time-average-basis-sq-gradient-norm}, and \cref{eq:abstract-predict-oscillation-variance}, respectively.

\subsection{Differential Complementarity Problems}\label{sec:derivations:dcp}

In this appendix, we give some brief background on differential complementarity problems \citep{stewart2011dynamics}, and we describe how to turn these into ordinary differential equations.

A \emph{differential complementarity problem} (DCP) \citep{stewart2011dynamics} is a dynamical system that is defined in terms of a complementarity relation.  In this paper, we will consider DCPs of the form:
\begin{align}
    \tfrac{d}{dt}w(t) = f(w(t)) + B(w(t))[\Sigma(t)] \quad \text{where} \quad 0 \preceq \Sigma(t) \perp A(w(t)) \succeq 0. \tag{DCP} \label{eq:dcp}
\end{align}
Here $f: \R^d \to \R^d$, $B: \R^d \to \R^d \otimes \sym(\R^d)$ and $A: \R^d \to \sym(\R^d)$ are given, and $w(t) \in \R^d$ and $\Sigma(t) \in \sym(\R^d)$ are dynamical variables that must respect \cref{eq:dcp} for almost all times $t$.

\begin{example}
    The following $1$-dimensional DCP models a particle moving to the right which hits a wall at $w = 1$:
    \begin{align*}
    \frac{dw(t)}{dt} = 1 - \Sigma(t) \qq{where} 0 \le \Sigma(t) \perp 1-w(t) \ge 0.
\end{align*}
    When $w(t) < 1$, complementarity forces $\Sigma(t) = 0$ so that $\frac{dw(t)}{dt} = 1$, i.e. the particle moves to the right. Once the particle has made contact with the wall at $w=1$, $\Sigma(t)$ must jump to $1$ so that $\frac{dw(t)}{dt^+} = 0$ in order to prevent the particle from violating the condition $w(t) \le 1$.  Thus, if $t^\star$ denotes the $t$ when $w(t)$ hits the wall, then we have $\Sigma(t) = 0$ for $t < t^\star$ and $\Sigma(t) = 1$ for $t > t^\star$. The choice of $\Sigma(t)$ at $t^\star$ itself is arbitrary and does not affect the DCP.
\end{example}

Throughout our derivations, we will search for a $\Sigma(t)$ that is right-continuous, e.g. in the above setting we would define $\Sigma(t^\star) = 1$, so that the leftwards force is applied the instant that $w(t)$ reaches $1$.  

To turn \cref{eq:dcp} into an ordinary differential equation with an explicit right-hand side, we will now prove some additional constraints that the system must satisfy.

First, we prove that if $A(w(t)) \succeq 0$ for all times $t$, then the right derivative $\tfrac{d}{dt^+} A(w(t))$ must be PSD over the subspace $\ker A(w(t))$.
We abbreviate $A(w(t))$ as $A(t)$.

\begin{lemma}[Semidefinite Tangent Cone]\label{lem:semidefinite_tangent_cone}
    Let $A: \R \to \sym(\R^d)$ be a matrix-valued function such that $A(t) \succeq 0$ for all $t$. If $A$ is right-differentiable at $t$, $\tfrac{d}{dt^+} A(t) \succeqOver{\U} 0$ where $\U = \ker A(t)$. If $A$ is differentiable at $t$, $\tfrac{d}{dt} A(t) \evalshort_
\U= 0$.
\end{lemma}
\begin{proof}
    Let $u \in \ker A(t)$. Then for any $\epsilon > 0$,
    \begin{align*}
        u^\top \qty[\frac{A(t+\epsilon)-A(t)}{\epsilon}] u = u^\top \qty[\frac{A(t+\epsilon)}{\epsilon}] u \ge 0.
    \end{align*}
    Taking $\epsilon \to 0$ proves that $\tfrac{d}{dt^+} A(t) \succeqOver{\U} 0$. By reversing time if $A$ is left differentiable at $t$ then $\tfrac{d}{dt^-}A(t) \preceqOver{\U} 0$. Combining these inequalities shows that if $A$ is differentiable at $t$ then $\frac{d}{dt} A(t) \evalshort_{\U} = 0$.
\end{proof}

Next, we show that if the complementarity relation $0 \preceq \Sigma(t) \perp A(w(t)) \succeq 0$ holds for all times $t$, then the right derivative $\tfrac{d}{dt^+} A(w(t))$ must satisfy its own complementarity relation: $0 \preceq \Sigma(t) \perp \tfrac{d}{dt} A(w(t)) \succeqOver{\ker A(w(t))} 0$.

\begin{lemma}[Differentiating the complementarity relation]\label{corollary:right_continuous_sdcp}
    If $A: \R \to \sym(\R^d)$ is right-differentiable at $t$, $0 \preceq \Sigma(t) \perp A(t) \succeq 0$ for all $t$, and $\Sigma(\cdot)$ is right-continuous at $t$, then $0 \preceq \Sigma(t) \perp \frac{d}{dt^+} A(t) \succeqOver{\U} 0$ where $\U = \ker A(t)$.
\end{lemma}
\begin{proof}
    The only condition that needs to be checked is $\Sigma(t) \perp \frac{d}{dt^+} A(t)$, since $\Sigma(t) \succeq 0$ is given and $\frac{d}{dt^+} A(t)\succeqOver{\U} 0$ is implied by \Cref{lem:semidefinite_tangent_cone}.
    First, for any $\epsilon > 0$, we have $\Sigma(t+\epsilon) \perp A(t+ \epsilon)$, which implies:
    \begin{align*}
        \ev{\Sigma(t+\epsilon),\frac{A(t+\epsilon)-A(t)}{\epsilon}} = -\ev{\Sigma(t+\epsilon),\frac{A(t)}{\epsilon}} \le 0.
    \end{align*}
    Taking $\epsilon \to 0$ and using right-continuity of $\Sigma(\cdot)$ implies $\ev{\Sigma(t), \frac{d}{dt^+}A(t)} \le 0$.
    Meanwhile, the assumption $0 \preceq \Sigma(t) \perp A(t) \succeq 0$ implies that $\spn \Sigma(t) \subseteq \U$.  Since $\tfrac{d}{dt^+}A(t) \succeqOver{\U} 0$ by \Cref{lem:semidefinite_tangent_cone}, we have that $\ev{\Sigma(t), \frac{d}{dt^+}A(t)} \ge 0$.
    Since $\ev{\Sigma(t), \frac{d}{dt^+}A(t)}$ is both $\le 0$ and $\ge 0$, it must be $=0$, i.e. $\Sigma(t) \perp \tfrac{d}{dt^+} A(t)$ as desired.
\end{proof}
Note that as discussed in \Cref{appendix:smoothness}, throughout the rest of this work, when we write $\tfrac{dw}{dt}$ we either mean $\tfrac{dw}{dt^+}$ or we mean that the statement holds for almost all $t$.

Now we have enough information to solve for $\Sigma(t)$. Explicitly, if $\U(w) := \ker A(w)$, we can define
\begin{align*}
    \alpha(w) := \nabla A(w)[f(w)] \in \sym(\R^d) \qc \beta(w) := \nabla A(w) B(w) \in \sym(\R^d)^{\otimes 2}.
\end{align*}
\begin{lemma}\label{lem:dcp_to_sdcp}
    If $w(t),\Sigma(t)$ satisfy \cref{eq:dcp} for almost all $t$, and $\Sigma(\cdot)$ is right-continuous, then $$\Sigma(t) \in \sdcp_{\U(w)}\qty(\alpha(w(t)),\beta(w(t)))$$ for all $t$, where $\U(w) = \ker A(w)$. Furthermore if $\beta(w(t))$ is symmetric positive definite as an operator on $\sym(\U(w))$ for all $t$, then $\Sigma(t) = \Sigma(w(t))$ is unique.
\end{lemma}

\begin{proof}
    First, \Cref{corollary:right_continuous_sdcp} implies that
    \begin{align*}
        0 \preceq \Sigma(t) \perp \tfrac{d}{dt} A(w(t)) \succeqOver{\U(w(t))} 0
    \end{align*}
    where $\U(w) = \ker A(w(t))$. By the chain rule,
    \begin{align*}
        \frac{d}{dt} A(w(t)) = \nabla A(w(t))\qty[\frac{dw}{dt}] = \nabla A(w(t))\qty[f(w(t)) + B(w(t))[\Sigma(t)]].
    \end{align*}
    Therefore, using that $\Sigma \in \sym(\U(w))$, we can expand $\frac{d}{dt} A(w(t))$ as:
    \begin{align*}
        \frac{d}{dt} A(w(t)) = \alpha(w(t)) + \beta(w(t))[\Sigma(t)],
    \end{align*}
    which implies that: $$\Sigma(t) \in \sdcp_{\U(w(t))}\qty(\alpha(w(t)),\beta(w(t))).$$
\end{proof}

This lets us turn \cref{eq:dcp} into an ODE:
\begin{tcolorbox}
\begin{equation}
\begin{gathered}
    \frac{dw}{dt} = f(w) + B(w)[\Sigma] \qq{where} \Sigma \in \sdcp_{\U(w)}\qty(\alpha(w),\beta(w)) \\
    \alpha(w) := \nabla A(w)[f(w)] \qc \beta(w) := \nabla A(w) B(w) \qc \U(w) := \ker A(w)
\end{gathered}\tag{DCP to ODE}
\end{equation}
\end{tcolorbox}

We can also define an equivalent basis-dependent version of this ODE, which makes clear that simulating the ODE only requires time and space that scale linearly in the dimension $d$.
For each $t$, let $U \in \R^{d \times k}$ denote a basis for the critical subspace $\U(w(t))$, where $k = \dim \U(w(t))$. Then we can write \cref{eq:dcp} in terms of $U$ as follows:
\begin{tcolorbox}
\begin{align}
\begin{split}
    \frac{dw}{dt} &= f(w) + B(w)[UXU^\top] \\
    X &\in \sdcp_{\R^k}(\alpha_U(w),\beta_U(w)) \\
    \alpha_U(w) &:= U^\top \alpha(w) U \in \sym(\R^k) \\
    \beta_U(w)[X] &:= U^\top \qty(\beta(w)[UXU^\top]) U \in \sym(\R^k) \quad\forall X \in \sym(\R^k).
\end{split}\tag{DCP to ODE, basis-dependent}
\end{align}
\end{tcolorbox}

Thus, to compute $\tfrac{dw}{dt}$, one need only compute $\alpha_U \in \sym(\R^k)$ and $\beta_U \in \sym(\R^k)^{\otimes 2}$, then solve a $k$-dimensional SDCP to obtain $X \in \sym(\R^k)$, then form $\tfrac{dw}{dt}$ in terms of $X$.

In practice, due to nonsmoothness of the DCP, we do not discretize it by computing $\tfrac{dw}{dt}$ and then taking Euler steps.
Instead, we discretize the DCP directly, as described below in \Cref{sec:derivations:dcp:discretize}.

\paragraph{Relation to the broader DCP literature}
    In the literature on differential complementarity problems, it is a standard practice to turn a DCP into an ODE by differentiating the constraints.
    In particular,  \cref{eq:dcp} is known as a \emph{pure index-one DCP} \citep[Section 5.2.1]{stewart2011dynamics}, because it needs to be differentiated exactly once in order to be turned into an explicit ODE.    

\paragraph{Existence and uniqueness of the ODE}
Existence and uniqueness for projected gradient flows (including the \gd central flow) is guaranteed by \citet{cornet1983ExistenceSlowSolutions}, under the assumption that $\beta(w)$ is positive definite on $\sym(\U(w))$ for all $w$, where $\U(w) = \ker A(w)$ denotes the critical subspace.  We view this assumption as mild.
Existence and uniqueness can be argued for DCPs more generally (including the other central flows) under the assumption that $\beta$ is positive definite over all $\sym(\R^d)$ \citep{stewart2011dynamics}.
However, this assumption is too strong for our setting; indeed, for \gd, $\beta$ has rank at most $d$ and hence cannot span the $\tfrac{d(d+1)}{2}$ dimensional space of $\sym(\R^d)$ whenever $d > 1$.  However, we suspect that under additional reasonable regularity conditions, this could be relaxed to a condition that $\beta$ need only be positive definite over $\sym(\ker A(w))$.

\paragraph{Smoothness of the ODE}
The ODE is non-smooth at \emph{breakpoints} where the dimension of $\U(w)=\ker A(w)$ changes.  In between the breakpoints, $\Sigma(t)$ is continuous and $w(t)$ is differentiable.
Moreover, the solution to the SDCP is given by the linear inverse $\Sigma(t) = U \beta^{-1}_U(w)[\alpha_U(w)] U^\top$ and we have $\tfrac{d}{dt} A(t) \evalshort_{\U(w)} = 0$.
At the breakpoints, however, $\Sigma(t)$ is discontinuous and $w(t)$ is not differentiable (although they are right-continuous and right-differentiable, respectively)

\subsubsection{Discretizing the DCP}
\label{sec:derivations:dcp:discretize}
We now describe how we discretize \cref{eq:dcp} in practice.
Let $\epsilon>0$ denote the discretization step size, and consider a grid of time steps $\mathcal T=\{t_0,t_1,\hdots,t_N\}$ that are spaced apart by $\epsilon$. For any $t\in\mathcal T$, let $w^{(t)}$ denote our estimate for the central flow’s $w(t)$, and let $\Sigma^{(t)}$ denote our estimate for the central flow’s $\Sigma(t)$.
In what follows, we will use $t\in\mathcal T$ to denote the current time step and $t+\epsilon\in\mathcal T$ to denote the next one.

It is not straightforward to discretize \cref{eq:dcp} using time-stepping. For example, it doesn’t make sense to search for a pair $w^{(t+\epsilon)},\,\Sigma^{(t)}$ that satisfies
\begin{align*}
    w^{(t+\epsilon)} \;=\; w^{(t)} + \epsilon\,\qty( f\qty(w^{(t)}) + B\qty(w^{(t)})\bqty{\Sigma^{(t)}} )
    \qand 0 \preceq \Sigma^{(t)} \perp A\qty(w^{(t)}) \succeq 0,
\end{align*}
as this would always be satisfied by $\Sigma^{(t)}=0$. We could instead search for a pair $w^{(t+\epsilon)},\,\Sigma^{(t)}$ which satisfies
\begin{align*}
    w^{(t+\epsilon)} \;=\; w^{(t)} + \epsilon\,\qty( f\qty(w^{(t)}) + B\qty(w^{(t)})\bqty{\Sigma^{(t)}} )
    \qand 0 \preceq \Sigma^{(t)} \perp A\qty(w^{(t+\epsilon)}) \succeq 0,
\end{align*}
where the complementarity constraint is enforced at time $t+\epsilon$, rather than at time $t$. However, it is still difficult to handle this constraint due to the nonlinearity of $A$. We will therefore \emph{linearize} $A$ around $w^{(t)}$.
Define $A^{\text{lin}}(w)$ as the linearization of $A$ around the current point $w^{(t)}$:
\begin{align}
    A^{\text{lin}}(w) \;:=\; A\qty(w^{(t)}) + \nabla A\qty(w^{(t)})\bqty{\,w - w^{(t)}\,}.
\end{align}
We can therefore look for a choice of $w^{(t+\epsilon)}$ and $\Sigma^{(t)}$ that together satisfy
\begin{align}
    w^{(t+\epsilon)} \;=\; w^{(t)} + \epsilon\,\qty( f\qty(w^{(t)}) + B\qty(w^{(t)})\bqty{\Sigma^{(t)}} )
    \qand 0 \preceq \Sigma^{(t)} \perp A^{\text{lin}}\qty(w^{(t+\epsilon)}) \succeq 0.
\end{align}
This implies the following set of conditions on $\Sigma^{(t)}$ alone:
\begin{align}
    0 \preceq \Sigma^{(t)} \perp 
    A\qty(w^{(t)}) + \epsilon \; \nabla A\qty(w^{(t)})\bqty{ f\qty(w^{(t)}) + B\qty(w^{(t)})\bqty{\Sigma^{(t)}} },
\end{align}
which we recognize as precisely an SDCP:
\begin{align}
    \Sigma^{(t)} \;\in\; \sdcp_{\mathbb{R}^d}\qty(
        A\qty(w^{(t)}) + \epsilon\, \nabla A\qty(w^{(t)})\bqty{ f\qty(w^{(t)}) },\;
        \epsilon\, \nabla A\qty(w^{(t)})\, B\qty(w^{(t)})
    ).
\end{align}
Thus one could solve for $\Sigma^{(t)}$ above, and then take an Euler step on $w$:
\begin{align*}
    w^{(t+\epsilon)} \;\leftarrow\; w^{(t)} + \epsilon\,\qty( f\qty(w^{(t)}) + B\qty(w^{(t)})\bqty{\Sigma^{(t)}} ).
\end{align*}
Unfortunately, it is not possible to directly run this ``idealized'' time-stepping scheme, as it is impractical to formulate or solve an SDCP over $\R^d$.
Therefore, we instead approximate it by projecting it onto the bottom eigendirections of $A$ which are ``close'' to the stability threshold at $0$. \jnote{spell out} In particular, if $U \in \R^{d \times k}$ is a basis of these directions and $\U = \spn U$, then we require $\Sigma^{(t)} \in \sym(\U)$, and we enforce:
\begin{align}
    0 \preceq \Sigma^{(t)} \perp A^{\text{lin}}\qty(w^{(t+\epsilon)}) \succeqOver{\U} 0.
\end{align}
This boils down to an SDCP over the low-dimensional subspace $\U$:
\begin{align}
    \Sigma^{(t)} \in \sdcp_{\U}\qty(
        A\qty(w^{(t)}) + \epsilon\, \nabla A\qty(w^{(t)})\bqty{ f\qty(w^{(t)}) },\;
        \epsilon\, \nabla A\qty(w^{(t)})\, B\qty(w^{(t)})
    ).
\end{align}
As described in \Cref{sec:derivations:sdcp}, solving this SDCP only requires solving a $k$-dimensional SDCP:
\begin{align}
    \Sigma^{(t)} \;=\; U\, X^{(t)} U^\top
    \qc
    X^{(t)} \;=\; \sdcp\qty(
        A_U\qty(w^{(t)}) + \epsilon \, \alpha_U\qty(w^{(t)}),\;
        \epsilon\, \beta_U\qty(w^{(t)})
    ),
\end{align}
where $A_U(w) \in \sym(\R^k)$, $\alpha_U(w) \in \sym(\R^k)$, and $\beta_U(w) \in \sym(\R^k)^{\otimes 2}$ are defined as
\begin{align}
    A_U(w) &:= U^\top A\qty(w)\, U, \\
    \alpha_U(w) &:= U^\top \bqty{ \nabla A\qty(w)\bqty{ f\qty(w) } } U, \\
    \beta_U(w)[X] &:= U^\top \bqty{ \nabla A\qty(w)\bqty{ B\qty(w)\bqty{U X U^\top} } } U.
\end{align}
Then, we update the weights using:
\begin{align*}
    w^{(t+\epsilon)} \;=\; w^{(t)} + \epsilon \; \qty( f\qty(w^{(t)}) + B\qty(w^{(t)})\bqty{U X^{(t)} U^\top} ).
\end{align*}
For the central flows in this work, $A$ and $B$ are such that $\alpha_U$, $\beta_U$, and $\tfrac{dw}{dt}$ can be computed efficiently.

\paragraph{Correctness of this discretization scheme} Under suitable conditions on $f,A,B$, this time-stepping scheme will converge to the solution of the DCP as $\epsilon \to 0$ \citep{stewart2011dynamics}. For technical reasons involving the rank of $\nabla A B$, our paper will not rigorously prove the convergence of this time-stepping scheme for our central flows. However, we do empirically observe that the dynamics converge to a limiting curve as $\epsilon \to 0$.

\subsection{Continuous-time approximation to an EMA}
\label{sec:continuous-time-ema}
In this appendix, we justify our continuous-time approximation to an exponential moving average (EMA).

Consider a discrete-time EMA of a continuous process $f(t)$ that is sampled at integer-valued times:
\begin{align}
    \nu_t = \beta_2 \nu_{t-1} + (1 - \beta_2) f(t). \label{eq:discrete-ema}
\end{align}
What is a good continuous-time approximation $\nu(t)$ to $\nu_t$?  This question arises when we derive stable and central flows for \rmsnorm and \rmsprop (where $f$ is the squared gradient norm or the elementwise squared gradient, respectively).

Subtracting $\nu_{t-1}$ from both sides of \cref{eq:discrete-ema} and rearranging yields:
\begin{align}
    \nu_t - \nu_{t-1} = (1 - \beta_2) (f(t) - \nu_{t-1}).
\end{align}
This suggests the following continuous-time approximation to \cref{eq:discrete-ema}:
\begin{align}
    \nu'(t) = (1 - \beta_2)(f(t) - \nu({t})) \label{eq:naive-continuous-ema}.
\end{align}
However, this approximation breaks down for small $\beta_2$.  Indeed, as $\beta_2 \to 0$, the discrete-time EMA \cref{eq:discrete-ema} adapts ``infinitely fast'' so that $\nu_t \approx f(t)$, yet the naive continuous-time approximation \cref{eq:naive-continuous-ema} does not have this property. Therefore, to obtain a continuous-time approximation that works well even for small $\beta_2$, we use the following alternative approximation:
\begin{align}
    \nu'(t) = \left(\frac{1 - \beta_2}{\beta_2} \right)(f(t) - \nu(t)) \label{eq:better-continuous-ema}.
\end{align}
The $\beta_2$ in the denominator ensures that when $\beta_2\approx 0$, $\nu(t)$ adapts ``infinitely'' fast to $f(t)$.

To give intuition for our approximation \cref{eq:better-continuous-ema}, suppose that we are using an EMA to track a one-dimensional \emph{linear} process $f(t)$ (i.e. $f'(t)$ is constant).
Then both the discrete and continuous time EMAs have closed forms.
The closed-form solution to the discrete time EMA \cref{eq:discrete-ema} can be written as:
\begin{align}
    \nu_t = f(\underbrace{t - \tau}_\text{time-lag}) + \underbrace{\beta_2^t \qty[\nu_0 - f(-\tau)]}_\text{burn in} \qq{where} \tau := \frac{\beta_2}{1-\beta_2}. \label{eq:appendix_ema_discrete}
\end{align}
Since the burn-in term vanishes exponentially with time, the steady state is that $\nu_t$ will track $f(t)$ with a ``time delay'' of $\tau := \frac{\beta_2}{1-\beta_2}$. 

Similarly, for a continuous-time EMA of the form $\nu'(t) = \gamma \qty[f(t) - \nu(t)]$ for some $\gamma$, the general solution is:
\begin{align}
    \nu(t) = f(\underbrace{t - \tau}_\text{time-lag}) + \underbrace{e^{-\gamma t} \qty[\nu(0) - f(-\tau)]}_\text{burn in} \qq{where} \tau := 1/\gamma.
    \label{eq:appendix_ema_continuous}
\end{align}
Thus, the steady state is that $\nu(t)$ will track $f(t)$ with a ``time delay'' of $\tau = 1/\gamma$.

To ensure that the continuous-time EMA asymptotically matches the discrete-time EMA, we need to set $\gamma$ so that the time delays match:
\begin{align*}
    \frac{\beta_2}{1-\beta_2} = \frac{1}{\gamma} \implies \gamma = \frac{1-\beta_2}{\beta_2},
\end{align*}
This motivates our choice of scaling factor in \cref{eq:better-continuous-ema}.

Note that if we instead wanted to match the \emph{burn in}, we would set $\gamma = \log(1/\beta_2)$. However, we believe it is more important to match the time-delay than the burn in.

\subsection{Miscellaneous math}
\label{sec:misc-math}
Here, we state some miscellaneous mathematical facts that are used elsewhere.

\begin{fact}\label{lem:gd_taylor_proof}
    Let $L(w)$ be three-times differentiable, and let $S(w) := \lambda_1(H(w))$ denote the top Hessian eigenvalue.
Suppose that the top Hessian eigenvalue at $\wbar$ has multiplicity 1, and let $u$ be the top Hessian eigenvector at $\overline{w}$. Then, for $w = \wbar + xu$, we have:
    \begin{align}
        \nabla L(w) = \nabla L(\overline{w}) + S(\overline{w}) x u + \tfrac{x^2}{2} \nabla S(\overline{w}) + o(x^2).
    \end{align}
\end{fact}
\begin{proof}
    We begin by writing the Taylor expansion of $\nabla L(w)$ around $\overline{w}$:
    \begin{align}
        \nabla L(w) = \nabla L(\overline{w}) + H(\overline{w}) x u + \tfrac{x^2}{2} \nabla^3 L(\overline{w})[u,u] + o(x^2).
    \end{align}
    Because $u$ is an eigenvector of $H(\overline{w})$ with eigenvalue $S(\overline{w})$, the second term can be simplified to $S(\overline{w}) x u$. Finally, by Danskin's theorem (or equivalently the standard formula for the derivative of an eigenvalue):
    \begin{align}
        \nabla_{\overline{w}} S(\overline{w}) = \nabla_{\overline{w}}\qty[\max_{\norm{v} = 1} v^\top H(\overline{w}) v] = \nabla_{\overline{w}} \qty[u^\top H(\overline{w}) u] = \nabla^3 L(\overline{w})[u,u]
    \end{align}
    where $u$ is the argmax of the second expression, i.e. the top eigenvector of the Hessian at $\overline{w}$.
\end{proof}

\begin{fact}
    For PSD matrices $X, Y \succeq 0$, it holds that $\tr(X Y) = 0$ if and only if $\spn X \perp \spn Y$.
    \label{fact:orthogonal-spans}
\end{fact}
\begin{proof}
    First, if $\spn X \perp \spn Y$, then $\spn Y \subseteq \ker X$.
    Thus, it must hold for every $u$ that $X Y u = 0$, implying that $XY = 0$ and hence $\tr(XY) = 0$.

    For the other direction, assume that $\tr(XY) = 0$. Then:
    \begin{align*}
        0 = \tr(XY) = \tr(Y^{1/2} X^{1/2} X^{1/2} Y^{1/2}) =  \norm{X^{1/2} Y^{1/2}}_F^2.
    \end{align*}
    This norm can only be zero if $X^{1/2} Y^{1/2} = 0$.
    Multiplying on the right by $Y^{1/2}$ and the left by $X^{1/2}$ gives $XY = 0$.  This implies that $\spn X \perp \spn Y$, as for any $x \in \spn X$ and $y \in \spn Y$, we have $x=Xu$ and $y = Yv$ for some $u,v$, and so
    \begin{align*}
        x^\top y = (Xu)^\top (Yv) = u^\top X Y v = 0.
    \end{align*}
\end{proof}
Note that since the span of a symmetric matrix is orthogonal to its kernel, $\spn X \perp \spn Y$ is equivalent to $\spn X \subseteq \ker Y$ and to $\spn Y \subseteq \ker X$.
Thus the following corollary is immediate:
\begin{corollary}
    For PSD matrices $X, Y \succeq 0$, it holds that $\tr(X Y) = 0$ if and only if $\spn X \subseteq \ker Y$.
    \label{fact:orthogonal-spans-corollary}
\end{corollary}

%% file: appendix-1-experimental-details.tex
\newpage
\section{Experimental Details}
\label{appendix:experimental-details}

\subsection{Implementation details}
\label{appendix:experimental-details:implementation}

Our code can be found at: \url{http://github.com/centralflows/centralflows}.

In order to reuse code between the central flows for \gd, \rmsnorm, and \rmsprop, we cast all three optimizers as instances of the generic adaptive preconditioned method that is described in \Cref{sec:arbitrary-preconditioned}.
This template assumes that the weights are updated via a preconditioned gradient step of the form $w_{t+1} = w_t - P(\nu_t)^{-1} \nabla L(w_t)$, where $P(\nu_t)$ is a preconditioner that is derived from some optimizer state $\nu_t$ that is in turn updated based on the gradients. For example, for \gd with learning rate $\eta$, the preconditioner is simply $P = \eta^{-1} I$.  The effective Hessian is defined as $P(\nu_t)^{-1} H(w_t)$, and the EOS condition is that the largest eigenvalue of this matrix (the effective sharpness $\Seff$) is 2.  See \Cref{sec:arbitrary-preconditioned} for more information.

\paragraph{Eigenvalue computation} To regularly recompute the top eigenvalues and eigenvectors of the effective Hessian, we use the LOBPCG algorithm \citep{knyazev2001toward}, which only requires access to the Hessian via Hessian-vector products, and which allows us to warm-start using the previously computed eigenvectors.  We were originally inspired by the LOBPCG implementation in Jax's \texttt{jax.experimental.sparse.linalg} \citep{jax2018github}.

\paragraph{How many eigenvalues to track?}  For all processes (i.e. discrete optimizers, central flows, stable flows), we track all eigenvalues of the effective Hessian that are above the threshold $1.5$.  We then track the same number of eigenvalues of the Hessian.  Note that for \gd and \rmsnorm the Hessian eigenvalues are trivially related to the effective Hessian eigenvalues, whereas for \rmsprop we need to do an extra eigenvalue solve to obtain the Hessian eigenvalues.

\paragraph{Discretizing the stable flow} To discretize the stable flows (e.g. gradient flow), we use Euler's method.
To discretize for one unit of time, we pick some integer \texttt{nsubsteps}, we set $\epsilon = 1/$\texttt{nsubsteps}, and we repeat $w \leftarrow w+ \epsilon \, \frac{dw}{dt}$ for \texttt{nsubsteps} times.
We dynamically adapt \texttt{nsubsteps} based on the current effective sharpness $\Seff$.
The basic criterion of update stability requires that $\epsilon < 2 / \Seff$.
To be on the safe side, and to guard against any implicit discretization effects, we enforce the stronger condition that $\epsilon < 0.5 / \Seff$, or equivalently that \texttt{nsubsteps} $\ge \lceil 2 \Seff \rceil$.
We also enforce a floor of \texttt{nsubsteps} $\ge 4$.
Thus, we set $\texttt{nsubsteps} = \max(4, \lceil 2 \Seff \rceil)$.

Since discretizing the stable flow would take a prohibitively long time in regions of weight space where the effective sharpness is too high, we automatically terminate the stable flow if the effective sharpness exceeds a certain threshold (we used \texttt{100}).

\paragraph{Discretizing the central flow} To discretize the central flows, we use the scheme described in \Cref{appendix:derivations:arbitrary-preconditioned:discretizing}.  This is in turn an instance of the general scheme described in \Cref{sec:derivations:dcp:discretize} for discretizing differential complementarity problems.
At each discretization time step, we do the following:
\begin{enumerate}
    \item \label{step:compute-top-eigenvectors} We compute the top eigenvalues and eigenvectors of the effective Hessian $P(\nu)^{-1} H(w)$.  
    In particular, we compute all eigenvalues greater than $2 - \tau$ for some small tolerance $\tau > 0$ (we use \texttt{0.05}), as well as the corresponding eigenvectors.  Let $k$ be the number of such eigenvalues.
    To compute eigenvalues of the non-symmetric matrix $P(\nu)^{-1} H(w)$, we first use warm-started LOBPCG to compute the eigenvalues $D \in \diag(\R^{k})$ and orthonormal eigenvectors $\tilde U \in \R^{d \times k}$ of the symmetric matrix $P(\nu)^{-1/2} H(w) P(\nu)^{-1/2}$.  The eigenvalues of $P(\nu)^{-1} H(w)$ are then $D$, and the eigenvectors are $U = P(\nu)^{-1/2} \tilde U \in \R^{d \times k}$.  Note that these eigenvectors are orthonormal w.r.t the preconditioner $P(\nu)$, i.e. $U^\top P(\nu) U = I$.
    \item \label{step:compute-nabla-H} We compute the third-derivative tensor $\nabla_U H(w) \in \sym(\R^k) \otimes \R^d$ defined in \cref{eq:arbitrary:nabla-H-U} by looping over all pairs $(u_i, u_j)$ of columns of $U\in \R^{d \times k}$ and computing the third derivative $\nabla_w [u_i^\top H(w) u_j] \in \R^d$, which can be done using automatic differentiation.  Note that due to the symmetry, it is only necessary to compute and store the $k+1 \choose 2$ ``upper triangular'' entries of this tensor, rather than the full $k^2$.
    \item We compute the tensor $\nabla P_U(\nu) \in \sym(\R^k) \otimes \R^{d_\nu}$ defined in \cref{eq:arbitrary:nabla-H-U}, which measures the gradient of the preconditioner w.r.t the optimizer state $\nu \in \R^{d_\nu}$.  Our implementation uses automatic differentiation to do this, though one could also simply hard-code the derivatives for the various optimizers of interest.  We similarly compute the tensor $\nabla^2 G_{PU}(\nu, g(w)) \in \R^{d_{\nu}} \otimes \sym(\R^k)$ defined in \cref{eq:arbitrary:nabla_GPU}, where $g(w) = \nabla L(w)$ and $d_\nu$ is the dimension of the optimizer state $\nu$.  This tensor measures the Hessian of the optimizer state w.r.t the gradient of the weights.  For these tensors, we also only need to compute and store the $k+1 \choose 2$ upper triangular entries.
    \item Using $\nabla H_U(w)$, $\nabla P_U(\nu)$, and $\nabla^2 G_{PU}(\nu, g(w))$, we compute the tensors $\alpha_U(w, \nu) \in \sym(\R^k)$ and $\beta_U(w, \nu) \in \sym(\R^k)^{\otimes 2}$ defined in \cref{eq:arbitrary:alpha-basis,eq:arbitrary:beta-basis}.  Here we simply materialize the full tensors, as $k$ is small.
    \item We solve the semidefinite complementarity problem:
    \begin{align*}
        X = \sdcp_{\R^k}(2 I - \Lambda + \epsilon \, \alpha_U, \epsilon\, \beta_U).
    \end{align*}
    We do so by formulating the SDCP as a semidefinite-constrained quadratic program \cref{eq:sdcp-basis-quadratic-program}, as described in \Cref{sec:derivations:sdcp}, and solving this using the convex programming library \texttt{cvxpy}.
    \item We take an Euler step of size $\epsilon$ on the weights $w$ and optimizer state $\nu$, as given in \cref{eq:arbitrary:update-w,eq:arbitrary:update-nu}.
\end{enumerate}

To discretize the flow for one unit of time, we pick some integer \texttt{nsubsteps}, set the discretization step size as $\epsilon = 1 / $\texttt{nsubsteps}, and repeat the above process for \texttt{nsubsteps} times. 
Note that because the central flows keep the effective sharpness controlled at 2, it is not necessary to dynamically adapt \texttt{nsubsteps} throughout training, as \texttt{nsubsteps} $\ge 2$ always sufffices to ensure stability.  We therefore used the fixed values \texttt{nsubsteps} $=4$ and $\epsilon = 0.25$.

The computational cost of each discretization step is dominated by the cost of computing the top eigenvalues and eigenvectors in step \ref{step:compute-top-eigenvectors} above, as well as the cost of computing the third derivatives in step \ref{step:compute-nabla-H} above.  The computational cost of the rest of the steps (including solving the SDCP) is negligible.  The time complexity of each discretization step scales quadratically with the number of eigenvalues that are at the edge of stability, $k$, as ${k+1 \choose 2} = \Theta(k^2)$ third derivatives need to be computed.

\paragraph{Verifying the central flow}
To assess whether the central flow accurately models the discrete optimizer, we run both processes simultaneously.  That is, we repeatedly both (a) take a step on the discrete optimizer; and (b) discretize the central flow for one unit of time.  Let $w_t$ denote the discrete optimizer's iterate at step $t$, and let $w^{(t)}$ denote our estimate for the central flow solution at time $t$.

To verify that the central flow approximates the long-term weight-space trajectory of the discrete optimizer, we record $\| w^{(t)} - w_t \|$, the Euclidean distance between the discrete optimizer and the central flow.

We use \cref{appendix:derivations:adaptive:time-average-basis-loss,appendix:derivations:adaptive:time-average-basis-sq-gradient-norm} to compute the central flow's predictions for the time-average of the train loss and squared gradient norm.
For any quantity $f$ (e.g. loss or squared gradient norm), let $\bar{f}(w^{(t)})$ denote the central flow's prediction for the time-average of $f$ at step/time $t$.
To assess the accuracy of this prediction, we compare $\{\bar{f}(w^{(t)})\}$ against a Gaussian smoothing of the empirical time series $\{f(w_t)\}$.  That is, for each $t$, we compare:
\begin{align*}
    \bar{f}(w^{(t)}) \quad \text{vs.} \quad \qty(\frac{1}{\sum_j c_j}) \sum_{j}^{} c_{j} f(w_{t+j}) \quad \text{where} \quad c_j = \exp\qty(-\frac{j^2}{2 \sigma^2}),
\end{align*}
where $\sigma^2$ is the bandwidth of the Gaussian kernel.  We describe below momentarily how we determine $\sigma^2$.

\jnote{get the details right about the wraparound}

\jnote{say that for GD (and perhaps for scalar rmsprop too), we do some transformation so as to return the eigenvalues of sigma rather than preconditioned sigma }

When predicting the covariance of the oscillations, we compare both sides of \cref{eq:abstract-predict-oscillation-variance}.  Let $(\lambda^{(t)}_i, v_i^{(t)})$ denote the $i$-th eigenvalue and eigenvector of the matrix $P^{1/2}(\nu^{(t)}) \, \Sigma^{(t)} \, P^{1/2}(\nu^{(t)})$, and define $x_i^{(t)} := \ev{v_i^{(t)}, P^{1/2}(\nu^{(t)}) (w(t) - w_t) }^2$.
Then for each eigenvalue index $i$, we compare the predicted time series $\{ \lambda_i^{(t)} \}$ against a Gaussian smoothing of the empirical time series $\{x_i^{(t)}\}$.  For \gd and \rmsnorm, where $P$ is a scalar, we post-hoc rescaled both quantities by $P$ so as to report the oscillations in terms of $\Sigma$ rather than $P^{1/2} \Sigma P^{1/2}$.

For some plots (e.g. those in the main paper), we picked the Gaussian kernel's bandwidth by visual inspection (we emphasize that it is not possible to turn a bad prediction into a good prediction by adjusting the bandwidth).  
In fact, for some figures (e.g. \Cref{fig:rmsprop-norm,fig:rmsprop}), we found it best to re-tune the bandwidth \emph{within} an experiment, whenever the number of unstable eigenvalues underwent a change.
On the other hand, for the plots on the ``bulk'' experiments section (\Cref{sec:bulk-experiments}), we picked a single bandwidth somewhat arbitrarily and used this for all experiments.

\paragraph{Warm-start} In our experiments, we first run the discrete optimizer for 10-15 steps and then use this as an initialization for the discrete optimizer, central flow, and stable flow.  The first reason why we do this is that the effective sharpness is sometimes much larger than $2$ at the original initialization (particularly for the adaptive optimizers), yet comes down below 2 within the very first few steps of training.   (The central flow is not currently defined when the effective sharpness is greater than 2.)  Another reason is that even when the effective sharpness is less than 2 at initialization, we observed that the quality of the central flow approximation is sometimes enhanced by this warm start.
This is potentially due to the size of the gradients during the first few steps, and the central flow could possibly be improved during this phase by incorporating the implicit gradient norm penalty from \cite{barrett2021implicit}.  However, we think it is likely that this source of deviation between the discrete optimizer and the stable / central flows is negligible in the long run (see \Cref{sec:igr}), and thus we hypothesize that the warm-starting could be removed in cases where the effective sharpness is less than 2 at initialization.

\paragraph{Second-order midpoints}
When reporting metrics from the discrete optimizers (\gd, \rmsnorm, and \rmsprop, as opposed to their central flows), we usually report the top Hessian eigenvalues measured not at the optimizer iterates $\{w_t\}$ themselves, but at the (second order) midpoints  $\{\hat{w}_t\}$, defined as $\hat{w}_t := \frac{1}{4} \left[2 w_t - w_{t-1} -w_{t+1} \right]$ (so named because it is the midpoint of the midpoints $\tfrac{1}{2}[w_{t} - w_{t-1}]$ and $\tfrac{1}{2}[w_{t+1}-w_t]$).  This empirically makes the Hessian trajectories slightly crisper (less ``noisy''), while not altering the fundamental patterns.

\subsection{Architecture details}
\label{appendix:experimental-details:architectures}

Here we describe our architectures.  Note that our code for all architectures can be found at:\\
\url{http://github.com/centralflows/centralflows}. 

\begin{tcolorbox}[enhanced jigsaw,colback=blue!10,left=2pt,right=2pt,top=2pt,bottom=2pt]
Both our derivations and the analytic formulas for the central flows rely on higher-order information about the loss function (e.g. Hessians and third derivatives). As a result, we require that all of the architectures are smooth. This rules out the commonly used ReLU activation \citep{nair2010rectified}.
\end{tcolorbox}

\paragraph{CNN} Our CNN has four layers, an initial channel width of 32, and 3x3 convolutional kernels.  It uses the GeLU activation function, average pooling, and a linear readout layer.

\paragraph{ResNet} 
We use a ResNet \citep{he2015deep} with 20 layers and GeLU activations. We use GroupNorm \citep{wu2018group} in place of BatchNorm \citep{ioffe2015batch}, as we empirically find that BatchNorm often leads to sub/super-quadraticity (see the discussion in \Cref{sec:higher-order}).

\paragraph{Vision Transformer} We use the Vision Transformer (ViT) \citep{dosovitskiy2021image} implementation from \citet{lucidrains2024vitpytorch}. Our ViT has depth 3, embedding dimension 64, number of heads 8, MLP dimension 256, and patch size 4.  We initialize the weights and biases of the final linear layer to 0.5 times the default, as this makes the curvature lower at initialization, which allows us to run \gd experiments at a broader range of learning rates.  For unknown reasons, we found that the core PyTorch LayerNorm implementation (written in C++) leads to third derivatives being computed incorrectly; thus, we substituted in an alternative implementation written in vanilla PyTorch, which empirically fixed the issue.

\paragraph{LSTM}  Our LSTM \citep{hochreiter1997long} has 2 layers, an embedding dimension of 48, and a hidden dimension of 48.

\paragraph{(Sequence) Transformer} Our sequence transformer has 4 layers, an embedding dimension of 32, an MLP dimension of 128, and 4 attention heads.  We disabled dropout, to make the network deterministic.  As with the ViT (see above), we initialize the weights and biases of the final linear layer to be zero, and we substitute a vanilla PyTorch LayerNorm implementation in place of the default C++ LayerNorm implementation.

\paragraph{Mamba} We use the Mamba \citep{gu2024mamba} implementation from  \citet{mamba_py}. Our Mamba has 2 layers and a model dimension of 64.  Unfortunately, the efficient Mamba kernel based on parallel scan did not work with PyTorch higher-order autotiff, so we needed to use the naive implementation of Mamba, which is slow.

\subsection{Dataset details}
\label{appendix:experimental-details:datasets}
Here we describe our datasets.  The code can be found at:\\
\url{http://github.com/centralflows/centralflows}.

\paragraph{CIFAR-10}  We test the vision architectures on a subset of CIFAR-10 that contains 1000 training examples, all from the first 4 CIFAR-10 classes.  We use the standard preprocessing of subtracting the dataset-wide channel-wise mean, and dividing by the dataset-wide channel-wise standard deviation.  When training using MSE loss, we encode the ground truth class as 1 and the others as 0.

\paragraph{Sorting} We test the sequence architectures on the synthetic sorting task described in \citet{karpathy2020mingpt}.  The network is fed a sequence of numbers and is then tasked (via a language modeling loss) with returning these numbers in sorted order. We used numbers 1 through 4, and sequences of length 8.  The size of the training datraset was usually 1,000 (except for Mamba, where it was 250).

%% file: appendix-2-miscellaneous.tex
\clearpage
\section{Miscellaneous}
\subsection{Implicit gradient regularization}
\label{sec:igr}
Recall from \Cref{sec:gd:deriving} that when \gd is stable, we approximate its trajectory by the gradient flow:
\begin{align}
     \frac{dw}{dt} &= - \eta \, \nabla L(w).  \label{eq:igr:gflow}
\end{align}
In particular, the central flow automatically reduces to \cref{eq:igr:gflow} whenever sharpness $S(w) < 2/\eta$.

On the other hand, \citet{barrett2021implicit} argued that \gd with step size $\eta$ should instead be approximated by a \emph{modified} gradient flow with a penalty on the squared gradient norm:
\begin{align}
    \frac{dw}{dt} &= - \eta \, \nabla \left[ L(w) + \frac{\eta}4{ \|\nabla L(w)\|^2} \right ], \label{eq:igr} \\
    &= - \eta \left[ \nabla L(w) + \frac{\eta}{2}  H(w) \nabla L(w) \right]. \label{eq:igr:igr-flow}
\end{align}
Subsequently, \citet{rosca2023continuous} showed how to improve the approximation by incorporating higher-order penalties, and \citet{cattaneo2023implicit} extended the approach to adaptive optimizers.

Empirically, we find that in the stable regime, the modified gradient flow \cref{eq:igr} is indeed a better approximation to \gd than the vanilla gradient flow \cref{eq:igr:gflow}.  This observation is illustrated in \Cref{fig:igr:igr-cnn-zoom}.  However, all things considered, we observe that in the stable regime, the vanilla gradient flow is \emph{already} a good enough approximation to \gd (\Cref{fig:igr:igr-cnn-zoom}). While \gd does differ from \gflow in deep learning, the vast majority of this difference appears to be due to the curvature-reduction effect of oscillations in the edge of stability regime, not to discretization error that manifests even in the stable regime.  This point is illustrated in \Cref{fig:igr:igr-cnn}.  Thus, in the interest of simplicity, we left out any implicit gradient regularizer from our central flows.

\begin{figure}[b!]
    \centering
    \includegraphics[width=0.8\linewidth]{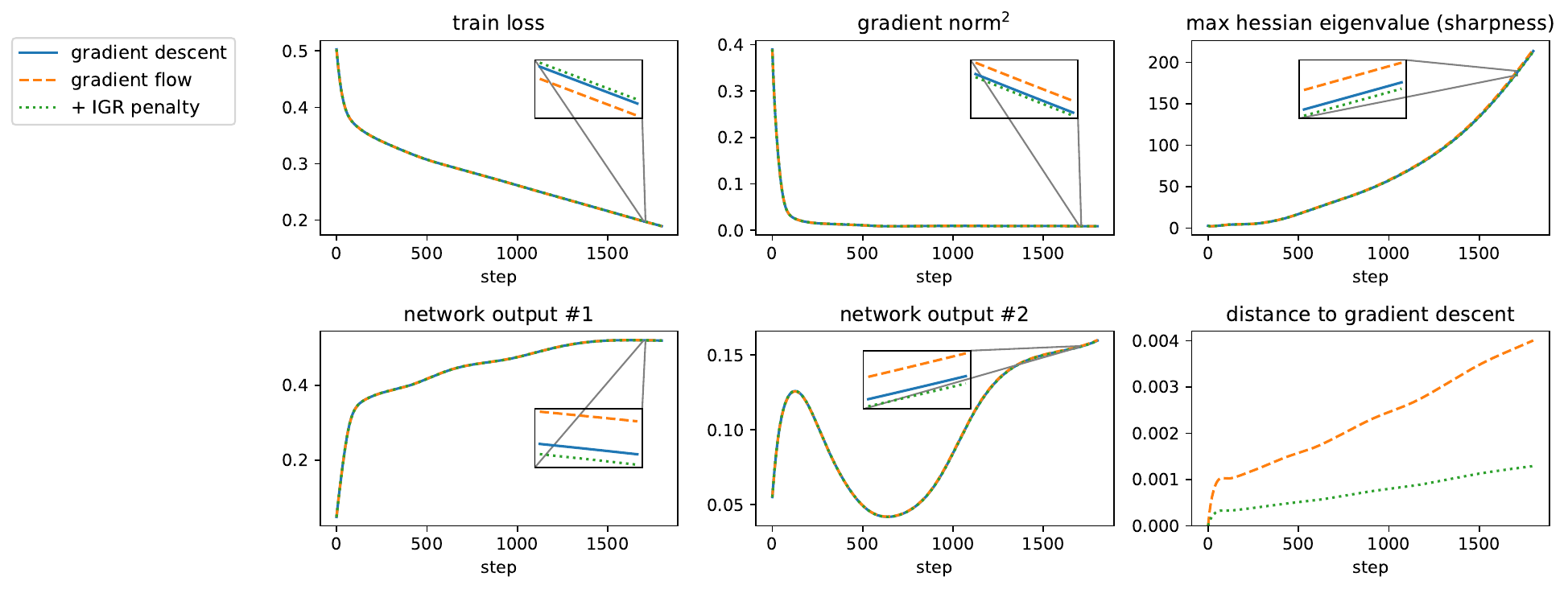}
    \caption{\textbf{In the stable regime, the IGR penalty marginally improves the accuracy of the gradient flow approximation, but this accuracy is already good.}  We train a network using gradient descent (blue), vanilla gradient flow \cref{eq:igr:gflow} (orange, dashed) and gradient flow with the IGR penalty \cref{eq:igr:igr-flow} (green, dotted), all with $\eta = 0.01$.  This figure shows the initial phase of training, when gradient descent is in the stable regime (i.e. sharpness is below $2/\eta)$.  Consistent with \citet{barrett2021implicit}, notice that in the stable regime, gradient flow + IGR is a visibly better approximation to \gd than vanilla gradient flow.  In particular, observe that the distance to the \gd trajectory is smaller for gradient flow + IGR than for vanilla gradient flow (bottom right).  Similarly, note that the network outputs on two examples, the train loss, and the squared gradient norm agree better under gradient flow + IGR than under vanilla gradient flow.  That said, notice that even the vanilla gradient flow is a good approximation to gradient descent in this regime.  Please refer to \Cref{fig:igr:igr-cnn} for the continuation of this experiment into the EOS regime.  \textit{Details}: a CNN is trained on a subset of CIFAR-10 using MSE loss.}
    \label{fig:igr:igr-cnn-zoom}
\end{figure}

\begin{figure}[t!]
    \centering
    \includegraphics[width=\linewidth]{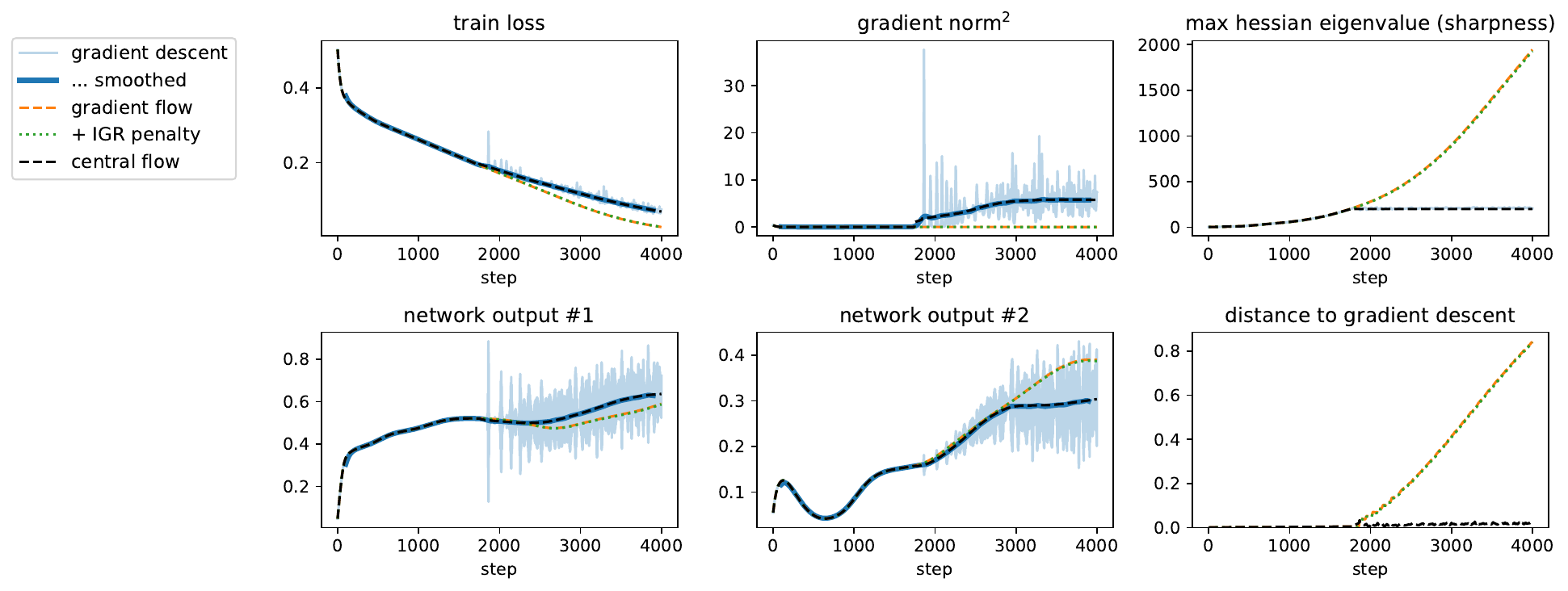}
    \caption{\textbf{In the EOS regime, the central flow accurately approximates the trajectory of gradient descent, whereas neither the original nor the IGR-penalized gradient flow does so.}  We continue the experiment from \Cref{fig:igr:igr-cnn-zoom} for more iterations, into the EOS regime.  Observe that neither the original gradient flow (orange, dashed) nor the IGR-penalized gradient flow (green, dotted) reasonably approximates the trajectory of gradient descent (blue) in this regime, whereas the central flow (black, dashed) does so.  In particular, notice that the distance from gradient descent to the central flow stays small, whereas the distance to both the original and IGR-penalized gradient flows grows large over time (bottom right).  Further, the central flow accurately predicts the time-averaged network outputs on two examples, as well as the train loss and squared gradient norm.}
    \label{fig:igr:igr-cnn}
\end{figure}

\paragraph{A subtle confounder}
Some works (e.g. \citet{geiping2022stochastic}) have observed that adding an \emph{explicit} squared gradient norm penalty can help full-batch training recover the superior generalization performance of minibatch training.  This would seem to support the argument of \citet{barrett2021implicit} that the implicit regularization of discrete gradient descent can be captured by a flow with a squared gradient norm penalty.  Yet, we believe that these results could be instead due to a subtle confounder: adding a squared gradient norm penalty changes the oscillatory EOS dynamics, and in particular, enhances the implicit curvature regularization.  

Consider running gradient descent with step size $\eta$, while adding an implicit gradient regularizer corresponding to some step size $\tau$.  The update rule is:
\begin{align}
    w_{t+1} &= w_t - \eta \left[\nabla L(w_t) + \tfrac{\tau}{2} H(w_t) \,\nabla L(w_t) \right]. \label{eq:igr:igr-gd}
\end{align}
On the one-dimensional quadratic function $L(w) = \tfrac{1}{2} S w^2$, the iterates would evolve according to:
\begin{align}
    w_{t+1} &= w_t - \eta \left[ S w_t + \tfrac{\tau}{2} S^2 w_t \right] \nonumber \\
    &= [1 - \eta S - \tfrac{1}{2} \eta\tau S^2] w_t.
\end{align}
Whereas vanilla \gd is unstable if $S > 2/\eta$, this iteration is unstable if $\eta S + \tfrac{1}{2} \eta \tau S^2 > 2 \iff S > \frac{\sqrt{1 + \tfrac{4 \tau}{\eta}} - 1}{\tau}$, which is $< \tfrac{2}{\eta}$. That is, it becomes unstable at \emph{lower} values of the sharpness $S$.  Accordingly, in line with the general EOS pattern, we find that on neural network objectives, while \gd implicitly constrains the sharpness to $2/\eta$, the update rule \cref{eq:igr:igr-gd} implicitly constrains the sharpness to this strictly smaller value (\Cref{fig:igr:confound}).  In other words, adding an explicit gradient norm penalty also results in stronger \emph{implicit} curvature regularization.  This acts as a subtle experimental confounder, which could explain the reports in the literature that explicitly penalizing the gradient norm substantially boosts generalization performance. 

Thus, we hypothesize that if the the IGR-penalized gradient flow \cref{eq:igr} were properly discretized, it would not yield improved generalization (as it would not substantially affect the trajectory).  Yet, if an IGR-penalized objective is optimized using a standard optimization algorithm, this could induce stronger implicit curvature regularization which substantially affects the trajectory and the generalization performance.

\paragraph{Momentum} Finally, we note that it is plausible that the IGR effect is negligible for vanilla \gd but relevant when momentum is used, as momentum amplifies the strength of the IGR effect \citep{ghosh2023implicit}.

\begin{figure}[t!]
    \centering
    \includegraphics[width=\linewidth]{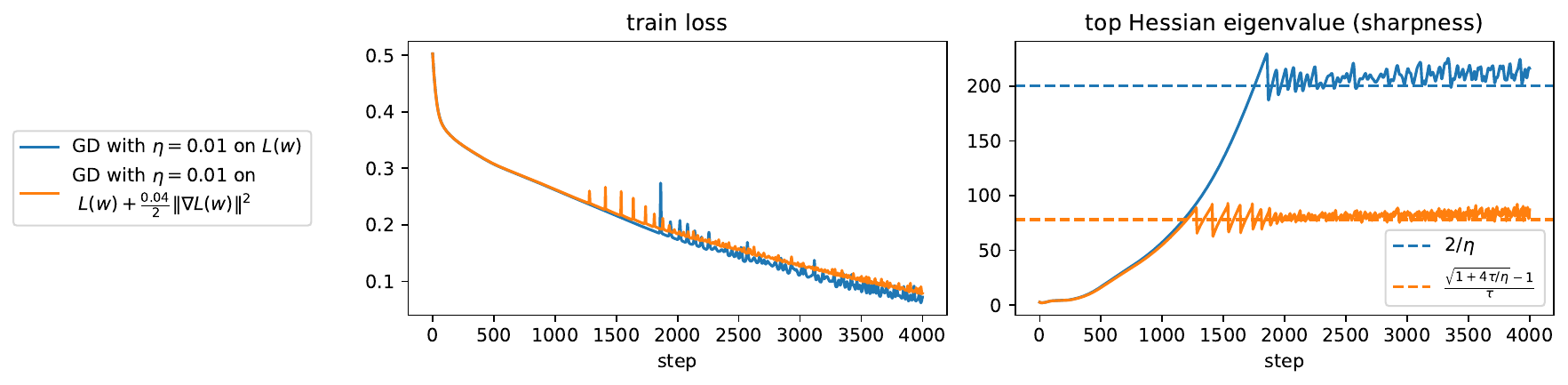}
    \caption{\textbf{A dangerous confounder: explicit \emph{gradient} regularization induces stronger implicit \emph{curvature} regularization.}  We train a CNN on a subset of CIFAR-10 using MSE loss.  In blue, we run gradient descent with step size $\eta = 0.01$ on the original objective $L(w)$.  In orange, we run gradient descent with step size $\eta=0.01$ on the implicitly regularized objective $L(w)  + \frac{\tau}{4} \|L(w) \|^2 $, with $\tau = 0.04$.  This mimics an attempt to capture the $\eta=0.04$ dynamics with an explicit gradient regularizer.  Observe that the explicit \emph{gradient} regularizer implicitly affects the \emph{curvature} dynamics, causing the sharpness to saturate at $\frac{\sqrt{1 + 4\tau/\eta}-1}{\tau} \approx 78.08$ instead of at $2/\eta= 200$. We believe that similar effects may be responsible for reports in the literature (e.g. \citet{geiping2022stochastic}) that explicit gradient norm regularization recovers the beneficial effects of large learning rates and small batch sizes.}
    \label{fig:igr:confound}
\end{figure}

\subsubsection{Implementation details}

To discretize the IGR-penalized gradient flow \cref{eq:igr:igr-flow}, we used a forward Euler scheme:
\begin{align}
    w^{(t+dt)} = w^{(t)} - \eta \, \epsilon \, \left[ \nabla L(w^{(t)}) + \tfrac{\eta}{2} H(w^{(t)}) \nabla L(w^{(t)}) \right]. \label{eq:igr:igr-euler}
\end{align}
We dynamically adapt the discretization step size $\epsilon$ based on the current sharpness (which we are already measuring).  On a quadratic function $L(w) = \tfrac{1}{2} Sw^2$ with sharpness $w$, the Euler method \cref{eq:igr:igr-euler} is convergent so long as:
\begin{align*}
    \epsilon \le \frac{2}{\eta S + \tfrac{1}{2} \eta^2 S^2}.
\end{align*}
To be on the safe side, and to try to avoid any implicit effects, we use a discretization step size of one-quarter that threshold. In particular, at every integer time $t$, we compute the sharpness $S(w)$, and set:
\begin{align*}
    m =  \lceil 2 \eta S + \eta^2 S^2 \rceil \quad \text{and} \quad \epsilon = 1/m,
\end{align*}
and we take $m$ Euler steps \cref{eq:igr:igr-euler} with discretization step size $\epsilon$.

\newpage
\subsection{Failure mode: higher-order terms}
\label{sec:higher-order}

Our theory models the objective using a local cubic Taylor approximation.
Sometimes, however, a cubic Taylor expansion is inadequate to capture the dynamics within the critical subspace, and this gives rise to a failure mode for the central flow, which was previously discussed in \citet[Appendix F]{damian2023selfstabilization}.

This failure mode is illustrated in \Cref{fig:higher-order:bad-sharpness}, which depicts a stretch of \gd where one eigenvalue is at the edge of stability and where the cyclic EOS dynamics have collapsed to a period-2 fixed point. (This makes for a simpler setting than the full cyclic dynamics, which helps us better illustrate the issue.)  Observe that the sharpness measured at the (second-order) midpoints between the \gd iterates is noticeably \emph{lower} than $2/\eta$, whereas the sharpness along the central flow is strictly \emph{equal} to $2/\eta$.
Further, observe that the actual squared displacement between \gd and the central flow is noticeably different from the central flow's prediction for this value, $\sigma^2(t)$.  This means that the central flow is applying the wrong strength of implicit sharpness regularization, which will cause error to accumulate over the long run.

These issues arise because the loss function along the top Hessian eigenvector is not well-modeled by its cubic Taylor expansion.
In \Cref{fig:higher-order:bad-curvature}, at various points during this stretch of training, we consider the line segment in between two successive iterates $\{\alpha w_{t} + (1 - \alpha) w_{t+1}: 0 \le \alpha \le 1 \}$, and we plot the curvature quantity $u^\top H(w) u$ along this line, where $u$ is the top Hessian eigenvector measured at the midpoint between the two iterates $\overline{w}$.
We also plot the first-order Taylor approximation of this curvature quantity, $ S(\overline{w}) + \nabla S(\overline{w})^\top (w - \overline{w})$, which arises from the local cubic Taylor approximation, as well as the second-order Taylor approximation of this curvature quantity, which arises from the local quartic Taylor approximation.
Observe that the curvature along this line segment is not well-modeled by its first-order Taylor approximation.
This is an indicator that the cubic Taylor approximation which we employ in our analysis is failing to hold within the local region that is being traversed via the oscillations.

\begin{figure}[b]
    \centering
    \includegraphics[width=0.6\linewidth]{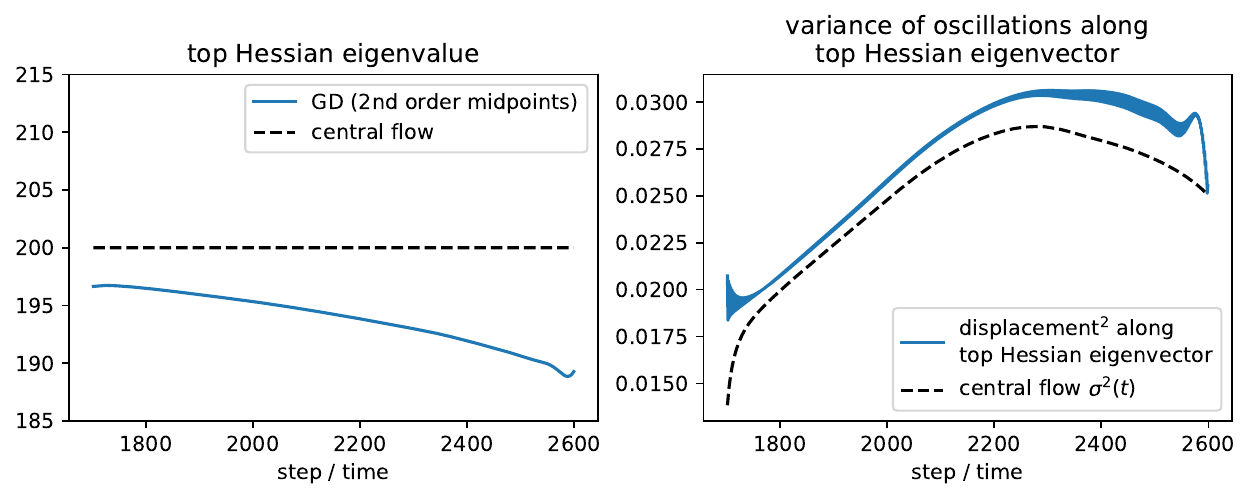}
    \caption{\textbf{Illustrating this failure mode for the central flow approximation.} This figure shows a segment of training which suffers from the failure mode discussed here.  Observe that the sharpness along the \gd trajectory (measured at the second-order midpoints) is different from the sharpness of the central flow, which is locked at $2/\eta$.  Further, the central flow poorly predicts the variance of the oscillations along the top Hessian eigenvector.  \textit{Details}: a CNN is trained on a two-class subset of CIFAR-10 with logistic loss.}
    \label{fig:higher-order:bad-sharpness}
\end{figure}

By contrast, \Cref{fig:higher-order:good-sharpness,fig:higher-order:good-curvature} depicts a different deep learning problem where the central flow approximation is more accurate, and where the local curvature \emph{is} well-described by the cubic Taylor expansion.

Please see \citet[Appendix F]{damian2023selfstabilization} for an extended discussion of this issue, in the special case of one unstable eigenvalue.
In this setting, the loss function can either be \emph{super}-quadratic along the top Hessian eigenvector, in which case the real curvature lies \emph{above} than its first-order Taylor approximation, and the curvature at the midpoint is \emph{less} than $2/\eta$; or it can be \emph{sub}-quadratic, in which case the real curvature lies \emph{below} its first-order Taylor approximation, and the curvature at the midpoint is \emph{greater} than $2/\eta$.
When multiple eigenvalues are unstable, we expect that the loss function could conceivably be subquadratic along some directions and superquadratic along others.

In the special case of one unstable eigenvalue, \citet{damian2023selfstabilization} derived a correction to their constrained trajectory (analogous to our central flow) which they empirically showed to match the real gradient descent trajectory even in the sub/super-quadratic setting.  Interestingly, with this correction, the implicit regularizer still takes the form $\sigma^2 \nabla S(w)$ for some $\sigma^2$; however, $\sigma^2$ cannot be determined solely from the local cubic Taylor approximation, and instead requires knowledge of the exact loss function along the top Hessian eigenvector direction.  It would be interesting to re-derive this correction under the central flow framework, and to extend it to the setting of multiple unstable eigenvalues.

\begin{figure}[h!]
    \centering
    \includegraphics[width=\linewidth]{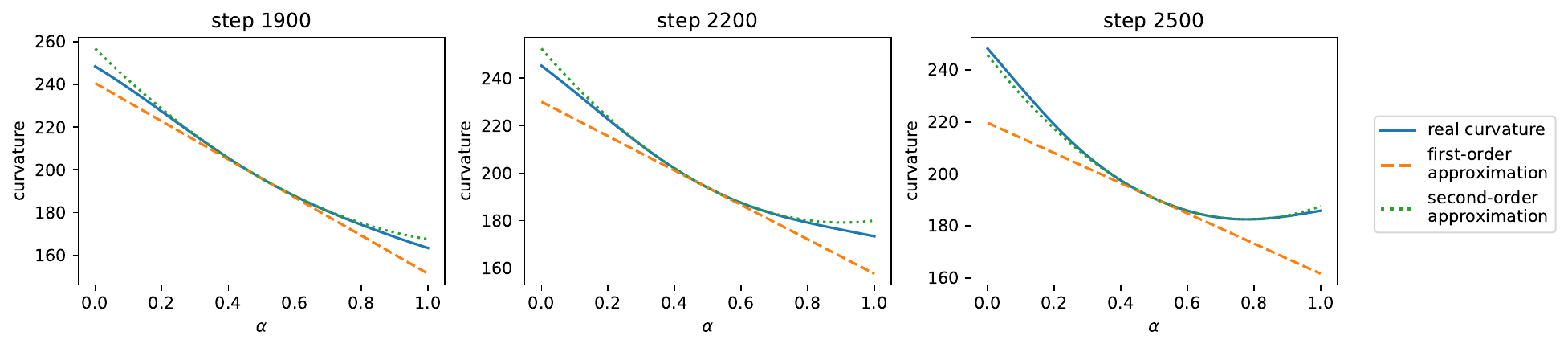}
    \caption{\textbf{When the central flow fails, the local cubic structure poorly predicts local curvature.} During the stretch of training depicted in \Cref{fig:higher-order:bad-sharpness}, at three different steps, we plot the curvature metric $u^\top H(w) u$ measured along the line segment between the current iterate and the next one: $\alpha w_{t} + (1-\alpha) w_{t+1}, \; 0 \le \alpha \le 1$, where $u$ denotes the top Hessian eigenvector measured at the midpoint between the two iterates.  Observe that this curvature metric (blue) is poorly predicted by its linear approximation around the midpoint (orange), which is based on the local cubic structure of the loss.}
    \label{fig:higher-order:bad-curvature}
\end{figure}

\begin{figure}[h!]
    \centering
    \includegraphics[width=0.6\linewidth]{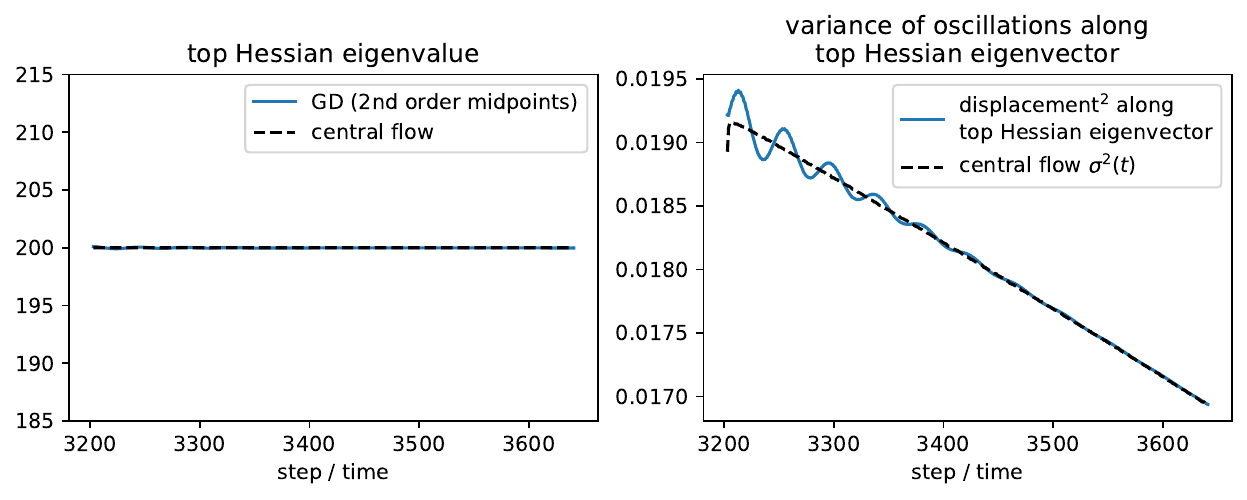}
    \caption{\textbf{A success case for the central flow.} This figure shows a segment of training which does \emph{not} suffer from this failure mode.  Observe that the sharpness along the \gd trajectory (measured at the second-order midpoints) is quite close to the sharpness along the central flow, which is locked at $2/\eta$.  Further, observe that the central flow accurately predicts the variance of the oscillations along the top Hessian eigenvector. \textit{Details}: a CNN is trained on a two-class subset of CIFAR-10 with MSE loss.}
    \label{fig:higher-order:good-sharpness}
\end{figure}

\begin{figure}[h!]
    \centering
    \includegraphics[width=0.9\linewidth]{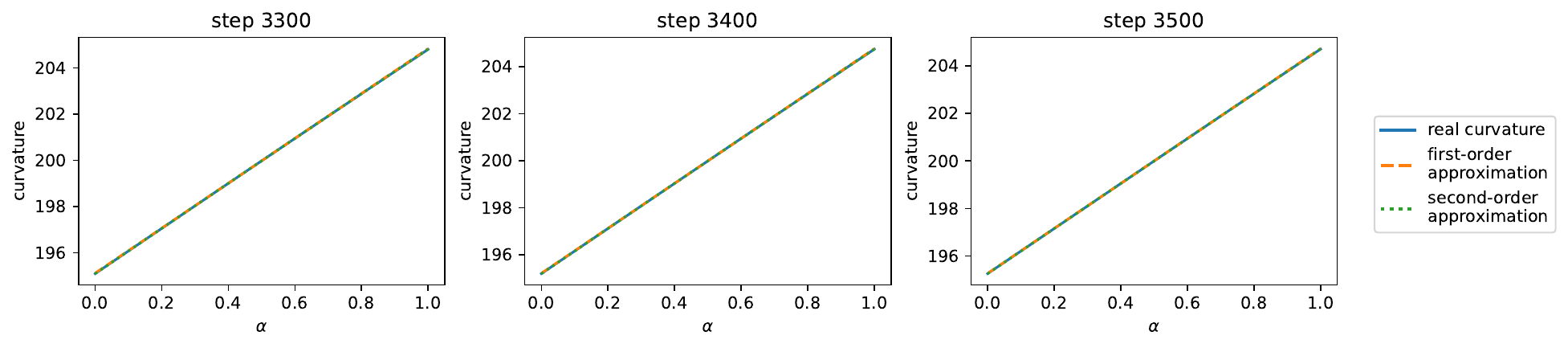}
    \caption{\textbf{When the central flow succeeds, the local cubic structure accurately predicts local curvature.} This figure shows the same curvature metric as \Cref{fig:higher-order:bad-curvature}, but in the ``success case'' setting of \Cref{fig:higher-order:good-sharpness}.  Observe that here, the local curvature is well-predicted by its local linearization.}
    \label{fig:higher-order:good-curvature}
\end{figure}

%% file: appendix-3-additional-figures.tex
\clearpage
\section{Supplementary Figures}
\label{sec:supplementary-figures}

\begin{letterfigures}
\begin{figure}[H]
    \centering
    \includegraphics[width=\linewidth]{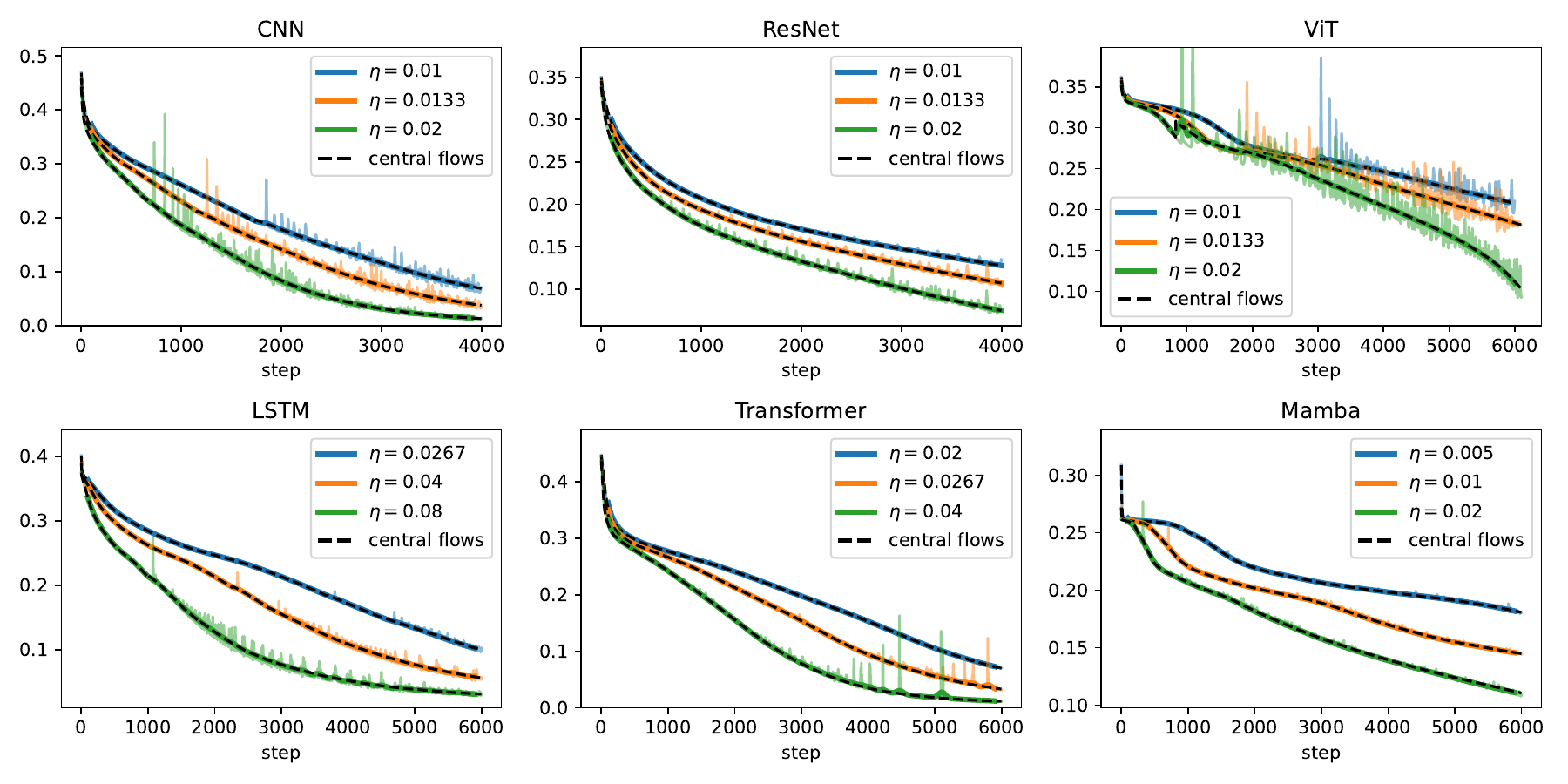}
    \caption{\textbf{Gradient descent: central flow predictions for loss curves (MSE).}  We compare the actual train loss curve (faint colors) and its Gaussian-smoothed version (thick colors) to the central flow's prediction \cref{eq:central-flow-predicted-loss} for the time-averaged loss curve (black dashed lines).
    Each subpanel is a different architecture, and each color is a different learning rate.  These plots use the mean squared error (MSE) loss; see \Cref{fig:experiments:gd:loss-curves-ce} for cross-entropy.}
    \label{fig:experiments:gd:loss-curves-mse}
\end{figure}

\begin{figure}[H]
    \centering
    \includegraphics[width=\linewidth]{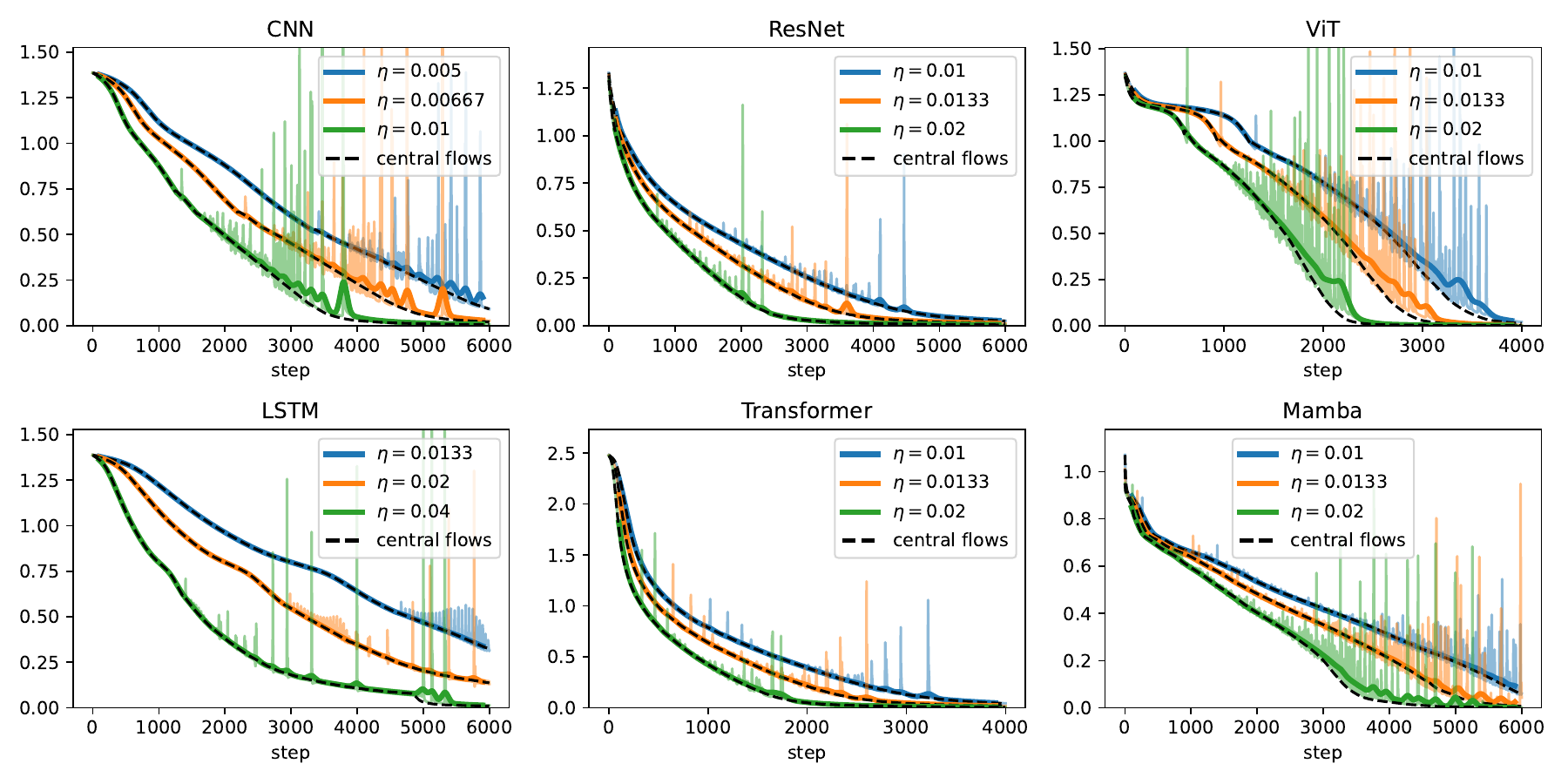}
    \caption{\textbf{Gradient descent: central flow predictions for loss curves (CE).}  Similar to \Cref{fig:experiments:gd:loss-curves-mse}, but for cross-entropy loss.  Here, the prediction is sometimes a bit off, especially at the end, and especially with larger learning rates; as discussed in \Cref{sec:experiments}, the central flow tends to be somewhat less accurate with cross-entropy loss.}
    \label{fig:experiments:gd:loss-curves-ce}
\end{figure}
\end{letterfigures}

\begin{letterfigures}
\begin{figure}[H]
    \centering
    \includegraphics[width=\linewidth]{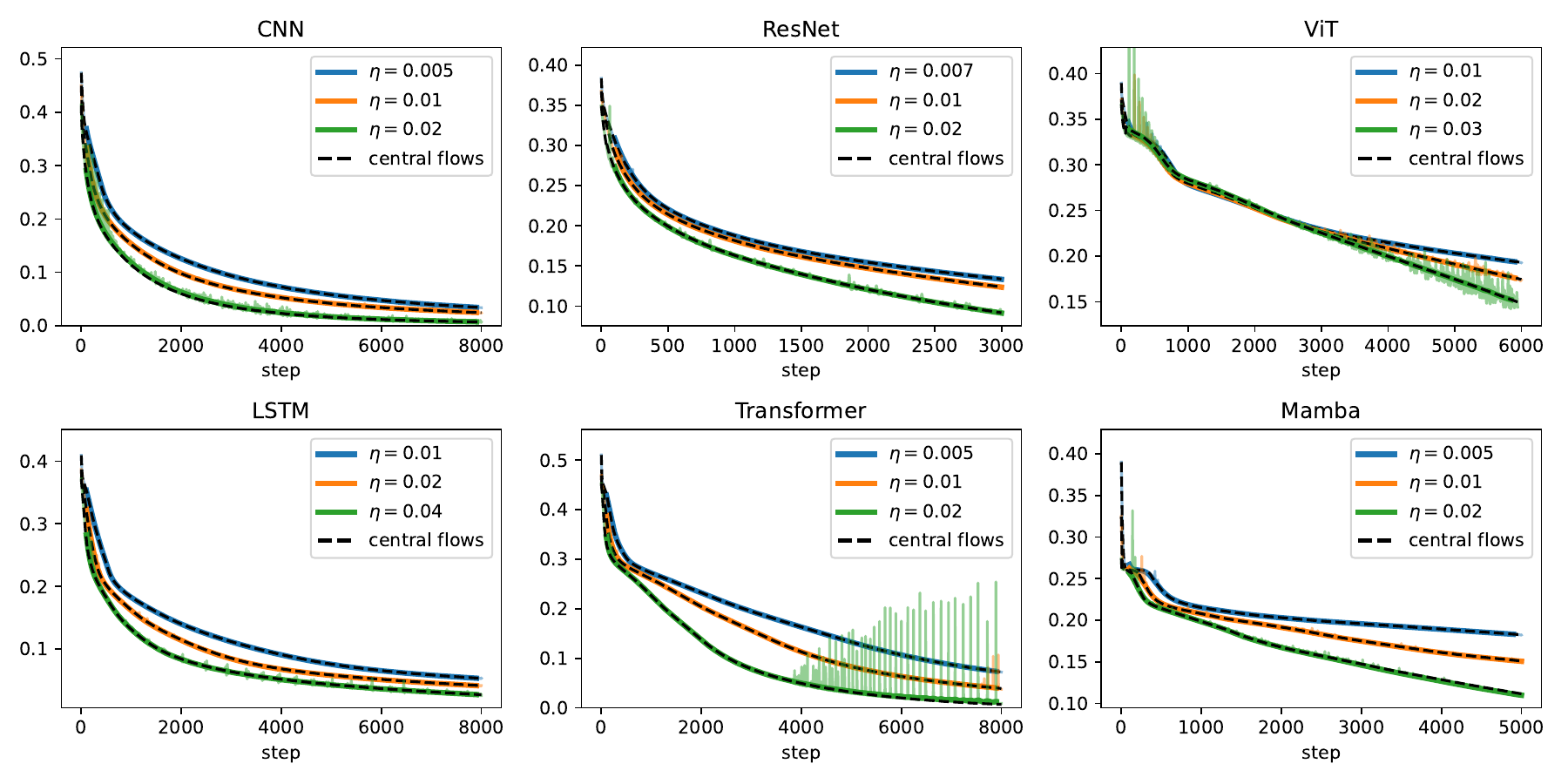}
    \caption{\textbf{Scalar RMSProp: central flow predictions for loss curves (MSE).}  We compare the actual train loss curve (faint colors) and its Gaussian-smoothed version (thick colors) to the central flow's prediction \cref{eq:rmsnorm-predict-loss} for the time-averaged loss curve (black dashed lines).
    Each subpanel is a different architecture, and each color is a different learning rate.  These plots use the mean squared error (MSE) loss; see \Cref{fig:experiments:scalar-rmsprop:loss-curves-ce} for cross-entropy.}
    \label{fig:experiments:scalar-rmsprop:loss-curves-mse}
\end{figure}

\begin{figure}[H]
    \centering
    \includegraphics[width=\linewidth]{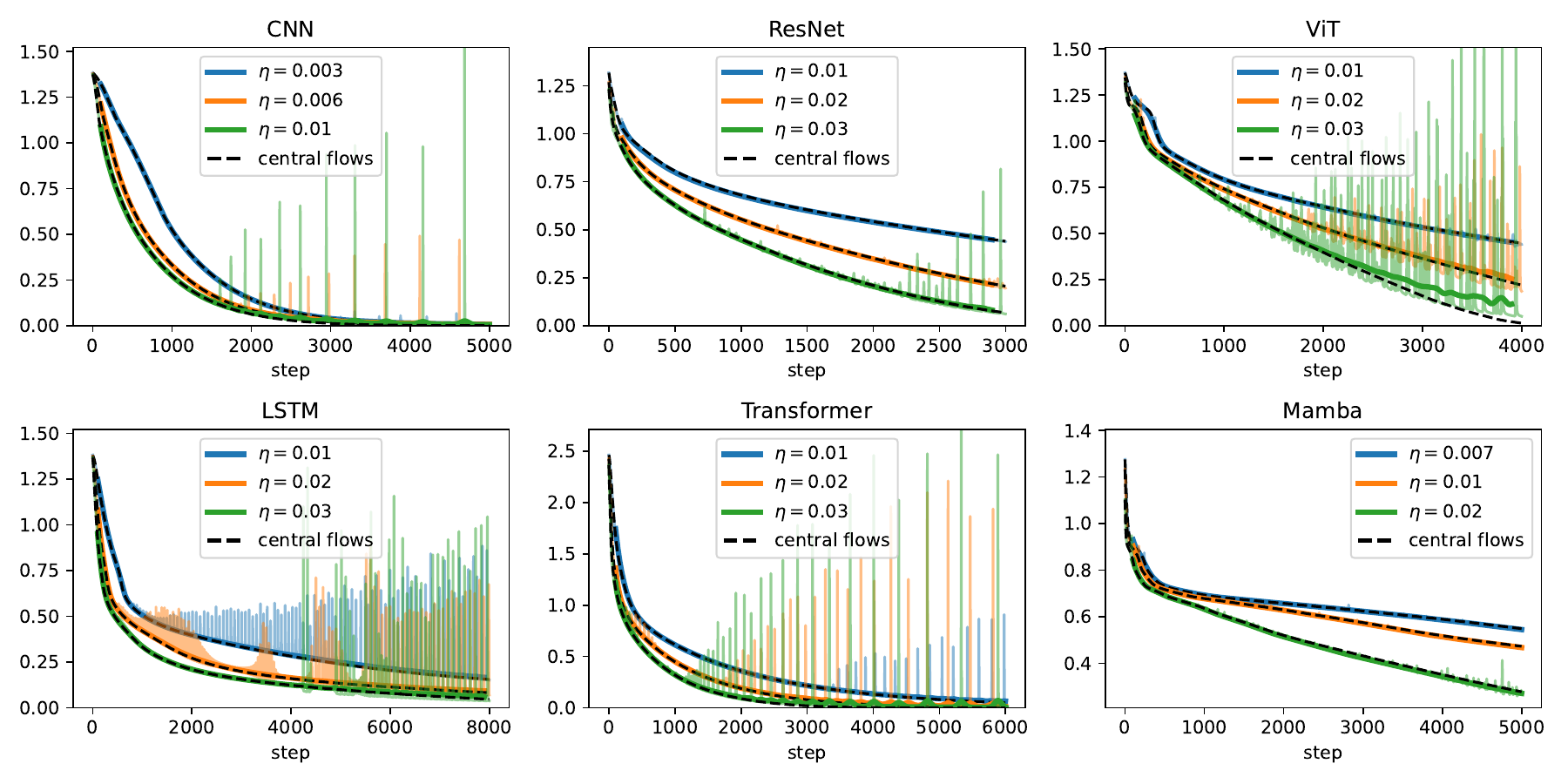}
    \caption{\textbf{Scalar RMSProp: central flow predictions for loss curves (CE).}  Similar to \Cref{fig:experiments:scalar-rmsprop:loss-curves-mse} but for cross entropy loss.}
    \label{fig:experiments:scalar-rmsprop:loss-curves-ce}
\end{figure}
\end{letterfigures}

\begin{letterfigures}
\begin{figure}[H]
    \centering
    \includegraphics[width=\linewidth]{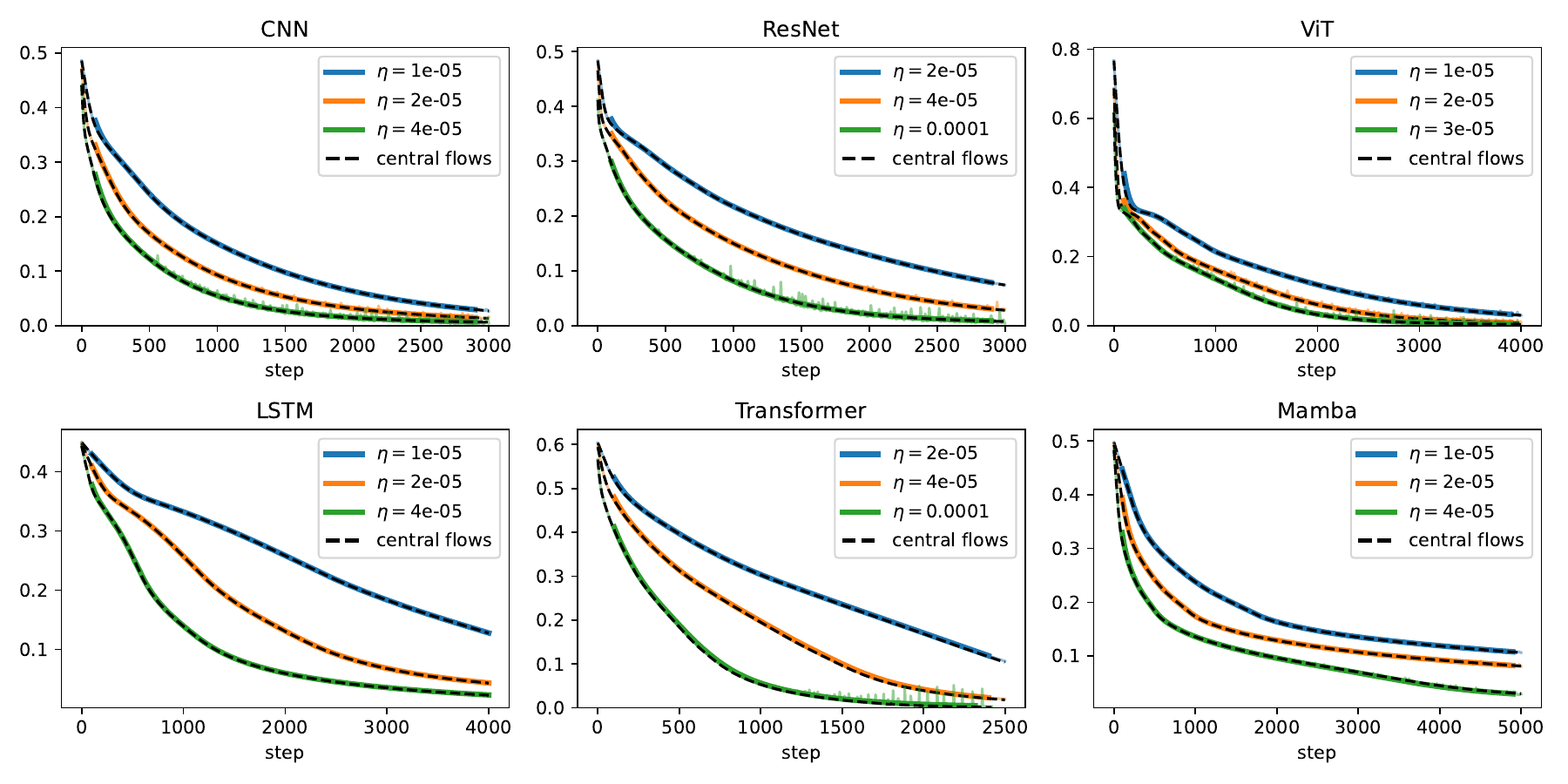}
    \caption{\textbf{RMSProp: central flow predictions for loss curves (MSE).}  We compare the actual train loss curve (faint colors) and its Gaussian-smoothed version (thick colors) to the central flow's prediction \cref{eq:rmsprop-predict-loss} for the time-averaged loss curve (black dashed lines).
    Each subpanel is a different architecture, and each color is a different learning rate.  These plots use the mean squared error (MSE) loss; see \Cref{fig:experiments:scalar-rmsprop:loss-curves-ce} for cross-entropy.}
    \label{fig:experiments:rmsprop:loss-curves-mse}
\end{figure}

\begin{figure}[H]
    \centering
    \includegraphics[width=\linewidth]{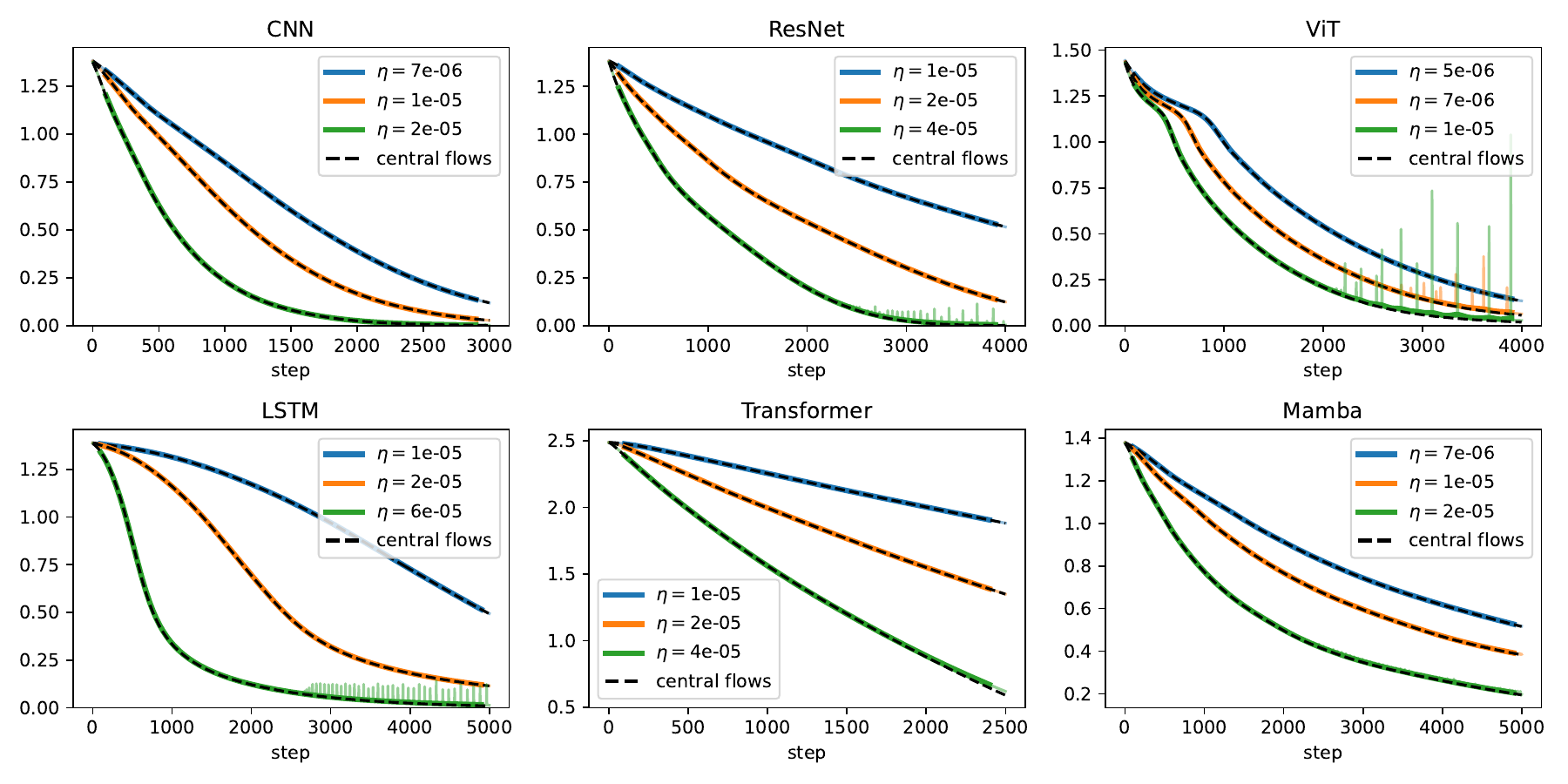}
    \caption{\textbf{RMSProp: central flow predictions for loss curves (CE).} Similar to \Cref{fig:experiments:rmsprop:loss-curves-mse} but for cross-entropy loss.}
    \label{fig:experiments:rmsprop:loss-curves-ce}
\end{figure}
\end{letterfigures}

\begin{figure}[H]
    \centering
    \includegraphics[width=\linewidth]{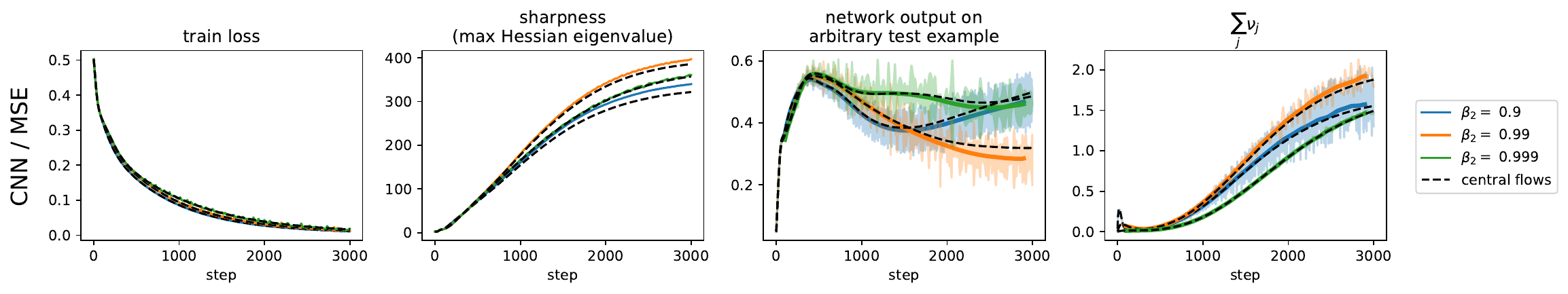}
    \includegraphics[width=\linewidth]{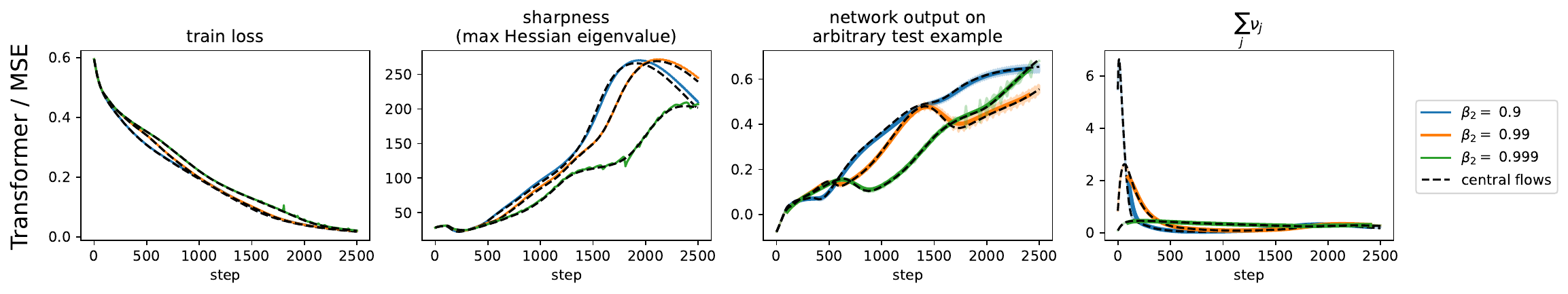}
    \caption{\textbf{Validating the RMSProp central flow across $\beta_2$ values}. We run both RMSProp and its central flow at multiple values of the $\beta_2$ hyperparameter (colors).  Observe that the central flows (black) well-approximate the trajectories of RMSProp.  \textit{Details}: the top row is a CNN trained on a subset of CIFAR-10 with MSE loss, $\eta = $ 2e-5, $\epsilon = $ 1e-8, and bias correction.   The bottom row is a Transformer trained on a synthetic sequence prediction task with MSE loss, $\eta =$ 4e-5, $\epsilon =$ 1e-8, and bias correction.}
    \label{fig:experiments:rmsprop:beta2}
\end{figure}

\begin{figure}[t]
    \centering
    \includegraphics[width=\linewidth]{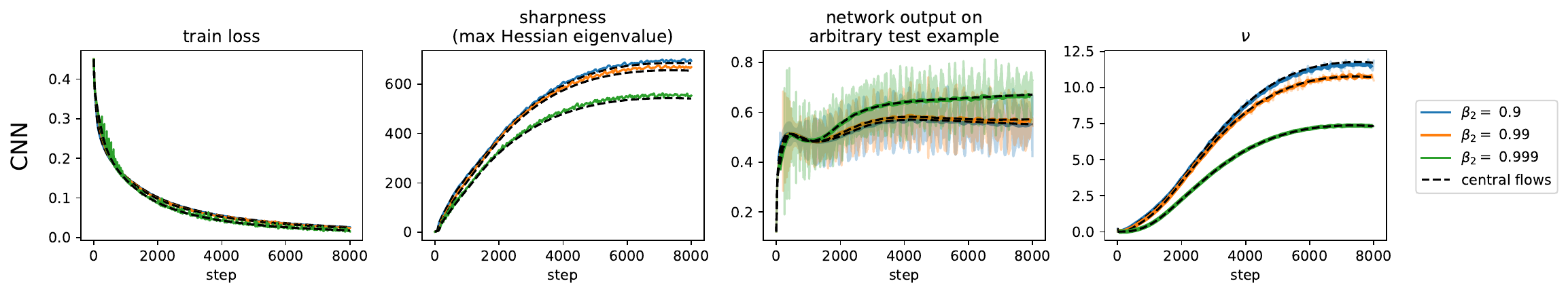}
    \caption{\textbf{Validating the Scalar RMSProp central flow across $\beta_2$ values}. We run both Scalar RMSProp and its central flow at multiple values of the $\beta_2$ hyperparameter (colors).  Observe that the central flows (black) well-approximate the trajectory of Scalar RMSProp.  \textit{Details}: a CNN is trained on a subset of CIFAR-10 with MSE loss, $\eta = 0.01$, and bias correction. }
    \label{fig:experiments:scalar-rmsprop:beta2}
\end{figure}

\begin{figure}[H]
    \centering
    \includegraphics[width=\linewidth]{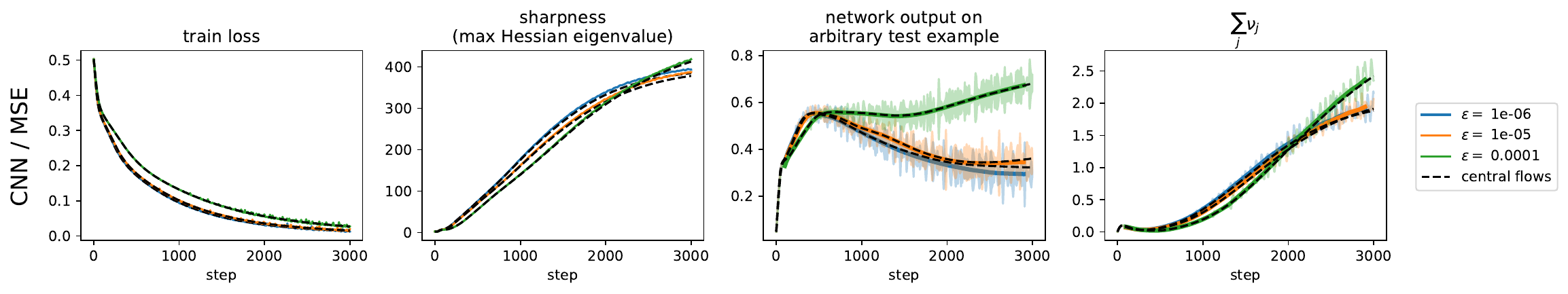}
    \includegraphics[width=\linewidth]{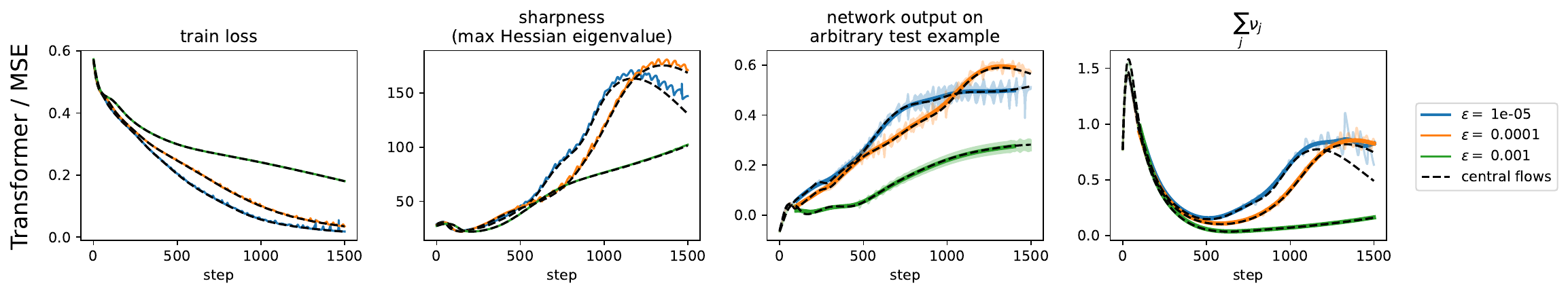}
    \caption{\textbf{Validating the RMSProp central flow across $\epsilon$ values}. We run both RMSProp and its central flow at multiple values of the $\epsilon$ hyperparameter (colors).  Observe that the central flows (black) well-approximate the trajectory of RMSProp.  \textit{Details}: the top row is a CNN trained on a subset of CIFAR-10 with MSE loss, $\eta = $ 2e-5, $\epsilon = $ 1e-8, and bias correction.  The bottom row is a Transformer trained on a synthetic sequence prediction task with MSE loss, $\eta =$ 4e-5, $\epsilon = $1e-8, and bias correction. }
    \label{fig:experiments:rmsprop:eps}
\end{figure}

\begin{figure}[h]
\centering
\includegraphics[width=0.4\textwidth]{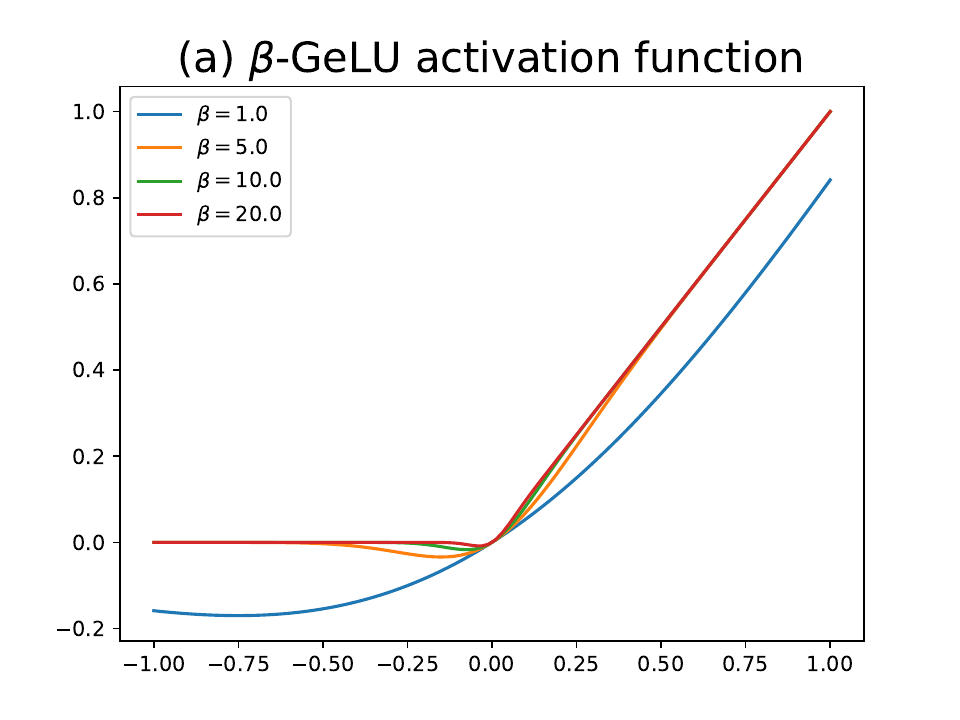}
\includegraphics[width=0.4\textwidth]{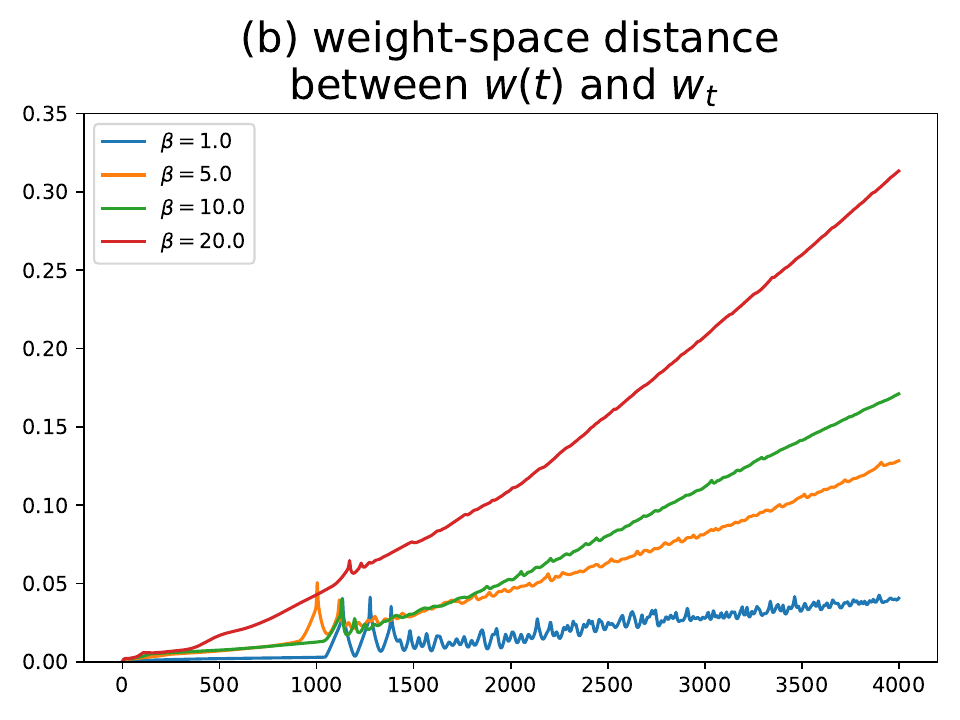}
\includegraphics[width=\textwidth]{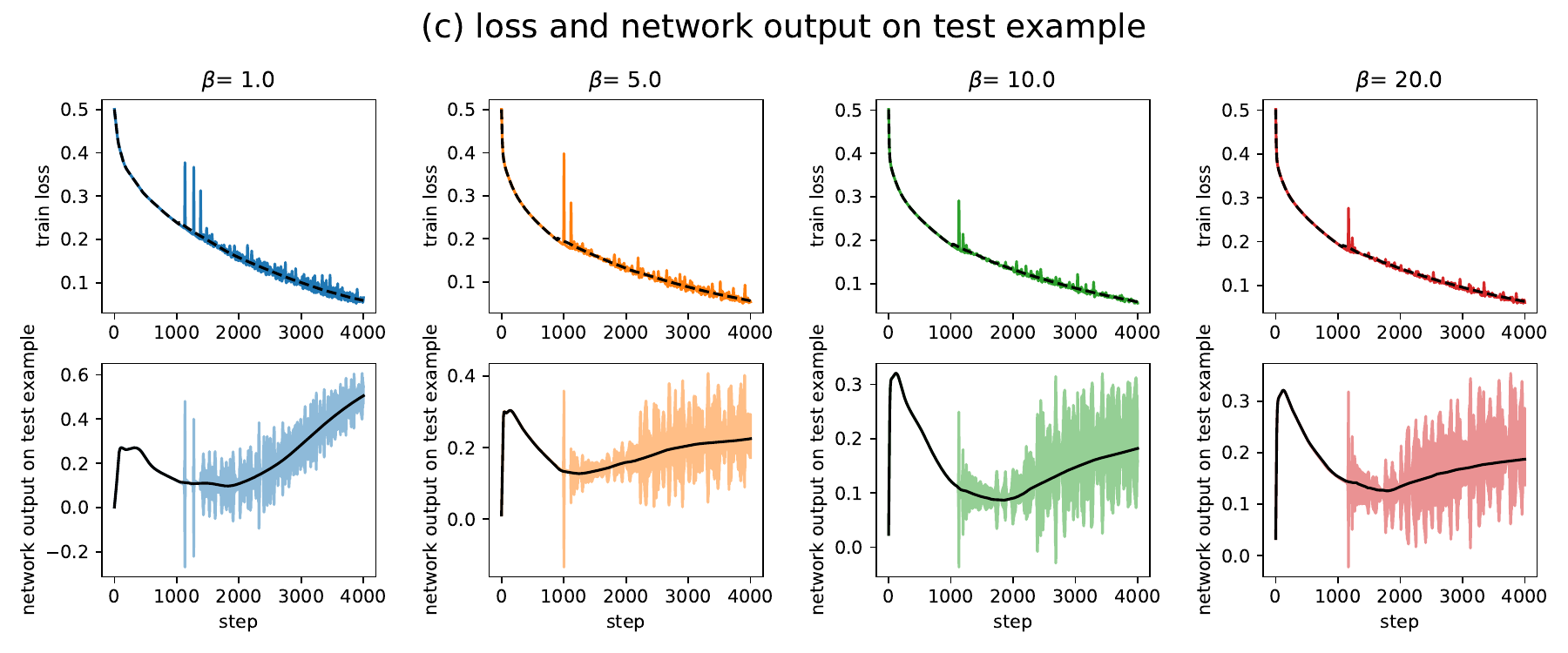}
\caption{\textbf{Accuracy of central flow degrades as activation function becomes less smooth.}  We consider networks with the $\beta$-GeLU activation function from \citet{dauphin2024neglected}, defined as $x \mapsto x\Phi(\beta x)$ where $\Phi$ is the standard Gaussian CDF.  This activation interpolates between (smooth) GeLU when $\beta=1$ and (non-smooth) ReLU when $\beta = \infty$.  Subfigure (a) plots this activation function with varying $\beta$.  Subfigure (b) shows that when $\beta$ is larger (i.e. when the activation is less smooth), the approximation error between the central flow $w(t)$ and the optimizer trajectory $w_t$ grows faster.  Subfigure (c) plots the loss curve, and the network's output on a test example, for both the optimizer trajectory and the central flow. Fortunately, even when $\beta = 20$, at which point $\beta$-GeLU is a very close approximation to ReLU, the central flow accurately predicts the overall training loss curve.}
\label{fig:beta-gelu-distances}
\end{figure}

\begin{figure}[H]
\centering
\includegraphics[width=\textwidth]{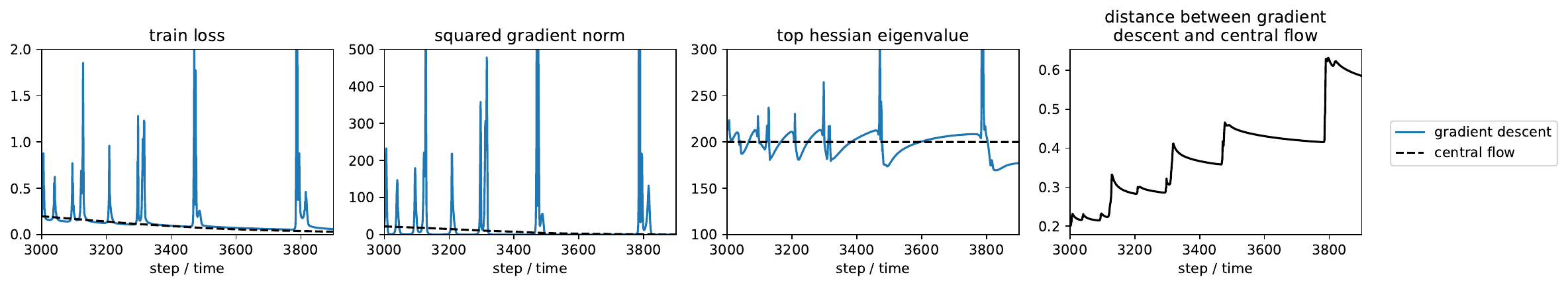}
\caption{\textbf{Large spikes degrade the accuracy of the central flow approximation}. Every few hundred iterations, there is a large spike (visible in e.g. the loss and the gradient norm), which causes the distance between gradient descent and the central flow to jump. For reasons we do not understand, such large spikes are relatively common during full-batch training with cross-entropy loss, especially near the end of training.  \textit{Details}: a CNN is trained on a subset of CIFAR-10 with cross-entropy loss, using gradient descent with $\eta = 0.01$.}
\label{fig:spikes-distance}
\end{figure}

\clearpage
\begin{figure}[H]
\centering
\includegraphics[width=0.8\linewidth]{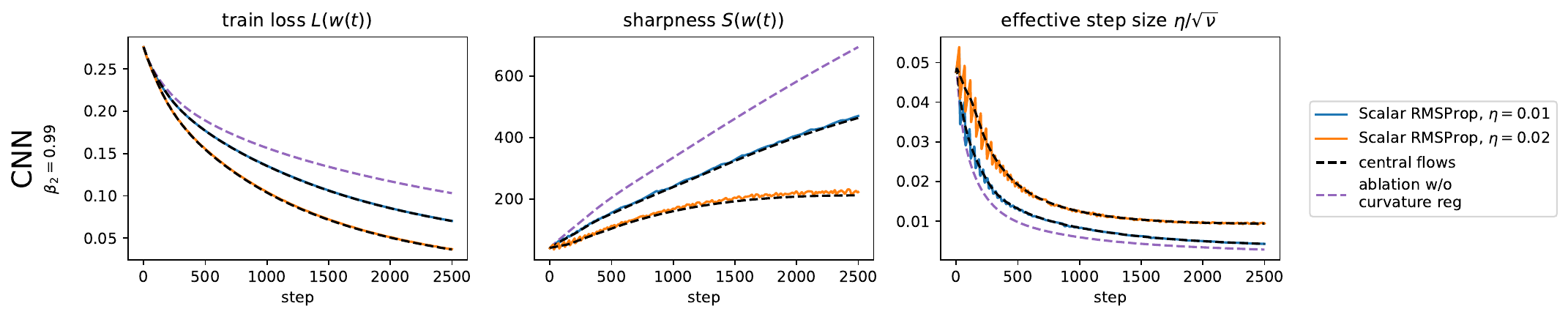}
\includegraphics[width=0.8\linewidth]{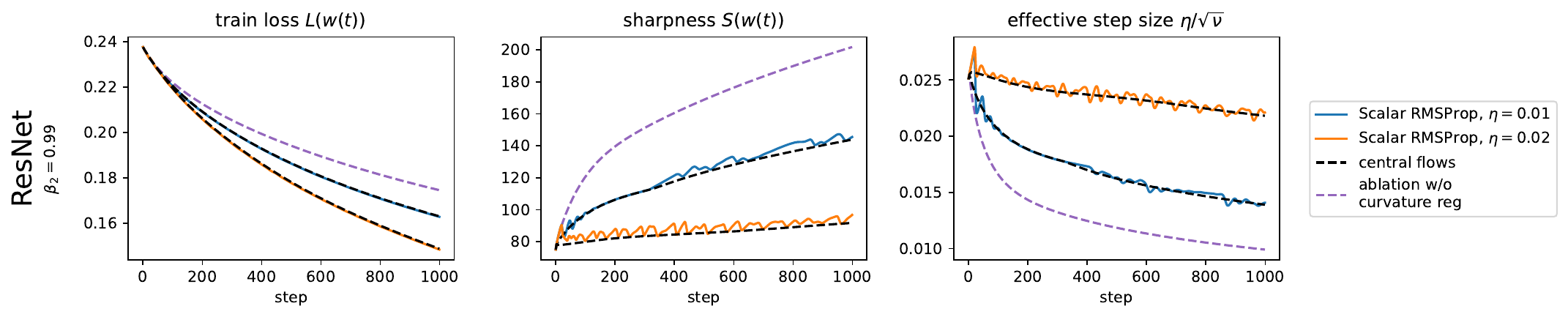}
\includegraphics[width=0.8\linewidth]{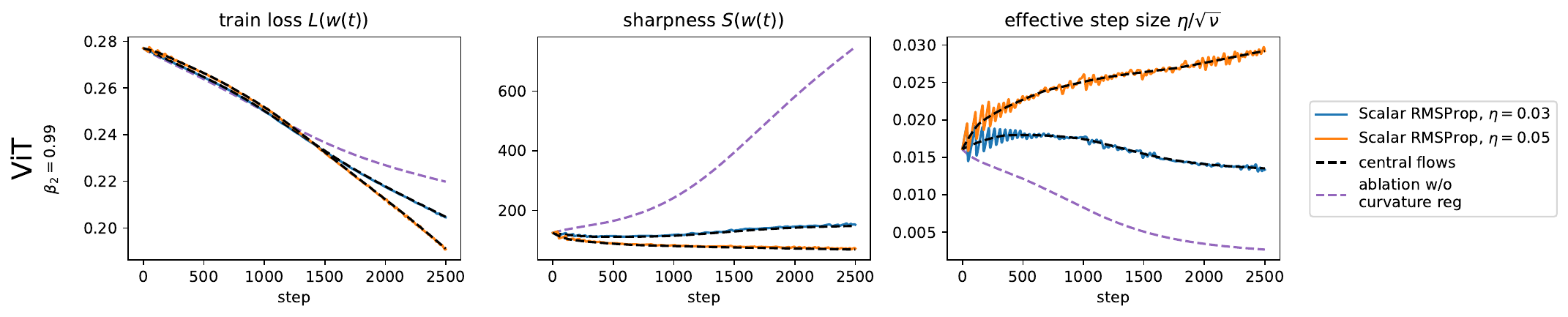}
\includegraphics[width=0.8\linewidth]{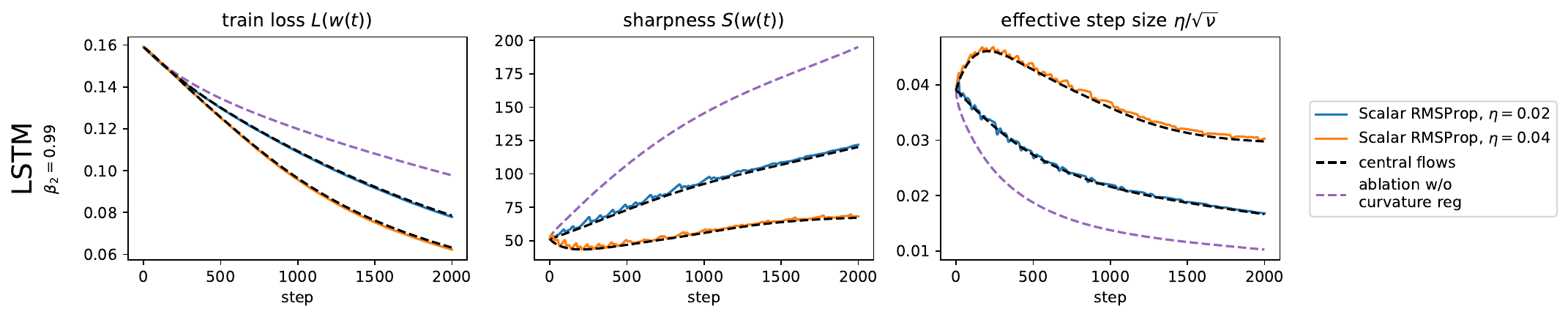}
\includegraphics[width=0.8\linewidth]{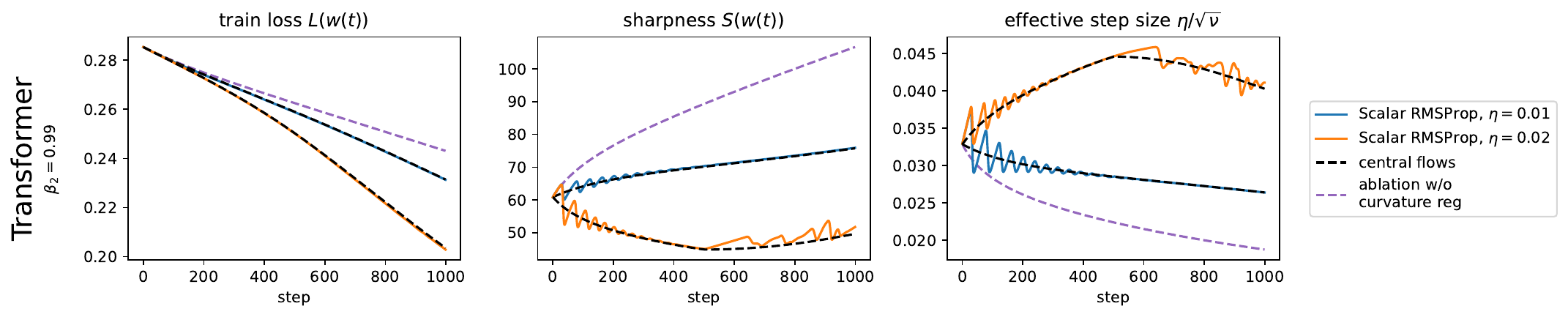}
\includegraphics[width=0.8\linewidth]{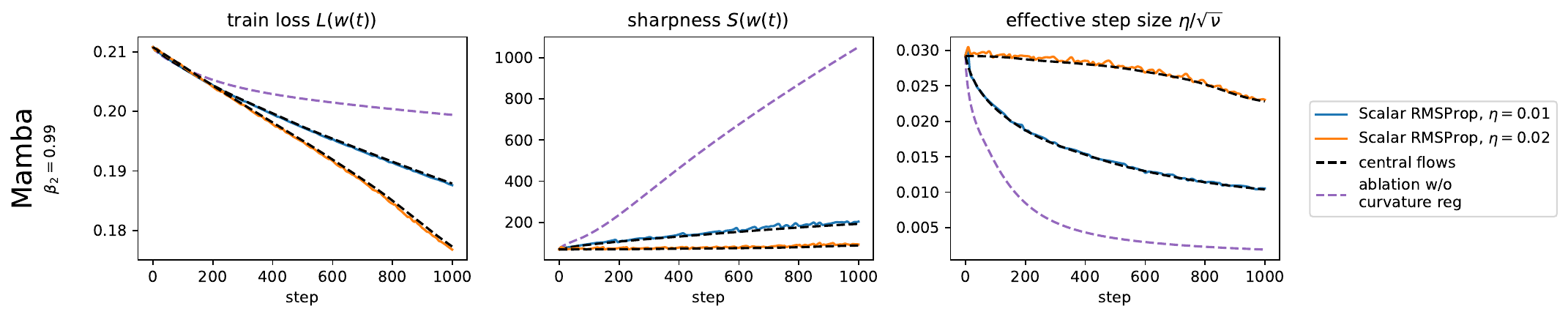}
\caption{\textbf{Implicit curvature reduction accelerates optimization for Scalar RMSProp}.  Starting from the same initialization, we run the Scalar RMSProp central flow at various learning rates, as well as an ablated flow $\frac{dw}{dt} = - \tfrac{2}{S(w)} \nabla L(w)$ with curvature regularization removed. These three flows all use the same step size strategy but differ in the strength of implicit curvature regularization.  Initially, the flows with higher curvature regularization often optimize slower; however, over the longer run, they are able to take larger steps and optimize faster.  Each row depicts a different deep learning setting.  All experiments use MSE loss.\protect\footnotemark}
\label{fig:rmsnorm-acc-by-reg}
\end{figure}
\footnotetext{For discrete \rmsnorm, we show the train loss not at the iterates themselves but at the second-order midpoints between iterates $\hat{w}_t := \tfrac{1}{4}[2w_t + w_{t-1} + w_{t+1}]$, which removes most of the oscillations along the top eigenvectors. This makes the train loss directly comparable to the flow losses.}

\newpage
\begin{figure}[H]
\vspace{-10px}
\centering
\includegraphics[width=0.85\linewidth]{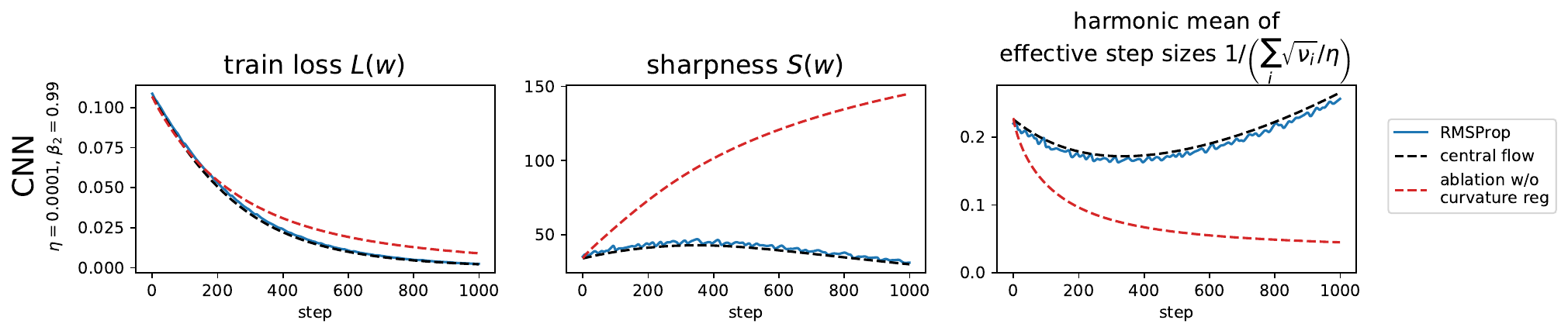}
\includegraphics[width=0.85\linewidth]{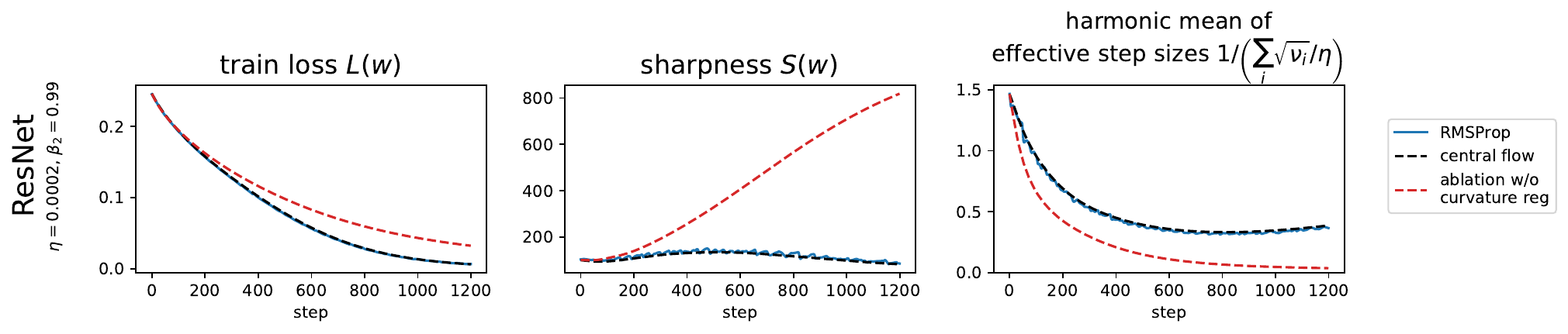}
\includegraphics[width=0.85\linewidth]{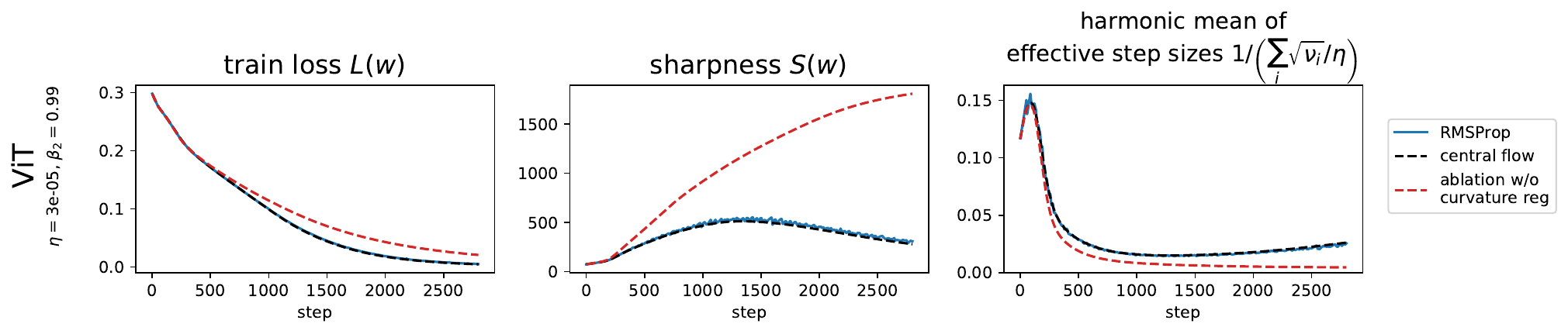}
\includegraphics[width=0.85\linewidth]{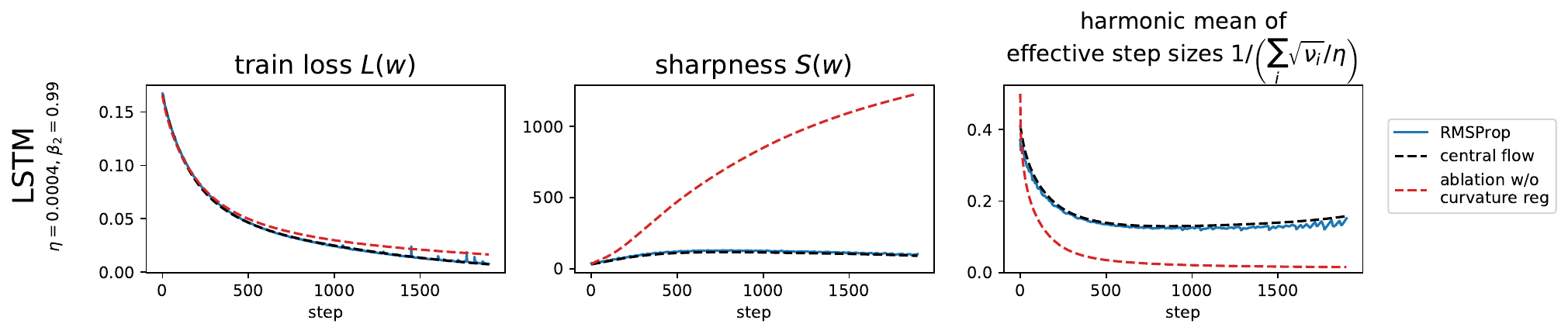}
\includegraphics[width=0.85\linewidth]{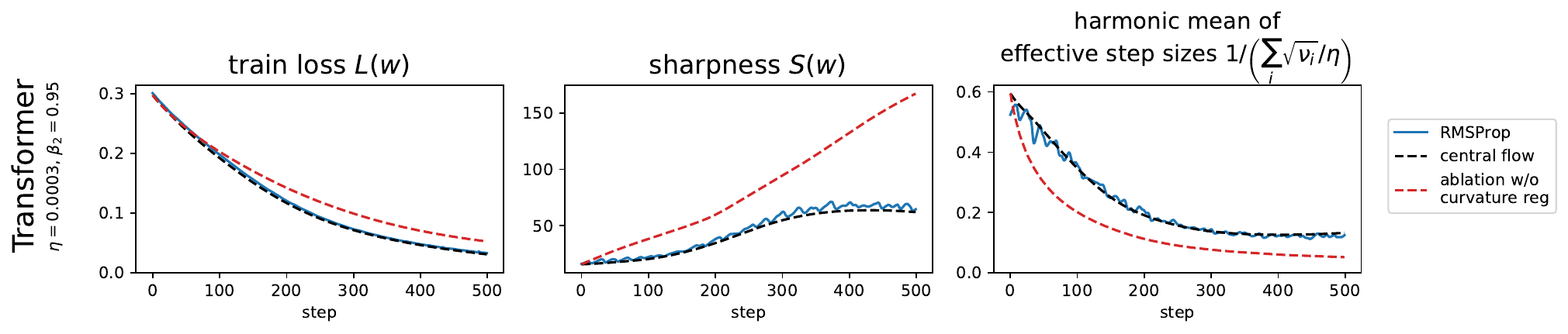}
\includegraphics[width=0.85\linewidth]{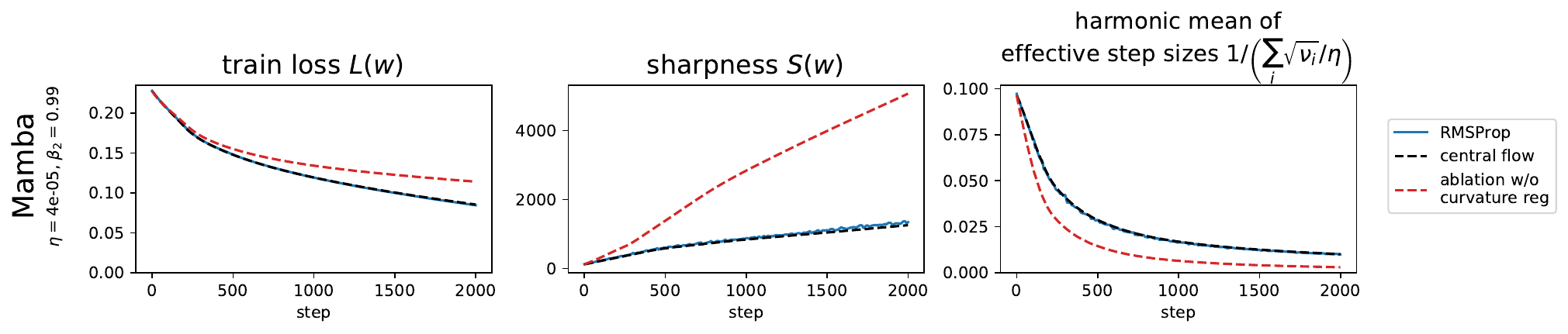}
\caption{\textbf{Implicit curvature reduction accelerates optimization for RMSProp}.  We compare RMSProp (blue) and its central flow (black) to an ablated flow (red) which leaves out the implicit curvature regularization, and maintains stability purely by the effect of oscillations on $\nu$. Over time, RMSProp and the central flow navigate to lower-curvature regions (center), where they take larger steps (right), and optimize faster (left) than the ablated flow. Each row is a different DL setting.  The left column plots the train loss,\protect\footnotemark the middle column plots the sharpness, and the right column plots the harmonic mean of the effective learning rates. These experiments all use MSE loss.}
\vspace{-15px}
\label{fig:rmsprop-acc-by-reg}
\end{figure}
\footnotetext{For discrete \rmsprop, we show the train loss not at the iterates themselves but at the second-order midpoints between iterates $\hat{w}_t := \tfrac{1}{4}[2w_t + w_{t-1} + w_{t+1}]$, which removes most of the oscillations along the top eigenvectors. This makes the train loss directly comparable to the flow losses.}

\newpage
\begin{letterfigures}
\begin{figure}[H]
\centering
\includegraphics[width=0.8\linewidth]{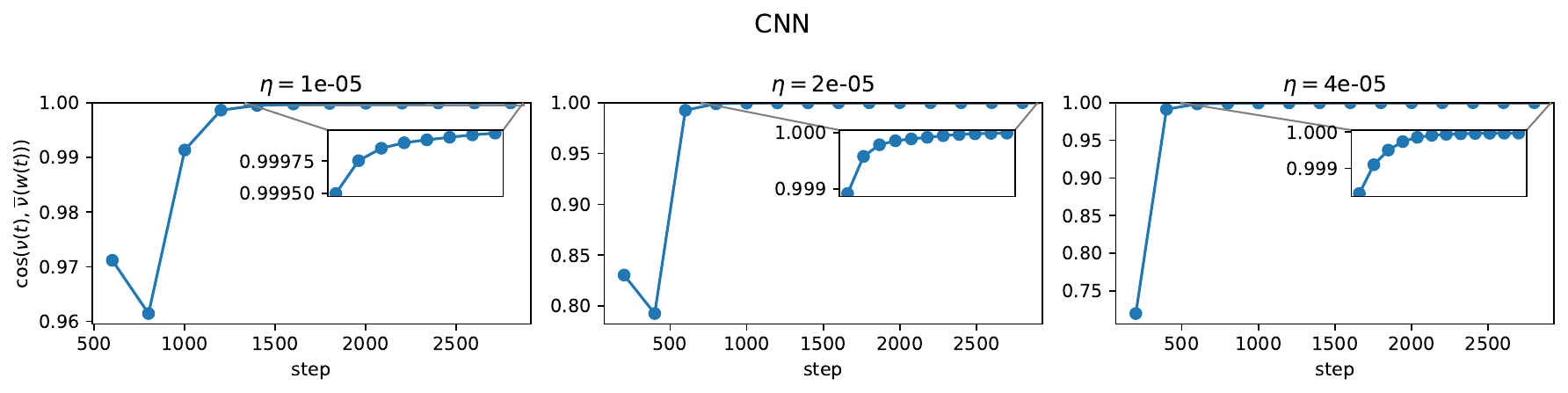}
\includegraphics[width=0.8\linewidth]{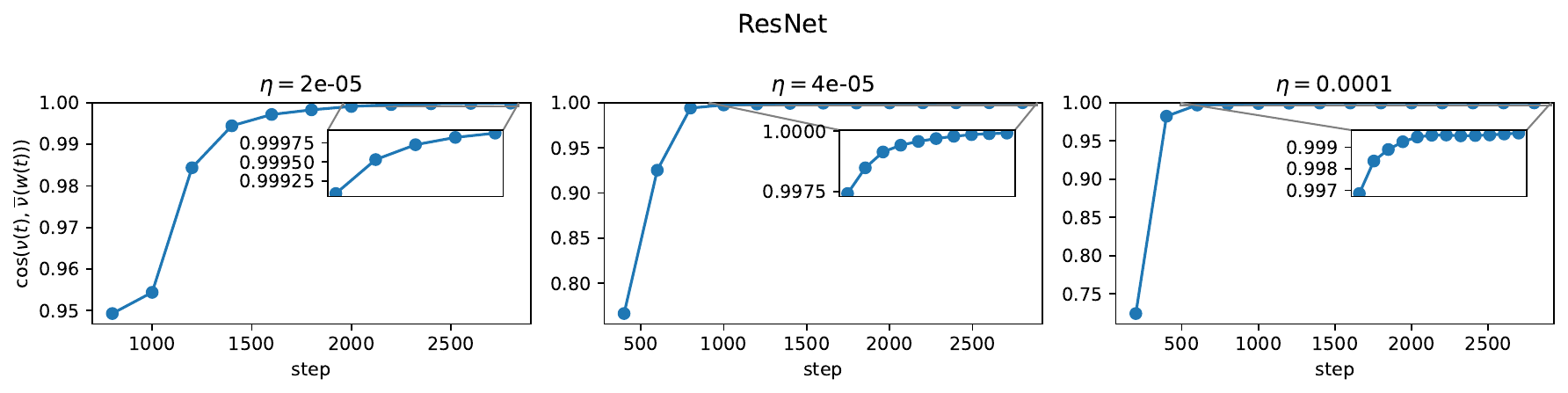}
\includegraphics[width=0.8\linewidth]{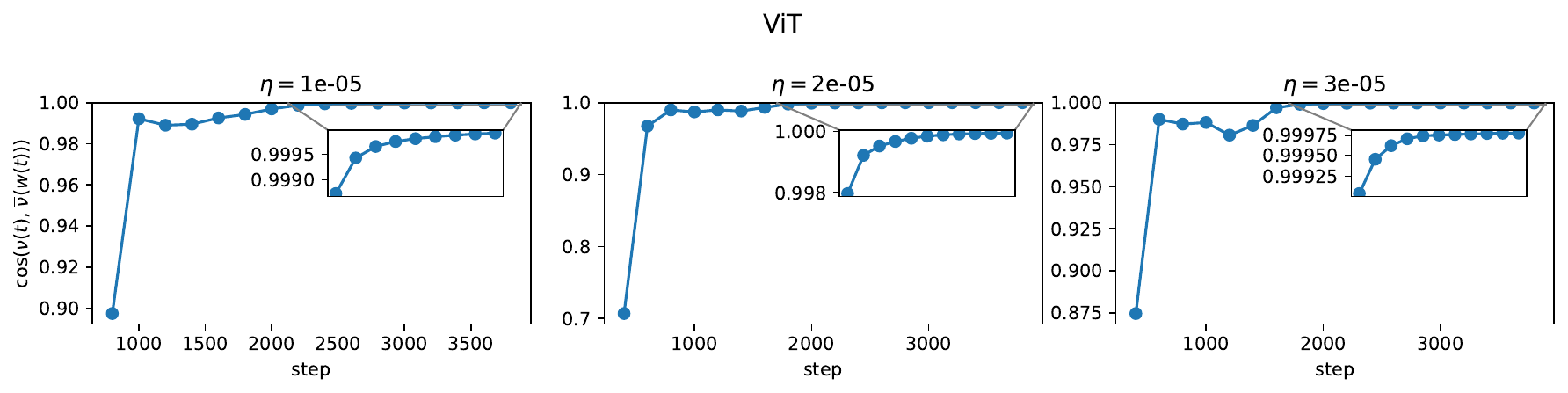}
\includegraphics[width=0.8\linewidth]{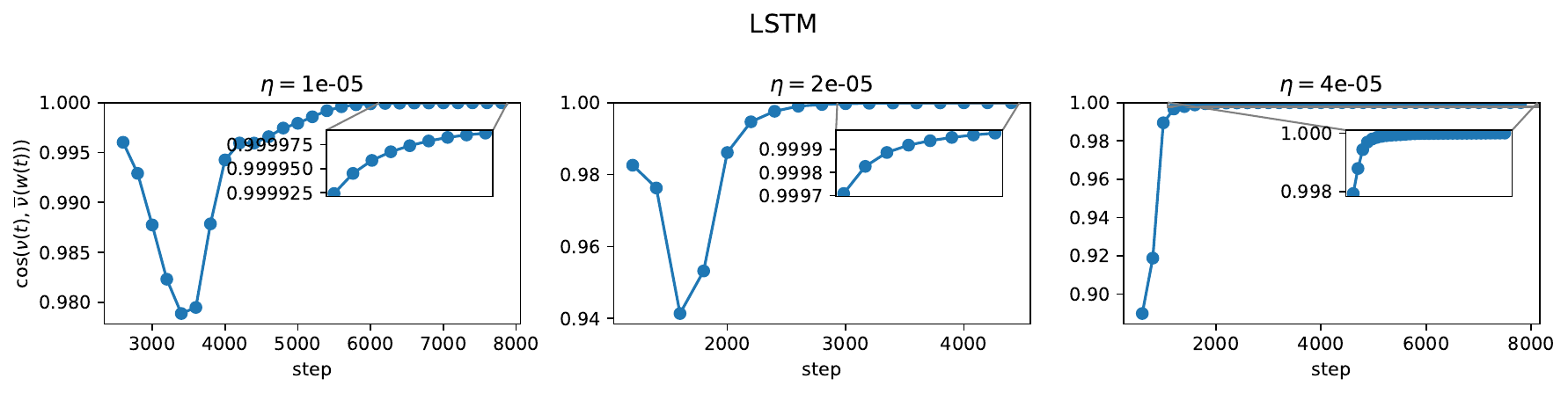}
\includegraphics[width=0.8\linewidth]{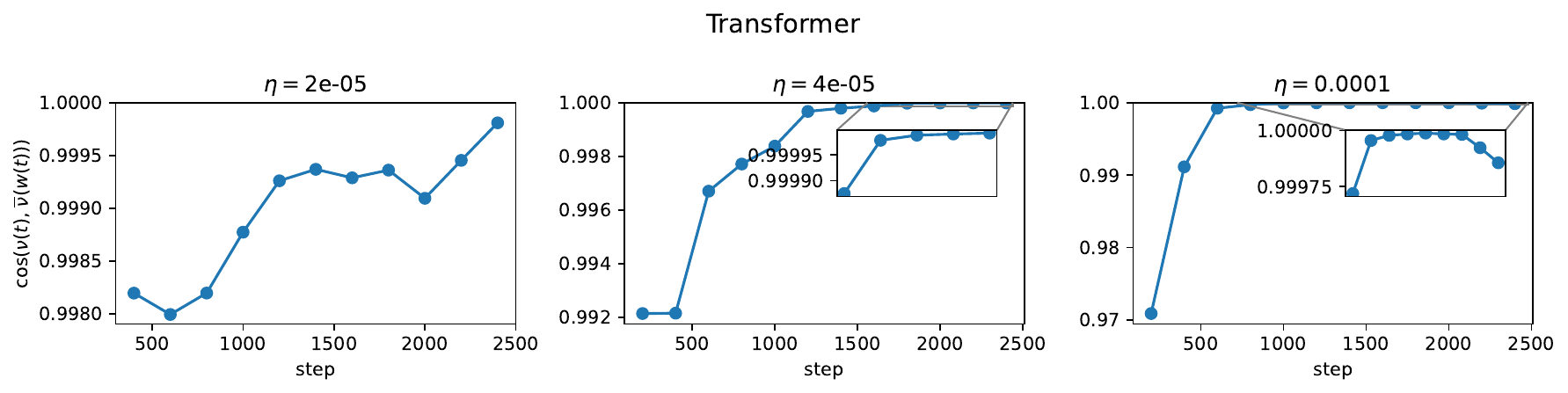}
\includegraphics[width=0.8\linewidth]{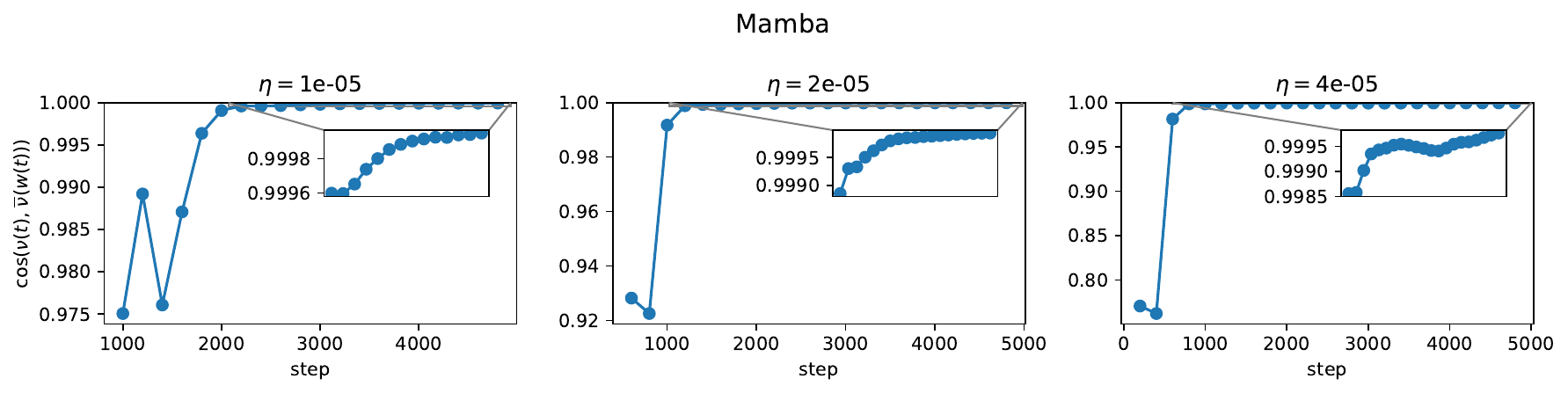}
    \caption{\textbf{The EMA $\boldsymbol{\nu}$ reaches stationarity during training (MSE).} While running the RMSProp central flow, beginning at the time when training enters EOS, we monitor the cosine similarity between the EMA $\nu(t)$ and the stationary EMA $\bar{\nu}(w(t))$.  This cosine similarity rises to high values (nearly 1) during training, implying that $\nu(t)$ reaches stationarity.}
    \label{fig:cosnu-mse}
\end{figure}

\begin{figure}[H]
\centering
\includegraphics[width=0.8\linewidth]{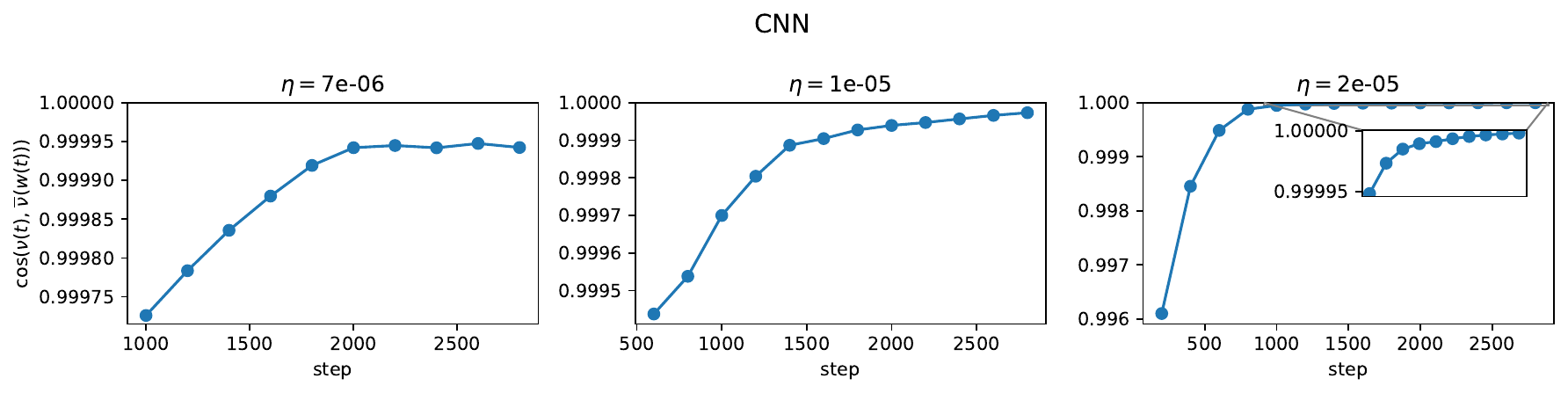}
\includegraphics[width=0.8\linewidth]{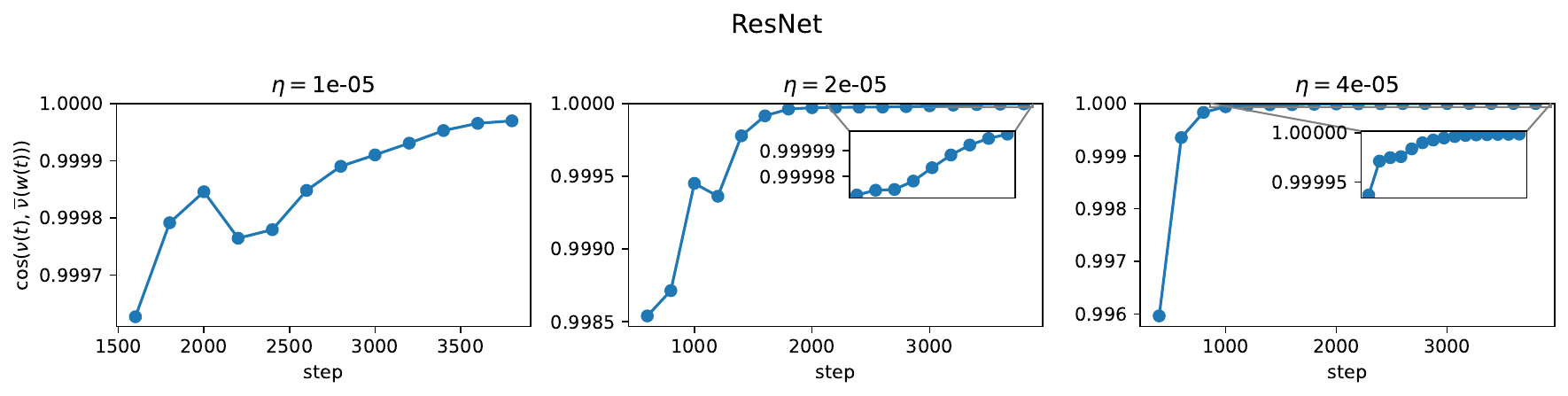}
\includegraphics[width=0.8\linewidth]{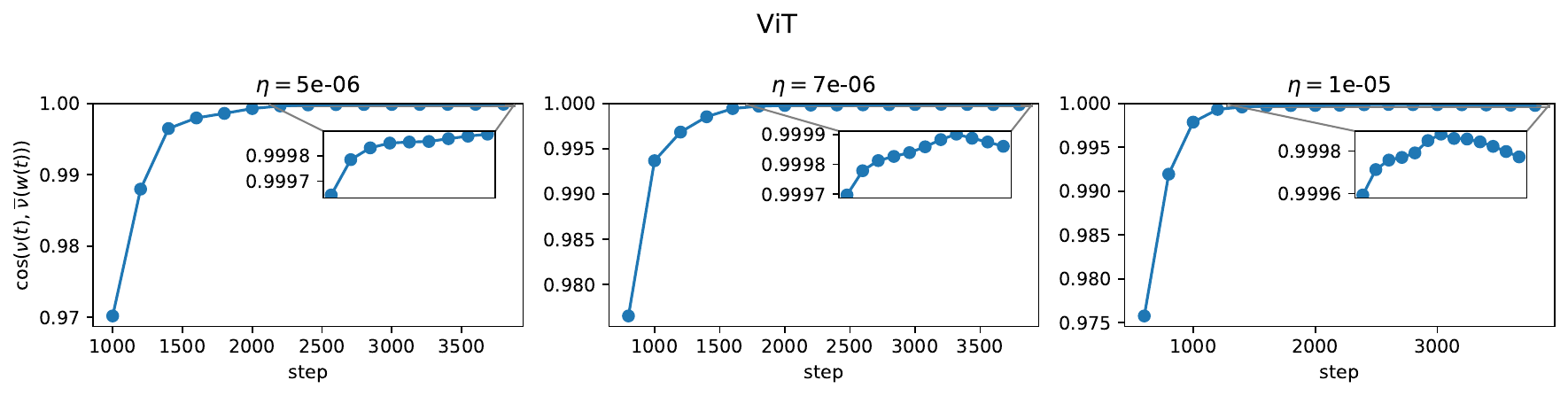}
\includegraphics[width=0.8\linewidth]{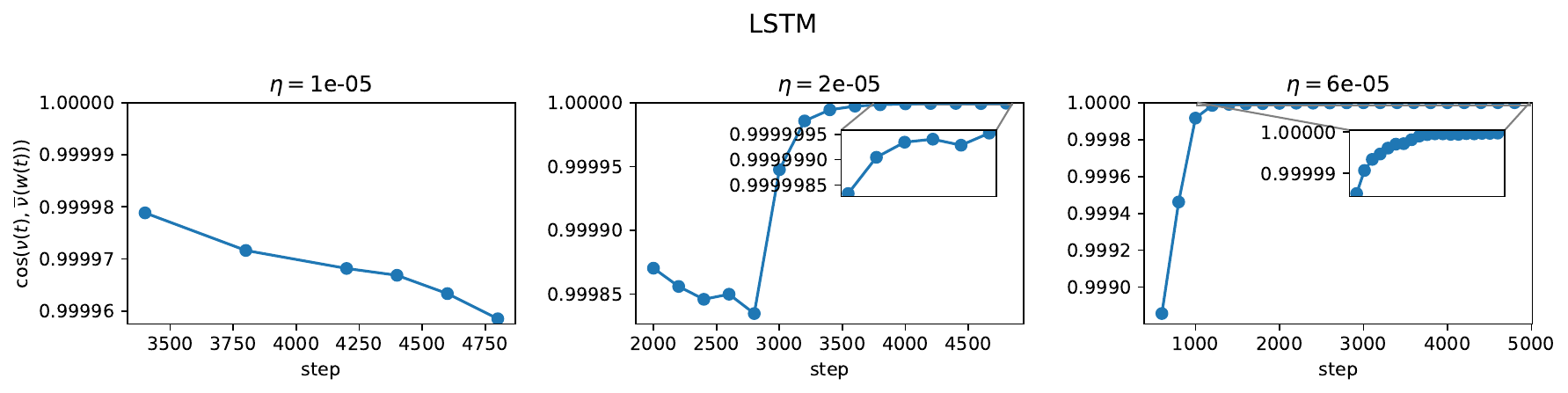}
\includegraphics[width=0.8\linewidth]{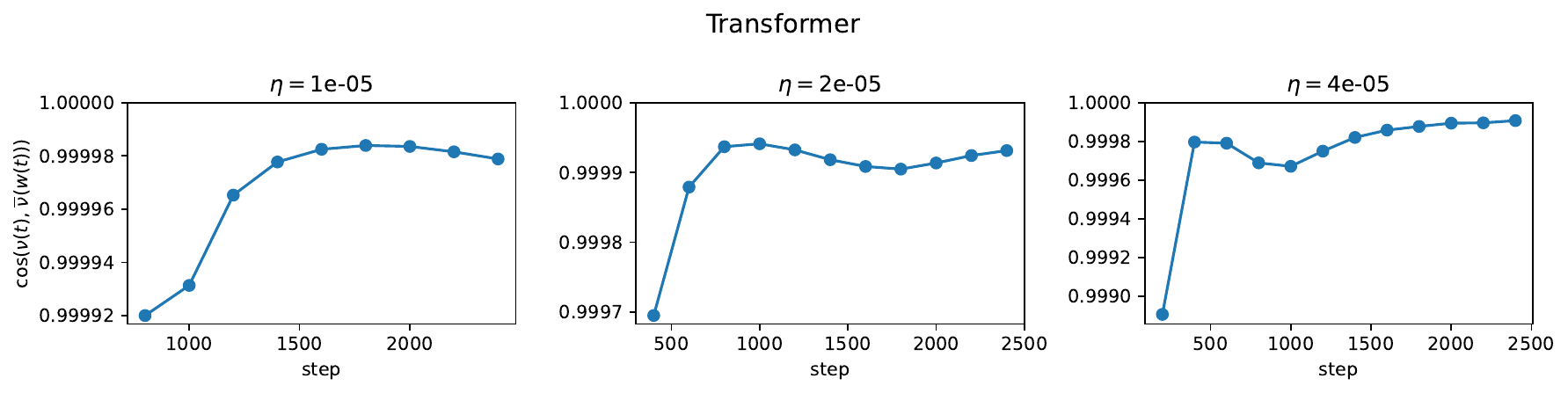}
\includegraphics[width=0.8\linewidth]{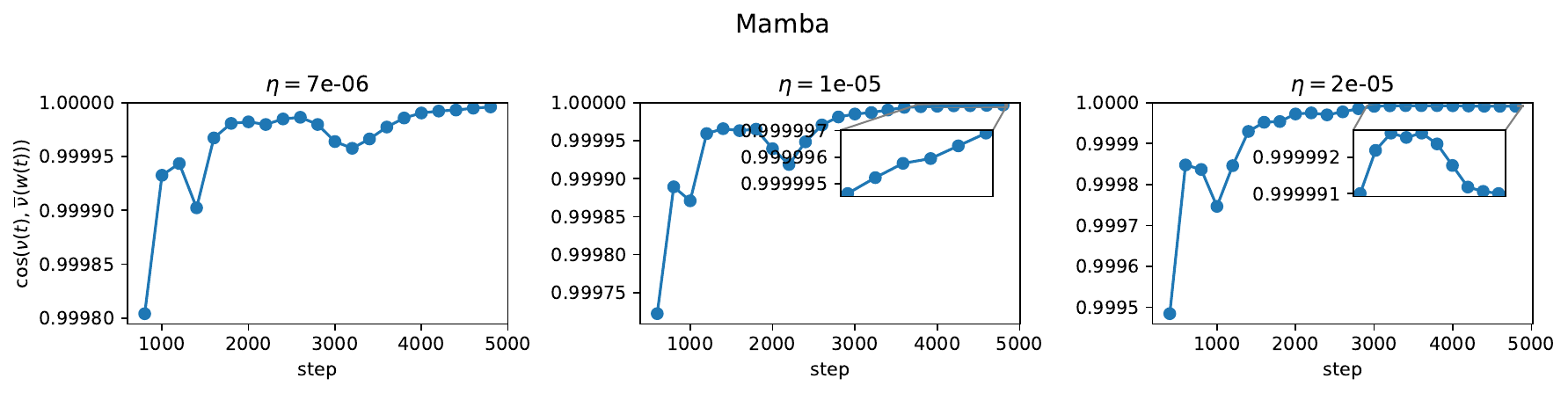}

    \caption{\textbf{The EMA $\boldsymbol{\nu}$ reaches stationarity during training (cross-entropy).} This figure is analogous to \Cref{fig:cosnu-mse} but for cross-entropy loss.}
    \label{fig:cosnu-ce}
\end{figure}
\end{letterfigures}

\begin{letterfigures}
\begin{figure}[H]
\centering
\includegraphics[width=0.8\linewidth]{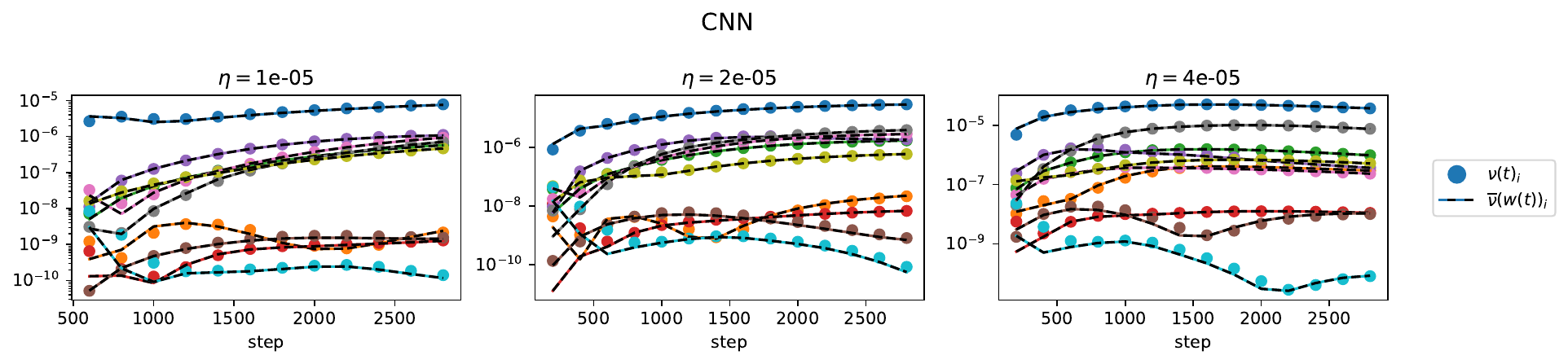}
\includegraphics[width=0.8\linewidth]{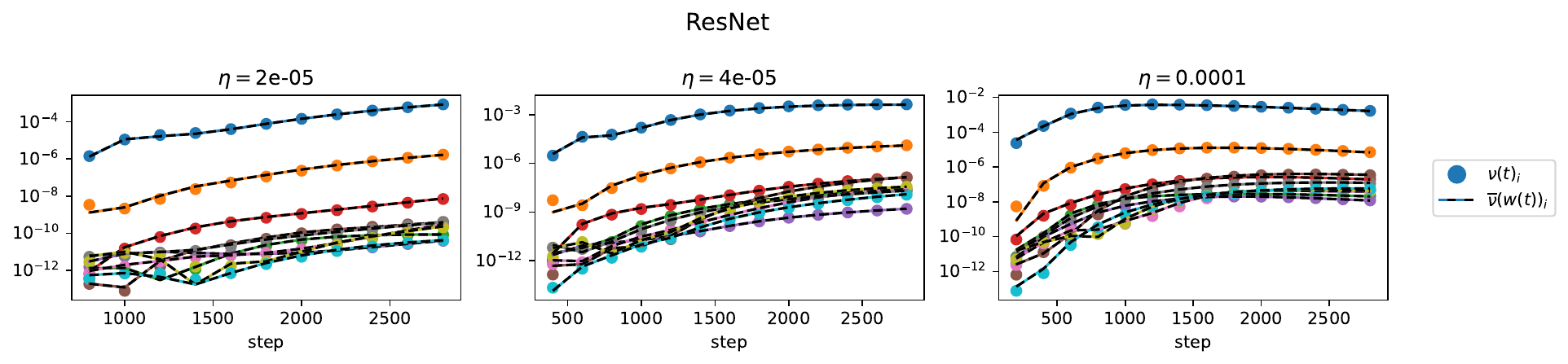}
\includegraphics[width=0.8\linewidth]{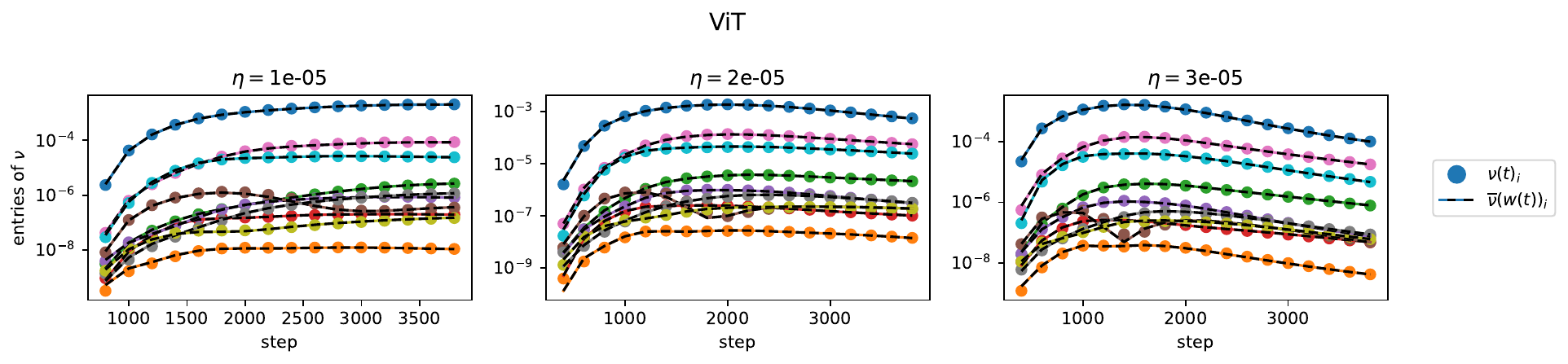}
\includegraphics[width=0.8\linewidth]{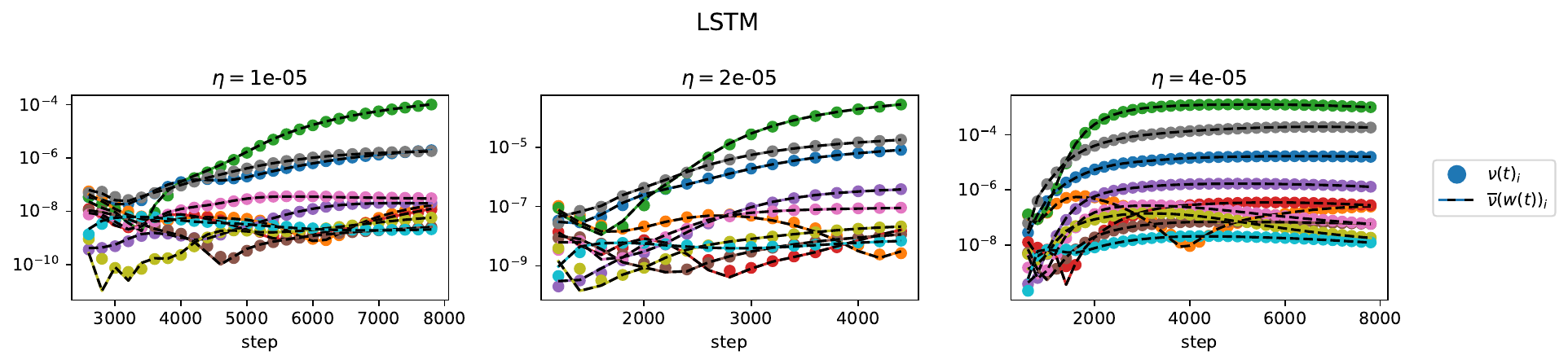}
\includegraphics[width=0.8\linewidth]{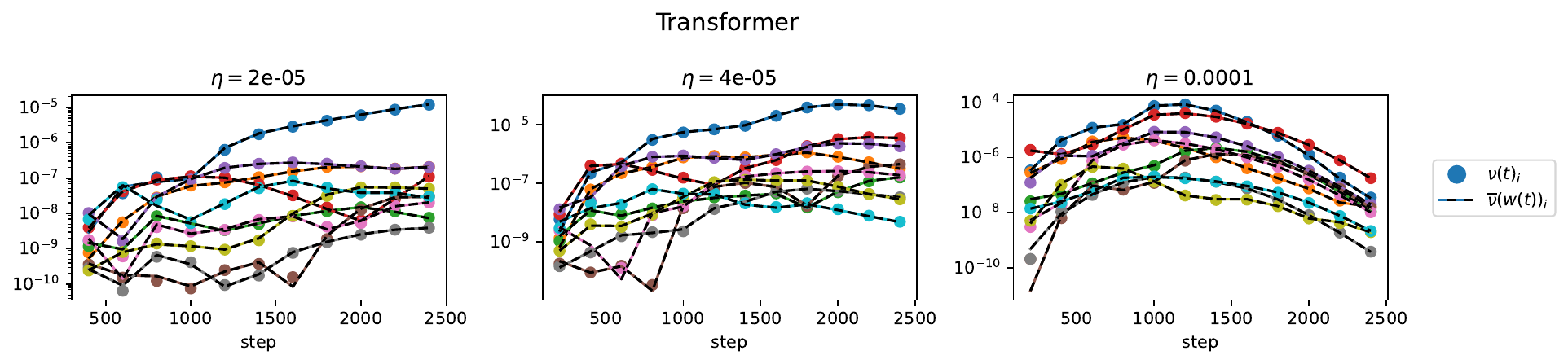}
\includegraphics[width=0.8\linewidth]{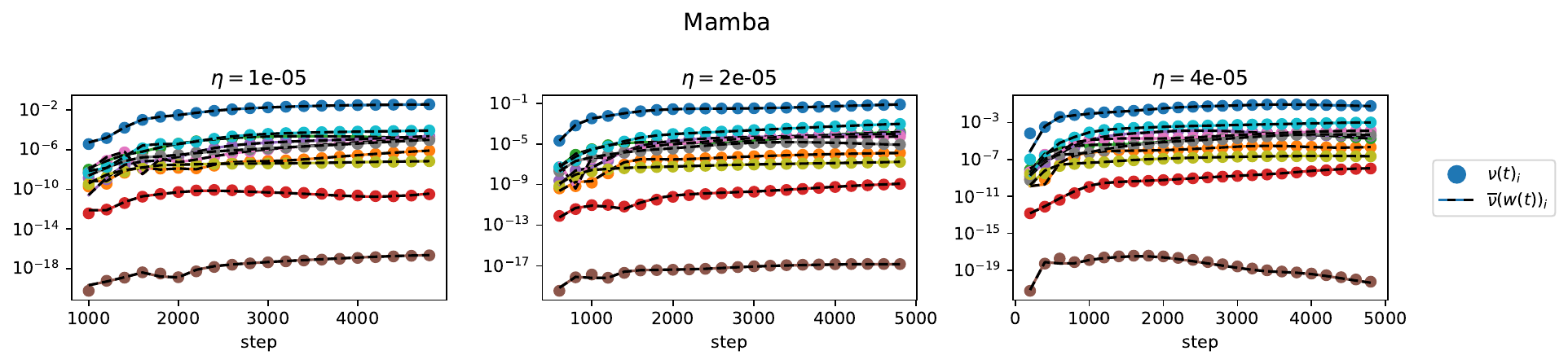}
    \caption{\textbf{Stationary EMA is accurate at a coordinate-wise level (MSE)}. While running the RMSProp central flow, starting at the time when training enters EOS, we plot the evolution of ten coordinates of the actual EMA $\nu(t)$ (dots) and the stationary EMA $\overline{\nu}(w(t))$ (half-black dashed lines).  Each color is a different coordinate, and the ten coordinates are uniformly spaced throughout the network. We can see that, starting soon after training reaches EOS, the stationary EMA $\overline{\nu}(w(t))$ becomes an excellent approximation to the real EMA $\nu(t)$, on a coordinatewise level.  We can also see that both the real EMA and the stationary EMA evolve significantly (in tandem) during this time.}
    \label{fig:nu-mse}
\end{figure}

\begin{figure}[H]
\centering
\includegraphics[width=0.8\linewidth]{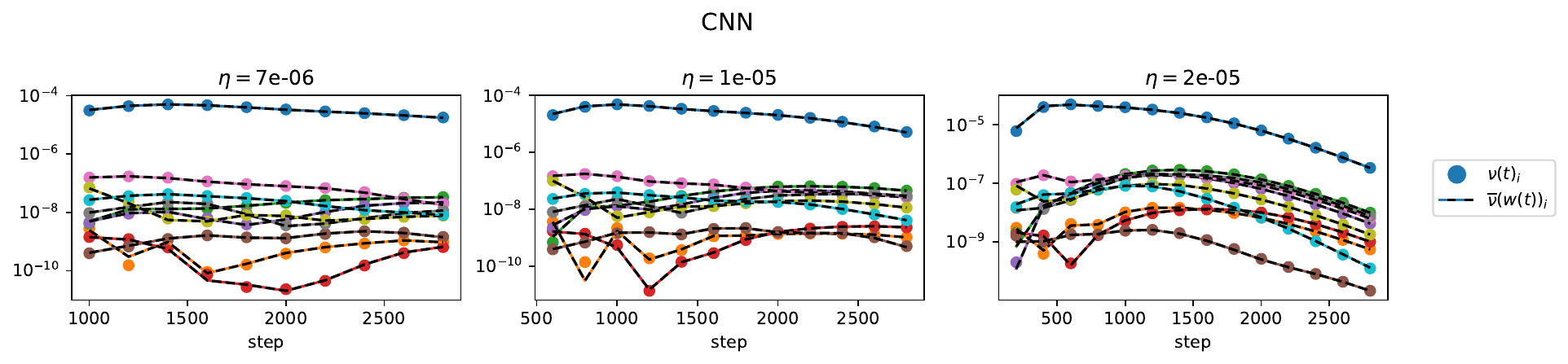}
\includegraphics[width=0.8\linewidth]{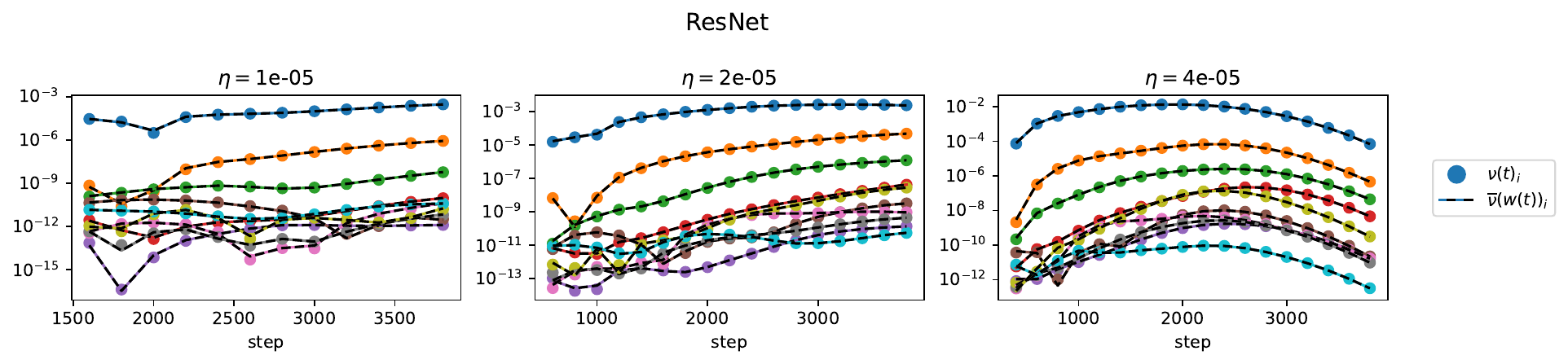}
\includegraphics[width=0.8\linewidth]{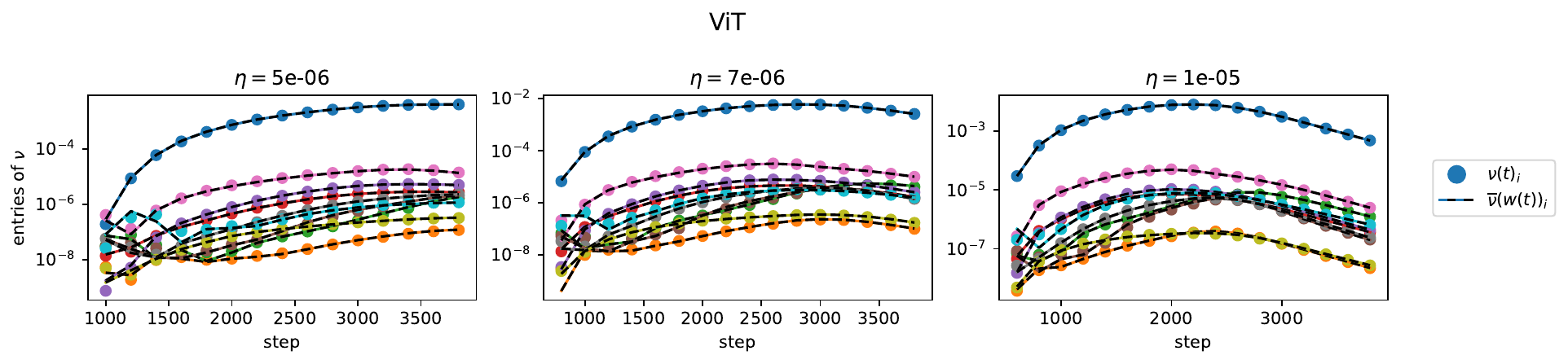}
\includegraphics[width=0.8\linewidth]{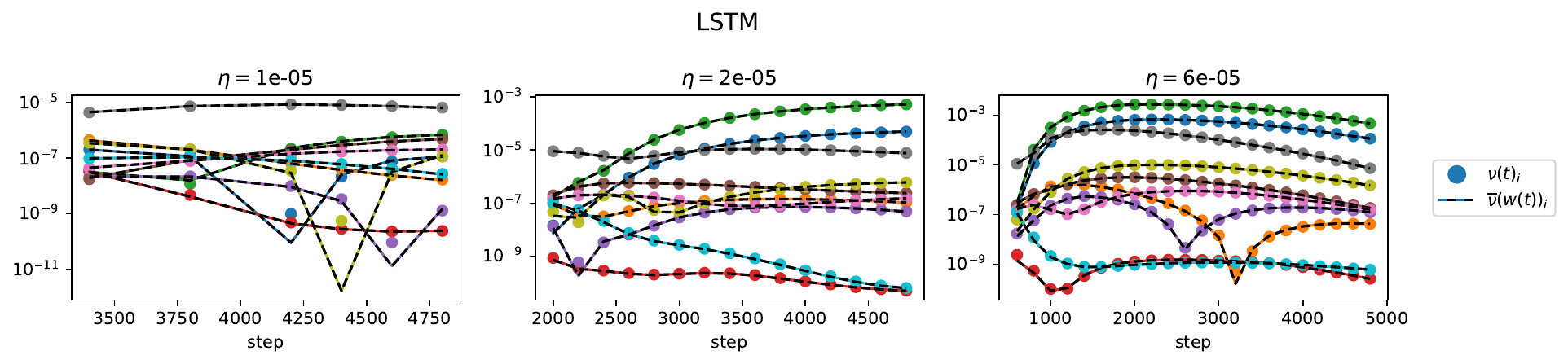}
\includegraphics[width=0.8\linewidth]{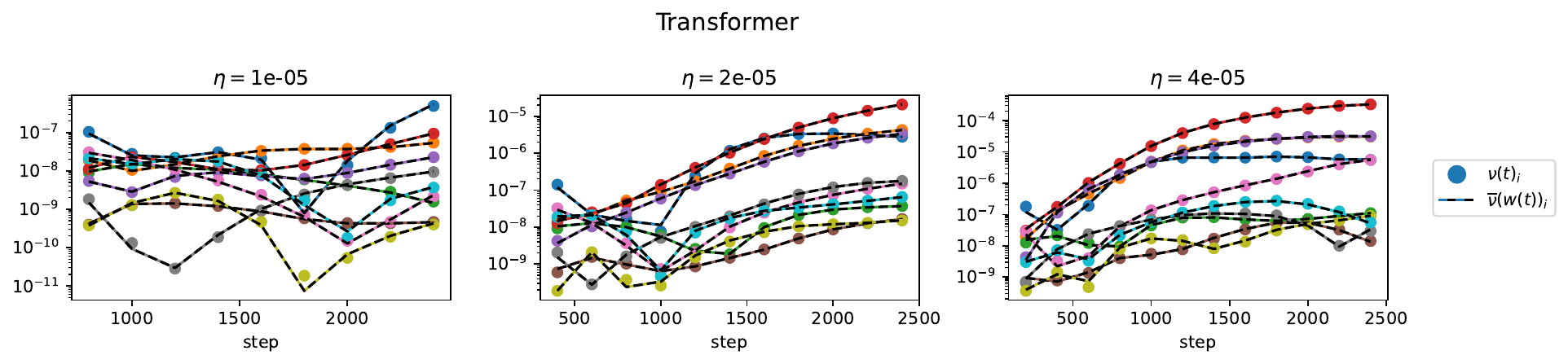}
\includegraphics[width=0.8\linewidth]{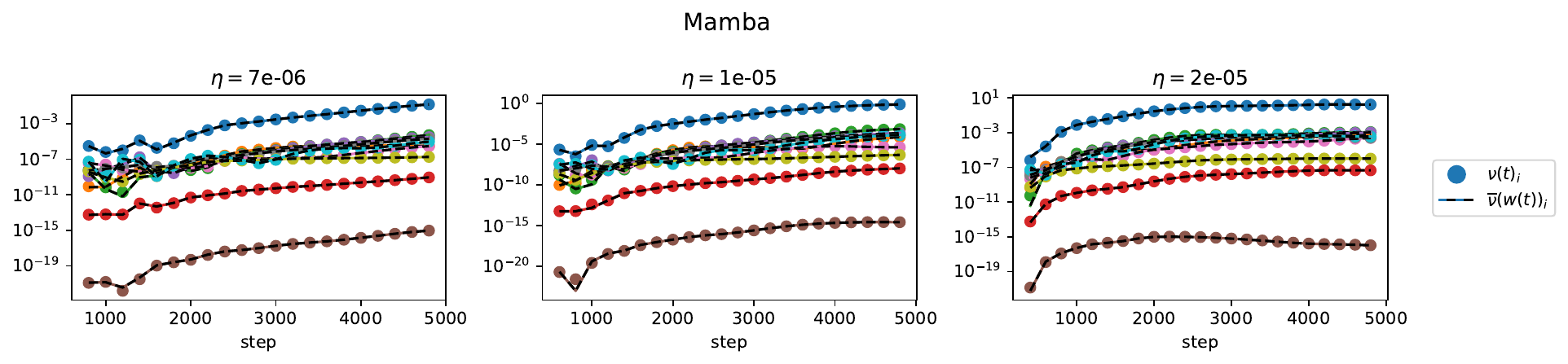}
    \caption{\textbf{Stationary EMA is accurate at a coordinate-wise level (CE)}. Analogous to \Cref{fig:nu-mse}, but for cross-entropy loss.}
    \label{fig:nu-ce}
\end{figure}
\end{letterfigures}

\begin{figure}[H]
\centering
\includegraphics[width=0.8\linewidth]{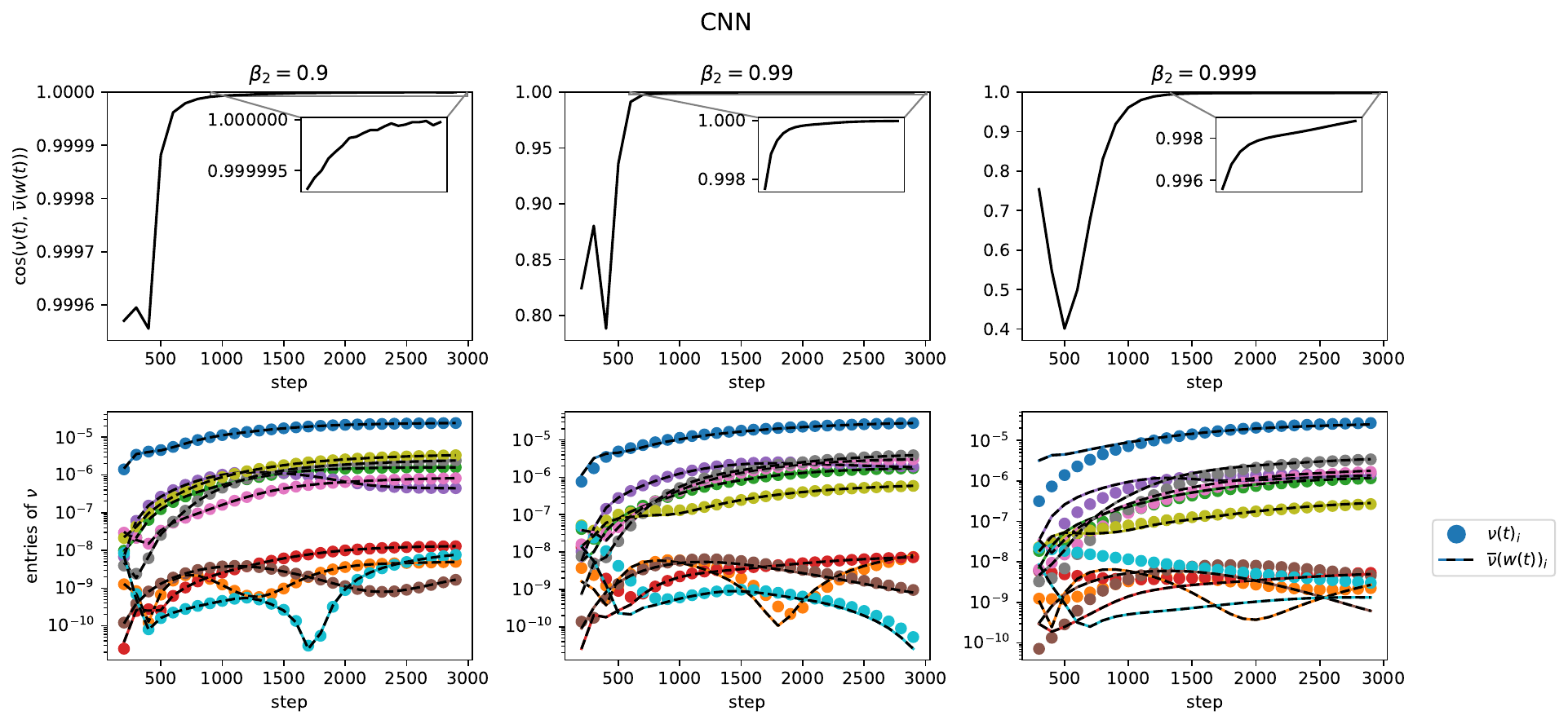}
    \caption{\textbf{Assessing how $\beta_2$ impacts the convergence of $
    \nu$ to stationarity.}  For multiple values of $\beta_2$ (columns), we monitor the closeness between $\nu(t)$ and the stationary value $\overline{\nu}(w(t))$ over time.  In particular, the top row reports the cosine similarity between these two vectors, and the bottom row compares ten individual coordinates (colors).  As one might expect, we see that when $\beta_2$ is smaller, $\nu(t)$ converges faster to $\overline{\nu}(w(t))$ and the ultimate similarity is higher.   Conversely, when $\beta_2$ is at the highest value of 0.999, some of the coordinates are noticeably off (e.g. the teal, brown, and orange coordinates.) \textit{Details:} a CNN is trained on a subset of CIFAR-10 using MSE loss at $\eta = $ 2e-5,  $\beta_2 \in \{0.9, 0.99, 0.999\}$, $\epsilon = $1e-8, and bias correction. }
    \label{fig:experiments:rmsprop:beta2-stationarity}
\end{figure}

\begin{letterfigures}
\begin{figure}[H]
\centering
\includegraphics[width=0.8\linewidth]{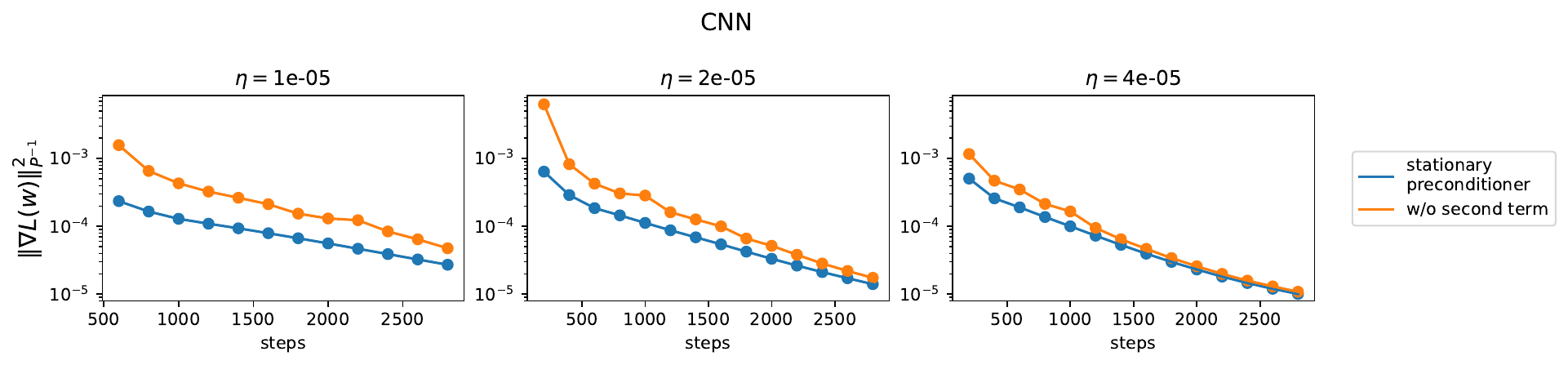}
\includegraphics[width=0.8\linewidth]{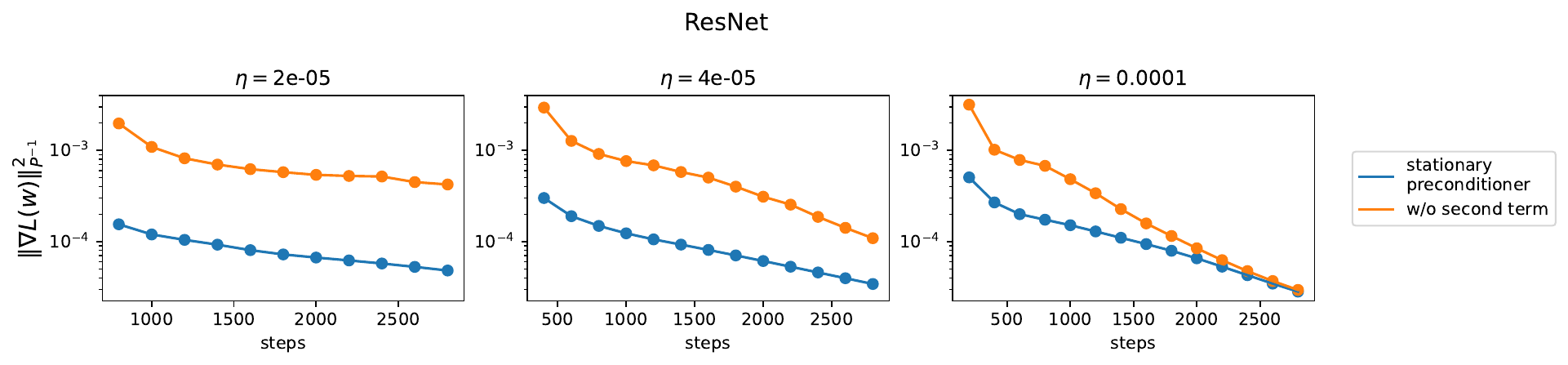}
\includegraphics[width=0.8\linewidth]{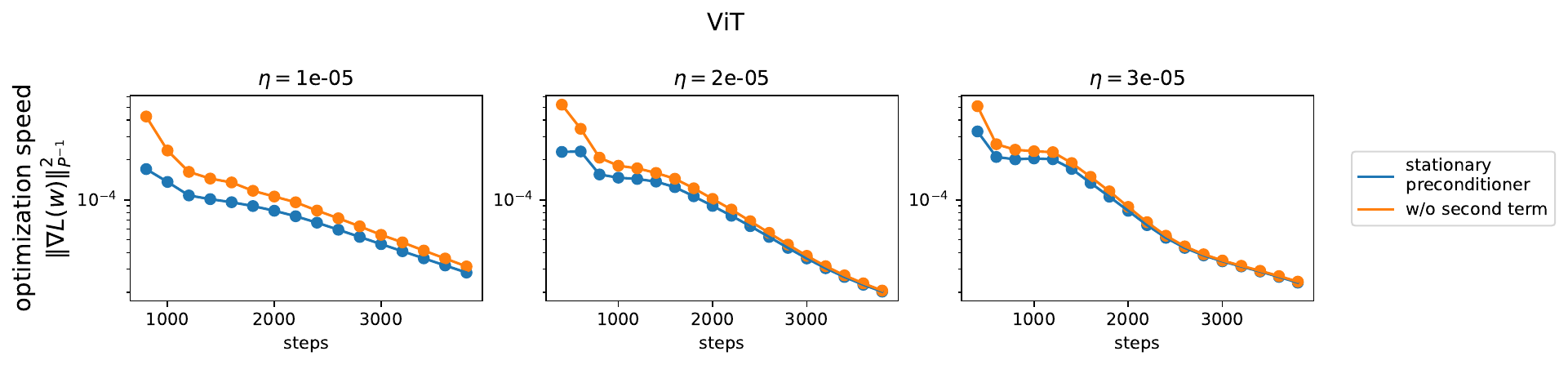}
\includegraphics[width=0.8\linewidth]{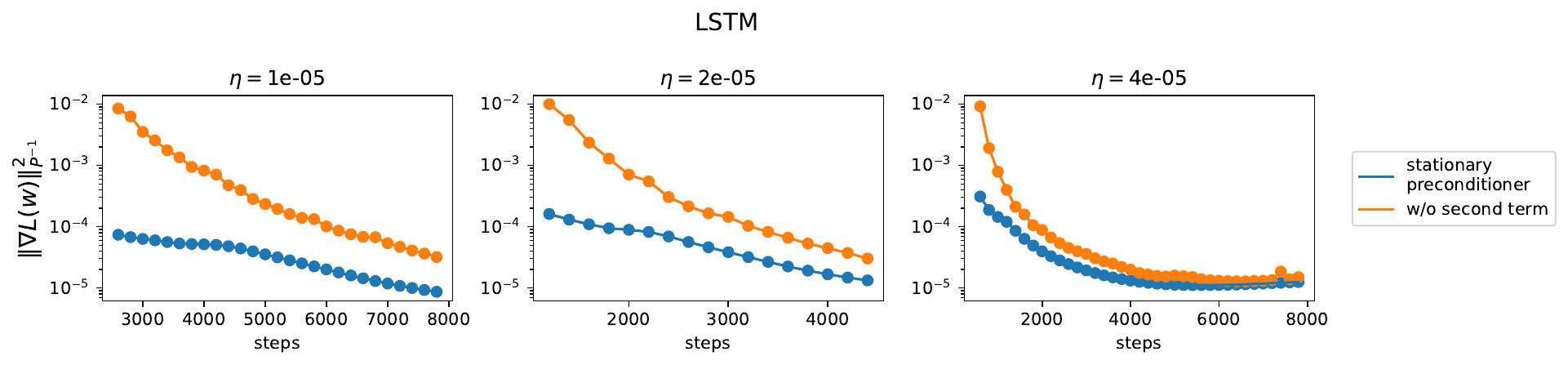}
\includegraphics[width=0.8\linewidth]{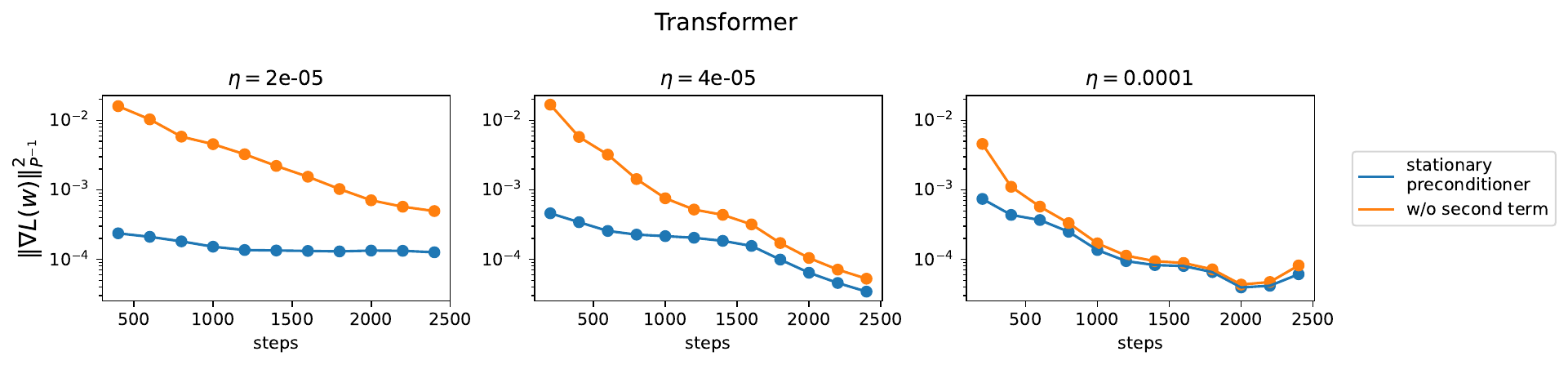}
\includegraphics[width=0.8\linewidth]{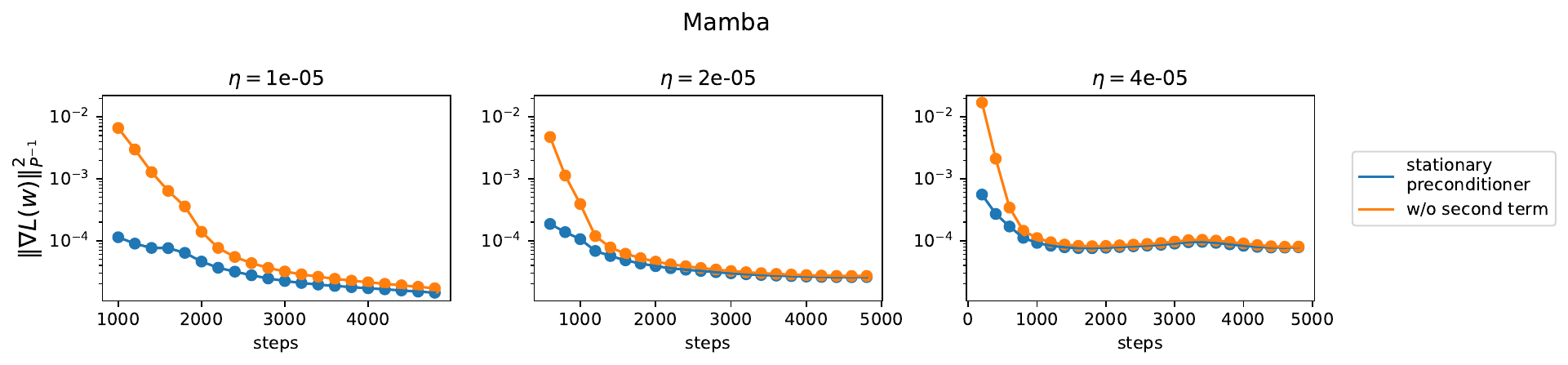}
    \caption{\textbf{RMSProp stationary preconditioner is suboptimal (MSE)}.  We compare the RMSProp stationary preconditioner, defined as the solution to the optimization problem \cref{eq:rmsprop_nu_convex_program}, to an alternative preconditioner defined as the solution to \cref{eq:rmsprop_nu_first_term_only}, a similar optimization problem but without the second term in the objective.  We assess each preconditioner $P$ by reporting $\| \nabla L(w) \|_{P^-1}^2 $, the instantaneous rate of decrease in the loss under the preconditioned gradient flow with preconditioner $P$. Observe that this value is higher under the alternative preconditioner (orange) than under the RMSProp stationary preconditioner (blue), meaning that the alternative preconditioner would decrease the loss faster.  The gap between the two preconditioners tends to be smaller when $\eta$ is larger, which is reasonable because the second term in \cref{eq:rmsprop_nu_convex_program} is proportional to $\frac{1}{\eta^2}$.}
\label{fig:stationary-suboptimal-mse}
\end{figure}

\begin{figure}[H]
\centering
\includegraphics[width=0.8\linewidth]{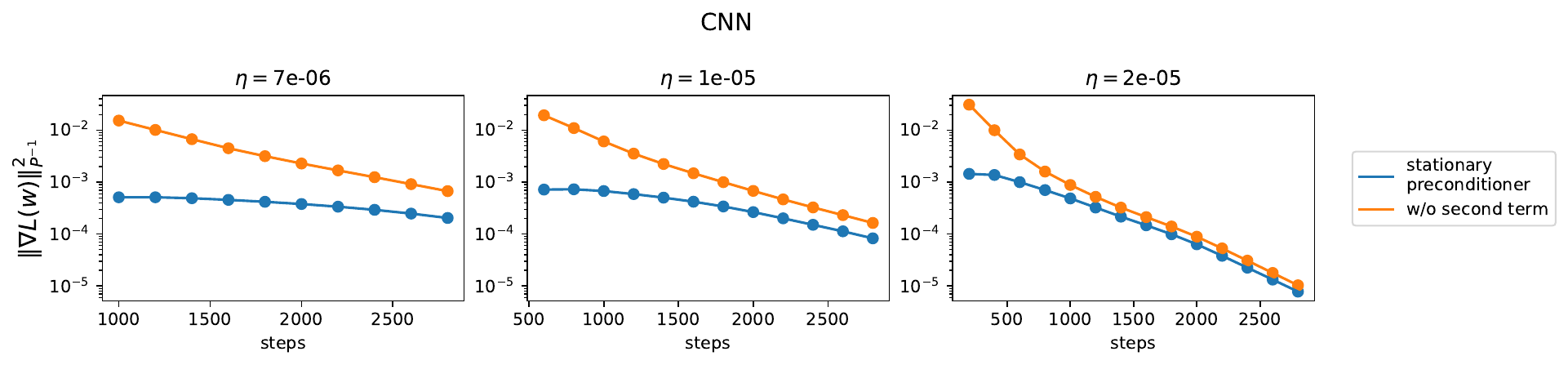}
\includegraphics[width=0.8\linewidth]{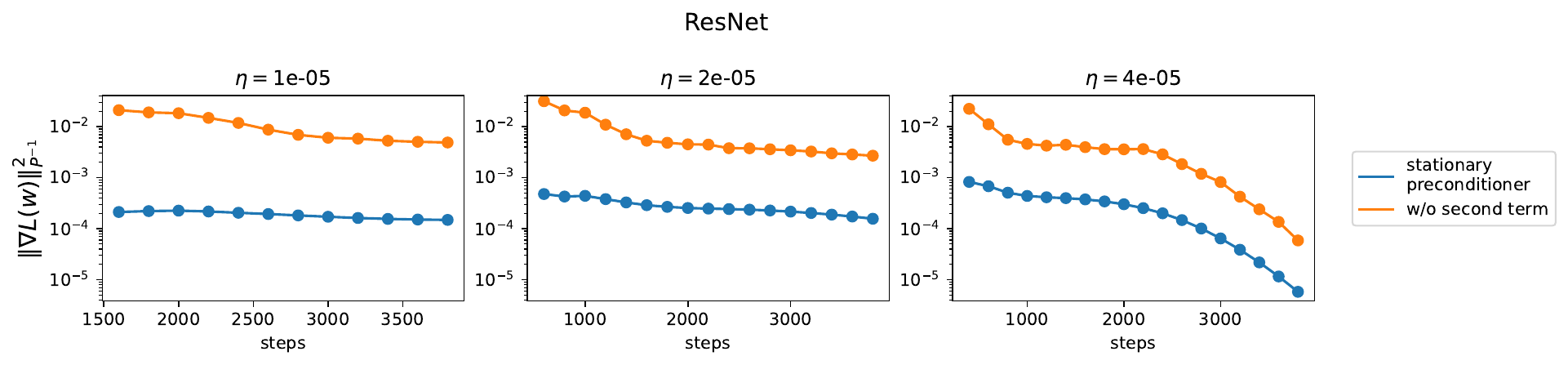}
\includegraphics[width=0.8\linewidth]{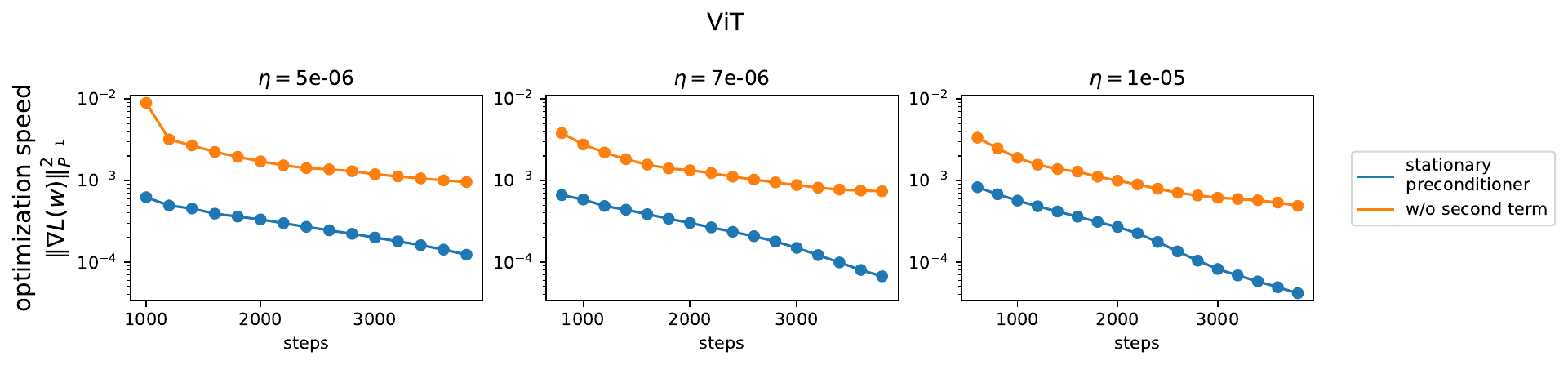}
\includegraphics[width=0.8\linewidth]{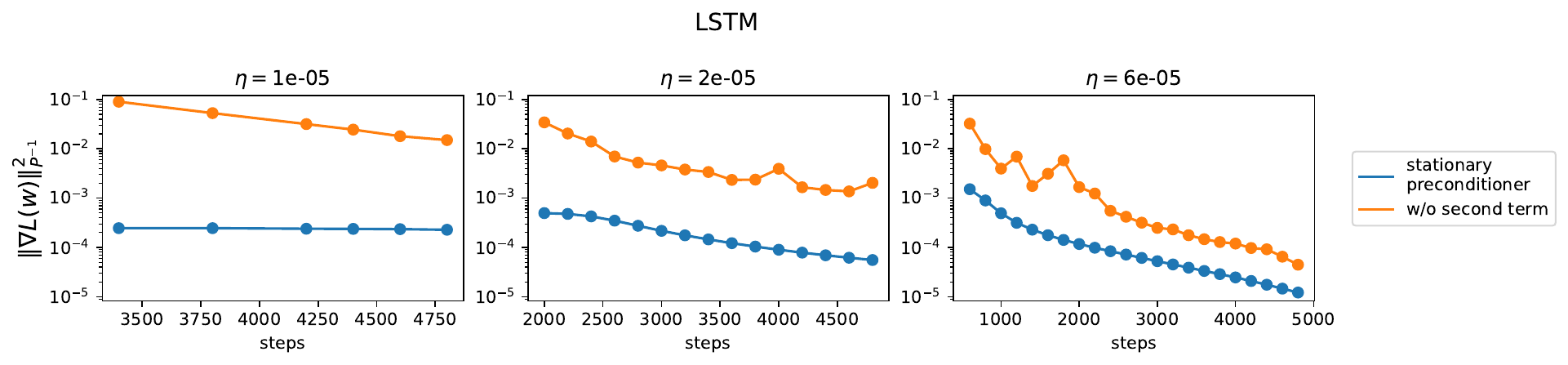}
\includegraphics[width=0.8\linewidth]{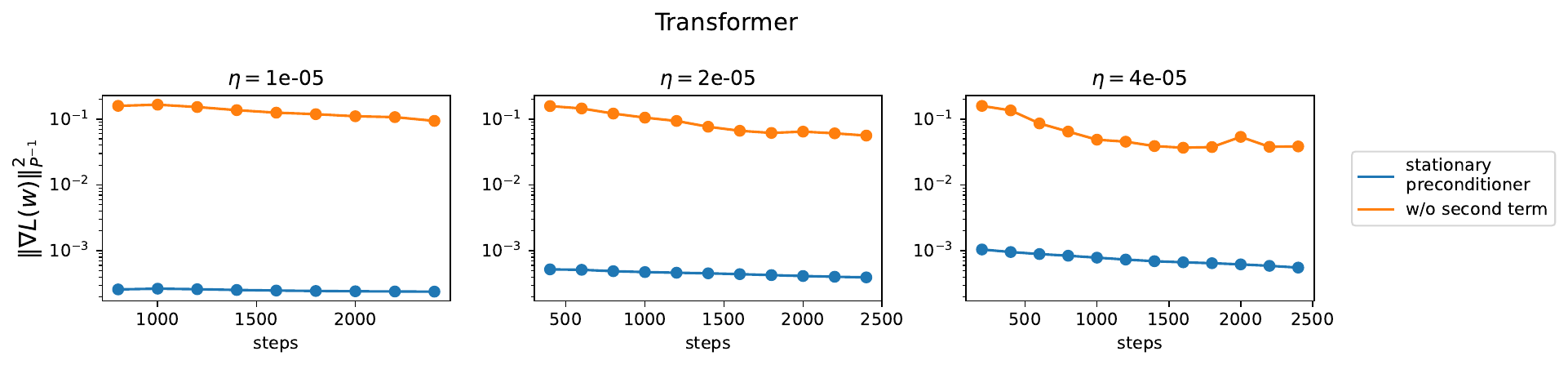}
\includegraphics[width=0.8\linewidth]{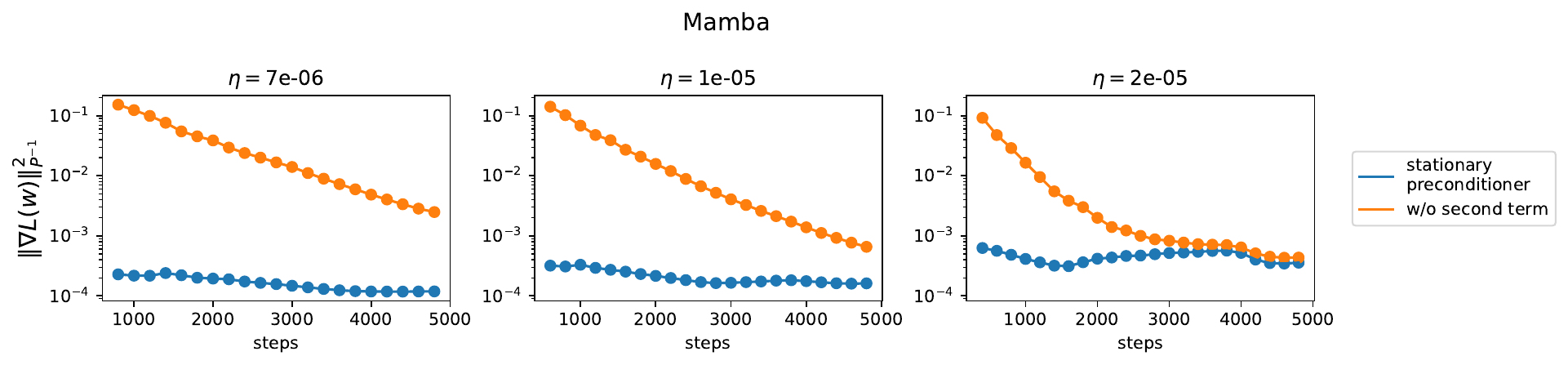}
    \caption{\textbf{RMSProp stationary preconditioner is suboptimal (cross-entropy)}. This figure is analogous to \Cref{fig:stationary-suboptimal-mse}, but for cross-entropy loss. }
    \label{fig:stationary-suboptimal-ce}
\end{figure}
\end{letterfigures}

\begin{letterfigures}
\begin{figure}[H]
\centering
\includegraphics[width=0.8\linewidth]{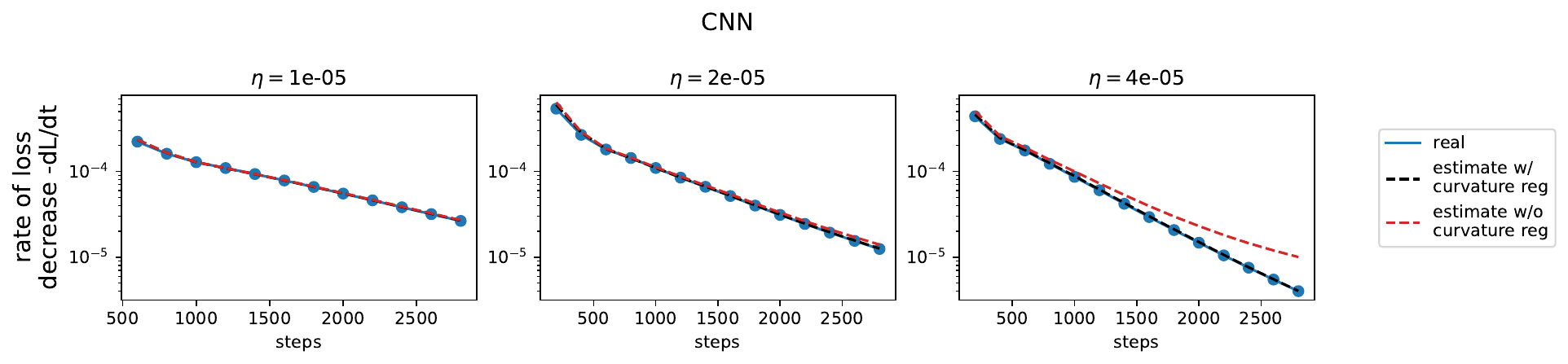}
\includegraphics[width=0.8\linewidth]{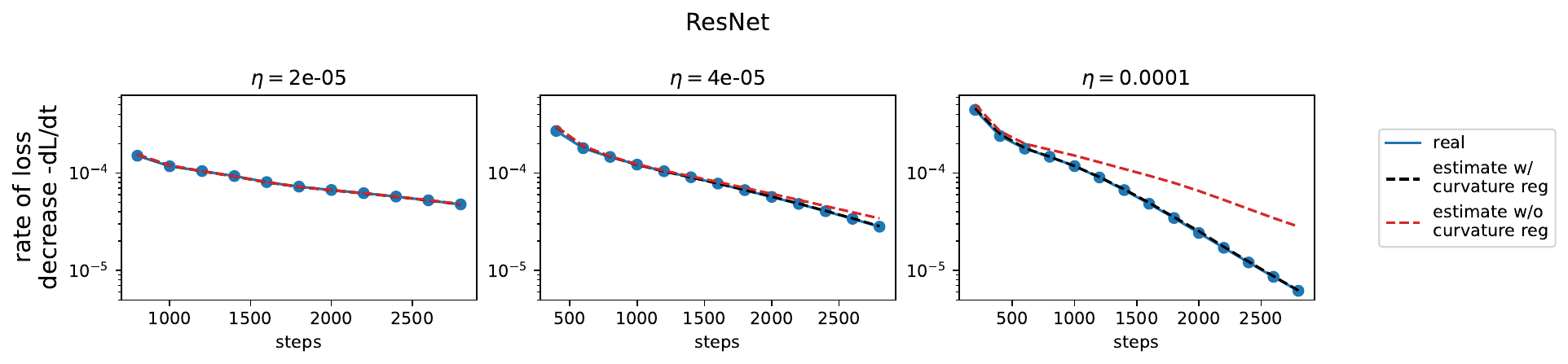}
\includegraphics[width=0.8\linewidth]{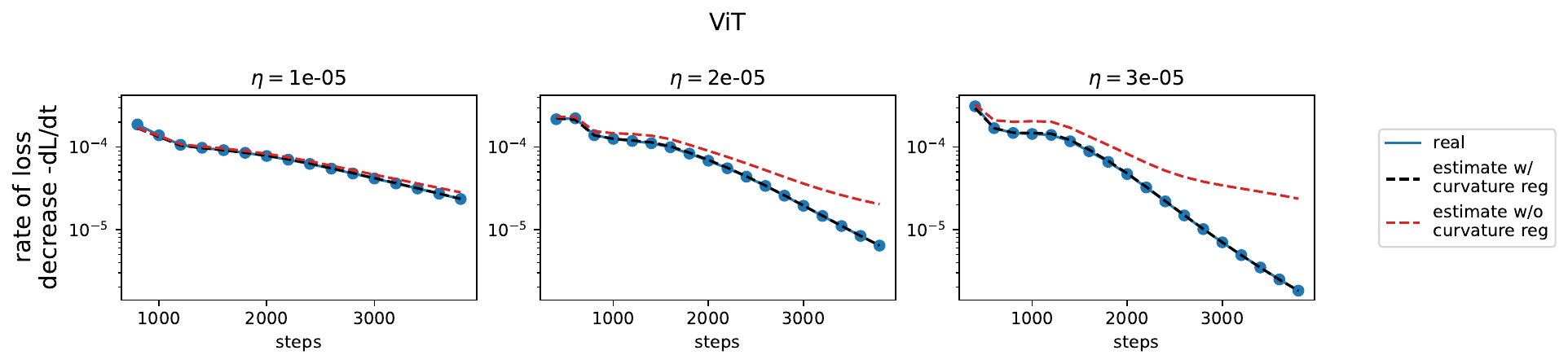}
\includegraphics[width=0.8\linewidth]{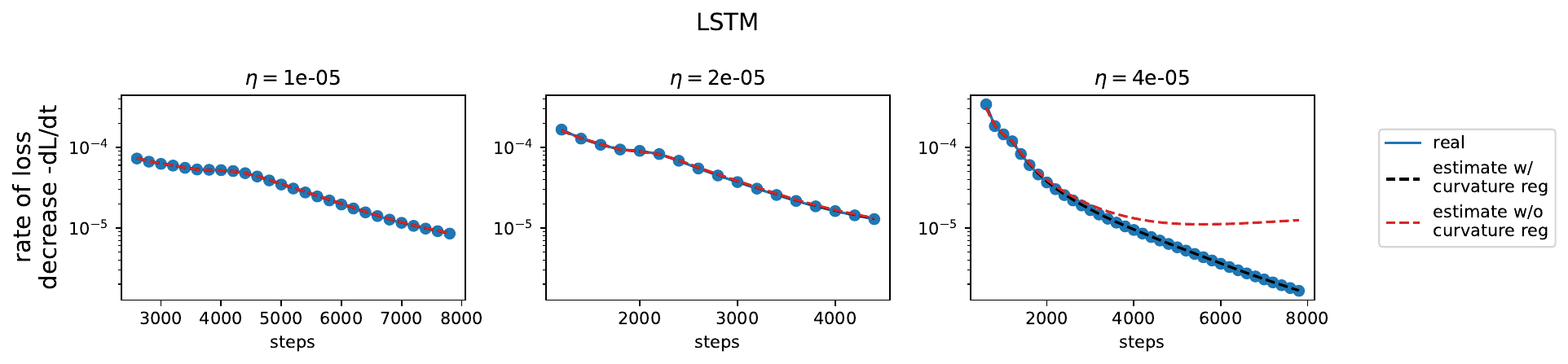}
\includegraphics[width=0.8\linewidth]{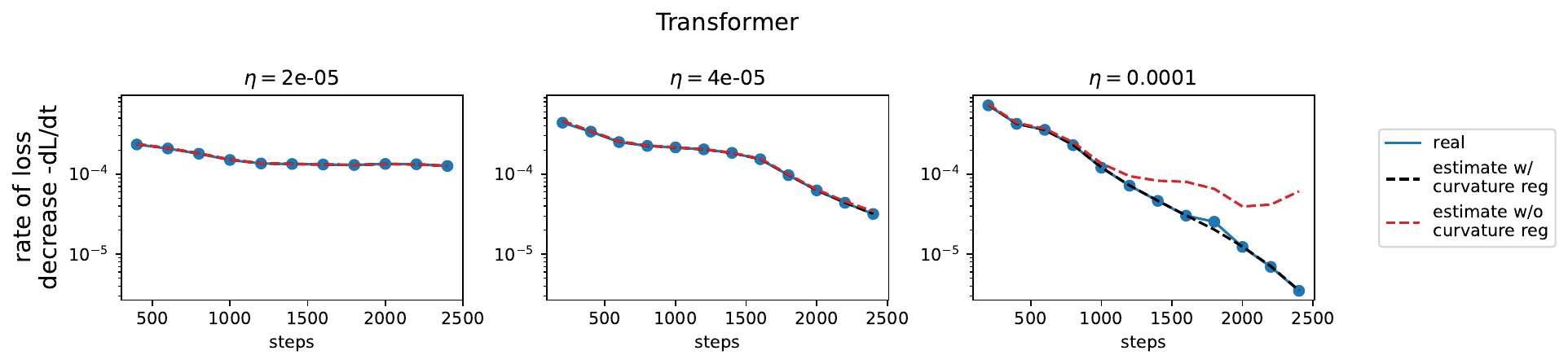}
\includegraphics[width=0.8\linewidth]{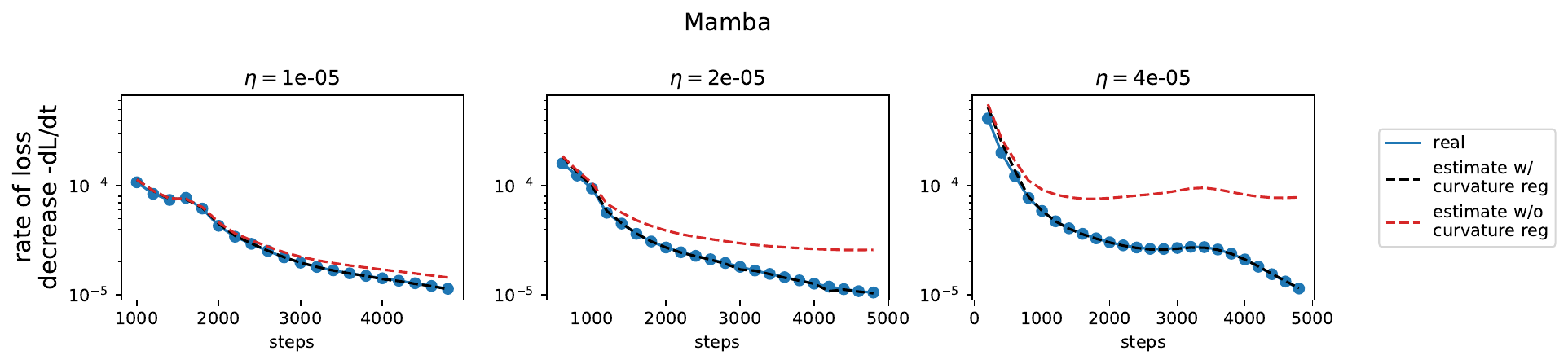}
    \caption{\textbf{Stationary flow accurately predicts the instantaneous speed of optimization (MSE)}.  The stationary flow \cref{eq:rmsprop-stationary-flow}, which incorporates an implicit curvature regularizer, predicts (black) the rate of loss decrease $-\tfrac{dL}{dt}$ (blue) more accurately than a naive estimate $\| \nabla L(w)\|^2_{\overline{P}^{-1}(w)}$ (in red) which uses the stationary preconditioner but does not incorporate curvature regularization.   Observe that the gap between the two estimates is larger when $\eta$ is larger, suggesting that, like Scalar RMSProp, the implicit regularization of RMSProp increases in strength with $\eta$.}
    \label{fig:optspeed-mse}
\end{figure}

\begin{figure}[H]
\centering
\includegraphics[width=0.8\linewidth]{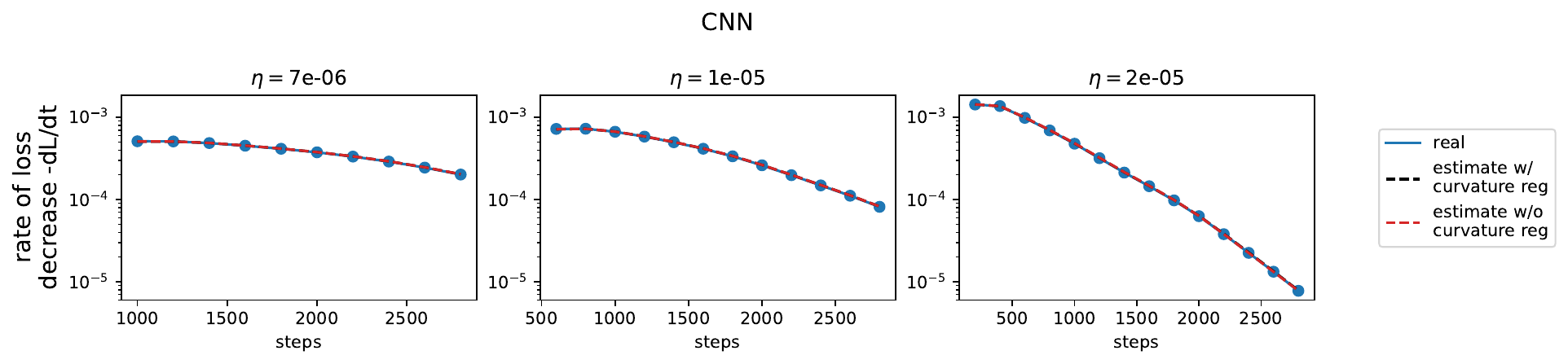}
\includegraphics[width=0.8\linewidth]{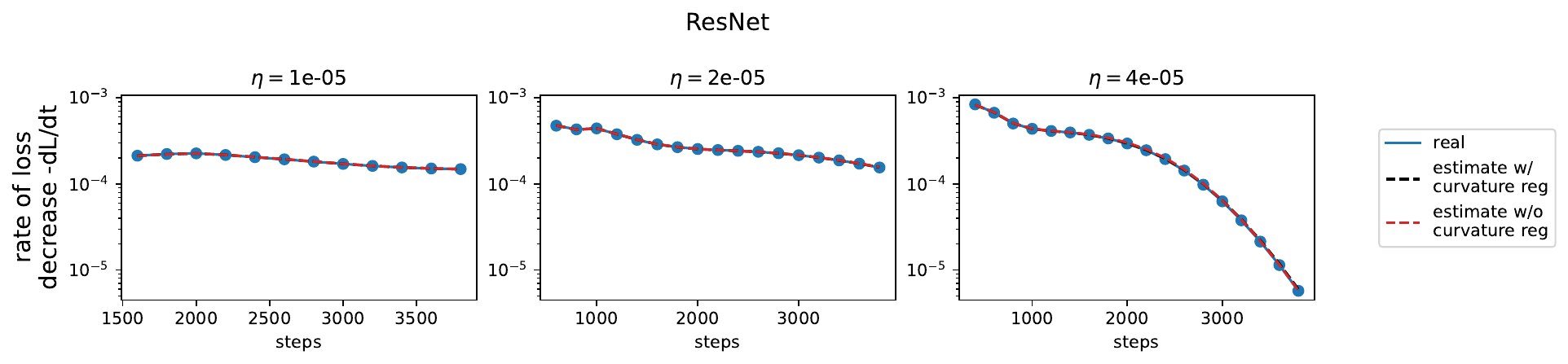}
\includegraphics[width=0.8\linewidth]{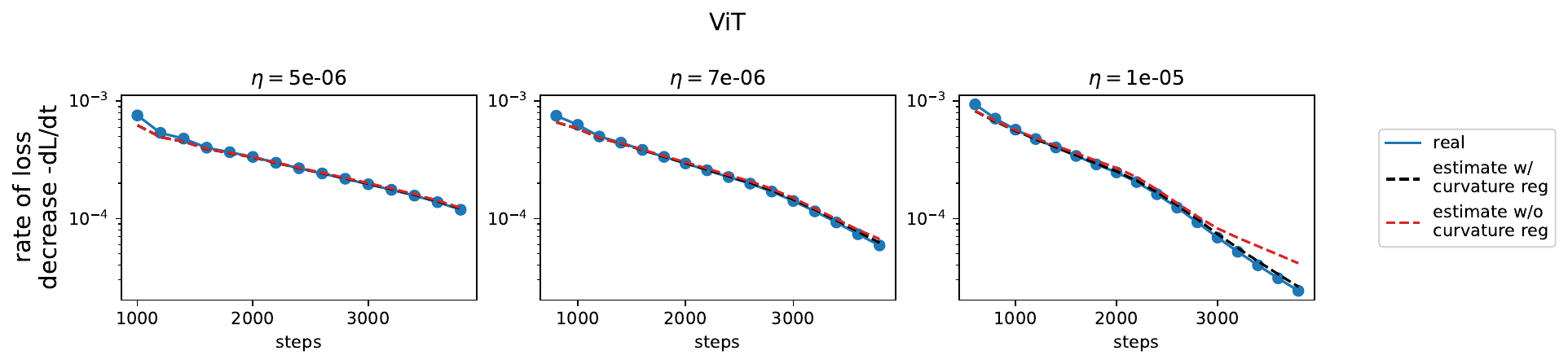}
\includegraphics[width=0.8\linewidth]{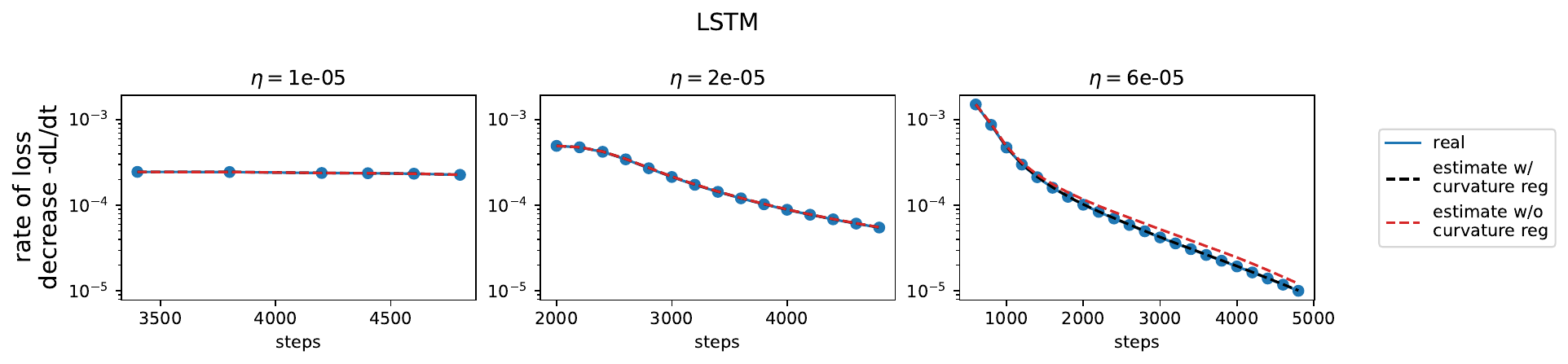}
\includegraphics[width=0.8\linewidth]{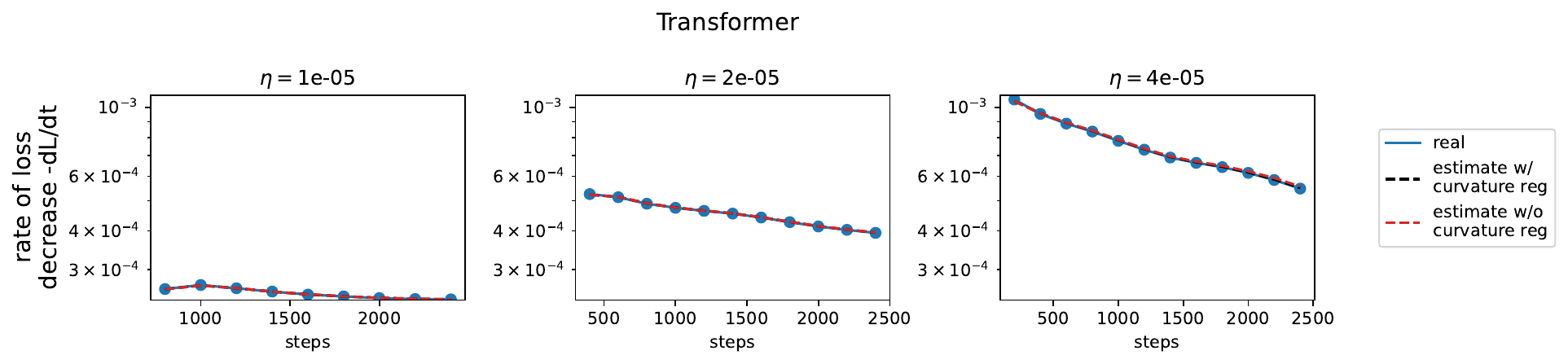}
\includegraphics[width=0.8\linewidth]{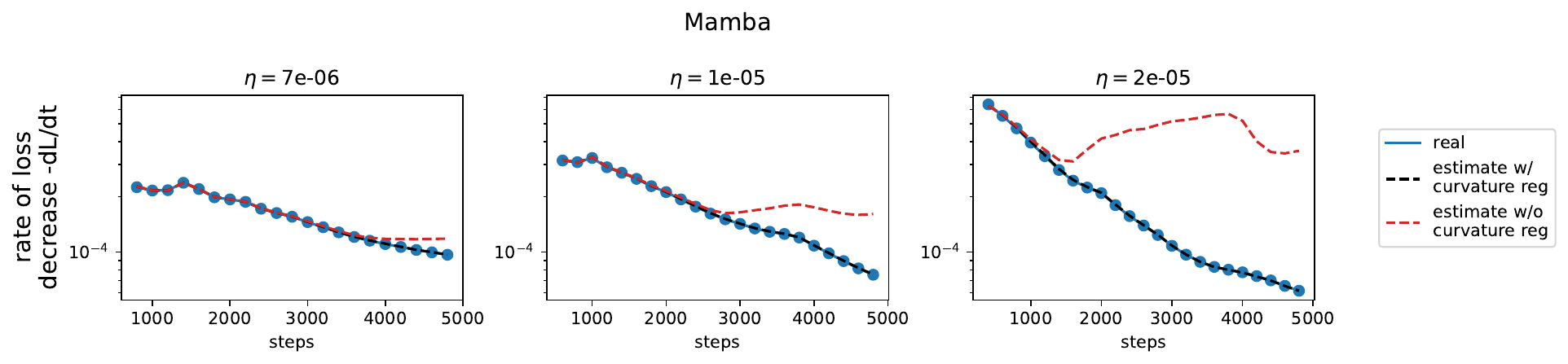}
    \caption{\textbf{Stationary flow accurately predicts the instantaneous speed of optimization (cross-entropy)}.  Same as \Cref{fig:optspeed-mse}, but with cross-entropy loss.}
    \label{fig:optspeed-ce}
\end{figure}
\end{letterfigures}

\begin{letterfigures}
\begin{figure}[H]
\centering
\includegraphics[width=\linewidth]{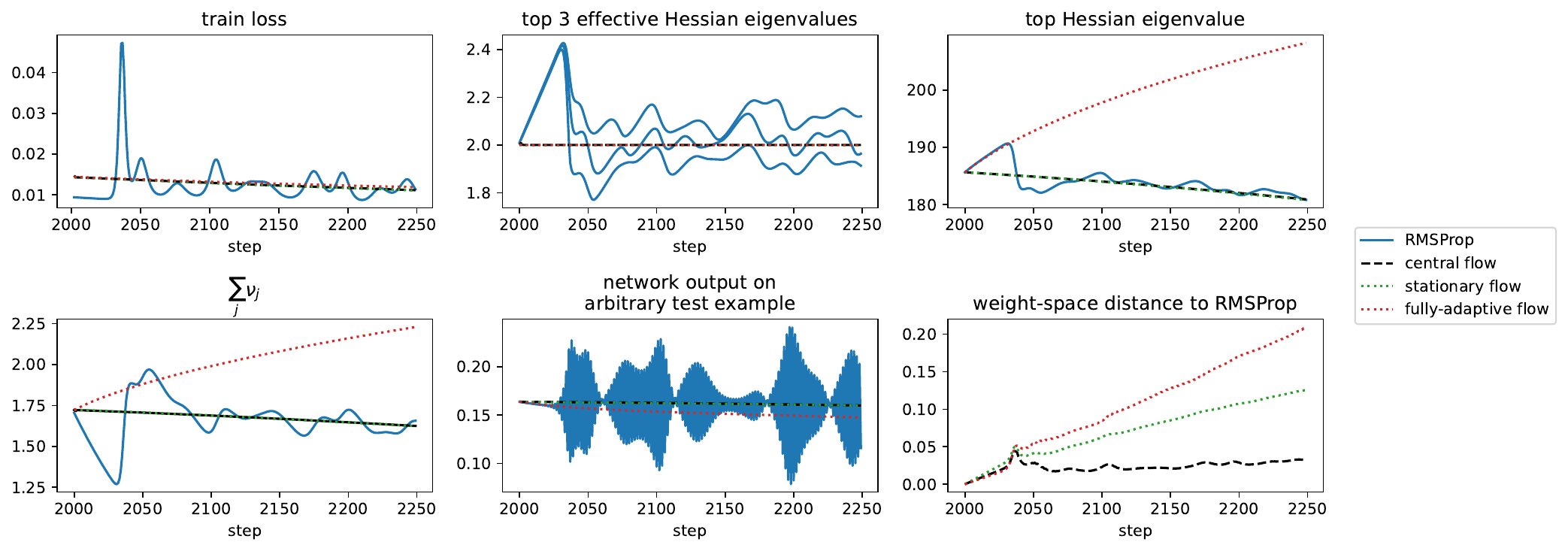}
    \caption{\textbf{Stationary flow can be accurate over moderate timescales}.  Starting at a point during training when $\nu$ has reached stationarity, we run the stationary flow \cref{eq:rmsprop-stationary-flow} (in green) alongside both RMSProp (in blue) and the central flow (in black).  As a baseline, we also run an ablated version of the stationary flow (in red) which adapts using the stationary $\nu$ but does not implicitly regularize curvature.  Observe that the stationary flow accurately tracks the central flow (and, in turn, RMSProp), whereas the baseline is a worse approximation.  This experiment uses a CNN trained on a subset of CIFAR-10 using MSE loss with hyperparameters $\eta = $ 4e-05, $\beta_2 = $ 0.99, and $\epsilon = $1e-8.}
    \label{fig:stationary-flow-cnn-mse}
\end{figure}

\begin{figure}[H]
\centering
\includegraphics[width=\linewidth]{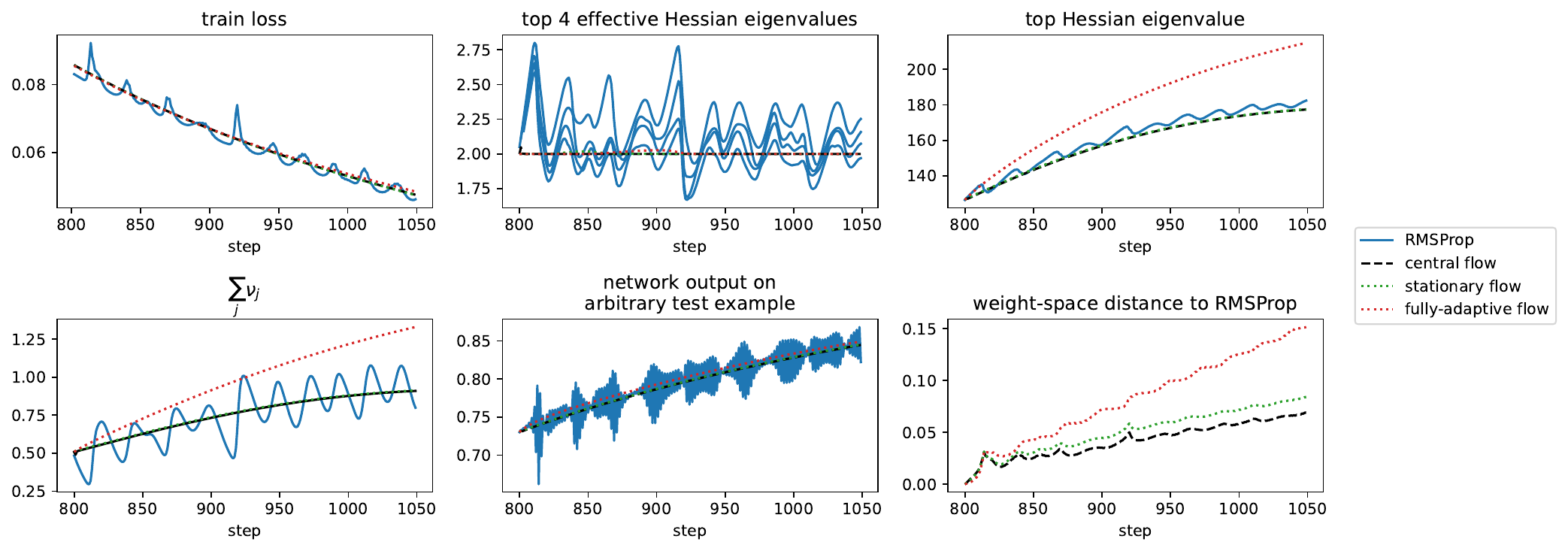}
    \caption{\textbf{Stationary flow can be accurate over moderate timescales}.  Same as \Cref{fig:stationary-flow-cnn-mse}, but using a Transformer trained on a synthetic sequence task using MSE loss with hyperparameters $\eta = $ 1e-4, $\beta_2 = $ 0.95, and $\epsilon = $1e-8.}
    \label{fig:stationary-flow-transformer-mse}
\end{figure}
\end{letterfigures}

\clearpage

%% file: appendix-4-bulk-experiments.tex
\newpage
\section{Bulk Experimental Data}
\label{sec:bulk-experiments}
This section contains the bulk experimental data from our central flow experiments:
\begin{itemize}
    \item \Cref{sec:bulk-experiments-gd} contains \gd experiments. See \Cref{fig:experiments:gd:example} for a fully annotated example of a \gd trajectory.
    \item \Cref{sec:bulk-experiments-scalar-rmsprop} contains \rmsnorm experiments. See \Cref{fig:experiments:scalar-rmsprop:example} for a fully annotated example of a \rmsnorm trajectory.
    \item \Cref{sec:bulk-experiments-rmsprop} contains \rmsprop experiments. See \Cref{fig:experiments:rmsprop:example} for a fully annotated example of a \rmsprop trajectory.
\end{itemize}

\subsection{Gradient Descent}
\label{sec:bulk-experiments-gd}
\begin{figure}[H]
    \centering
    \includegraphics[width=0.8\linewidth]{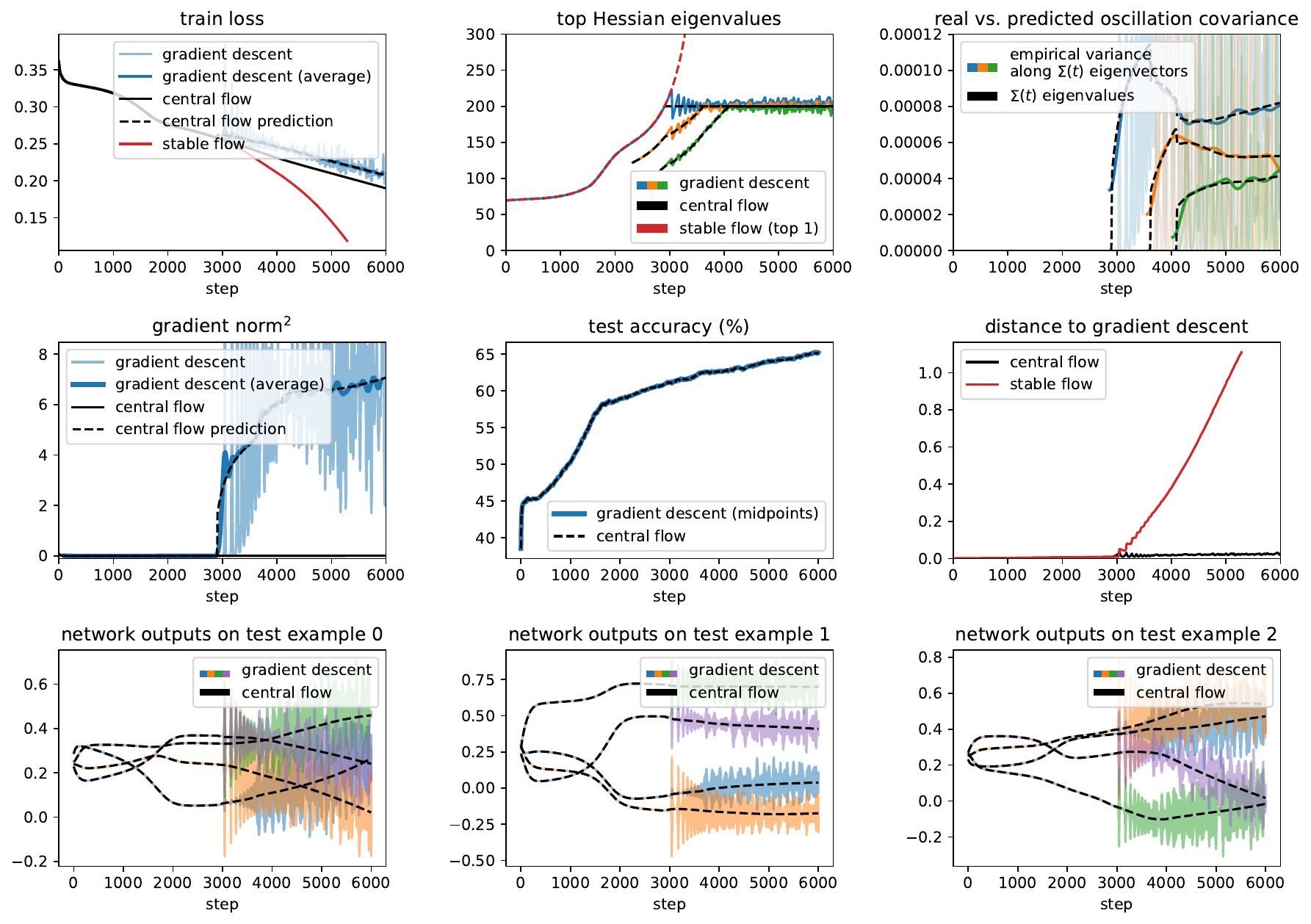}
    \caption{\small \textbf{Annotated example of a gradient descent experiment.} Using gradient descent with $\eta = 2/200$, we train a ViT on a subset of CIFAR-10.  The central flow (black) accurately models the trajectory of gradient descent (blue), whereas gradient flow (red) takes a different path.  As described in \Cref{appendix:experimental-details:implementation}, we terminate gradient flow once the sharpness gets too high.
    \\\textbf{Top left}: The loss along the central flow (solid black) decreases monotonically, whereas the loss along the \gd trajectory (light blue) behaves non-monotonically once the dynamics enter EOS.   While the \gd loss is higher than the central flow loss, the central flow can accurately predict the \emph{time-averaged} loss along the \gd trajectory, using  \cref{eq:gd-predict-loss} (dashed black); this can be seen to match the empirical time average of the \gd loss curve (dark blue).  Finally, the train loss along gradient flow (in red) decreases faster, because it follows a different, unregularized path.  \\\textbf{Top center}: We plot the top three Hessian eigenvalues under gradient descent (colors) and under the central flow (black).  Under GD, the top Hessian eigenvalues equilibrate around $2/\eta$; under the central flow they are fixed exactly at $2/\eta$. In red, we plot the top Hessian eigenvalue under the gradient flow, which rises beyond $2/\eta$. Note that for GD, we report the Hessian eigenvalues at the second-order midpoints (see \Cref{appendix:experimental-details:implementation}), rather than at the iterates themselves, as this makes for clearer plots.
    \\\textbf{Top right}: We show that the central flow's $\Sigma(t)$ accurately predicts the covariance of the oscillations.  In black, we plot the nonzero eigenvalues of $\Sigma(t)$; the number is always the same as the number of Hessian eigenvalues at $2/\eta$.   In faint colors, we plot the squared magnitude of the displacement between \gd and the central flow along each eigenvector of $\Sigma(t)$. In thick colors, we plot the time-averages of these displacements, i.e. the empirical variance of the oscillations along each eigenvector of $\Sigma(t)$.  Observe that the eigenvalues of $\Sigma(t)$ accurately predict the instantaneous variance of the oscillations along the corresponding eigenvectors, as we expect from \cref{eq:gd-predict-oscillation-variance}.
    \\\textbf{Middle left}: We plot the squared gradient norm along the \gd trajectory (light blue) and its empirical time-average (dark blue). In dashed black, we plot the central flow's prediction \cref{eq:gd-predict-gradient-norm-sq} for the time-averaged squared gradient norm along the trajectory; this prediction is quite accurate. In solid black, we plot the squared gradient norm along the central flow, which is much smaller, indicating that that most of the gradient norm comes from the oscillations.
    \\\textbf{Middle center}: We plot the test accuracy under gradient descent (blue) and the central flow (black).  For \gd, we report the test accuracy at second-order midpoints, as this removes much of the oscillations. Because the central flow matches the \gd trajectory, the test accuracy is nearly the same across both trajectories.
    \\\textbf{Middle right}: The Euclidean distance in weight space between \gd and the central flow (black) stays small over time, indicating that these two trajectories stay close.  By contrast, the distance between \gd and the \emph{gradient} flow (red) grows rapidly once the dynamics enter EOS.
    \\\textbf{Bottom row}: We show the network's final-layer predictions on three arbitrary examples.  Under \gd (colors) these predictions oscillate due to the oscillations in weight space.  Under the central flow (black), the predictions evolve smoothly while still following the same macroscopic path.} 
    \label{fig:experiments:gd:example}
\end{figure}

\input{images/bulk-gd/figures-mse}
\clearpage
\input{images/bulk-gd/figures-ce}
\clearpage

\subsection{Scalar RMSProp}
\label{sec:bulk-experiments-scalar-rmsprop}

\begin{figure}[H]
    \centering
    \includegraphics[width=0.8\linewidth]{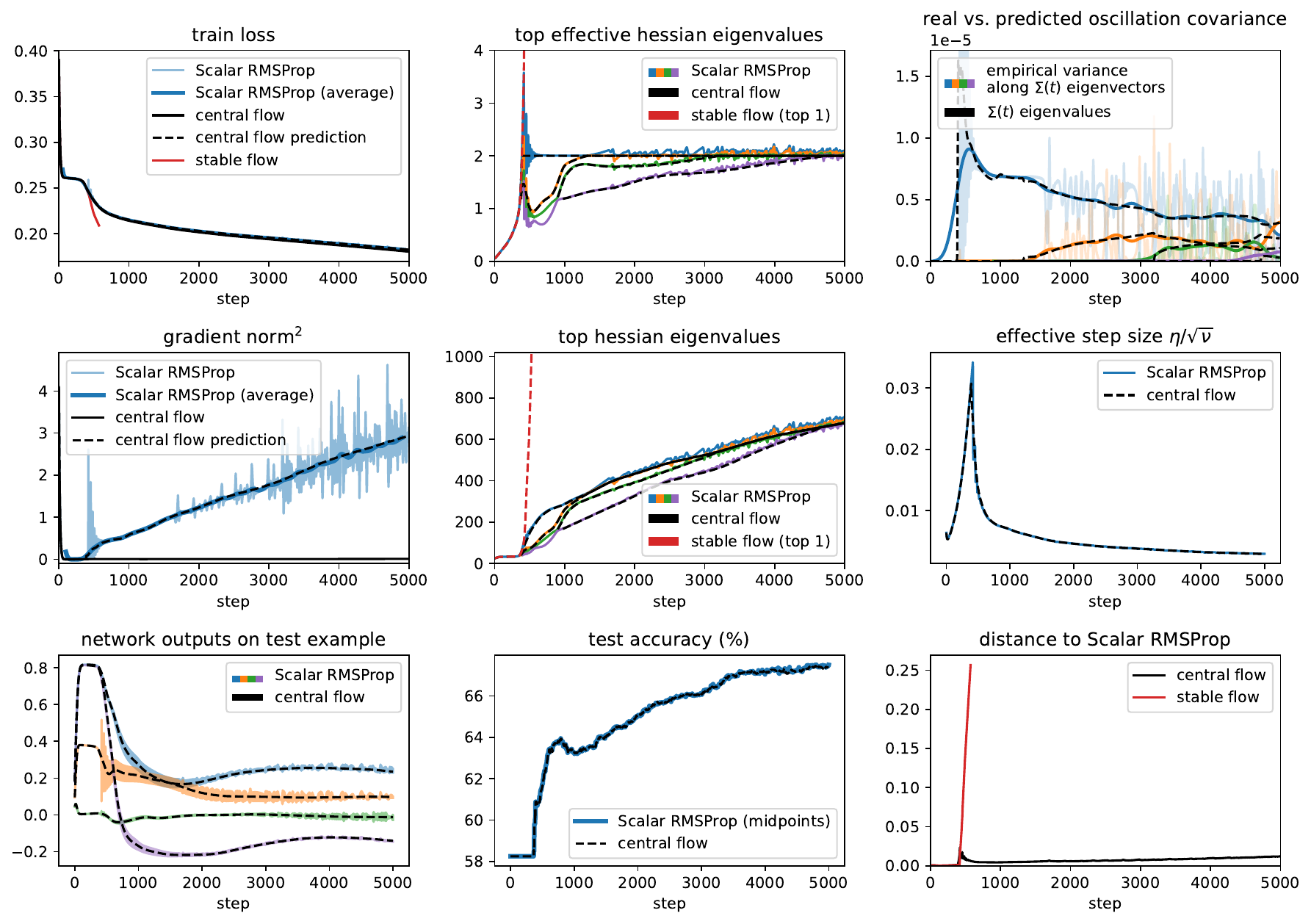}
    \caption{\textbf{Annotated example of a \protect \rmsnorm experiment.} Using Scalar RMSProp with $\eta = 2/400$, $\beta_2 = 0.99$, and bias correction, we train a Mamba network on a synthetic sequence prediction task with MSE loss. The central flow (black) accurately models the long-term trajectory of Scalar RMSProp (blue), whereas the stable flow (red) takes a different path.  As described in \Cref{appendix:experimental-details:implementation}, we terminate the stable flow once the effective sharpness gets too high.
    \vspace{0.5em}
    \\\textbf{Top left}: See \Cref{fig:experiments:gd:example} caption.  The central flow's prediction for the time-averaged loss is given by \cref{eq:rmsnorm-predict-loss}.
    \\\textbf{Top center}:  We plot the top several eigenvalues of the effective Hessian $\tfrac{\eta}{\sqrt{\nu}} H(w)$ under both \rmsnorm (colors) and its central flow (dashed black).  Under \rmsnorm, these eigenvalues equilibrate around the critical threshold 2, whereas under the central flow they are fixed exactly at 2.  We also plot the top eigenvalue under the ``stable flow'' baseline (red), and this increases far above 2.
    \\\textbf{Top right}: See \Cref{fig:experiments:gd:example} caption.  This plot is validating \cref{eq:rmsnorm-predict-oscillation-variance}.
    \\\textbf{Middle left}:  See \Cref{fig:experiments:gd:example} caption. The central flow's prediction for the time-average is given by \cref{eq:rmsnorm-predict-gradient-norm-sq}.
    \\\textbf{Middle center}: We plot the top several eigenvalues of the ``raw'' Hessian $H(w)$, under both \rmsnorm (colors) and the central flow (dashed black). These evolve throughout training, even as top eigenvalues of the \emph{effective} Hessian are equilibrating at the critical threshold (top center).  In red, we plot the top Hessian eigenvalue under the stable flow.  
    \\\textbf{Middle right}: We plot the effective step size $\eta / \sqrt \nu$ under both \rmsnorm (blue) and the central flow (dashed black). This effective step size oscillates under \rmsnorm, but varies smoothly under the central flow.
    \\\textbf{Bottom left}: We show the network's final-layer predictions on an arbitrary example.  Under Scalar RMSProp (colors) these predictions oscillate due to the oscillations in weight space.  Under the central flow (dashed black), the predictions evolve smoothly while following the same macroscopic path.
    \\\textbf{Bottom center}: See \Cref{fig:experiments:gd:example} caption.
    \\\textbf{Bottom right}: See \Cref{fig:experiments:gd:example} caption.
    }
    \label{fig:experiments:scalar-rmsprop:example}
\end{figure}

\clearpage
\input{images/bulk-scalar-rmsprop/figures-mse}
\clearpage
\input{images/bulk-scalar-rmsprop/figures-ce}

\subsection{RMSProp}
\label{sec:bulk-experiments-rmsprop}

\begin{figure}[H]
    \centering
    \includegraphics[width=0.8\linewidth]{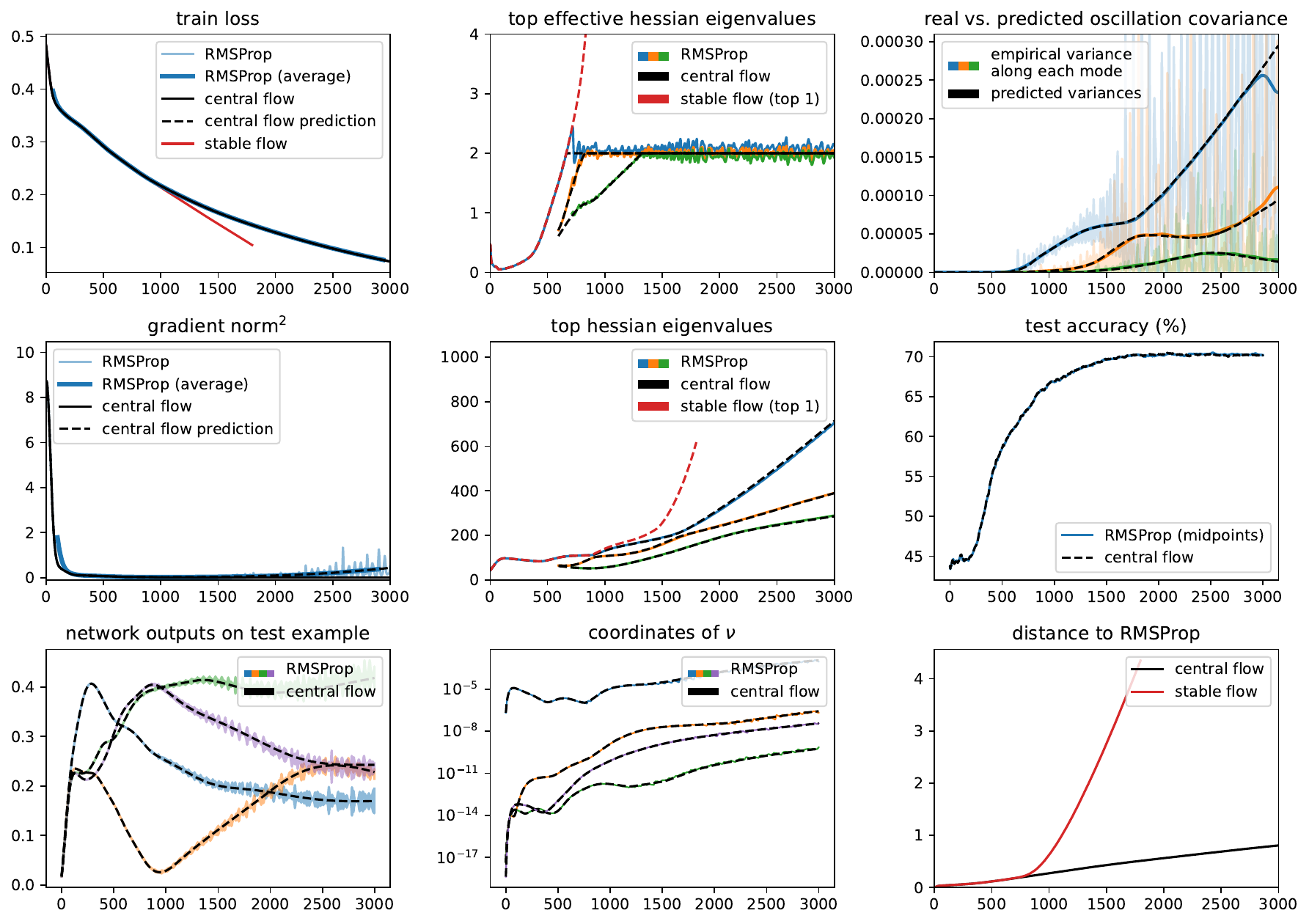}
    \caption{\textbf{Annotated example of a \protect \rmsprop experiment.} Using RMSProp with $\eta = 2 \times 10^{-5}$, $\beta_2 = 0.99$, $\epsilon = 10^{-8}$ and bias correction, we train a ResNet on a subset of CIFAR-10 with MSE loss. The central flow (black) accurately models the long-term trajectory of RMSProp (blue), whereas the stable flow (red) takes a different path.  As described in \Cref{appendix:experimental-details:implementation}, we terminate the stable flow once the effective sharpness gets sufficiently large.
    \vspace{0.5em}
    \\\textbf{Top left}: See \Cref{fig:experiments:gd:example} caption.  The central flow's prediction for the time-averaged loss is given by \cref{eq:rmsprop-predict-loss}.
    \\\textbf{Top center}:  We plot the top several eigenvalues of the effective Hessian $\diag \qty[\tfrac{\eta}{\sqrt{\nu}}] H(w)$ under both \rmsprop (colors) and its central flow (dashed black).  Under \rmsprop, these eigenvalues equilibrate around the critical threshold 2, whereas under the central flow they are fixed exactly at 2.  We also plot the top eigenvalue under the ``stable flow'' baseline (red), which increases far above 2.
   \\\textbf{Top right}: We show that the central flow accurately predicts the covariance of the oscillations.  In particular, we show that each nonzero eigenvalue $\lambda_i(t)$ of $P(t)^{1/2} \, \Sigma(t) \, P(t)^{1/2}$ accurately predicts the $P$-whitened variance of oscillations along the corresponding eigenvector $v_i(t)$, as we expect from \cref{eq:rmsprop-predict-oscillation-variance}. In black, we plot the nonzero eigenvalues of $P(t)^{1/2} \, \Sigma(t) \, P(t)^{1/2}$. 
   In faint colors, we plot the squared magnitude of the $P$-whitened displacement between \rmsprop and the central flow along each eigenvector $v_i(t)$ (see \cref{eq:rmsprop-predict-oscillation-variance}), and in thick colors, we plot the time-averages of these displacements, i.e. the empirical variances of the oscillations.  Observe that each eigenvalue accurately predicts the variance of the oscillations along the corresponding eigenvector.
    \\\textbf{Middle left}:  See \Cref{fig:experiments:gd:example} caption. The central flow's prediction for the time-average is given by \cref{eq:rmsprop-predict-gradient-norm-sq}.
    \\\textbf{Middle center}: See \Cref{fig:experiments:scalar-rmsprop:example} caption.  The central flow's prediction for the time-averaged loss is given by \cref{eq:rmsprop-predict-loss}.
     \\\textbf{Middle right}: See \Cref{fig:experiments:gd:example} caption.
    \\\textbf{Bottom left}:  See \Cref{fig:experiments:scalar-rmsprop:example} caption.
    \\\textbf{Bottom center}:  We plot four arbitrary coordinates of the EMA $\nu$ under both RMSProp (colors) and the central flow (dashed black).  The coordinates of $\nu$ are oscillatory along the RMSProp trajectory, while varying smoothly along the central flow.
    \\\textbf{Bottom right}:   See \Cref{fig:experiments:gd:example} caption.
    }
    \label{fig:experiments:rmsprop:example}
\end{figure}

\clearpage
\input{images/bulk-rmsprop/figures-mse}
\clearpage
\input{images/bulk-rmsprop/figures-ce}

%% file: images/bulk-gd/figures-mse.tex
\begin{specialfigures}

    \begin{figure}[H]
        \centering
        \includegraphics[width=0.8\linewidth]{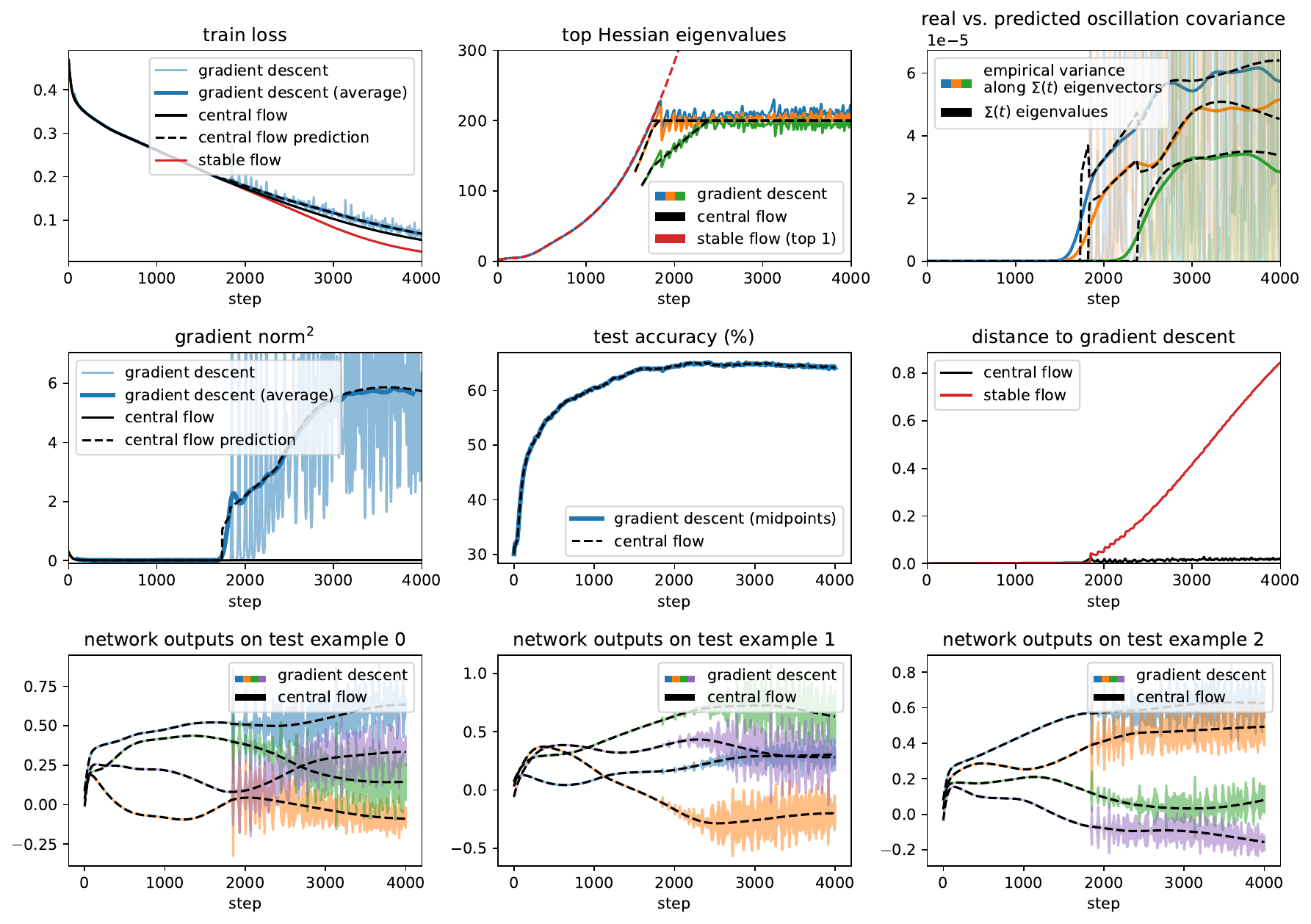}
        \caption{Gradient descent central flow for a CNN with MSE loss, $\eta=$ 0.005.}
        \label{fig:bulk-gd:mse-cnn-0}
    \end{figure}
                
    \begin{figure}[H]
        \centering
        \includegraphics[width=0.8\linewidth]{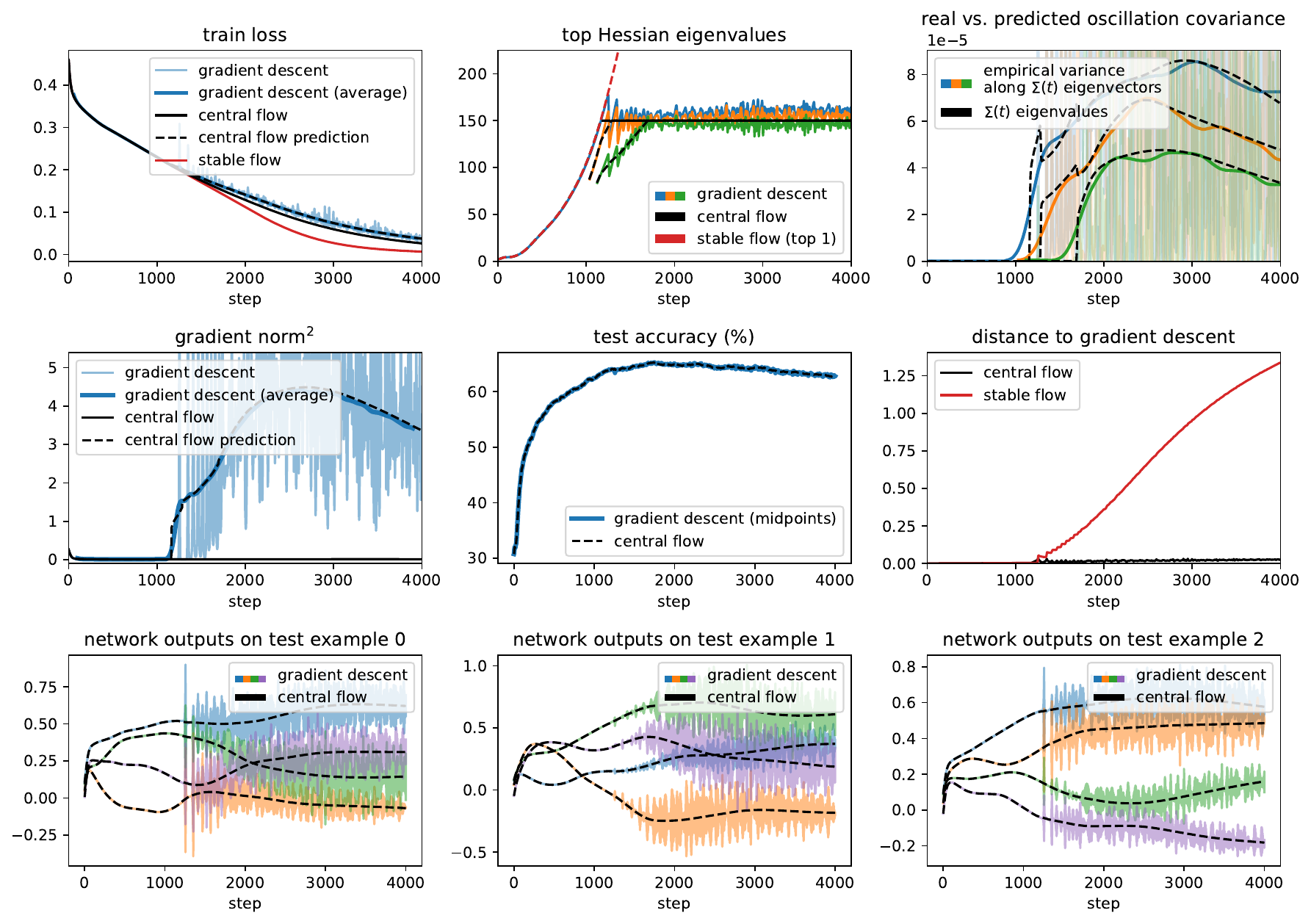}
        \caption{Gradient descent central flow for a CNN with MSE loss, $\eta=$ 0.006666.}
        \label{fig:bulk-gd:mse-cnn-1}
    \end{figure}
                
    \begin{figure}[H]
        \centering
        \includegraphics[width=0.8\linewidth]{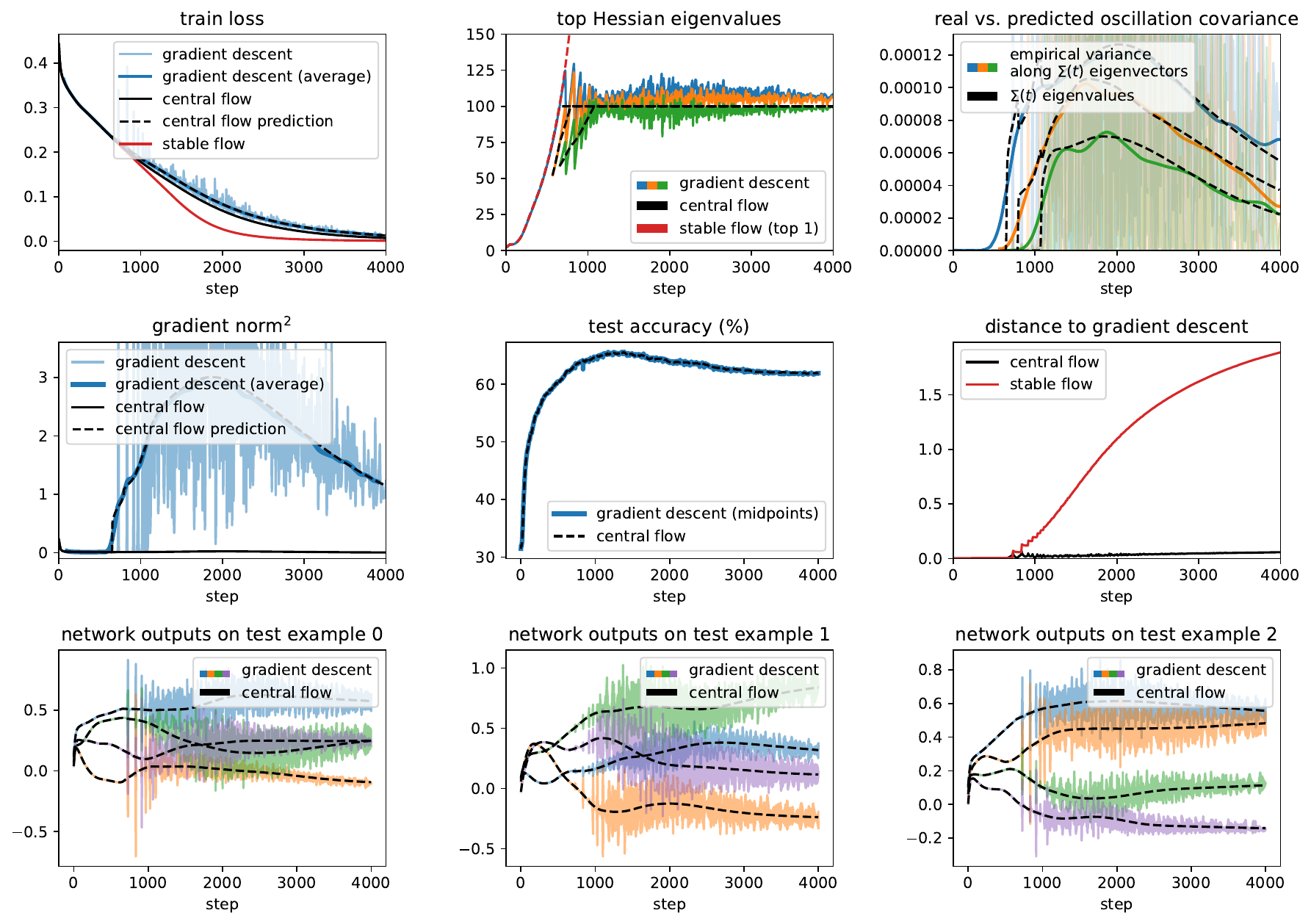}
        \caption{Gradient descent central flow for a CNN with MSE loss, $\eta=$ 0.01.}
        \label{fig:bulk-gd:mse-cnn-2}
    \end{figure}
                
    \begin{figure}[H]
        \centering
        \includegraphics[width=0.8\linewidth]{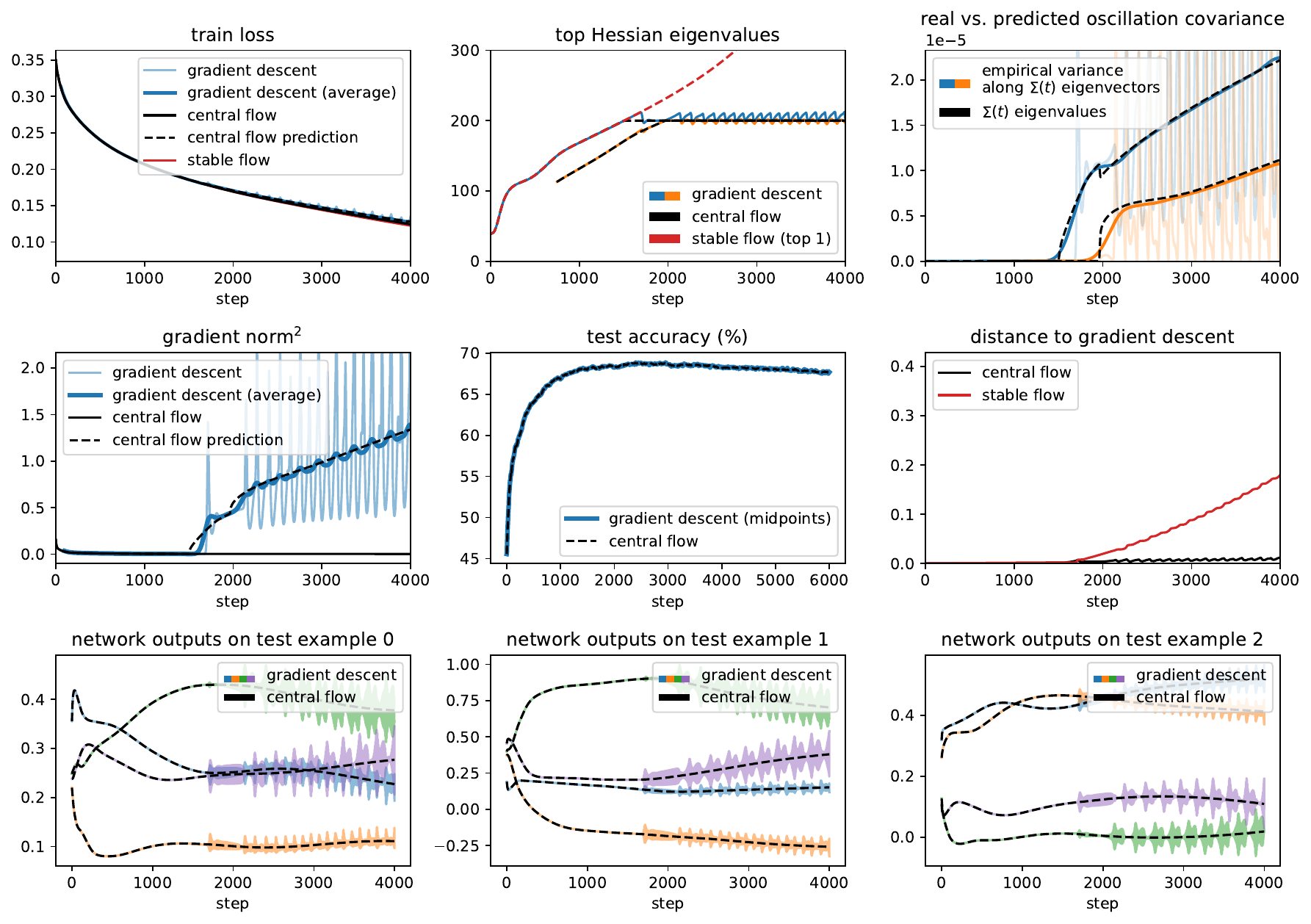}
        \caption{Gradient descent central flow for a ResNet with MSE loss, $\eta=$ 0.01.}
        \label{fig:bulk-gd:mse-resnet-0}
    \end{figure}
                
    \begin{figure}[H]
        \centering
        \includegraphics[width=0.8\linewidth]{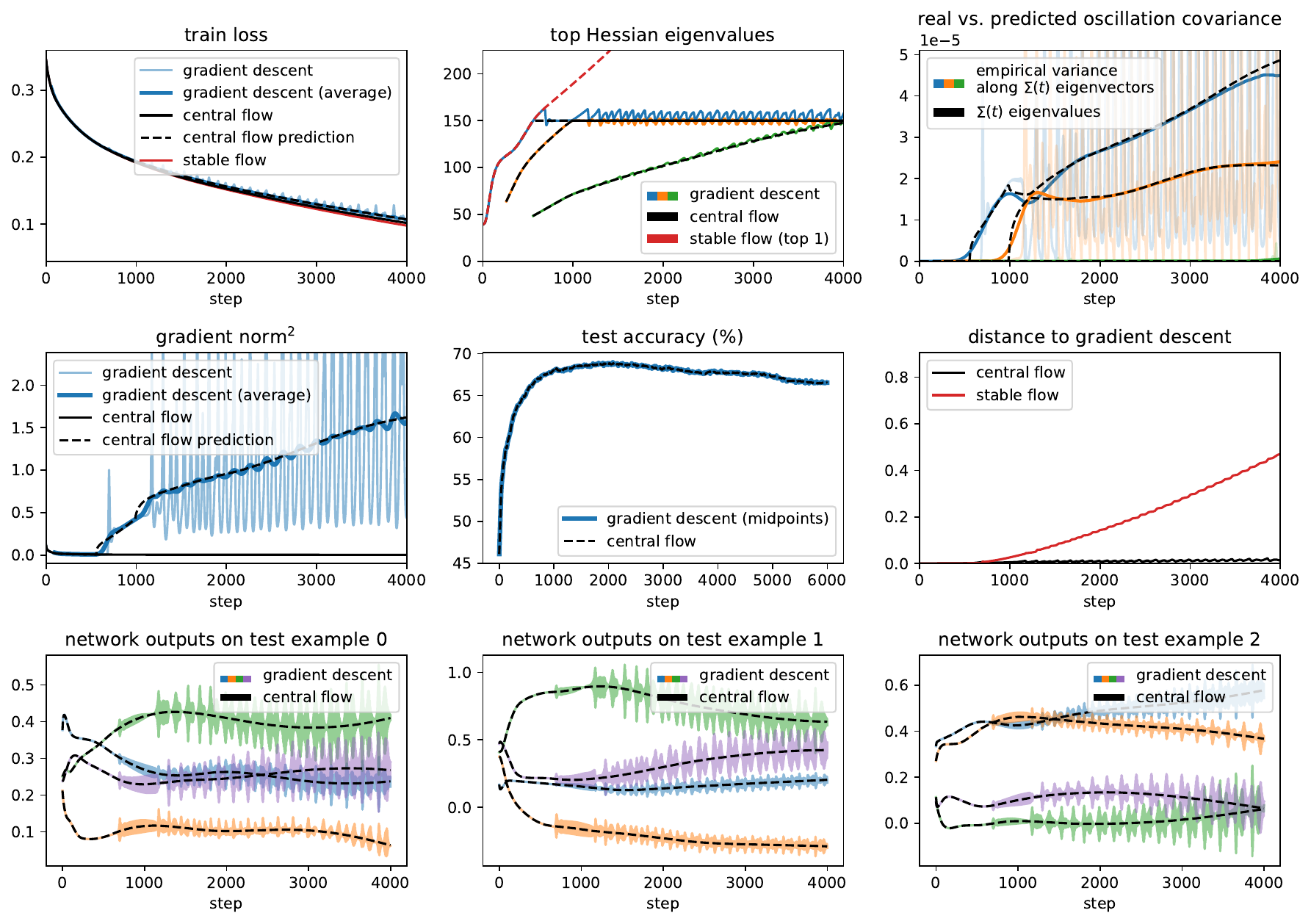}
        \caption{Gradient descent central flow for a ResNet with MSE loss, $\eta=$ 0.013333.}
        \label{fig:bulk-gd:mse-resnet-1}
    \end{figure}
                
    \begin{figure}[H]
        \centering
        \includegraphics[width=0.8\linewidth]{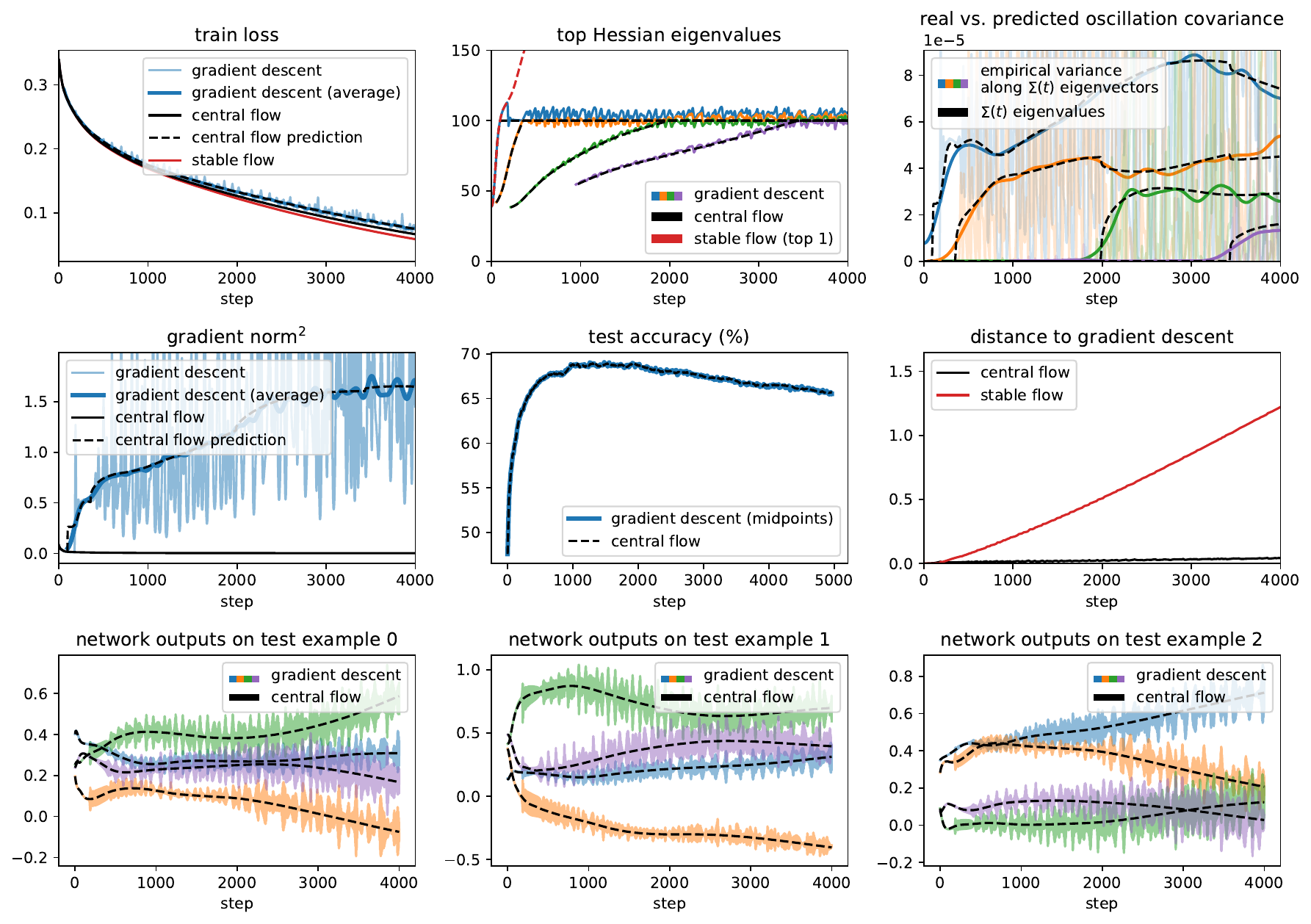}
        \caption{Gradient descent central flow for a ResNet with MSE loss, $\eta=$ 0.02.}
        \label{fig:bulk-gd:mse-resnet-2}
    \end{figure}
                
    \begin{figure}[H]
        \centering
        \includegraphics[width=0.8\linewidth]{images/bulk-gd/mse-vit-0.pdf}
        \caption{Gradient descent central flow for a ViT with MSE loss, $\eta=$ 0.01.}
        \label{fig:bulk-gd:mse-vit-0}
    \end{figure}
                
    \begin{figure}[H]
        \centering
        \includegraphics[width=0.8\linewidth]{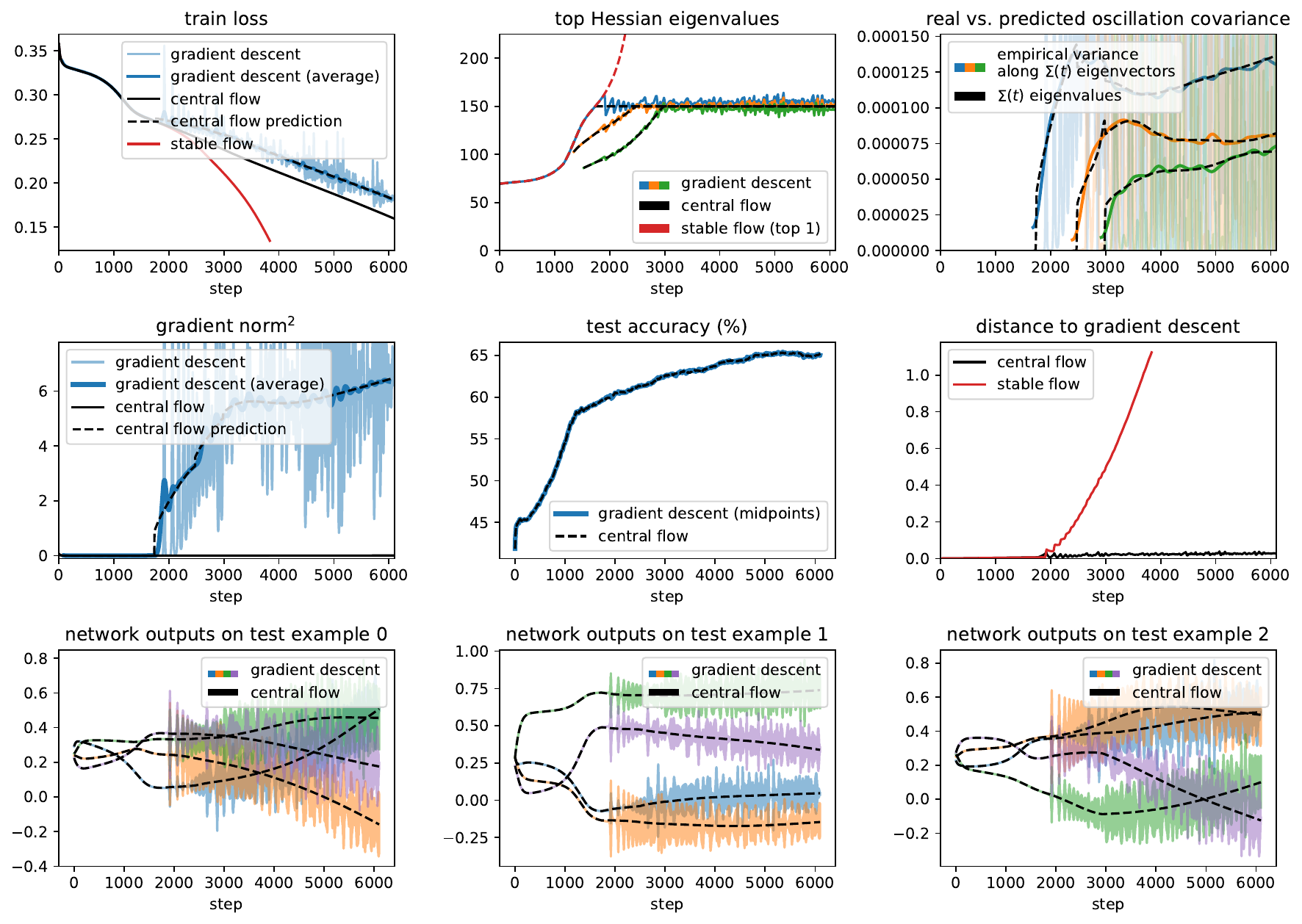}
        \caption{Gradient descent central flow for a ViT with MSE loss, $\eta=$ 0.013333.}
        \label{fig:bulk-gd:mse-vit-1}
    \end{figure}
                
    \begin{figure}[H]
        \centering
        \includegraphics[width=0.8\linewidth]{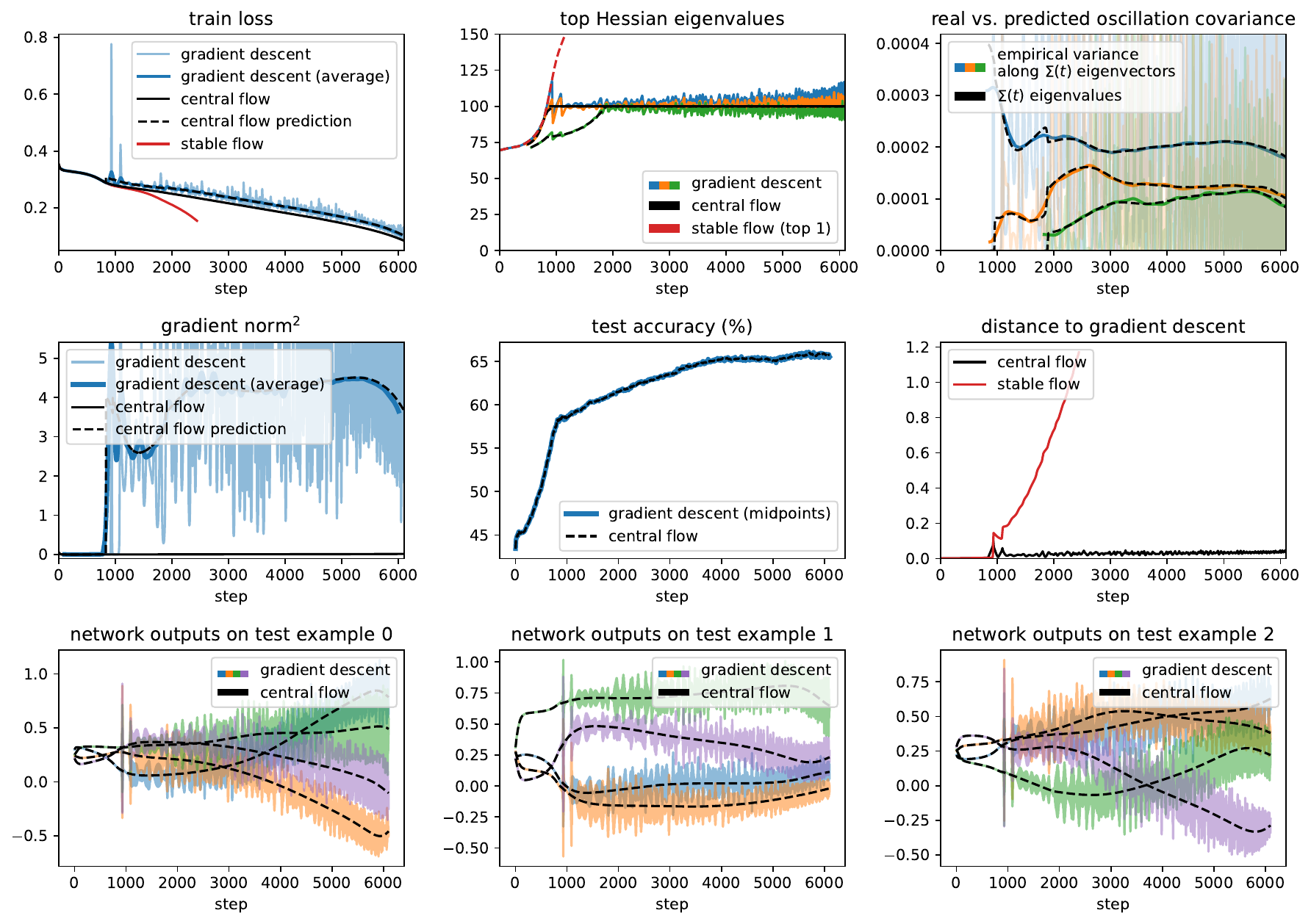}
        \caption{Gradient descent central flow for a ViT with MSE loss, $\eta=$ 0.02.}
        \label{fig:bulk-gd:mse-vit-2}
    \end{figure}
                
    \begin{figure}[H]
        \centering
        \includegraphics[width=0.8\linewidth]{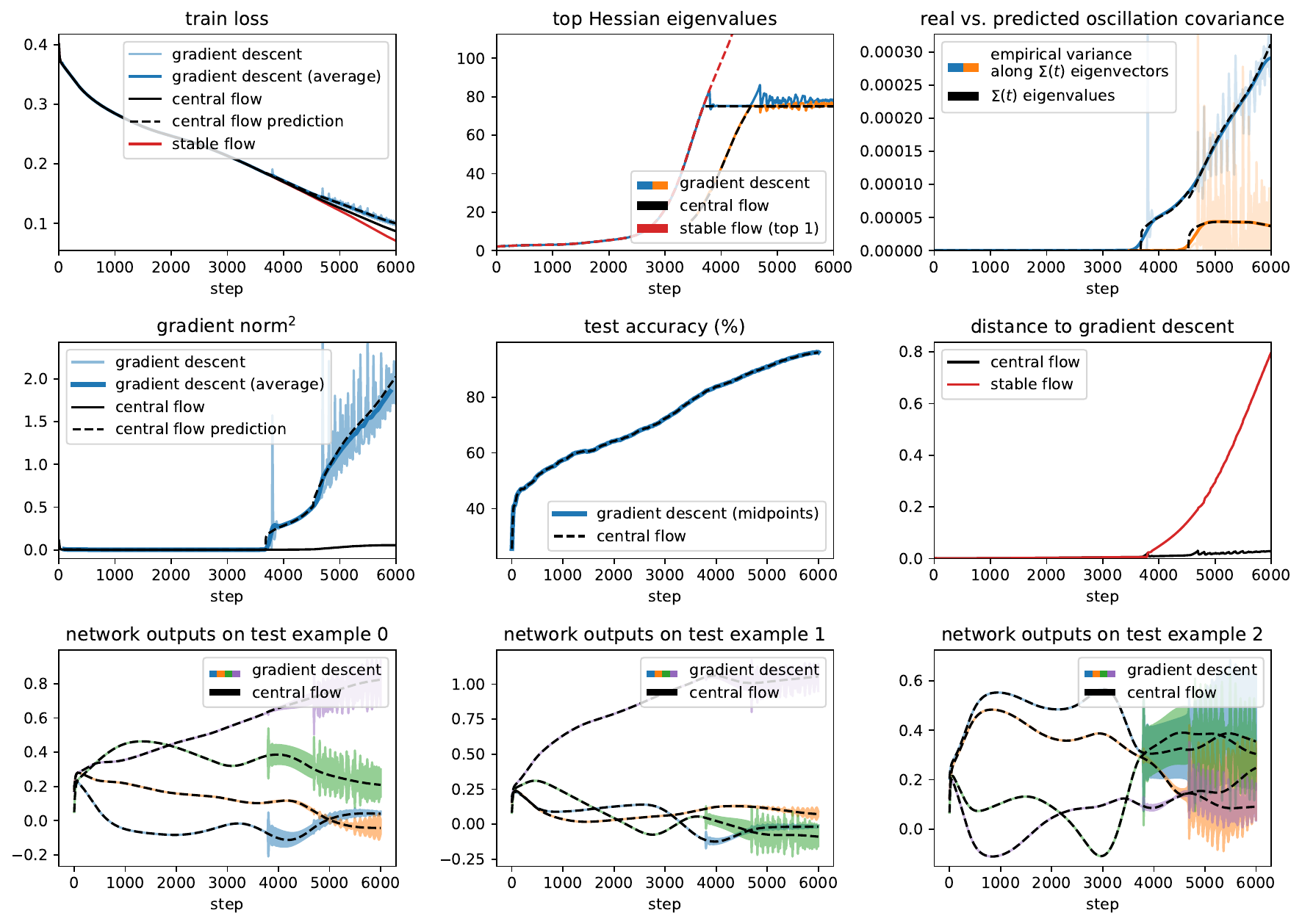}
        \caption{Gradient descent central flow for a LSTM with MSE loss, $\eta=$ 0.01333.}
        \label{fig:bulk-gd:mse-lstm-0}
    \end{figure}
                
    \begin{figure}[H]
        \centering
        \includegraphics[width=0.8\linewidth]{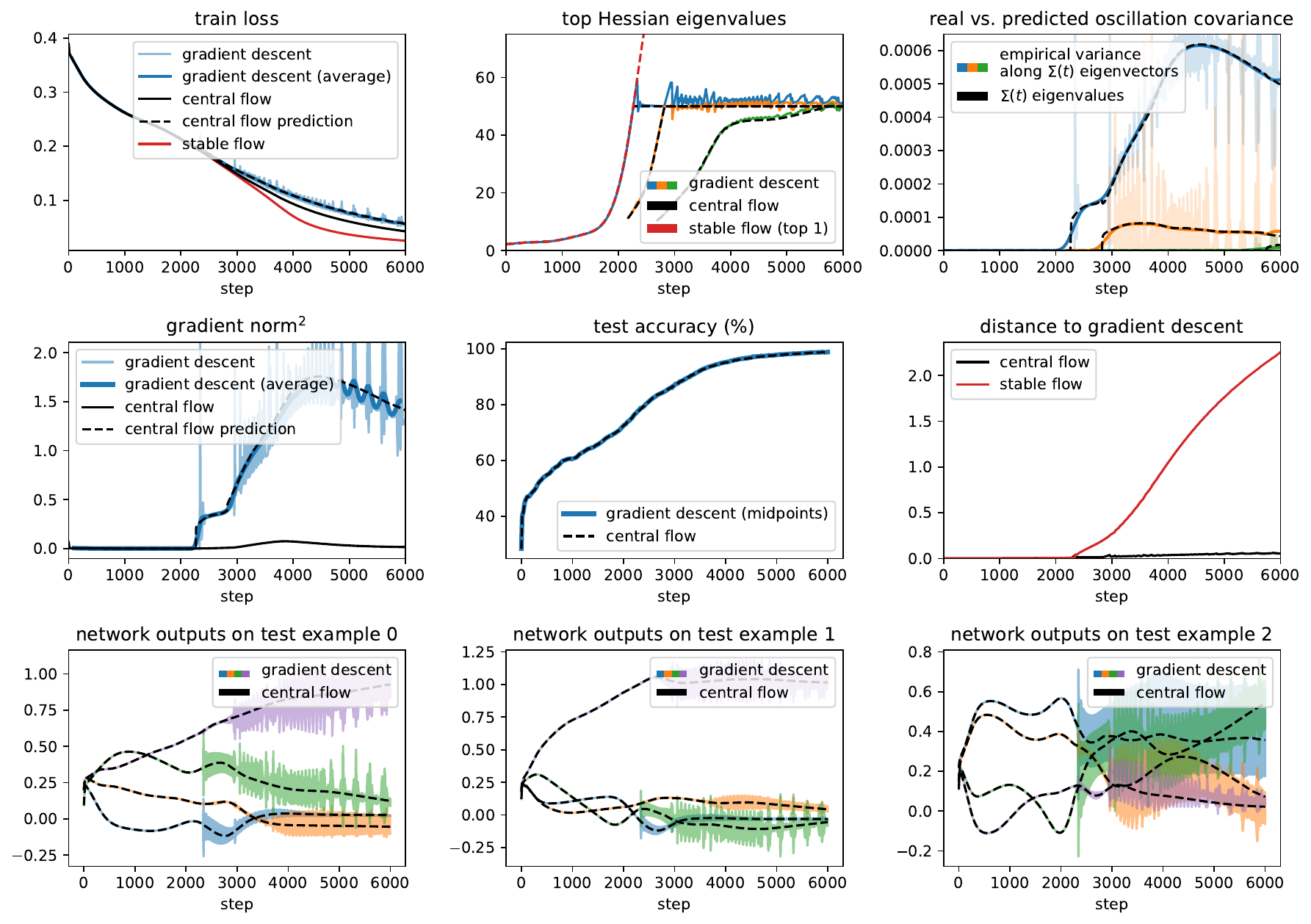}
        \caption{Gradient descent central flow for a LSTM with MSE loss, $\eta=$ 0.02.}
        \label{fig:bulk-gd:mse-lstm-1}
    \end{figure}
                
    \begin{figure}[H]
        \centering
        \includegraphics[width=0.8\linewidth]{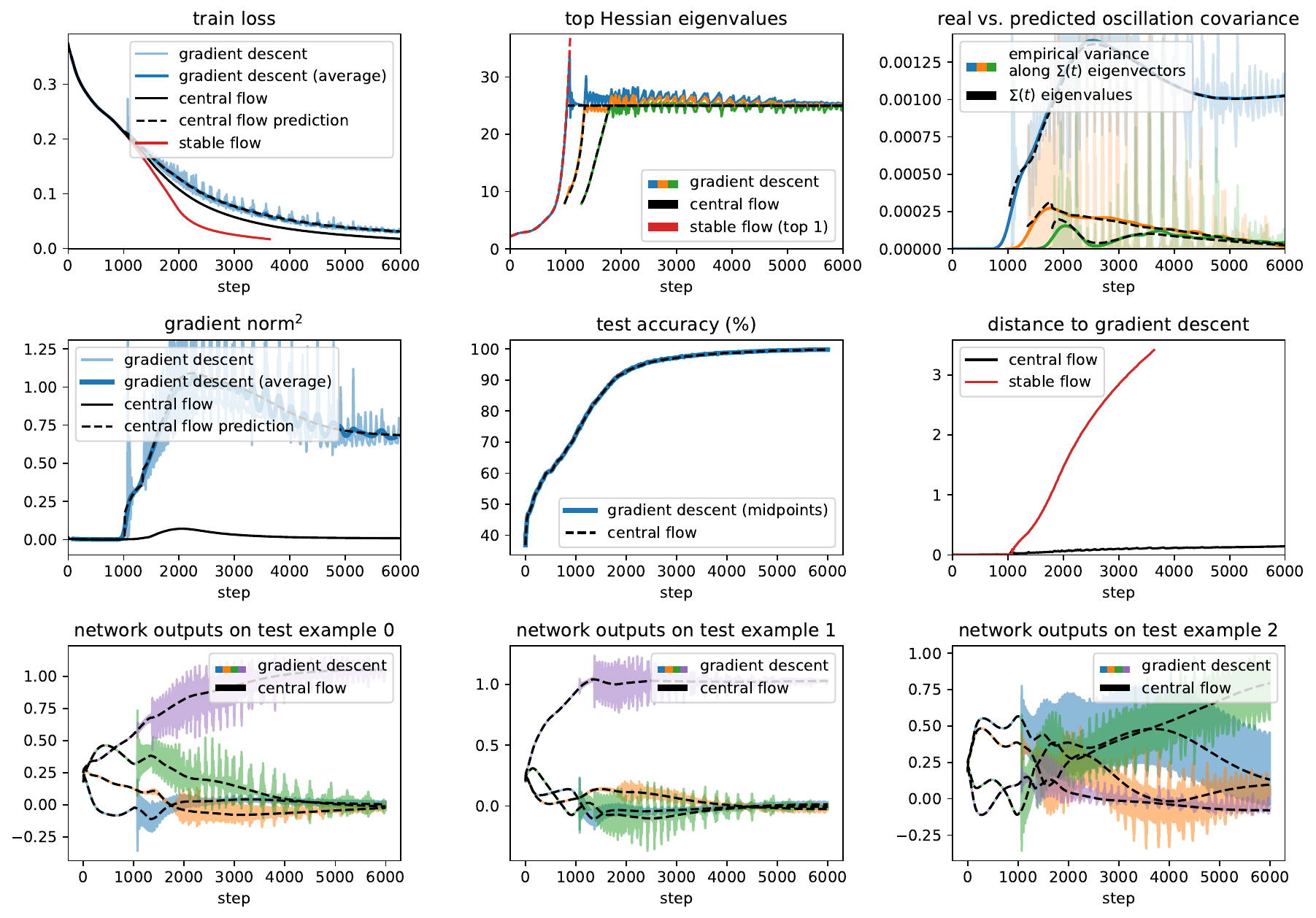}
        \caption{Gradient descent central flow for a LSTM with MSE loss, $\eta=$ 0.04.}
        \label{fig:bulk-gd:mse-lstm-2}
    \end{figure}
                
    \begin{figure}[H]
        \centering
        \includegraphics[width=0.8\linewidth]{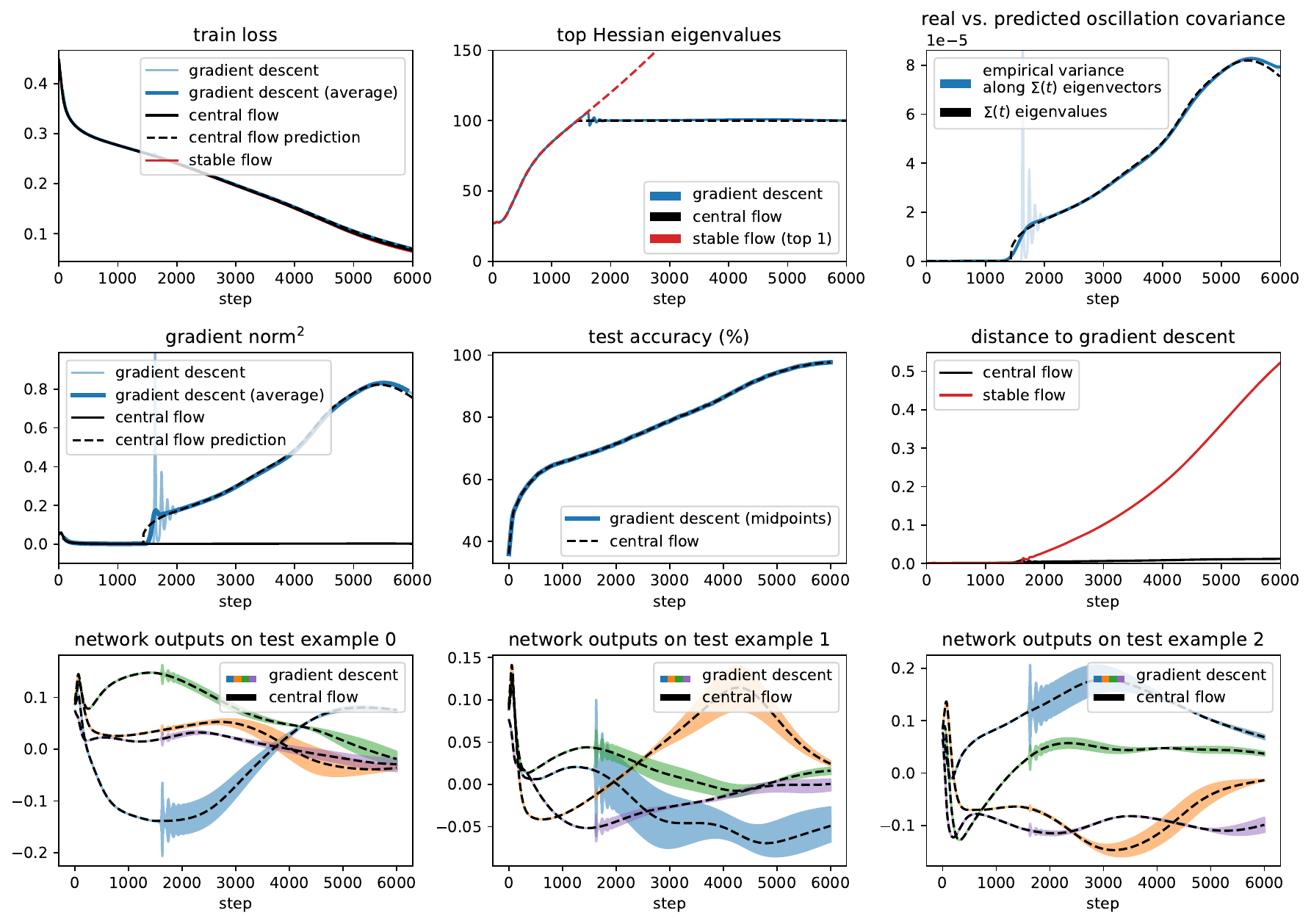}
        \caption{Gradient descent central flow for a Transformer with MSE loss, $\eta=$ 0.01.}
        \label{fig:bulk-gd:mse-transformer-0}
    \end{figure}
                
    \begin{figure}[H]
        \centering
        \includegraphics[width=0.8\linewidth]{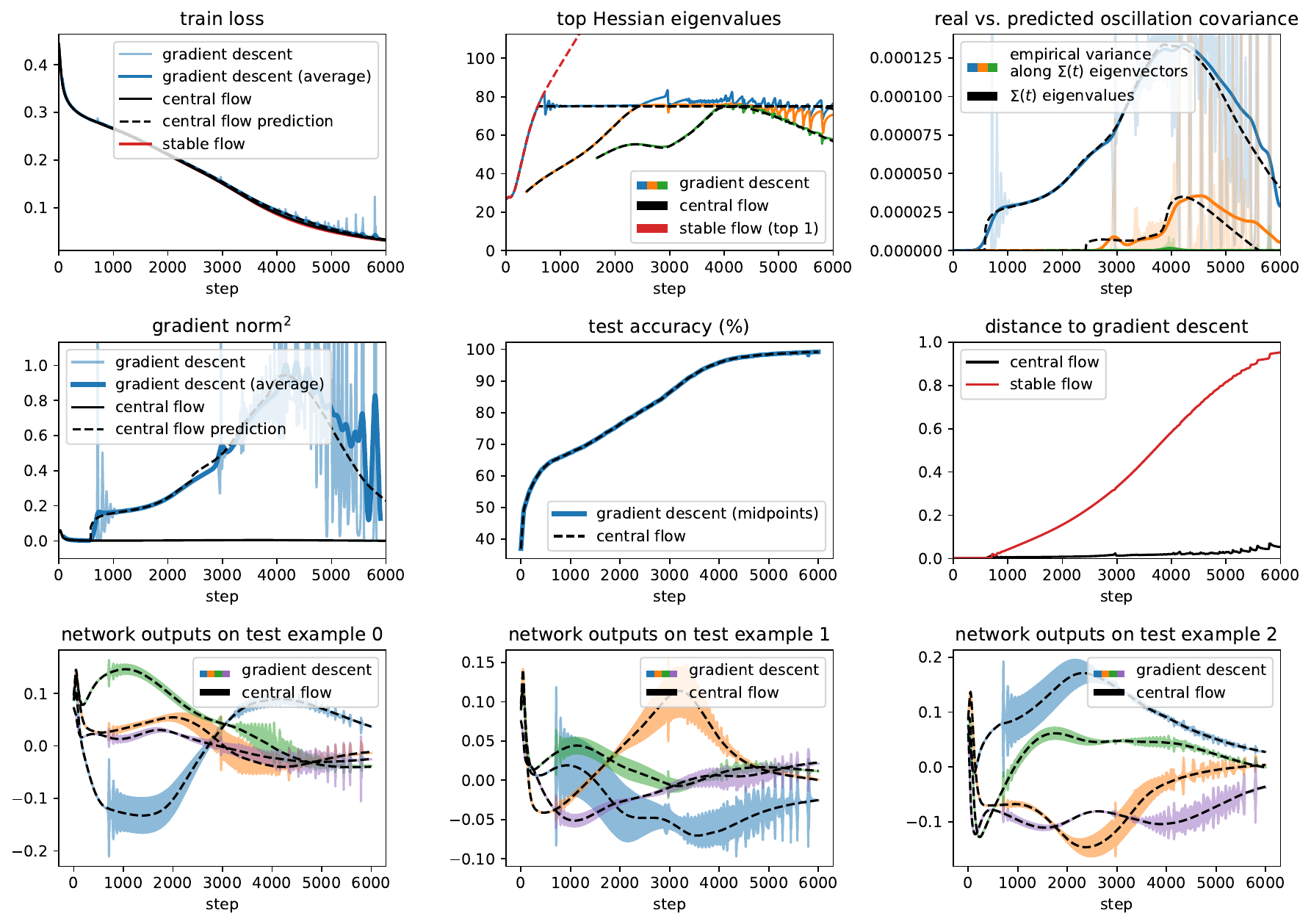}
        \caption{Gradient descent central flow for a Transformer with MSE loss, $\eta=$ 0.013333.}
        \label{fig:bulk-gd:mse-transformer-1}
    \end{figure}
                
    \begin{figure}[H]
        \centering
        \includegraphics[width=0.8\linewidth]{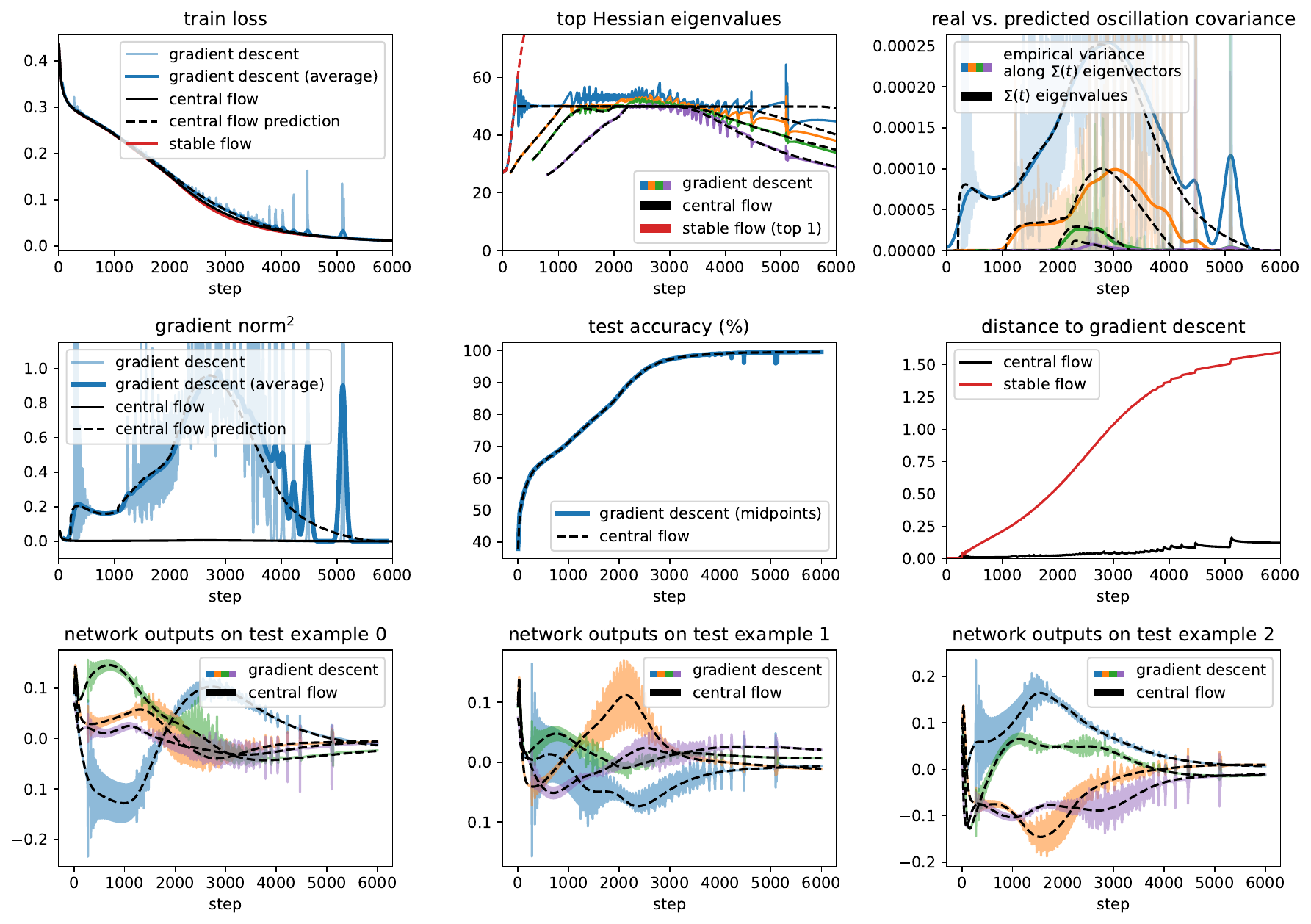}
        \caption{Gradient descent central flow for a Transformer with MSE loss, $\eta=$ 0.02.}
        \label{fig:bulk-gd:mse-transformer-2}
    \end{figure}
                
    \begin{figure}[H]
        \centering
        \includegraphics[width=0.8\linewidth]{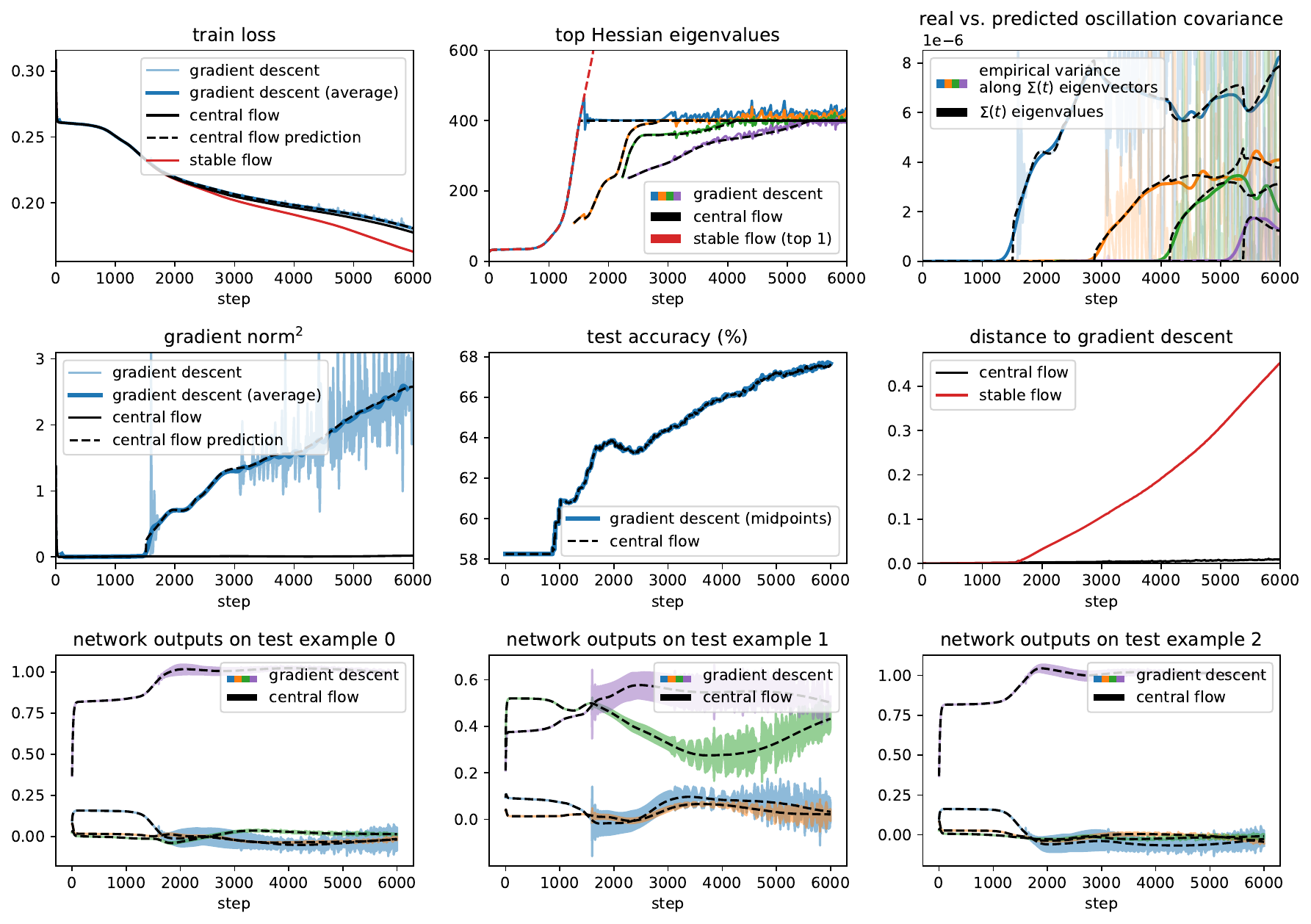}
        \caption{Gradient descent central flow for a Mamba with MSE loss, $\eta=$ 0.01.}
        \label{fig:bulk-gd:mse-mamba-0}
    \end{figure}
                
    \begin{figure}[H]
        \centering
        \includegraphics[width=0.8\linewidth]{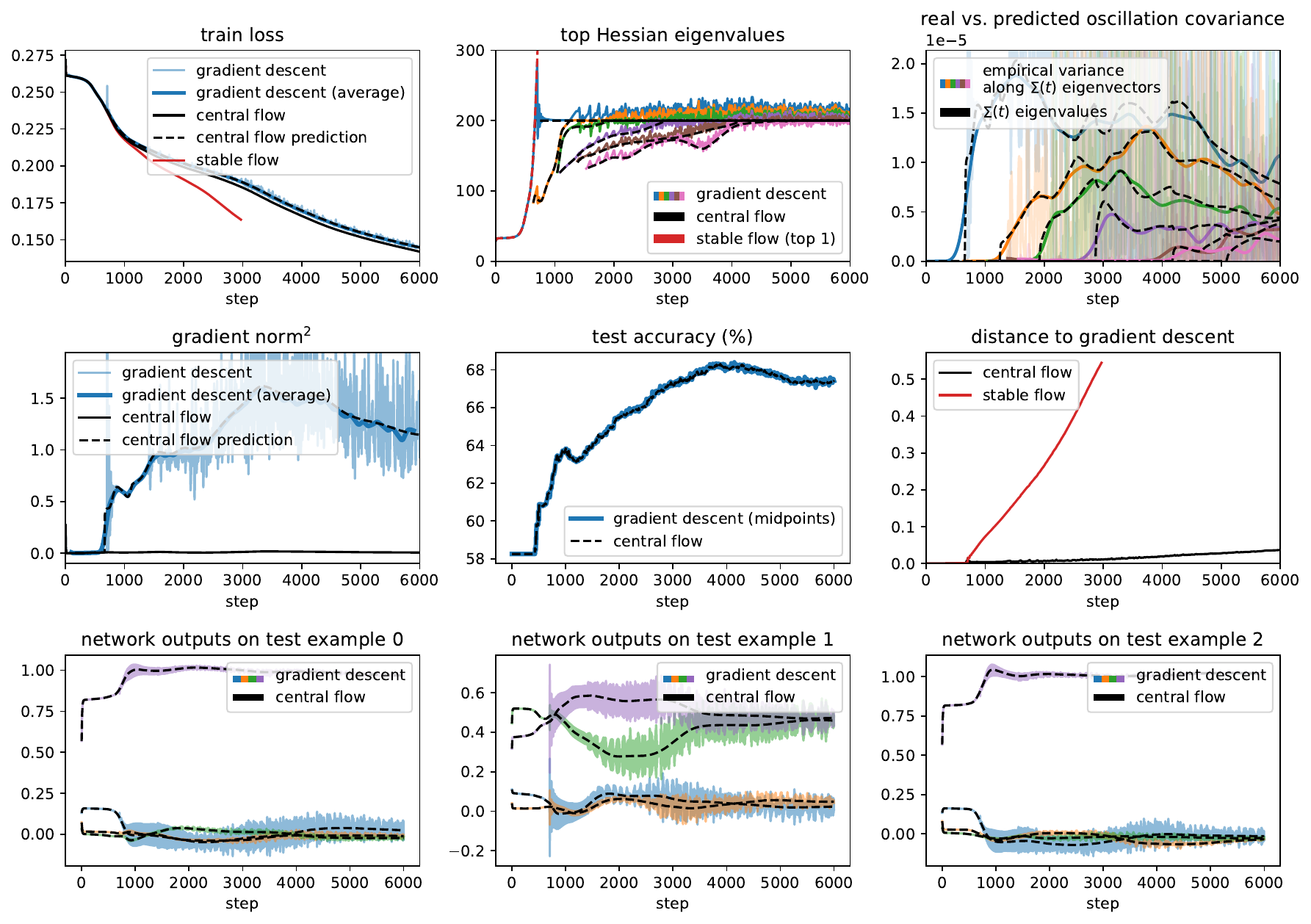}
        \caption{Gradient descent central flow for a Mamba with MSE loss, $\eta=$ 0.013333.}
        \label{fig:bulk-gd:mse-mamba-1}
    \end{figure}
                
    \begin{figure}[H]
        \centering
        \includegraphics[width=0.8\linewidth]{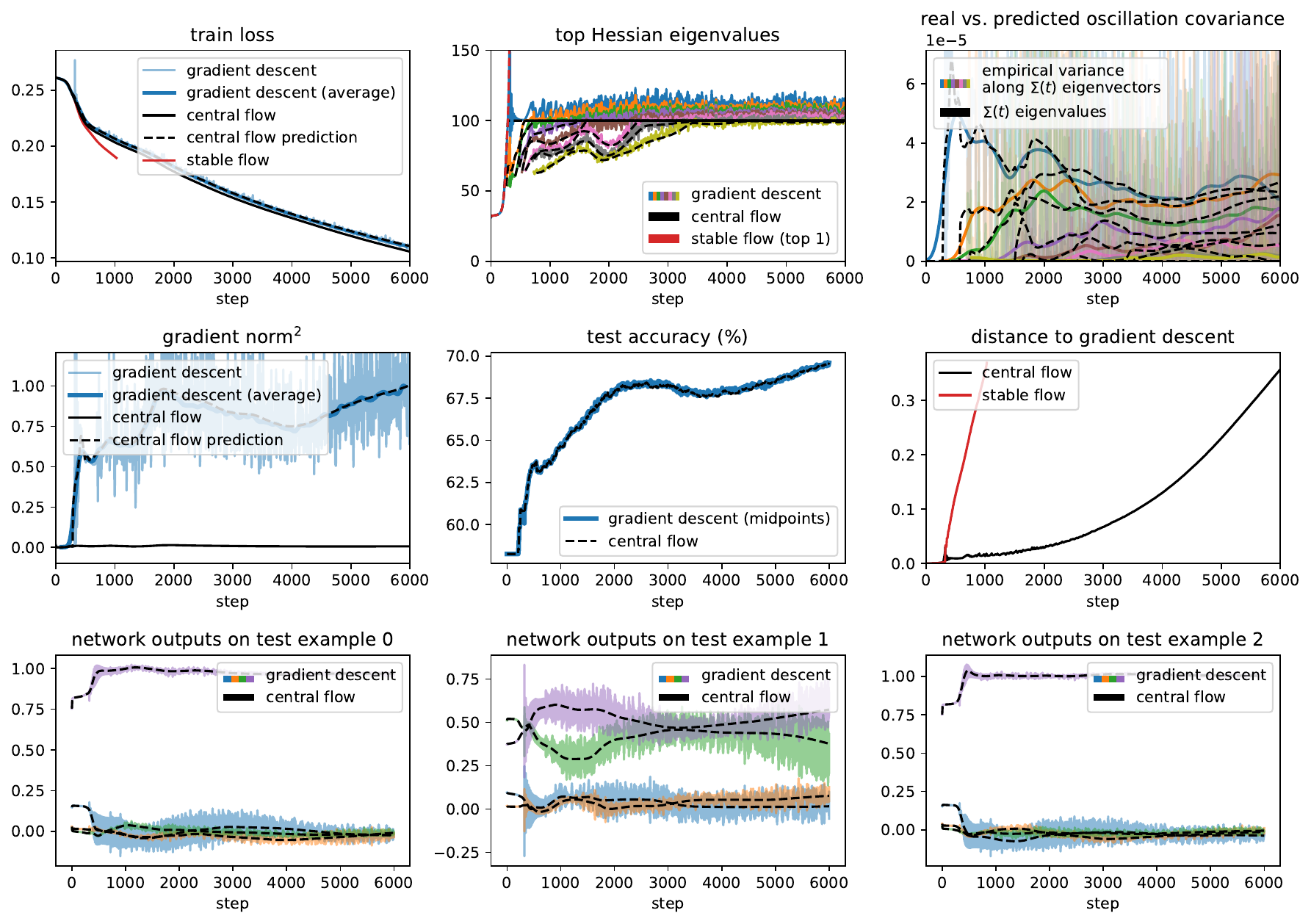}
        \caption{Gradient descent central flow for a Mamba with MSE loss, $\eta=$ 0.02.}
        \label{fig:bulk-gd:mse-mamba-2}
    \end{figure}
                \end{specialfigures}

%% file: images/bulk-gd/figures-ce.tex
\begin{specialfigures}

    \begin{figure}[H]
        \centering
        \includegraphics[width=0.8\linewidth]{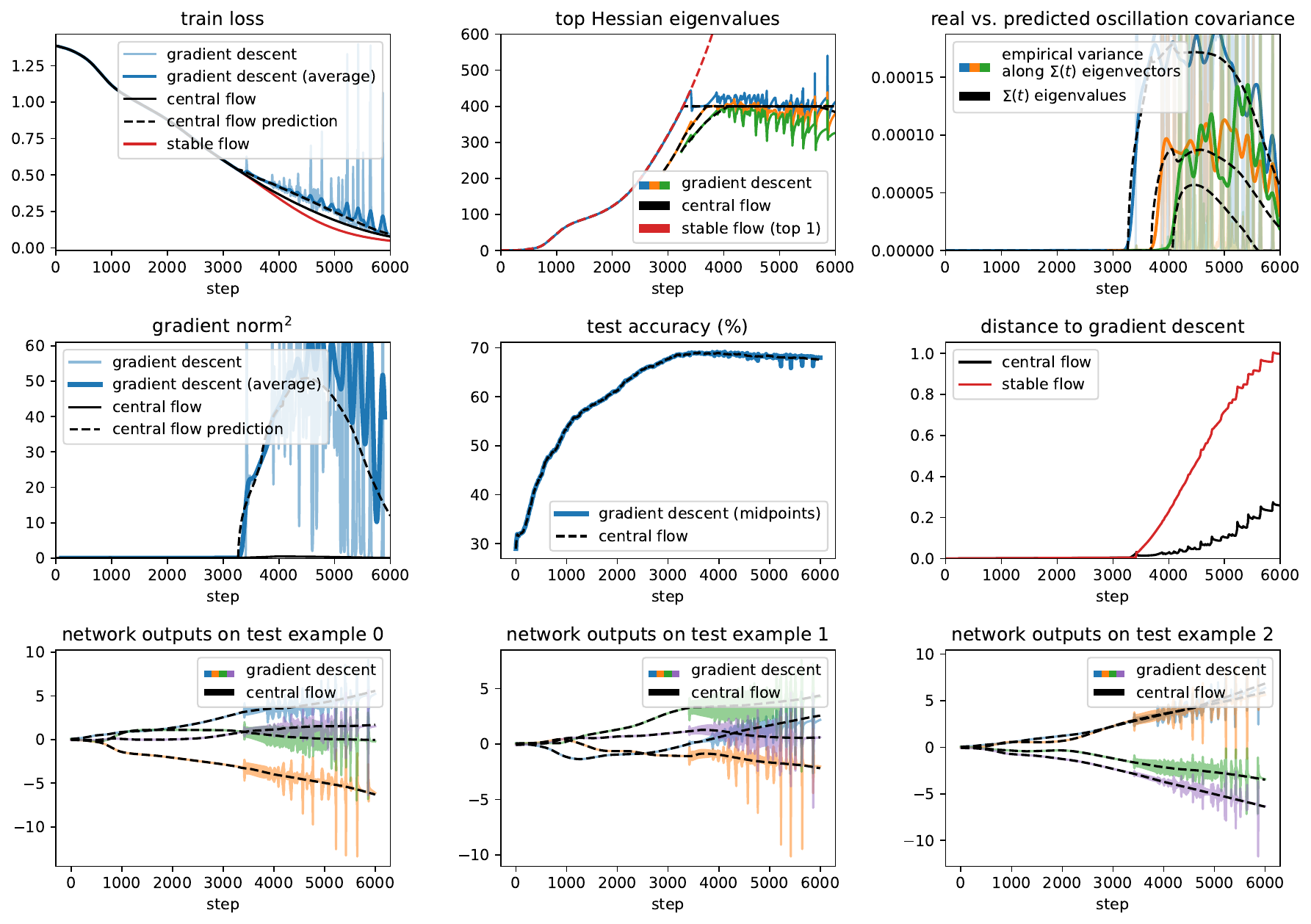}
        \caption{Gradient descent central flow for a CNN with CE loss, $\eta=$ 0.005.}
        \label{fig:bulk-gd:mse-cnn-0}
    \end{figure}
                
    \begin{figure}[H]
        \centering
        \includegraphics[width=0.8\linewidth]{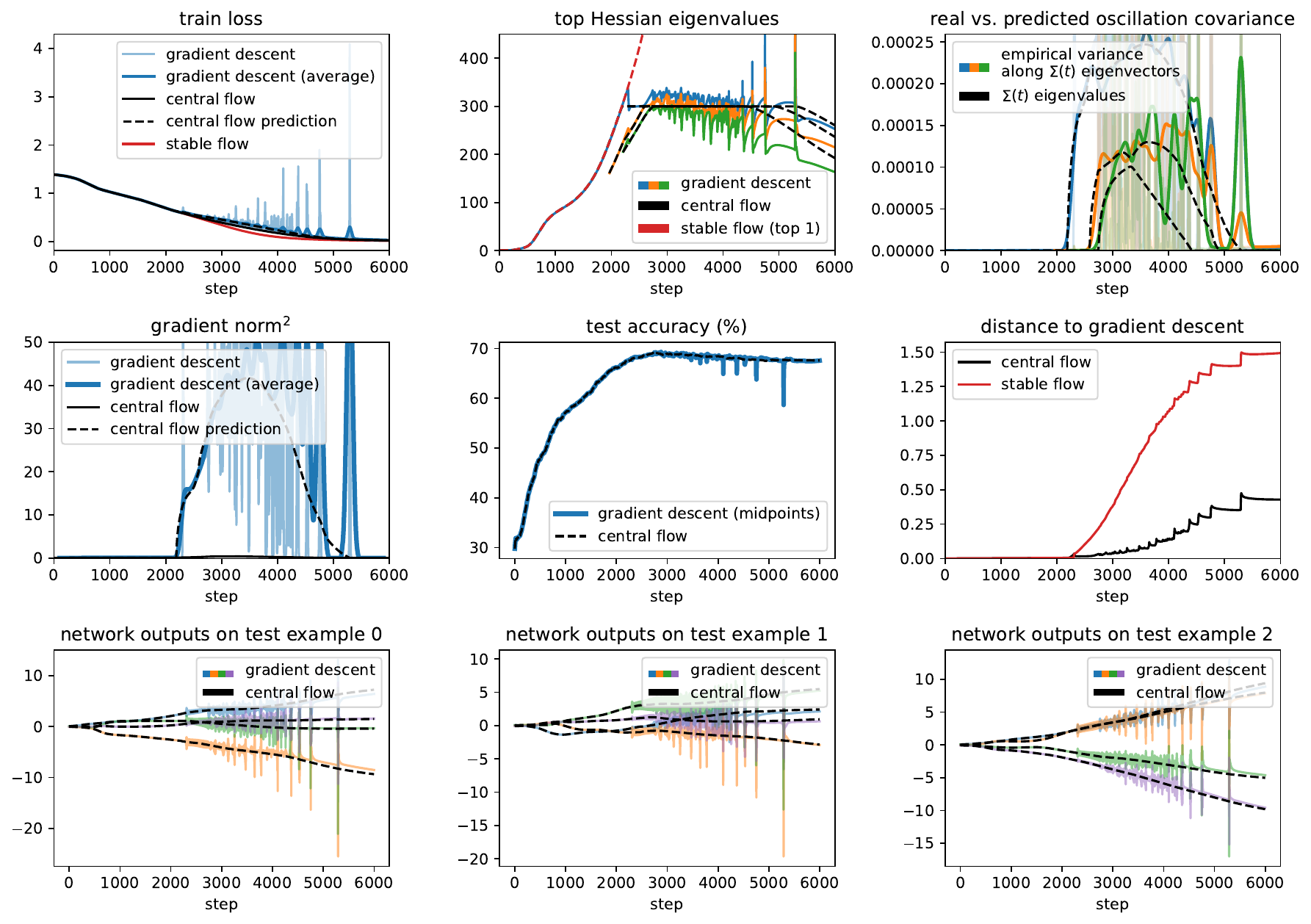}
        \caption{Gradient descent central flow for a CNN with CE loss, $\eta=$ 0.006666.}
        \label{fig:bulk-gd:mse-cnn-1}
    \end{figure}
                
    \begin{figure}[H]
        \centering
        \includegraphics[width=0.8\linewidth]{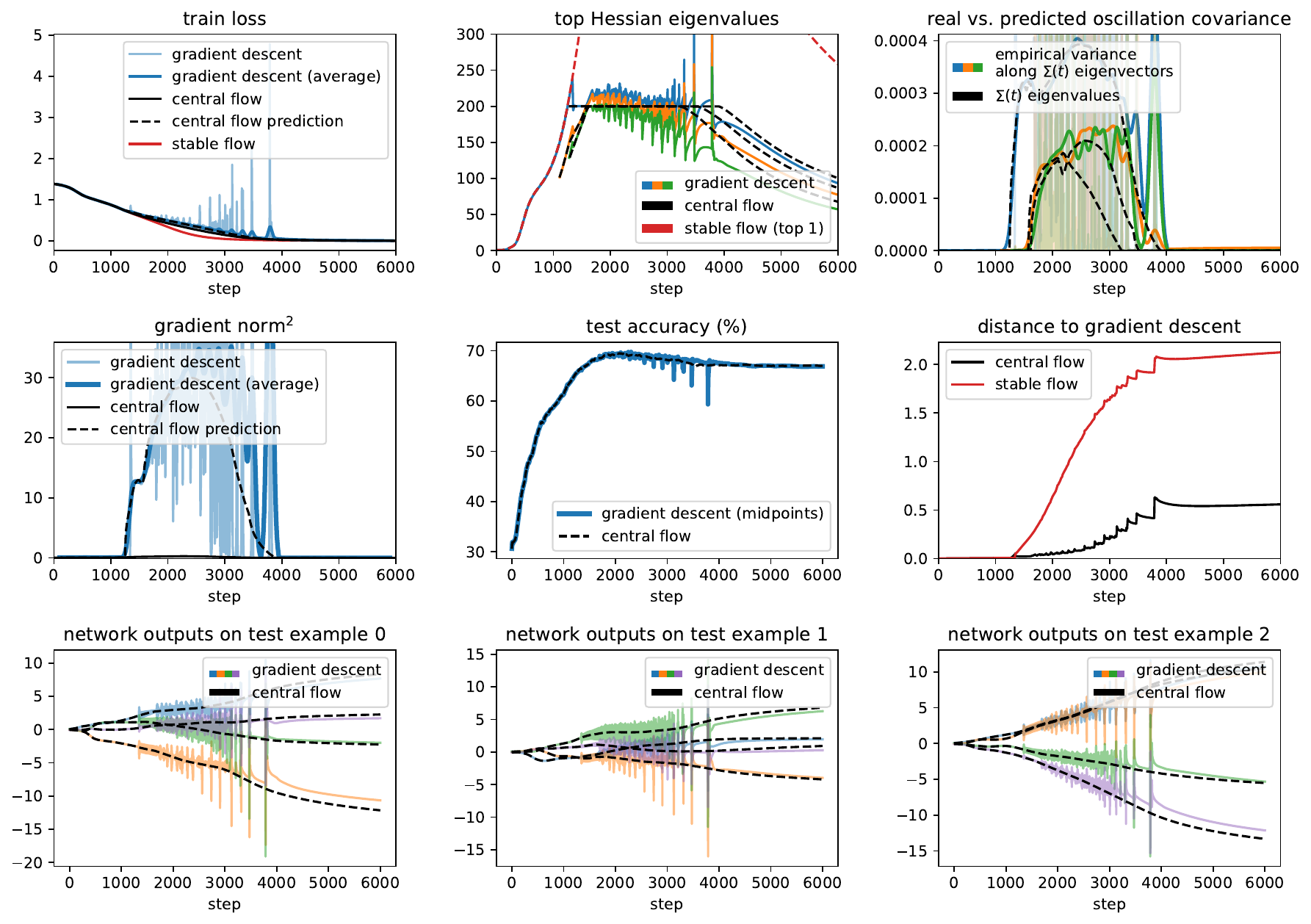}
        \caption{Gradient descent central flow for a CNN with CE loss, $\eta=$ 0.01.}
        \label{fig:bulk-gd:mse-cnn-2}
    \end{figure}
                
    \begin{figure}[H]
        \centering
        \includegraphics[width=0.8\linewidth]{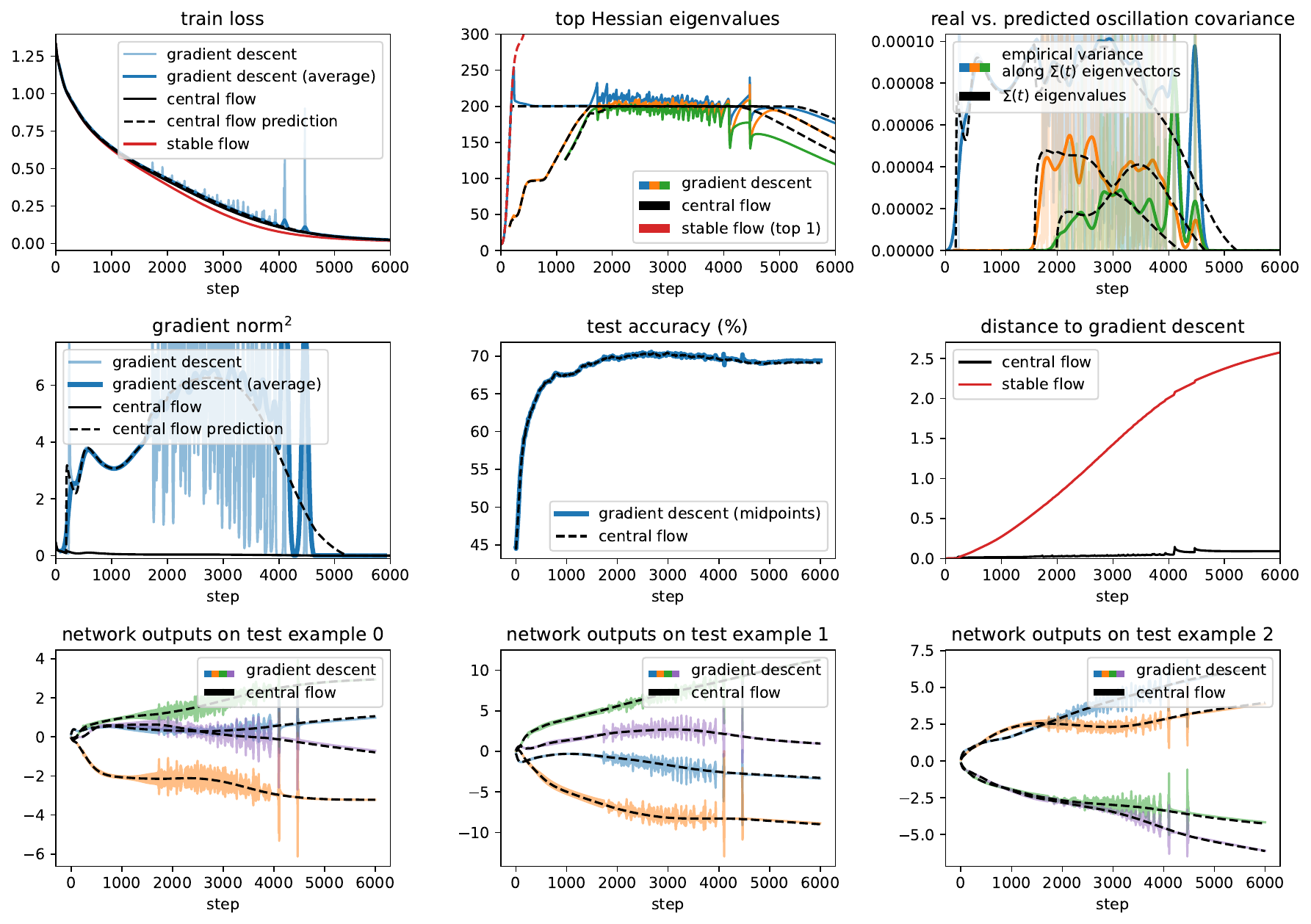}
        \caption{Gradient descent central flow for a ResNet with CE loss, $\eta=$ 0.01.}
        \label{fig:bulk-gd:mse-resnet-0}
    \end{figure}
                
    \begin{figure}[H]
        \centering
        \includegraphics[width=0.8\linewidth]{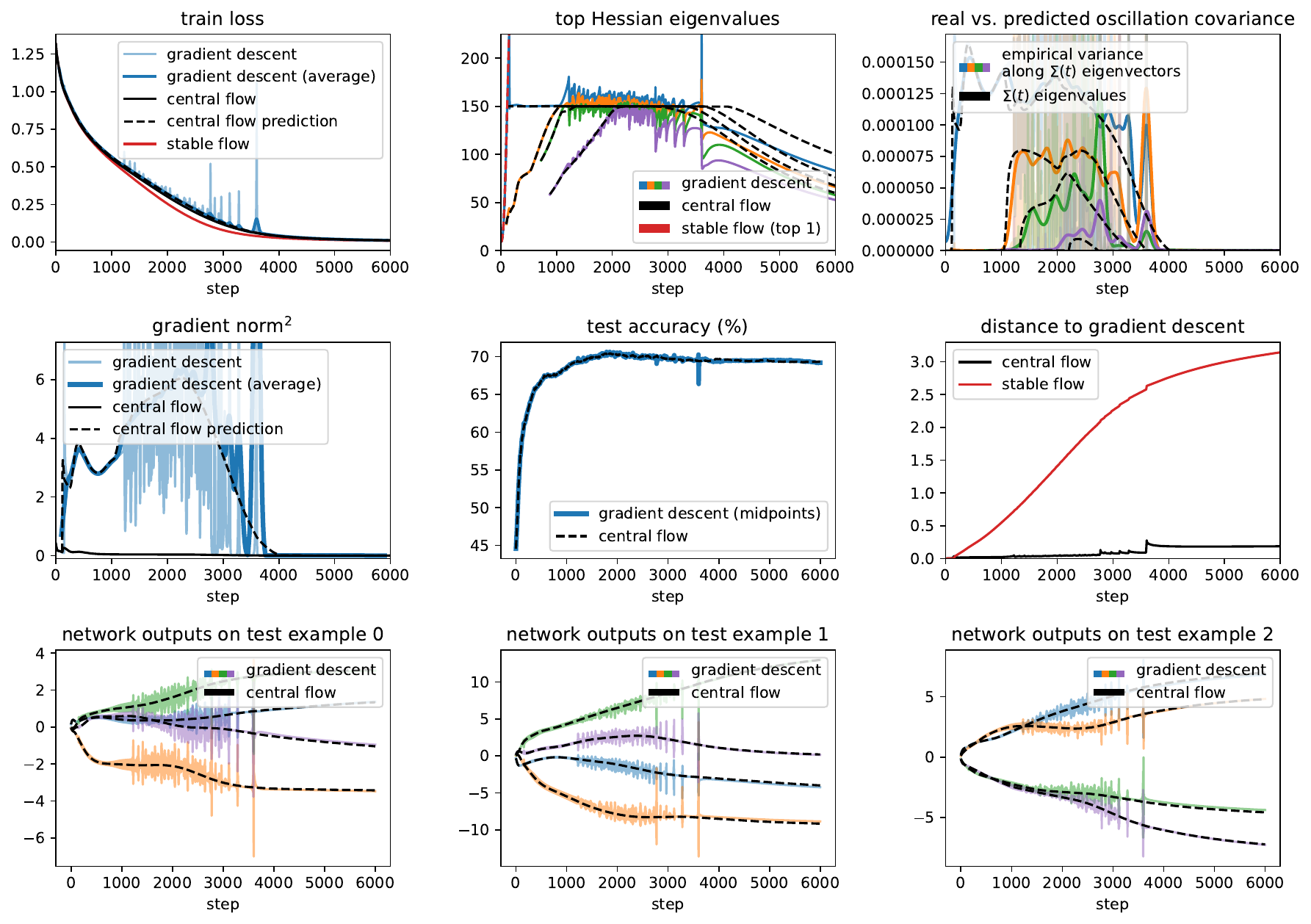}
        \caption{Gradient descent central flow for a ResNet with CE loss, $\eta=$ 0.013333.}
        \label{fig:bulk-gd:mse-resnet-1}
    \end{figure}
                
    \begin{figure}[H]
        \centering
        \includegraphics[width=0.8\linewidth]{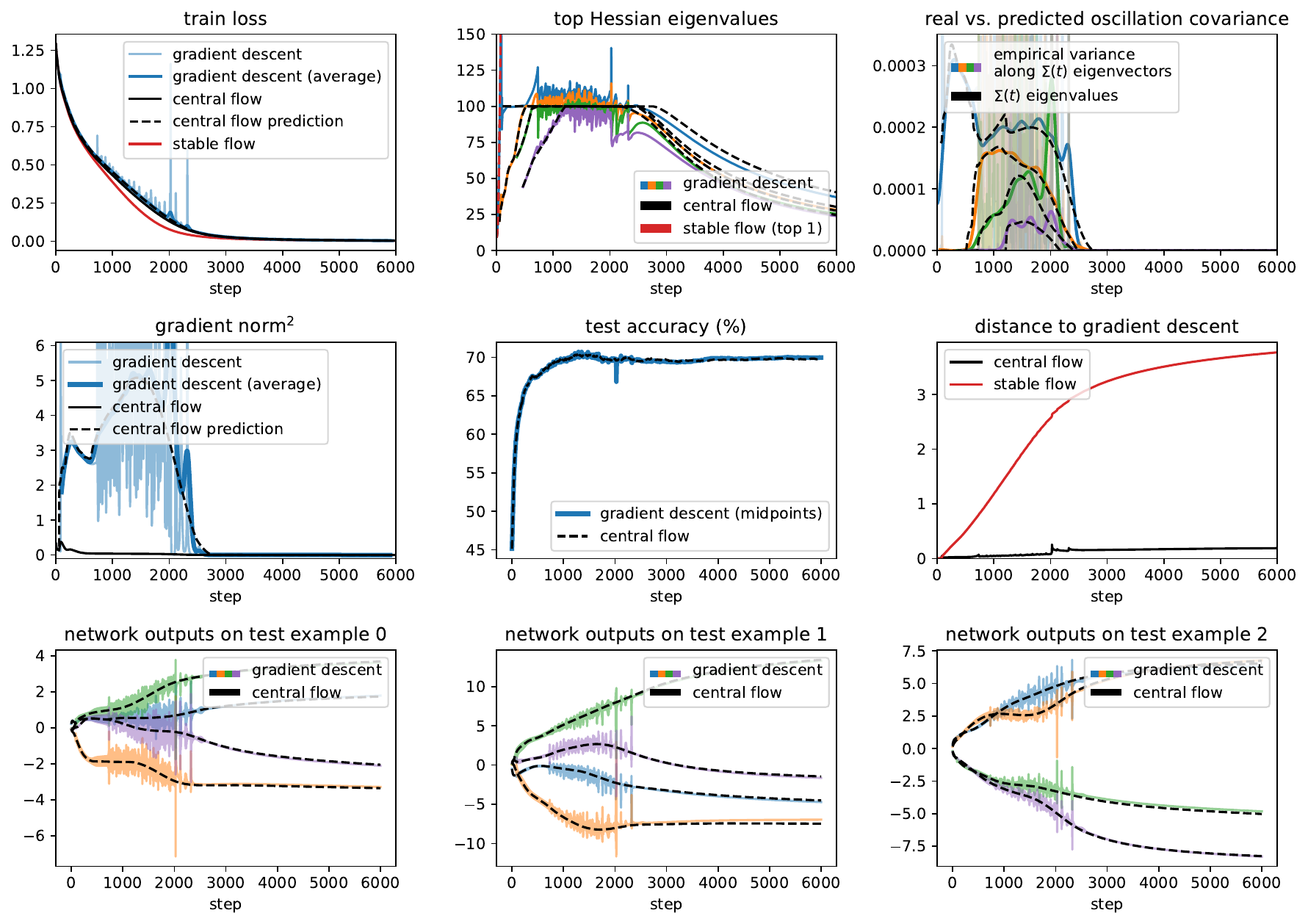}
        \caption{Gradient descent central flow for a ResNet with CE loss, $\eta=$ 0.02.}
        \label{fig:bulk-gd:mse-resnet-2}
    \end{figure}
                
    \begin{figure}[H]
        \centering
        \includegraphics[width=0.8\linewidth]{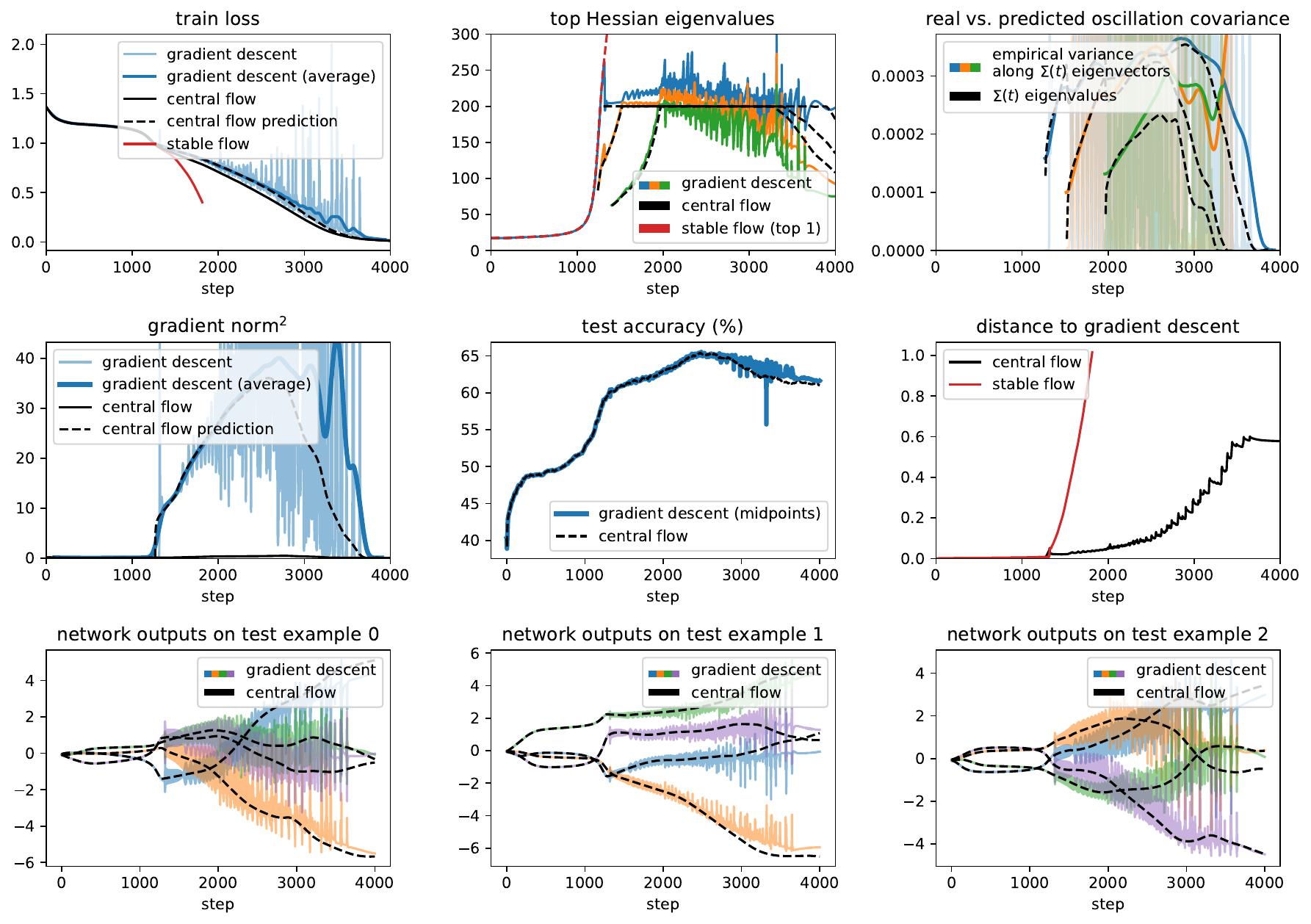}
        \caption{Gradient descent central flow for a ViT with CE loss, $\eta=$ 0.01.}
        \label{fig:bulk-gd:mse-vit-0}
    \end{figure}
                
    \begin{figure}[H]
        \centering
        \includegraphics[width=0.8\linewidth]{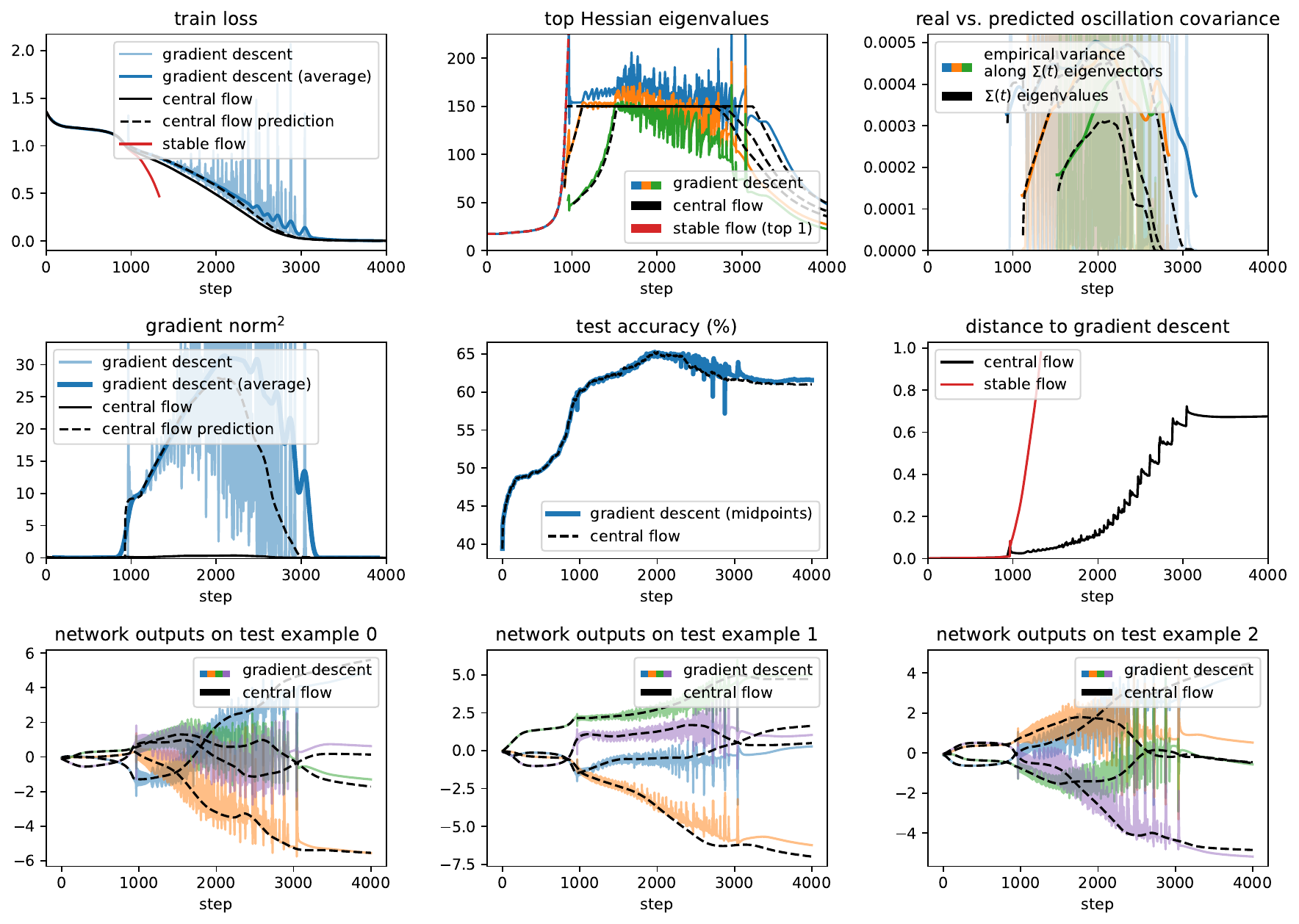}
        \caption{Gradient descent central flow for a ViT with CE loss, $\eta=$ 0.013333.}
        \label{fig:bulk-gd:mse-vit-1}
    \end{figure}
                
    \begin{figure}[H]
        \centering
        \includegraphics[width=0.8\linewidth]{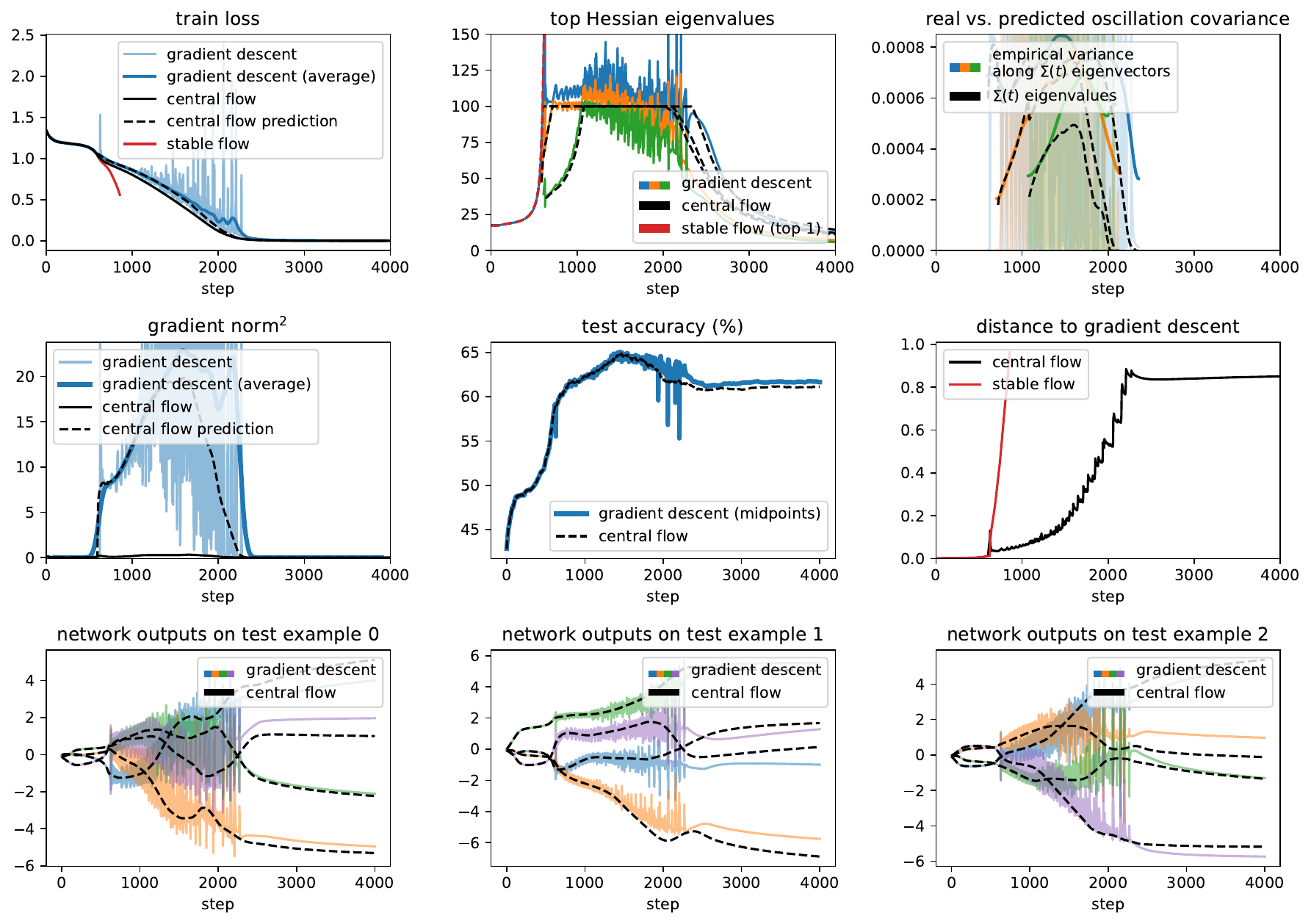}
        \caption{Gradient descent central flow for a ViT with CE loss, $\eta=$ 0.02.}
        \label{fig:bulk-gd:mse-vit-2}
    \end{figure}
                
    \begin{figure}[H]
        \centering
        \includegraphics[width=0.8\linewidth]{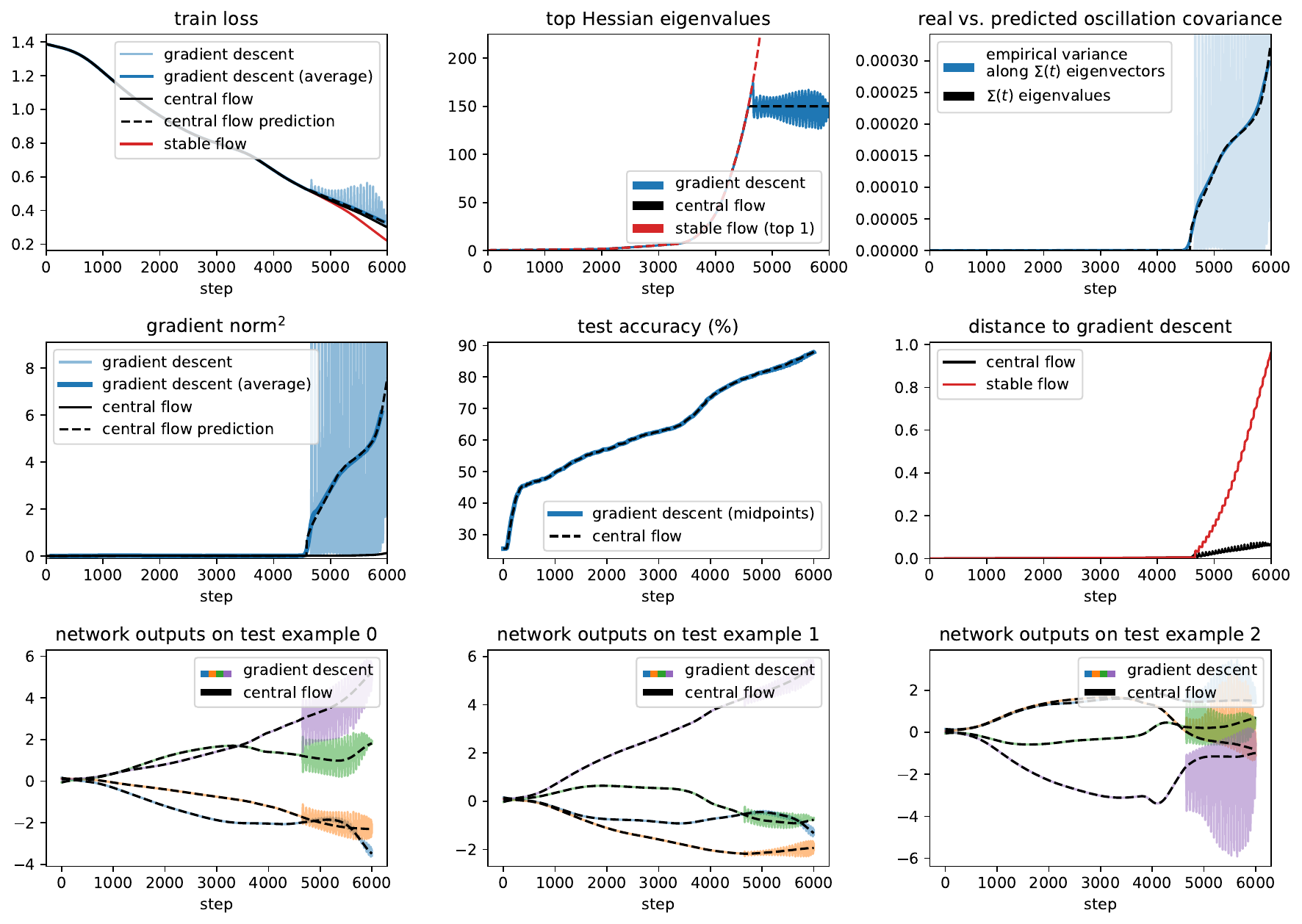}
        \caption{Gradient descent central flow for a LSTM with CE loss, $\eta=$ 0.01333.}
        \label{fig:bulk-gd:mse-lstm-0}
    \end{figure}
                
    \begin{figure}[H]
        \centering
        \includegraphics[width=0.8\linewidth]{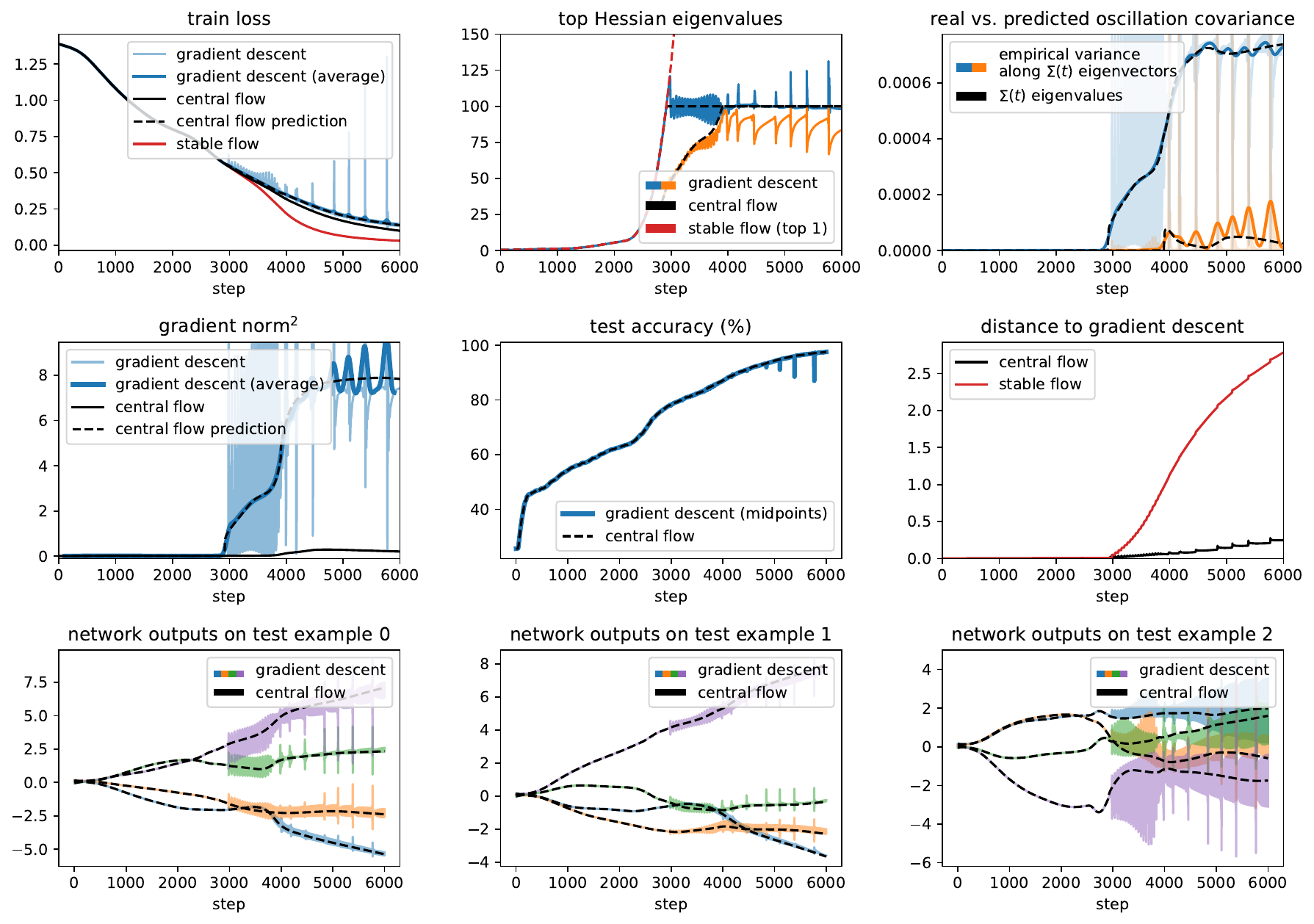}
        \caption{Gradient descent central flow for a LSTM with CE loss, $\eta=$ 0.02.}
        \label{fig:bulk-gd:mse-lstm-1}
    \end{figure}
                
    \begin{figure}[H]
        \centering
        \includegraphics[width=0.8\linewidth]{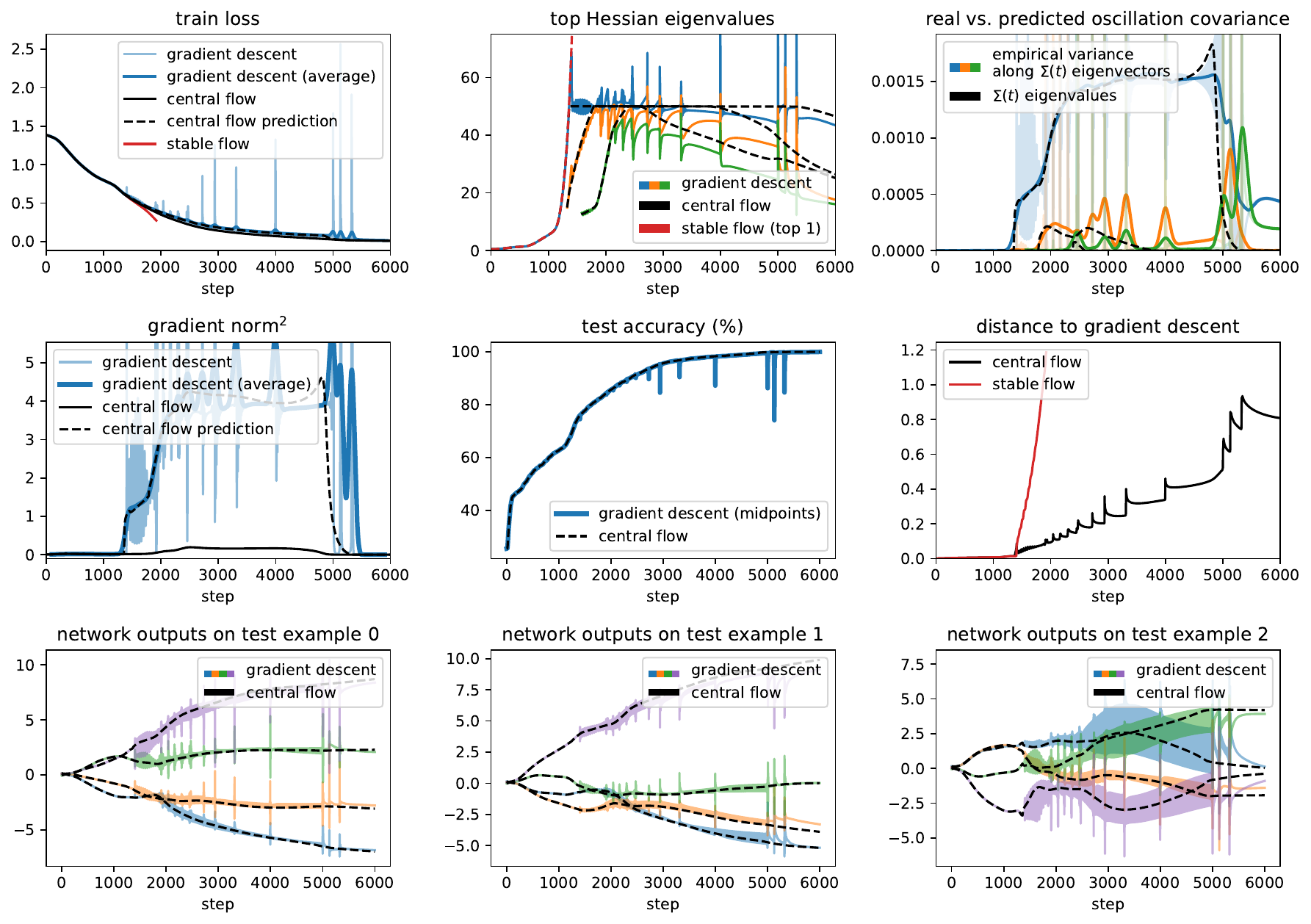}
        \caption{Gradient descent central flow for a LSTM with CE loss, $\eta=$ 0.04.}
        \label{fig:bulk-gd:mse-lstm-2}
    \end{figure}
                
    \begin{figure}[H]
        \centering
        \includegraphics[width=0.8\linewidth]{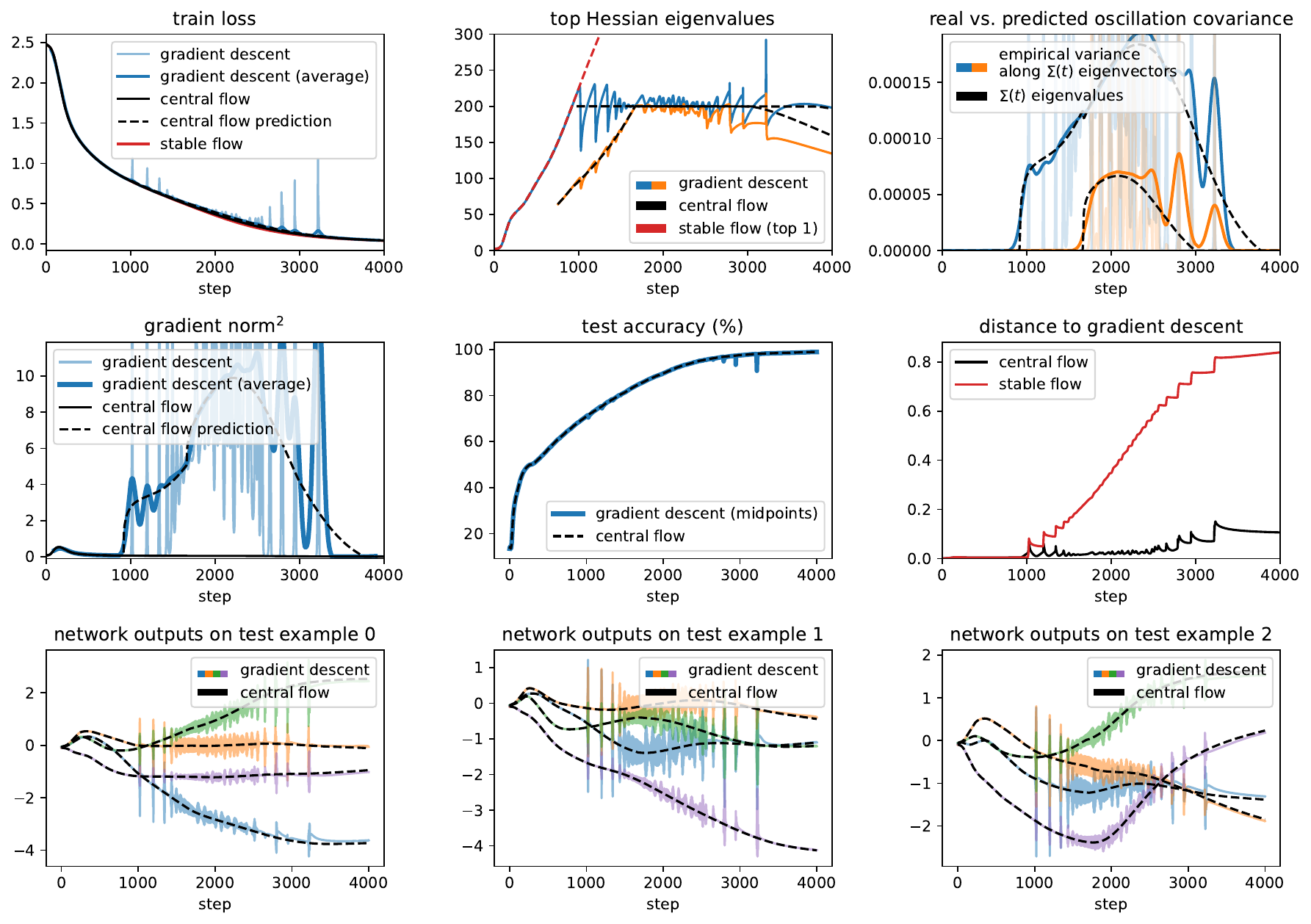}
        \caption{Gradient descent central flow for a Transformer with CE loss, $\eta=$ 0.01.}
        \label{fig:bulk-gd:mse-transformer-0}
    \end{figure}
                
    \begin{figure}[H]
        \centering
        \includegraphics[width=0.8\linewidth]{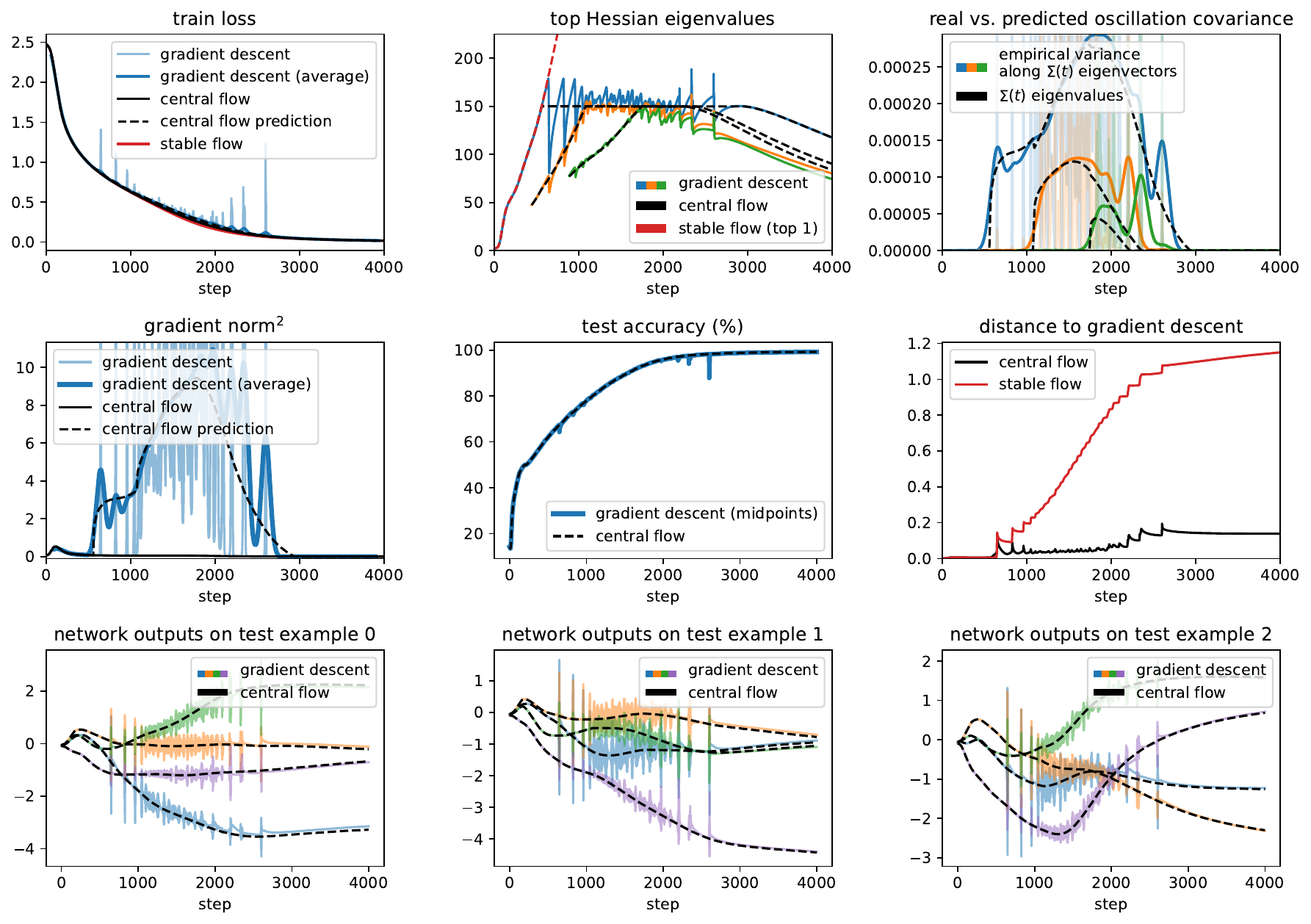}
        \caption{Gradient descent central flow for a Transformer with CE loss, $\eta=$ 0.013333.}
        \label{fig:bulk-gd:mse-transformer-1}
    \end{figure}
                
    \begin{figure}[H]
        \centering
        \includegraphics[width=0.8\linewidth]{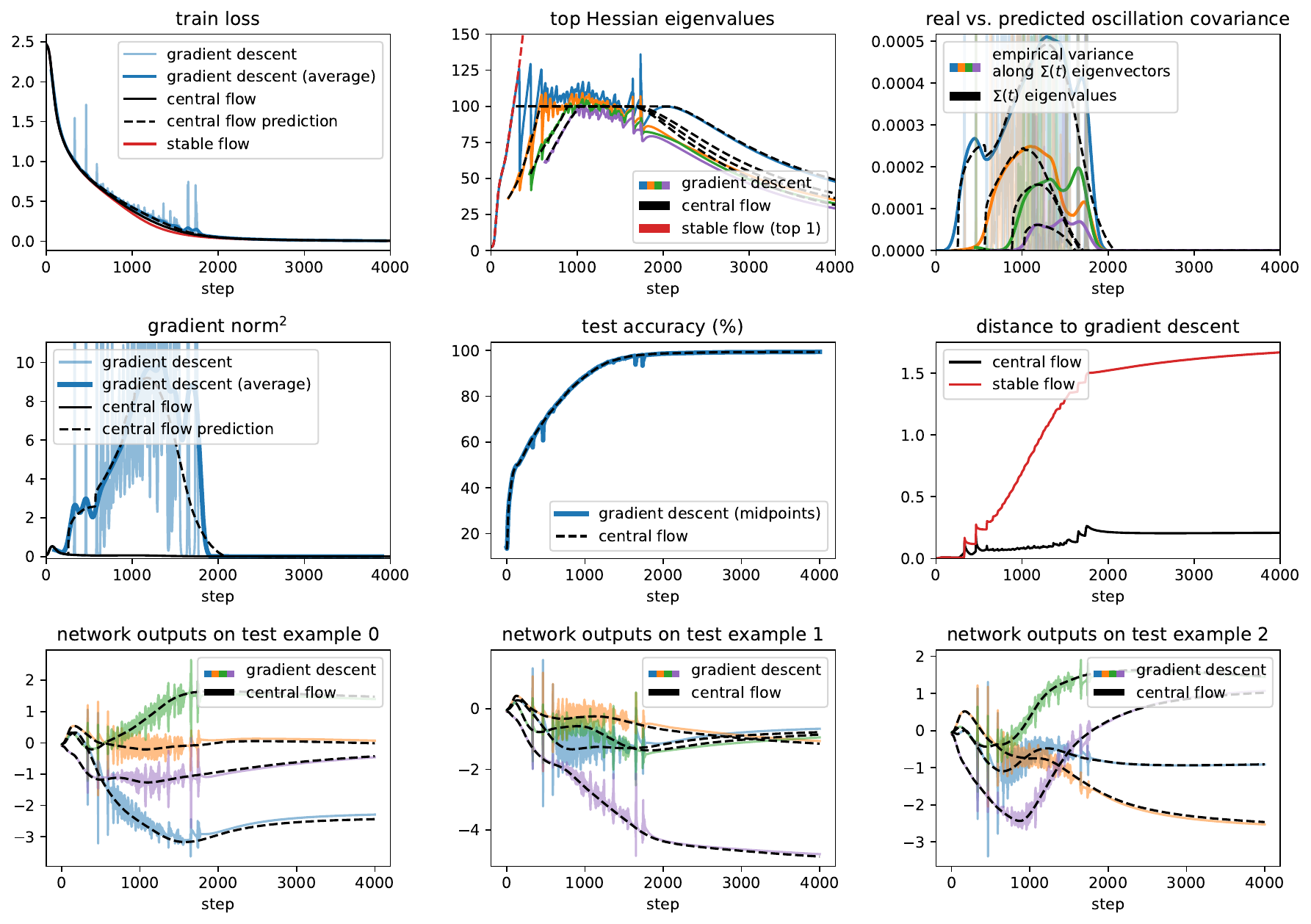}
        \caption{Gradient descent central flow for a Transformer with CE loss, $\eta=$ 0.02.}
        \label{fig:bulk-gd:mse-transformer-2}
    \end{figure}
                
    \begin{figure}[H]
        \centering
        \includegraphics[width=0.8\linewidth]{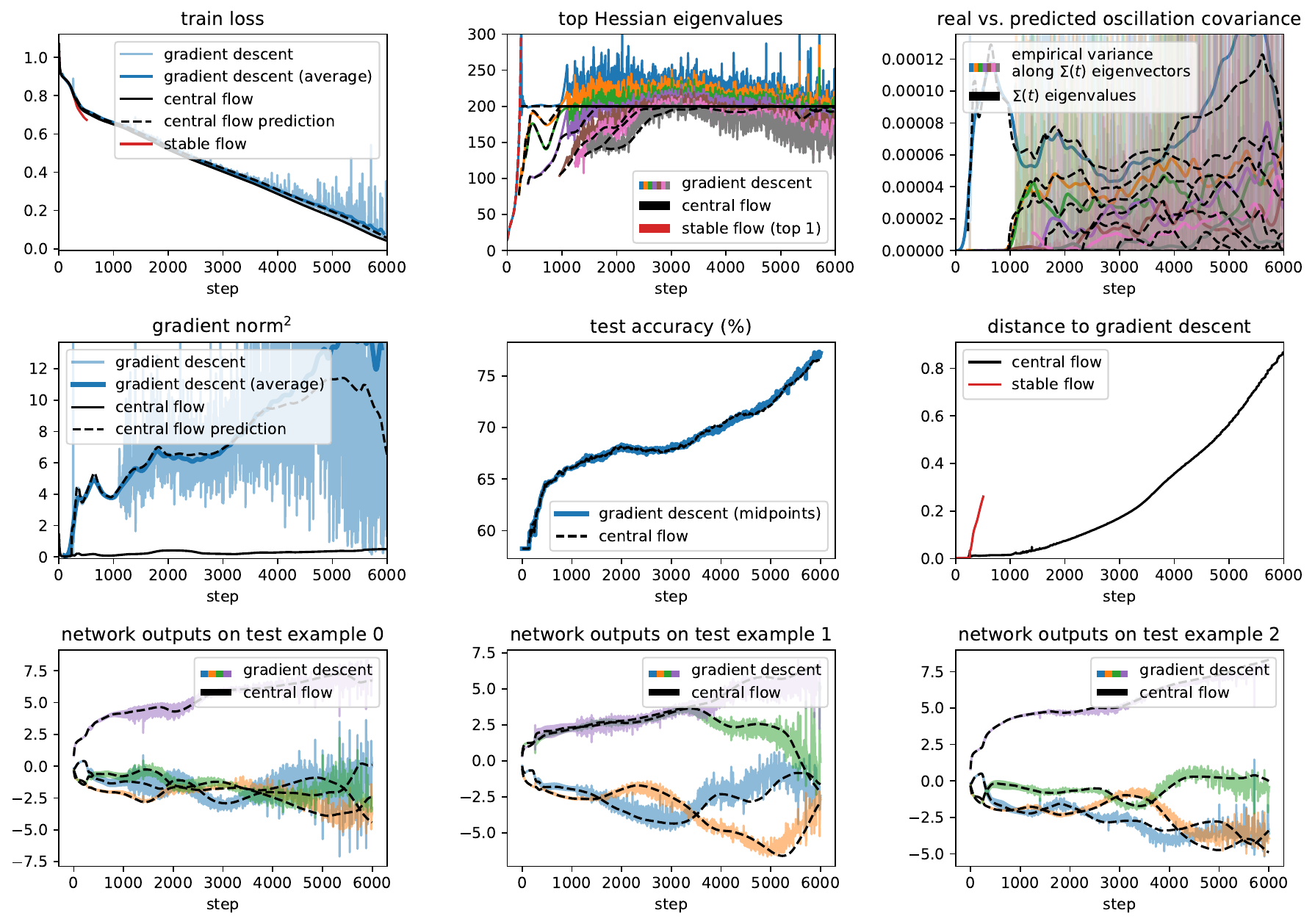}
        \caption{Gradient descent central flow for a Mamba with CE loss, $\eta=$ 0.01.}
        \label{fig:bulk-gd:mse-mamba-0}
    \end{figure}
                
    \begin{figure}[H]
        \centering
        \includegraphics[width=0.8\linewidth]{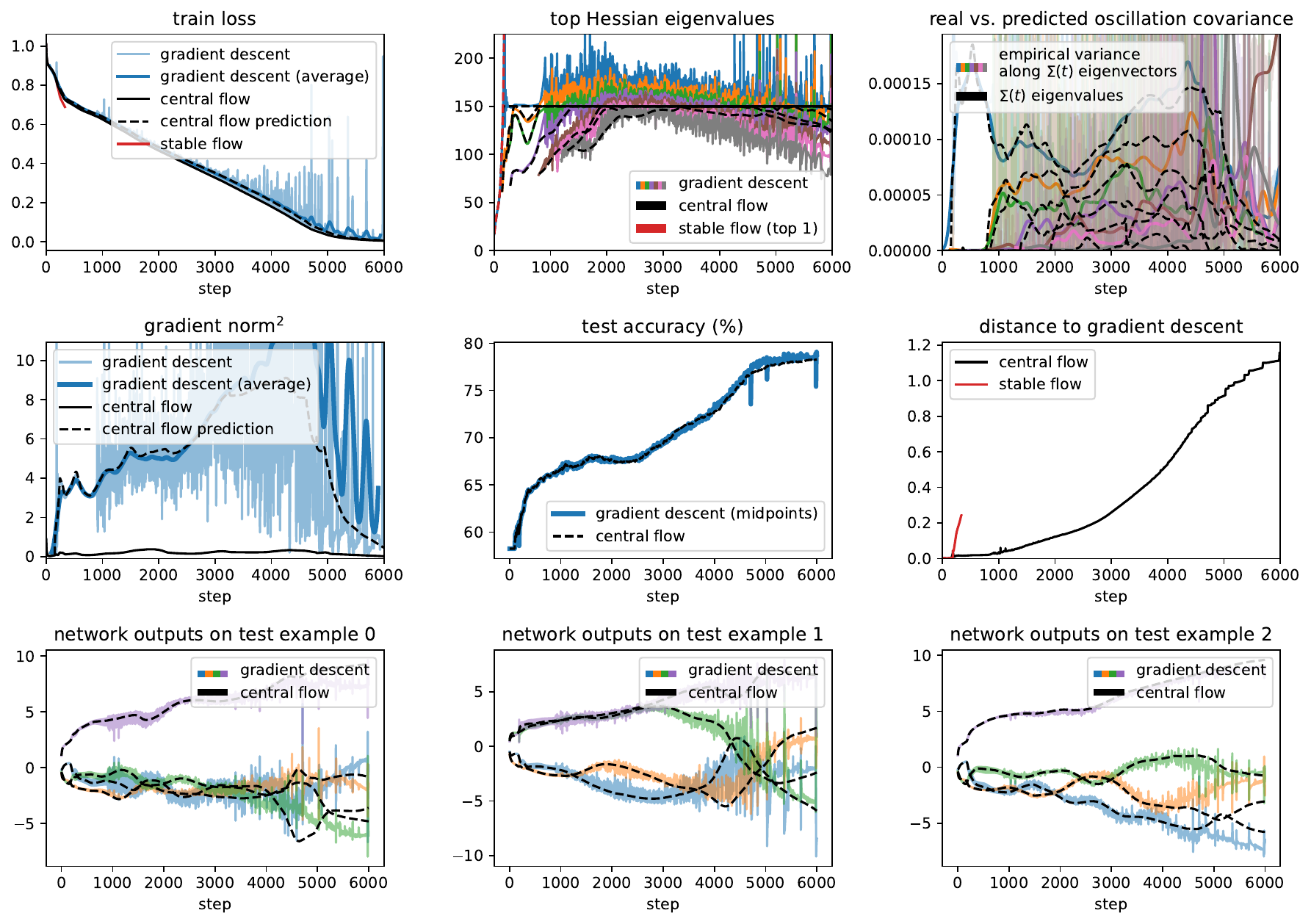}
        \caption{Gradient descent central flow for a Mamba with CE loss, $\eta=$ 0.013333.}
        \label{fig:bulk-gd:mse-mamba-1}
    \end{figure}
                
    \begin{figure}[H]
        \centering
        \includegraphics[width=0.8\linewidth]{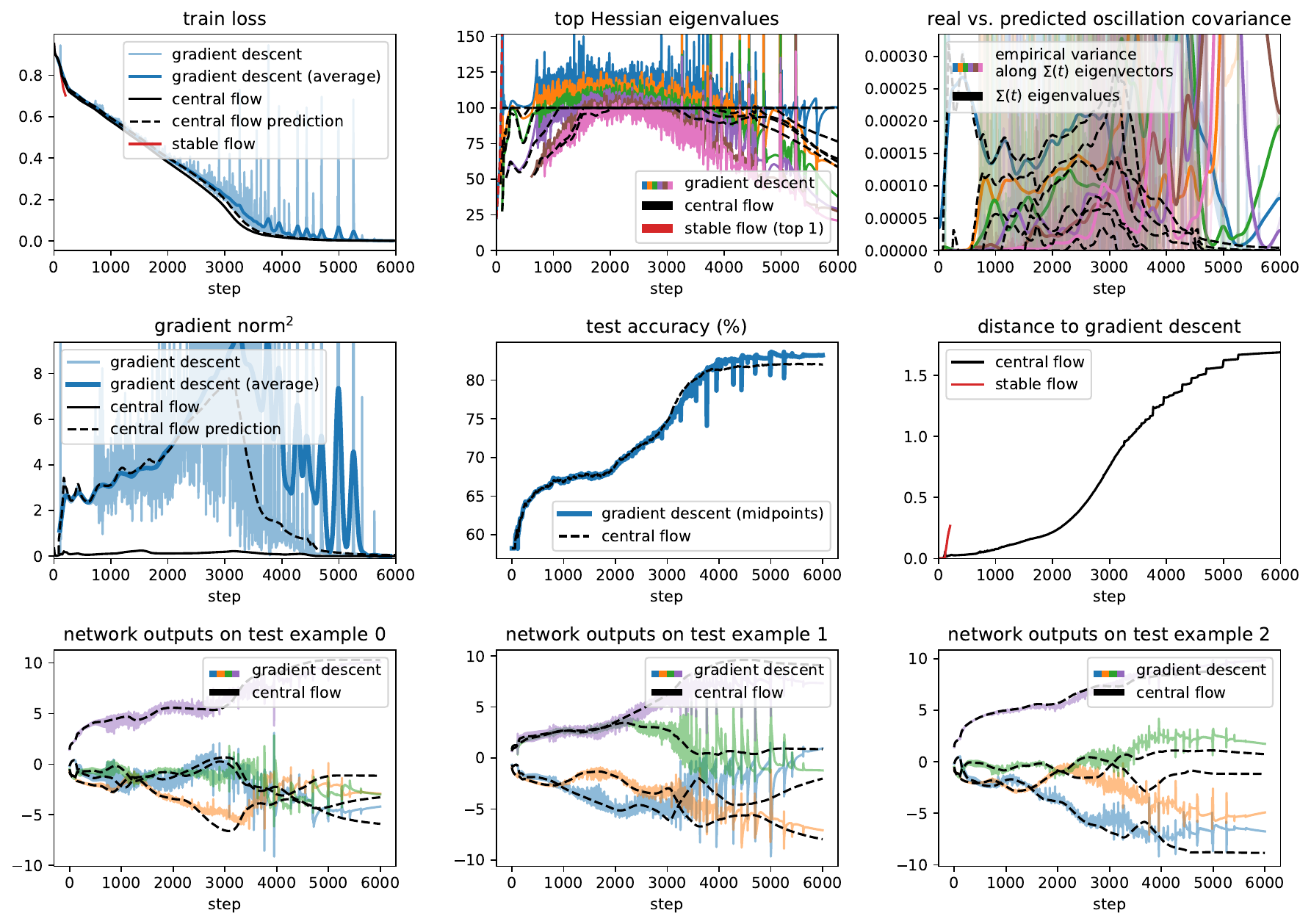}
        \caption{Gradient descent central flow for a Mamba with CE loss, $\eta=$ 0.02.}
        \label{fig:bulk-gd:mse-mamba-2}
    \end{figure}
                \end{specialfigures}

%% file: images/bulk-scalar-rmsprop/figures-mse.tex
\begin{specialfigures}

    \begin{figure}[H]
        \centering
        \includegraphics[width=0.8\linewidth]{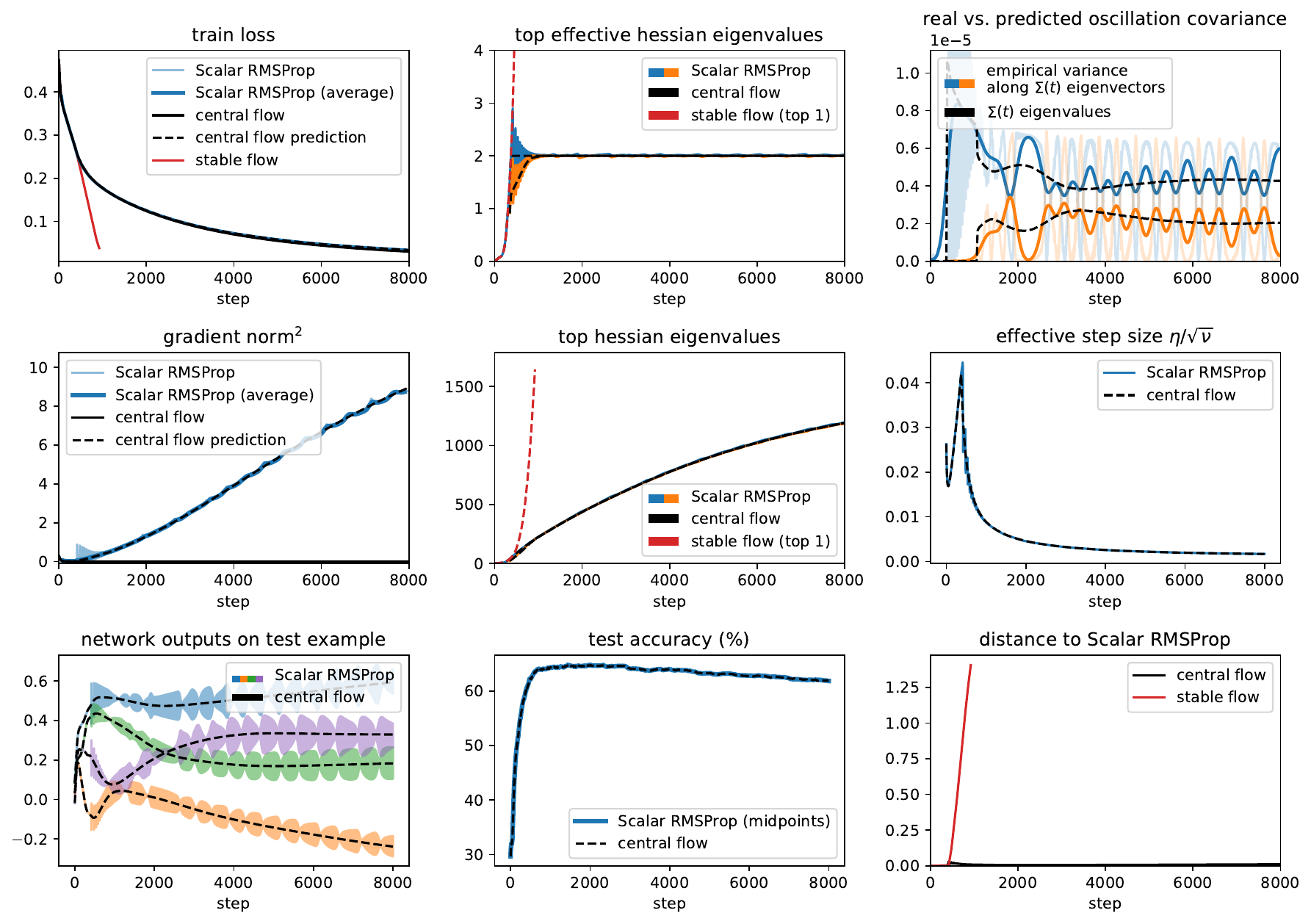}
        \caption{Scalar RMSProp central flow for a CNN with MSE loss, $\eta= $ 0.003, $\beta_2 = $ 0.99, and bias correction.}
        \label{fig:bulk-scalar-rmsprop:mse-cnn-0}
    \end{figure}
                
    \begin{figure}[H]
        \centering
        \includegraphics[width=0.8\linewidth]{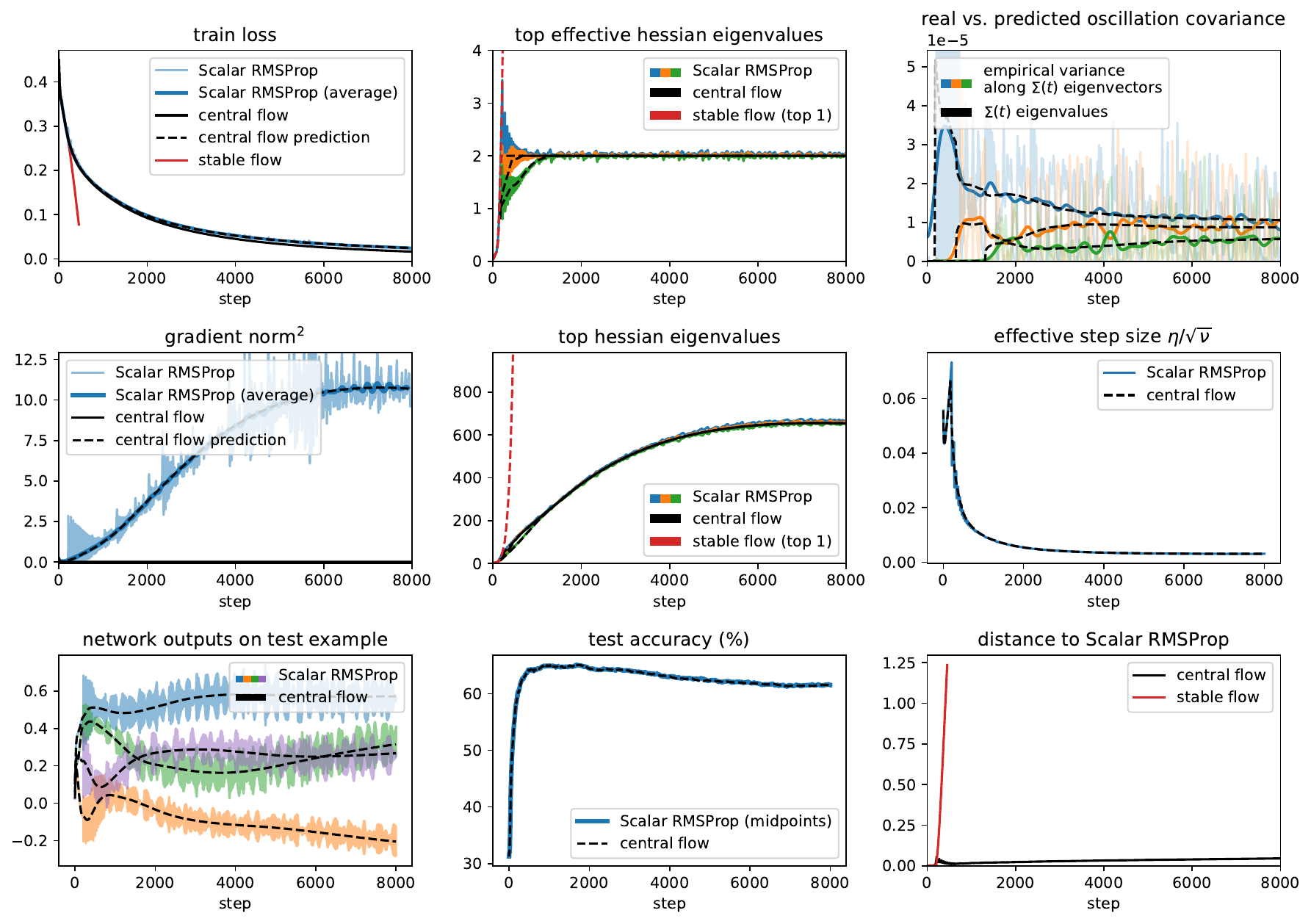}
        \caption{Scalar RMSProp central flow for a CNN with MSE loss, $\eta= $ 0.006, $\beta_2 = $ 0.99, and bias correction.}
        \label{fig:bulk-scalar-rmsprop:mse-cnn-1}
    \end{figure}
                
    \begin{figure}[H]
        \centering
        \includegraphics[width=0.8\linewidth]{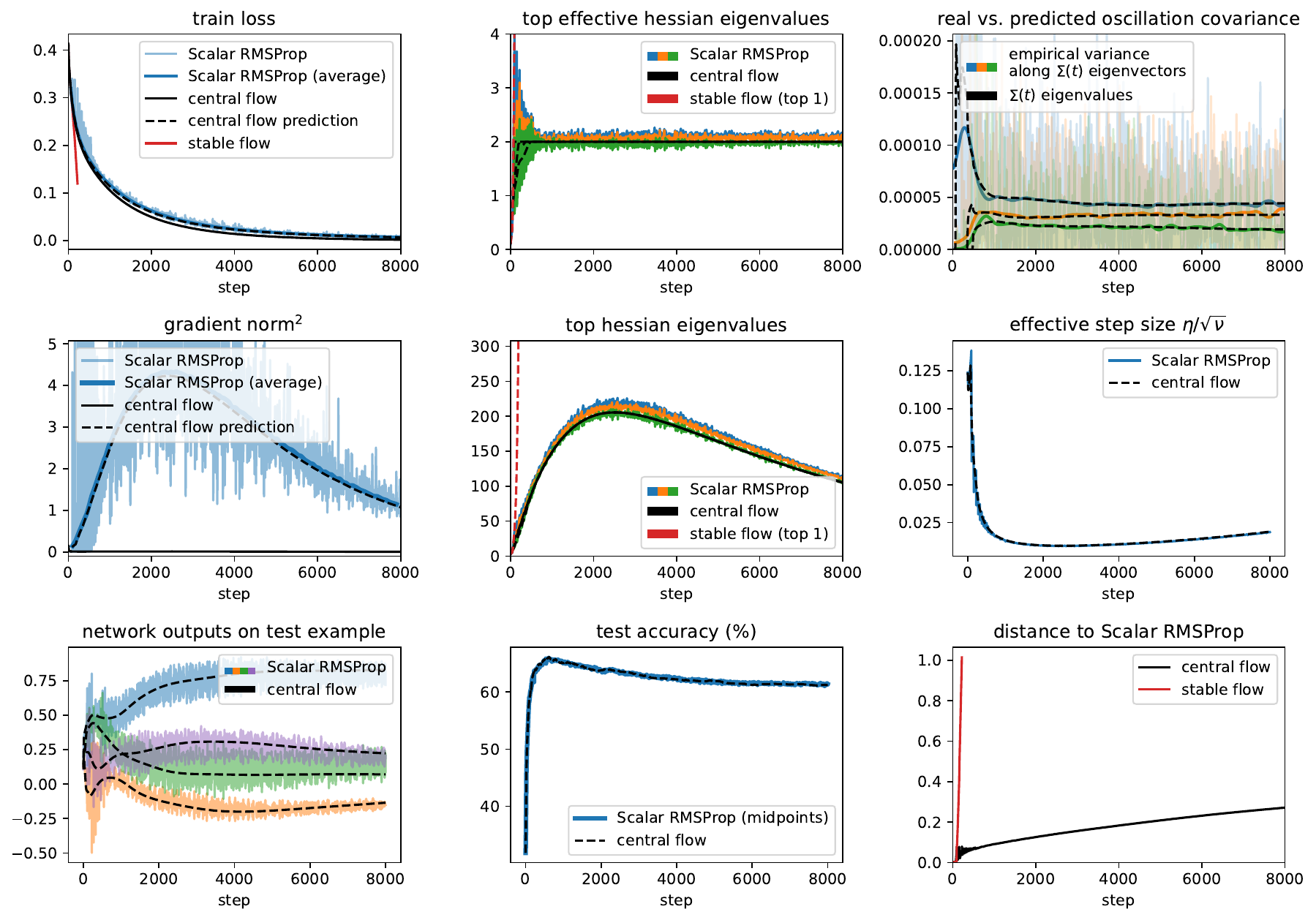}
        \caption{Scalar RMSProp central flow for a CNN with MSE loss, $\eta= $ 0.01, $\beta_2 = $ 0.99, and bias correction.}
        \label{fig:bulk-scalar-rmsprop:mse-cnn-2}
    \end{figure}
                
    \begin{figure}[H]
        \centering
        \includegraphics[width=0.8\linewidth]{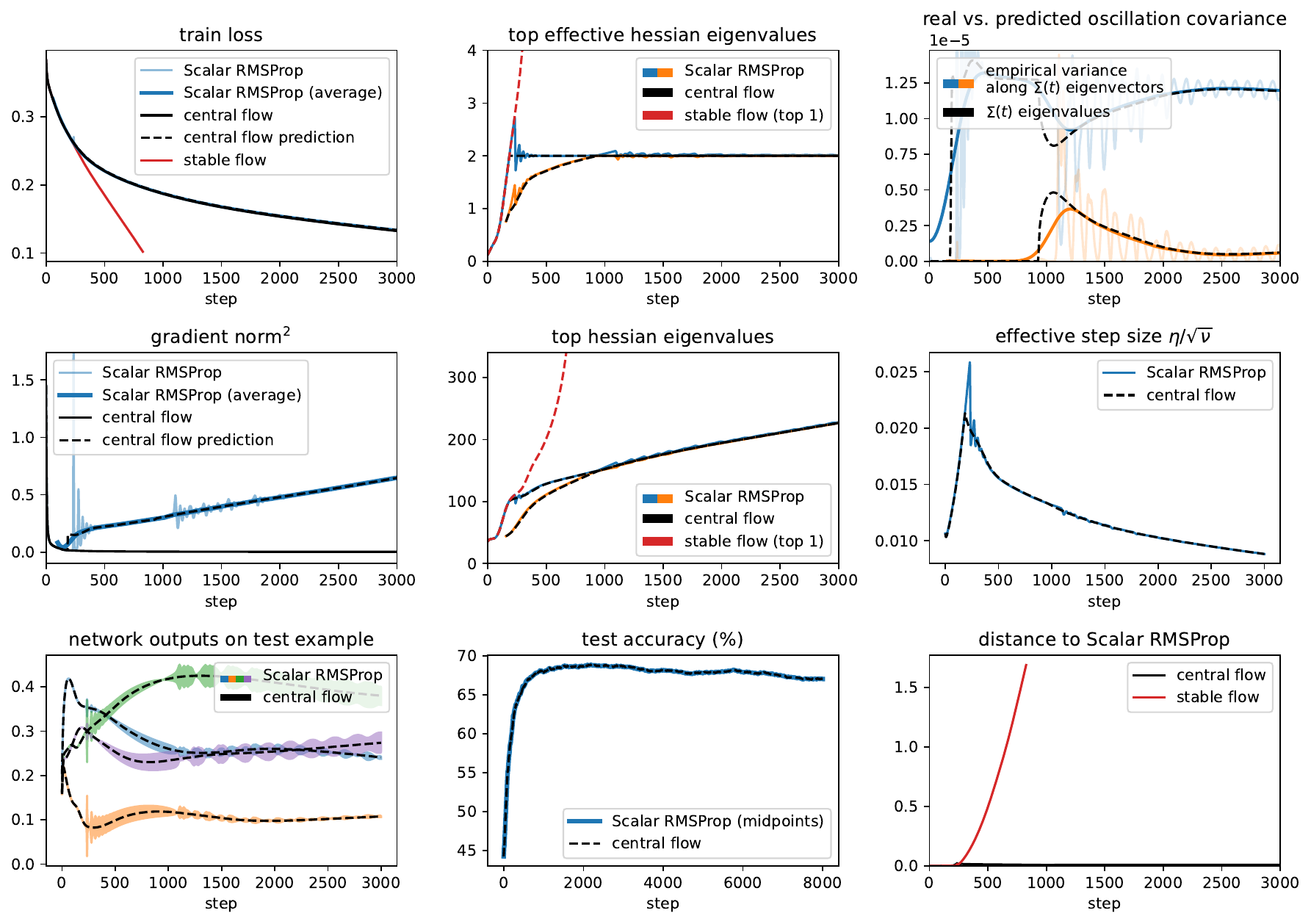}
        \caption{Scalar RMSProp central flow for a ResNet with MSE loss, $\eta= $ 0.01, $\beta_2 = $ 0.99, and bias correction.}
        \label{fig:bulk-scalar-rmsprop:mse-resnet-0}
    \end{figure}
                
    \begin{figure}[H]
        \centering
        \includegraphics[width=0.8\linewidth]{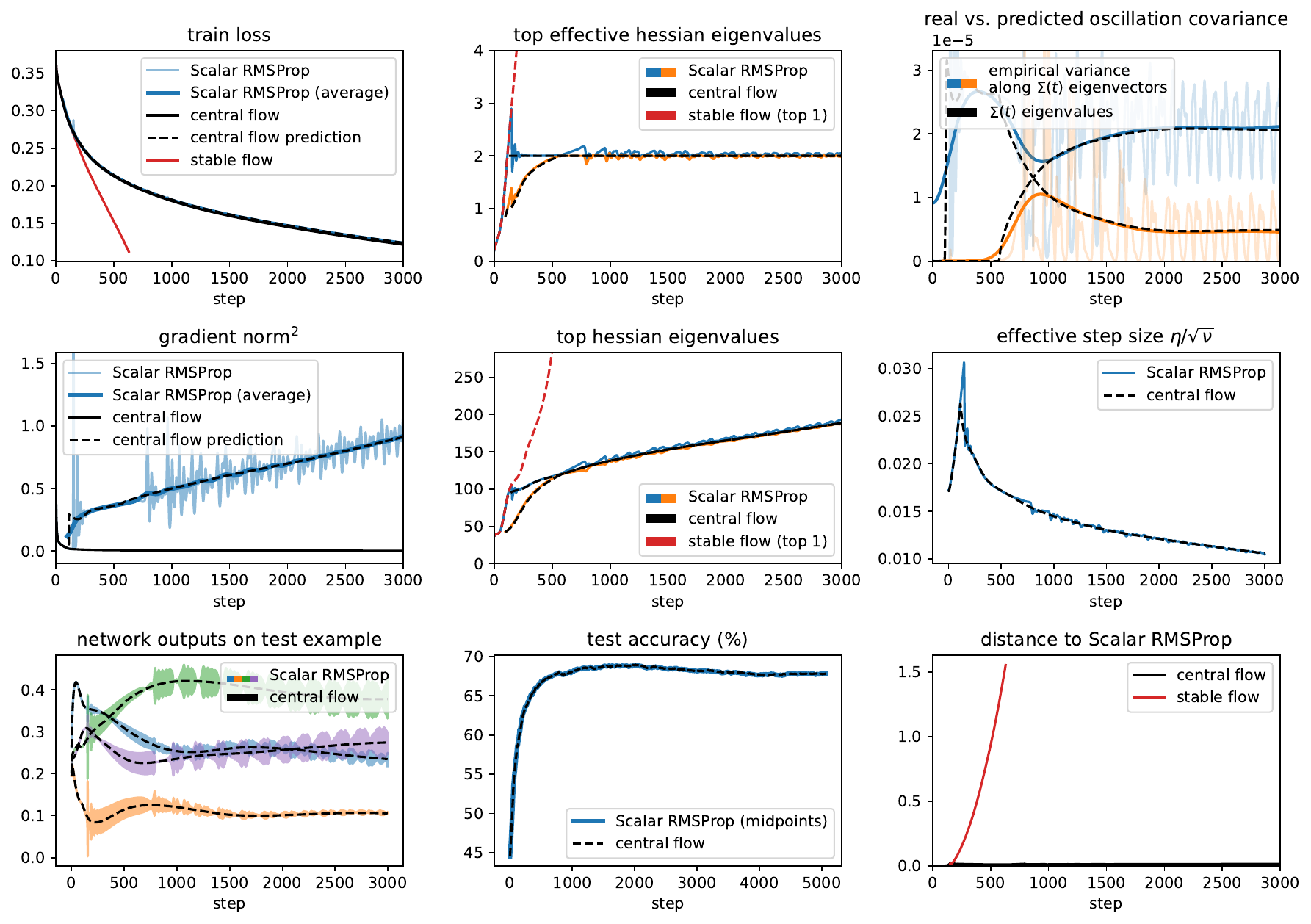}
        \caption{Scalar RMSProp central flow for a ResNet with MSE loss, $\eta= $ 0.02, $\beta_2 = $ 0.99, and bias correction.}
        \label{fig:bulk-scalar-rmsprop:mse-resnet-1}
    \end{figure}
                
    \begin{figure}[H]
        \centering
        \includegraphics[width=0.8\linewidth]{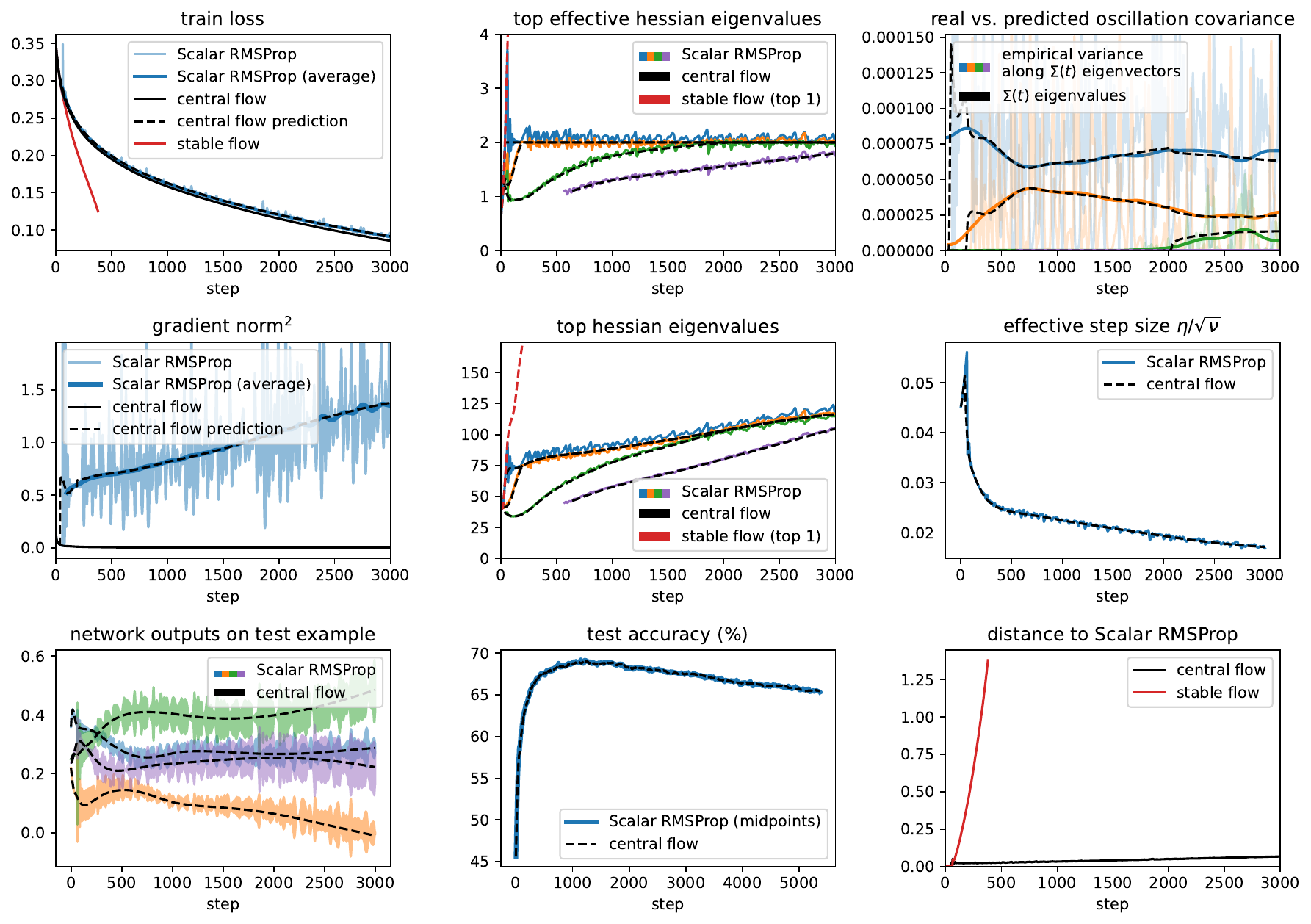}
        \caption{Scalar RMSProp central flow for a ResNet with MSE loss, $\eta= $ 0.03, $\beta_2 = $ 0.99, and bias correction.}
        \label{fig:bulk-scalar-rmsprop:mse-resnet-2}
    \end{figure}
                
    \begin{figure}[H]
        \centering
        \includegraphics[width=0.8\linewidth]{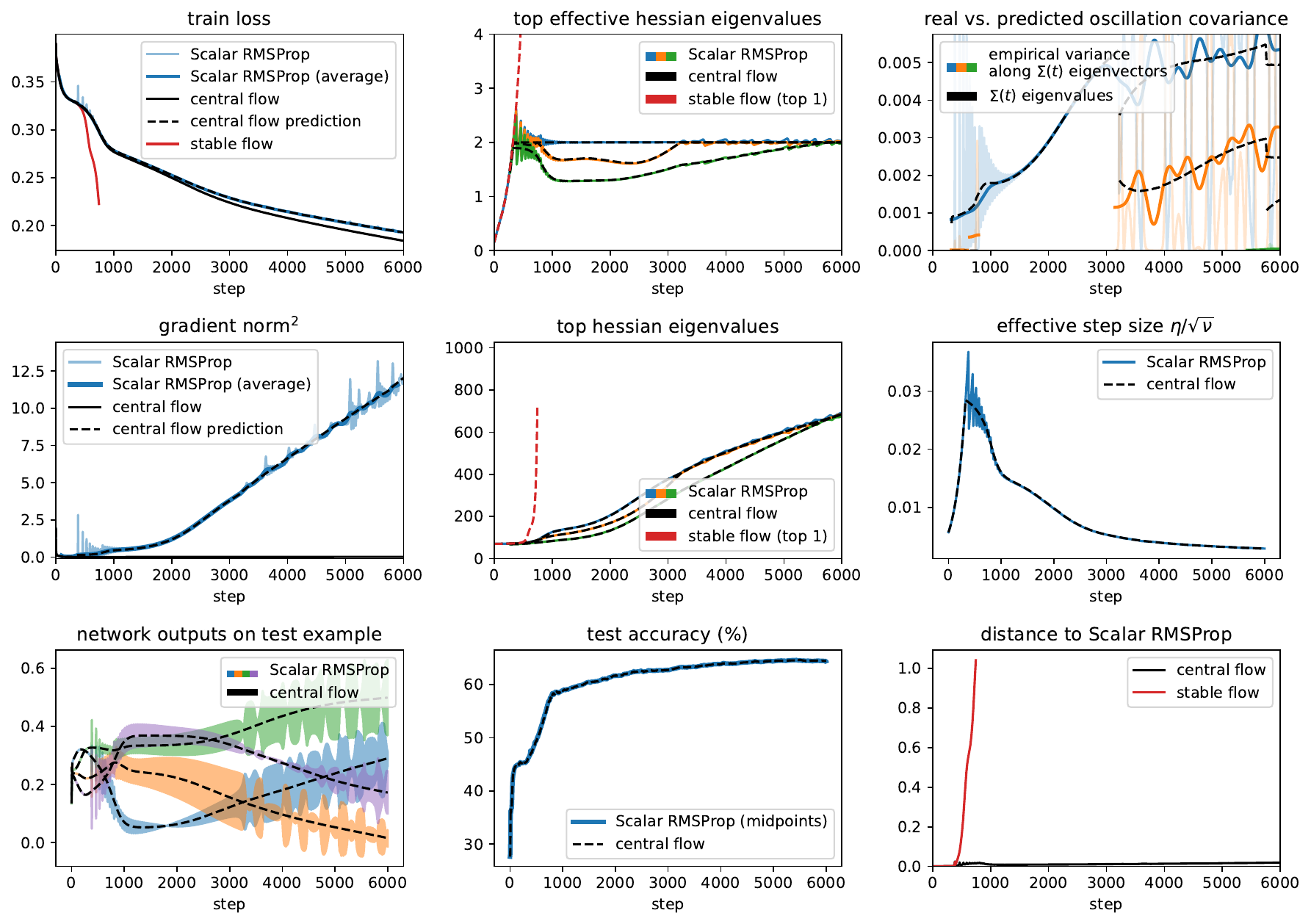}
        \caption{Scalar RMSProp central flow for a ViT with MSE loss, $\eta= $ 0.01, $\beta_2 = $ 0.99, and bias correction.}
        \label{fig:bulk-scalar-rmsprop:mse-vit-0}
    \end{figure}
                
    \begin{figure}[H]
        \centering
        \includegraphics[width=0.8\linewidth]{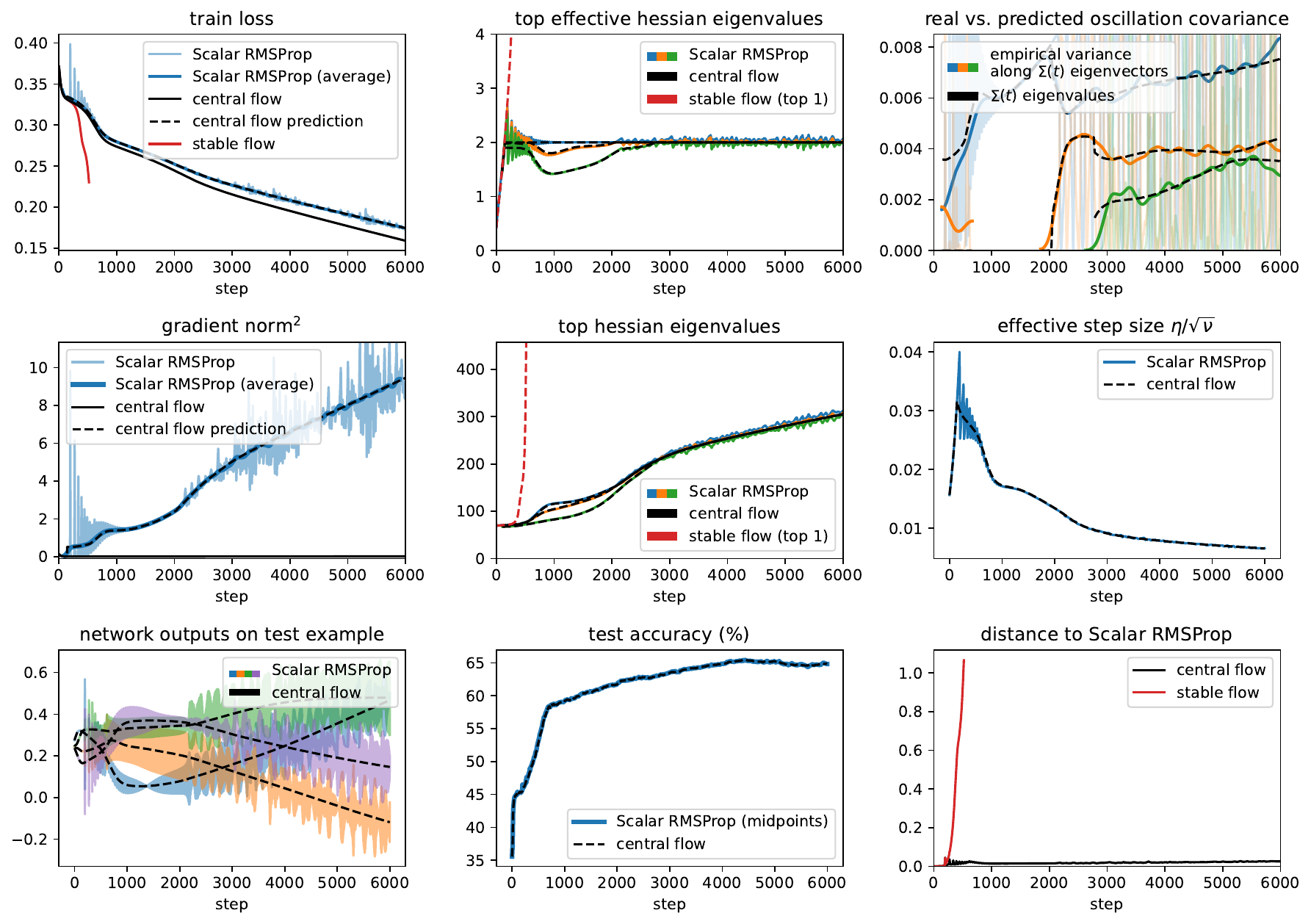}
        \caption{Scalar RMSProp central flow for a ViT with MSE loss, $\eta= $ 0.02, $\beta_2 = $ 0.99, and bias correction.}
        \label{fig:bulk-scalar-rmsprop:mse-vit-1}
    \end{figure}
                
    \begin{figure}[H]
        \centering
        \includegraphics[width=0.8\linewidth]{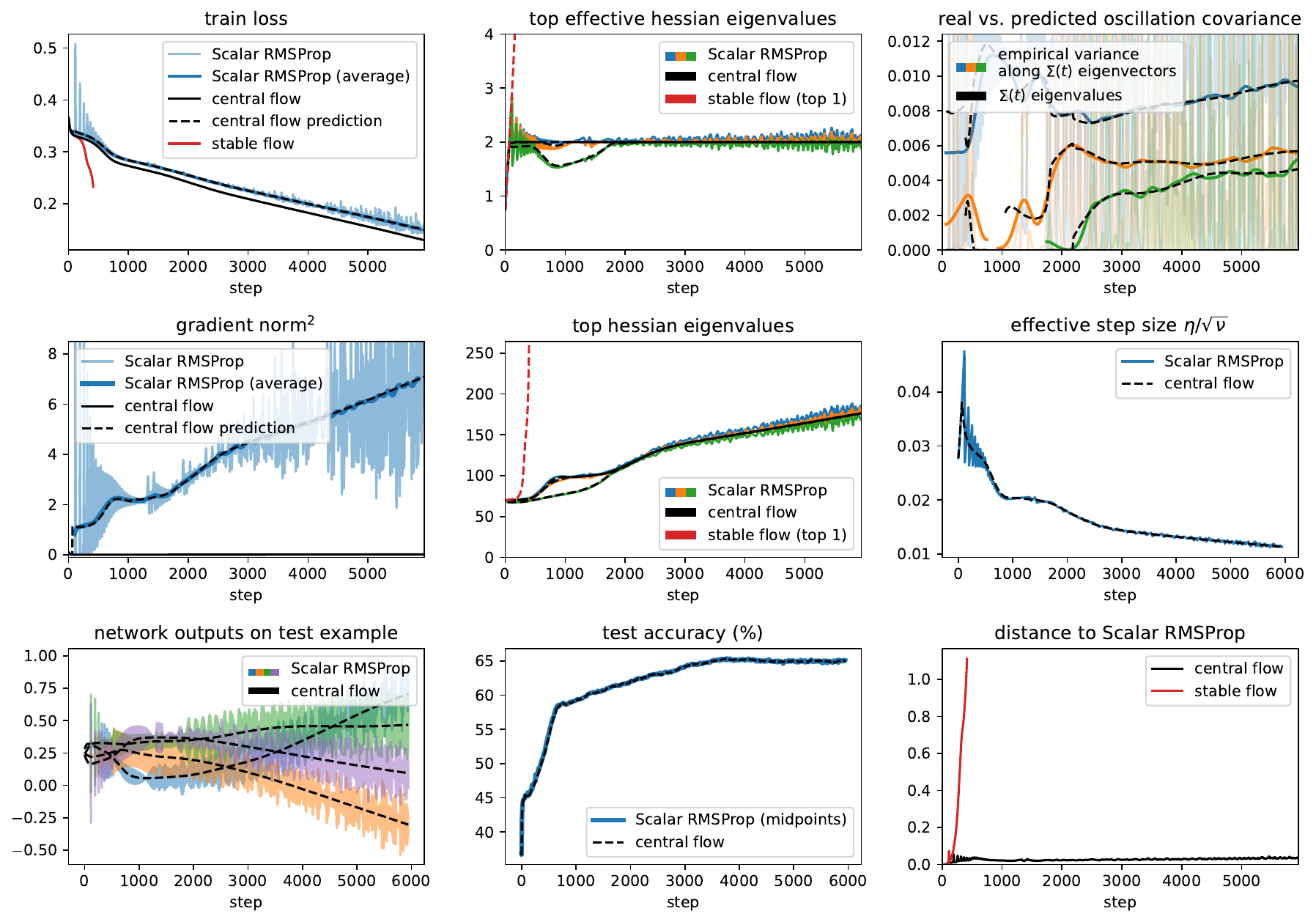}
        \caption{Scalar RMSProp central flow for a ViT with MSE loss, $\eta= $ 0.03, $\beta_2 = $ 0.99, and bias correction.}
        \label{fig:bulk-scalar-rmsprop:mse-vit-2}
    \end{figure}
                
    \begin{figure}[H]
        \centering
        \includegraphics[width=0.8\linewidth]{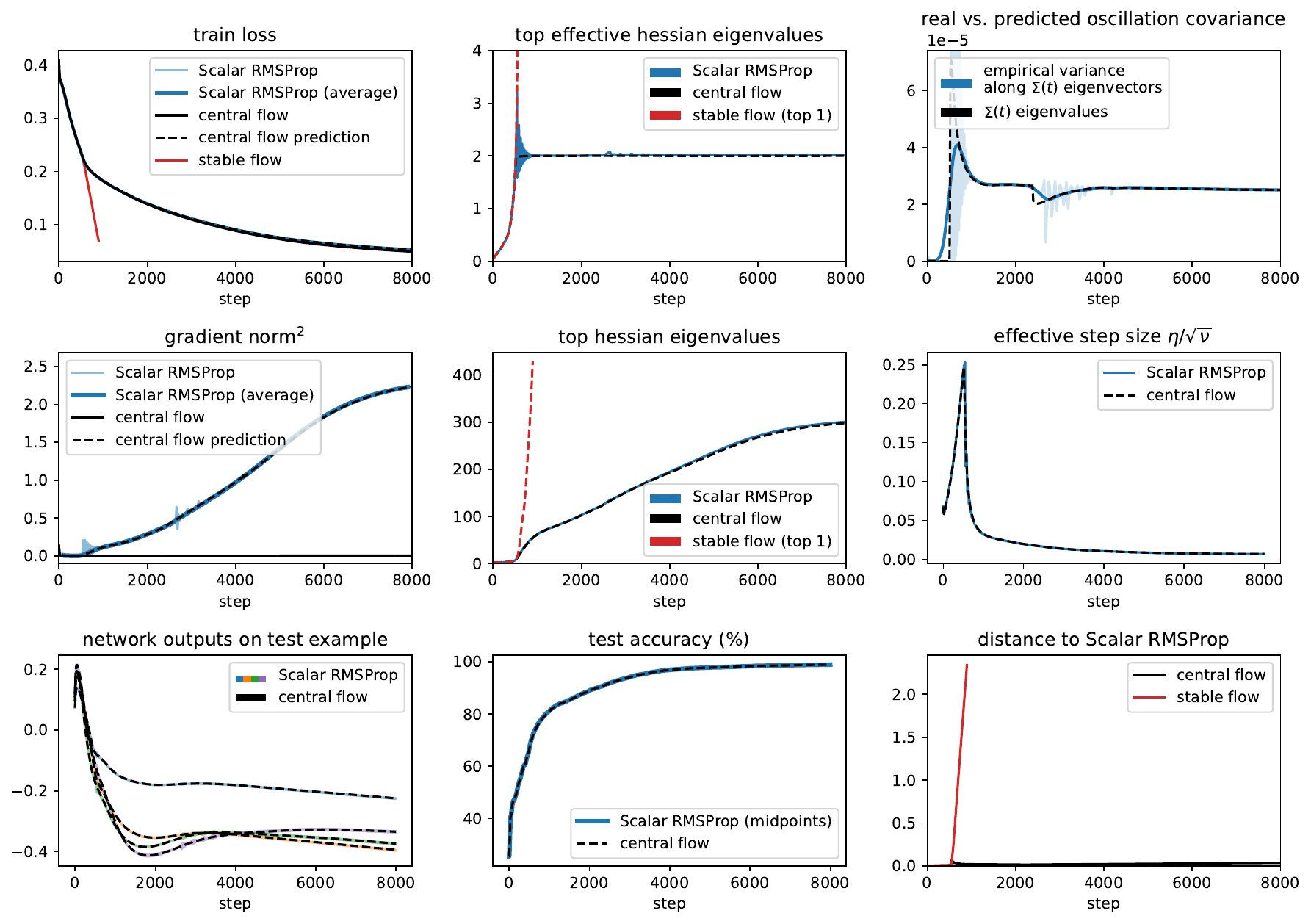}
        \caption{Scalar RMSProp central flow for a LSTM with MSE loss, $\eta= $ 0.01, $\beta_2 = $ 0.99, and bias correction.}
        \label{fig:bulk-scalar-rmsprop:mse-lstm-0}
    \end{figure}
                
    \begin{figure}[H]
        \centering
        \includegraphics[width=0.8\linewidth]{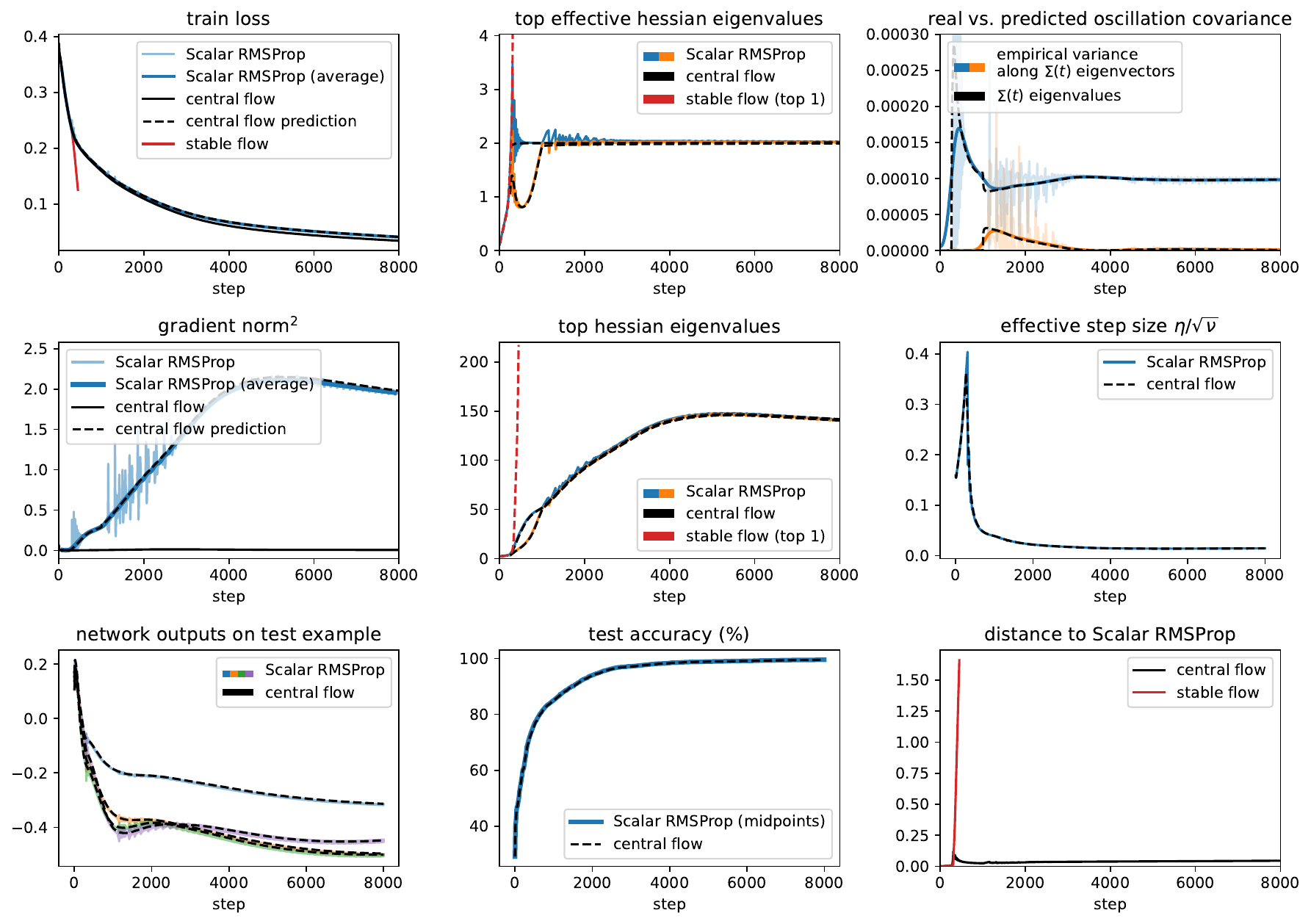}
        \caption{Scalar RMSProp central flow for a LSTM with MSE loss, $\eta= $ 0.02, $\beta_2 = $ 0.99, and bias correction.}
        \label{fig:bulk-scalar-rmsprop:mse-lstm-1}
    \end{figure}
                
    \begin{figure}[H]
        \centering
        \includegraphics[width=0.8\linewidth]{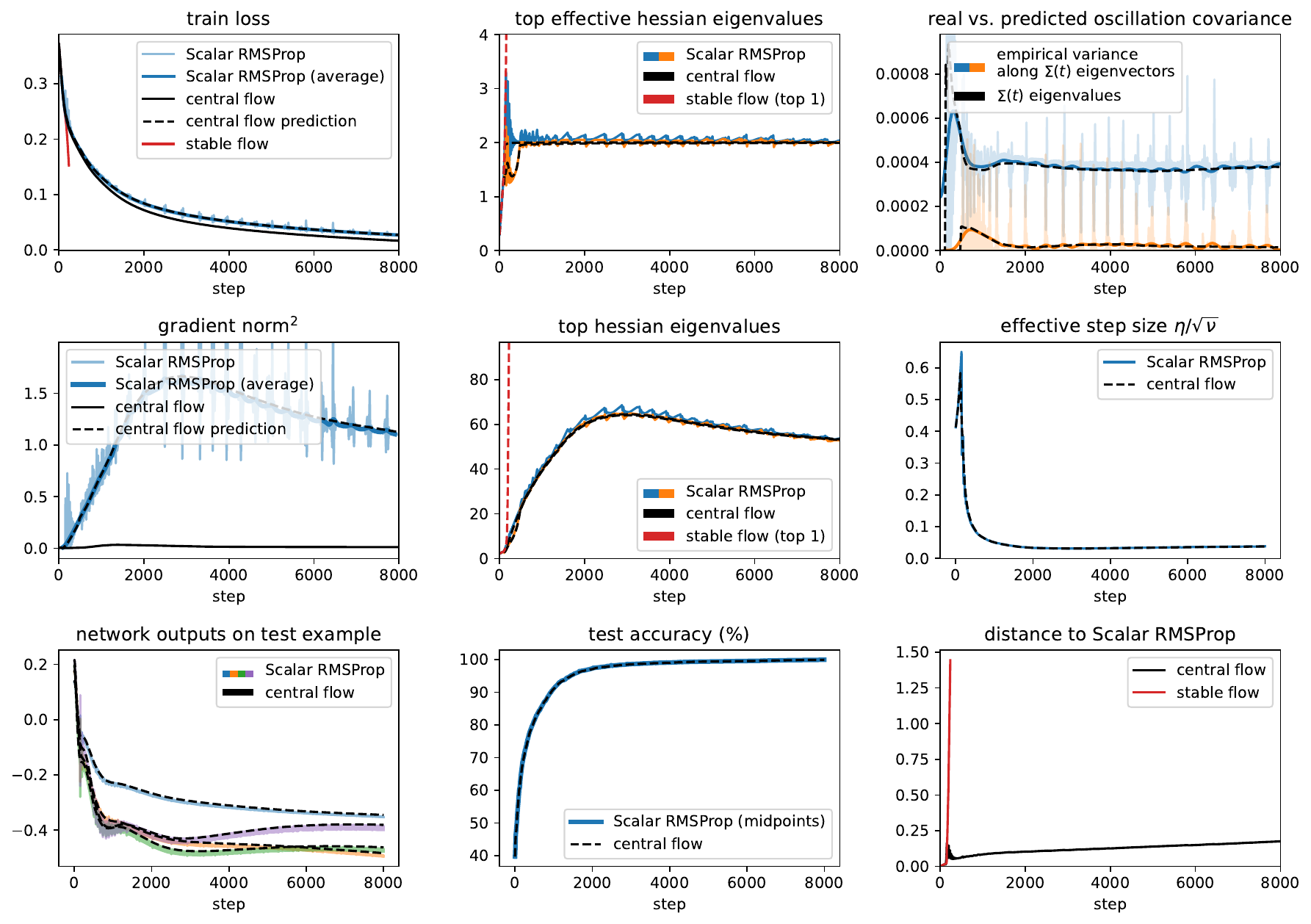}
        \caption{Scalar RMSProp central flow for a LSTM with MSE loss, $\eta= $ 0.03, $\beta_2 = $ 0.99, and bias correction.}
        \label{fig:bulk-scalar-rmsprop:mse-lstm-2}
    \end{figure}
                
    \begin{figure}[H]
        \centering
        \includegraphics[width=0.8\linewidth]{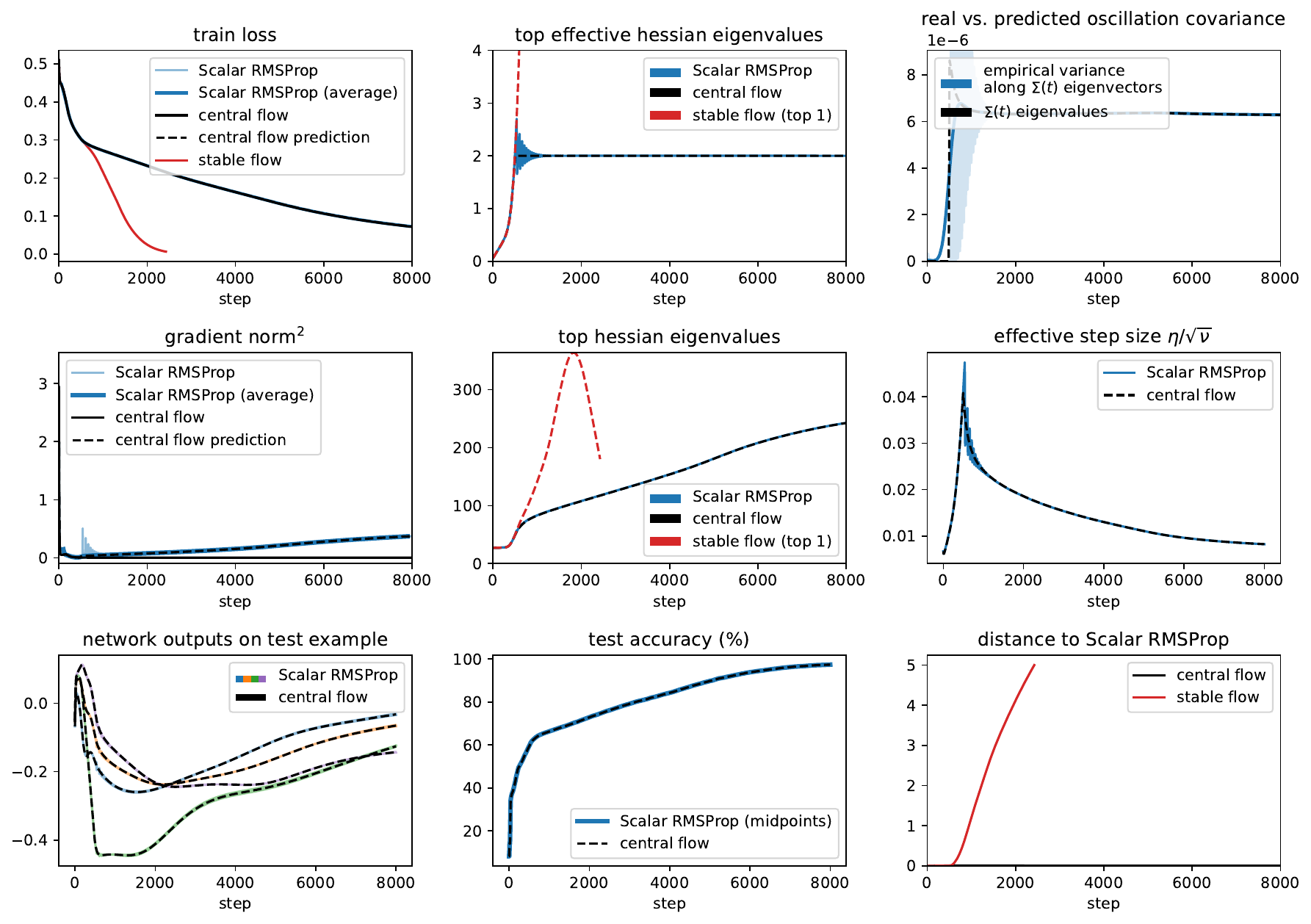}
        \caption{Scalar RMSProp central flow for a Transformer with MSE loss, $\eta= $ 0.01, $\beta_2 = $ 0.99, and bias correction.}
        \label{fig:bulk-scalar-rmsprop:mse-transformer-0}
    \end{figure}
                
    \begin{figure}[H]
        \centering
        \includegraphics[width=0.8\linewidth]{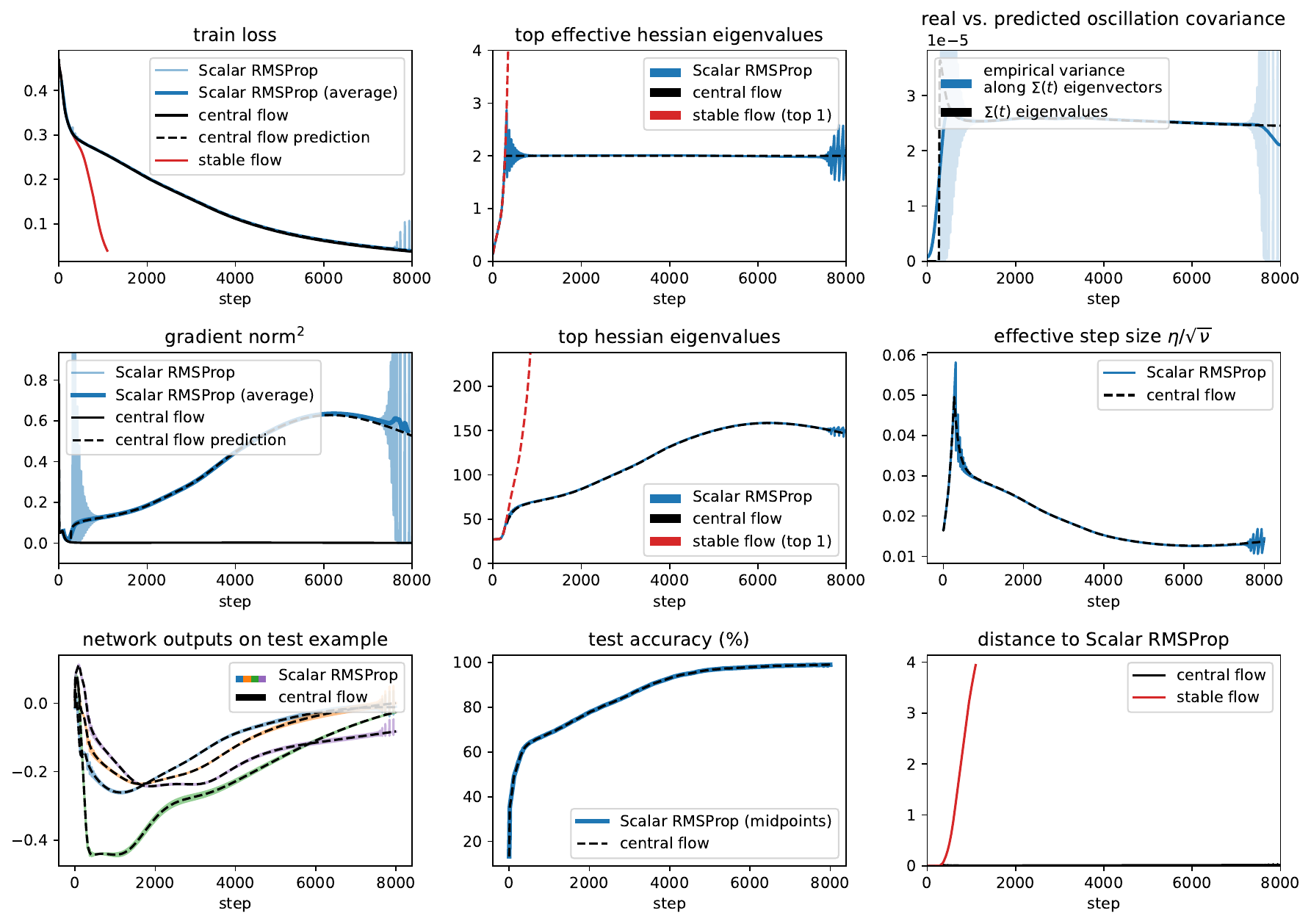}
        \caption{Scalar RMSProp central flow for a Transformer with MSE loss, $\eta= $ 0.02, $\beta_2 = $ 0.99, and bias correction.}
        \label{fig:bulk-scalar-rmsprop:mse-transformer-1}
    \end{figure}
                
    \begin{figure}[H]
        \centering
        \includegraphics[width=0.8\linewidth]{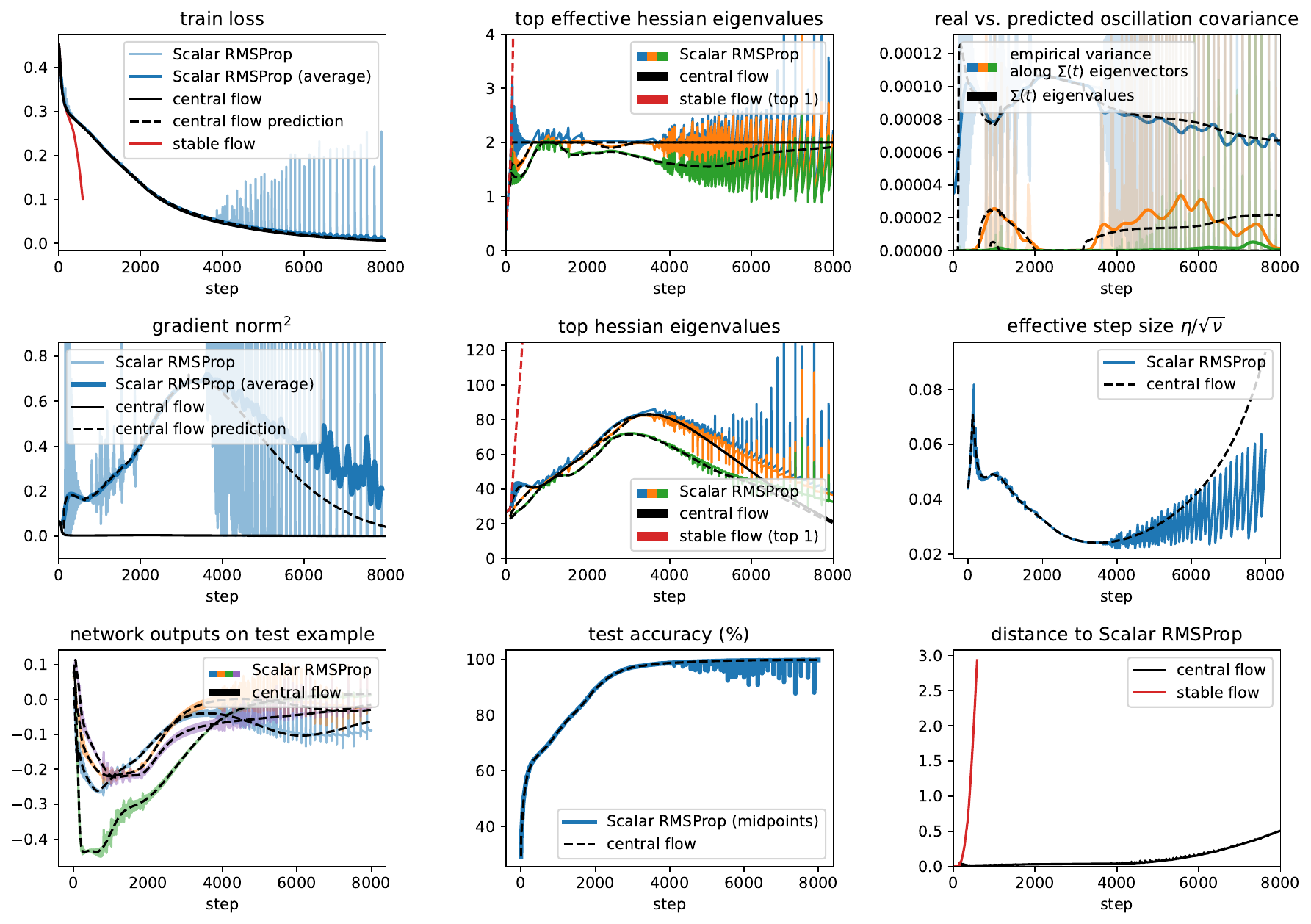}
        \caption{Scalar RMSProp central flow for a Transformer with MSE loss, $\eta= $ 0.03, $\beta_2 = $ 0.99, and bias correction.}
        \label{fig:bulk-scalar-rmsprop:mse-transformer-2}
    \end{figure}
                
    \begin{figure}[H]
        \centering
        \includegraphics[width=0.8\linewidth]{images/bulk-scalar-rmsprop/mse-mamba-0.pdf}
        \caption{Scalar RMSProp central flow for a Mamba with MSE loss, $\eta= $ 0.007, $\beta_2 = $ 0.99, and bias correction.}
        \label{fig:bulk-scalar-rmsprop:mse-mamba-0}
    \end{figure}
                
    \begin{figure}[H]
        \centering
        \includegraphics[width=0.8\linewidth]{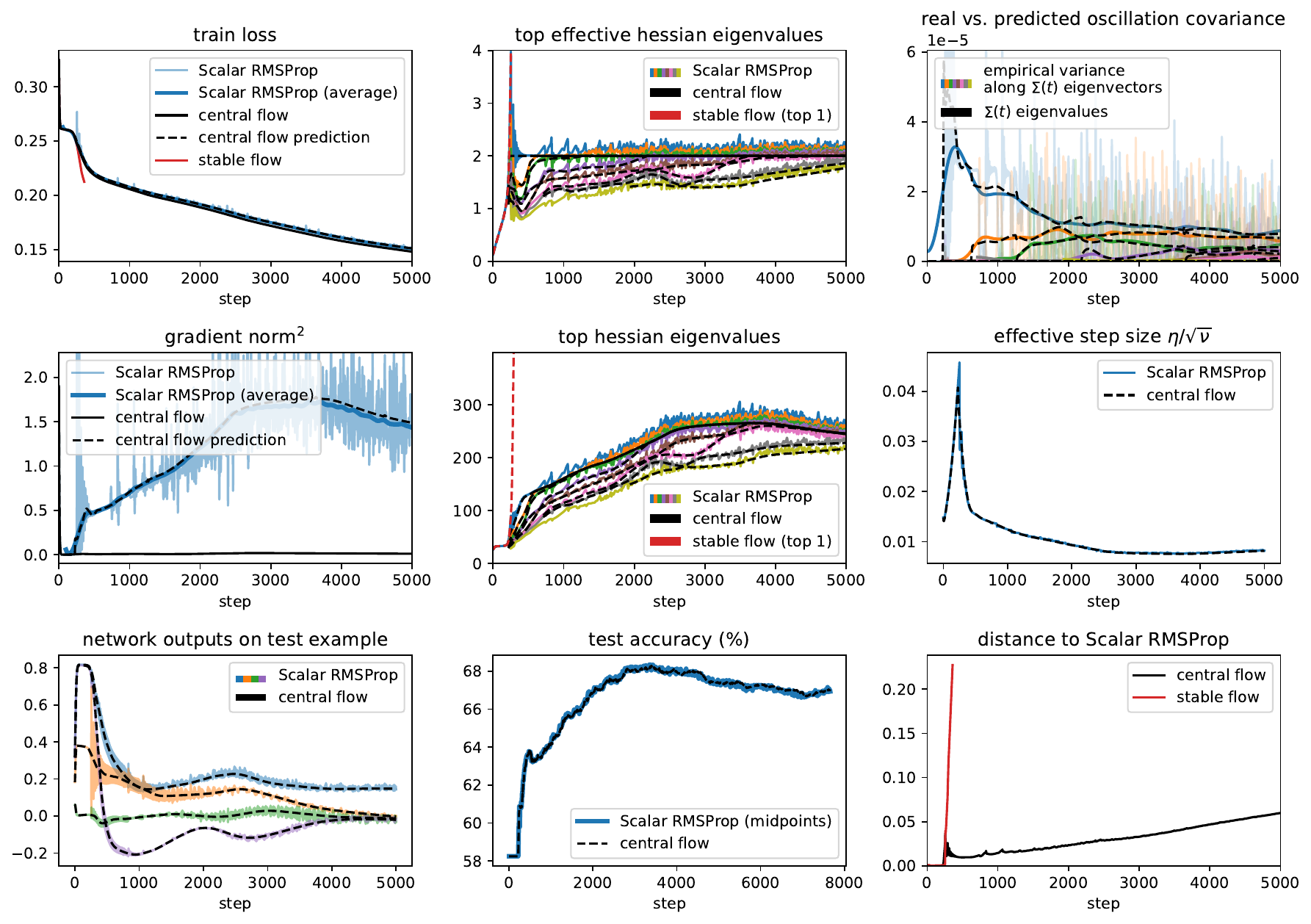}
        \caption{Scalar RMSProp central flow for a Mamba with MSE loss, $\eta= $ 0.01, $\beta_2 = $ 0.99, and bias correction.}
        \label{fig:bulk-scalar-rmsprop:mse-mamba-1}
    \end{figure}
                
    \begin{figure}[H]
        \centering
        \includegraphics[width=0.8\linewidth]{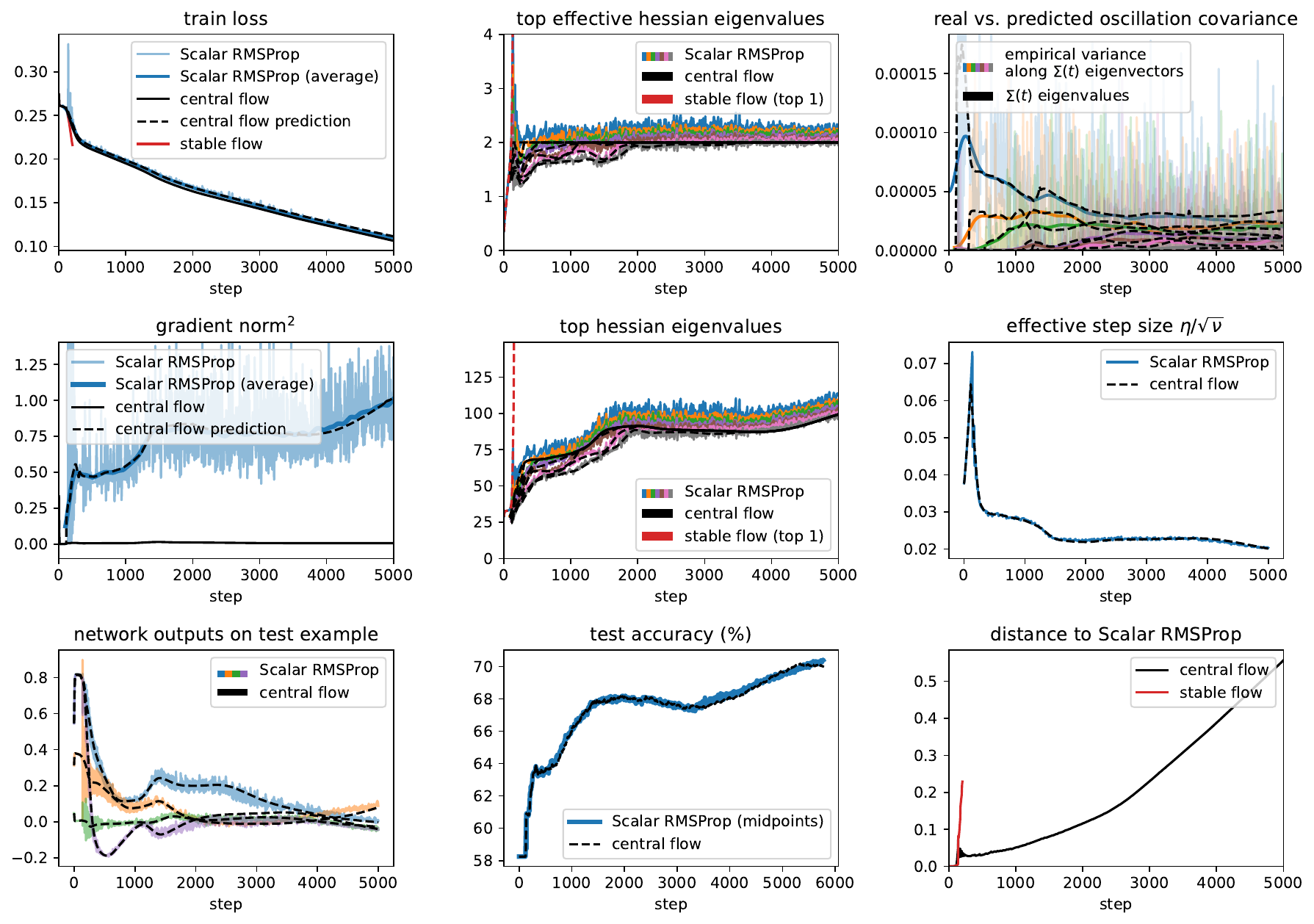}
        \caption{Scalar RMSProp central flow for a Mamba with MSE loss, $\eta= $ 0.02, $\beta_2 = $ 0.99, and bias correction.}
        \label{fig:bulk-scalar-rmsprop:mse-mamba-2}
    \end{figure}
                \end{specialfigures}

%% file: images/bulk-scalar-rmsprop/figures-ce.tex
\begin{specialfigures}

    \begin{figure}[H]
        \centering
        \includegraphics[width=0.8\linewidth]{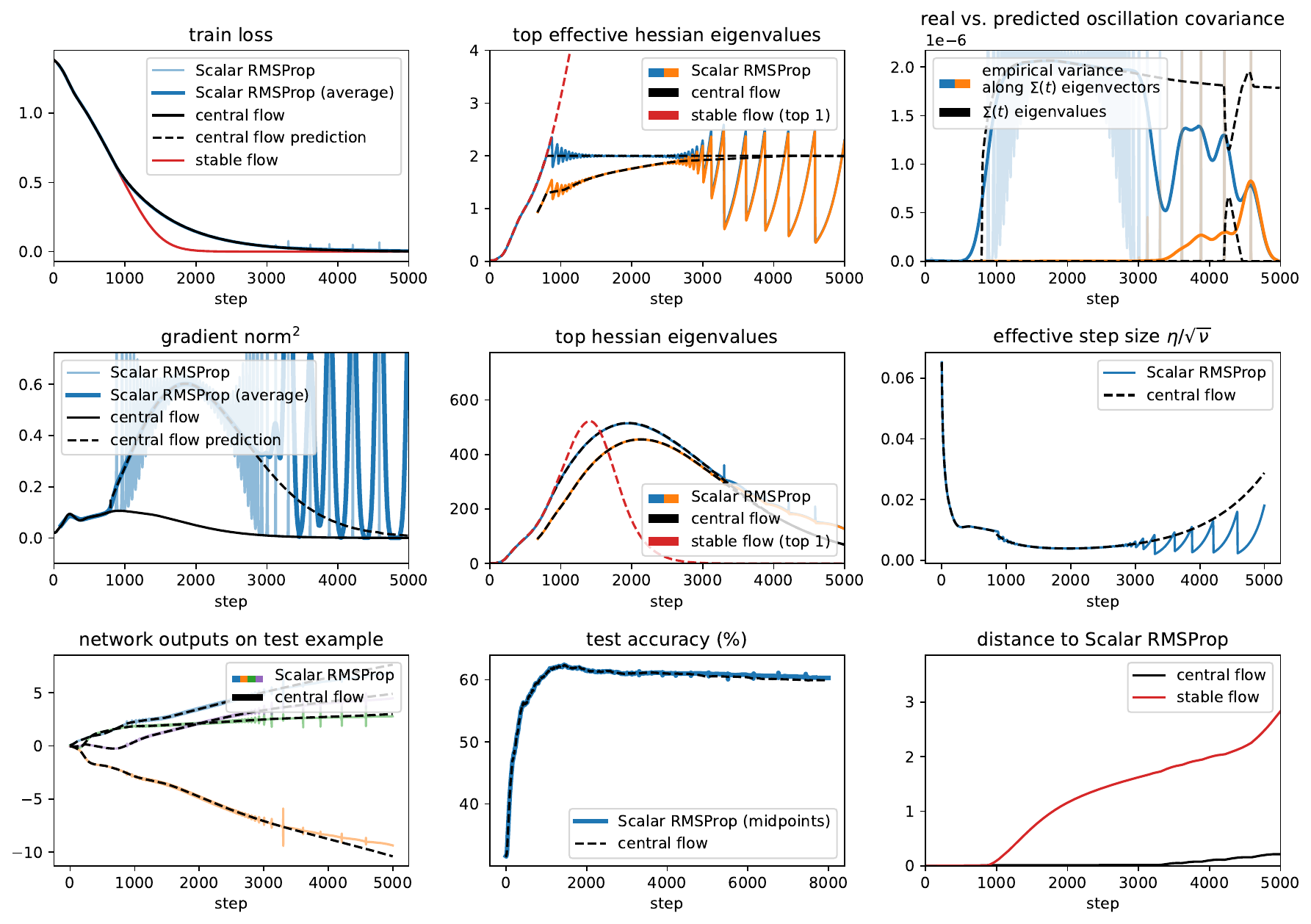}
        \caption{Scalar RMSProp central flow for a CNN with CE loss, $\eta= $ 0.003, $\beta_2 = $ 0.99, and bias correction.}
        \label{fig:bulk-scalar-rmsprop:ce-cnn-0}
    \end{figure}
                
    \begin{figure}[H]
        \centering
        \includegraphics[width=0.8\linewidth]{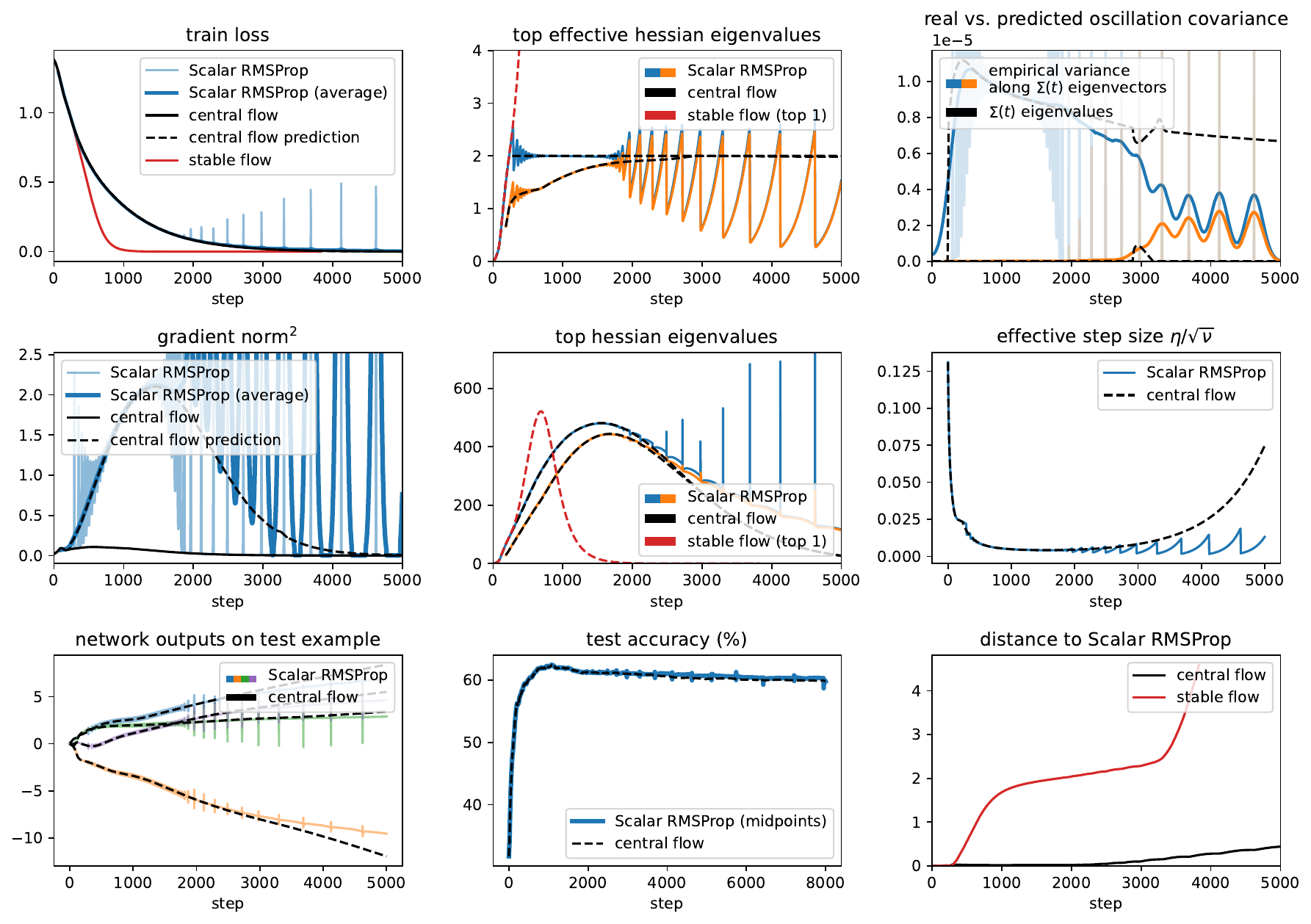}
        \caption{Scalar RMSProp central flow for a CNN with CE loss, $\eta= $ 0.006, $\beta_2 = $ 0.99, and bias correction.}
        \label{fig:bulk-scalar-rmsprop:ce-cnn-1}
    \end{figure}
                
    \begin{figure}[H]
        \centering
        \includegraphics[width=0.8\linewidth]{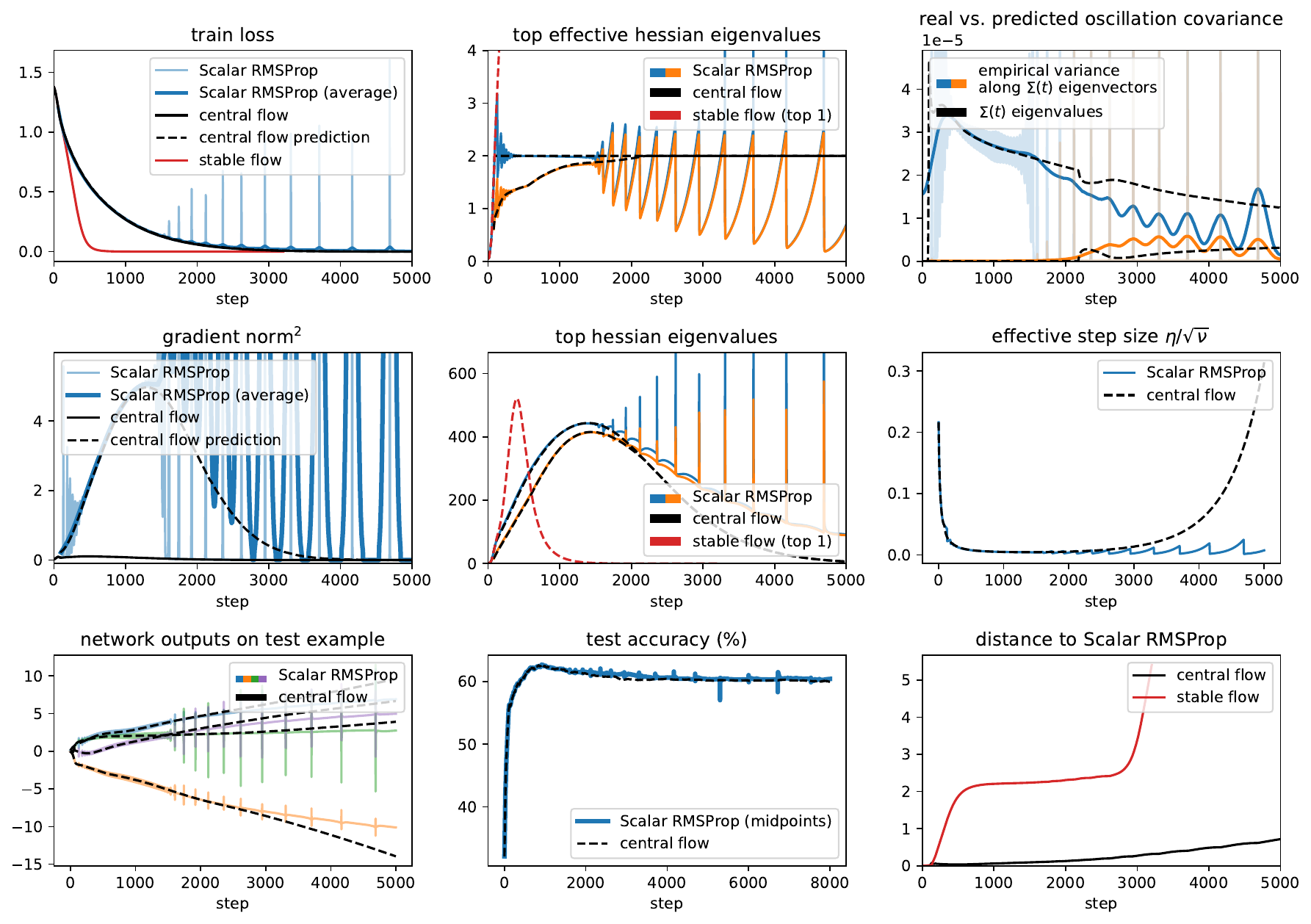}
        \caption{Scalar RMSProp central flow for a CNN with CE loss, $\eta= $ 0.01, $\beta_2 = $ 0.99, and bias correction.}
        \label{fig:bulk-scalar-rmsprop:ce-cnn-2}
    \end{figure}
                
    \begin{figure}[H]
        \centering
        \includegraphics[width=0.8\linewidth]{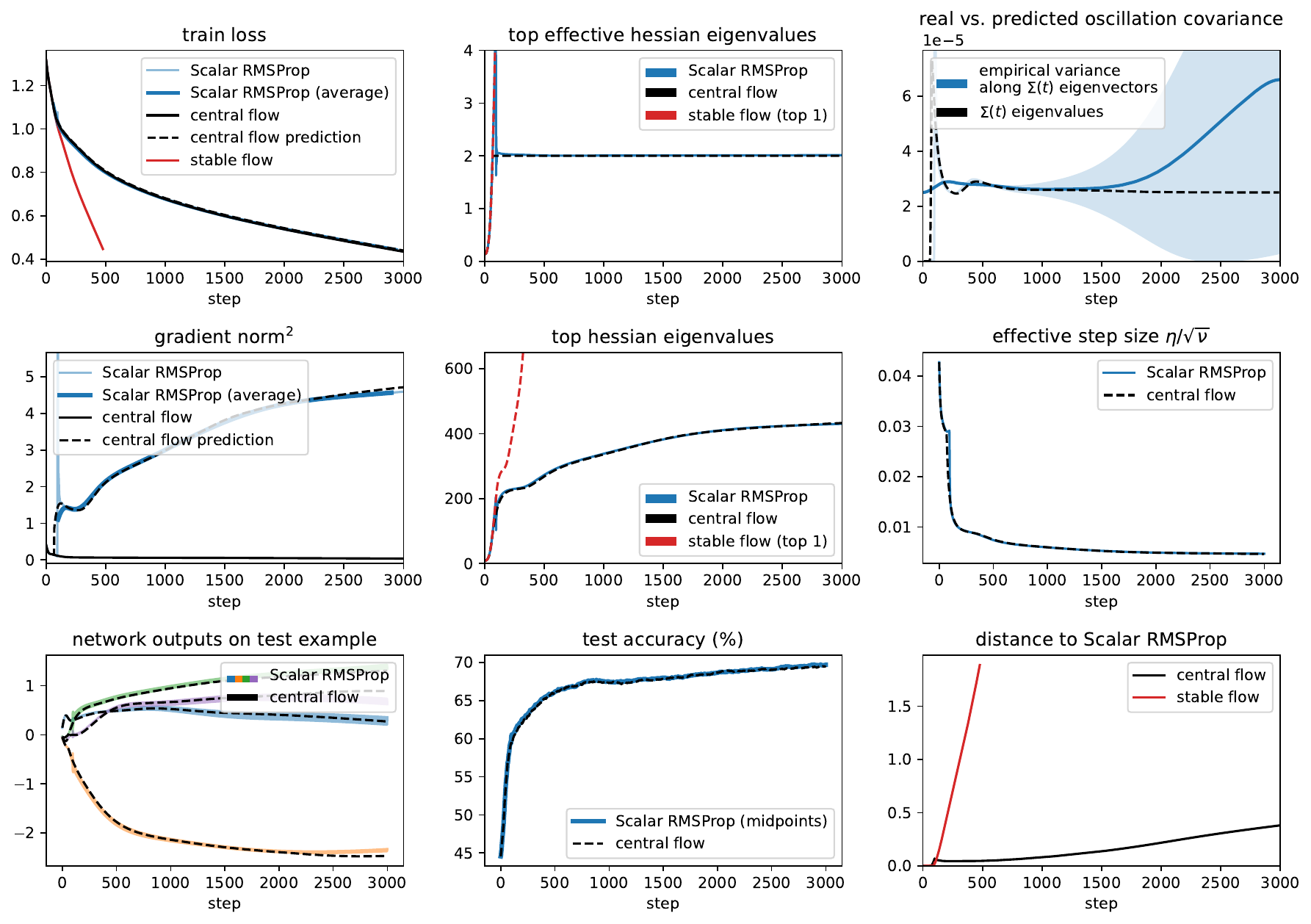}
        \caption{Scalar RMSProp central flow for a ResNet with CE loss, $\eta= $ 0.01, $\beta_2 = $ 0.99, and bias correction.}
        \label{fig:bulk-scalar-rmsprop:ce-resnet-0}
    \end{figure}
                
    \begin{figure}[H]
        \centering
        \includegraphics[width=0.8\linewidth]{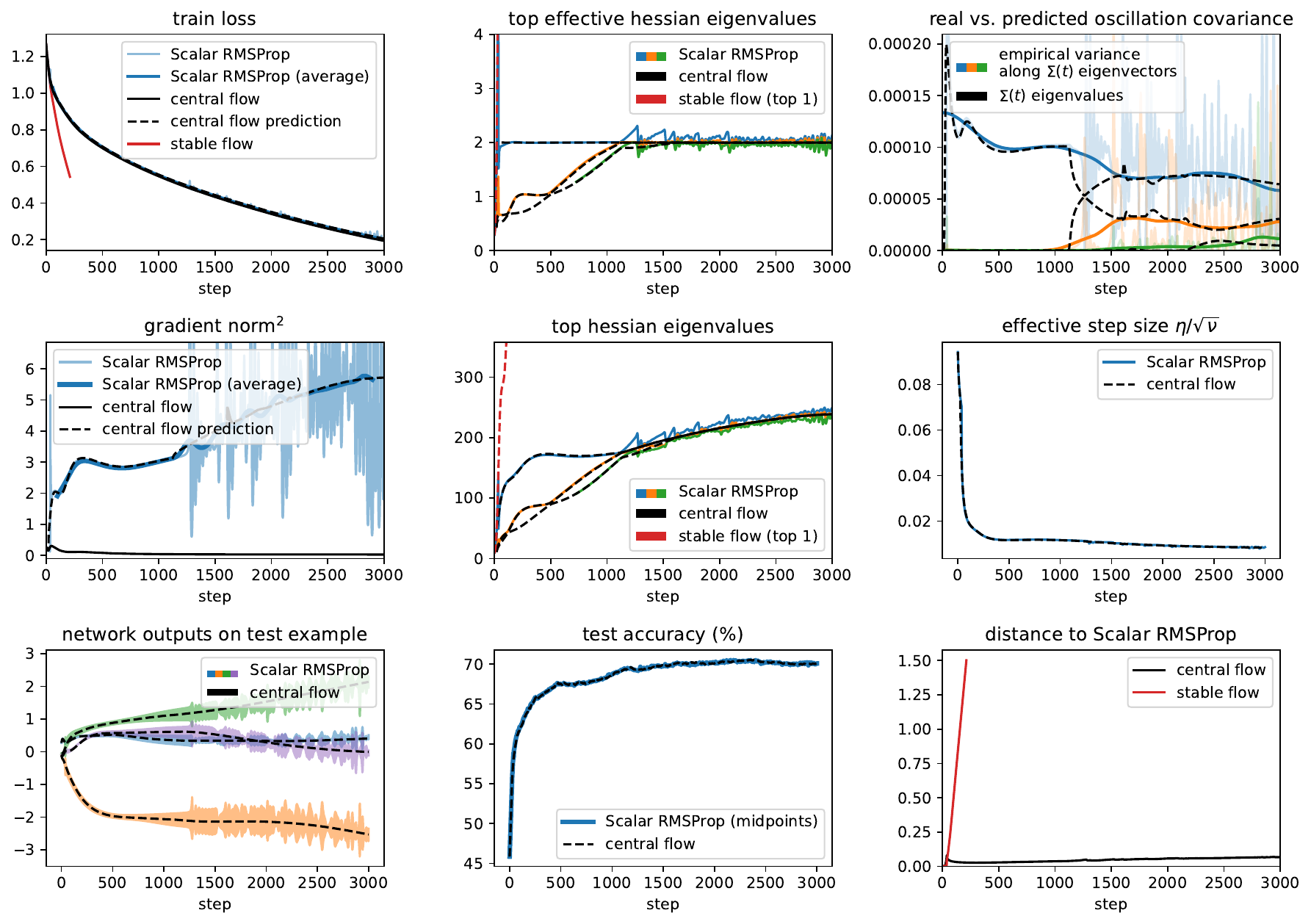}
        \caption{Scalar RMSProp central flow for a ResNet with CE loss, $\eta= $ 0.02, $\beta_2 = $ 0.99, and bias correction.}
        \label{fig:bulk-scalar-rmsprop:ce-resnet-1}
    \end{figure}
                
    \begin{figure}[H]
        \centering
        \includegraphics[width=0.8\linewidth]{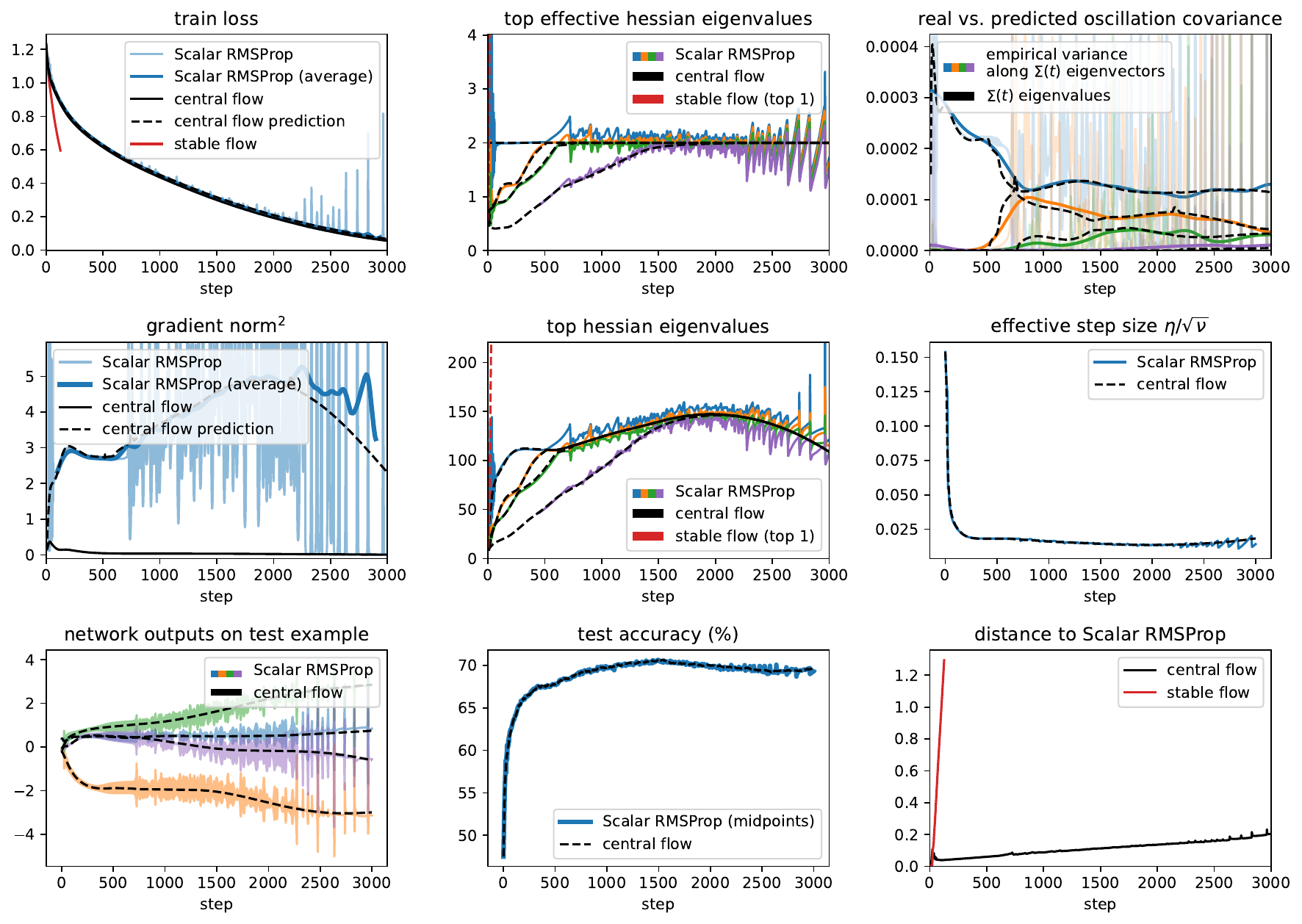}
        \caption{Scalar RMSProp central flow for a ResNet with CE loss, $\eta= $ 0.03, $\beta_2 = $ 0.99, and bias correction.}
        \label{fig:bulk-scalar-rmsprop:ce-resnet-2}
    \end{figure}
                
    \begin{figure}[H]
        \centering
        \includegraphics[width=0.8\linewidth]{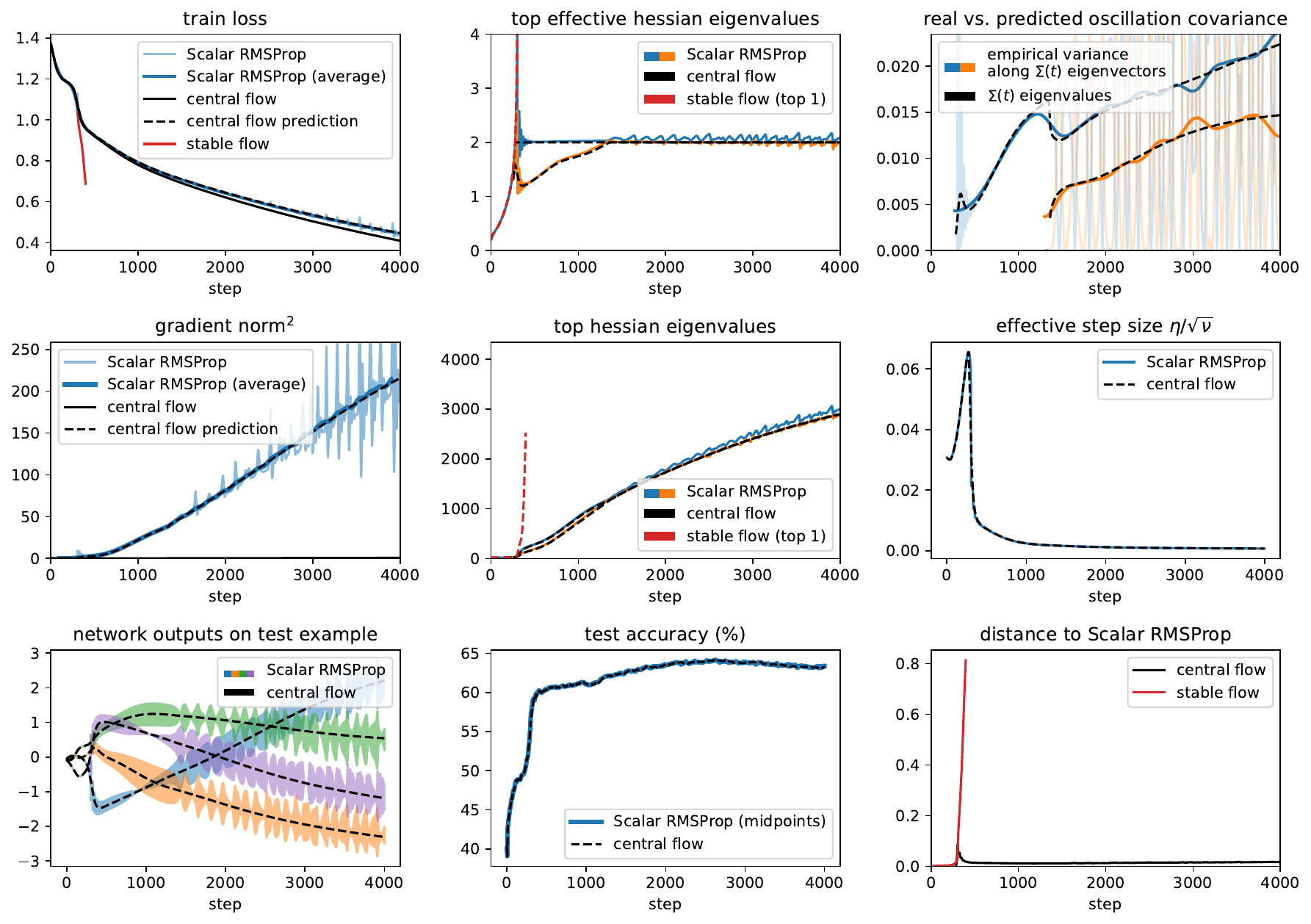}
        \caption{Scalar RMSProp central flow for a ViT with CE loss, $\eta= $ 0.01, $\beta_2 = $ 0.99, and bias correction.}
        \label{fig:bulk-scalar-rmsprop:ce-vit-0}
    \end{figure}
                
    \begin{figure}[H]
        \centering
        \includegraphics[width=0.8\linewidth]{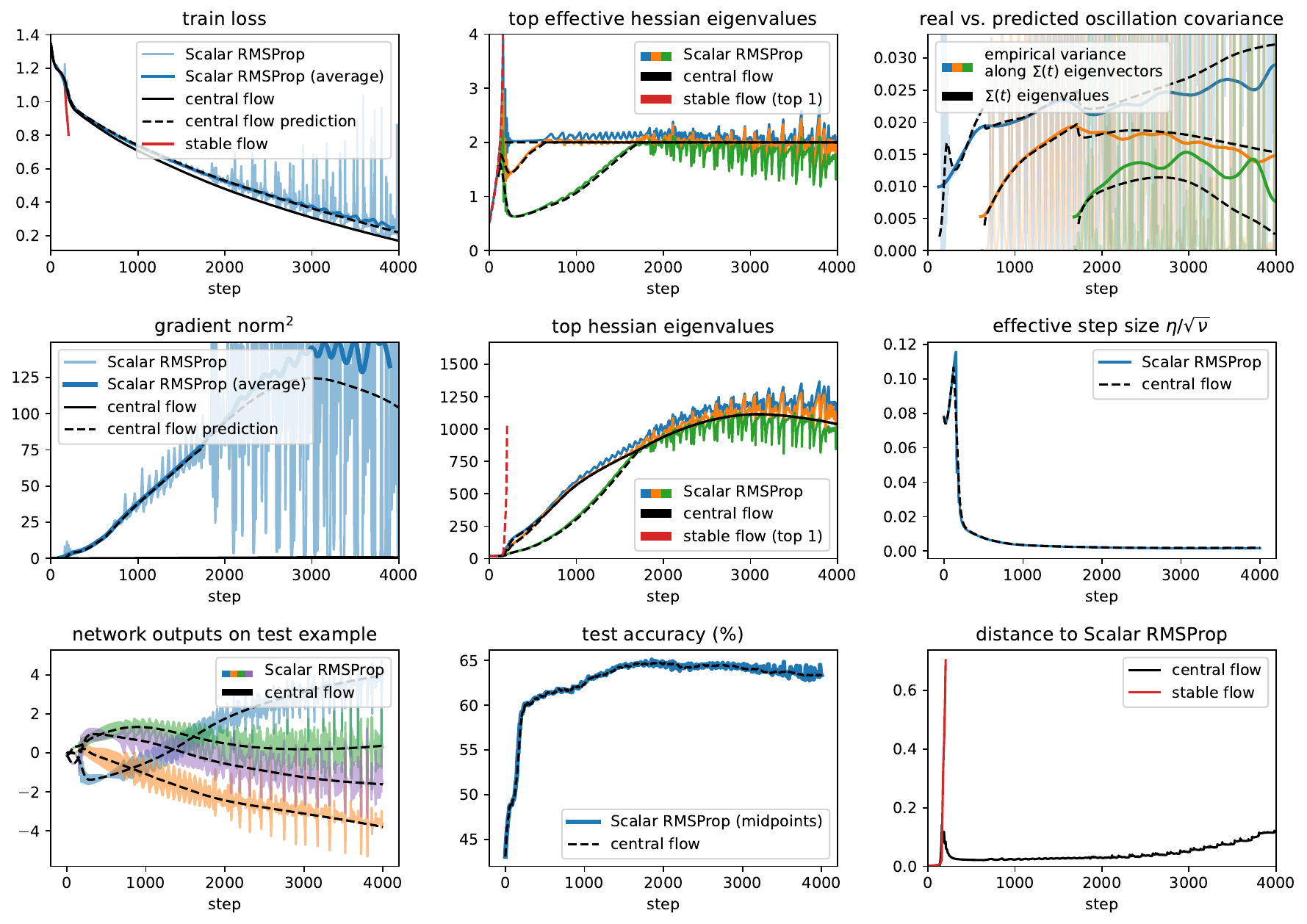}
        \caption{Scalar RMSProp central flow for a ViT with CE loss, $\eta= $ 0.02, $\beta_2 = $ 0.99, and bias correction.}
        \label{fig:bulk-scalar-rmsprop:ce-vit-1}
    \end{figure}
                
    \begin{figure}[H]
        \centering
        \includegraphics[width=0.8\linewidth]{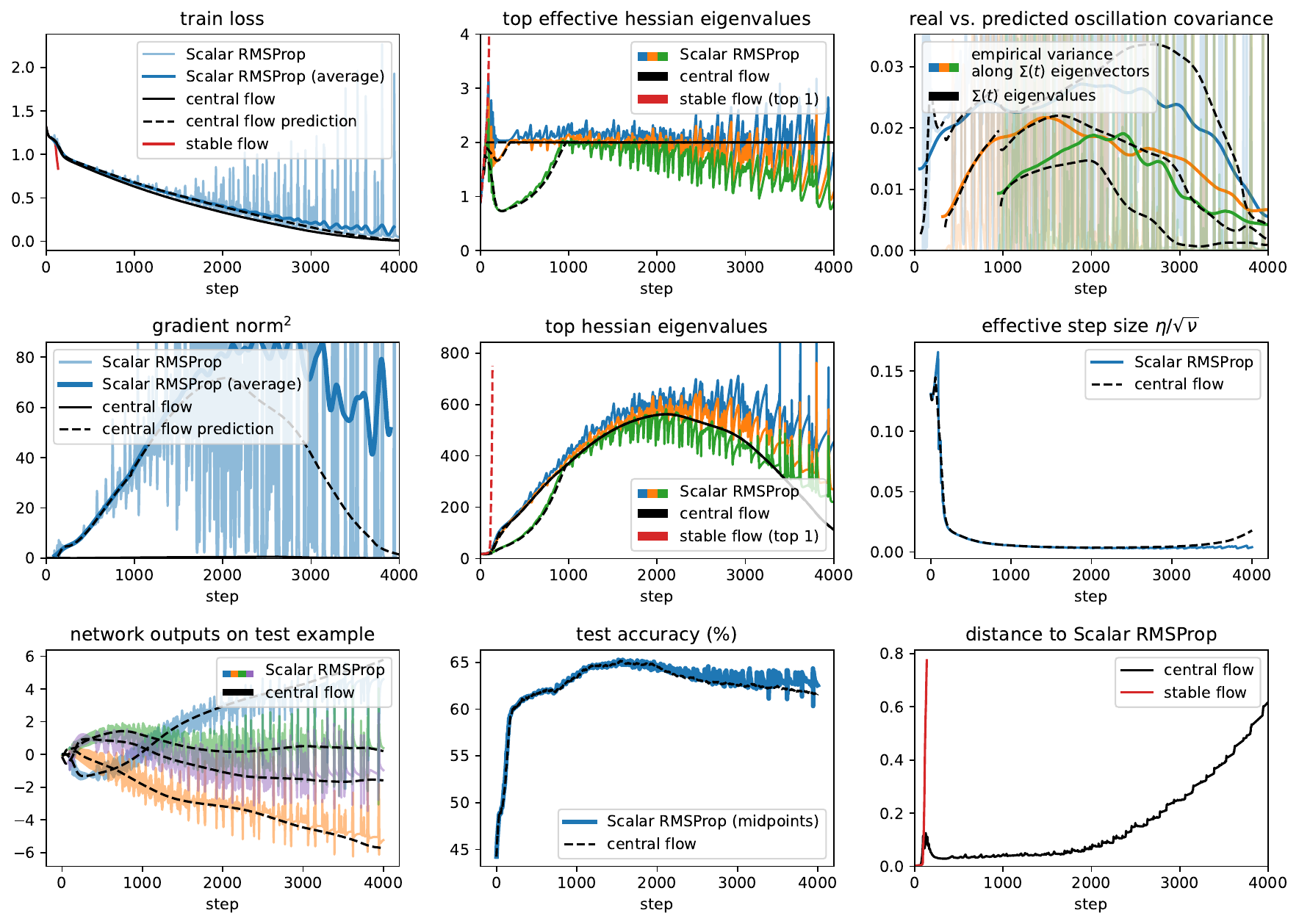}
        \caption{Scalar RMSProp central flow for a ViT with CE loss, $\eta= $ 0.03, $\beta_2 = $ 0.99, and bias correction.}
        \label{fig:bulk-scalar-rmsprop:ce-vit-2}
    \end{figure}
                
    \begin{figure}[H]
        \centering
        \includegraphics[width=0.8\linewidth]{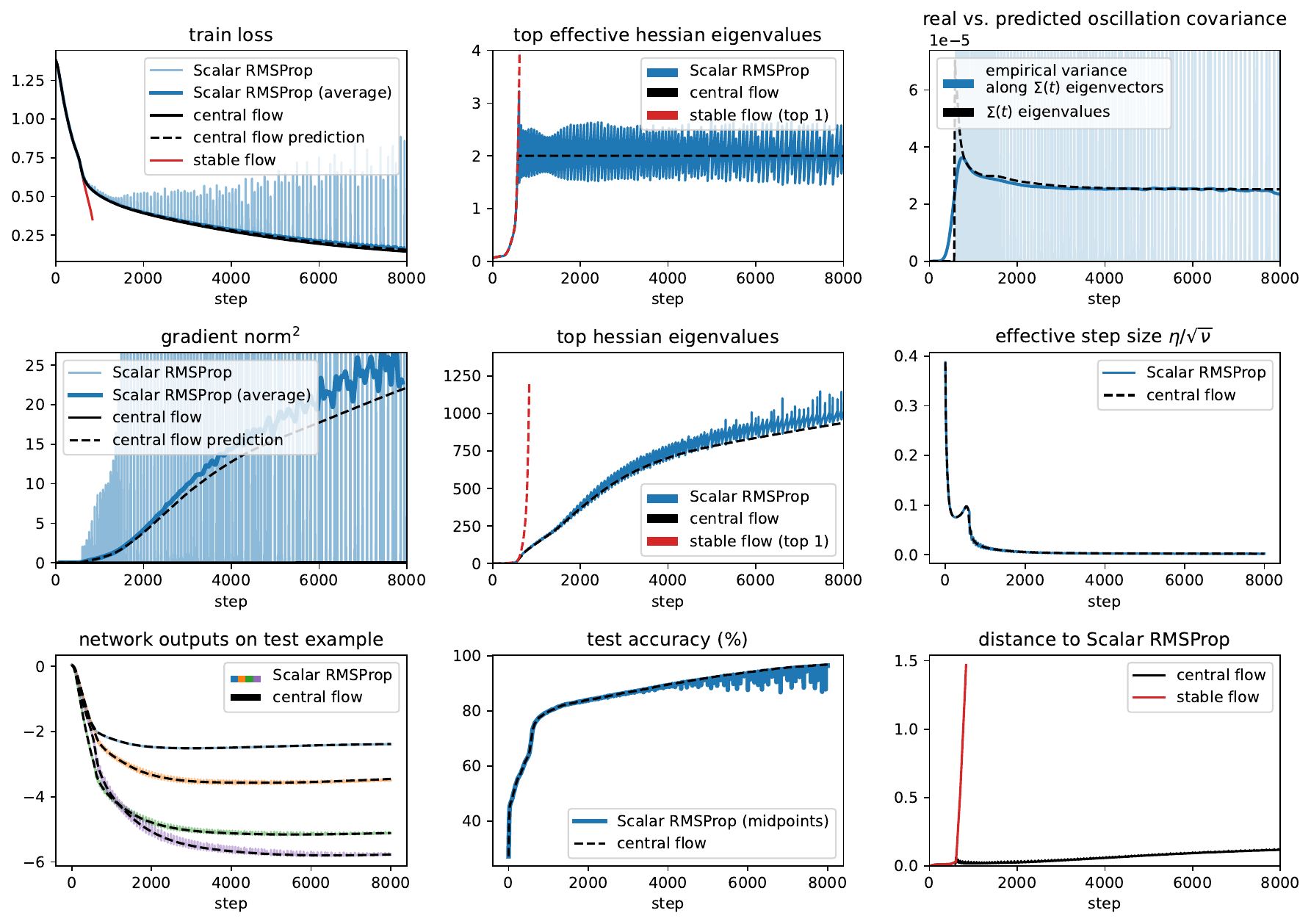}
        \caption{Scalar RMSProp central flow for a LSTM with CE loss, $\eta= $ 0.01, $\beta_2 = $ 0.99, and bias correction.}
        \label{fig:bulk-scalar-rmsprop:ce-lstm-0}
    \end{figure}
                
    \begin{figure}[H]
        \centering
        \includegraphics[width=0.8\linewidth]{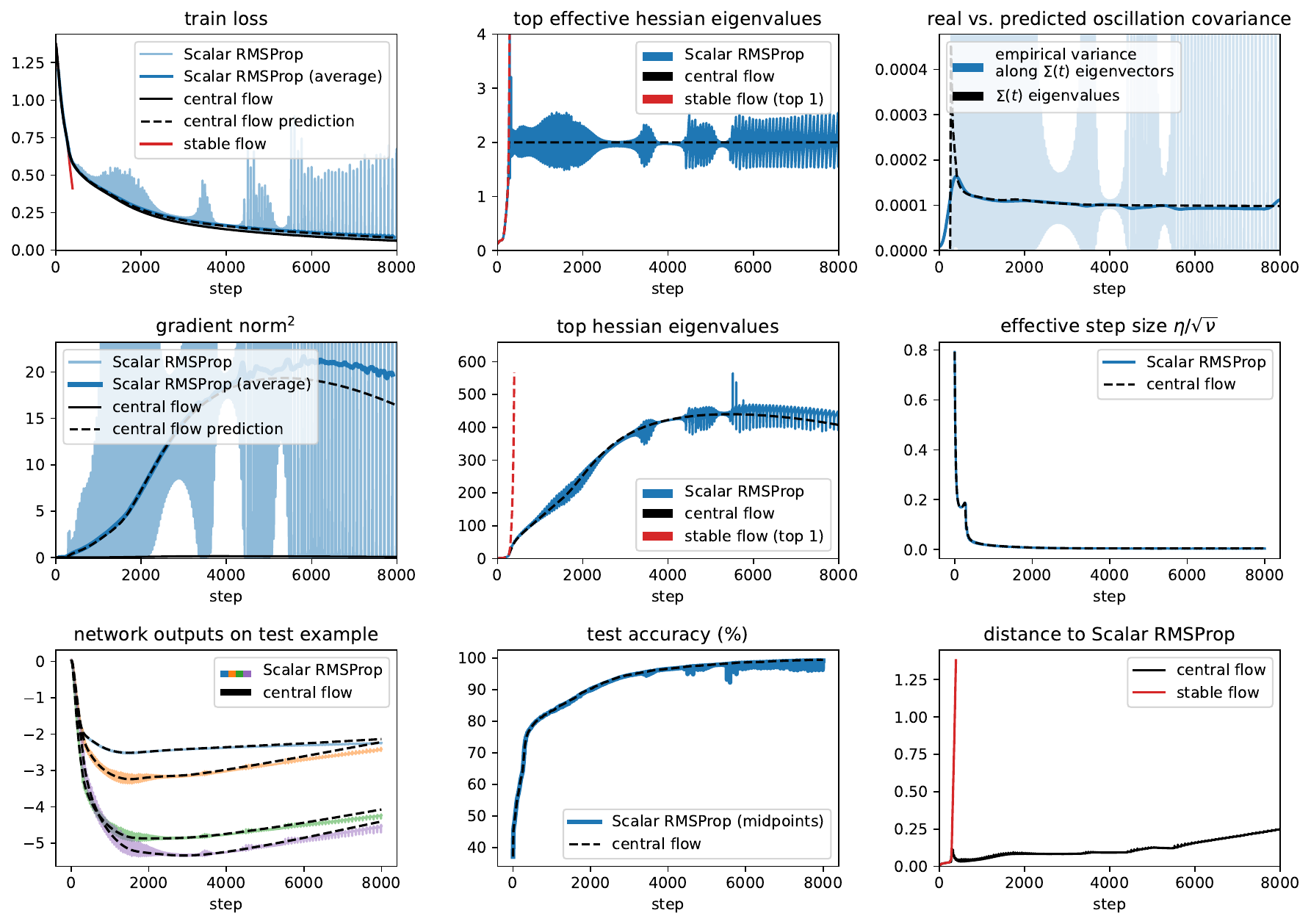}
        \caption{Scalar RMSProp central flow for a LSTM with CE loss, $\eta= $ 0.02, $\beta_2 = $ 0.99, and bias correction.}
        \label{fig:bulk-scalar-rmsprop:ce-lstm-1}
    \end{figure}
                
    \begin{figure}[H]
        \centering
        \includegraphics[width=0.8\linewidth]{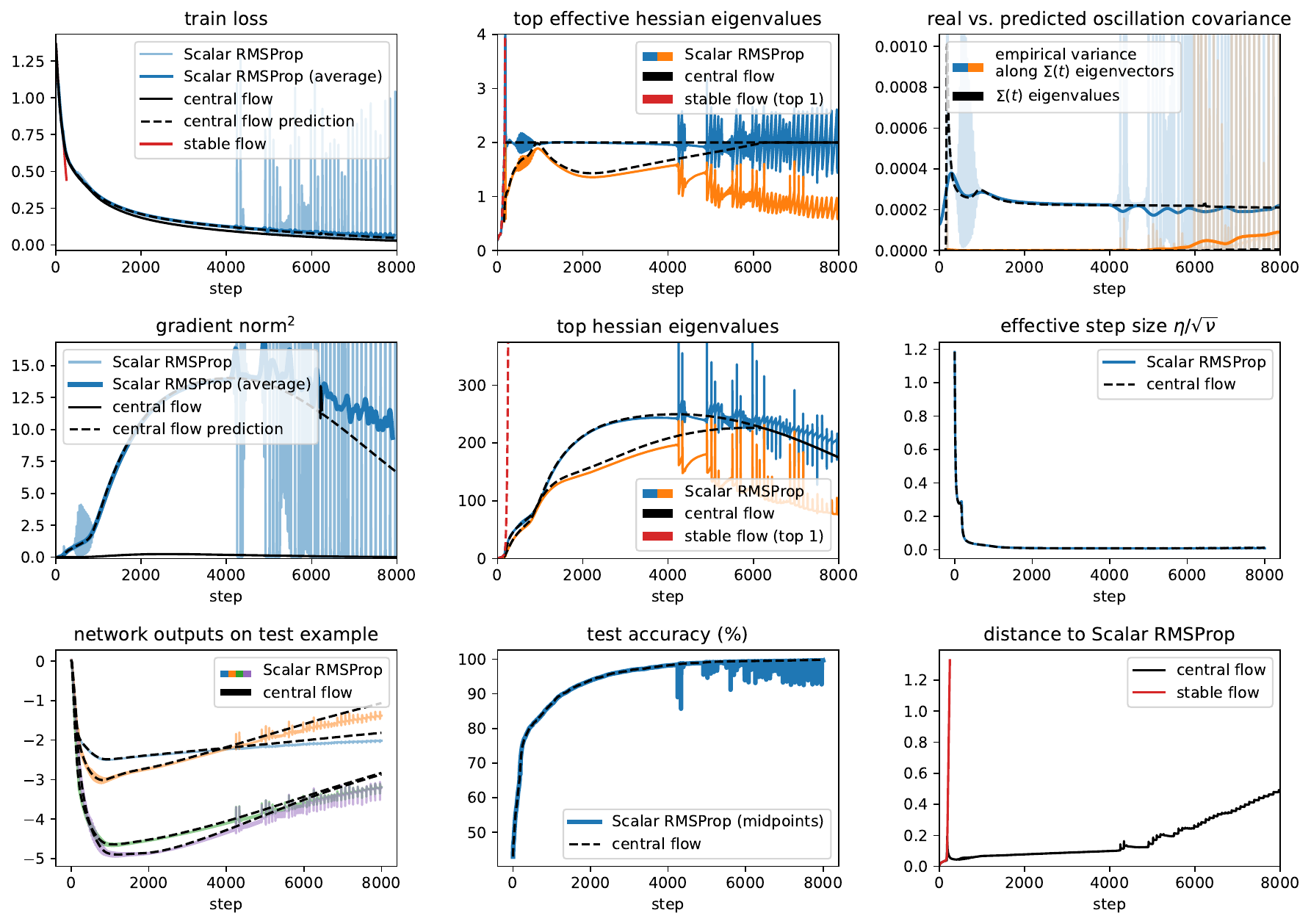}
        \caption{Scalar RMSProp central flow for a LSTM with CE loss, $\eta= $ 0.03, $\beta_2 = $ 0.99, and bias correction.}
        \label{fig:bulk-scalar-rmsprop:ce-lstm-2}
    \end{figure}
                
    \begin{figure}[H]
        \centering
        \includegraphics[width=0.8\linewidth]{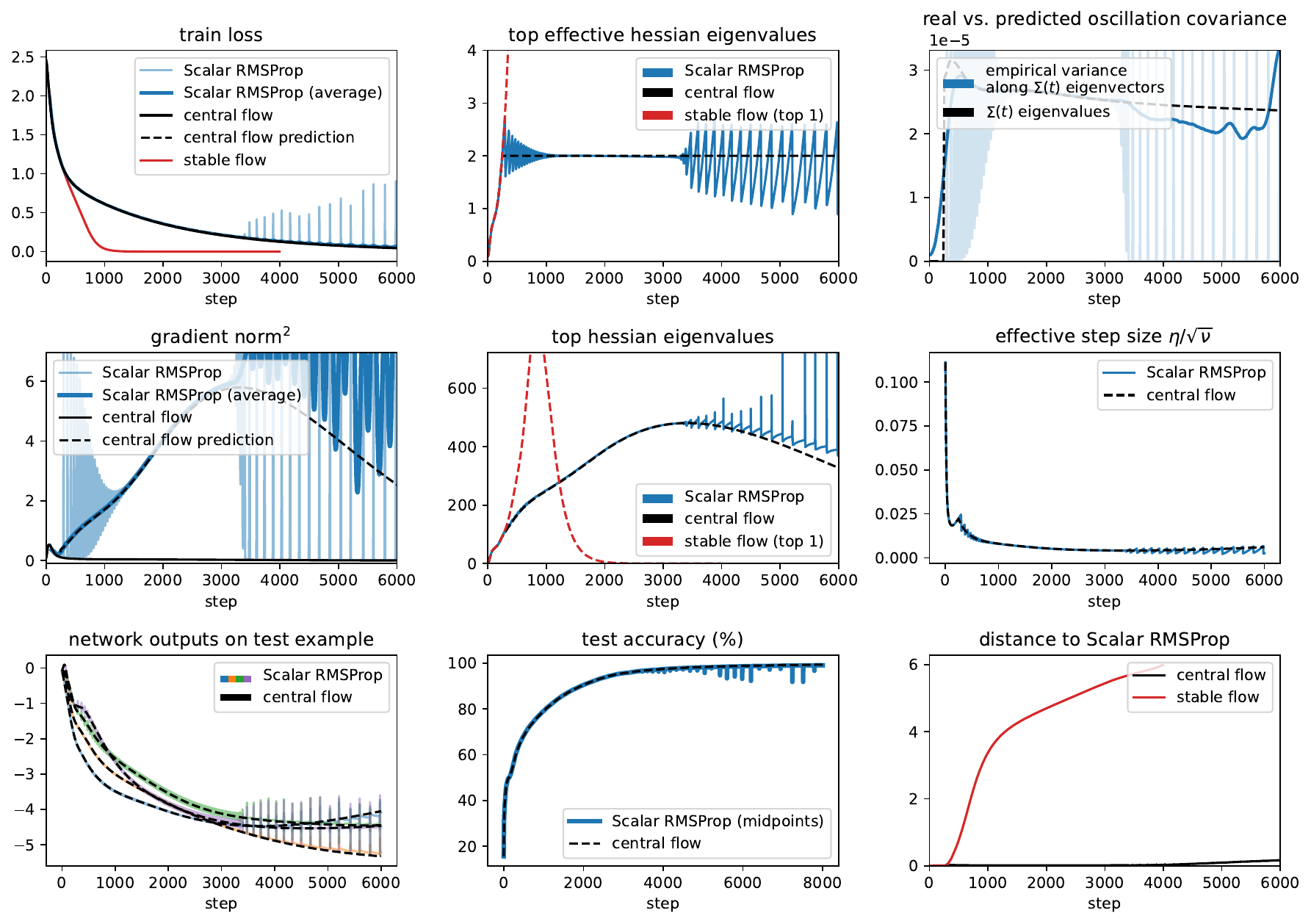}
        \caption{Scalar RMSProp central flow for a Transformer with CE loss, $\eta= $ 0.01, $\beta_2 = $ 0.99, and bias correction.}
        \label{fig:bulk-scalar-rmsprop:ce-transformer-0}
    \end{figure}
                
    \begin{figure}[H]
        \centering
        \includegraphics[width=0.8\linewidth]{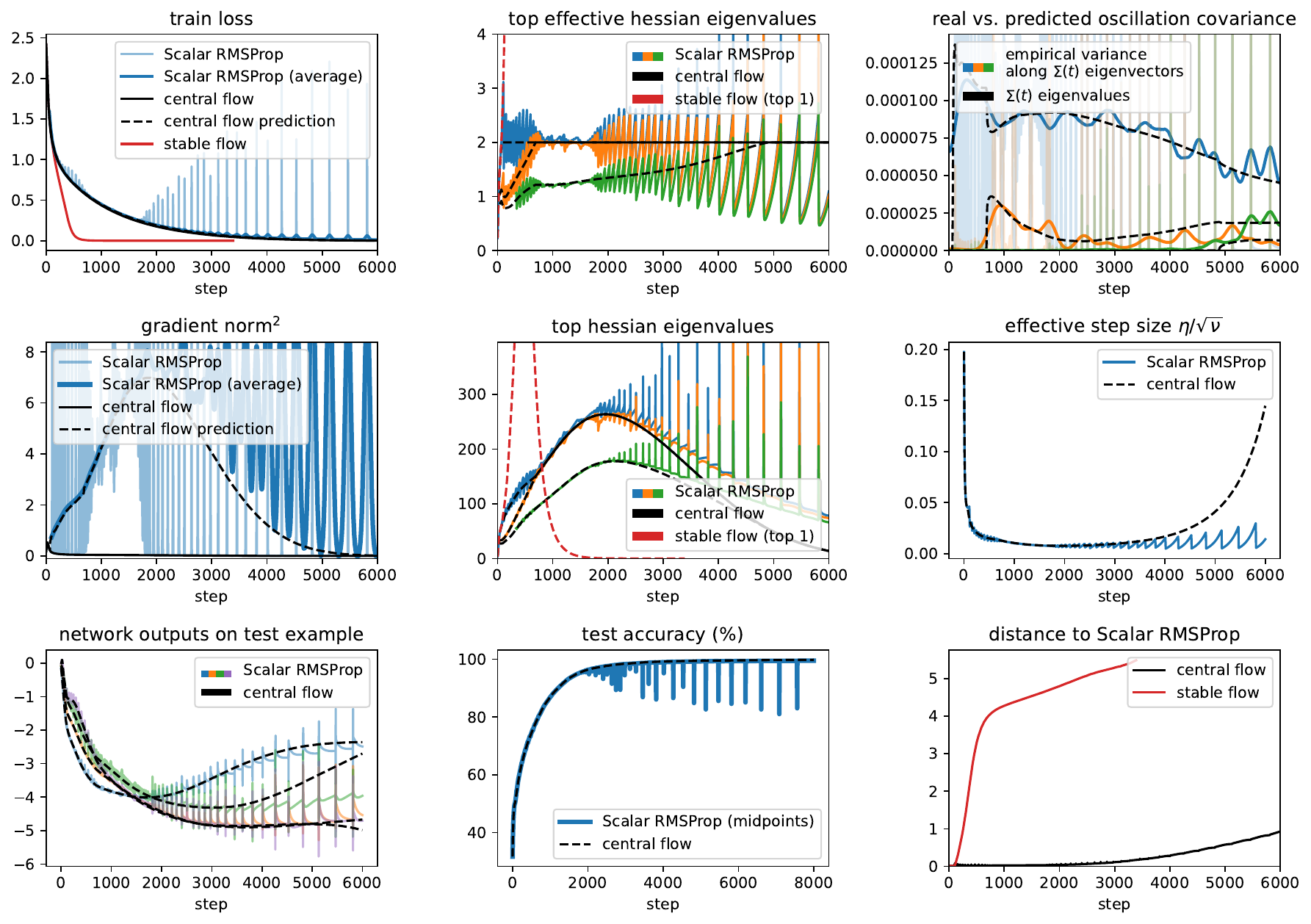}
        \caption{Scalar RMSProp central flow for a Transformer with CE loss, $\eta= $ 0.02, $\beta_2 = $ 0.99, and bias correction.}
        \label{fig:bulk-scalar-rmsprop:ce-transformer-1}
    \end{figure}
                
    \begin{figure}[H]
        \centering
        \includegraphics[width=0.8\linewidth]{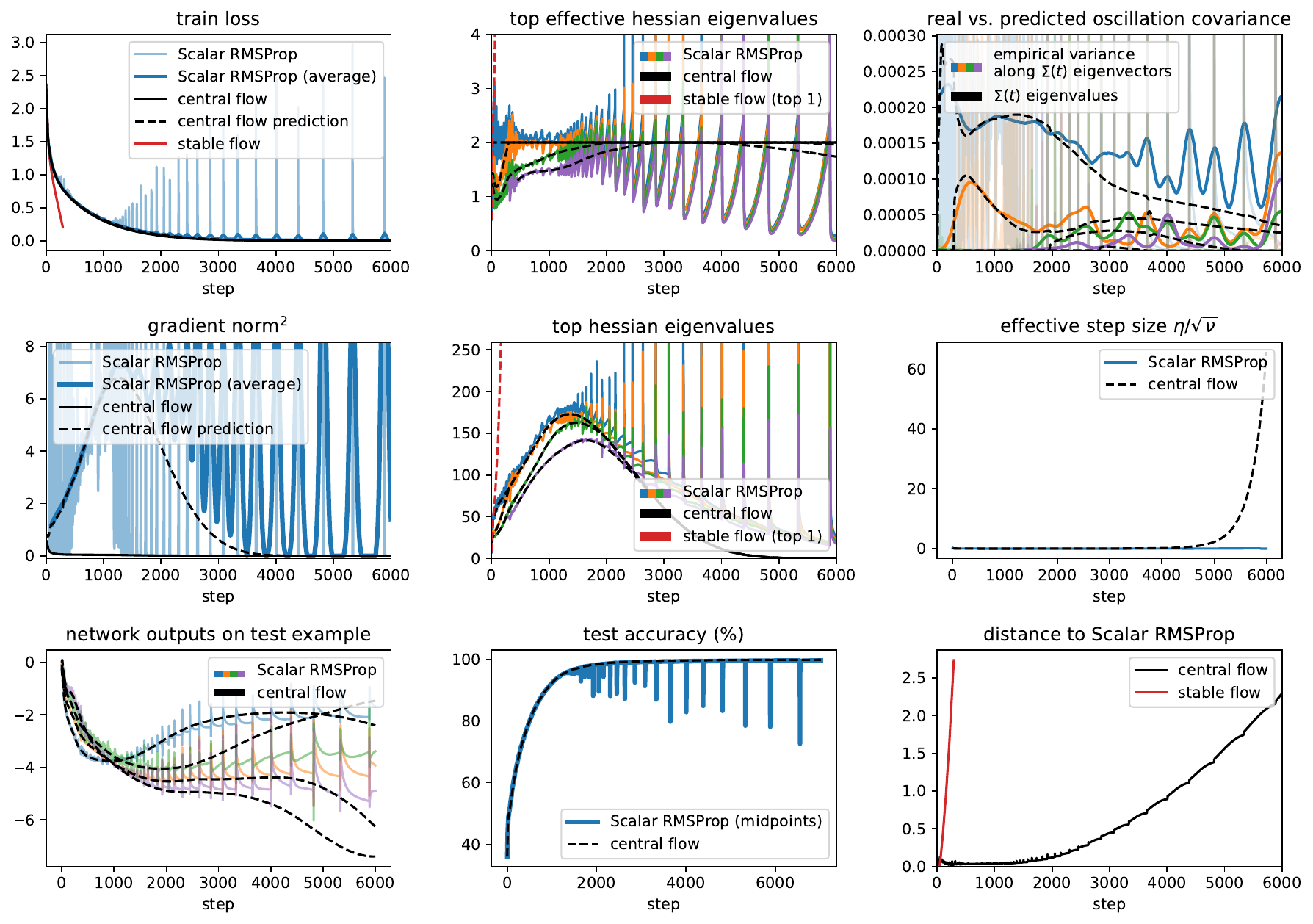}
        \caption{Scalar RMSProp central flow for a Transformer with CE loss, $\eta= $ 0.03, $\beta_2 = $ 0.99, and bias correction.}
        \label{fig:bulk-scalar-rmsprop:ce-transformer-2}
    \end{figure}
                
    \begin{figure}[H]
        \centering
        \includegraphics[width=0.8\linewidth]{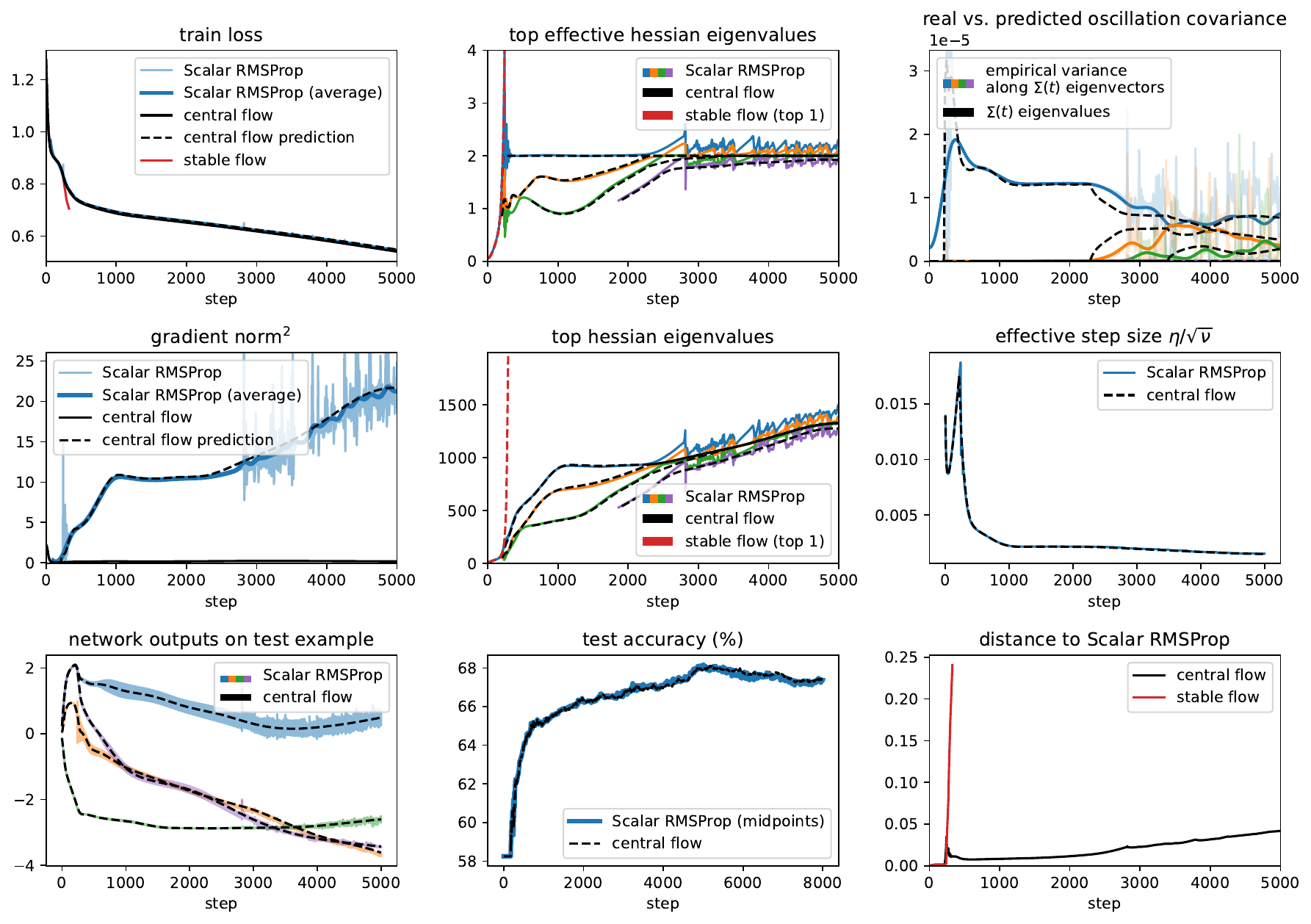}
        \caption{Scalar RMSProp central flow for a Mamba with CE loss, $\eta= $ 0.007, $\beta_2 = $ 0.99, and bias correction.}
        \label{fig:bulk-scalar-rmsprop:ce-mamba-0}
    \end{figure}
                
    \begin{figure}[H]
        \centering
        \includegraphics[width=0.8\linewidth]{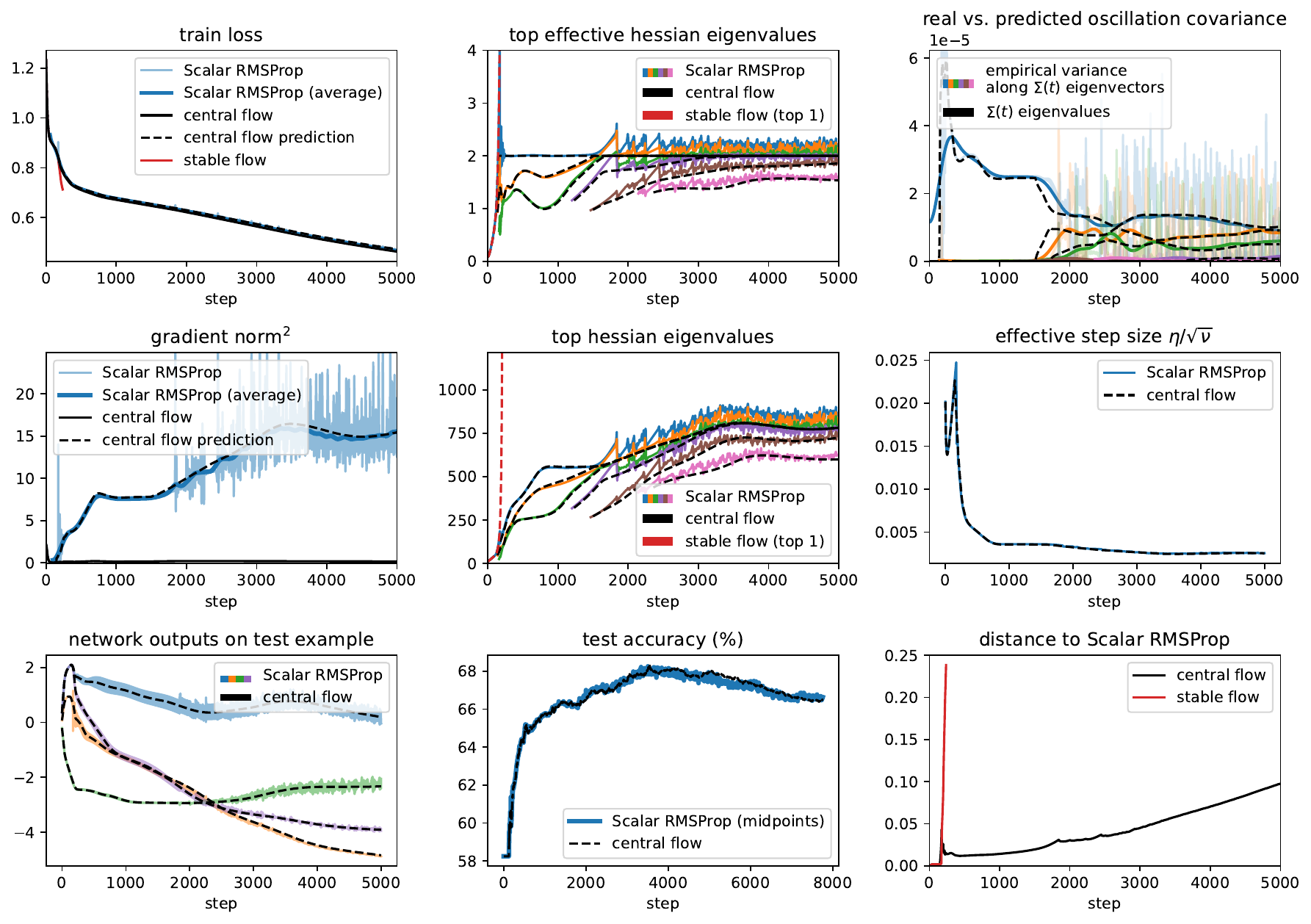}
        \caption{Scalar RMSProp central flow for a Mamba with CE loss, $\eta= $ 0.01, $\beta_2 = $ 0.99, and bias correction.}
        \label{fig:bulk-scalar-rmsprop:ce-mamba-1}
    \end{figure}
                
    \begin{figure}[H]
        \centering
        \includegraphics[width=0.8\linewidth]{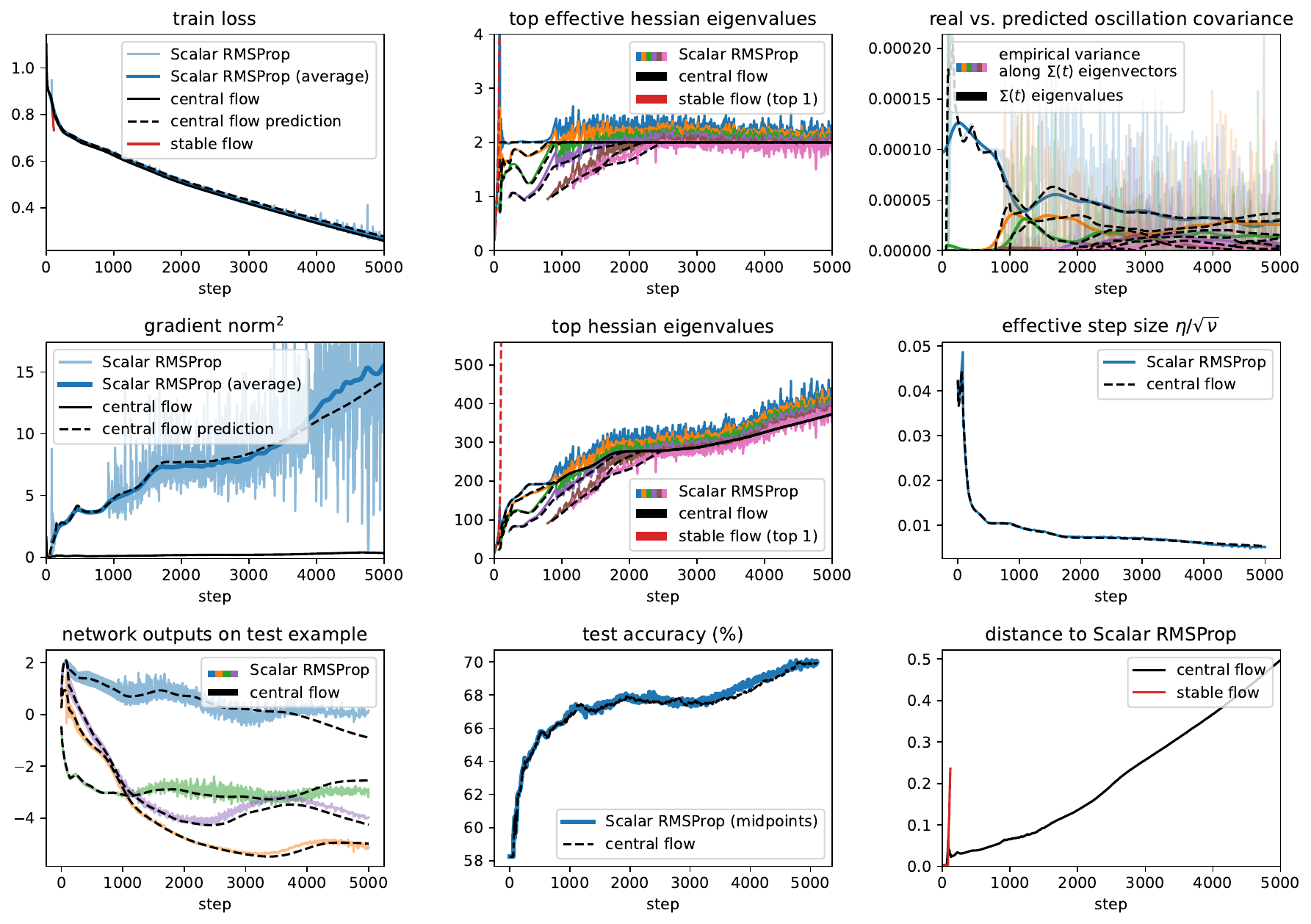}
        \caption{Scalar RMSProp central flow for a Mamba with CE loss, $\eta= $ 0.02, $\beta_2 = $ 0.99, and bias correction.}
        \label{fig:bulk-scalar-rmsprop:ce-mamba-2}
    \end{figure}
                \end{specialfigures}

%% file: images/bulk-rmsprop/figures-mse.tex
\begin{specialfigures}

    \begin{figure}[H]
        \centering
        \includegraphics[width=0.8\linewidth]{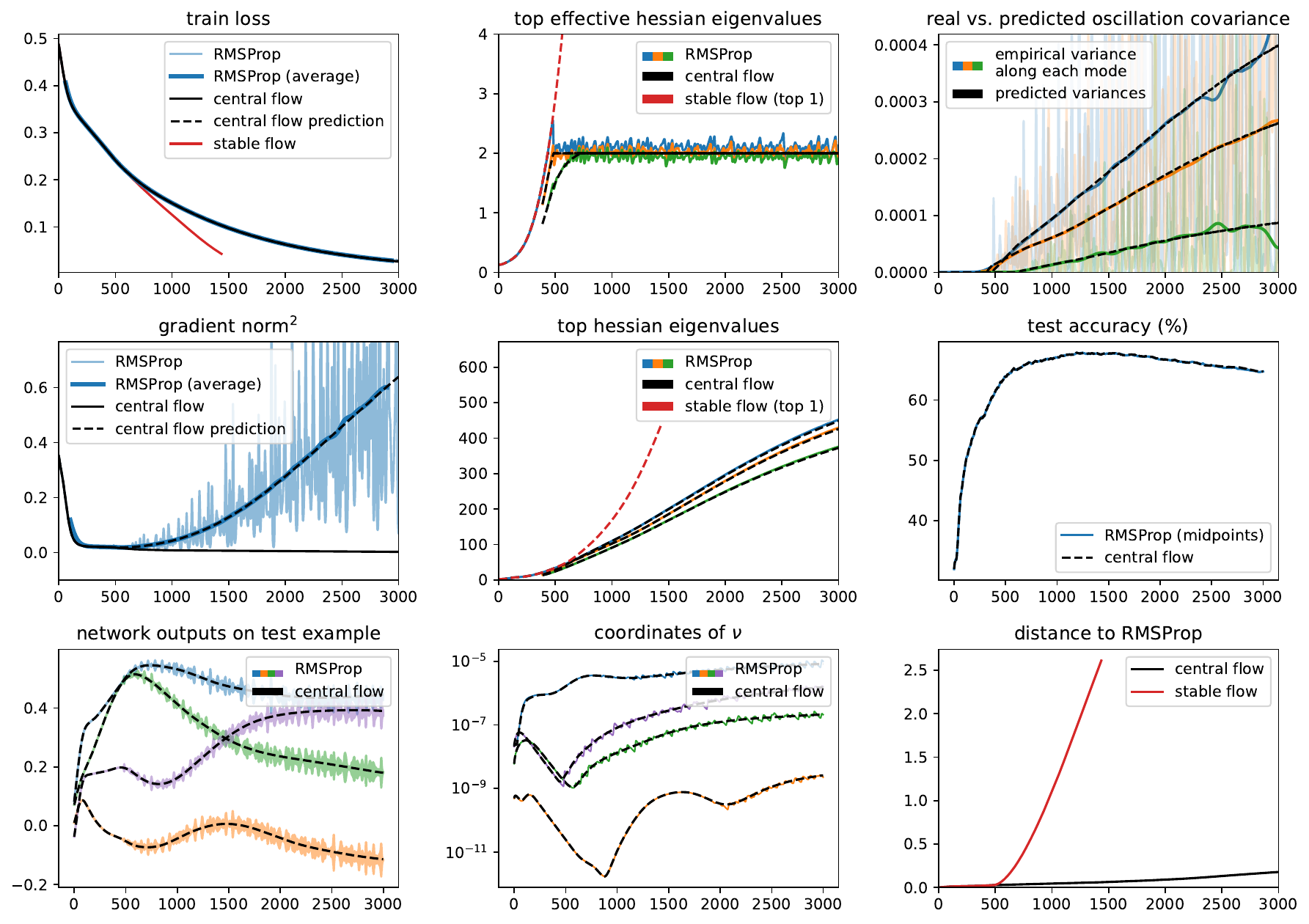}
        \caption{RMSProp central flow for a CNN with MSE loss, $\eta = $ 1e-05, $\beta_2 = $ 0.99, $\epsilon = $ 1e-08, and bias correction.}
        \label{fig:bulk-rmsprop:mse-cnn-0}
    \end{figure}
                
    \begin{figure}[H]
        \centering
        \includegraphics[width=0.8\linewidth]{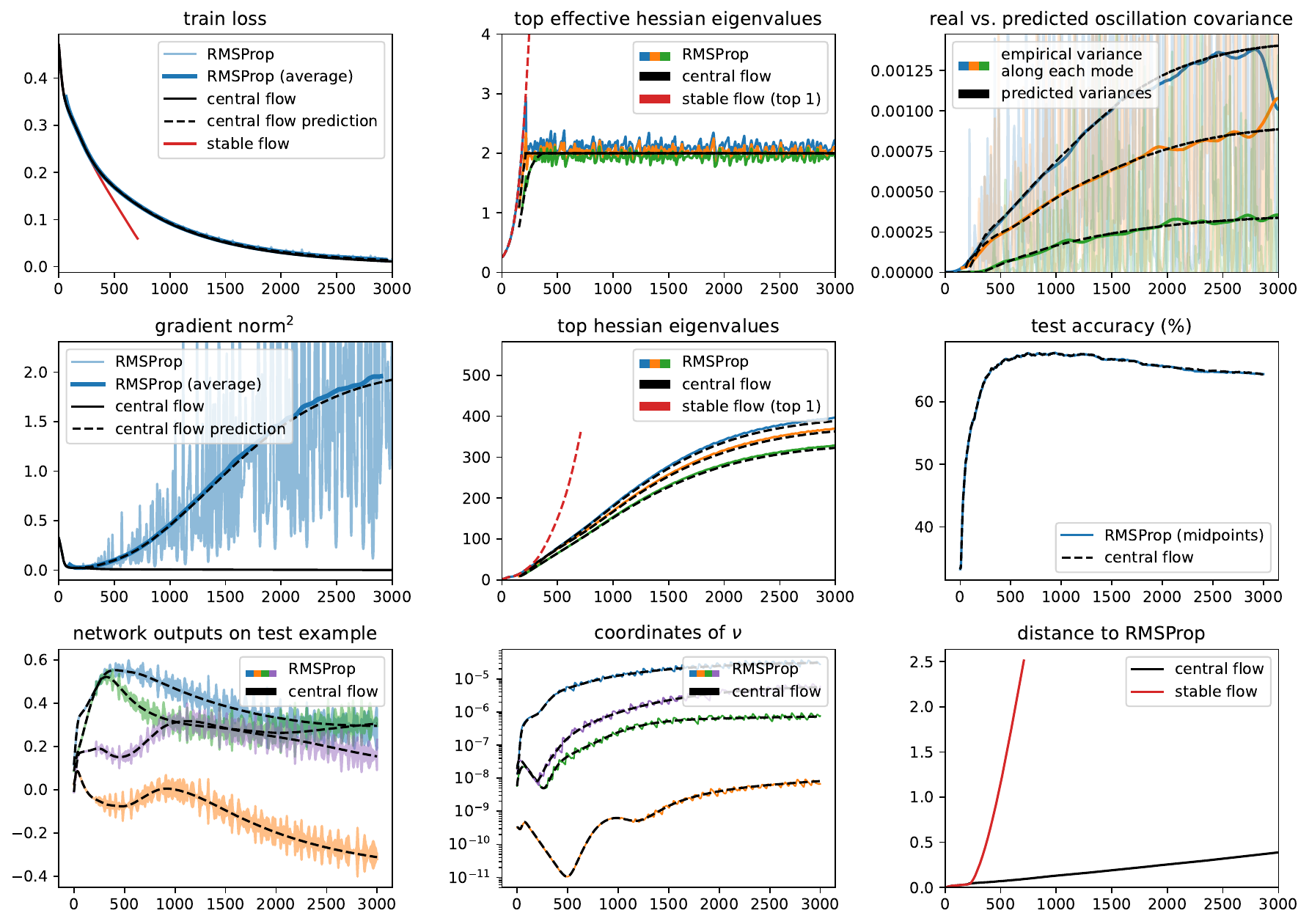}
        \caption{RMSProp central flow for a CNN with MSE loss, $\eta = $ 2e-05, $\beta_2 = $ 0.99, $\epsilon = $ 1e-08, and bias correction.}
        \label{fig:bulk-rmsprop:mse-cnn-1}
    \end{figure}
                
    \begin{figure}[H]
        \centering
        \includegraphics[width=0.8\linewidth]{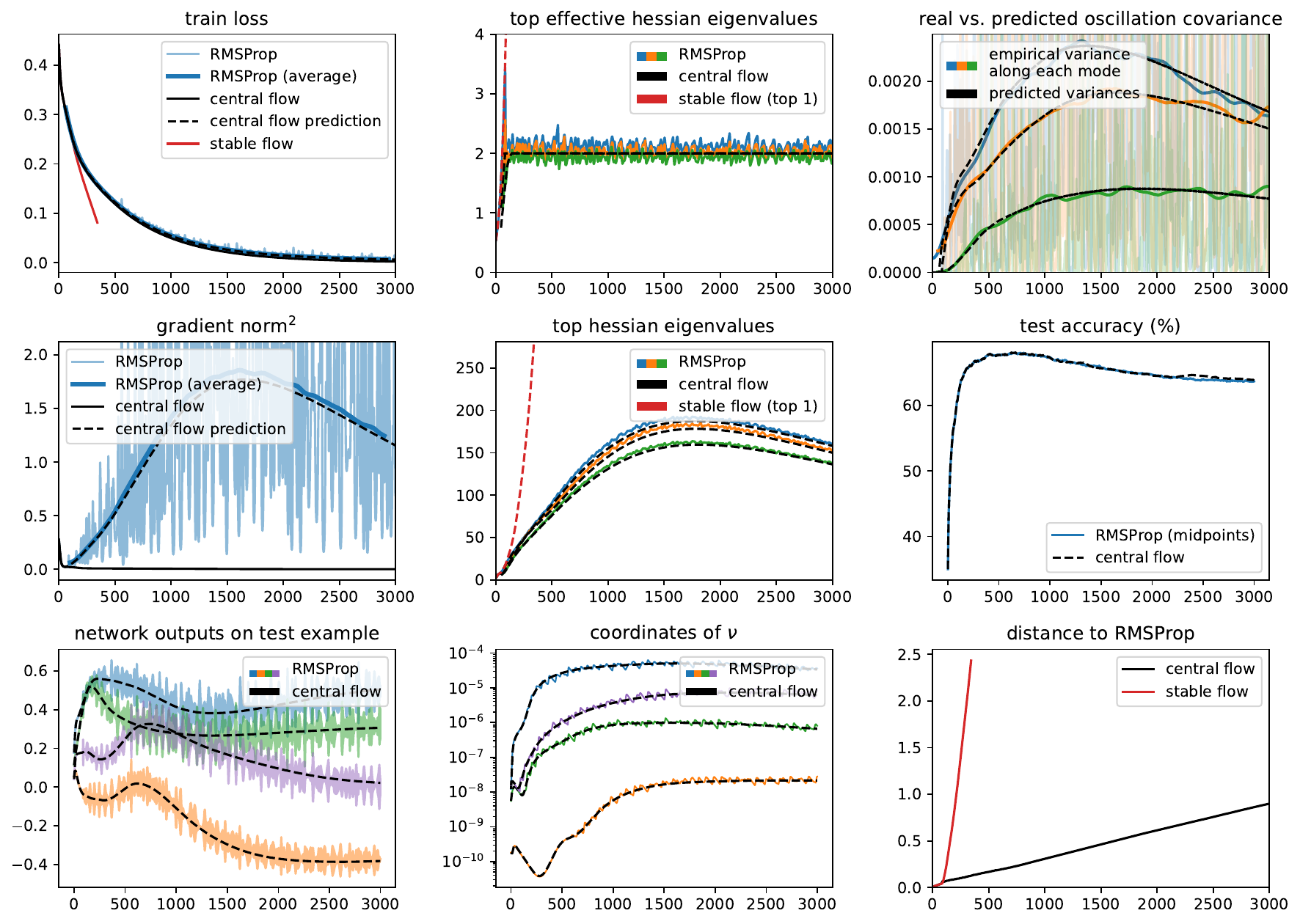}
        \caption{RMSProp central flow for a CNN with MSE loss, $\eta = $ 4e-05, $\beta_2 = $ 0.99, $\epsilon = $ 1e-08, and bias correction.}
        \label{fig:bulk-rmsprop:mse-cnn-2}
    \end{figure}
                
    \begin{figure}[H]
        \centering
        \includegraphics[width=0.8\linewidth]{images/bulk-rmsprop/mse-resnet-0.pdf}
        \caption{RMSProp central flow for a ResNet with MSE loss, $\eta = $ 2e-05, $\beta_2 = $ 0.99, $\epsilon = $ 1e-08, and bias correction.}
        \label{fig:bulk-rmsprop:mse-resnet-0}
    \end{figure}
                
    \begin{figure}[H]
        \centering
        \includegraphics[width=0.8\linewidth]{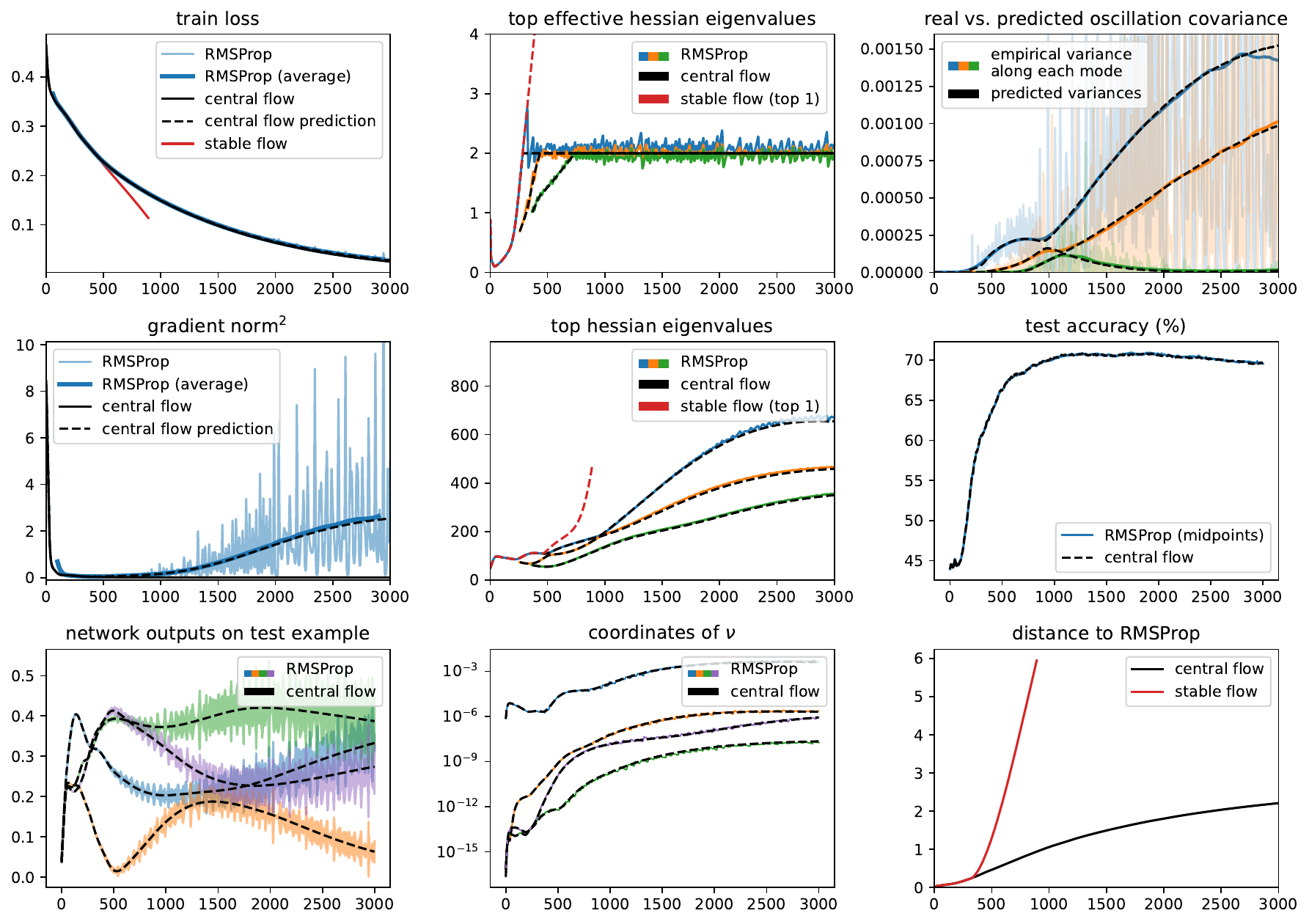}
        \caption{RMSProp central flow for a ResNet with MSE loss, $\eta = $ 4e-05, $\beta_2 = $ 0.99, $\epsilon = $ 1e-08, and bias correction.}
        \label{fig:bulk-rmsprop:mse-resnet-1}
    \end{figure}
                
    \begin{figure}[H]
        \centering
        \includegraphics[width=0.8\linewidth]{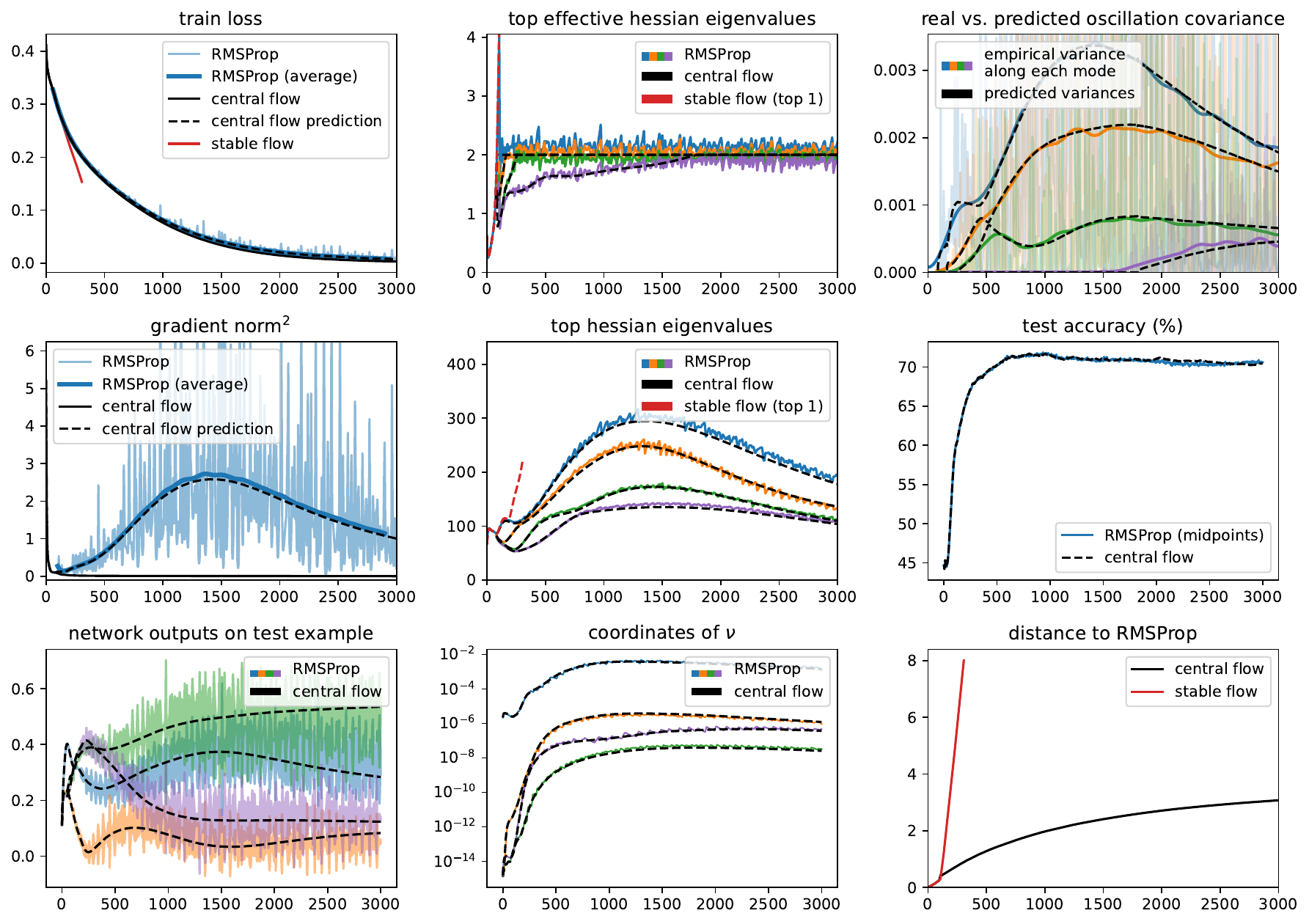}
        \caption{RMSProp central flow for a ResNet with MSE loss, $\eta = $ 0.0001, $\beta_2 = $ 0.99, $\epsilon = $ 1e-08, and bias correction.}
        \label{fig:bulk-rmsprop:mse-resnet-2}
    \end{figure}
                
    \begin{figure}[H]
        \centering
        \includegraphics[width=0.8\linewidth]{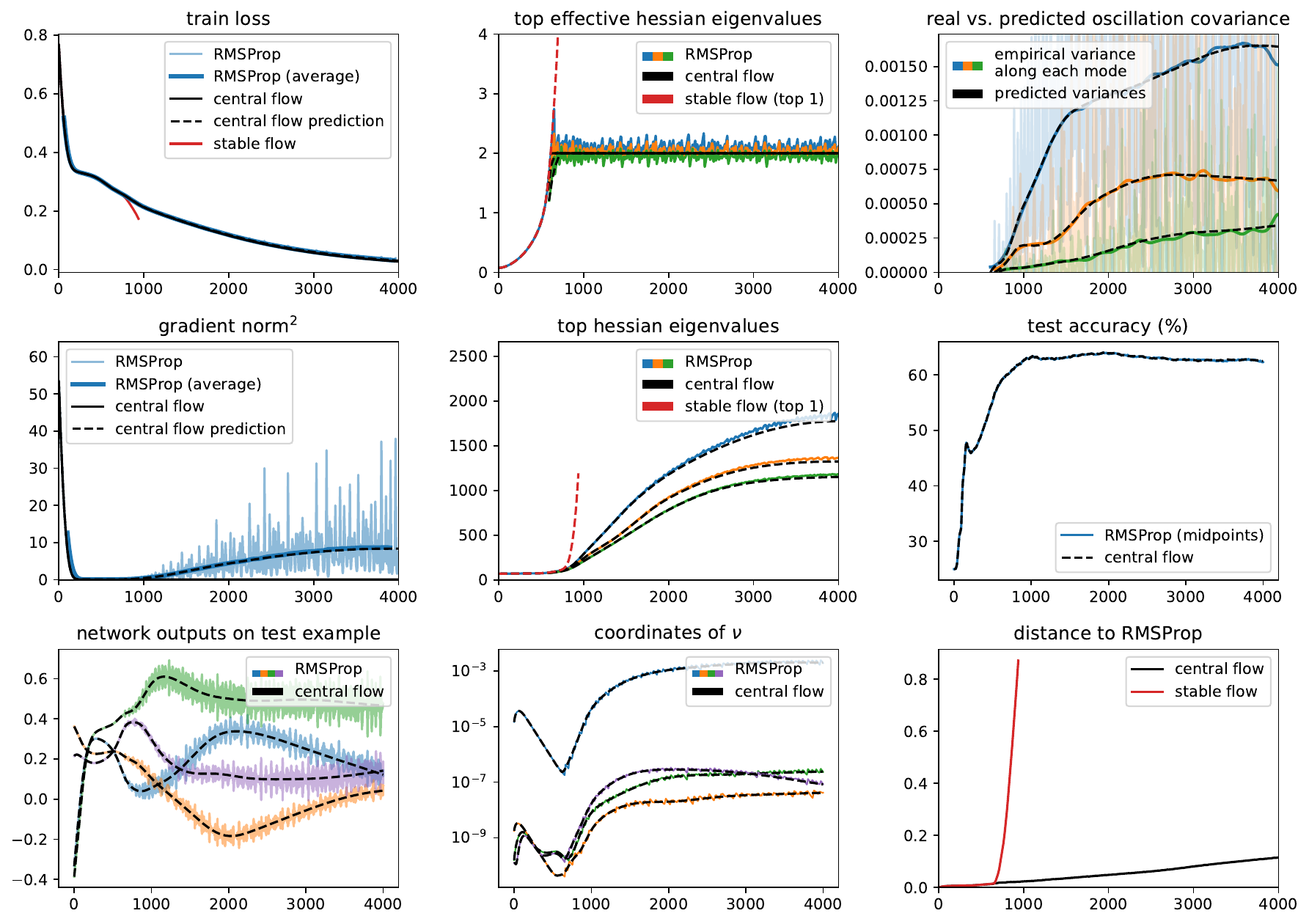}
        \caption{RMSProp central flow for a ViT with MSE loss, $\eta = $ 7e-06, $\beta_2 = $ 0.95, $\epsilon = $ 1e-08, and bias correction.}
        \label{fig:bulk-rmsprop:mse-vit-0}
    \end{figure}
                
    \begin{figure}[H]
        \centering
        \includegraphics[width=0.8\linewidth]{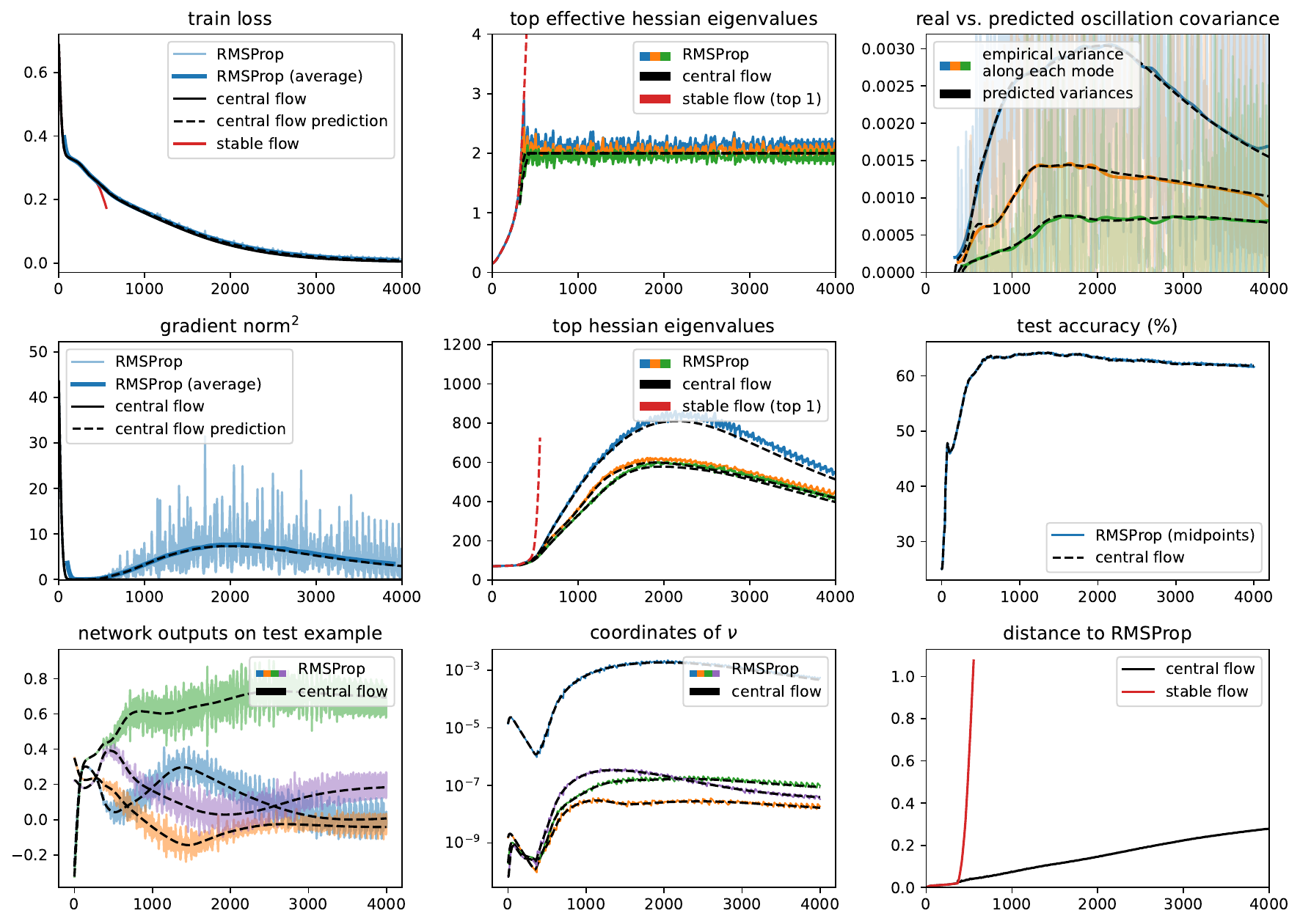}
        \caption{RMSProp central flow for a ViT with MSE loss, $\eta = $ 1e-05, $\beta_2 = $ 0.95, $\epsilon = $ 1e-08, and bias correction.}
        \label{fig:bulk-rmsprop:mse-vit-1}
    \end{figure}
                
    \begin{figure}[H]
        \centering
        \includegraphics[width=0.8\linewidth]{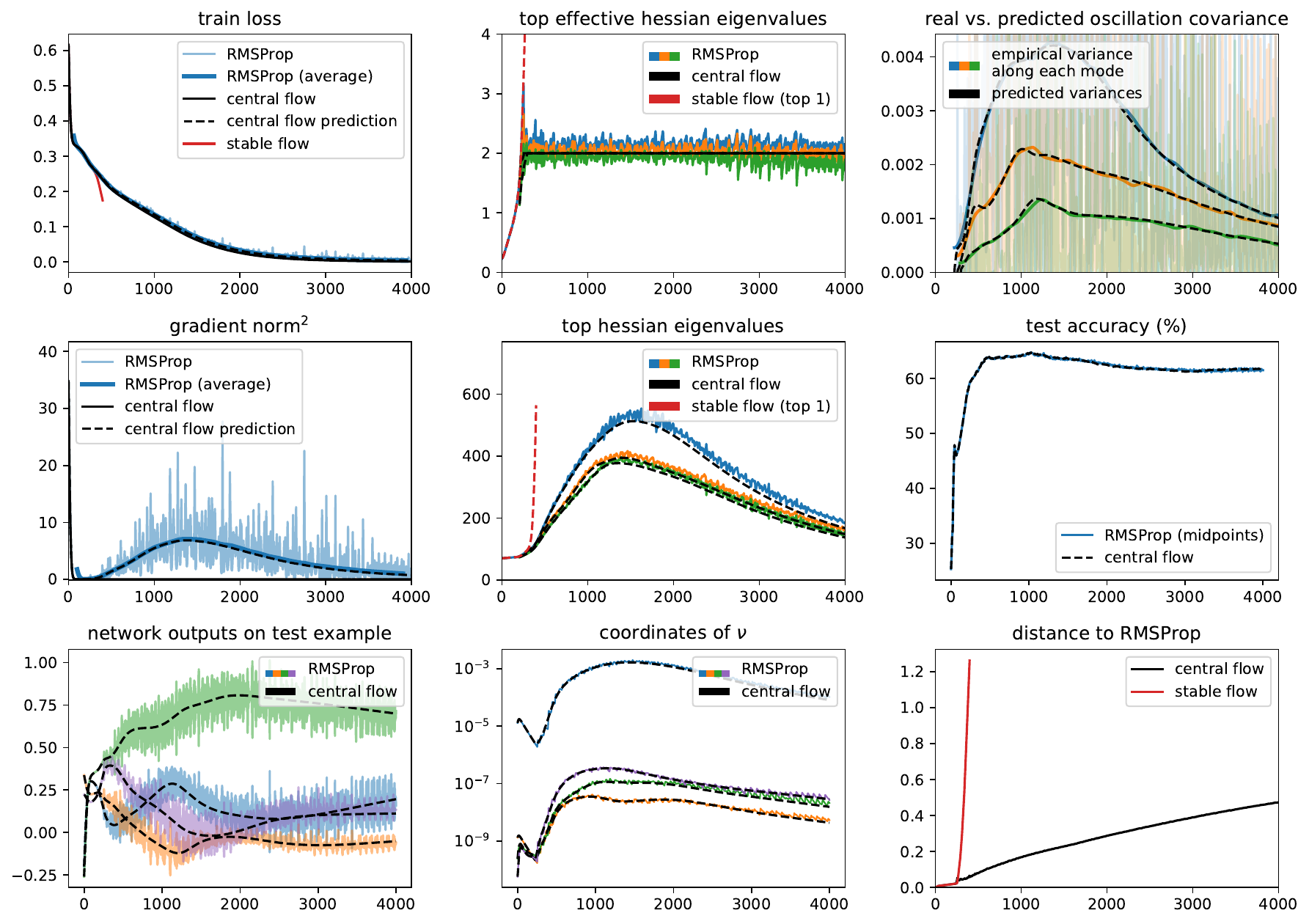}
        \caption{RMSProp central flow for a ViT with MSE loss, $\eta = $ 2e-05, $\beta_2 = $ 0.95, $\epsilon = $ 1e-08, and bias correction.}
        \label{fig:bulk-rmsprop:mse-vit-2}
    \end{figure}
                
    \begin{figure}[H]
        \centering
        \includegraphics[width=0.8\linewidth]{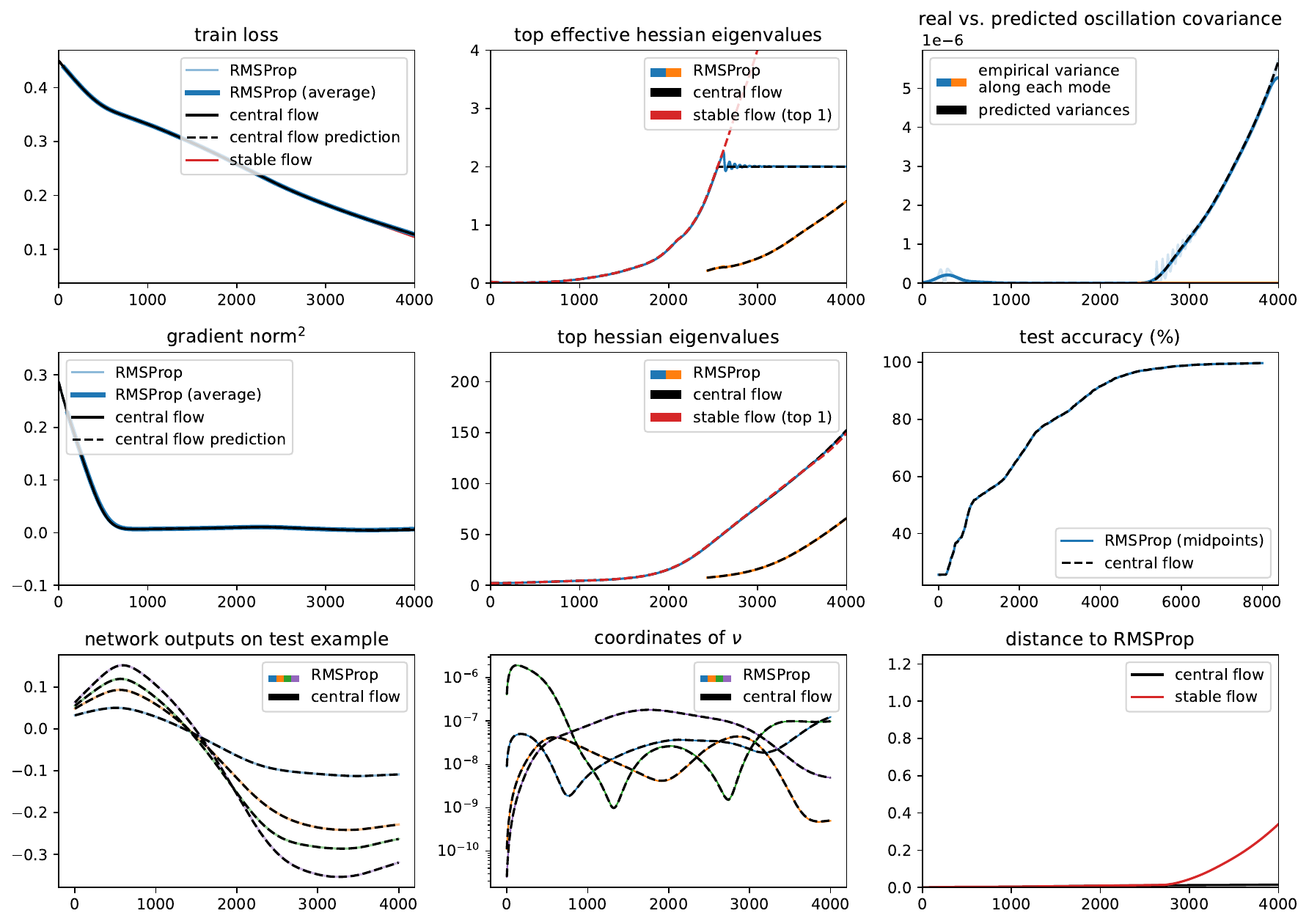}
        \caption{RMSProp central flow for a LSTM with MSE loss, $\eta = $ 1e-05, $\beta_2 = $ 0.99, $\epsilon = $ 1e-08, and bias correction.}
        \label{fig:bulk-rmsprop:mse-lstm-0}
    \end{figure}
                
    \begin{figure}[H]
        \centering
        \includegraphics[width=0.8\linewidth]{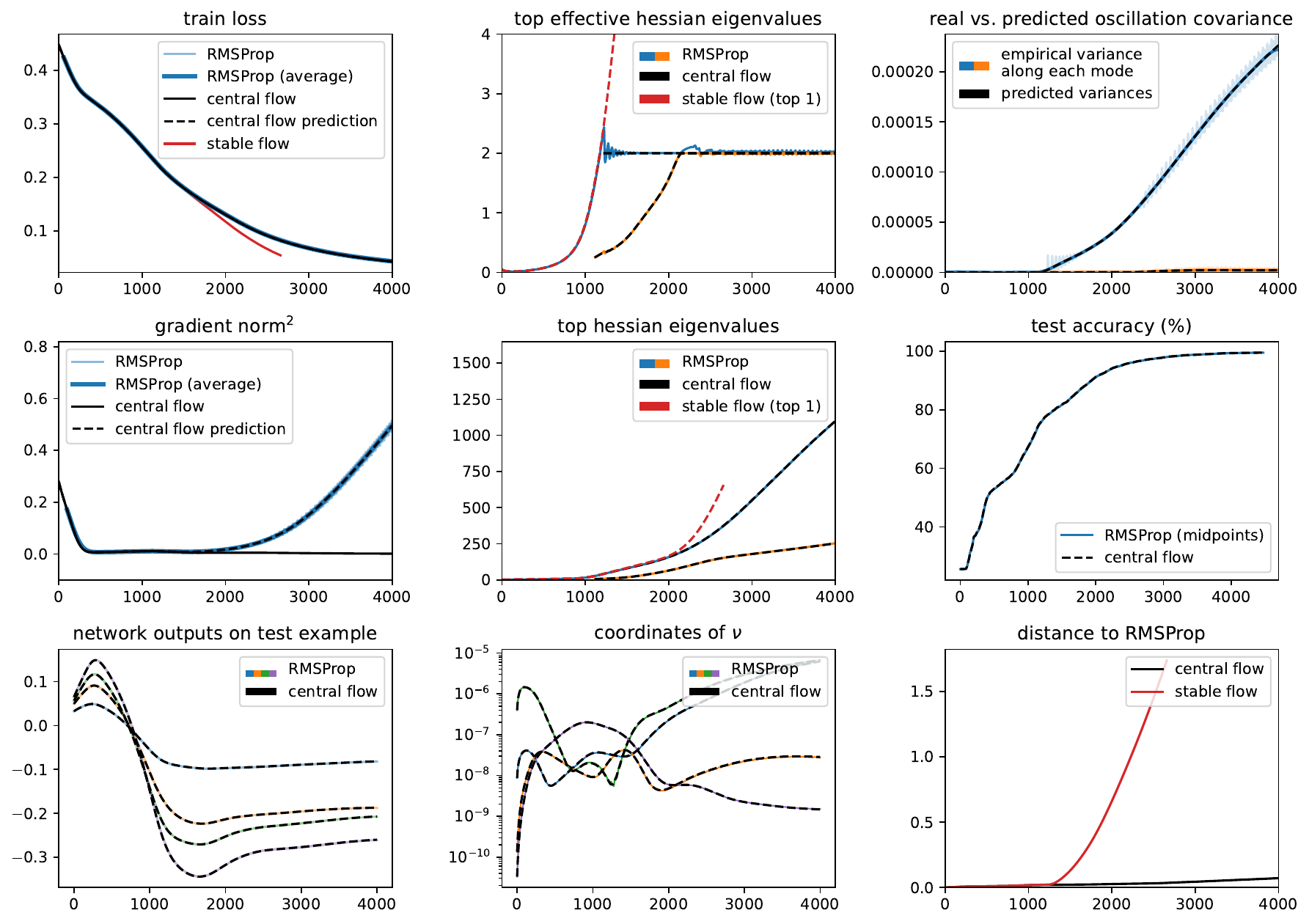}
        \caption{RMSProp central flow for a LSTM with MSE loss, $\eta = $ 2e-05, $\beta_2 = $ 0.99, $\epsilon = $ 1e-08, and bias correction.}
        \label{fig:bulk-rmsprop:mse-lstm-1}
    \end{figure}
                
    \begin{figure}[H]
        \centering
        \includegraphics[width=0.8\linewidth]{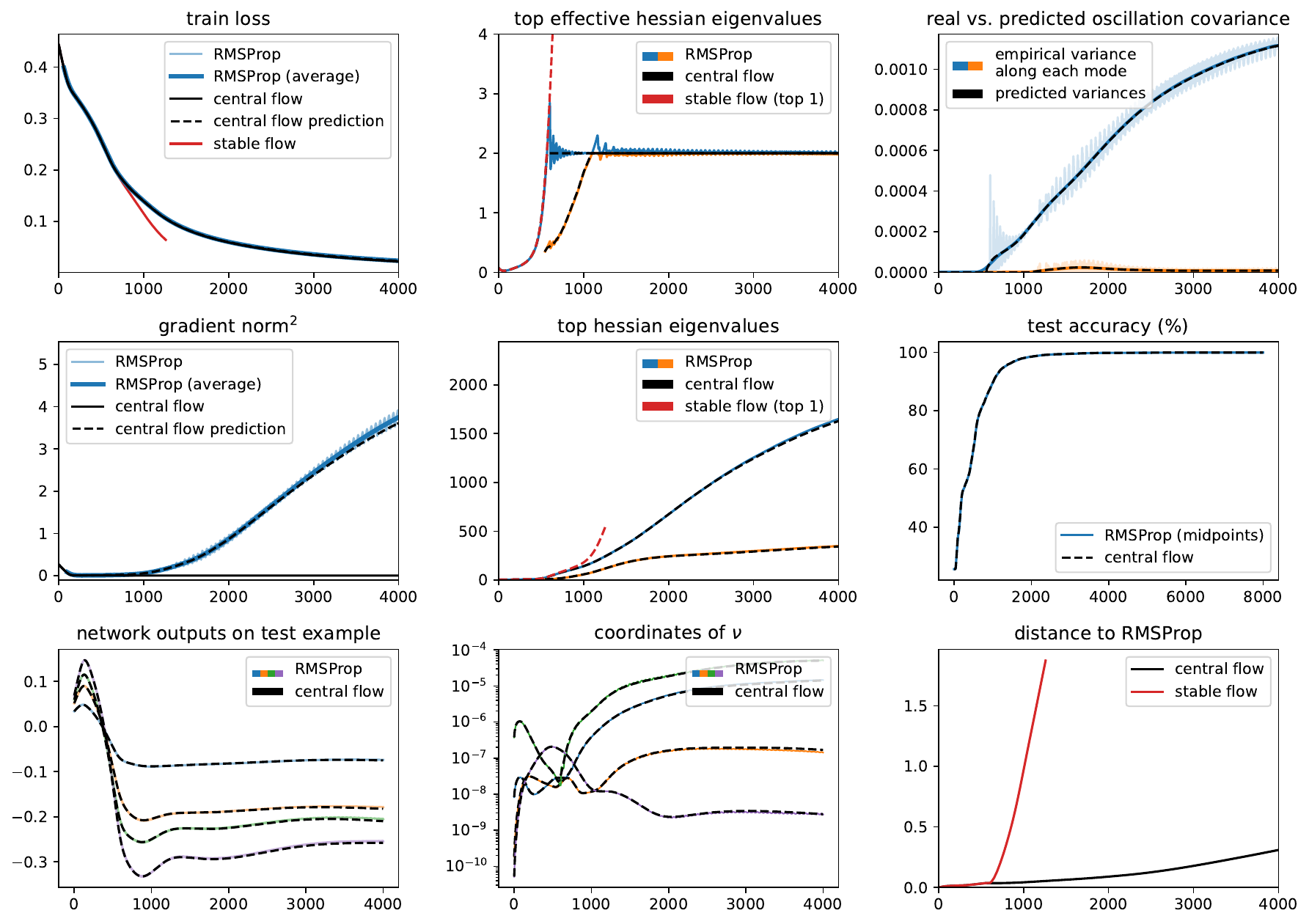}
        \caption{RMSProp central flow for a LSTM with MSE loss, $\eta = $ 4e-05, $\beta_2 = $ 0.99, $\epsilon = $ 1e-08, and bias correction.}
        \label{fig:bulk-rmsprop:mse-lstm-2}
    \end{figure}
                
    \begin{figure}[H]
        \centering
        \includegraphics[width=0.8\linewidth]{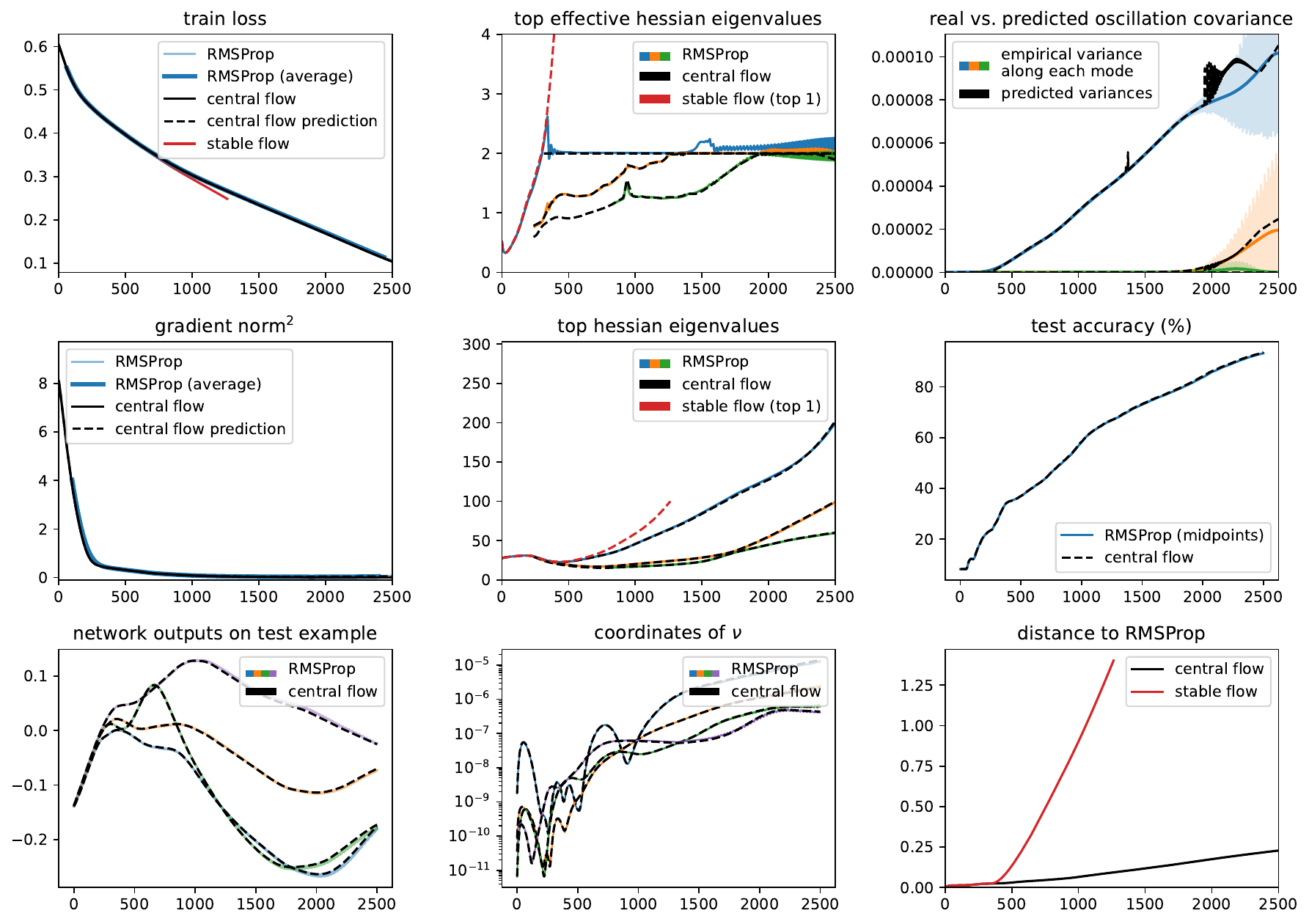}
        \caption{RMSProp central flow for a Transformer with MSE loss, $\eta = $ 2e-05, $\beta_2 = $ 0.95, $\epsilon = $ 1e-08, and bias correction.}
        \label{fig:bulk-rmsprop:mse-transformer-0}
    \end{figure}
                
    \begin{figure}[H]
        \centering
        \includegraphics[width=0.8\linewidth]{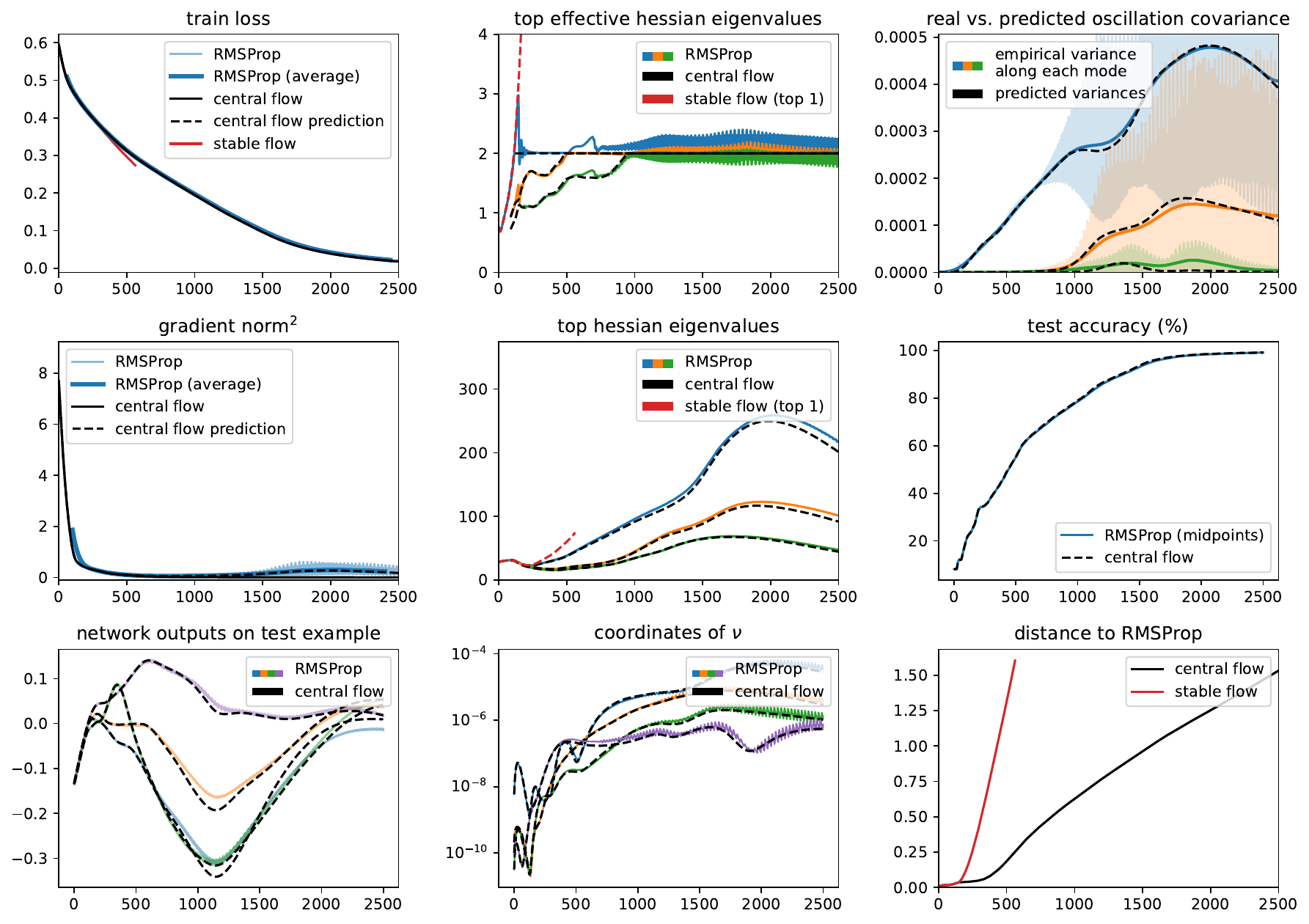}
        \caption{RMSProp central flow for a Transformer with MSE loss, $\eta = $ 4e-05, $\beta_2 = $ 0.95, $\epsilon = $ 1e-08, and bias correction.}
        \label{fig:bulk-rmsprop:mse-transformer-1}
    \end{figure}
                
    \begin{figure}[H]
        \centering
        \includegraphics[width=0.8\linewidth]{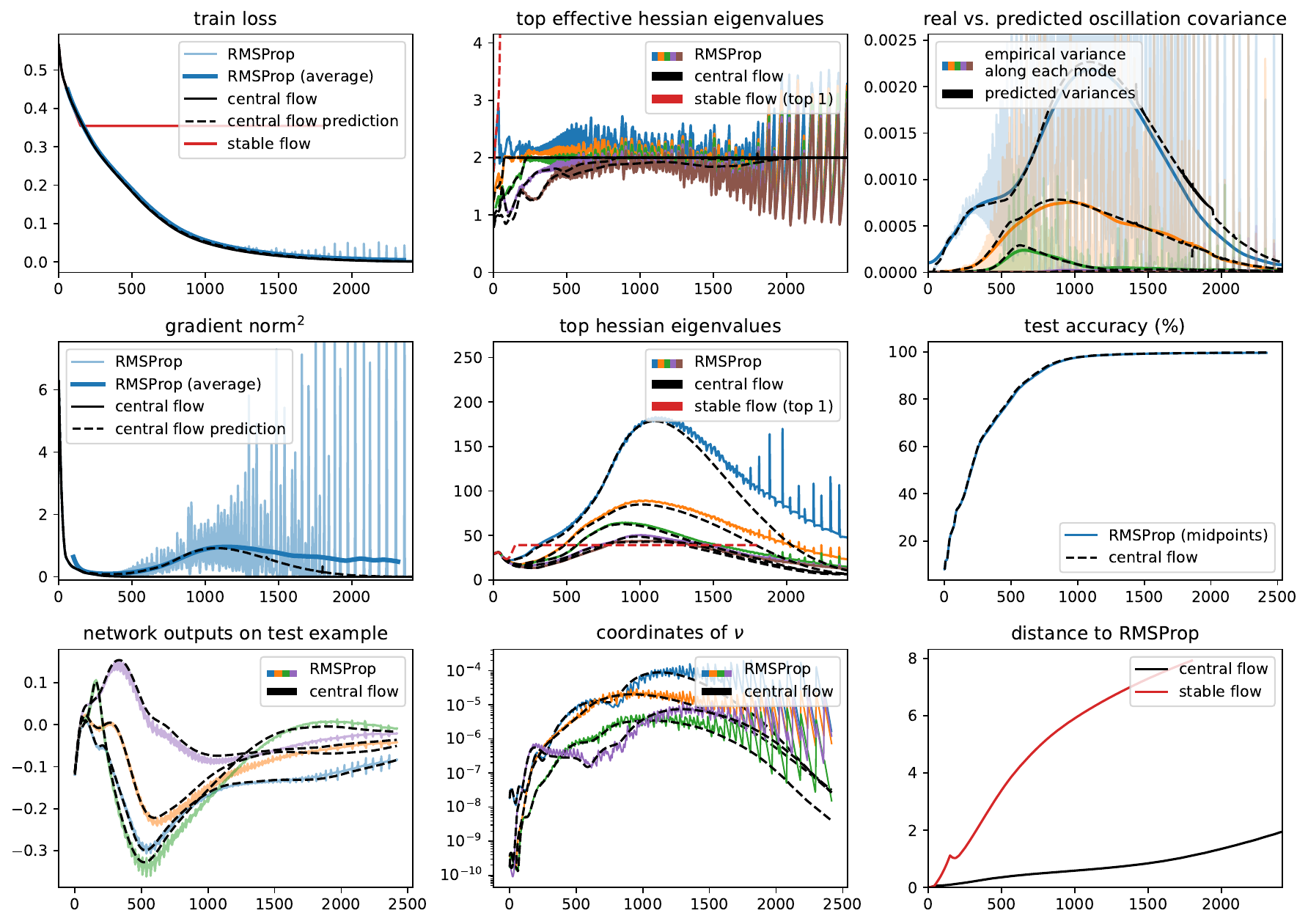}
        \caption{RMSProp central flow for a Transformer with MSE loss, $\eta = $ 0.0001, $\beta_2 = $ 0.95, $\epsilon = $ 1e-08, and bias correction.}
        \label{fig:bulk-rmsprop:mse-transformer-2}
    \end{figure}
                
    \begin{figure}[H]
        \centering
        \includegraphics[width=0.8\linewidth]{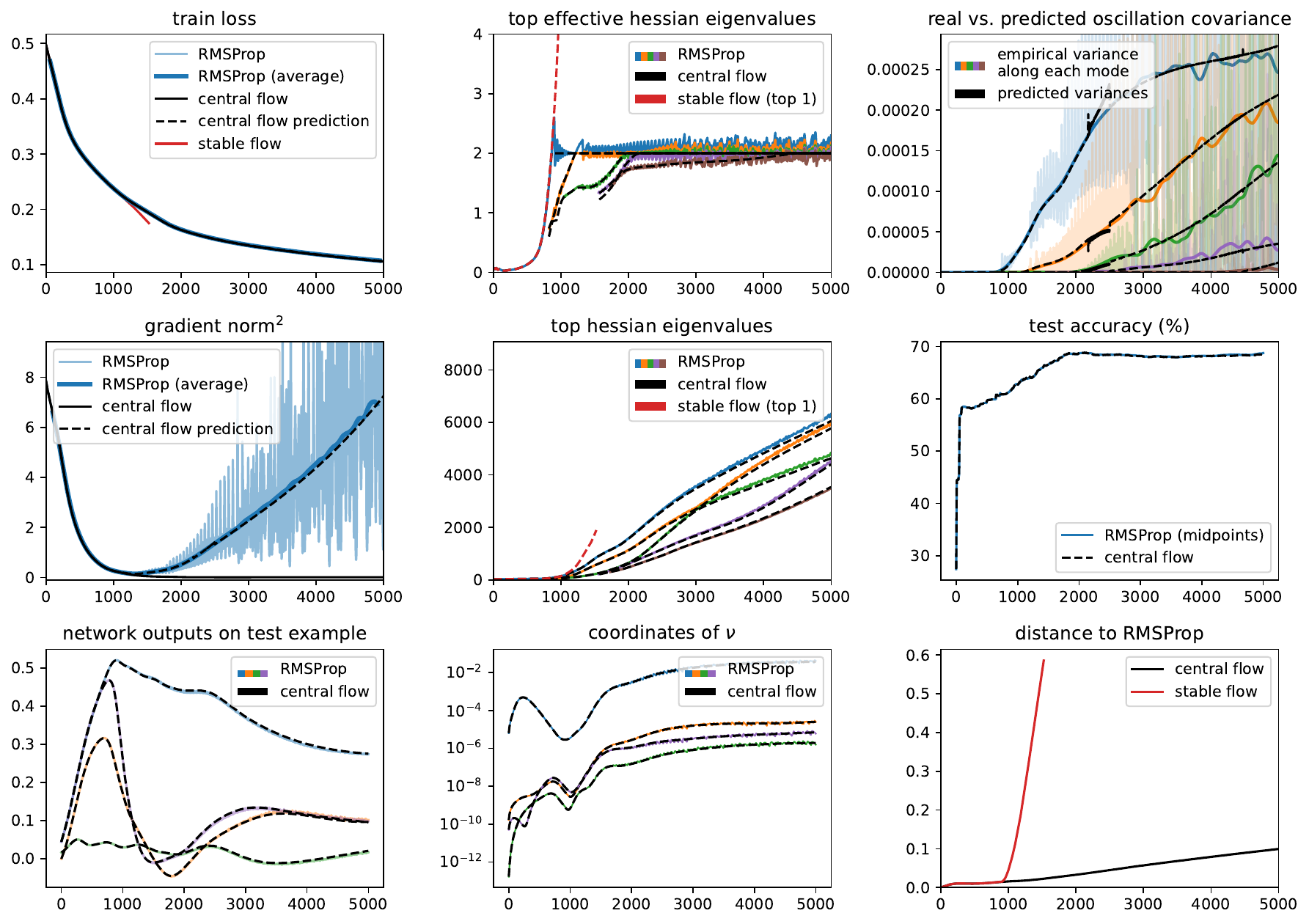}
        \caption{RMSProp central flow for a Mamba with MSE loss, $\eta = $ 1e-05, $\beta_2 = $ 0.99, $\epsilon = $ 1e-08, and bias correction.}
        \label{fig:bulk-rmsprop:mse-mamba-0}
    \end{figure}
                
    \begin{figure}[H]
        \centering
        \includegraphics[width=0.8\linewidth]{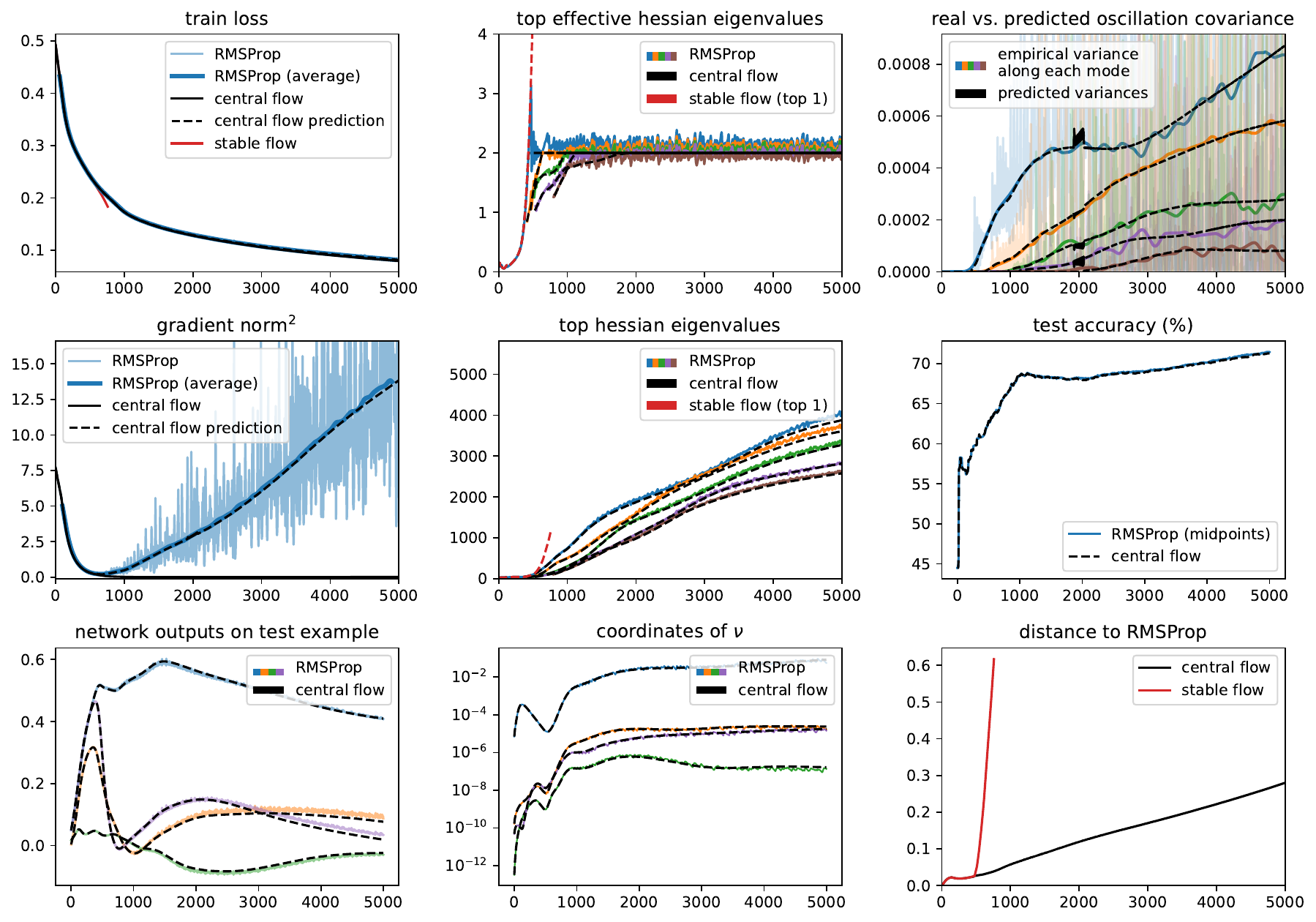}
        \caption{RMSProp central flow for a Mamba with MSE loss, $\eta = $ 2e-05, $\beta_2 = $ 0.99, $\epsilon = $ 1e-08, and bias correction.}
        \label{fig:bulk-rmsprop:mse-mamba-1}
    \end{figure}
                
    \begin{figure}[H]
        \centering
        \includegraphics[width=0.8\linewidth]{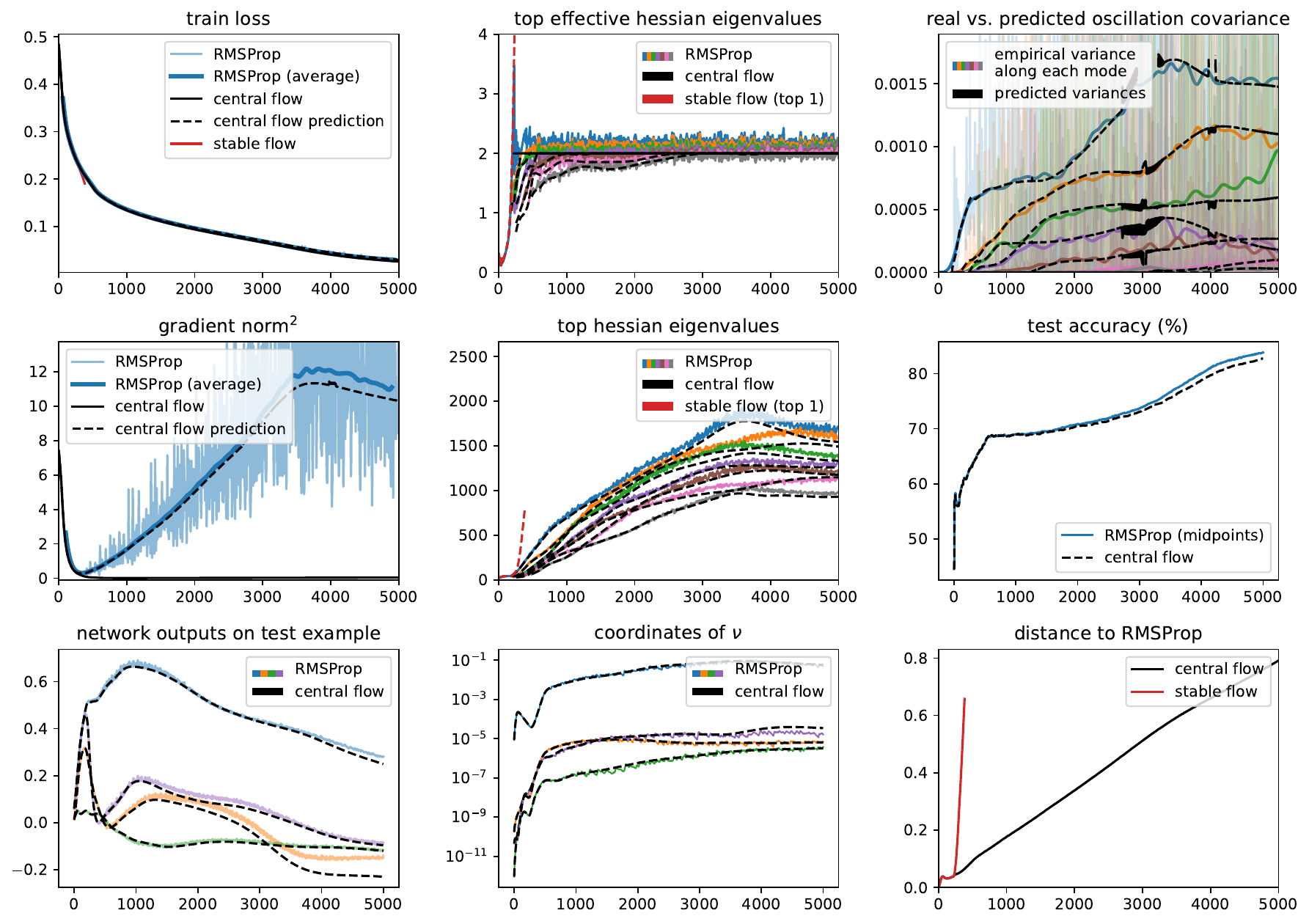}
        \caption{RMSProp central flow for a Mamba with MSE loss, $\eta = $ 4e-05, $\beta_2 = $ 0.99, $\epsilon = $ 1e-08, and bias correction.}
        \label{fig:bulk-rmsprop:mse-mamba-2}
    \end{figure}
                \end{specialfigures}

%% file: images/bulk-rmsprop/figures-ce.tex
\begin{specialfigures}

    \begin{figure}[H]
        \centering
        \includegraphics[width=0.8\linewidth]{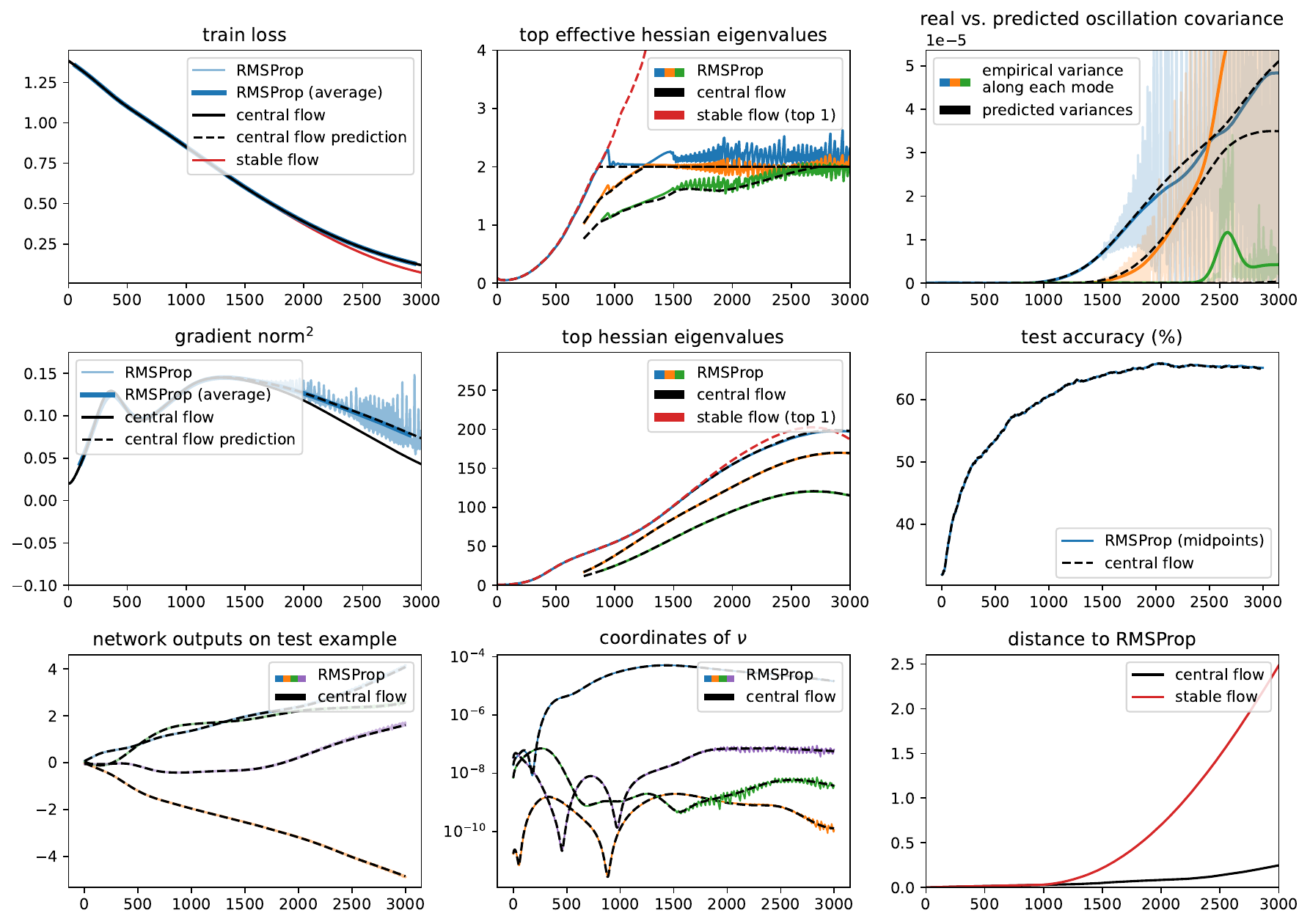}
        \caption{RMSProp central flow for a CNN with CE loss, $\eta = $ 7e-06, $\beta_2 = $ 0.95, $\epsilon = $ 1e-08, and bias correction.}
        \label{fig:bulk-rmsprop:ce-cnn-0}
    \end{figure}
                
    \begin{figure}[H]
        \centering
        \includegraphics[width=0.8\linewidth]{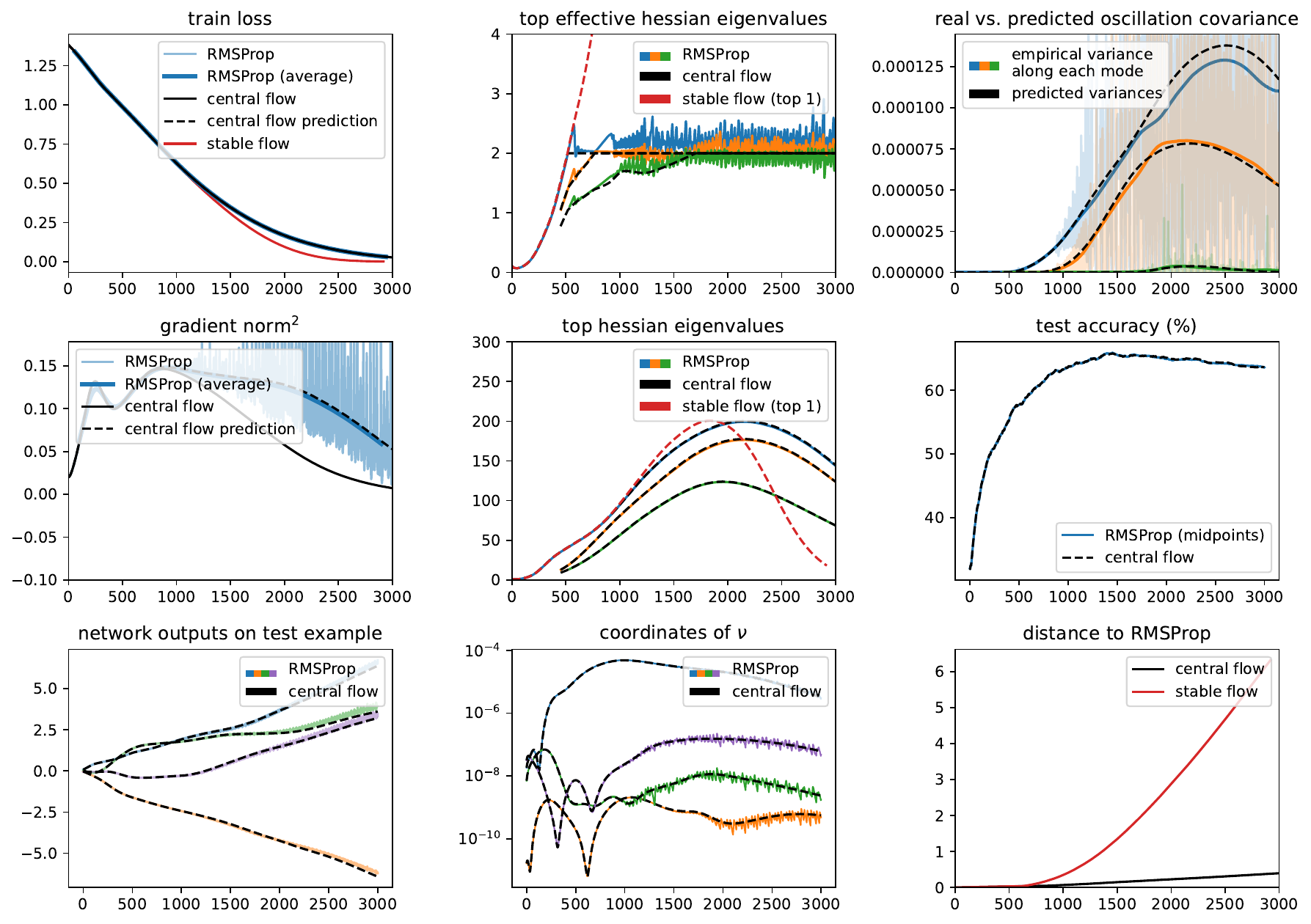}
        \caption{RMSProp central flow for a CNN with CE loss, $\eta = $ 1e-05, $\beta_2 = $ 0.95, $\epsilon = $ 1e-08, and bias correction.}
        \label{fig:bulk-rmsprop:ce-cnn-1}
    \end{figure}
                
    \begin{figure}[H]
        \centering
        \includegraphics[width=0.8\linewidth]{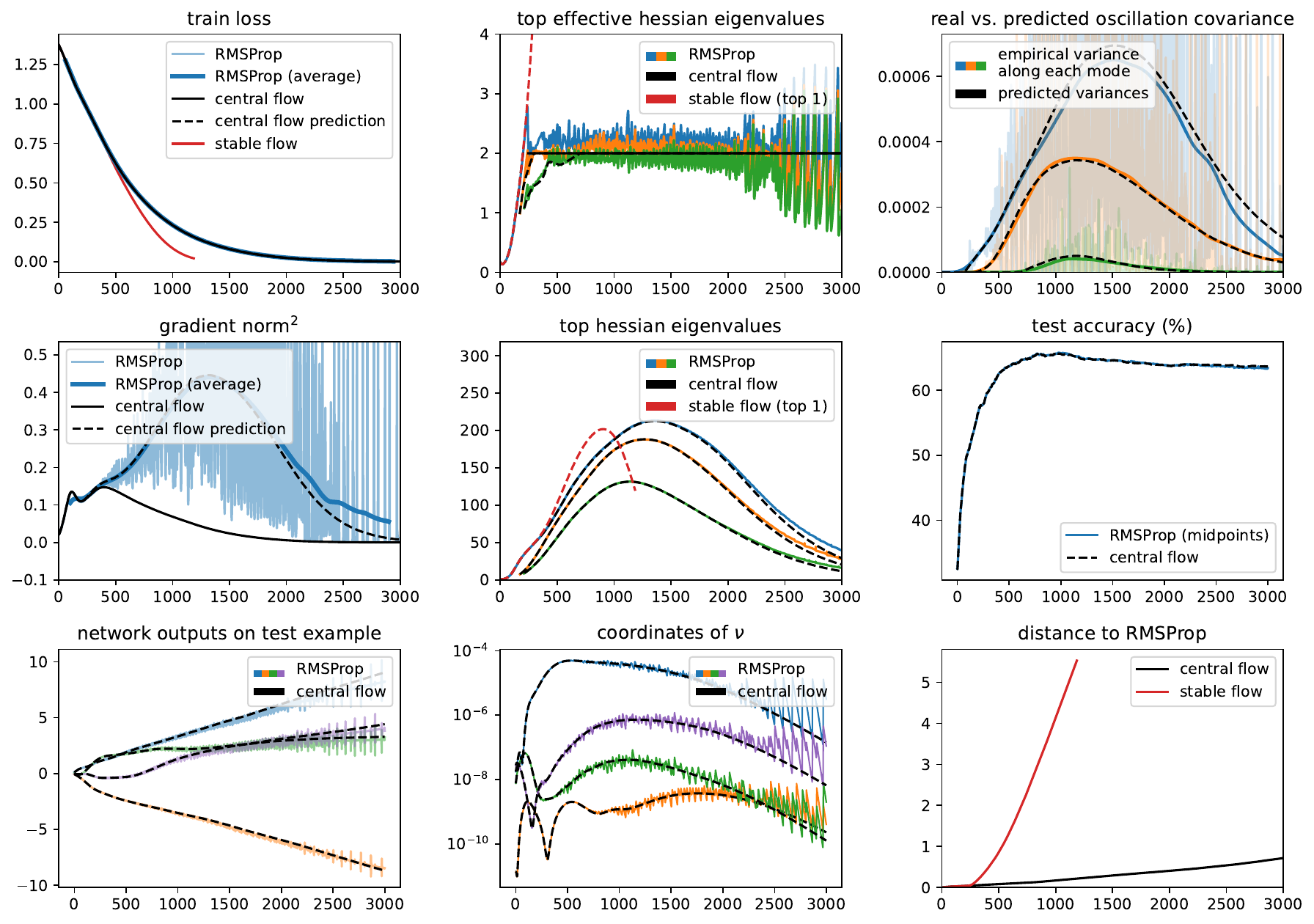}
        \caption{RMSProp central flow for a CNN with CE loss, $\eta = $ 2e-05, $\beta_2 = $ 0.95, $\epsilon = $ 1e-08, and bias correction.}
        \label{fig:bulk-rmsprop:ce-cnn-2}
    \end{figure}
                
    \begin{figure}[H]
        \centering
        \includegraphics[width=0.8\linewidth]{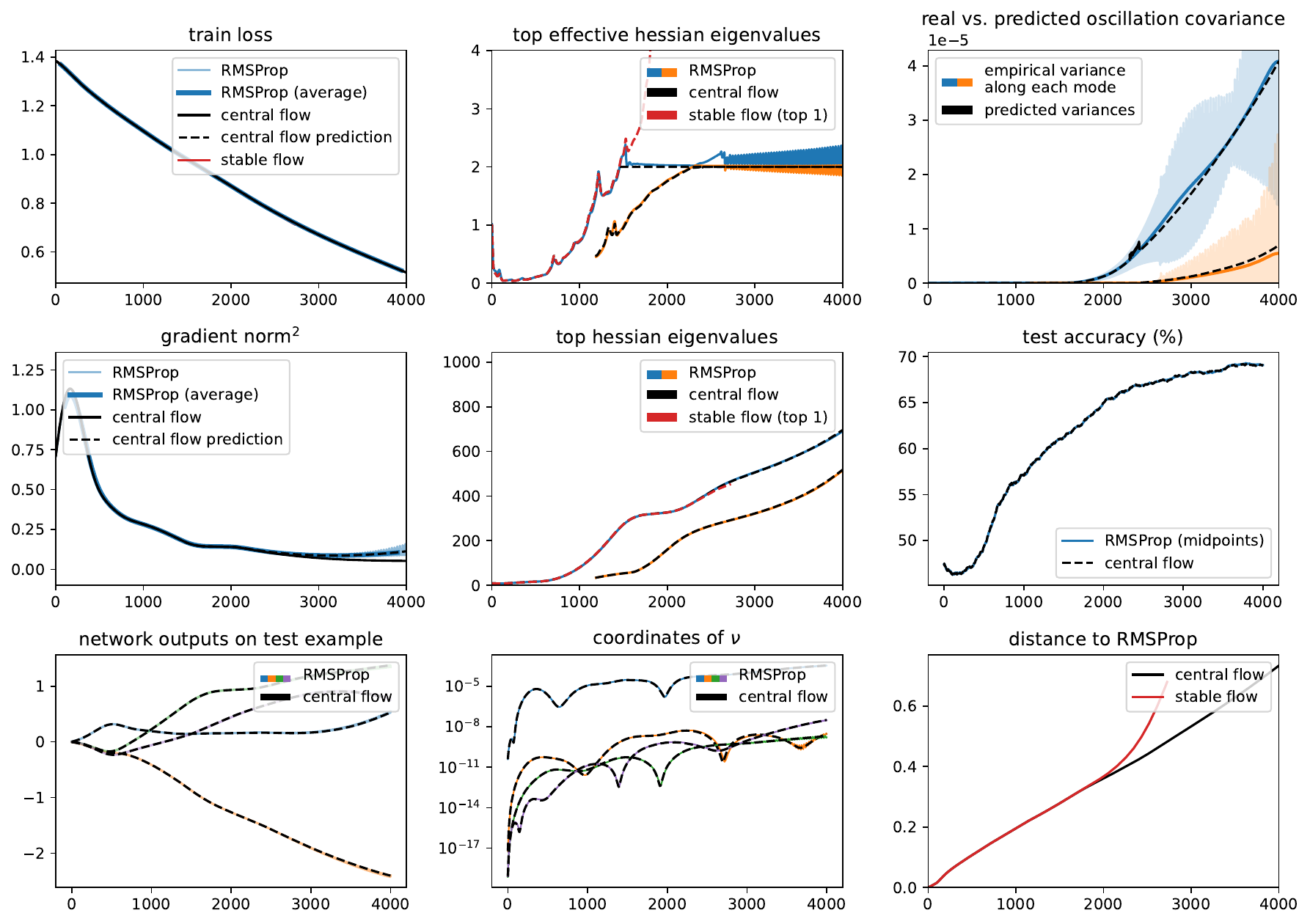}
        \caption{RMSProp central flow for a ResNet with CE loss, $\eta = $ 1e-05, $\beta_2 = $ 0.95, $\epsilon = $ 1e-08, and bias correction.}
        \label{fig:bulk-rmsprop:ce-resnet-0}
    \end{figure}
                
    \begin{figure}[H]
        \centering
        \includegraphics[width=0.8\linewidth]{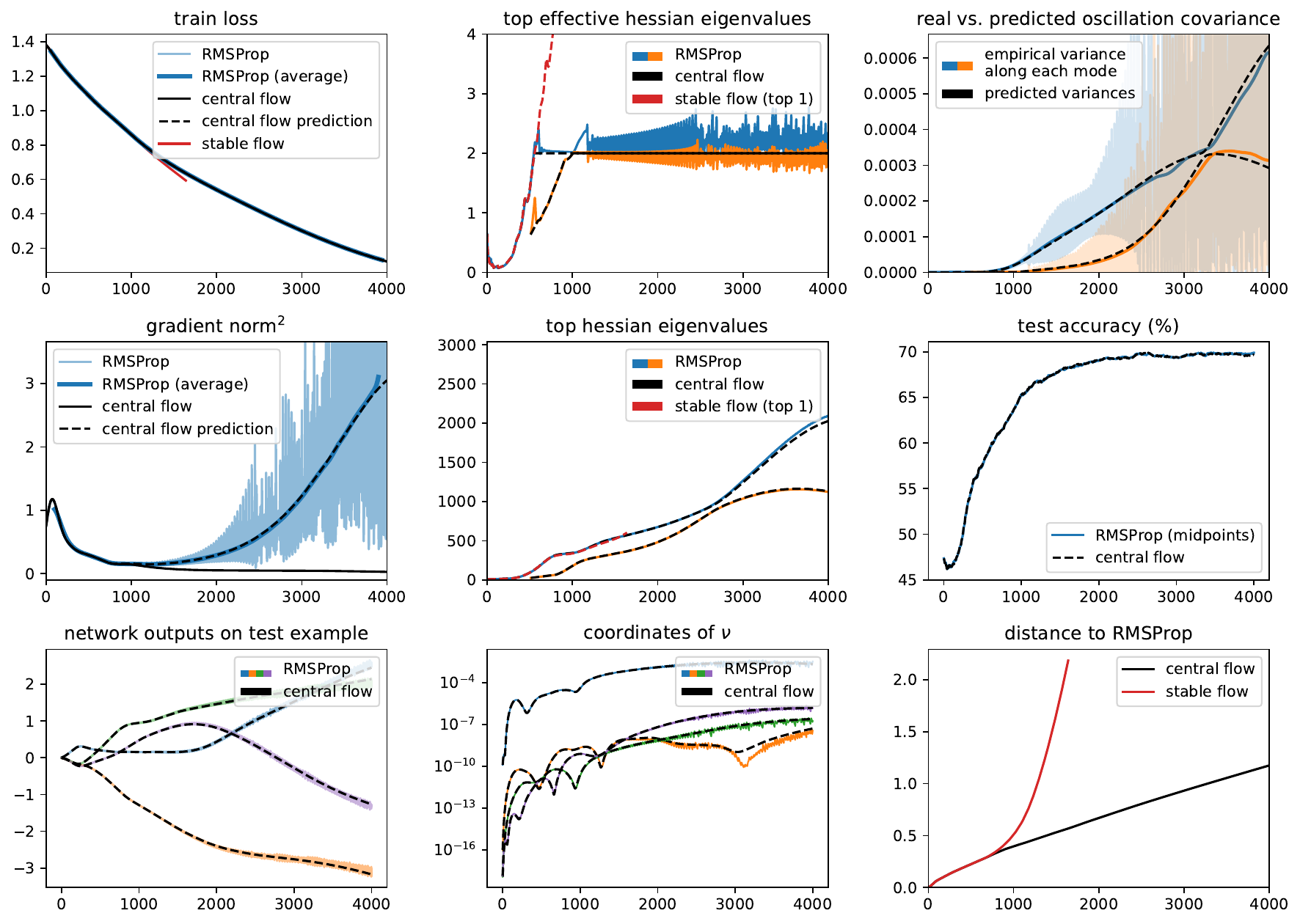}
        \caption{RMSProp central flow for a ResNet with CE loss, $\eta = $ 2e-05, $\beta_2 = $ 0.95, $\epsilon = $ 1e-08, and bias correction.}
        \label{fig:bulk-rmsprop:ce-resnet-1}
    \end{figure}
                
    \begin{figure}[H]
        \centering
        \includegraphics[width=0.8\linewidth]{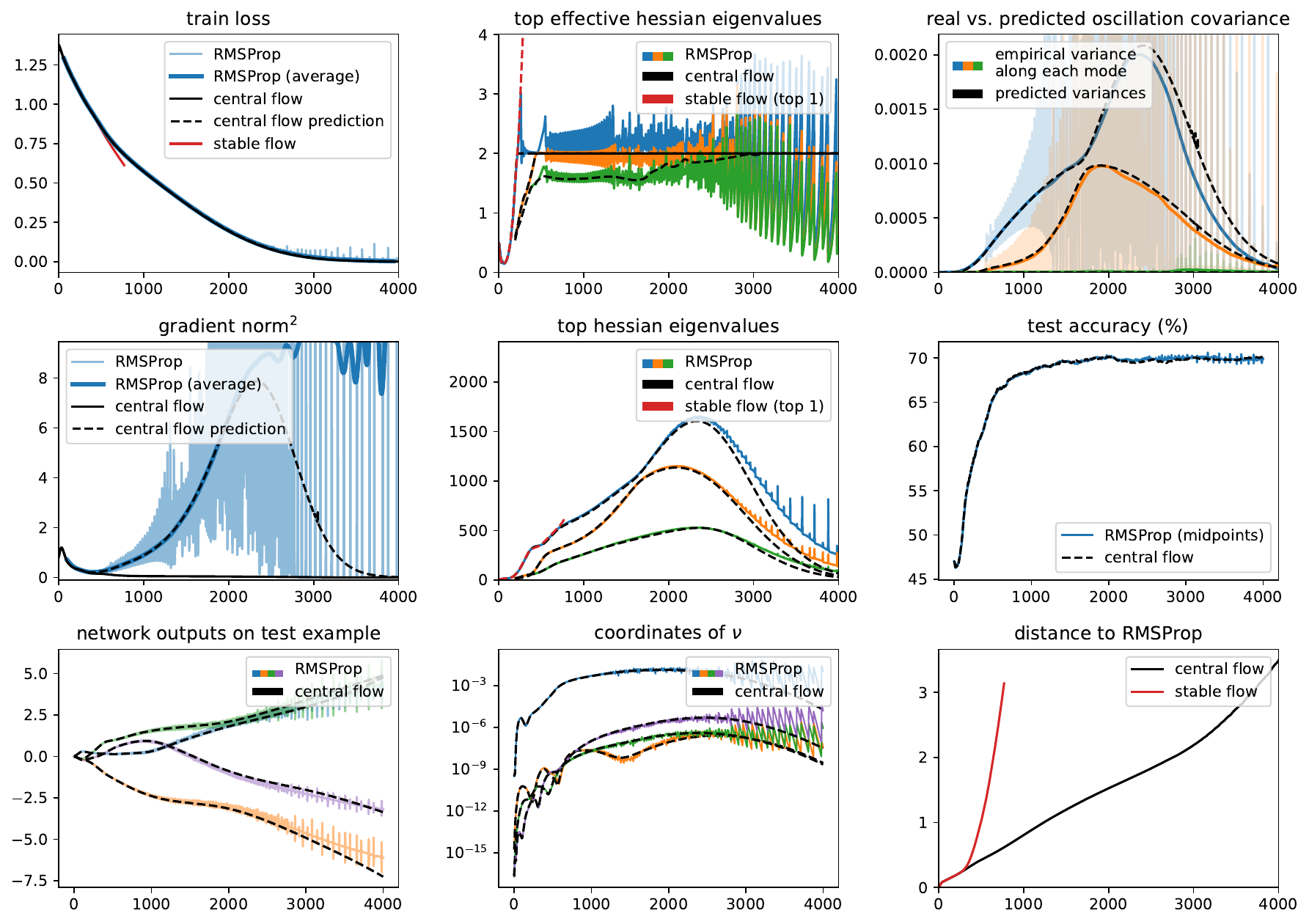}
        \caption{RMSProp central flow for a ResNet with CE loss, $\eta = $ 4e-05, $\beta_2 = $ 0.95, $\epsilon = $ 1e-08, and bias correction.}
        \label{fig:bulk-rmsprop:ce-resnet-2}
    \end{figure}
                
    \begin{figure}[H]
        \centering
        \includegraphics[width=0.8\linewidth]{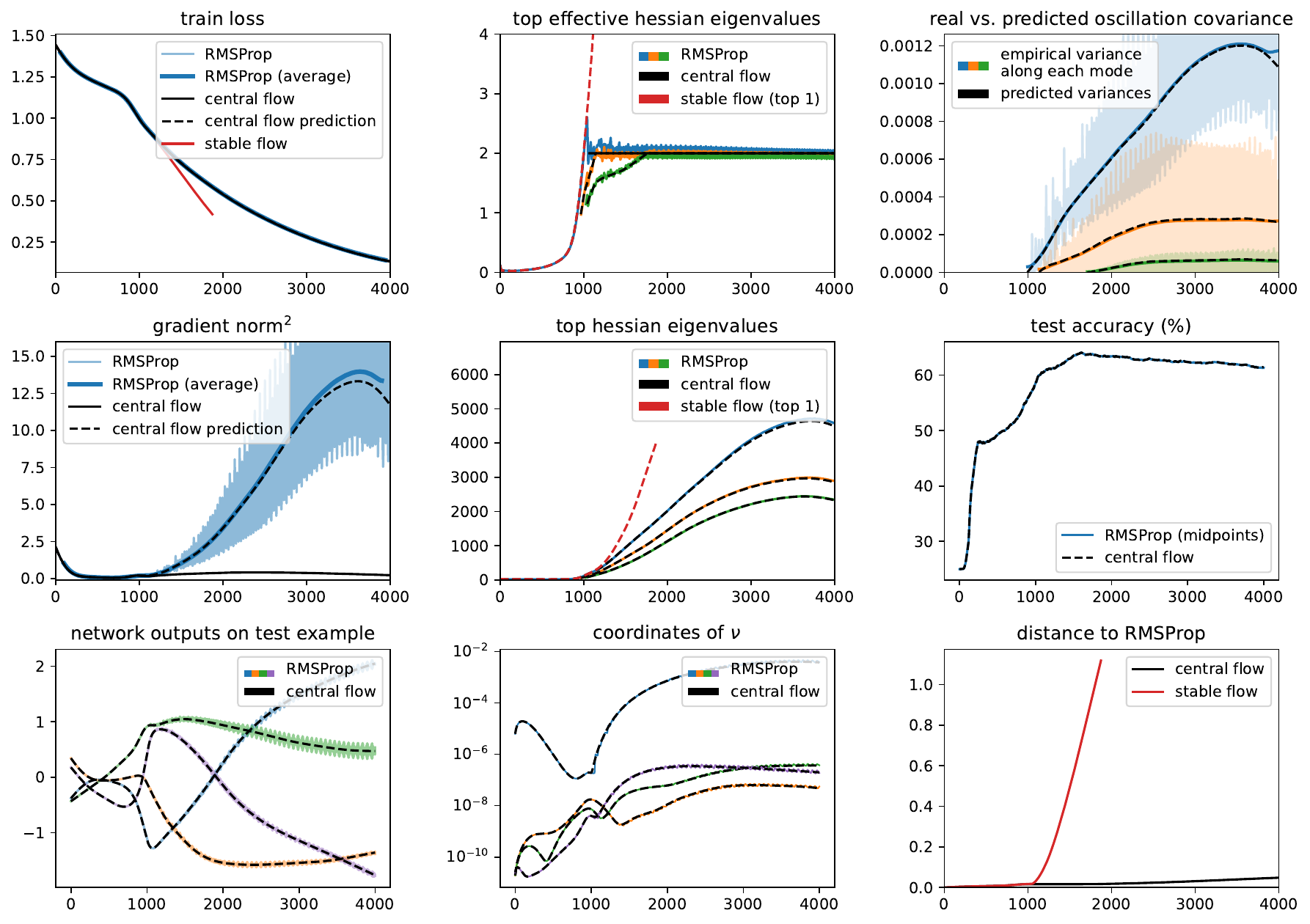}
        \caption{RMSProp central flow for a ViT with CE loss, $\eta = $ 5e-06, $\beta_2 = $ 0.95, $\epsilon = $ 1e-08, and bias correction.}
        \label{fig:bulk-rmsprop:ce-vit-0}
    \end{figure}
                
    \begin{figure}[H]
        \centering
        \includegraphics[width=0.8\linewidth]{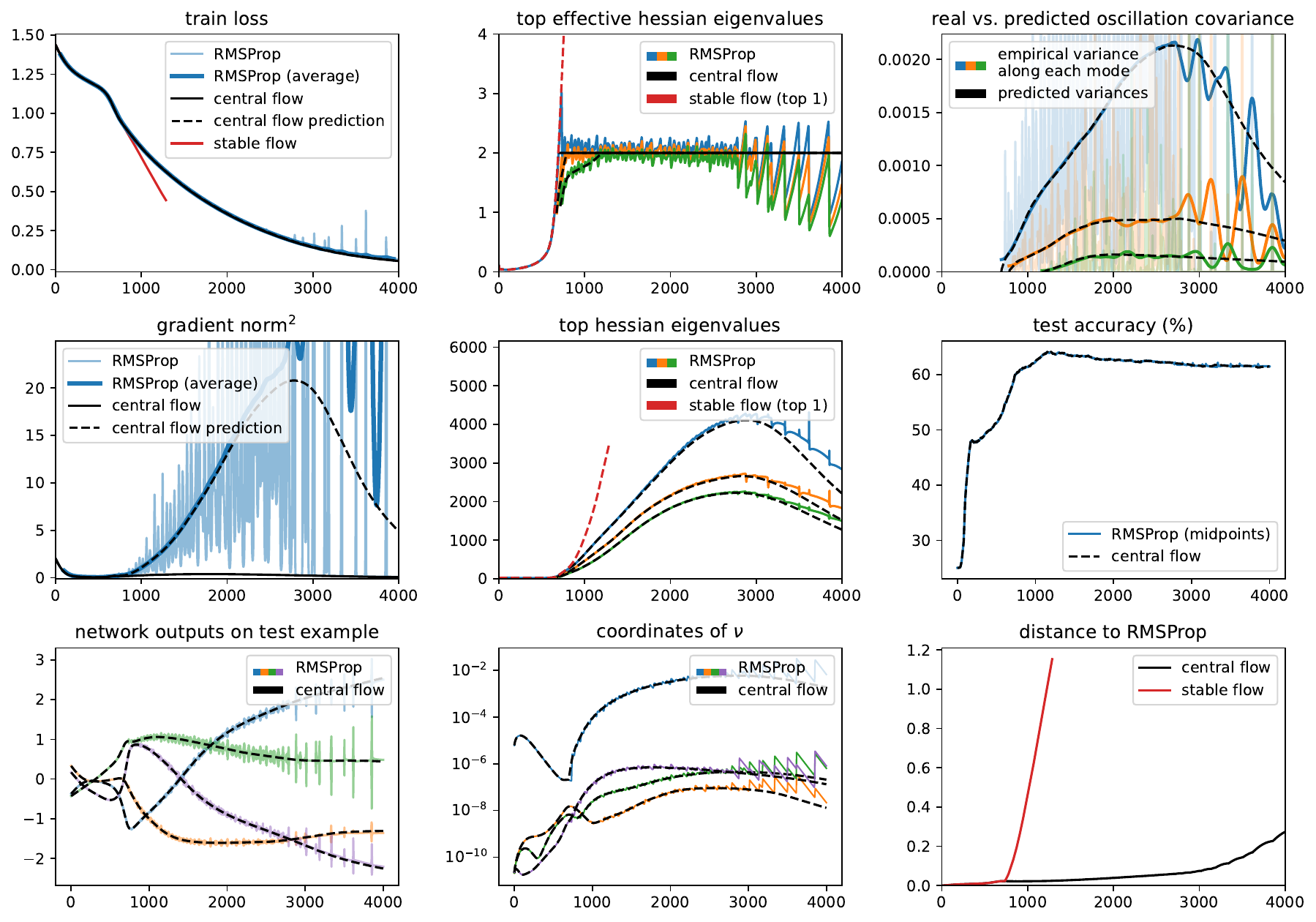}
        \caption{RMSProp central flow for a ViT with CE loss, $\eta = $ 7e-06, $\beta_2 = $ 0.95, $\epsilon = $ 1e-08, and bias correction.}
        \label{fig:bulk-rmsprop:ce-vit-1}
    \end{figure}
                
    \begin{figure}[H]
        \centering
        \includegraphics[width=0.8\linewidth]{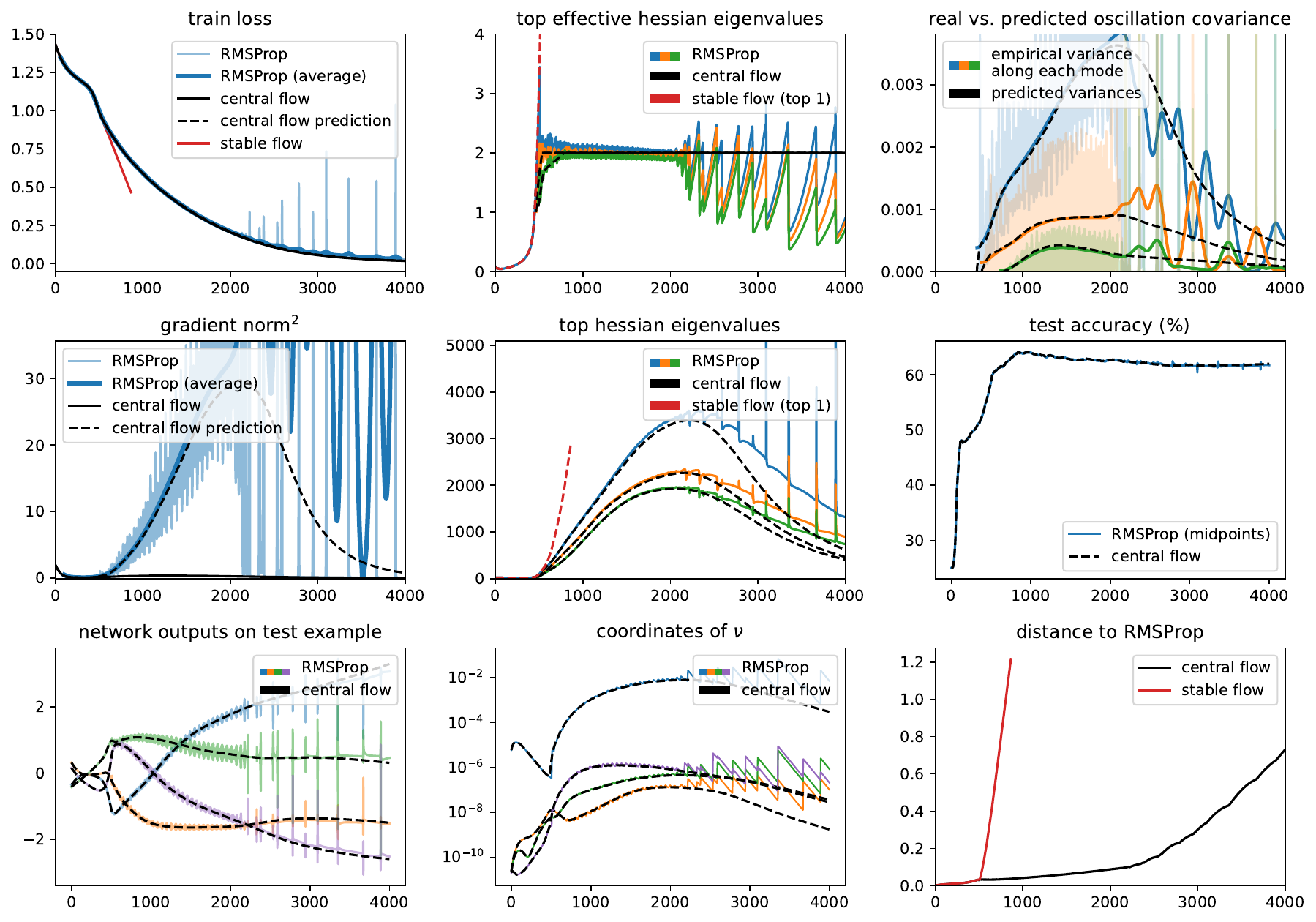}
        \caption{RMSProp central flow for a ViT with CE loss, $\eta = $ 1e-05, $\beta_2 = $ 0.95, $\epsilon = $ 1e-08, and bias correction.}
        \label{fig:bulk-rmsprop:ce-vit-2}
    \end{figure}
                
    \begin{figure}[H]
        \centering
        \includegraphics[width=0.8\linewidth]{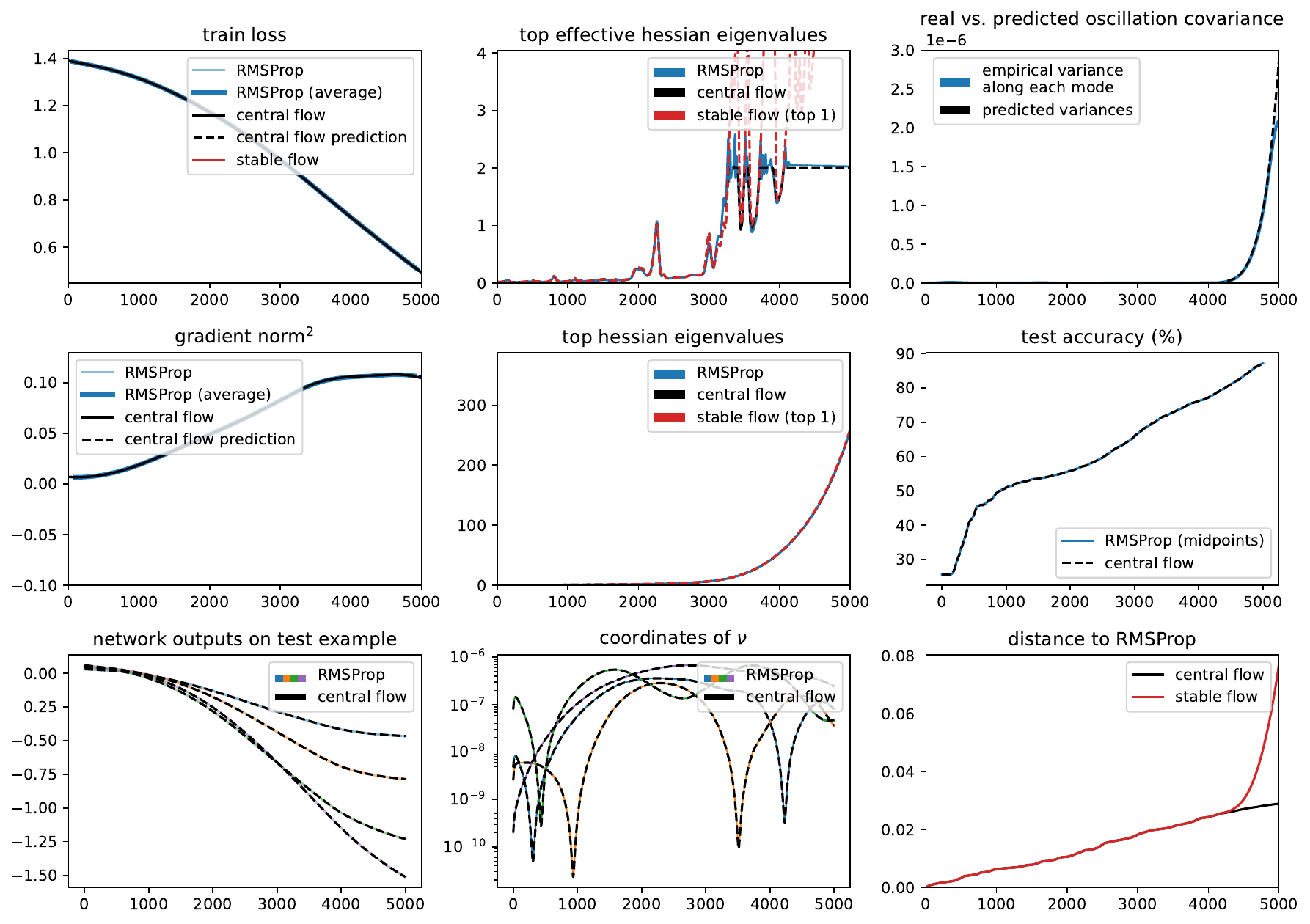}
        \caption{RMSProp central flow for a LSTM with CE loss, $\eta = $ 1e-05, $\beta_2 = $ 0.95, $\epsilon = $ 1e-08, and bias correction.}
        \label{fig:bulk-rmsprop:ce-lstm-0}
    \end{figure}
                
    \begin{figure}[H]
        \centering
        \includegraphics[width=0.8\linewidth]{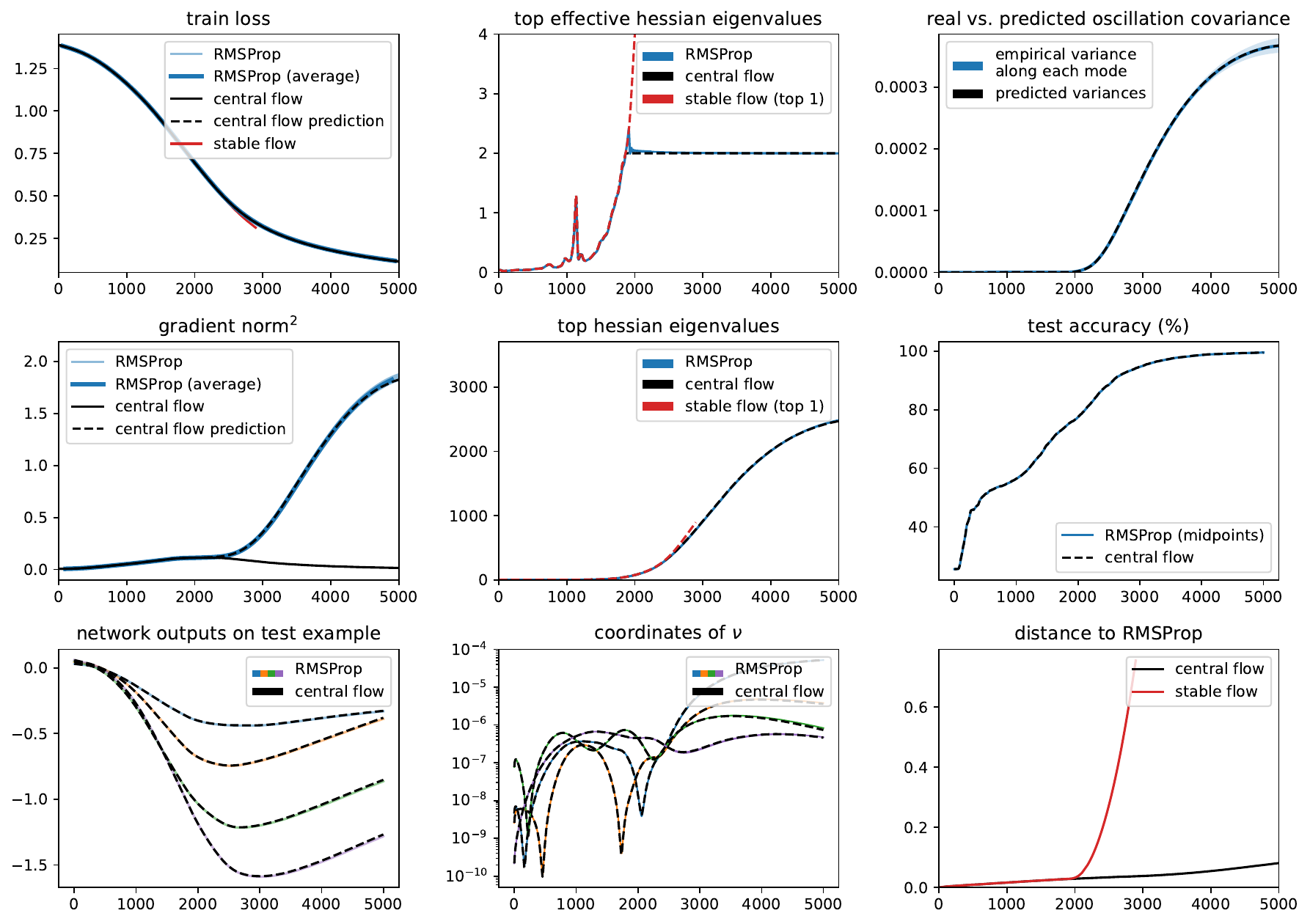}
        \caption{RMSProp central flow for a LSTM with CE loss, $\eta = $ 2e-05, $\beta_2 = $ 0.95, $\epsilon = $ 1e-08, and bias correction.}
        \label{fig:bulk-rmsprop:ce-lstm-1}
    \end{figure}
                
    \begin{figure}[H]
        \centering
        \includegraphics[width=0.8\linewidth]{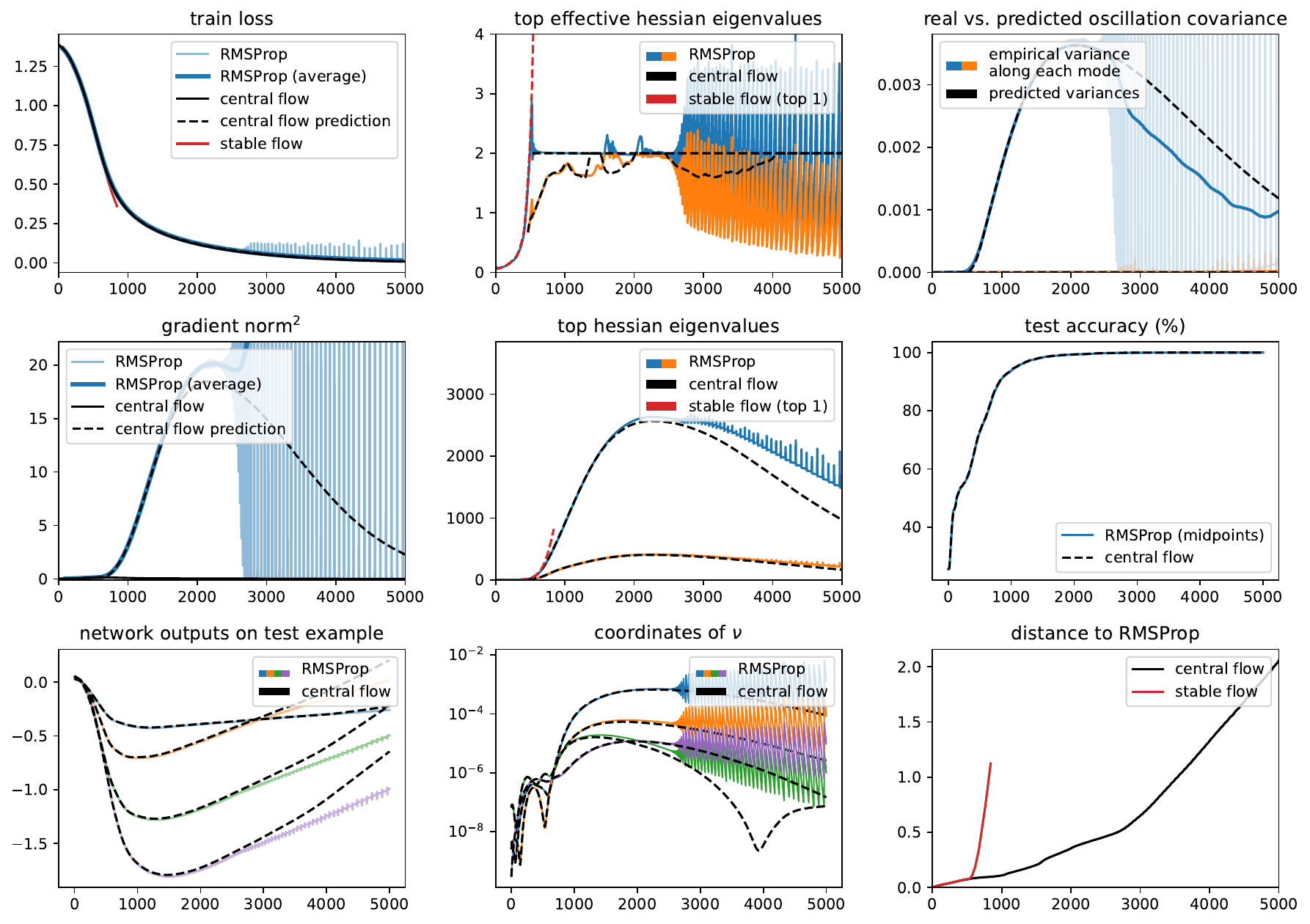}
        \caption{RMSProp central flow for a LSTM with CE loss, $\eta = $ 6e-05, $\beta_2 = $ 0.95, $\epsilon = $ 1e-08, and bias correction.}
        \label{fig:bulk-rmsprop:ce-lstm-2}
    \end{figure}
                
    \begin{figure}[H]
        \centering
        \includegraphics[width=0.8\linewidth]{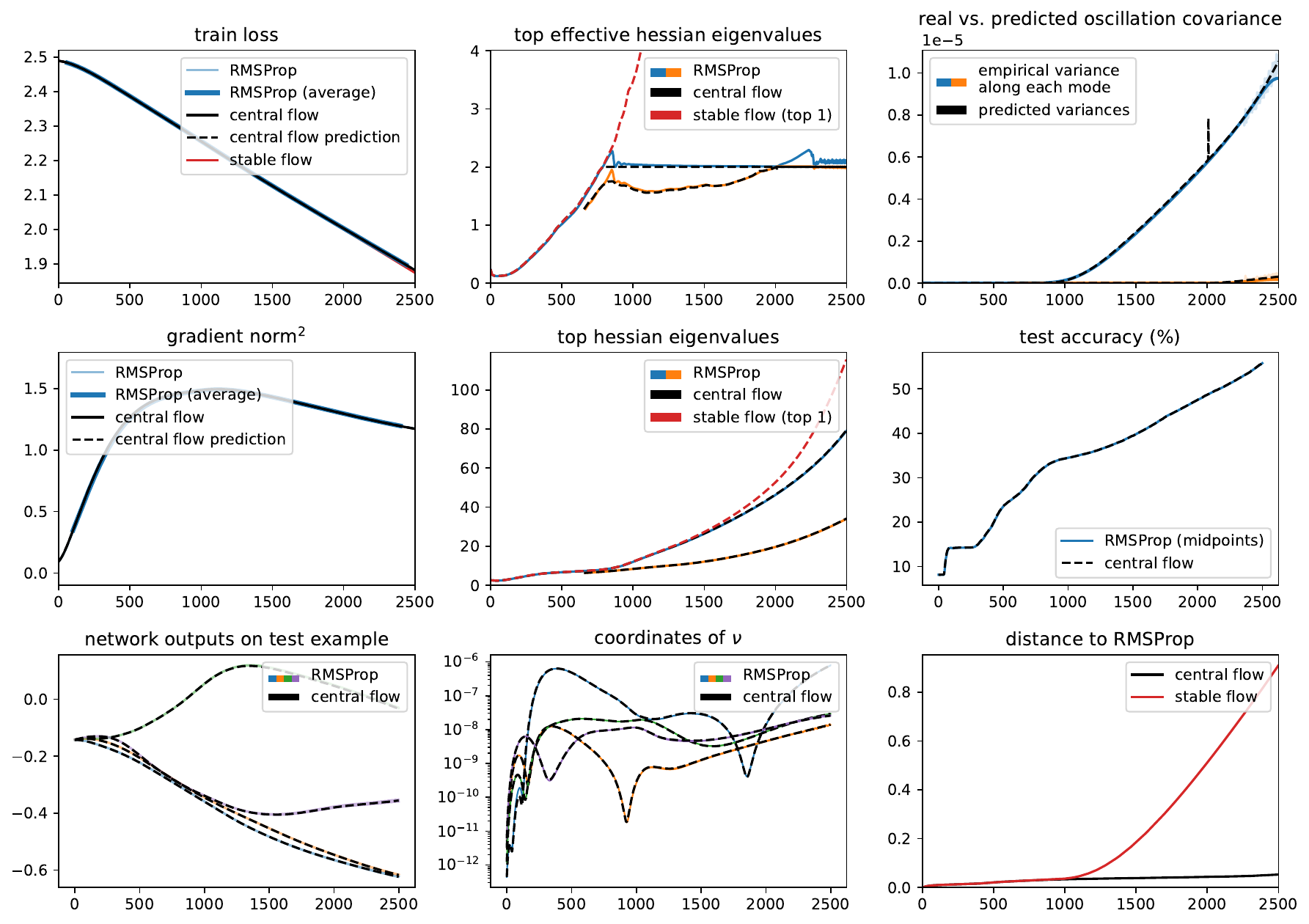}
        \caption{RMSProp central flow for a Transformer with CE loss, $\eta = $ 1e-05, $\beta_2 = $ 0.95, $\epsilon = $ 1e-08, and bias correction.}
        \label{fig:bulk-rmsprop:ce-transformer-0}
    \end{figure}
                
    \begin{figure}[H]
        \centering
        \includegraphics[width=0.8\linewidth]{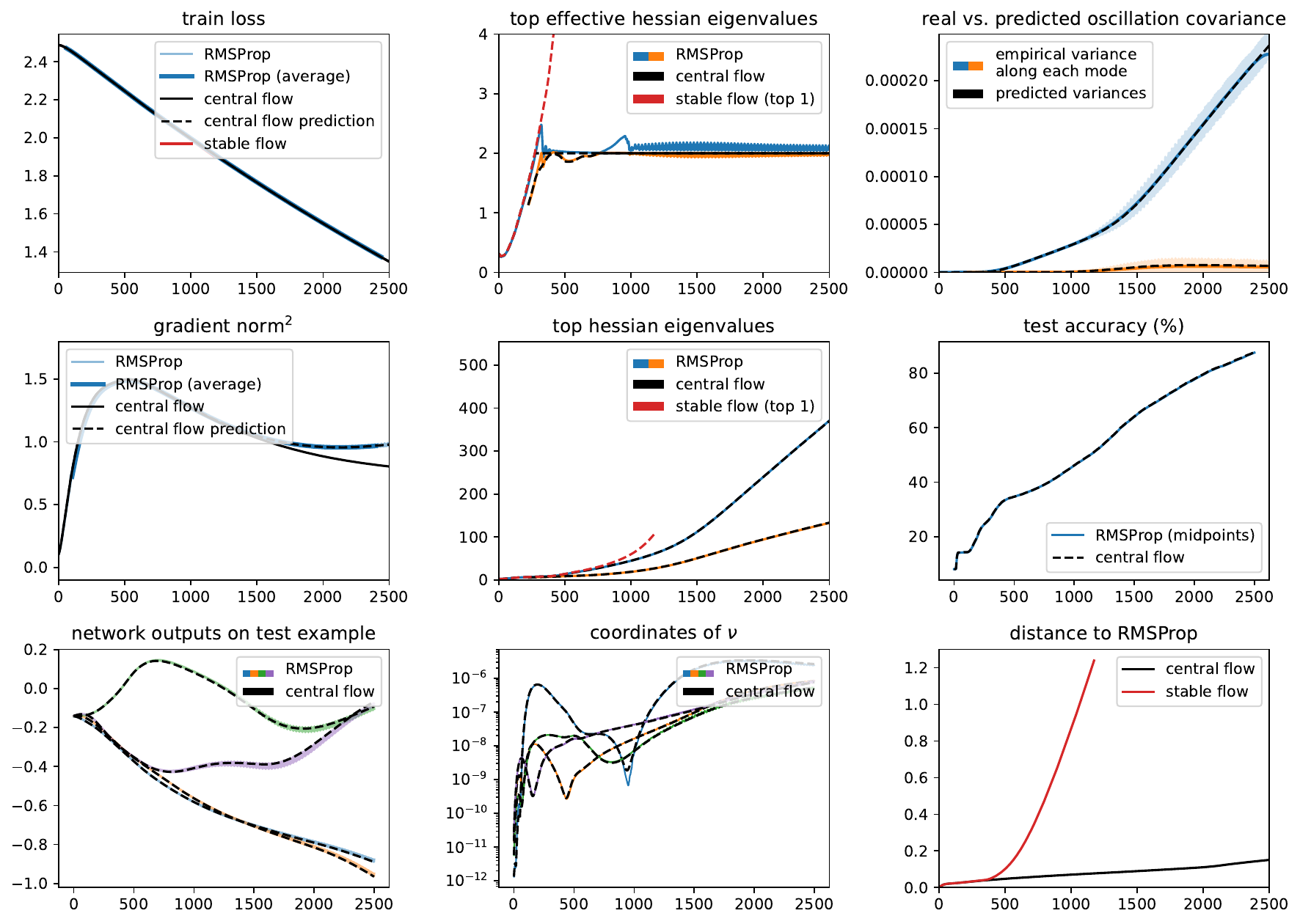}
        \caption{RMSProp central flow for a Transformer with CE loss, $\eta = $ 2e-05, $\beta_2 = $ 0.95, $\epsilon = $ 1e-08, and bias correction.}
        \label{fig:bulk-rmsprop:ce-transformer-1}
    \end{figure}
                
    \begin{figure}[H]
        \centering
        \includegraphics[width=0.8\linewidth]{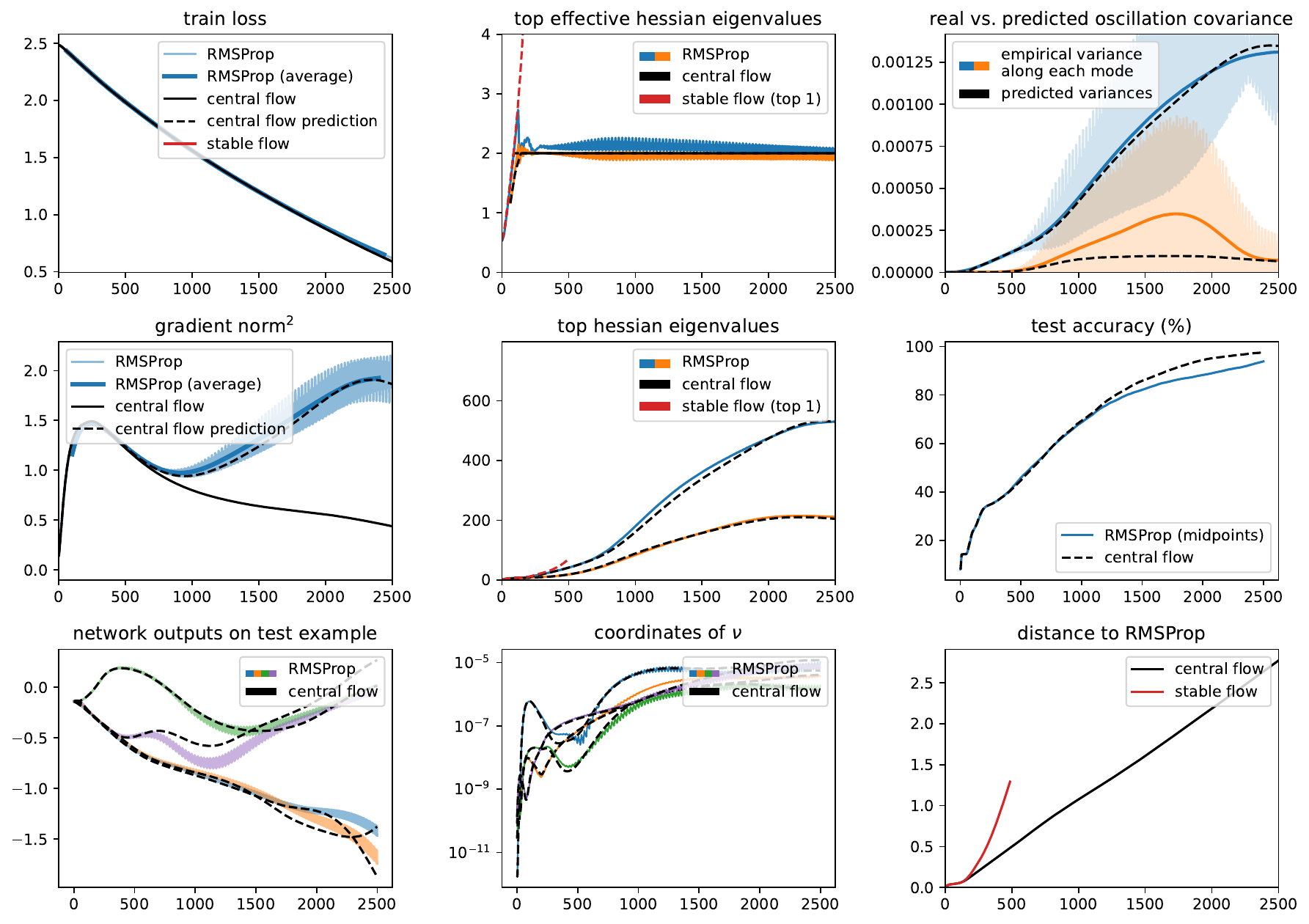}
        \caption{RMSProp central flow for a Transformer with CE loss, $\eta = $ 4e-05, $\beta_2 = $ 0.95, $\epsilon = $ 1e-08, and bias correction.}
        \label{fig:bulk-rmsprop:ce-transformer-2}
    \end{figure}
                
    \begin{figure}[H]
        \centering
        \includegraphics[width=0.8\linewidth]{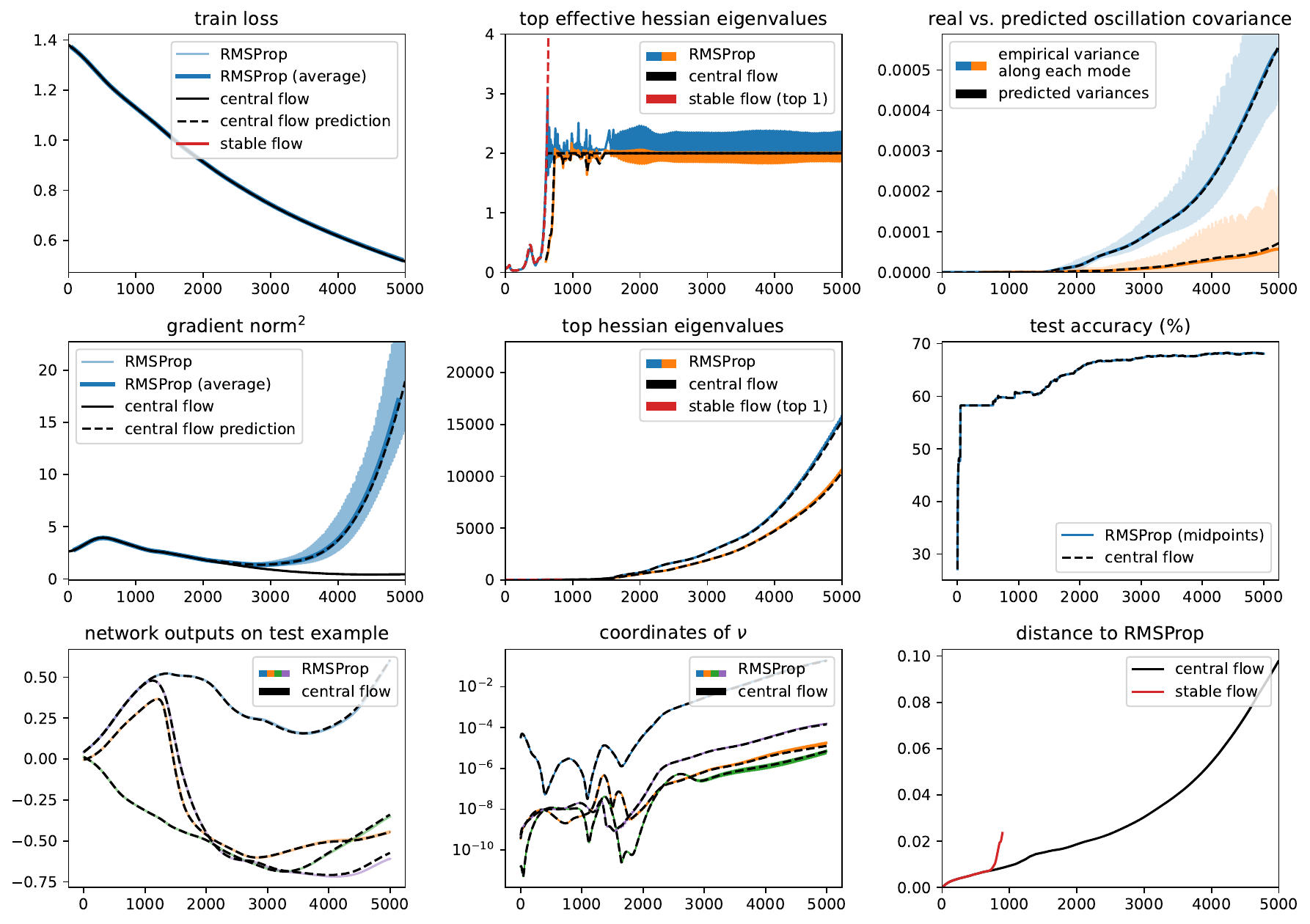}
        \caption{RMSProp central flow for a Mamba with CE loss, $\eta = $ 7e-06, $\beta_2 = $ 0.95, $\epsilon = $ 1e-08, and bias correction.}
        \label{fig:bulk-rmsprop:ce-mamba-0}
    \end{figure}
                
    \begin{figure}[H]
        \centering
        \includegraphics[width=0.8\linewidth]{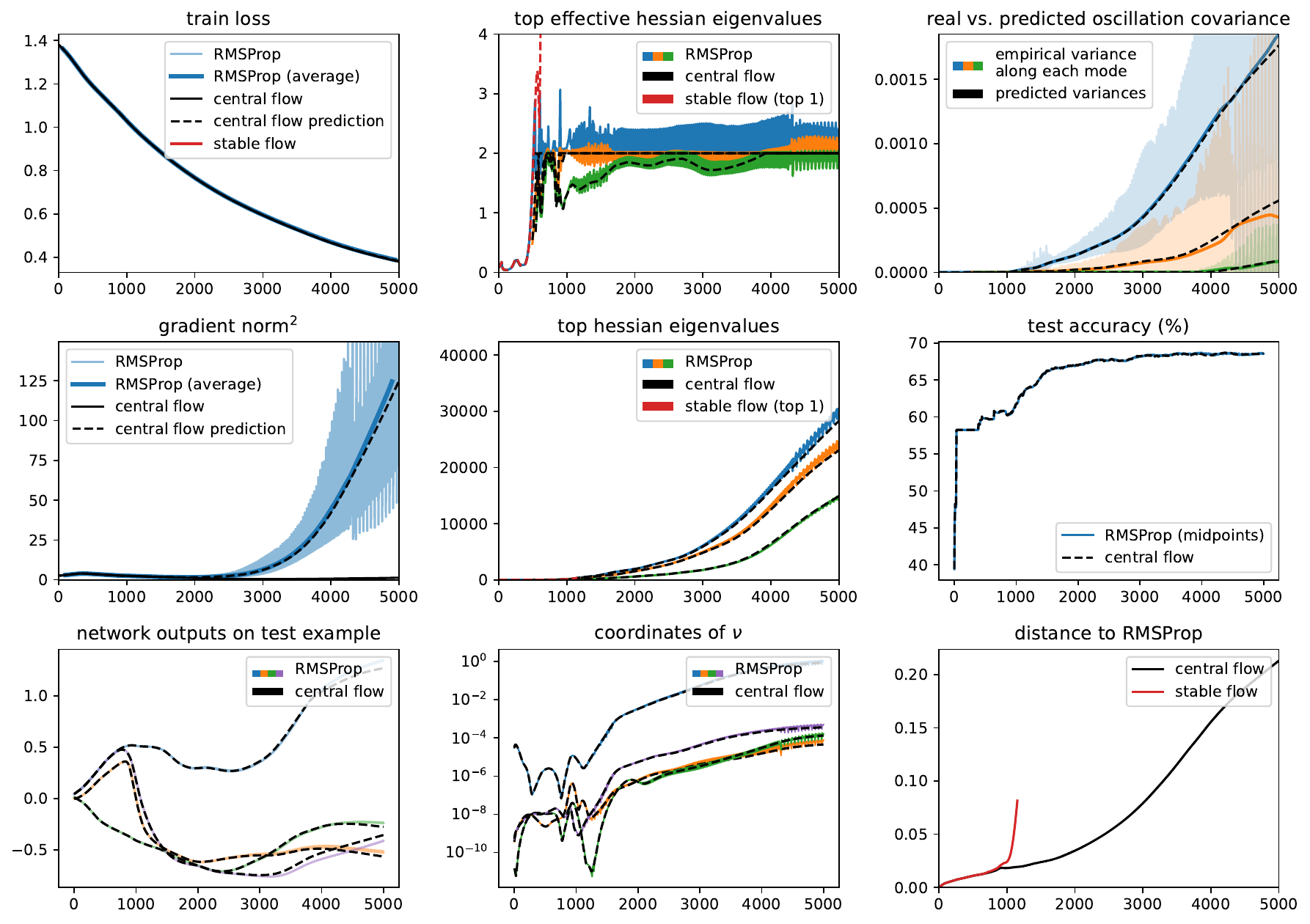}
        \caption{RMSProp central flow for a Mamba with CE loss, $\eta = $ 1e-05, $\beta_2 = $ 0.95, $\epsilon = $ 1e-08, and bias correction.}
        \label{fig:bulk-rmsprop:ce-mamba-1}
    \end{figure}
                
    \begin{figure}[H]
        \centering
        \includegraphics[width=0.8\linewidth]{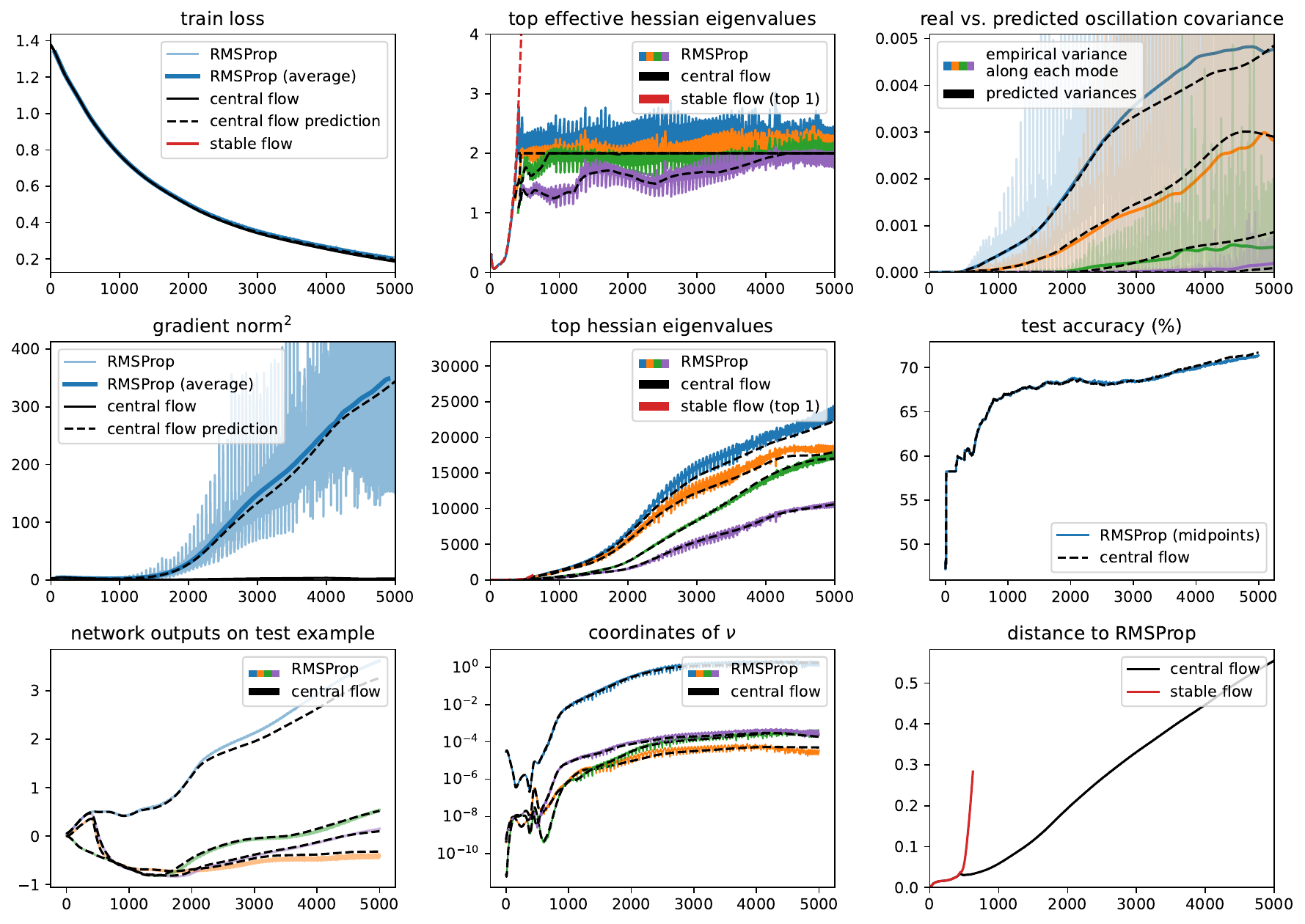}
        \caption{RMSProp central flow for a Mamba with CE loss, $\eta = $ 2e-05, $\beta_2 = $ 0.95, $\epsilon = $ 1e-08, and bias correction.}
        \label{fig:bulk-rmsprop:ce-mamba-2}
    \end{figure}
                \end{specialfigures}